%% file: csthesis21.tex
\documentclass[edeposit,tocnosub,noragright,centerchapter,fullpagesingle,12pt]{uiuc_csthesis21}

\makeatletter

\usepackage{setspace}  
\usepackage[numbers, sort]{natbib}  
\usepackage{url}  

\usepackage{lscape}  
\def\fillandplacepagenumber{
	\par
	\pagestyle{empty}
	\vbox to 0pt{\vss}
	\vfill
	\vbox to 0pt{
		\baselineskip 0pt
		\hbox to \linewidth{\hss}
		\baselineskip\footskip
		\hbox to \linewidth{\hfil\thepage\hfil}\vss
	}
}

\usepackage{graphicx}  
\usepackage{caption}
\usepackage{booktabs}  
\usepackage{amsfonts}
\usepackage{amsmath}
\usepackage{amssymb}
\usepackage{amsthm}

\theoremstyle{definition}
\newtheorem{definition}{Definition}[chapter]
\newtheorem{lemma}{Lemma}[chapter]
\newtheorem{theorem}{Theorem}[chapter]
\newtheorem{corollary}{Corollary}[chapter]

\newtheorem{remark}{Remark}[chapter]
\newtheorem{assumption}{Assumption}[chapter]
\newtheorem{prop}{Proposition}[chapter]
\newtheorem{example}{Example}[chapter]

\usepackage{listings}  
\usepackage{algorithm,algorithmic}

\input{macros}

\phdthesis

\title{Classification Performance Metric Elicitation and its Applications}
\author{Gaurush Hiranandani}
\department{Computer Science}
\degreeyear{2021}

\advisor{Oluwasanmi Koyejo}

\committee{Associate Professor Oluwasanmi Koyejo, Chair\\
        Professor Srikant Rayadurgam\\Professor Paris Smaragdis\\Associate Professor Shivani Agarwal, University of Pennsylvania} 

\begin{document}

%

%
\maketitle

\parindent 1em%

\frontmatter

%
\input{abstract}

%
\input{dedication}

%
\input{ack}

%
\tableofcontents
\addtocontents{toc}{\vspace{-30.0pt}}
\mainmatter

%

\input{introduction}
\input{me}
\input{binary/header}
\input{multiclass/header}
\input{fair/header}
\input{quadratic/header}
\input{blackbox/header}
\input{practical/header}

\input{conclusions}






%

\appendix

\input{supp_main}

%
\bibliographystyle{IEEE_ECE}
\bibliography{thesisrefs}  


\backmatter

\end{document}

%% file: macros.tex
\usepackage{graphicx}
\usepackage{subfigure}
\usepackage{caption}
\usepackage{diagbox}

\usepackage{amsmath}
\usepackage{dsfont}
\usepackage{amssymb,commath}

\usepackage{bm,xfrac,xcolor}
\usepackage{nccmath}
\usepackage[toc,page,header]{appendix}

\usepackage{bbm}
\usepackage{enumitem}
\usepackage{refcount}
\usepackage{mathtools}
\usepackage{stackengine}
\usepackage{multirow}
\usepackage{makecell}

\usepackage{comment}
\usepackage{mathrsfs}

\makeatletter
\newcommand*{\addFileDependency}[1]{
  \typeout{(#1)}
  \@addtofilelist{#1}
  \IfFileExists{#1}{}{\typeout{No file #1.}}
}
\makeatother


\newcommand\barbelow[1]{\stackunder[1.2pt]{$#1$}{\rule{1.5ex}{.1ex}}}

\DeclarePairedDelimiter{\ceil}{\lceil}{\rceil}
\def\myfnt{\ifx\protect\@typeset@protect\expandafter\footnote\else\expandafter\@gobble\fi}
\makeatother

\makeatletter
\def\BState{\State\hskip-\ALG@thistlm}
\makeatother

\usepackage{placeins,afterpage}

\makeatletter
\newcommand*\bigcdot{\mathpalette\bigcdot@{.5}}
\newcommand*\bigcdot@[2]{\mathbin{\vcenter{\hbox{\scalebox{#2}{$\m@th#1\bullet$}}}}}
\makeatother

\usepackage{tikz}
\tikzset{
  c/.style={every coordinate/.try}
}

\newcommand\numberthis{\addtocounter{equation}{1}\tag{\theequation}}

\mathchardef\mhyphen="2D

\newcommand{\offdiag}{\mathit{off\mhyphen diag}}

\newcommand{\diag}{\mathit{diag}}

\newcommand{\tupr}{\rmbf^{1:m}}
\newcommand{\tuprhat}{\hat{\rmbf}^{1:m}}

\newcommand{\TP}{\emph{TP}}
\newcommand{\TN}{\emph{TN}}

\newcommand{\TPs}{\emph{TP }}

\newcommand{\fair}{\textrm{\textup{fair}}}
\newcommand{\quadr}{\textrm{\textup{quad}}}
\newcommand{\qmean}{\textrm{\textup{qmean}}}
\newcommand{\cov}{\textrm{\textup{cov}}}
\newcommand{\eo}{\textrm{\textup{EO}}}

\newcommand{\Ccal}{{\cal C}}
\newcommand{\Dcal}{{\cal D}}
\newcommand{\Ecal}{{\cal E}}

\newcommand{\Hcal}{{\cal H}}

\newcommand{\Lcal}{{\cal L}}
\newcommand{\Mcal}{{\cal M}}

\newcommand{\Rcal}{{\cal R}}
\newcommand{\Scal}{{\cal S}}

\newcommand{\Xcal}{{\cal X}}
\newcommand{\Ycal}{{\cal Y}}
\newcommand{\Zcal}{{\cal Z}}

\newcommand{\Bmbb}{\mathbb{B}}

\newcommand{\Pmbb}{\mathbb{P}}

\newcommand{\Rmbb}{\mathbb{R}}

\newcommand{\Umbb}{\mathbb{U}}

\newcommand{\Zmbb}{\mathbb{Z}}

\newcommand{\thetambf}{\bm{\theta}}

\newcommand{\smbfbar}{\oline{\mathbf{s}}}
\newcommand{\cmbfbar}{\oline{\mathbf{c}}}

\newcommand{\dmbfbar}{\oline{\mathbf{d}}}

\newcommand{\cmbfhat}{\hat{\mathbf{c}}}
\newcommand{\rmbfhat}{\hat{\mathbf{r}}}
\newcommand{\ambfhat}{\hat{\mathbf{a}}}
\newcommand{\fmbfhat}{\hat{\mathbf{f}}}
\newcommand{\dmbfhat}{\hat{\mathbf{d}}}
\newcommand{\bmbfhat}{\hat{\mathbf{b}}}
\newcommand{\Bmbfhat}{\hat{\mathbf{B}}}
\newcommand{\lambdahat}{\hat{\lambda}}

\newcommand{\ambfbar}{\oline{\mathbf{a}}}

\newcommand{\fmbfbar}{\oline{\mathbf{f}}}
\newcommand{\bmbfbar}{\oline{\mathbf{b}}}
\newcommand{\Bmbfbar}{\oline{\mathbf{B}}}
\newcommand{\lambdabar}{\oline{\lambda}}

\newcommand{\Ambf}{\mathbf{A}}
\newcommand{\ambf}{\mathbf{a}}
\newcommand{\Bmbf}{\mathbf{B}}
\newcommand{\bmbf}{\mathbf{b}}
\newcommand{\Cmbf}{\mathbf{C}}
\newcommand{\cmbf}{\mathbf{c}}

\newcommand{\dmbf}{\mathbf{d}}

\newcommand{\embf}{\mathbf{e}}

\newcommand{\fmbf}{\mathbf{f}}

\newcommand{\Imbf}{\mathbf{I}}

\newcommand{\mmbf}{\mathbf{m}}

\newcommand{\ombf}{\mathbf{o}}

\newcommand{\Qmbf}{\mathbf{Q}}

\newcommand{\Rmbf}{\mathbf{R}}
\newcommand{\rmbf}{\mathbf{r}}

\newcommand{\smbf}{\mathbf{s}}

\newcommand{\tmbf}{\mathbf{t}}

\newcommand{\umbf}{\mathbf{u}}

\newcommand{\vmbf}{\mathbf{v}}
\newcommand{\Wmbf}{\mathbf{W}}
\newcommand{\wmbf}{\mathbf{w}}

\newcommand{\xmbf}{\mathbf{x}}

\newcommand{\ymbf}{\mathbf{y}}

\newcommand{\zmbf}{\mathbf{z}}

\newcommand{\oline}[1]{\mkern 1.5mu\overline{\mkern-1.5mu#1}}

\newcommand{\abar}{\oline{a}}

\newcommand{\bbar}{\oline{b}}
\newcommand{\Cbar}{\oline{C}}

\newcommand{\dbar}{\oline{d}}

\renewcommand{\hbar}{\oline{h}}

\newcommand{\Sbar}{\oline{S}}
\newcommand{\sbar}{\oline{s}}

\newcommand{\ahat}{\hat{a}}

\newcommand{\hhat}{\hat{h}}

\newcommand{\ophi}{\overline{\phi}}

\newcommand\inner[2]{\langle #1, #2 \rangle}

\newcommand{\Cii}{C_{ii}}
\newcommand{\Cij}{C_{ij}}
\newcommand{\1}{{\mathbf 1}}

\newcommand{\hphi}{\hat{\phi}}
\newcommand{\hvarphi}{\hat{\varphi}}

\newcommand{\hPsi}{\hat{\Psi}}

\newcommand{\bphi}{\oline{\phi}}
\newcommand{\bvarphi}{\oline{\varphi}}

\newcommand{\bPsi}{\oline{\Psi}}

\newcommand{\scl}[1]{\scalebox{0.9}{#1}}

\newcommand{\alphambf}{\bm{\alpha}}
\newcommand{\taumbf}{\bm{\tau}}

\newcommand{\gammambf}{\bm{\gamma}}

\newcommand{\etambfbar}{\oline{\bm{\eta}}}

\newcommand{\etabar}{\oline{\eta}}

\newcommand{\bmmbf}{\oline{\mathbf{m}}}
\newcommand{\bmone}{\oline{m}_{11}}
\newcommand{\bmzero}{\oline{m}_{00}}

\newcommand{\bell}{\oline{\ell}}
\newcommand{\btau}{\oline{\tau}}

\newcommand{\hmmbf}{\hat{\mathbf{m}}}

\newcommand{\hpone}{\hat{p}_{11}}
\newcommand{\hpzero}{\hat{p}_{00}}
\newcommand{\hqone}{\hat{q}_{11}}
\newcommand{\hqzero}{\hat{q}_{00}}
\newcommand{\hqnot}{\hat{q}_{0}}

\newcommand{\sphi}{\phi^*}
\newcommand{\smmbf}{\mathbf{m}^*}
\newcommand{\smone}{{m_{11}^*}}
\newcommand{\smzero}{{m_{00}^*}}
\newcommand{\spone}{{p_{11}^*}}
\newcommand{\spzero}{{p_{00}^*}}
\newcommand{\sqone}{{q_{11}^*}}
\newcommand{\sqzero}{{q_{00}^*}}
\newcommand{\sqnot}{{q_{0}^*}}

\newcommand{\tmmbf}{\barbelow{\mathbf{m}}}
\newcommand{\tsmbf}{\barbelow{\mathbf{s}}}

\newcommand{\tell}{\barbelow{\ell}}
\newcommand{\ttau}{\barbelow{\tau}}

\newcommand{\pphi}{{\phi'}}

\newcommand{\ppone}{{p}_{11}'}
\newcommand{\ppzero}{{p}_{00}'}
\newcommand{\pqone}{{q}_{11}'}
\newcommand{\pqzero}{{q}_{00}'}
\newcommand{\pqnot}{{q}_{0}'}

\newcommand{\ppphi}{{\phi''}}

\newcommand{\pppone}{{p}_{11}''}
\newcommand{\pppzero}{{p}_{00}''}
\newcommand{\ppqone}{{q}_{11}''}
\newcommand{\ppqzero}{{q}_{00}''}
\newcommand{\ppqnot}{{q}_{0}''}

\DeclareMathOperator*{\argmin}{argmin}
\DeclareMathOperator*{\argmax}{argmax}

\renewcommand{\P}{\mathbb{P}}

\newcommand{\baligned}     {\begin{aligned}}
	\newcommand{\ealigned}     {\end{aligned}}
\newcommand{\barray}       {\begin{array}}
	\newcommand{\earray}       {\end{array}}
\newcommand{\bbmatrix}     {\begin{bmatrix}}
	\newcommand{\ebmatrix}     {\end{bmatrix}}
\newcommand{\bcases}       {\begin{cases}}
	\newcommand{\ecases}       {\end{cases}}
\newcommand{\bcenter}      {\begin{center}}
	\newcommand{\ecenter}      {\end{center}}
\newcommand{\bcolumn}      {\begin{column}}
	\newcommand{\ecolumn}      {\end{column}}
\newcommand{\bcolumns}     {\begin{columns}}
	\newcommand{\ecolumns}     {\end{columns}}
\newcommand{\benumerate}   {\begin{enumerate}}
	\newcommand{\eenumerate}   {\end{enumerate}}
\newcommand{\bequation}    {\begin{equation}}
	\newcommand{\eequation}    {\end{equation}}
\newcommand{\bequationn}   {\begin{equation*}}
	\newcommand{\eequationn}   {\end{equation*}}
\newcommand{\bfigure}      {\begin{figure}}
	\newcommand{\efigure}      {\end{figure}}
\newcommand{\bflushright}  {\begin{flushright}}
	\newcommand{\eflushright}  {\end{flushright}}
\newcommand{\bitemize}     {\begin{itemize}}
	\newcommand{\eitemize}     {\end{itemize}}
\newcommand{\bpmatrix}     {\begin{pmatrix}}
	\newcommand{\epmatrix}     {\end{pmatrix}}
\newcommand{\bsubequations}{\begin{subequations}}
	\newcommand{\esubequations}{\end{subequations}}
\newcommand{\btable}       {\begin{table}}
	\newcommand{\etable}       {\end{table}}
\newcommand{\btabular}     {\begin{tabular}}
	\newcommand{\etabular}     {\end{tabular}}
\newcommand{\bvmatrix}     {\begin{vmatrix}}
	\newcommand{\evmatrix}     {\end{vmatrix}}

\newcommand{\bequali}      {\bsubequations\begin{align}}
	\newcommand{\eequali}      {\end{align}\esubequations}

\newcommand{\balgorithm}  {\begin{algorithm}}
	\newcommand{\ealgorithm}  {\end{algorithm}}
\newcommand{\balgorithmic}{\begin{algorithmic}}
	\newcommand{\ealgorithmic}{\end{algorithmic}}
\newcommand{\bassumption} {\begin{assumption}}
	\newcommand{\eassumption} {\end{assumption}}
\newcommand{\bcorollary}  {\begin{corollary}}
	\newcommand{\ecorollary}  {\end{corollary}}
\newcommand{\bdefinition} {\begin{definition}}
	\newcommand{\edefinition} {\end{definition}}
\newcommand{\bexample}    {\begin{example}}
	\newcommand{\eexample}    {\end{example}}
\newcommand{\bprop}    {\begin{prop}}
	\newcommand{\eprop}    {\end{prop}}
\newcommand{\blemma}      {\begin{lemma}}
	\newcommand{\elemma}      {\end{lemma}}
\newcommand{\bproblem}    {\begin{problem}}
	\newcommand{\eproblem}    {\end{problem}}
\newcommand{\bproof}      {\begin{proof}}
	\newcommand{\eproof}      {\end{proof}}
\newcommand{\bremark}     {\begin{remark}}
	\newcommand{\eremark}     {\end{remark}}
\newcommand{\btheorem}    {\begin{theorem}}
	\newcommand{\etheorem}    {\end{theorem}}

\newcommand{\prodRcal}{\Rcal^{1:m}}

\counterwithin{algorithm}{chapter}

\usepackage[customcolors]{hf-tikz}
\usepackage{adjustbox}
\usepackage[normalem]{ulem}
\usepackage{cancel}
\usetikzlibrary{calc,shapes.geometric}


%% file: abstract.tex
\begin{abstract}
Given a learning problem with real-world tradeoffs, which cost function should the model be trained to optimize? This is the metric selection problem in machine learning. Despite its practical interest, there is limited formal guidance on how to select metrics for machine learning applications. This thesis outlines metric elicitation as a principled framework for selecting the performance metric that best reflects implicit user preferences. Once specified, the evaluation metric can be used to compare and train models. 

In this manuscript, we formalize the problem of Metric Elicitation and devise novel strategies for eliciting classification performance metrics using pairwise preference feedback over classifiers. Specifically, we provide novel strategies for eliciting linear and linear-fractional metrics for \emph{binary} and \emph{multiclass} classification problems, which are then extended to a framework that elicits \emph{group-fair} performance metrics in the presence of multiple sensitive groups. All the elicitation strategies that we discuss are robust to both finite sample and feedback noise, thus are useful in practice for real-world applications.

Using the tools and the geometric characterizations of the feasible confusion statistics space from the binary, multiclass, and multiclass-multigroup classification setups, we further provide strategies to elicit from a wider range of complex, modern multiclass metrics defined by quadratic functions of predictive rates by exploiting their local linear structure. This strategy can then be easily extended to eliciting  metrics of higher order polynomials. From application perspective, we also propose to use the metric elicitation framework in optimizing complex black box metrics that is amenable to deep network training. In particular, the linear elicitation strategies can be used to elicit local-linear approximation of the black-box metrics, which are then exploited by existing iterative optimization routines. Lastly, to bring theory closer to practice, we conduct a preliminary real-user study that shows the efficacy of the metric elicitation framework in recovering the users' preferred performance metric in a binary classification setup.
\end{abstract}

%% file: dedication.tex
\begin{dedication}
``To my parents, brother, and sister-in-law for their love and support."
\end{dedication}

%% file: ack.tex
\begin{acknowledgments}
The only goal to pursue a Ph.D. for me was to bridge my knowledge gap. From that perspective, I could not have asked for a better \emph{advising} and \emph{guidance} than what my advisor, Professor Oluwasanmi Koyejo, provided. While writing this thesis, I came to believe that I have been successful in achieving my goal to a great extent. I owe every success during my Ph.D. to my esteemed advsisor, Professor Oluwasanmi Koyejo. His expertise and guidance were invaluable for my research. I will always cherish the alignment in our thinking around research work, and the kind of freedom that I had while working with you. Your support and constructive feedback on every idea that I came up with was immensely encouraging. There is always so much to learn from you, especially, the context switching between multiple projects and a very high standard for work-ethics. 

I would like to thank my doctoral committee members: Professor Oluwasanmi Koyejo, Professor Srikant Rayadurgam, Professor Paris Smaragdis, and Professor Shivani Agarwal, who have always been very helpful and have given extremely thoughtful feedback on my thesis research. I am incredibly honored to be able to have them on my Ph.D. committee and feel the utmost gratitude for all their help and support. 

I would like to acknowledge my colleagues from the machine-learning group. 
We have enjoyed sharing offices and been good
friends. I would like to thank Professor Matus Telgarsky, Professor Pierre Moulin, Professor Ruoyu Sun, Professor Jiawei Han, and Professor Nan Jiang, who lectured in the outstanding courses I have taken at University of Illinois at Urbana-Champaign (UIUC). I would also like to express my gratitude to the computer science department for giving a nice and friendly environment to work with. 

Without the support of my mentors, my goal of bridging the knowledge gap would not have been possible. I would like to thank Harikrishna Narasimhan, Mahdi Milani Fard, Nikhil Rao, Sumeet Katariya,  Karthik Subbian, Prateek Jain, Branislav Kveton, Atanu Sinha, Shiv Kumar Saini,  Sunav Choudhary, and Sumit Shekhar for their guidance and sharing of knowledge. I cannot describe in words how much I have learnt from my colleagues/co-authors at various universities and Adobe Research, where I used to work before joining the Ph.D. program. I am thankful to each and everyone of them. 
I would like to thank my former advisors, Professor Harish Karnick and Professor Jean-Marc Schlenker, who inspired me in the very beginning of my research career and encouraged me to
pursue a doctorate degree. I would also like to thank my previous colleagues at Indian Institute of Technology, Kanpur. 

Words cannot express my gratitude towards my parents, Jayshree Hiranandani and Narendra Hiranandani. They have always been there for everything! I am grateful to my parents for the sacrifices they have made in order to make me reach where I am today. I would also like to thank my brother, Dharmendra Hiranandani, who has been an inspiration and idol for me since my childhood. His guidance has been immensely helpful all throughout. The learning from our discussions over several aspects of life and career are deeply engraved inside me. They have been the force behind my achievements. I am also grateful to my sister-in-law, Ruchira Bhelekar Hiranandani, from whom I learn something each and every day regarding positive and cheerful attitude towards life. 

During the writing of this thesis, I was going through ACL surgery rehabilitation. The timely deposit of this thesis could not have been possible without the support of my friends Monika Salkar, Ishita Jain, Siddhansh Agarwal, and Amber Srivastava.     

Lastly, I would like to thank the Computer Science Department at UIUC for awarding me the C.L. and Jane W.-S. Liu Award, which acted as a catalyst for my research on Metric Elicitation. I would also like to thank Intel, Microsoft Azure, and Google Cloud Platform for providing computational resources to support my research. 
\end{acknowledgments}

%% file: introduction.tex
\chapter{Introduction}
\label{chp:introduction}

\emph{Given a class prediction problem, which performance metric should the classifier optimize?} Machine learning practitioners often encounter this question in different forms. For example, natural language processing practitioners could face the question, \emph{``What is a good summary of a given article?~\cite{modani2016summarizing}"}  Similarly, for computer vision folks, \emph{``What is a good caption for a given image?"} poses an identical challenge~\cite{anderson2016spice}. In the field of music/audio research, the question, \emph{``When is one piece of music similar to another?"} may get similar treatment~\cite{berenzweig2004large}. Medical predictions are another application, where ignoring cost sensitive trade-offs can directly impact lives~\cite{sox1988medical}. Even companies in the industry struggle to find an answer to similar questions as specialized teams of statisticians/economists are routinely hired to monitor many metrics -- since optimizing the wrong metric directly translates into lost revenue~\cite{Dmitriev2016MeasuringM, choudhary2018sparse}. 
Unfortunately, there is scant formal guidance within the machine learning literature for how a practitioner might choose an appropriate metric, beyond a few default choices~\cite{caruana2004data, ferri2009experimental, sokolova2009systematic, hiranandani2019clustered}, and even less guidance on selecting a metric that reflects the preferences of the practitioners.

To address this issue, we propose the framework of
\emph{Metric Elicitation (ME)}, where the goal is to estimate a performance metric that best reflect implicit user preferences. This framework enables a practitioner to adjust the performance metrics based on the application, context, and population at hand. The motivation is that by employing metrics that reflect a user's innate trade-offs, one can learn models that best capture the user preferences. On its face, ME simply requires querying a user (oracle) to determine the quality she assigns to classifiers (learned using standard classification data); however, humans are often inaccurate when asked to provide absolute preferences~\cite{qian2013active, hiranandani2017poster}. Therefore, we propose gathering feedback in the form of pairwise classifier comparison queries, where the user is asked to compare two classifiers and provide an indicator of relative preference. Using such queries, ME aims to elicit the innate performance metric of the user. See Figure~\ref{fig:ME} for the visual intuition of the framework. 

\begin{figure*}[t]
	\centering 	\includegraphics[scale=0.55]{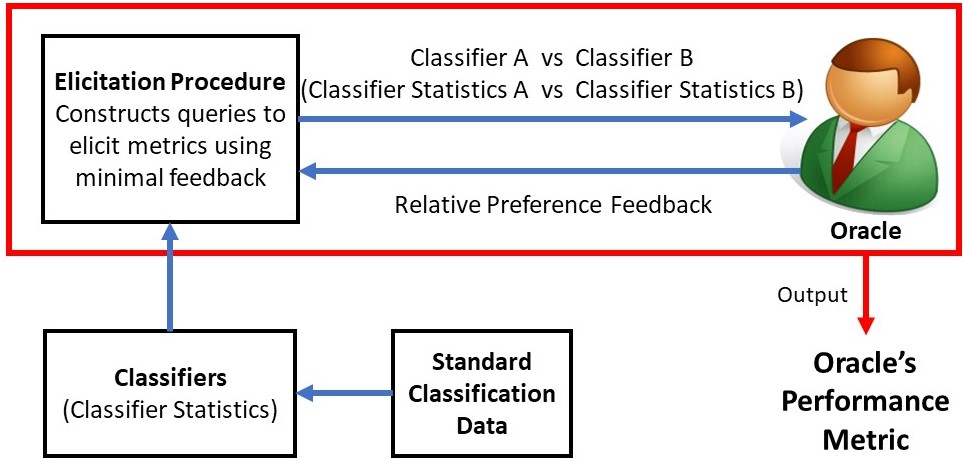}
	\caption{\textbf{Illustration of the Metric Elicitation Framework.} Our goal is to efficiently estimate the oracle’s performance metric. We assume that the models are summarized via measurements (classifier statistics), and the metric is a function of these measurements. The elicitation procedure poses pairwise comparisons queries of the type classifier A vs classifier B (equiv. classifier statistics A vs classifier statistics B). Based on relative preference feedback from the oracle, the framework elicits the oracle's metric in as few queries as possible.}
	\label{fig:ME}
\end{figure*}

We focus on eliciting the most common performance metrics that are functions of either confusion matrix or predictive rates elements~\cite{koyejo2014consistent, sokolova2009systematic}, commonly referred as \emph{measurements} or \emph{classifier statistics} in this manuscript.\footnote{Metrics depending on factors such as model complexity and interpretability are beyond the scope of this manuscript.} Thus, a classifier comparison query can be conceptually represented by a classifier statistics comparison query. Despite this apparent simplification, the problem remains challenging because one can only query feasible classifier statistics, i.e, classifier statistics for which there exists a classifier. To solve this problem, we introduce new characterizations of the space of feasible classifier statistics (associated with binary, multiclass, multiclass-multigroup classification problems) enabling the design of binary-search type procedures that identify the innate performance metric of the oracle. 
Furthermore, all the proposed procedures remain robust, both to noise from classifier estimation and to noise in the pairwise comparison itself. Thus, our work directly results in practical algorithms. The utility of ME is illustrated via the following real life applications.

{\bf Motivating Application 1: Medical Decision-Making using Cost-Sensitive Classification.}
Automated medical decision-making is an important application, where ignoring cost trade-offs can directly impact lives~\cite{sox1988medical}. Consider the case of cancer diagnosis and treatment support under the binary classification setting, where a doctor's unknown, innate performance metric may be approximated by a linear function of the confusion matrix elements, i.e., she has some innate reward values for True Positives and True Negatives -- equivalently, costs for False Positives and False Negatives -- based on known consequences of misdiagnosis, i.e, side-effects of treating a healthy patient vs. mortality rate for not treating a sick patient. Here, the doctor takes the role of the {\em oracle}. Our proposed approach exploits the space of confusion matrices associated with all possible classifiers that can be learned from standard classification data to determine the underlying rewards (equivalently, costs) provably using the least possible number of pairwise comparison queries posed to the doctor. Once the metric is elicited, it can be used to evaluate classifiers and/or train any future classifiers.

{\bf Motivating Application 2: Fair Machine learning.}
Machine learning models are increasingly applied for important decision-making tasks such as hiring and sentencing~\cite{dwork2012fairness,  singla2015learning, corbett2017algorithmic}. Yet, it is increasingly clear that automated decision-making is susceptible to bias; whereby decisions made by the algorithm are unfair to certain subgroups. To this end, several fairness metrics have been proposed -- all with the goal of reducing discrimination and bias from automated decision-making~\cite{barocas2017fairness}. One of the most difficult steps involved in practical deployment is the decision of which fairness metric to employ. This is further exacerbated by the observation that common metrics often lead to contradictory outcomes~\cite{kleinberg2017inherent}. Our approach for metric elicitation can be directly used to solve the {\em fairness metric selection} problem. Here, perhaps groups of ethicists or other relevant decision makers take the role of the {\em oracle}, and group-specific predictive rates correspond to the {\em query space} of interest -- which are easily approximated for any classifier. Metric elicitation can be used to formally quantify these intuitions -- specifying the quantitative metric that is best be applied to measuring or optimizing fairness for a given machine learning task, or to quantify the tradeoff between predictive performance and fairness. 

The applications of the proposed Metric Elicitation framework goes beyond just specifying user preferred performance metrics. It can also be used to learn classifiers that optimize complex performance metrics~\cite{hiranandani2020optimization} -- an aspect often crucial for practical applications. Several existing optimization algorithms are iterative in nature, where in each iteration, a local-linear objective is optimized. The iterates over the optimization routine are then combined to get to the final classifier. If the form of the metric is not known, then \emph{obtaining} the local-linear objective of the metric boils down to \emph{eliciting linear performance metric} in a local neighborhood. Thus, the tools from the  Metric Elicitation framework can be readily applied for optimizing black-box metrics. We discuss one such procedure, which optimizes black-box metrics in the presence of a \emph{machine} oracle, that when queried for a classifier returns an absolute quality feedback for the classifier. We then briefly discuss how the proposed procedure can be extended for a \emph{human} oracle that provides pairwise preference feedback, along with the challenges associated with it. 

Lastly, we conduct a preliminary user study, where we (a) build upon existing visualizations for confusion matrices to ask for pairwise preferences, and (b) try to elicit a linear performance metric using our proposed procedure in a binary classification setup associated with cancer diagnosis. The goal of this preliminary study is to test certain assumptions, check workflow of the implementation, and provide future guidance on visualizing confusion matrices for pairwise comparisons and finally eliciting actual performance metrics in real-life scenarios.

\section{Contributions and Thesis Organization}
\label{sec:contributions}

We first briefly summarize the contributions from this thesis. We then dig deep into each contribution later in Chapters~\ref{chp:me}-\ref{chp:practical}. 

\begin{enumerate}[label=(\alph*)]
\item \textbf{Metric elicitation framework (Chapter~\ref{chp:me}).} We formalize \emph{Metric Elicitation (ME)} -- a principled framework for determining supervised classification metrics from user feedback. For the case of pairwise feedback, we show that under certain conditions metric elicitation is equivalent to learning preferences between pairs of classifier statistics such as confusion matrices or predictive rates.

\item \textbf{Binary classification performance metric elicitation (Chapter~\ref{chp:binary}).} When the underlying metric is linear in the binary classification setup, we propose an elicitation strategy to recover the oracle's metric, whose query complexity decays logarithmically with the desired resolution. We also show that our query-complexity rates match the lower bound.  We further extend the linear metric elicitation algorithm to elicit more complex yet prevalent linear-fractional binary classification performance metrics. 

\item 
\textbf{Multiclass classification performance metric elicitation (Chapter~\ref{chp:multiclass}).} We extend work on binary classification setup by proposing ME strategies for the more complicated multiclass classification setting -- thus significantly increasing the use cases for ME. We propose two algorithms for multiclass classification metric elicitation that use multiple binary-search subroutines that recover the oracle's linear metric. One of the proposed algorithms assumes a sparsity condition on the metric, and thus is useful when the number of classes is large. Similar to the binary case, we further provide algorithms for eliciting linear-fractional multiclass classification performance metrics. 

\item \textbf{Fair performance metric elicitation (Chapter~\ref{chp:fair}).} With respect to applications to fairness, we  devise a novel strategy to elicit group-fair performance metrics for multiclass classification problems with multiple sensitive groups that also includes selecting the trade-off between predictive performance and fairness violation. 
Our procedure exploits the \emph{piecewise} linearity of the metric in group-specific predictive rates, uses binary-search based subroutines, and recovers the metric with linear query complexity.

    \item \textbf{Extension to quadratic metric elicitation and beyond (Chapter~\ref{chp:quadratic}).} The previous ME strategies can only handle metrics that are linear or quasi-linear functions of classifier statistics, which can be restrictive in domains  where the metrics are more complex and nuanced, e.g., ~\cite{menon2013statistical, narasimhan2018learning, hardt2016equality}. 
Thus, we propose novel strategies for eliciting metrics defined by \emph{quadratic} functions of classifier statistics, which can easily be applied to fair metric elicitation setups as well. We are thus be able to handle a more general family of metrics that can better capture a practitioner's innate preferences. We further generalize quadratic elicitation strategy to higher-order polynomial functions. The idea is to approximate a $d$-th order polynomial locally with $(d-1)$-th order polynomials and recursively apply our procedure to the lower-order polynomials.

\item \textbf{Optimizing black-box metrics through metric elicitation (Chapter~\ref{chp:blackbox}).} 
We consider learning to optimize a classification metric defined by a black-box function of the confusion matrix. Such black-box learning settings are ubiquitous, for example, when the learner only has query access to the metric of interest, or in noisy-label and domain adaptation applications where the learner must evaluate the metric via performance evaluation using a small validation sample. Our approach is to adaptively learn example weights on the training dataset such that the resulting weighted objective best approximates the metric on the validation sample. We use the fact that the example weights can be seen as a gradient for the metric and estimated through metric elicitation procedure, where a \emph{machine} oracle responds with absolute quality value of a classifier on a clean validation dataset.  We show how to model and estimate the example weights and use them to iteratively post-shift a pre-trained class probability estimator to construct a classifier. We also analyze the resulting procedure's statistical properties. Experiments on various label noise, domain shift, and fair classification setups confirm that our proposal compares favorably to the state-of-the-art baselines for each application.

\item \textbf{Eliciting real-user metric preferences (Chapter~\ref{chp:practical}).} Beyond technical contributions, our research raises novel questions with regards to classifier or classifier statistics visualization and interpretability for eliciting human preferences. 
We explore existing human-computer interface techniques for this task, including work on visualizing confusion matrices for non-expert users. We create a web user-interface and conduct a preliminary user-study in the binary classification setup in order to elicit real-users' performance metrics and devise procedures to evaluate the fidelity of the metrics that are recovered through the proposed metric elicitation framework.

\end{enumerate}

All our metric elicitation procedures (contributions (a)-(e)) are shown to be robust to both  finite sample and oracle feedback noise, thus are useful in practice. Our methods can be applied either by querying preferences over classifiers or classifiers statistics. Such an equivalence is crucial for practical applications. We provide statistical consistency guarantees of our black-box optimization algorithm (contribution (f)) that uses metric elicitation techniques in the presence of \emph{machine} oracles. We briefly discuss how this algorithm can be extended in the presence of \emph{human} oracles that provide pairwise feedback (including feedback from A/B tests) and the challenges associated with it. The related literature corresponding to each sub-topic is provided in the respective chapter. We draw out conclusions and future work in Chapter~\ref{chp:conclusion}.  Lastly, all the proofs are provided in the corresponding chapters' appendices. 

%% file: me.tex
\chapter{Metric Elicitation}
\label{chp:me}

In this section, we formally describe the problem of Metric Elicitation. We first lay out some  preliminaries and standard notations  corresponding to classification problems that are common to the entire manuscript. 

\textbf{Notation.} For $k \in \Zmbb_+$, we denote the index set by $[k] = \{1, 2, \cdots , k\}$ and use $\Delta_k$ to denote the $(k-1)$-dimensional simplex. We denote the inner product of vectors by $\inner{\cdot}{\cdot}$ and the Hadamard product by $\odot$. For a matrix $\Ambf$, $\offdiag(\Ambf)$ returns a vector of off-diagonal elements of $\Ambf$ in row-major form, and  $\diag(\Ambf)$ returns a vector of diagonal elements of $\Ambf$.  We denote the $\ell_2$-norm and $\ell_\infty$-norm of a vector by $\|\cdot\|_2$ and $\|\cdot\|_\infty$, respectively. 

\section{Preliminaries}
\label{sec:background}

We consider the standard $k$-class classification setting with $X \in \Xcal$ and $Y \in [k]$ representing the input and output random variables, respectively. We assume access to a sample $\{(\xmbf, y)_i\}_{i=1}^n$ of $n$ examples generated \emph{iid} from a distribution $ \Pmbb(X, Y)$. 
We work with (randomized) classifiers 
\begin{equation}
h : \Xcal \rightarrow \Delta_k    
\end{equation}
that takes in a feature vector $x$ as input and outputs its prediction in the form of a probability distribution over the $k$-classes. We further use 
\begin{equation}\Hcal = \{h : \Xcal \rightarrow \Delta_k\}\end{equation}
 to denote the set of all classifiers. 

\emph{Measurements (Classifier Statistics):} 
We assume $q$  measurements (classifier statistics) of each model $h$, with measurement functions $\{ g_i: \Hcal \times \Pmbb \rightarrow \Rmbb\}_{i=1}^q$. We denote the measurements  (classifier statistics) of a classifier $h$ by a vector $\cmbf\smbf(h, \Pmbb) = (g_1(h, \Pmbb), \dots, g_q(h, \Pmbb))$. Examples of such statistics  for a classifier include its confusion matrix $C_{ij}(h) = \Pmbb(Y = i, h = j)$ for $i, j \in [k]$, predictive rate matrix $R_{ij}(h) = \Pmbb(h = j| Y = i)$ for $i, j \in [k]$, etc. 

\emph{Metrics:} We consider performance metrics  that are defined by a general function $\phi : [0, 1]^{q}  \rightarrow \Rmbb$ of classifier statistics $\cmbf\smbf$: 
\begin{equation}
\phi(\cmbf\smbf(h, \Pmbb)).
\end{equation}
Since the scale of the metric does not affect the learning problem~\cite{narasimhan2015consistent}, we allow $\phi$ to be bounded. Observe that for these purposes, the metric is invariant to positive multiplicative scaling and additive bias. One common example of such metrics is linear metric, which given coefficient vector $\ambf \in \Rmbb^q$ with $\Vert \ambf \Vert_2 = 1$ (without  loss of generality, due to scale-invariance) is given by: 
\begin{equation}
\phi^{\text{lin}}  = \inner{\ambf}{\cmbf\smbf(h, \Pmbb)}.
\end{equation}

\emph{Feasible classifier statistics:} We will restrict our attention to only those classifier statistics that are feasible, i.e., can be achieved by some classifier. This allows us to build elicitation methods that can be applied either by querying preferences over classifiers or classifiers statistics. The set of all feasible classifier statistics is given by: 
\begin{equation}
\Ccal\Scal = \{\cmbf\smbf(h, \Pmbb) \,:\, h \in \Hcal \}.
\end{equation}
For simplicity, we will suppress the dependence on $\Pmbb$ and $h$ if it is clear from the context. 

\section{Metric Elicitation: Problem Setup}
\label{sec:me}

We now describe the problem of \emph{Metric Elicitation}. There's an \textit{unknown} metric $\phi$, and we seek to elicit its form by posing queries to an \emph{oracle} asking  which of two classifiers is more preferred by it. The oracle has access to the underlying metric $\phi$ and provides answers by comparing its value on the two classifiers.
\bdefinition
[Oracle Query] Given two classifiers $h_1, h_2$ (equiv. to classifier statistics $\cmbf\smbf_1, \cmbf\smbf_2$ respectively), a query to the Oracle (with metric $\phi$) is represented by:
\begin{align}
\Gamma(h_1, h_2\,;\, \phi) = \Omega(\cmbf\smbf_1, \cmbf\smbf_2\,;\,\phi) = \1[\phi(\cmbf\smbf_1) > \phi(\cmbf\smbf_2)], 
\end{align}
\noindent where $\Gamma: \Hcal \times \Hcal \rightarrow \{0,1\}$ and $\Omega: \Ccal\Scal \times \Ccal\Scal \rightarrow \{0, 1\}$. The query asks whether $h_1$ is preferred to $h_2$ (equiv. if $\cmbf\smbf_1$ is preferred to $\cmbf\smbf_2$), as measured by $\phi$. 
\label{me-def:query}
\edefinition

In practice, the oracle can be an expert, a group of experts, or an entire user population. The ME framework can be applied by posing classifier comparisons directly via interpretable learning techniques~\cite{ribeiro2016should, doshi2017towards} or via A/B testing~\cite{tamburrelli2014towards, hiranandanionline}. For example, in an internet-based applications one may perform A/B testing by deploying two classifiers A and B with two different sub-populations of users and use their level of engagement to decide which of the two classifiers is preferred. For other applications, we may present to the user, visualizations of the measurements such as predictive rates for two different classifiers (e.g.,  \cite{zhang2020joint,beauxis2014visualization}), and have the user provide pairwise feedback. 

Since the metrics we consider are functions of only the classifier statistics, queries comparing classifiers are the same as queries on the associated classifier statistics. So for convenience, we will have our algorithms pose queries comparing two (feasible) classifier statistics, but they can be equivalently seen as comparing two classifiers. 
We next formally state the ME problem.

\bdefinition [Metric Elicitation with Pairwise Queries (given $\Pmbb$)] Suppose that the oracle's (unknown) performance metric is $\phi$.  Using oracle queries of the form $\Omega(\cmbf\smbf_1, \cmbf\smbf_2\,;\,\phi)$, recover a metric $\hphi$ such that $\Vert\phi - \hphi\Vert < \kappa$ under a  suitable norm $\Vert \cdot \Vert$ for sufficiently small error tolerance $\kappa > 0$.
\label{me-def:mePop}
\edefinition

Notice that Definition~\ref{me-def:mePop} involves true population quantities $\cmbf\smbf_1, \cmbf\smbf_2$. However, in practice, we are given only finite samples. This leads to a more practical definition of the metric elicitation problem.

\bdefinition [Metric Elicitation with Pairwise Queries (given $\{(\xmbf_i,y_i)\}_{i=1}^n$)]
The same problem as stated in Definition~\ref{me-def:mePop}, except that the queries are of the form $\Omega(\hat{\cmbf\smbf}_1, \hat{\cmbf\smbf}_2)$, where $\hat{\cmbf\smbf}_1, \hat{\cmbf\smbf}_2$ are the estimated classifier statistics from the given samples.
\label{me-def:meFinite}
\edefinition
The performance of ME is evaluated by both the query complexity and the quality of the elicited metric~\cite{hiranandani2018eliciting, hiranandani2019multiclass}. As is standard in the decision theory literature~ \cite{koyejo2015consistent, hiranandani2018eliciting, hiranandani2019multiclass}, we present our ME approach by first assuming access to population quantities such as the population classifier statistics $\cmbf\smbf(h, \Pmbb)$ as in Definition~\ref{me-def:mePop}, then examine estimation error from finite samples, i.e., with empirical rates $\hat{\cmbf\smbf}(h, \{(\xmbf,y)_i\}_{i=1}^n)$ as in Definition~\ref{me-def:meFinite}. Lastly, in all our proposed metric elicitation strategies, we work with the following noise model:
\bdefinition
Oracle Feedback Noise $(\epsilon_\Omega\geq0)$: 
The oracle may provide wrong answers whenever
$|\phi(\cmbf\smbf)-\phi(\cmbf\smbf')|<\epsilon_\Omega$. Otherwise, it provides correct answers. 
\label{me-def:noise}
\edefinition
Simply put, if the classifier statistics $\cmbf\smbf, \cmbf\smbf'$ are close as measured by $\phi$, then the oracle responses may be incorrect. We show robustness of our approaches under this noise model. We next discuss elicitation strategies for the different classification scenarios starting with the binary classification problem setup.
\newpage

%% file: binary/header.tex
\chapter{Binary Classification Performance Metric Elicitation}
\label{chp:binary}


\input{binary/introduction}
\input{binary/background}
\input{binary/confusionMatrices}
\input{binary/algorithm_linear}
\input{binary/algorithm_linearfrac}
\input{binary/noise}
\input{binary/experiments}
\input{binary/conclusion}

%% file: binary/introduction.tex

In this chapter, we focus on eliciting binary classification performance metrics from pairwise feedback, where a practitioner is queried to provide relative preference between two classifiers. Here, we choose our measurement space to be the space of feasible confusion matrices associated with the classifiers for binary classification. By exploiting key geometric properties of the space of confusion matrices, we obtain provably query efficient algorithms for eliciting performance metrics. 
We emphasize that the notion of pairwise classifier comparison 
is not new and is already prevalent in the industry. An example is A/B testing \cite{tamburrelli2014towards}, 
where the whole population of users acts as an oracle.\footnote{In A/B testing, sub-populations of users are shown classifier A vs. classifier B, and their responses determine the overall preference. Interestingly, while each person is shown a sample output from one of the classifiers, the entire user population acts as the oracle for comparing classifiers.}  Similarly, classifier comparison by a single expert is becoming commonplace due to advances in the field of interpretable machine learning~\cite{ribeiro2016should, doshi2017towards}.

In this first edition of metric elicitation strategies, we focus on the most common performance metrics which are functions of the confusion matrix \cite{koyejo2014consistent, narasimhan2015consistent, sokolova2009systematic}, particularly, linear and ratio-of-linear functions. 
This includes almost all modern metrics such as 
accuracy, $F_\beta$-Measure, Jaccard Similarity Coefficient~\cite{sokolova2009systematic}, etc. 
By construction, pairwise classifier comparisons may be conceptually represented by their associated pairwise confusion matrix comparisons. 
Despite this apparent simplification, the problem remains challenging because one can only query feasible confusion matrices, i.e. confusion matrices for which there exists a classifier. As we show, our characterization of the space of confusion matrices enables the design of efficient binary-search type procedures that identify the innate performance metric of the oracle. 
While classifier (confusion matrix) comparisons may introduce additional noise, our approach remains robust, both to noise from classifier (confusion matrix) estimation, and to noise in the comparison itself. Thus, our work directly results in a practical algorithm. 

\textbf{Example:}  
Consider the case of cancer diagnosis, where a doctor's unknown, innate performance metric is a linear function of the confusion matrix, i.e., she has some innate reward values for True Positives and True Negatives 
-- 
equivalently (equiv.), costs for False Positives and False Negatives 
-- 
based on known consequences of misdiagnosis. 
Here, the doctor takes the role of the oracle. 
Our proposed approach exploit the space of confusion matrices associated with all possible classifiers that can be learned from standard classification data 
and determine the underlying rewards (equiv., costs) provably using the least possible number of pairwise comparison queries posed to the doctor. 

Our contributions in this chapter are summarized as follows:
\begin{itemize}
	\item When the underlying metric is linear, we propose a binary search algorithm that can recover the metric with query complexity that decays logarithmically with the desired resolution. We further show that our query-complexity rates match the lower bound. 
	\item We extend the elicitation algorithm to more complex linear-fractional performance metrics.
	\item We prove robustness of the proposed approach under feedback and classifier estimation noise.
\end{itemize}

All the proofs in this chapter are provided in Appendix~\ref{apx:binary}. 

%% file: binary/background.tex
\section{Background}
\label{sec:background}

Let $X \in \Xcal$ and $Y \in \{0, 1\}$ represent the input and output random variables respectively (0 = negative class, 1 = positive class). 
We assume a dataset of size $n$, $\{(x_i, y_i)\}_{i=1}^n$, generated \emph{iid} from a data generating distribution $ \Pmbb \overset{\text{iid}}{\sim} (X, Y)$. 
Let $f_X$ be the marginal distribution for $\Xcal$. Let $\eta(x) = \Pmbb(Y = 1 | X = x)$ and $\zeta = \Pmbb(Y = 1)$ represent the conditional and the unconditional probability of the positive class, respectively. Note that the earlier term is a function of the input $x$; whereas, the latter is a constant. 
We denote a classifier by $h$, and let $\Hcal = \{h : \Xcal \rightarrow [0, 1]\}$ be the set of all classifiers. 
A confusion matrix for a classifier $h$ is denoted by $C(h, \Pmbb) \in \Rmbb^{2 \times 2}$, comprising true positives (TP), false positives (FP), false negatives (FN), and true negatives (TN) and is given by:
\begin{ceqn}
\begin{align}
	C_{11} &= TP(h, \Pmbb) = \Pmbb(Y = 1, h = 1), \nonumber \\
	C_{01} &= FP(h, \Pmbb) = \Pmbb(Y = 0, h = 1), \nonumber \\
	C_{10} &= FN(h, \Pmbb) = \Pmbb(Y = 1, h = 0), \nonumber \\ 
	C_{00} &= TN(h, \Pmbb) = \Pmbb(Y = 0, h = 0). 
	\label{bin-eq:components}
\end{align}\nopagebreak
\end{ceqn}\nopagebreak
Clearly,  $\sum_{i, j} C_{ij} = 1$. We denote the set of all confusion matrices by $\Ccal = \{C(h, \Pmbb)  : h \in \Hcal\}$. 
Under the population law $\Pmbb$, the components of the confusion matrix can be further decomposed as:
\begin{equation}
FN(h, \Pmbb) = \zeta - TP(h, \Pmbb) \quad \text{and} \quad FP(h, \Pmbb) = 1 - \zeta - TN(h, \Pmbb).
\end{equation}
This decomposition reduces the four dimensional space to two dimensional space. Therefore, the set of confusion matrices can be defined as \begin{equation}
\Ccal = \{(TP(h, \Pmbb), TN(h, \Pmbb)) : h \in \Hcal\}.
\end{equation}
For clarity, we will suppress the dependence on $\Pmbb$ in our notation. In addition, we will subsume the notation $h$ if it is implicit from the context and denote the confusion matrix by $C = (TP, TN)$.

We represent the boundary of the set $\Ccal$ by $\partial\Ccal$. Any hyperplane (line) $\ell$ in the $(tp, tn)$ coordinate system is given by:
\begin{equation}
\ell := a\cdot tp + b\cdot tn = c, \quad \text{ where } a, b, c \in \Rmbb.
\end{equation}
Let $\phi : [0, 1]^{2 \times 2}  \rightarrow \Rmbb$ be the performance metric for a classifier $h$ determined by its confusion matrix $C(h)$. Without loss of generality (w.l.o.g.), we assume that $\phi$ is a utility, so that larger values are better. 

\subsection{Types of Performance Metrics}
\label{ssec:metrics}

We consider two of the most common families of binary classification metrics, namely linear and linear-fractional functions of the confusion matrix \eqref{bin-eq:components}.
\bdefinition
Linear Performance Metric (LPM): We denote this family by $\varphi_{LPM}$. Given constants (representing weights) $\{a_{11}, a_{01}, a_{10}, a_{00}\} \in \Rmbb^4$, we define the metric as:
\begin{ceqn}
\begin{align}
\phi(C) &= a_{11}TP + a_{01}FP  + a_{10}FN + a_{00}TN \nonumber \\
				  &= m_{11}TP + m_{00}TN + m_0,   
\label{eq:linear}
\end{align}
\end{ceqn}
where $m_{11} = (a_{11} - a_{10})$, $m_{00} = (a_{00} - a_{01})$, and $m_0 = a_{10}\zeta + a_{01}(1-\zeta)$. 
\edefinition
\bexample Weighted Accuracy (WA)~\cite{steinwart2007compare}:
\begin{equation}
WA = w_1TP+w_2TN,
\end{equation}
where $w_1, w_2 \in [0, 1]$ 
($w_1, w_2$ can be shifted and scaled to $[0, 1]$ without changing the learning problem ~\cite{narasimhan2015consistent}).
\label{ex:lossbased}
\eexample
\bdefinition
Linear-Fractional Performance Metric (LFPM): We denote this family by $\varphi_{LFPM}$. Given constants  $\{a_{11}, a_{01}, a_{10}, a_{00}$, $b_{11}, b_{01}, b_{10}, b_{00}\} \in \Rmbb^8$, we define the metric as: 
\begin{ceqn}
\begin{align}
\phi(C) &= \frac{a_{11}TP +  a_{01}FP +  a_{10}FN +  a_{00}TN}{b_{11}TP +  b_{01}FP +  b_{10}FN +  b_{00}TN} \nonumber \\
&= \frac{p_{11}TP +  p_{00}TN +  p_0}{q_{11}TP +  q_{00}TN +  q_0},
\label{linear-fractional}
\end{align}
\end{ceqn}
where $p_{11} = (a_{11} - a_{10})$, $p_{00} = (a_{00} - a_{01})$, $q_{11} = (b_{11} - b_{10})$, $q_{00} = (b_{00} - b_{01})$, $p_0 = a_{10}\zeta + a_{01}(1 - \zeta)$, $q_0 = b_{10}\zeta + b_{01}(1 - \zeta)$.
\edefinition
\bexample
The $F_\beta$ measure and the Jaccard similarity coefficient (JAC)~\citep{sokolova2009systematic}:
\begin{align}
    F_{\beta} = \frac{TP}{\frac{TP}{1+\beta^2} - \frac{TN}{1+\beta^2} + \frac{\beta^2\zeta + 1 - \zeta}{1+\beta^2}}, \, JAC = \frac{TP}{1-TN}
\label{ex:lf-examples}
\end{align}
\eexample

\subsection{Bayes Optimal and Inverse Bayes Optimal Classifiers}
\label{ssec:bayes}

Given a performance metric $\phi$, the Bayes utility $\btau$ is the optimal value of the performance metric over all classifiers, i.e., 
\begin{equation}
\btau = \sup_{h \in \Hcal}\phi(C(h)) = \sup_{C \in \Ccal}\phi(C).
\end{equation}
The Bayes classifier $\hbar$ (when it exists) is the classifier that optimizes the performance metric, so 
\begin{equation}
\hbar = \argmax\limits_{h \in \Hcal}\phi(C(h)).
\end{equation}
Similarly, the Bayes confusion matrix is given by 
\begin{equation}
\Cbar = \argmax\limits_{C \in \Ccal}\phi(C).
\end{equation} 
We further define the inverse Bayes utility   
\begin{equation}
\ttau = \inf_{h \in \Hcal}\phi(C(h)) = \inf_{C \in \Ccal}\phi(C).
\end{equation}
The inverse Bayes classifier is given by \begin{equation}
\barbelow{h} = \argmin\limits_{h \in \Hcal}\phi(C(h)).
\end{equation}
Similarly, the inverse Bayes confusion matrix is given by: 

\begin{equation}
\barbelow{C} = \argmin\limits_{C \in \Ccal}\phi(C).
\end{equation}

Notice that for $\phi \in \varphi_{LPM}$ \eqref{eq:linear}, 
the Bayes classifier predicts the label which  maximizes the expected utility conditioned on the instance, as discussed below.
\bprop Let $\phi \in \varphi_{LPM}$, then 
\begin{equation}
\hbar(x) = \left\{\begin{array}{lr}
			 \1[\eta(x) \geq \frac{m_{00}}{m_{11} + m_{00}}],& \; m_{11} + m_{00} \geq 0 \\
			 \1[\frac{m_{00}}{m_{11} + m_{00}} \geq \eta(x) ],& \;  o.w. 
	 	 \end{array}\right\}
\end{equation}
is a Bayes optimal classifier \emph{w.r.t} 
$\phi$. Further, the inverse Bayes classifier is given by $\barbelow{h}= 1 - \hbar$.
\label{pr:bayeslinear}
\eprop

\vspace{-1cm}
\subsection{Problem Setup}
\label{ssec:query}

We borrow the problem setup from Chapter~\ref{chp:me}, particularly, the definitions of oracle query (Definition \ref{me-def:query}) and Metric Elicitation with finite samples (Definition~\ref{me-def:meFinite}). Since our choice of measurements is the confusion matrix entries, for ease of understanding, we re-state these definitions after replacing classifier statistics by confusion matrices for binary classification.  

We first formalize \emph{oracle query}. 
Recall that by the definition of confusion matrices~\eqref{bin-eq:components}, there exists a surjective mapping from $\Hcal \rightarrow \Ccal$.
An oracle is queried to determine relative preference between two classifiers. However, since we only consider metrics which are functions of the confusion matrix, a comparison query over classifiers becomes equivalent to a comparison query over confusion matrices in our setting.

\bdefinition
Oracle Query: Given two classifiers $h, h'$ (equiv. to confusion matrices $C, C'$ respectively), a query to the Oracle (with metric $\phi$) is represented by:
\begin{ceqn}
\begin{align}
\Gamma(h, h' \,;\,\phi) = \Omega(C, C'\,;\,\phi) &= \1[\phi(C) > \phi(C')] =: \1[C \succ C'],
\end{align}
\end{ceqn}
where $\Gamma: \Hcal \times \Hcal \rightarrow \{0,1\}$ and $\Omega: \Ccal \times \Ccal \rightarrow \{0, 1\}$. The query denotes whether $h$ is preferred to $h'$ (equiv. to $C$ is preferred to $C'$) as measured according to $\phi$.
\label{def:query}
\edefinition
We emphasize that depending on practical convenience, the oracle may be asked to compare either confusion matrices or classifiers achieving the corresponding confusion matrices, via approaches discussed in the beginning of  Chapter~\ref{chp:binary}. Henceforth, for simplicity of notation, we will treat any comparison query as confusion matrix comparison query. Next, we state the metric elicitation problem. 

\bdefinition 
    Metric Elicitation (given $\{(x_i,y_i)\}_{i=1}^n$): Suppose that the oracle's true, unknown performance metric is $\phi$.  Recover a metric $\hphi$ by querying the oracle for as few pairwise comparisons of the form $\Omega(\hat C, \hat C')$, where $\hat C, \hat C'$ are the estimated confusion matrices from the samples, such that $\Vert\phi - \hphi\Vert_{\_\_} < \kappa$ for sufficiently small $\Rmbb \ni\kappa > 0$ and for any suitable norm $\Vert \cdot \Vert_{\_\_}$.
\label{eq:mefinite}
\edefinition

Ultimately, we want to perform ME as described in  Definition~\ref{eq:mefinite}. A good approach to do so is to first solve ME by assuming access to the appropriate population quantities such as the population confusion matrices $\Cmbf(h, \Pmbb)$, and then consider practical implementation using estimated confusion matrices from finite data, i.e., $\Cmbf(h, \{(x_i,y_i)\}_{i=1}^n)$. This is a standard approach in decision theory (see e.g. \cite{koyejo2015consistent}), where estimation error from finite samples is adjudged as a noise source and handled accordingly. 

%% file: binary/confusionMatrices.tex
\section{Confusion Matrices}
\label{bin-sec:confusion}\vspace*{-1ex}

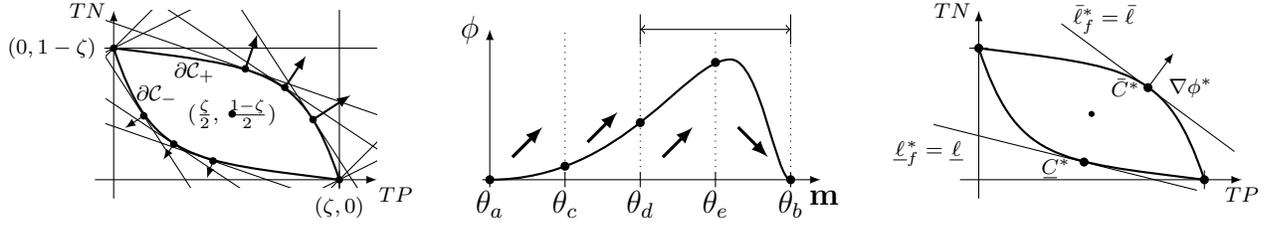
\begin{figure*}[t]
	\centering
	\begin{tikzpicture}[scale = 1.0]
    

    	\begin{scope}[shift={(-5,0)},scale = 0.5]\scriptsize
    	\def\r{0.1};
    	\def\s{0.06};
	
	\draw[thick] (0,3.5) .. controls (4,3) and (5,3) .. (6,0)
    	.. controls (2,0.5) and (1,0.5) .. (0,3.5);
    \draw[-latex] (0,-.5)--(0,4.5); 
    \draw[-latex] (-.5,0)--(7,0);
    \node[left] at (0,4.5) {$TN$};
    \node[below] at (7.5,0) {$TP$};
    \draw (6,0) +(0,0.25) -- +(0,-.25);
    \draw (0,3.5) +(.25,0) -- +(-.25,0);
    \node[below] at (6,-.25) {$(\zeta, 0)$};
    \node[left] at (-0.25,3.5) {$(0, 1-\zeta)$};
    
    \coordinate (C*) at (4.55,2.455);    
    \coordinate (C1) at (3.5,2.95);
    \coordinate (C2) at (5.3,1.6);
    
	\coordinate (Cent) at (3.15,1.75);    
    
    \coordinate (Ct) at (1.6,0.95);
    \coordinate (Ct1) at (2.64,0.5);
    \coordinate (Ct2) at (0.8,1.7); 
    \coordinate (C+) at (2.15,2.85);
    
    \coordinate (C-) at (1.15,2.25);

    \fill[color=black] 
    		(0,3.5) circle (\r)
    		(6,0) circle (\r)
        (Cent) circle (\r)
        (C*) circle (\r)
        (C1) circle (\r)
        (C2) circle (\r)
        (Ct) circle (\r)
        (Ct1) circle (\r)
        (Ct2) circle (\r);
    
    	\clip (-0.2,-0.2) rectangle (7,4.5);   
    
    \draw (C*) +(-24,16) -- +(24,-16);
    \draw[-latex, thick] (C*) -- +(.56,0.84);
    \draw (C1) +(-23,8) -- +(23,-8);
    \draw[-latex,thick] (C1) -- +(.32,0.92);
    \draw (C2) +(-13,20) -- +(13,-20);
    \draw[-latex,thick] (C2) -- +(1,.65);
    
    \draw (Ct) +(-24,16) -- +(24,-16);
    \draw[-latex] (Ct) -- +(-.28,-0.42);
    \draw (Ct1) +(-23,8) -- +(23,-8);
    \draw[-latex] (Ct1) -- +(-.16,-0.46);
    \draw (Ct2) +(-13,20) -- +(13,-20);
    \draw[-latex] (Ct2) -- +(-.5,-.325);
    
	
	\draw(6,0) +(0,10) -- +(0,-20);
	\draw(6,0) +(10,10) -- +(-20,-20);
	\draw(6,0) +(20,10) -- +(-20,-10);
	\draw(0,3.5) +(10,0) -- +(-20,0);    
	\draw(0,3.5) +(10,10) -- +(-20,-20);    
	\draw(0,3.5) +(10,5) -- +(-20,-10);    
    
    
    \node at (Cent) {{$(\frac{\zeta}{2},\, \frac{1-\zeta}{2})$}};
    
    \node at (C+) {{$\partial\Ccal_+$}};
    \node at (C-) {{$\partial\Ccal_-$}};

    \end{scope}

    
	\def\r{0.06};
	
    \draw[thick] (0,0) .. controls (1.8,0) and (2.6,1.6) .. (3.2,1.6) 
    ..controls (3.6,1.6) and (3.8,0) .. (4,0);
    \draw[-latex] (0,-.1)--(0,2); 
    \draw[-latex] (-0.1,0)--(4.4,0);
    \node[left] at (0,2) {$\phi$};
    \node[below right] at (4.1,0) {$\mathbf m$};
   
   	\coordinate (C1) at (0,0.00);
    \coordinate (C2) at (1,0.18);
    \coordinate (C3) at (2,0.76);
    \coordinate (C4) at (3,1.56);
    \coordinate (C5) at (4,0.00);
    
    \node[below] at (0,0) {$\theta_a$};
    \node[below] at (1,0) {$\theta_c$};
    \node[below] at (2,0) {$\theta_d$};
    \node[below] at (3,0) {$\theta_e$};
    \node[below] at (4,0) {$\theta_b$};
    
    \foreach \x in {1,2,3,4} {
    	\draw (\x,-.1) -- (\x,.1);
        \draw[dotted] (\x,0) -- (\x,2);
    }
    \fill[color=black] 
    		(C1) circle (\r)
    		(C2) circle (\r)
            (C3) circle (\r)
            (C4) circle (\r)
            (C5) circle (\r);   
    
    \draw[very thick,-latex] (0.3,0.3) -- (0.7,0.7);
    \draw[very thick,-latex] (1.3,0.5) -- (1.7,0.9);
    \draw[very thick,-latex] (2.3,0.3) -- (2.7,0.7);
    \draw[very thick,-latex] (3.3,0.7) -- (3.7,0.3);
    
    \draw (2,1.8)--(2,2.2) (4,1.8)--(4,2.2);
    \draw[<->] (2,2)--(4,2);

    
     \begin{scope}[shift={(6.5,0)},scale = 0.5]\scriptsize
     
     \def\r{0.12};
	
	\draw[thick] (0,3.5) .. controls (4,3) and (5,3) .. (6,0)
    	.. controls (2,0.5) and (1,0.5) .. (0,3.5);
    \draw[-latex] (0,-.5)--(0,4.5); 
    \draw[-latex] (-.5,0)--(7,0);
    \node[left] at (0,4.5) {$TN$};
    \node[below] at (7,0) {$TP$};
    \draw (6,0) +(0,0.25) -- +(0,-.25);
    \draw (0,3.5) +(.25,0) -- +(-.25,0);
    
    \coordinate (C*) at (4.5,2.45);
    \coordinate (l*) at (2.3,4.25);
    \coordinate (Ct) at (2.8,0.47);
    \coordinate (lt) at (-0.2,1.25); 
    
    \fill[color=black] (0,3.5) circle (\r)
    		(6,0) circle (\r)
            (3,1.75) circle (0.08)
            (C*) circle (\r)
            (Ct) circle (\r);
            
    \node[right] at (l*) {$\bar{\ell}_f^* = \bar{\ell}$};
    \node[below left] at (lt) {$\barbelow{\ell}^*_f = \barbelow{\ell}$};
    
    \draw (C*) +(-2.3,1.7) -- +(2.3,-1.7);
    \draw[-latex] (C*) +(0,0) -- +(0.68,0.92);
    \node[left] at (C*) {$\bar{C}^*$};
    \node[below right] at ($(C*)+(0.34,0.46)$) {$\nabla \phi^*$};
   	
    \draw (Ct) +(-4,1) -- +(3,-0.75);
    \node[below left] at ($(Ct)+(-0.15,0.30)$) {$\barbelow{C}^*$};
     \end{scope}
\end{tikzpicture}
	\caption{\small{(a)} \normalsize  Supporting hyperplanes (with normal vectors) and resulting geometry of $\Ccal$; \small(b) \normalsize Sketch of Algorithm~\ref{bin-alg:linear}; \small(c) \normalsize Maximizer $\Cbar^*$ and minimizer $\barbelow{C}^*$ along with the supporting hyperplanes for LFPMs.}
	\label{bin-fig:lin-fr}
\end{figure*}

ME will require  confusion matrices that are achieved by all possible 
classifiers, thus it is necessary to characterize the set $\Ccal$ in a way which is useful for the task.

\bassumption
We assume $g(t)=\Pmbb[\eta(X)\geq t]$ is continuous and 
strictly decreasing for $t \in [0, 1]$.
\label{bin-as:eta}
\eassumption
\vspace*{-1ex}
This is equivalent to standard assumptions \citep{koyejo2014consistent} that the event $\eta(X)=t$ has positive density but zero probability. 
Note that this requires $X$ to have no point mass. 

\bprop{(Properties of $\Ccal$ --- Figure~\ref{bin-fig:lin-fr}(a).)}\label{pr:strict-convex}
The set of confusion matrices $\Ccal$ is convex, closed, contained in the rectangle $[0,\zeta]\times [0,1-\zeta]$ (bounded), and $180$-degree rotationally symmetric around the center-point $(\frac{\zeta}{2}, \frac{1-\zeta}{2})$. Under Assumption~\ref{bin-as:eta}, $(0,1-\zeta)$ and $(\zeta,0)$ are the only vertices of $\Ccal$, and $\Ccal$ is strictly convex. Thus, any supporting hyperplane of $\Ccal$ is tangent at only one point.\footnote{Additional visual intuition about the geometry of C (via an example) is given in Appendix~\ref{appendix:visualization}.}
\eprop

\subsection{LPM Parametrization and Connection with Supporting Hyperplanes of $\Ccal$}
\label{ssec:parametrization}
For an LPM $\phi$ \eqref{eq:linear}, Proposition~\ref{pr:strict-convex} guarantees the existence of a unique Bayes confusion matrix on the boundary $\partial \Ccal$.  
This is because optimum for a linear function over a strictly convex set is unique and lies on the boundary~\cite{boyd2004convex}.
Note 
that any linear function with the same trade-offs for \TPs and \TN, i.e. same $(m_{11}, m_{00})$, is maximized at the same boundary point regardless of the bias term $m_0$. 
Thus, different LPMs can be generated by varying trade-offs $\mmbf = (m_{11}, m_{00})$ such that $\norm{\mmbf} = 1$ and $m_0 = 0$. The condition $\norm{\mmbf} = 1$ does not affect the learning problem as discussed in Example~\ref{ex:lossbased}. In other words, the performance metric is scale invariant. 
This allows us to represent the family of linear metrics $\varphi_{LPM}$ by a single parameter~$\theta \in [0, 2\pi]$:
\begin{ceqn}
\begin{equation}
	\varphi_{LPM} = \{ \mmbf = (\cos\theta, \sin\theta) : \theta \in [0, 2\pi]\}.
	\label{set:lfpm}
\end{equation}
\end{ceqn}
Given $\mmbf$ (equiv. to $\theta$), we can recover the Bayes classifier using Proposition~\ref{pr:bayeslinear}, and then the Bayes confusion matrix  $\Cbar_\theta$ = $\Cbar_\mmbf = (\oline{TP}_{\mmbf}, \oline{TN}_{\mmbf})$ using \eqref{bin-eq:components}. Under Assumption~\ref{bin-as:eta}, due to strict convexity of $\Ccal$, the Bayes confusion matrix $\Cbar_\mmbf$ is unique; therefore, we have that
\begin{ceqn}
\begin{align}
\langle \mmbf, C \rangle <  \langle \mmbf, \Cbar_\mmbf \rangle \qquad \forall \; C \in \Ccal, C \neq \Cbar_\mmbf.
\end{align}
\end{ceqn}
Notice the connection between the linear performance metrics and the supporting hyperplanes of the set $\Ccal$ (see Figure~\ref{bin-fig:lin-fr}(a)). Given $\mmbf$, there exists a supporting hyperplane tangent to $\Ccal$ at only $\Cbar_\mmbf$
defined as follows:
\begin{ceqn}
\bequation
\bell_\mmbf \coloneqq m_{11}\cdot tp + m_{00}\cdot tn = m_{11}\oline{TP}_{\mmbf} +  m_{00}\oline{TN}_{\mmbf}.
\label{eq:support}
\eequation
\end{ceqn}

Clearly, if $m_{11}$ and $m_{00}$ are of opposite sign (i.e., $\theta \in (\sfrac{\pi}{2}, \pi)\cup(\sfrac{3\pi}{2}, 2\pi)$), then $\hbar_{\mmbf}$ is the trivial classifier predicting either 1 or 0 everywhere. In other words, if the slope of the hyperplane is positive, then it touches the set $\Ccal$ either at $(\zeta, 0)$ or $(0, 1 - \zeta)$.  
When $m_{11}, m_{00} \neq 0$ with the same sign (i.e., $\theta \in (0, \sfrac{\pi}{2})\cup(\pi, \sfrac{3\pi}{2})$), 
then the Bayes confusion matrix is away from the two vertices. Now, we may split the boundary $\partial\Ccal$ as follows:
\bdefinition 
\label{def:boundary}
The Bayes confusion matrices for LPMs with $m_{11}, m_{00} \geq 0$ $(\theta \in [0, \sfrac{\pi}{2}])$ form the upper boundary, denoted by $\partial\Ccal_+$. The Bayes confusion matrices for LPMs with $m_{11}, m_{00} < 0$ $(\theta \in (\pi, \sfrac{3\pi}{2}))$ form the lower boundary, denoted by $\partial\Ccal_-$. From Proposition ~\ref{pr:bayeslinear}, it follows that the confusion matrices in $\partial\Ccal_+$ and $\partial\Ccal_-$ correspond to the classifiers of the form  $\1[\eta(x)\geq \delta]$ and $\1[\delta \geq \eta(x)]$, respectively, for some $\delta \in [0, 1]$.
\edefinition

%% file: binary/algorithm_linear.tex
\section{Algorithms}
\label{sec:algorithms}

In this section, we propose binary-search type algorithms, which exploit the geometry of the set $\Ccal$ (Section \ref{bin-sec:confusion}) to find the maximizer / minimizer and the associated supporting hyperplanes for any quasiconcave / quasiconvex metrics. These algorithms are then used to elicit LPMs and LFPMs, both of which belong to both quasiconcave and quasiconvex function families.

We allow \emph{noisy} oracles; however, for simplicity, 
we will first discuss algorithms and elicitation with no-noise, and then show that they are robust to the noisy feedback (Section~\ref{sec:noise}). 
Moreover, as one typically prefers metrics which reward correct classification, we first discuss metrics that are monotonically increasing in both $\TP$ and $\TN$. The monotonically decreasing case is discussed in Appendix~\ref{appendix:decreasing} as a natural extension. 

The following lemma for any quasiconcave and quasiconvex metrics forms the basis of our proposed algorithms. 
\blemma
Let 
$\rho^+:[0,1]\to \partial\mathcal C_+$, $\rho^-:[0,1]\to \partial\mathcal C_-$
be continuous, bijective, parametrizations of the upper and lower boundary,
respectively. Let $\phi:\mathcal C\to \mathbb R$ be a quasiconcave function,
and $\psi:\mathcal C\to \mathbb R$ be a quasiconvex function, 
which are 
monotone increasing in both $TP$ and $TN$.
Then the composition $\phi\circ \rho^+: [0,1]\to\mathbb R$ is
quasiconcave (and therefore unimodal) on the interval $[0, 1]$, and $\psi\circ\rho^-:[0,1]\to\mathbb R$ is quasiconvex (and therefore unimodal) on the interval $[0,1]$. \label{lem:quasiconcave}
\elemma
The unimodality of quasiconcave (quasiconvex) metrics on the upper (lower) boundary of the set $\Ccal$ 
along with the one-dimensional parametrization of $\mmbf$ using $\theta \in [0, 2\pi]$ (Section \ref{bin-sec:confusion}) allows us to devise binary-search-type methods to find the maximizer $\Cbar$, the minimizer $\barbelow{C}$, and the first order approximation of $\phi$ at these points, i.e., the supporting hyperplanes at $\Cbar$ and $\barbelow{C}$. 
\balgorithm[t]
\caption{Quasiconcave Metric Maximization}
\label{bin-alg:linear}
\balgorithmic[1]
\STATE \textbf{Input:} $\epsilon > 0$ and oracle $\Omega$.
\STATE \textbf{Initialize:} $\theta_a = 0$, $\theta_b = \frac{\pi}{2}$.
\WHILE {$\abs{\theta_b - \theta_a} > \epsilon$}
\STATE Set $\theta_c = \frac{3\theta_a + \theta_b}{4}$, $\theta_d = \frac{\theta_a + \theta_b}{2}$, and $\theta_e = \frac{\theta_a + 3\theta_b}{4}$. Set corresponding slopes ($\mmbf$'s) using \eqref{set:lfpm}. \\
\STATE Obtain $\hbar_{\theta_a}$,$\hbar_{\theta_c}$,$\hbar_{\theta_d}$, $\hbar_{\theta_e}, \hbar_{\theta_b}$ using Proposition \ref{pr:bayeslinear}. Compute $\Cbar_{\theta_a}$,$\Cbar_{\theta_c}$,$\Cbar_{\theta_d}$,$\Cbar_{\theta_e}, \Cbar_{\theta_b}$ using~\eqref{bin-eq:components}.
\STATE Query $\Omega(\Cbar_{\theta_c}, \Cbar_{\theta_a}), \Omega(\Cbar_{\theta_d}, \Cbar_{\theta_c}), \Omega(\Cbar_{\theta_e}, \Cbar_{\theta_d}),$ and $\Omega(\Cbar_{\theta_b}, \Cbar_{\theta_e})$.
\STATE If $\Cbar_\theta \succ \Cbar_{\theta'}\prec \Cbar_{\theta''}$ for consecutive $\theta<\theta'<\theta''$, assume the default order $\Cbar_\theta \prec \Cbar_{\theta'}\prec \Cbar_{\theta''}.$
\STATE {\bf if }($\Cbar_{\theta_a} \succ \Cbar_{\theta_c}$) Set $\theta_b = \theta_d$.
\STATE {\bf elseif }($\Cbar_{\theta_a} \prec \Cbar_{\theta_c} \succ \Cbar_{\theta_d}$) Set $\theta_b = \theta_d$.
\STATE {\bf elseif }($\Cbar_{\theta_c} \prec \Cbar_{\theta_d} \succ \Cbar_{\theta_e}$) Set $\theta_a = \theta_c$,  $\theta_b = \theta_e$.
\STATE {\bf elseif }($\Cbar_{\theta_d} \prec \Cbar_{\theta_e} \succ \Cbar_{\theta_b}$) Set $\theta_a = \theta_d$.
\STATE {\bf else } 
 Set $\theta_a = \theta_d$.
\ENDWHILE
\STATE\textbf{Output:} $\bmmbf, \Cbar,$ and $\bell$, where $\bmmbf = \mmbf_d$ ($\theta_d$), $\Cbar = {\Cbar}_{\theta_d},$ and $\bell := \langle \bmmbf, (tp, tn) \rangle = \langle \bmmbf, {\Cbar} \rangle$.
\ealgorithmic
\ealgorithm

\textbf{Algorithm \ref{bin-alg:linear}.} \emph{Maximizing quasiconcave metrics and finding supporting hyperplanes at the optimum:} Since $\phi$ is monotonically increasing in both \TPs and \TN, and $\mathcal C$ is convex, the maximizer must be on the upper boundary. 
Hence, we start with the interval $[\theta_a = 0, \theta_b = \frac{\pi}{2}]$ (Definition~\ref{def:boundary}). We divide it into four equal parts and set slopes using \eqref{set:lfpm} in line 4 (see  Figure~\ref{bin-fig:lin-fr}(b) for visual intuition). 
Then, we compute the Bayes classifiers using Proposition \ref{pr:bayeslinear} and the associated Bayes confusion matrices in line 5. We pose four pairwise queries to the oracle in line 6. 
Line 7 gives the default direction to binary search in case of out-of-order responses.\footnote{Due to finite samples, $\Ccal$'s boundary may have staircase-type bumps in practice. This may lead to out-of-order responses, even when the metric is unimodal \emph{w.r.t.} $\theta$.} In lines 8-12, we shrink the search interval by half based on oracle responses. We stop when the search interval becomes smaller than a given $\epsilon > 0$ (tolerance). Lastly, we output the slope $\bmmbf$, the Bayes confusion $\Cbar$, and the supporting hyperplane $\bell$ at that point.

\balgorithm[t]
\caption{Quasiconcave Metric Minimization}
\label{bin-alg:linearmin}
\balgorithmic[1]
\STATE Follow Algorithm~\ref{bin-alg:linear} except:
\STATE \textbf{Initialize:} $\theta_a = \pi$, $\theta_b = \frac{3\pi}{2}$.
\STATE \textbf{Invert Responses:} Replace oracle responses $C\prec C'$ with $C\succ C'$ and vice versa.
\ealgorithmic
\ealgorithm

\textbf{Algorithm 3.2.}
\emph{Minimizing quasiconvex metrics and finding supporting hyperplane at the optimum:} The same algorithm can be used for quasiconvex minimization with only two changes. First, we start with $\theta \in [\pi,\frac 32\pi]$, because the optimum will lie on the lower boundary $\partial\mathcal C_-$. Second, we check for $C\prec C'$ whenever Algorithm~\ref{bin-fig:lin-fr} checks for $C\succ C'$, and vice versa. Here, we output the counterparts, i.e., slope $\tmmbf$,  inverse Bayes Confusion matrix $\barbelow{C}$, and supporting hyperplane $\tell$.  

\vspace{-0.5cm}
\section{Metric Elicitation}
\label{bin-sec:me}
\vskip -0.25cm

In this section, we discuss how Algorithms~\ref{bin-alg:linear}, 3.2,  and~\ref{alg:grid-search} (discussed later) are used as subroutines to elicit LPMs and LFPMs. See Figure~\ref{fig:me} for a brief summary.

\vspace{-0.6cm}
\subsection{Eliciting LPMs}
\label{ssec:elicit_linear}
\vskip -0.2cm
Suppose that the oracle's metric is $ \varphi_{LPM} \ni \sphi = \smmbf$, where, WLOG, $\norm{\smmbf} =1$ and $m_0^* = 0$ (Section \ref{bin-sec:confusion}). 
Application of Algorithm~\ref{bin-alg:linear} to the oracle, who responds according to $\smmbf$, returns the maximizer and supporting hyperplane at that point. Since the true performance metric is linear, we take the elicited metric, $\hmmbf$, to be the slope of the resulting supporting hyperplane.

\begin{figure}[t] \centering \fbox{ \parbox{0.95\columnwidth}{ {\begin{small}
\textbf{LPM Elicitation} (True metric $\sphi = \smmbf$)\vspace{1pt}
\benumerate[noitemsep,nolistsep,leftmargin=*]
\item Run Algorithm~\ref{bin-alg:linear} to get $\Cbar^*$ and a hyperplane $\bell$.
\item Set the elicited metric to be the slope of $\bell$. \\
\eenumerate 
\vskip -.6 cm
\begin{tikzpicture}%
        \draw (6,0) +(0,0) -- +(9.6,0);
     \end{tikzpicture}

\textbf{LFPM Elicitation} (True metric $\sphi$)\vspace{1pt}
\benumerate[noitemsep,nolistsep,leftmargin=*]
\item Run Algorithm~\ref{bin-alg:linear} to get $\Cbar^*$, a hyperplane $\bell$, and SoE~\eqref{bin-eq:lin-fr-equi}.
\item Run Algorithm~3.2 to get $\barbelow{C}^*$, a hyperplane $\tell$, and SoE~\eqref{eq:lin-fr-equi-lower}.
\item Run the oracle-query independent Algorithm~\ref{alg:grid-search} to get the elicited metric, which satisfies both the SoEs.
\eenumerate
\vspace{2pt}
\end{small}
}
}}
\vskip -0.1cm
\caption{LPM and LFPM elicitation procedures.}
\label{fig:me}
\end{figure}

%% file: binary/algorithm_linearfrac.tex
\subsection{Eliciting LFPMs}
\label{ssec:elicit_linearfrac}

An LFPM is given by \eqref{linear-fractional}, where $p_{11}, p_{00}, q_{11}$, and $q_{00}$ are not simultaneously zero. Also, it is bounded over $\Ccal$. As scaling and shifting does not change the linear-fractional form, \emph{w.l.o.g.}, we may take $\phi(C) \in [0, 1] \, \forall  C \in \Ccal$ with positive numerator and denominator.
\bassumption
Let $\phi \in \varphi_{LFPM}$ \eqref{linear-fractional}. We assume that $p_{11}, p_{00} \geq 0$, $p_{11} \geq q_{11}$, $p_{00} \geq q_{00}$, $p_0 = 0$, $q_0 = (p_{11} - q_{11})\zeta + (p_{00} - q_{00})(1 - \zeta)$, and $p_{11} + p_{00} = 1$.
\label{assump:sufficient}
\eassumption

\bprop
The conditions in Assumption \ref{assump:sufficient} are sufficient for $\phi \in \varphi_{LFPM}$ to be bounded in $[0,1]$ and simultaneously monotonically increasing in TP and TN.
\vskip -0.1cm
\label{bin-prop:sufficient}
\eprop

The conditions in Assumption~\ref{assump:sufficient} 
are reasonable as we want to elicit any unknown bounded, monotonically increasing LFPM. To no surprise, examples outlined in \eqref{ex:lf-examples} and Koyejo et al. \cite{koyejo2014consistent} satisfy these conditions. 
We first provide intuition for eliciting LFPMs (Figure~\ref{fig:me}). We obtain two hyperplanes: one at the maximizer on the upper boundary, and other at the minimizer on the lower boundary. This results in two nonlinear systems of equations (SoEs) having only one degree of freedom, but they are satisfied by the true unknown metric. 
Thus, the elicited metric is one where solutions to the two systems match pointwise on the confusion matrices. 
Formally, suppose that the oracle's metric is:

\begin{align}
\sphi(C) = \frac{\spone TP +  \spzero TN}{\sqone TP +  \sqzero TN +  \sqnot}.
\end{align}
Let $\btau^*$ and $\barbelow{\tau}^*$ be the maximum and minimum value of $\sphi$ over $\Ccal$, respectively, i.e., 
\begin{equation}
\underline{\tau}^*\leq \phi^*(C)\leq \overline{\tau}^* \; \forall \, C\in \Ccal.
\end{equation}
Under Assumption  \ref{bin-as:eta}, we have a hyperplane 
\begin{align}
\bell_f^* := (\spone - \btau^*\sqone)tp +  (\spone - \btau^*\sqone)tn = \btau^*\sqnot 
\label{eq:lf-support}
\end{align}
touching the set $\Ccal$ only at $(\oline{TP}^*, \oline{TN}^*)$ on the upper boundary $\partial \Ccal_+$.  
Similarly, we have a hyperplane
\bequation
{\tell}^*_f := (\spone - \ttau^* \sqone)tp +  (\spzero - \ttau^* \sqzero)tn = \ttau^* \sqnot,
\label{eq:lf-lower-support}
\eequation
which touches the set $\Ccal$ only at $(\barbelow{TP}^*, \barbelow{TN}^*)$ on the lower boundary $\partial \Ccal_-$. 
To help with intuition, see Figure \ref{bin-fig:lin-fr}(c). 
Since LFPM is quasiconcave, Algorithm \ref{bin-alg:linear} returns a hyperplane
${\bell := \bmone tp + \bmzero tn = \oline{C}_0}$,
where $\oline{C}_0 = \bmone \oline{TP}^* + \bmzero \oline{TN}^*$. This is equivalent to $\bell_f^*$ up to a constant multiple; therefore, the true metric is the solution to the following non-linear SoE:
\begin{align}
\spone - \btau^*\sqone = \alpha \bmone,  \spzero - \btau^*\sqzero = \alpha \bmzero,
  \btau^*\sqnot = \alpha \oline{C}_0,
\end{align}
where $\alpha \geq 0$, because LHS and $\oline{m}$'s are non-negative. Additionally, we ignore the case when $\alpha = 0$, since this would imply a constant $\phi$. 
Next, we may divide the above equations by $\alpha > 0$ on both sides so that all the coefficients $\oline{p}^*$'s and $\oline{q}^*$'s are factored by $\alpha$. This does not change $\sphi$; thus, the SoE becomes:
\begin{align}
\ppone - \btau^*\pqone =  \bmone,  \ppzero - \btau^*\pqzero =  \bmzero, 
 \btau^*\pqnot =  \oline{C}_0.
\label{bin-eq:lin-fr-equi}
\end{align}
Notice that none of the conditions in Assumption \ref{assump:sufficient} are changed except $\ppone + \ppzero = 1$. However, we may still use this condition to learn a constant $\alpha$ times the true metric, which does not harm the elicitation problem. 

As LFPM is also quasiconvex, Algorithm~3.2 gives a hyperplane 
${{\tell} := {\barbelow{m}}_{11} tp + {\barbelow{m}}_{00} tn = {\barbelow{C}}_0},$
where ${\barbelow{C}}_0 = {\barbelow{m}}_{11} {\barbelow{TP}}^* + {\barbelow{m}}_{00} {\barbelow{TN}}^*$. This is equivalent to ${\tell}^*_f$ up to a constant multiple; thus, the true metric is also the solution of the following SoE:
\begin{align}
\spone - {\ttau}^* \sqone = \gamma {\barbelow{m}}_{11}, 
\spzero - {\ttau}^* \sqzero = \gamma {\barbelow{m}}_{00},
 {\ttau}^* \sqnot = \gamma \barbelow{C}_0,
\end{align}
where $\gamma \leq 0$ since LHS is positive, but $\barbelow{m}$'s are negative. Again, we may assume $\gamma < 0$. By dividing the above equations by $-\gamma$ on both sides, all the coefficients ${p}^*$'s and ${q}^*$'s are factored by $-\gamma$. This does not change $\sphi$; thus, the system of equations becomes the following:
\begin{ceqn}
\begin{align}
\pppone - {\ttau}^* \ppqone =  {\barbelow{m}}_{11}, 
\pppzero - {\ttau}^* \ppqzero =  {\barbelow{m}}_{00}, 
  {\ttau}^* \ppqnot =  \barbelow{C}_0.
\label{eq:lin-fr-equi-lower}
\end{align}
\end{ceqn}
\bprop
Under Assumption~\ref{assump:sufficient}, knowing $p_{11}'$ solves the system of equations~\eqref{bin-eq:lin-fr-equi} as follows:
\begin{align}
	p_{00}' &= 1 - p_{11}', \, \nonumber q_0' = \oline{C}_0\frac{P'}{Q'}, \nonumber \\
	q_{11}' &= (p_{11}' - \bmone)\frac{P'}{Q'}, \, q_{00}' = (p_{00}' - \bmzero)\frac{P'}{Q'}, 
\end{align}
where $P' = p_{11}'\zeta + p_{00}'(1 - \zeta)$ and $Q' = P' + \oline{C}_0 - \bmone\zeta -  \bmzero(1 - \zeta)$. 
\label{pr:solvesystem}
\eprop

\balgorithm[t]
\caption{Grid Search for Best Ratio}
\label{alg:grid-search}
\balgorithmic[1]
\STATE \textbf{Input:} $k, \Delta$. 
\STATE \textbf{Initialize:} $\sigma_{opt} = \infty, p_{11, opt}' = 0$.
\STATE Generate $C_1,...,C_k$ on $\partial C_+$ and $\partial C_-$ (Section \ref{bin-sec:confusion}).
\FOR {($\ppone = 0$; $\ppone \leq 1$; $\ppone = \ppone + \Delta$)}
\STATE Compute $\pphi$, $\ppphi$ using Proposition \ref{pr:solvesystem}. Compute array $r = [\frac{\pphi(C_1)}{\ppphi(C_1)},...,\frac{\pphi(C_k)}{\ppphi(C_k)}]$. Set $\sigma = \text{std}(r).$\\
\STATE {\bf if }($\sigma < \sigma_{opt}$) Set $\sigma_{opt} = \sigma$ and $p_{11, opt}' = \ppone$.
\ENDFOR
\STATE\textbf{Output:} $p_{11, opt}'$.
\ealgorithmic
\ealgorithm
Now assume we know $p_{11}'$. Using Proposition~\ref{pr:solvesystem}, we may solve the system~\eqref{bin-eq:lin-fr-equi} and obtain a metric, say $\phi'$. System~\eqref{eq:lin-fr-equi-lower} can be solved analogously, provided we know $p_{11}''$, to get a metric, say $\phi''$. Notice that when $\sfrac{\spone}{\spzero}=\sfrac{\ppone}{\ppzero}=\sfrac{\pppone}{\pppzero}$, then $\sphi (C) = \phi'(C)/\alpha = -\phi''(C)/\gamma$. This means that when the  true ratios of $p$'s are known, then $\phi'$, $\phi''$ are constant multiples of each other. So, to know the true $\ppone$ (or, $\pppone$) is to search the grid $[0,1]$ and select the one where the ratios of $\phi'$ and $\phi''$ are constant on a number of confusion matrices. Since we can generate many confusion matrices on $\partial \Ccal_+$ and $\partial \Ccal_-$ (vary $\delta$ in Definition~\ref{def:boundary}), we can estimate the ratio $p_{11}'$ to $p_{00}'$ using grid search based Algorithm~\ref{alg:grid-search}. We may then use Proposition~\ref{pr:solvesystem} for the output of Algorithm~\ref{alg:grid-search} and set the elicited metric $\hphi = \pphi$.
Note that Algorithm~\ref{alg:grid-search} is independent of oracle queries and easy to implement, thus it is suitable for the purpose.

%% file: binary/noise.tex
\vspace{-0.2cm}
\section{Guarantees}
\label{sec:noise}
In this section, we discuss guarantees for the elicitation procedures (Section~\ref{bin-sec:me}) in the presence of (a) confusion matrices' estimation noise from finite samples and (b) oracle feedback noise with the following notion that is borrowed from Definition~\ref{me-def:noise}.
\bdefinition
Oracle Feedback Noise $(\epsilon_\Omega\geq0)$: 
The oracle may provide wrong answers whenever 
$|\phi(C)-\phi(C')|<\epsilon_\Omega$. Otherwise, it provides correct answers. 
\edefinition
Simply put, if the confusion matrices are close as measured by $\phi$, then the oracle responses can be wrong. Moving forward to the guarantees, we make two assumptions which hold in most common settings.
\begin{assumption}\label{as:sup-norm-convergence}
	Let $\{\hat \eta_i(x)\}_{i=1}^{n}$ be a sequence of estimates of $\eta(x)$ depending on the sample size. 
	We assume that $\Vert\eta - \hat \eta_i\Vert_\infty \stackrel{P}{\to} 0$.
\end{assumption}
\begin{assumption}\label{as:low-weight-around-opt}
	For quasiconcave $\phi$, recall that the Bayes classifier 
	is of the form $h = \1[\eta(x)\geq \delta]$. Let $\oline{\delta}$ be the threshold that maximizes $\phi$. 
    We assume that the probability that $\eta(X)$ lies near $\oline{\delta}$ is bounded from below and above. Formally, 
    \bequation
    k_0 \nu \leq \Pmbb\left[(\oline{\delta}-\eta(X))\in[0,\nu]\right],
    \Pmbb\left[(\eta(X)-\oline{\delta})\in[0,\nu]\right]\leq k_1 \nu
    \eequation
    for any $0<\nu\leq\frac{2}{k_0}\sqrt{k_1\epsilon_\Omega}$ and some $k_1 \geq k_0 > 0$.
\end{assumption}
Assumption \ref{as:sup-norm-convergence} is arguably natural, as most estimation is parametric, where the function classes are sufficiently well behaved. Assumption \ref{as:low-weight-around-opt} ensures 
that near the optimal threshold $\oline{\delta}$, the values of $\eta(X)$ have bounded density. 
In other words, when $X$ has no point mass, the slope of $\eta(X)$ where it attains the optimal threshold $\oline{\delta}$ is neither vertical nor horizontal. 
We start with guarantees for the algorithms in their respective tasks.
\begin{theorem}\label{thm:quasi}
Given $\epsilon,\epsilon_\Omega \geq 0$ and a 1-Lipschitz metric $\phi$ that is monotonically increasing in TP, TN. If it is quasiconcave (quasiconvex) then Algorithm \ref{bin-alg:linear} (Algorithm~3.2) finds an approximate maximizer $\Cbar$ (minimizer $\barbelow{C}$). Furthemore, $(i)$ the algorithm returns the supporting hyperplane at that point, $(ii)$ 
the value of $\phi$ at that point is within $O(\sqrt{\epsilon_\Omega} +  \epsilon)$ of the optimum, and $(iii)$ the number of queries is $O(\log\frac1\epsilon)$.
\end{theorem}
\blemma
\label{lem:lower-bound}
Under our model, no algorithm can find the maximizer (minimizer) in fewer than~$O(\log\frac1\epsilon)$ queries.
\elemma
Theorem~\ref{thm:quasi} and Lemma~\ref{lem:lower-bound}, guarantee that Algorithm~\ref{bin-alg:linear} (Algorithm~3.2), for a quasiconcave (quasiconvex) metric, finds a confusion matrix and a hypeplane which is close to the true maximizer (minimizer) and its associated supporting hyperplane, using just the optimal number of queries. Further, since binary search always tends towards the optimal whenever responses are correct, the algorithms necessarily terminate within a confidence interval of the true maximizer. Thus, we can take $\epsilon$ sufficiently small so that the only error that arises is due to the feedback noise $\epsilon_\Omega$. Now, we present our main result which guarantees effective LPM elicitation. 
Guarantees in LFPM elicitation follow naturally as discussed in the proof of Theorem~\ref{thm:linear}  (Appendix~\ref{appendix:proofs}).
\btheorem\label{thm:linear}
Let $\varphi_{LPM} \ni \sphi = \smmbf$ be the true performance metric. Under Assumption~\ref{as:low-weight-around-opt}, given $\epsilon > 0$, LPM elicitation (Section~\ref{ssec:elicit_linear}) outputs a 
performance metric $\hphi = \hmmbf$, such that $\norm{\smmbf - \hmmbf}_\infty \leq \sqrt{2}\epsilon + \frac 2{k_0}\sqrt{2k_1\epsilon_\Omega}$.
\etheorem

So far, we assumed access to the confusion matrices. However, in practice, we need to estimate them using samples $\{(x_i, y_i)\}_{i=1}^{n}$. 
We now discuss robustness of the algorithms working with samples.
 Recall that, as a standard consequence of Chernoff-type bounds~\cite{boucheron2013concentration}, sample estimates of true-positive and true-negative are consistent estimators. 
Therefore, with high probability, we can estimate the confusion matrix within any desired tolerance, provided we have sufficient samples.  This implies that we can also estimate the $\phi$ values within any tolerance since LPM and LFPM are 1-Lipschitz due to \eqref{set:lfpm} and  Assumption \ref{assump:sufficient}, respectively. Thus, with high probability, the elicitation procedures gather correct oracle's preferences within feedback noise $\epsilon_\Omega$.
Further,  
we may prove the following lemma which allow us to control the 
 error in optimal classifiers from using the  
 estimated $\hat\eta(x)$ rather than the true $\eta(x)$.

\begin{lemma}\label{lem:sample-Cs-optimize-well}
Let $h_{\theta}$ and $\hat h_{\theta}$ be two classifiers estimated using $\eta$ and $\hat\eta$, respectively. 
Further, let ${\oline{\theta}}$ be such that $h_{{\oline{\theta}}} = \argmax_{\theta}\phi(h_{\theta})$. Then
	${\Vert C(\hat h_{{\oline{\theta}}}) - C(h_{{\oline{\theta}}}) \Vert_\infty=O( \Vert {\hat\eta}_n-\eta\Vert_\infty})$.
\end{lemma}
The errors due to using $\hat\eta$, instead of true $\eta$ may propel in the results discussed earlier, however, only in the bounded sense. This shows that our elicitation approach is robust to feedback and finite sample noise. 

%% file: binary/experiments.tex
\vspace{-0.5cm}
\section{Experiments}
\label{experiments}
\vskip -0.1cm

In this section, we empirically validate the theory and investigate the sensitivity due to sample estimates.\footnote{A subset of results is shown here. Please refer Appendix~\ref{appendix:experiments} for extended set of results.}

\vspace{-0.5cm}
\subsection{Synthetic Data Experiments}
\label{bin-ssec:theoryexp}

\begin{table}
	\caption{LPM elicitation at tolerance $\epsilon = 0.02$ radians. 
	}
	\label{bin-tab:LPMtheory}
	\vspace{-0.5cm}
	\begin{center}
		\begin{small}
				\begin{tabular}{|c|c|c|c|}
					\hline
					  $\sphi = \smmbf$ & $\hphi = \hmmbf$ & $\sphi = \smmbf$ & $\hphi = \hmmbf$ \\ \hline 
					(0.98,0.17) & (0.99,0.17) & (-0.94,-0.34) & (-0.94,-0.34) \\
					(0.64,0.77) & (0.64,0.77) & (-0.50,-0.87) & (-0.50,-0.87)  \\
					\hline
				\end{tabular}
		\end{small}
	\end{center}
\end{table}

\begin{table*}
	\caption{LFPM Elicitation for synthetic distribution (Section \ref{bin-ssec:theoryexp}) and Magic (\textsc{M}) dataset  (Section \ref{bin-ssec:realexp}). $\alpha$ and $\sigma$ are the mean and standard deviation of $\sfrac{\hphi}{\sphi}$ evaluated over a subset of confusion matrices used in Algorithm~\ref{alg:grid-search}.   
}
	\label{tab:LFPMtheoryreal}
	\vspace{-0.5cm}
	\begin{center}
		\begin{small}
            \resizebox{\textwidth}{!}{%
			\begin{tabular}{|c|c|c|c|c|c|c|}
				\hline
                
				True Metric & \multicolumn{3}{|c|}{Results on Synthetic Distribution (Section \ref{bin-ssec:theoryexp})} & \multicolumn{3}{|c|}{Results on Real World Dataset \textsc{M}  (Section \ref{bin-ssec:realexp})} \\ \hline
				$(\spone, \spzero),  (\sqone, \sqzero, \sqnot)$ & $(\hpone, \hpzero),  (\hqone, \hqzero, \hqnot)$ & $\alpha$ & $\sigma$ & $(\hpone, \hpzero),  (\hqone, \hqzero, \hqnot)$ & $\alpha$ & $\sigma$ \\
				\hline
				(1.00,0.00),(0.50,-0.50,0.50) & (1.00,0.00),(0.25,-0.75,0.75) & 0.92 & 0.03  & (1.00,0.00),(0.25,-0.75,0.75) & 0.90 & 0.06 \\
				(0.20,0.80),(-0.40,-0.20,0.80) & (0.12, 0.88),(-0.43, 0.002, 0.71) & 1.02 & 0.006 & (0.19,0.81),(-0.38,-0.13,0.70) & 1.02 & 0.004  \\
				\hline
			\end{tabular}}
		\end{small}
	\end{center}
\end{table*}

\begin{figure*}[t]
	\centering 
	\subfigure[Table \ref{tab:LFPMtheoryreal}, line 1, col 2]{
		{\includegraphics[width=3.6cm]{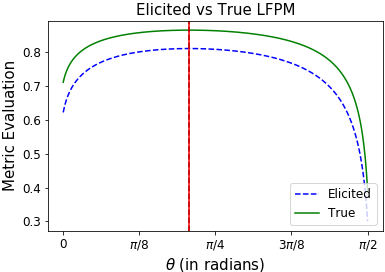}}
		\label{fig:lfpm_1}
	}
	\subfigure[Table \ref{tab:LFPMtheoryreal}, line 2, col 2]{
		{\includegraphics[width=3.6cm]{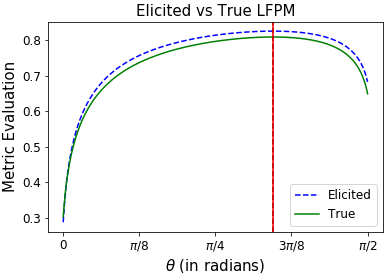}}
		\label{fig:lfpm_6}
	}
	\subfigure[Table \ref{tab:LFPMtheoryreal}, line 1, col 5]{
		{\includegraphics[width=3.6cm]{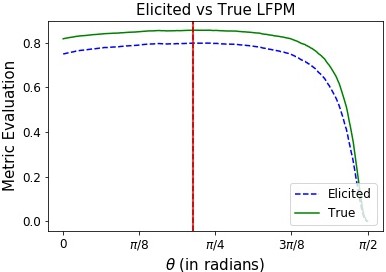}}
    \label{fig:lfpm_1_magic}
    }
    \subfigure[Table \ref{tab:LFPMtheoryreal}, line 2, col 5]{
		{\includegraphics[width=3.6cm]{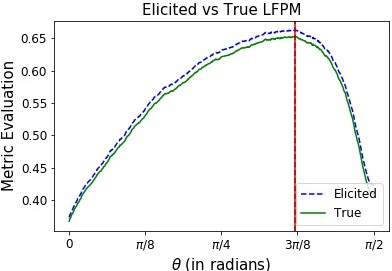}}
		\label{fig:lfpm_6_magic}
	}
	\vskip -0.2cm
	\caption{True (solid green) and elicited (dashed blue) LFPMs  for synthetic distribution and dataset M from Table~\ref{tab:LFPMtheoryreal}. 
	The solid red and coinciding dashed black vertical lines are \emph{argmax} of the true and elicited metric, respectively. 
}
	\label{fig:lfpm-theory}
\end{figure*}

We assume a joint probability for $\Xcal = [-1,1]$ and $\Ycal = \{0, 1\}$ given by $f_X = \Umbb[-1,1]$ and $\eta(x) = \frac{1}{1 + e^{ax}}$, where $\Umbb[-1,1]$ is the uniform distribution on $[-1, 1]$, and $a$ is a parameter controlling the degree of noise in the labels. We fix $a = 5$ in our experiments. To verify LPM elicitation, we first define a true metric $\sphi$. 
This specifies the query outputs in line 6 of Algorithm \ref{bin-alg:linear} (Algorithm~3.2). Then we run LPM elicitation procedure (Section~\ref{ssec:elicit_linear}) to check whether or not we compute the same metric. Some results are shown in Table \ref{bin-tab:LPMtheory}. We elicit the true metrics even for $\epsilon = 0.02$ radians.

Next, we elicit LFPM. We define a true metric $\sphi$ by $\{(\spone, \spzero),  (\sqone, \sqzero, \sqnot)\}$. Then we follow the LFPM elicitation procedure (Section~\ref{ssec:elicit_linearfrac}), where Algorithms~\ref{bin-alg:linear} and 3.2 are run with $\epsilon=0.05$ and Algorithm~\ref{alg:grid-search} is run with $k = 2000$ and $\Delta = 0.01$. The elicited metric $\hat{\phi}$ is denoted by $\{(\hpone, \hpzero), (\hqone, \hqzero, \hqnot)\}$ and presented in Table~\ref{tab:LFPMtheoryreal} (Column 2). 
We also present mean ($\alpha$) and standard deviation ($\sigma$) of the ratio of the elicited metric $\hphi$ to the true metric $\sphi$ over a subset of confusion matrices (columns 3 and 4). 
For improved comparisons, Figure~\ref{fig:lfpm-theory} shows the true and elicited metrics evaluated on selected pairs of $(TP, TN) \in \partial\Ccal_+$. The metrics are plotted together after sorting the slope parameter $\theta$. Clearly, the elicited metric is a constant multiple of the true metric. We also see that the \emph{argmax} of the true and elicited metric coincide, thus validating Theorem~\ref{thm:quasi}.

\vspace{-0.5cm}
\subsection{Real-World Data Experiments}
\label{bin-ssec:realexp}

Now, we validate the elicitation procedures with two real-world datasets. 
The datasets are: (a) Breast Cancer (BC) Wisconsin Diagnostic dataset \cite{street1993nuclear} containing 569 instances, and (b) Magic (M) dataset \cite{dvovrak2007softening} containing 19020 instances. For both the datasets, we standardize the features and split the data into two parts $\Scal_1$ and $\Scal_2$. On $\Scal_1$, we learn the estimator $\hat{\eta}$ using regularized logistic regression model. 
We use $\Scal_2$ for making predictions and computing sample confusion matrices. 

We randomly selected twenty-eight LPMs by choosing $\theta^*$ ($\smmbf)$. 
We then used Algorithm~\ref{bin-alg:linear} (Algortihm~3.2) with different tolerance $\epsilon$ and for different datasets and recovered the estimate $\hat{\mmbf}$ using LPM elicitation. In Table~\ref{tab:app:LPMreal} of Appendix~\ref{appendix:experiments}, we report the proportion of the number of times when our procedure failed to recover the true ${\smmbf}$. We see improved elicitation for dataset $M$, suggesting that ME improves with larger datasets. In particular, for dataset $M$, we elicit all the metrics within threshold $\epsilon = 0.11$ radians. We also observe that $\epsilon = 0.02$ is an overly tight tolerance for both the datasets leading to many failures. This is because the elicitation routine gets stuck at the closest achievable confusion matrix from finite samples, which need not be optimal within the given (small) tolerance. 

Next, we evaluate LFPM elicitation using dataset $M$. We define the same true metrics and follow the same LFPM elicitation process as defined in Section~\ref{bin-ssec:theoryexp}. 
In Table \ref{tab:LFPMtheoryreal} (columns 5, 6, and 7), we present the elicitation results along with mean $\alpha$ and standard deviation $\sigma$ of the ratio of the elicited metric and the true metric. 
We also show the true and elicited metrics evaluated on the selected pairs of $(TP, TN) \in \partial\Ccal_+$ in Figure~\ref{fig:lfpm-theory}, ordered by the parameter $\theta$. We see that the elicited metrics are equivalent to the true metrics up to a constant. 

%% file: binary/conclusion.tex
\vspace{-0.5cm}
\section{Related Work}
\label{sec:discussion}

Our work may be compared to ranking from pairwise comparisons \cite{wauthier2013efficient}. However, we note that our results depend on novel geometric ideas on the space of confusion matrices. Thus, instead of a ranking problem, we show that ME in standard models can be reduced to just finding the maximizer (and minimizer) of an unknown function which in turn yields the true metric -- resulting in low query complexity. A direct ranking approach adds unnecessary complexity to achieve the same task. Further, in contrast to our approach, most large margin ordinal regression based ranking \cite{herbrich2000large} fail to control which samples are queried. There is another line of work, which actively controls the query samples for ranking, e.g., \cite{jamieson2011active}. However, to our knowledge, this requires that the number of objects is finite and finite dimensional -- thus cannot be directly applied to ME without significant modifications, e.g. exploiting confusion matrix properties, as we have. Learning a performance metric which correlates with human preferences has been studied before \cite{janssen2007meta, peyrard2017learning}; however, these studies learn a regression function over some predefined features which is fundamentally different from our problem. Lastly, while \cite{caruana2004data, ferri2009experimental} address how one might qualitatively choose between metrics, none addresses our central contribution -- a principled approach for eliciting the ideal metric from user feedback.

\vspace{-0.5cm}
\section{Concluding Remarks}
\label{sec:conclusion}
 
We conceptualize \emph{metric elicitation} for the binary classification setup and elicit linear and linear-fractional metrics using preference feedback over pairs of classifiers. We propose provably query efficient and robust algorithms to elicit metrics that exploit key geometric properties of the set of confusion matrices associated with the binary classification tasks. 

%% file: multiclass/header.tex
\chapter{Multiclass Classification Performance Metric Elicitation}
\label{chp:multiclass}

\input{multiclass/introduction}
\input{multiclass/background}
\input{multiclass/confusion}
\input{multiclass/d-elicitation}
\input{multiclass/elicitation}
\input{multiclass/extensions}
\input{multiclass/guarantees}
\input{multiclass/experiments}
\input{multiclass/discussion}
\input{multiclass/relatedwork}
\input{multiclass/conclusion}

%% file: multiclass/introduction.tex

Conceptually, Metric Elicitation (ME) is applicable to any learning setting. However, the  proposed methods in the previous chapter were limited to eliciting binary classification performance metrics. This chapter extends the previous work by proposing ME strategies for the more complicated multiclass classification setting -- thus significantly increasing the use cases for ME. 
Similar to the binary case, we consider the most common families of performance metrics which are functions of the confusion matrix~\cite{narasimhan2015consistent}, which is our choice of measurement space in this chapter; however, in this case, the elements of the confusion matrix summarize multiclass error statistics.

In order to perform efficient multilcass performance metric elicitation, we study novel geometric properties of the space of multiclass confusion matrices. 
Our analysis reveals that due to structural differences between the space of binary and multiclass confusions, we can not trivially extend the elicitation procedure used for binary to the multiclass case.
Instead, we provide novel strategies for eliciting linear functions of the multiclass confusion matrix and extend elicitation to more complicated yet popular functional forms such as linear-fractional functions of the confusion matrix elements~\cite{narasimhan2018learning}. Specifically, the elicitation procedures involve binary-search type algorithms that are robust to both finite sample and oracle feedback noise. In addition, the proposed methods can be applied either by querying pairwise classifier preferences or pairwise confusion matrix preferences.  

\begin{table}[t]
    \centering
    \caption{The Bayes Optimal (BO) and Restricted-Bayes Optimal (RBO) entities.}
    \begin{tabular}{|c|c|}\hline
      Name & Definition \\ \hline
        BO confusion $\cmbfbar$  over a subset $\Scal \subseteq \Ccal$ 
        & $ \displaystyle\argmax_{\cmbf \in \Scal \subseteq \Ccal}\phi(\cmbf)$ \\
        RBO classifier $\hbar_{k_1, k_2}$
        & $\displaystyle\argmax_{h \in \Hcal_{k_1, k_2}}\psi(\dmbf(h))$ \\
        RBO diagonal confusion $\dmbfbar_{k_1, k_2}$ & $\displaystyle\argmax_{\dmbf \in \Dcal_{k_1, k_2}}\psi(\dmbf)$ \\
        \hline
      \end{tabular}
    \label{tab:bayes}
\end{table}

In summary, our main contributions are novel query efficient metric elicitation algorithms for multiclass classification. We first study ME for linear functions of the confusion matrix and then discuss extensions to more complicated functional forms such as the linear-fractional and arbitrary monotonic functions of the confusion matrix. Lastly, we show that the proposed procedures are robust to finite sample and feedback noise, thus are useful in practice. 
All the proofs in this chapter are provided in Appendix~\ref{apx:multiclass}.

\textbf{Notation.} Matrices and vectors are denoted by bold upper case and bold lower case letters, respectively.  Recall that, given a matrix $\Ambf$, $\offdiag(\Ambf)$ returns a vector of off-diagonal elements of $\Ambf$ in row-major form, and $\diag(\Ambf)$ returns a vector of diagonal elements of $\Ambf$. $\norm{\cdot}_1$, $\norm{\cdot}_2$,  and $\norm{\cdot}_\infty$ denote the $\ell_1$-norm,  $\ell_2$-norm, and $\ell_\infty$-norm, respectively. 

%% file: multiclass/background.tex
\section{Preliminaries}
\label{sec:preliminaries}

The standard multiclass classification setting comprises $k$ classes with $X \in \Xcal$ and $Y \in [k]$ representing the input and output random variables, respectively. We have access to a dataset of size $n$ denoted by $\{(\xmbf, y)_i\}_{i=1}^n$, 
generated \emph{iid} from a distribution $ \Pmbb(X, Y)$. 
Let $\eta_i(\xmbf) = \Pmbb(Y = i | X = \xmbf)$ and $\zeta_i = \Pmbb(Y = i)$ for $i \in [k]$ be the conditional and the unconditional probability of the $k$ classes, respectively. 
Let $\Hcal = \{h : \Xcal \rightarrow \Delta_k\}$ be the set of all classifiers. A confusion matrix for a classifier $h$ is denoted by $\Cmbf(h, \Pmbb) \in \Rmbb^{k \times k}$, where its elements are given by:
\vspace{-0.1cm}
\begin{align}
	C_{ij}(h, \Pmbb) = \Pmbb(Y = i, h = j) \quad \text{for} \; i, j \in [k].
	\label{mult-eq:components}
\end{align}
\vskip -0.1cm
Under the population law $\Pmbb$, it is useful to keep the following decomposition in mind: 
\vspace{-0.1cm}
\begin{equation}
    \Pmbb(Y = i, h = i) = \zeta_i - \Pmbb(Y = i, h \neq i) \implies \Cii(h, \Pmbb) = \zeta_i - \sum_{j=1,j\neq i}^k \Cij(h, \Pmbb).
    \label{mult-eq:decomp}
\end{equation} 
\vskip -0.1cm
Using this decomposition, any confusion matrix is uniquely represented by its $q \coloneqq (k^2 - k)$ off-diagonal elements. Hence, we will represent a confusion matrix $\Cmbf(h, \Pmbb)$ by a vector $\cmbf(h, \Pmbb) = \offdiag(\Cmbf(h, \Pmbb))$, and interchangeably refer the confusion matrix as a vector of \emph{`off-diagonal confusions'}. The space of off-diagonal confusions is denoted by 
\bequation
\Ccal = \{\cmbf(h, \Pmbb) = \offdiag(\Cmbf(h, \Pmbb)) : h \in \Hcal \}.
\eequation
For clarity, we will suppress the dependence on $\Pmbb$ and $h$ if it is clear from the context.

Performance of a classifier is often determined by just the misclassification and not the type of misclassification, especially when the number of classes is large. Therefore, we will also consider metrics that only depend on correct and incorrect predictions, namely $\Pmbb(Y = i, h = i)$ and $\Pmbb(Y = i, h \neq i)$. 
Following the decomposition in~\eqref{mult-eq:decomp}, such metrics require only the diagonal elements of the original confusion matrices. Given a confusion  matrix $\Cmbf$, we will denote its diagonal by $\dmbf = \diag(\Cmbf)$ and refer it as the vector of \emph{`diagonal confusions'}. The space of diagonal confusions is represented by 
\bequation
\Dcal = \{\dmbf = \diag(\Cmbf(h)) : h \in \Hcal\}.
\eequation

Let $\phi : [0, 1]^{q}  \rightarrow \Rmbb$ and $\psi : [0, 1]^{k} \rightarrow \Rmbb$ be the performance metrics for a classifier $h$ determined by its corresponding off-diagonal and diagonal  confusion entries $\cmbf(h)$ and $\dmbf(h)$, respectively. Without loss of generality (w.l.o.g.), we assume the metrics $\phi$ and $\psi$ are utilities so that larger values are preferred. Furthermore, the metrics are scale invariant as global scale does not affect the learning problem~\cite{narasimhan2015consistent}. 
For this chapter, we assume the following regularity assumption on the data distribution.
\bassumption\label{assumption:eta}
We assume that the functions $g_{ij}(r) = \Pmbb\left[\frac{\eta_i(X)}{\eta_j(X)} \geq r\right] \, \forall \, i, j \in [k]$ are continuous and strictly decreasing for $r \in [0, \infty)$. 
\label{mult-as:eta}
\eassumption
Intuitively, this weak assumption ensures that when the cost or reward tradeoffs for the classes change, the preferred confusions for those tradeoffs also change (and vice-versa). 

\vspace{-0.3cm}
\subsection{Bayes Optimal and Restricted Bayes Optimal Confusions and Classifiers}
\label{ssec:bayes}

As illustrated in Table~\ref{tab:bayes}, the Bayes Optimal (BO) confusion $\cmbfbar$ represents the optimal value of the off-diagonal confusions according to the metric $\phi$ over a subset $\Scal\subseteq \Ccal$. 
This is analogously defined for $\psi$ and $\Dcal$. The Restricted Bayes Optimal (RBO) entities are of interest for diagonal metrics $\psi$, and indicate the case where classifiers are `restricted' to predict only classes $k_1, k_2 \in [k]$. Thus $\Hcal_{k_1, k_2}$ and $\Dcal_{k_1, k_2}$ denote the space of classifiers which exclusively predict either $k_1$ or $k_2$ and the associated space of diagonal confusions, respectively. 
Note that for such restricted classifiers $h$, $C_{ii}(h) = d_i(h)$ evaluates to zero at every index $i \neq k_1, k_2$.

\vspace{-0.3cm}
\subsection{Performance Metrics}
\label{ssec:metrics}
We first discuss elicitation for the following two major types of metrics.

\bdefinition Diagonal Linear Performance Metric (DLPM): We denote this family by $\varphi_{DLPM}$. Given $ \ambf \in \Rmbb^{k}$ such that $\norm{\ambf}_1=1$ ( w.l.o.g., due to scale invariance), the metric is defined as: 
\vspace{-0.1cm}
\bequation
\psi(\dmbf) \coloneqq \inner{\ambf}{\dmbf}.
\eequation
This is also called weighted accuracy~\cite{narasimhan2015consistent, hiranandani2021optimizing} and focuses on correct classification.
\label{def:d-linear}
\edefinition

\bdefinition Linear Performance Metric (LPM): We denote this family by $\varphi_{LPM}$. Given $\ambf \in \Rmbb^{q}$ such that $\norm{\ambf}_2 =1$ (w.l.o.g., due to scale invariance), the metric is defined as: 
\vspace{-0.1cm}
\bequation
\phi(\cmbf) \coloneqq \inner{\ambf}{\cmbf}.
\eequation
Cost-sensitive linear metrics belong to $\varphi_{LPM}$~\cite{abe2004iterative} and focus on the types of misclassifications. 
\label{def:linear}
\vskip -0.4cm
\edefinition

\vspace{-0.1cm}
The difference of norms in the definitions is only for simplicity of exposition and chosen to best complement the underlying metric elicitation algorithm and vice-versa. 
Moreover, notice that the elements of diagonal confusions ($\dmbf$'s) and off-diagonal confusions ($\cmbf$'s) reflect correct and incorrect classification, respectively. 
Thus, according to standard practice, w.l.o.g., we focus on eliciting monotonically increasing DLPMs and monotonically decreasing LPMs in their respective arguments. 

\vspace{-0.3cm}
\subsection{Metric Elicitation; Problem Setup}
\label{ssec:me}
This section describes the problem of \emph{Metric Elicitation} and the associated \emph{oracle query}. Our definitions follow from Chapter~\ref{chp:me}, extended so the confusion elements and the performance metrics correspond to the multiclass classification setting. The following definitions hold analogously for the diagonal case by replacing $\phi, \cmbf$ and $\Ccal$ by $\psi, \dmbf$, and $\Dcal$, respectively.

\bdefinition
[Oracle Query] Given two classifiers $h, h'$ (equivalent to off-diagonal confusions $\cmbf, \cmbf'$ respectively), a query to the Oracle (with metric $\phi$) is represented by:
\begin{align}
\Gamma(h, h'\,;\,\phi) = \Omega(\cmbf, \cmbf'\,;\,\phi) &= \1[\phi(\cmbf) > \phi(\cmbf')] =: \1[\cmbf \succ \cmbf'],
\end{align}
where $\Gamma: \Hcal \times \Hcal \rightarrow \{0,1\}$ and $\Omega: \Ccal \times \Ccal \rightarrow \{0, 1\}$. The query asks whether $h$ is preferred to $h'$ (equivalent to $\cmbf$ is preferred to $\cmbf'$), as measured by $\phi$. 
\label{def:query}
\edefinition
\vskip -0.1cm
We elicit metrics which are functions of the confusion matrix, thus comparison queries using classifiers are indistinguishable from comparison queries using confusions.
Henceforth, for simplicity of notation, we denote any query as confusions based  query. Next, we formally state the ME problem.

\bdefinition [Metric Elicitation with Pairwise Queries (given $\{(\xmbf,y)_i\}_{i=1}^n$)] Suppose that the oracle's (unknown) performance metric is $\phi$.  Using oracle queries of the form $\Omega(\cmbfhat, \cmbfhat')$, where $\cmbfhat, \cmbfhat'$ are the estimated off-diagonal confusions from samples, recover a metric $\hphi$ such that $\Vert\phi - \hphi\Vert < \kappa$ under a  suitable norm $\Vert \cdot \Vert$ for sufficiently small error tolerance $\kappa > 0$.
\label{def:me}
\edefinition

The performance of ME is evaluated both by the fidelity of the recovered metric and the query complexity. 
Given the formal definitions, we can now proceed. As is standard in the decision theory literature~ \cite{koyejo2015consistent, hiranandani2018eliciting}, we present our ME solution by first assuming access to population quantities such as the population confusions $\cmbf(h, \Pmbb)$, then examine practical implementation by considering the estimation error from finite samples e.g. with empirical confusions $\cmbfhat(h, \{(\xmbf,y)_i\}_{i=1}^n)$. 

%% file: multiclass/confusion.tex
\vspace{-0.5cm}
\section{Geometry and Parametrizations of the Query Spaces}
\label{mult-sec:confusion}
For any query based approach, it is important to understand the structure of the query space. 
Thus, we first study the properties of the query spaces and then develop parametrizations required for efficient elicitation. Readers may find these properties independently useful in other applications as well. 
\vspace{-0.5cm}
\subsection{Geometry of the space of diagonal confusions $\Dcal$ and parametrization of its boundary}
\label{ssec:dlpmproperties}

Let $\vmbf_i \in \Rmbb^k$ for $i \in [k]$ be the vectors with $\zeta_i$ at the $i$-th index and zero everywhere else. 
Notice that $\vmbf_i$'s are the diagonal confusions of the trivial classifiers predicting only class $i$ on the entire space $\Xcal$.

\bprop
[Geometry of $\Dcal$ -- Figure~\ref{fig:lin-fr}~(a)] Under Assumption~\ref{mult-as:eta}, the space of diagonal confusions $\Dcal$ is strictly convex, closed, and contained in the box $[0, \zeta_1]\times\cdots\times[0, \zeta_k]$. The diagonal confusions $\vmbf_i \, \forall \, i \in [k]$ are the only vertices of $\Dcal$. Moreover, for any $k_1, k_2 \in [k]$, the 2-dimensional $(k_1, k_2)$ axes-aligned face of $\Dcal$ is $\Dcal_{k_1, k_2}$ (Figure~\ref{fig:lin-fr}~(b)), which is equivalent  to the space of binary classification confusion matrices confined to classes $k_1, k_2$. In particular, $\Dcal_{k_1, k_2}$ is strictly convex.
\label{prop:Cd}
\eprop

\begin{figure*}[t]
	\centering
	\begin{tikzpicture}[scale = 1.2]
    

    	\begin{scope}[shift={(-5.0,0)},scale = 0.5]\scriptsize
    	\def\r{0.1};
    	\def\s{0.06};
	
	\draw[thick] (0,3.5) .. controls (4,3) and (5,3) .. (7,0);
    \draw[thick] (7,0)	.. controls (2,0.5) and (1,0.5) .. (0,3.5);
    \draw[thick] (0, 3.5)	.. controls (0.5,3.75) and (2,6) .. (2.8,4);
    \draw[dashed] (0,3.5) .. controls (0.5,2) and (2,3.25) .. (2.8,4);
    \draw[thick] (2.8,4)	.. controls (5,3.5) and (7,3) .. (7,0);
    \draw[dashed] (2.8,4) .. controls (2,2) and (2,1) .. (7,0);

    \draw[-latex] (0,-.5)--(0,4.5); 
    \draw[-latex] (-.5,0)--(8,0);
    \draw[-latex] (-0.35,-0.5)--(3.5,5);
    
    \node[above] at (7.5,0) {$d_{1}$};
    \node[left] at (0,4.5) {$d_{2}$};
    \node[left] at (3.3, 5) {$d_{3}$};
    
    
    \node[below] at (7,-.25) {
        $\vmbf_1 =  (\zeta_1, 0, 0)$};
    \node[left] at (-0.25,3.5) {$\vmbf_2 =$};
    \node[left] at (-0.25,3.0) {$(0, \zeta_2, 0)$};
    \node[right] at (3.0,4.25) {$\vmbf_3 = (0, 0, \zeta_3)$};
    
    \coordinate (C*) at (4.55,2.455);    
    \coordinate (C1) at (3.5,2.95);
    \coordinate (C2) at (5.3,1.6);
    
	\coordinate (Cent) at (3.65,2.25);    
    
    \coordinate (Ct) at (1.6,0.95);
    \coordinate (Ct1) at (2.64,0.5);
    \coordinate (Ct2) at (0.8,1.7); 
    \coordinate (C+) at (2.15,2.85);
    
    \coordinate (C-) at (1.15,2.25);
    
    \coordinate (labelleft) at (3.65, -1.10); 
    \node at (labelleft) {{\normalsize{(a)}}};
    
    \fill[color=black] 
    		(0,3.5) circle (\r)
    		(7,0) circle (\r)
    		(2.8, 4) circle (\r);
        
    
    	\clip (-0.2,-0.2) rectangle (7,4.5);   
    
    
    
	
    
    
    

    \end{scope}

    
    \begin{scope}[shift={(-0.2,0)},scale = 0.5]\scriptsize
    	\def\r{0.1};
    	\def\s{0.06};
	
	\draw[thick] (0,3.5) .. controls (4,3) and (5,3) .. (6,0)
    	.. controls (2,0.5) and (1,0.5) .. (0,3.5);
    \draw[-latex] (0,-.5)--(0,4.5); 
    \draw[-latex] (-.5,0)--(7,0);
    \node[left] at (0,4.5) {$d_{k_2}$};
    \node[below] at (7.5,0) {$d_{k_1}$};
    \draw (6,0) +(0,0.25) -- +(0,-.25);
    \draw (0,3.5) +(.25,0) -- +(-.25,0);
    
    \node[below] at (6,-.25) {$(\zeta_{k_1}, 0)$};
    \node[left] at (-0.25,3.5) {$(0, \zeta_{k_2})$};
    
    \coordinate (C*) at (4.55,2.455);    
    \coordinate (C1) at (3.75,2.9);
    \coordinate (C2) at (5.3,1.6);
    
	\coordinate (Cent) at (3.15,1.75);    
    
    \coordinate (Ct) at (1.60,1.05);
    \coordinate (Ct1) at (2.64,0.5);
    \coordinate (Ct2) at (0.8,1.7); 
    \coordinate (C+) at (2.15,2.85);
    
    \coordinate (C-) at (1.15,2.25);
    
    \coordinate (labelmiddle) at (3.15, -1.10);
    \node at (labelmiddle) {{\normalsize{(b)}}};
    
    \fill[color=black] 
    		(0,3.5) circle (\r)
    		(6,0) circle (\r)
        (Cent) circle (\r);
    
    	\clip (-0.2,-0.2) rectangle (7,4.5);   
    
    
    
	
	\draw(6,0) +(0,10) -- +(0,-20);
	\draw(0,3.5) +(10,0) -- +(-20,0);    
    
    
    \node at (Cent) {{$(\frac{\zeta_{k_1}}{2},\, \frac{\zeta_{k_2}}{2})$}};
    
    \node[right] at (C1) {\tiny{$\partial\Dcal^+_{k_1, k_2}$}};
    \node[below] at (Ct) {\tiny{$\partial\Dcal^-_{k_1, k_2}$}};
    
    \end{scope}

	
   
    
    
    
    

    
    \begin{scope}[shift={(4.3,0)},scale = 0.5]\scriptsize
    
    \def\r{0.12};
    
    \coordinate (C*) at (3.8,3.05);
    \coordinate (l*) at (0.95,4.3);
    \coordinate (Ct) at (2.7,0.3);
    \coordinate (lt) at (-0.40,1.20);
    \coordinate (Cent) at (3,1.75);
    \coordinate (Space) at (0.50,0.1);
    \coordinate (Sphere) at (5,1.50);
    \coordinate (Lambda) at (3.75,2);
    \coordinate (C) at (1.525,1.35);
    \coordinate (fC) at (0.90,1.15);
    
    \coordinate (u1) at (-0.30, 0.10);
    \coordinate (uextra12) at (1.25, 3.25);
    \coordinate (u2) at (4, 4.5);
    \coordinate (uextra2k) at (5.75, 1);
    \coordinate (uk) at (3.75, -0.5);
    
    \coordinate (labelright) at (3,-1.10);
    \node at (labelright) {\normalsize{(c)}};
    
    \draw (C) -- (1.25,0.70);
    \draw (Cent) -- (4.5, 1.75);
    
    \fill[color=black] 
            (Cent) circle (0.08)
            (C*) circle (\r)
            (Ct) circle (\r)
            (C) circle (0.04)
            
            (u1) circle (0.08)
            (u2) circle (0.08)
            (uk) circle (0.08);
    
    \draw (u1) .. controls (-0.2,1) and (0.8, 2.75) .. (uextra12) 
    -- (u2) .. controls (7,4) and (6.5,2) .. (uextra2k) -- (uk) .. controls  (2.5,-0.75) and (-0.1,-0.5) .. (u1);
    	
    \draw[thick] (Cent) circle (1.5cm);
    
    \node[right] at (l*) {$\barbelow{\ell}^\ast$};
    \node[above left] at (lt) {$\bell^\ast$};
    \node at (Space) {\large{$\Ccal$}};
    \node at (Sphere) {\large{$\Scal_\lambda$}};
    \node at (Lambda) {$\lambda$};
    \node[right] at (C) {\tiny$\cmbf$\normalsize};
    \node at (fC) {\tiny{$f^*_{(\cmbf)}$}\normalsize};
    
    \draw (C*) +(-2.3,1.7) -- +(2.5,-1.9);
    \draw[-latex] (C*) +(0,0) -- +(0.68,0.92);
    \node[below left] at ($(C*)+(0.15,0)$) {$\barbelow{\cmbf}^*$};
    \node[below right] at ($(C*)+(0.34,0.46)$) {$-\nabla \phi^*$};
   	
    \draw (Ct) +(-4,1) -- +(3,-0.75);
    \node[below left] at ($(Ct)+(0,0.15)$) {$\cmbfbar^*$};
    \node[below] at (Cent) {$\ombf$};
    
    \node[above left] at (u1) {{$\umbf_1$}};
    \node[above] at (u2) {{$\umbf_2$}};
    \node[below right] at (uk) {{$\umbf_k$}};
    
    
    
    
    

    
    
   	
    
     \end{scope}
\end{tikzpicture}
    \caption{(a) Geometry of the space of diagonal confusions $\Dcal$ for $k=3$: a strictly convex space. Notice that each of the three axis-aligned faces are equivalent in geometry to the following figure in (b); (b) Geometry of diagonal confusions when restricted to classifiers predicting only classes $k_1$ and $k_2$ i.e. $\Dcal_{k_1, k_2}$;
	(c) A sphere $S_\lambda$ centered at $\ombf$ with radius $\lambda$, contained in the convex space of off-diagonal confusions $\Ccal$. \small$f^\ast{(\cmbf)}\,$\normalsize denotes the distance of $\cmbf$ from the hyperplane $\bell^\ast$ tangent at $\cmbfbar^\ast$. 
	}
	\label{fig:lin-fr}
\end{figure*}
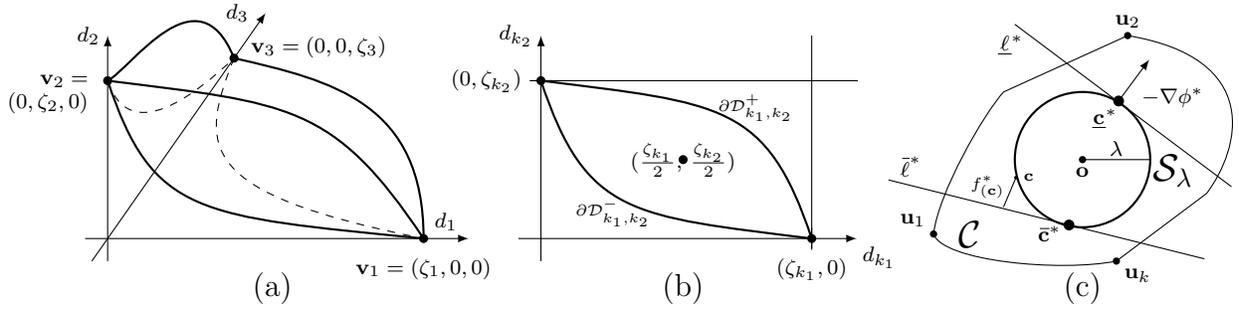

Proposition~\ref{prop:Cd} characterizes the geometry of the space of diagonal confusions $\Dcal$. 
Figure~\ref{fig:lin-fr}(a) illustrates this geometry when $k=3$.  
Interestingly, the 2-dimensional axes-aligned faces of $\Dcal$ (Figure~\ref{fig:lin-fr}~(b)) have exactly the same geometry as the space of binary classification confusion matrices (compare this with Figure~\ref{bin-fig:lin-fr}), where recall that a binary classification confusion matrix is uniquely determined by its two diagonal elements due to~\eqref{mult-eq:decomp}. 
We will exploit the set $\Dcal_{k_1, k_2}$ (more specifically, its boundary) for the elicitation task. 
Now notice that for $\psi \in \varphi_{DLPM}$, the RBO classifier restricted to predict classes $k_1, k_2$, predicts the label (out of the two possible choices) that maximizes the expected utility
conditioned on the instance. This is discussed below. 

\bprop 
Let $\psi \in \varphi_{DLPM}$ be  parametrized by $\ambf$ such that $\Vert \ambf \Vert_1=1$, and let $k_1, k_2 \in [k]$, then
\bequation
\hbar_{k_1, k_2}(\xmbf) = \left\{\begin{array}{lr}
			 k_1, & \; \text{if} \; a_{k_1} \eta_{k_1}(\xmbf) \geq a_{k_2} \eta_{k_2} (\xmbf) \\
			 k_2,& \;  o.w. 
	 	 \end{array}\right\}
\eequation
is the Restricted Bayes Optimal classifier (restricted to classes $k_1, k_2$) with respect to
$\psi$. 
\label{prop:d-k1k2-bayes}
\eprop

For a metric $\psi \in \varphi_{DLPM}$,  Proposition~\ref{prop:d-k1k2-bayes} provides RBO classifiers in $\Hcal_{k_1, k_2}$, which further gives us RBO diagonal confusions $\dmbfbar_{k_1, k_2}$ using~\eqref{mult-eq:components}. 
We know that this $\dmbfbar_{k_1, k_2}$ is unique, since any linear metric over a strictly convex domain ($\Dcal_{k_1, k_2}$) is maximized at a unique point on the boundary~\cite{boyd2004convex}. So, given a DLPM, we have access to a unique point in the query space. This allows us to define and then parametrize a subset of the query space, specifically, the upper boundary of $\Dcal_{k_1, k_2}$ through DLPMs.

\bdefinition 
The upper boundary of $\Dcal_{k_1, k_2}$, denoted by $\partial\Dcal^+_{k_1, k_2}$, constitutes the RBO diagonal confusions confined to classes $k_1, k_2 \in [k]$ for monotonically increasing DLPMs $(a_{i}\geq 0 \,\forall\,i \in [k])$ such that at least one out of $a_{k_1}$ or $a_{k_2}$  is non-zero (i.e., $a_{k_1} + a_{k_2} > 0$). 
\label{def:d-k1k2-boundary}
\edefinition

\textbf{Parameterizing the upper boundary $\partial \Dcal^+_{k_1, k_2}$.} 
Let $m \in [0, 1]$. Construct a DLPM by setting $a_{k_1} =  m$, $a_{k_2} = 1 - m$, and $a_i = 0$ for $i \neq k_1, k_2$. By using Proposition~\ref{prop:d-k1k2-bayes} and~\eqref{mult-eq:components}, obtain its RBO diagonal confusions, which by definition lies on the upper boundary. Thus, varying $m$ in this process, parametrizes the upper boundary $\partial\Dcal^+_{k_1, k_2}$. We denote this parametrization by $\nu(m; k_1, k_2)$, where $\nu: ([0,1]; k_1, k_2) \to \partial\mathcal{D}^+_{k_1,k_2}$. 

\subsection{Geometry of the space $\Ccal$ and parametrization of the enclosed sphere}
\label{ssec:lpmproperties}

Recall that, unlike the diagonal case, we focus on eliciting LPMs monotonically decreasing in the elements of the off-diagonal confusions (Section~\ref{ssec:metrics}). 
To this end, let $\umbf_i \in \Ccal$ for $i \in [k]$ be the off-diagonal confusions achieved by trivial classifiers predicting only class $i$ on the entire space $\Xcal$. 

\bprop
[Geometry of $\Ccal$ -- Figure~\ref{fig:lin-fr}~(c)] The space of off-diagonal confusions $\Ccal$ is convex and contained in the box $[0, \zeta_1]^{(k-1)}\times \cdots \times [0, \zeta_k]^{(k-1)}$. $\{\umbf_i\}_{i=1}^k$ belong to the set of vertices of $\Ccal$. $\Ccal$ always contains the point $\ombf = \frac{1}{k} \sum_{i=1}^k \umbf_i$ which corresponds to the off-diagonal confusions of the trivial classifier that randomly predicts each class with equal probability on the entire space $\Xcal$. 
\label{mult-prop:C}
\eprop

We find that 
the space of off-diagonal confusions $\Ccal$ has quite different geometry than the diagonal case. For instance, $\Ccal$ is not strictly convex. Nevertheless, since $\Ccal$ is convex and always contains the point $\ombf$, we may make the following assumption. 
Please see Figure~\ref{fig:lin-fr}(c) for an illustration.
\bassumption
There exists a $q$-dimensional sphere $\Scal_\lambda \subset \Ccal$ of radius $\lambda > 0$ centered at $\ombf$.  
\label{as:sphere}
\eassumption
Such a sphere always exists as long as the class-conditional distributions are not completely overlapping, i.e., there is some signal for non-trivial classification. 
A method to obtain $\Scal_\lambda$ is discussed in Section~\ref{mult-sec:guarantees}. 
Now recall that the optimum for a linear function optimized over a sphere is given by the slope of the function scaled by the radius of the sphere. This is formalized as a trivial lemma below.

\blemma
Let $\phi \in \varphi_{LPM}$ be parametrized by $\ambf$ such that $\norm{\ambf}_2=1$, then the unique optimal off-diagonal confusion $\cmbfbar$ over the sphere $\Scal_\lambda$ is a point on the boundary of $\Scal_\lambda$ given by $\cmbfbar = \lambda \ambf +\ombf$. 
\label{mult-lem:spherebayes}
\elemma

Given an LPM, Lemma~\ref{mult-lem:spherebayes} provides a unique point in the query space $\Scal_\lambda\subset\Ccal$. This gives us an opportunity to characterize and then parametrize a subset of the query space through LPMs. Since we focus on eliciting monotonically decreasing LPMs, we parametrize the lower boundary of $\Scal_\lambda$.  

\bdefinition
The lower boundary of $\Scal_\lambda$, denoted by $\partial \Scal^-_{\lambda}$, constitutes the set of optimal off-diagonal confusions over the sphere $\Scal_\lambda$ for LPMs with $ a_{i} \leq 0 \; \forall \,i \in [q]$ (monotonically decreasing condition). 
\edefinition

\textbf{Parameterizing the lower boundary of the enclosed sphere $\partial \Scal^-_{\lambda}$.} 
We follow the standard method for parametrizing points on the surface of a sphere via angles. Let $\thetambf$ be a ($q-1$)-dimensional vector of angles, where all the angles except the primary angle are in second quadrant, i.e., $\{\theta_i \in [\pi/2, \pi]\}_{i=1}^{q-2}$, and the primary angle is in the third quadrant, i.e., $ \theta_{(q-1)} \in [\pi, 3\pi/2]$. Construct an LPM $(\Vert \ambf \Vert_2=1)$ by setting $a_i = \Pi_{j=1}^{i-1} \sin\theta_j \cos{\theta_i}$ for $i \in [q-1]$ and $a_q = \Pi_{j=1}^{q-1} \sin\theta_j$. The choice of the quadrants ensures the monontonically decreasing condition, i.e., $\{a_i \leq 0\}_{i=1}^q$. By using Lemma~\ref{mult-lem:spherebayes}, obtain its BO off-diagonal confusions over the sphere $\Scal_\lambda$, which clearly lies on the lower boundary. Thus, varying $\thetambf$ in this procedure, parametrizes the lower boundary $\partial\Scal^-_{\lambda}$. We denote this parametrization by $\mu(\thetambf)$, where $\mu: [\pi/2, \pi]^{q-2} \times [\pi, 3\pi/2] \to \partial \Scal^-_{\lambda}$.

%% file: multiclass/d-elicitation.tex
\section{Metric Elicitation}
\label{mult-sec:me}
Using the outlined parametrizations $\{\nu, \mu\}$, we propose efficient binary-search type algorithms to elicit oracle's implicit performance metric. We will first discuss elicitation with no \emph{feedback} noise from the oracle. We will later show robustness to noisy feedback in  Section~\ref{mult-sec:guarantees}. 

\subsection{DLPM Elicitation}
\label{sssec:dlpmelicit}

The following lemma concerning a broader family of metrics is the route to our elicitation procedures. Since both linear and linear-fractional functions are quasiconcave, the lemma applies to both.

\blemma
Let $\psi : \Dcal \rightarrow \Rmbb$ be a quasiconcave metric which is
monotone increasing in all $\{d_{i}\}_{i=1}^k$. For $k_1, k_2 \in [k]$, let 
$\rho^+:[0,1]\to \partial\Dcal^+_{k_1, k_2}$ be a continuous, bijective, parametrization of the upper boundary. 
Then the composition $\psi \circ \rho^+: [0,1]\to\mathbb R$ is
quasiconcave and thus unimodal on $[0, 1]$.
\label{lemma:slice}
\elemma

\bremark
Under Assumption~\ref{mult-as:eta}, every supporting hyperplane of $\Dcal_{k_1, k_2}$ supports a unique point on the boundary $\partial \Dcal^{+}_{k_1, k_2}$ and vice-versa (Proposition~\ref{prop:Cd}); therefore, the composition $\psi \circ \rho^+$ has no flat regions. In other words, the function $\psi \circ \rho^+$ is concave.
\label{rem:concave}
\eremark

The proof of Lemma~\ref{lemma:slice} first shows that any quasiconcave metric $\psi$ defined on the space $\Dcal$ is also quasiconcave on the restricted space $\Dcal_{k_1, k_2}$, and then shows the quasiconcavity and thus the unimodality (due to the one-dimensional parametrization of $\partial\Dcal^+_{k_1, k_2}$) of $\psi$ on a further restricted space $\partial\Dcal^+_{k_1, k_2}$. Furthermore, Remark~\ref{rem:concave} reveals that the function $\psi \circ \rho^+$ is concave, allowing us to devise the following binary-search type method for elicitation. 

\balgorithm[t]
\caption{DLPM Elicitation}
\label{mult-alg:dlpm}
\balgorithmic[1]
\STATE \textbf{Input:} $\epsilon > 0$, oracle $\Omega$, $\ahat_1 = 1$
\FOR{$i=2,\cdots ,k$}
\STATE \textbf{Initialize:} $m^a = 0$, $m^b = 1$.
\WHILE{$\abs{m^b - m^a} > \epsilon$}
\STATE Set $m^c = \frac{3 m^a + m^b}{4}$, $m^d = \frac{m^a + m^b}{2}$, and $m^e = \frac{m^a + 3 m^b}{4}$.\\
\STATE Set $\dmbfbar^a_{1, i} = \nu(m^a; 1, i)$ (i.e. parametrization of $\partial \Dcal^+_{1, i}$ in Section~\ref{ssec:dlpmproperties}). Similarly, set $\dmbfbar^c_{1, i}, \dmbfbar^d_{1, i}, \dmbfbar^e_{1, i}, \dmbfbar^b_{1, i}$. 
\STATE Query  $\Omega(\dmbfbar^c_{1,i}, \dmbfbar^a_{1,i}),  \Omega(\dmbfbar^d_{1,i}, \dmbfbar^c_{1,i})$,  $\Omega(\dmbfbar^e_{1,i}, \dmbfbar^d_{1,i}), \text{ and } \Omega(\dmbfbar^b_{1,i}, \dmbfbar^e_{1,i})$.
\STATE $[m^a, m^b] \leftarrow$ \emph{ShrinkInterval-1} (responses).
\ENDWHILE
\STATE Set $m^d = \frac{m^a+m^b}{2}$. Then set $\ahat_i = \frac{1-m^d}{m^d}\hat a_1$.
\ENDFOR
\STATE \textbf{Output:} $\ambfhat = \left(\frac{\ahat_1}{\Vert \ambfhat \Vert_1}, \cdots , \frac{\ahat_k}{\Vert \ambfhat \Vert_1}\right)$. 
\ealgorithmic
\ealgorithm

Suppose that the oracle's metric is $\psi^* \in \varphi_{DLPM}$ parametrized by $\ambf^*$ where $\norm{\ambf^*}_1=1$, $\{a^*_i\}_{i=1}^k \geq 0$ (Section~\ref{ssec:metrics}). Using the parametrization $\nu$, Algorithm~4.1  
returns an estimate $\ambfhat$ of $\ambf^\ast$. It  takes two classes at a time, class $1$ and class $i$. Since the metric is unimodal on $\partial\Dcal^+_{1, i}$ (Lemma~\ref{lemma:slice}), the algorithm applies binary-search in the inner while-loop to estimate the ratio $a^*_i/a^*_1$. The \emph{ShrinkInterval-1} subroutine shrinks the interval $[m^a, m^b]$ into half based on the oracle responses in the usual binary-search way for searching the optimum (Figure~\ref{mult-append:fig:shrink1}, Appendix~\ref{append:sec:shrink}). The algorithm repeats this $(k-1)$ times to estimate the ratios \{$a^*_2/a^*_1, \dots, a^*_k/a^*_1\}$. Finally, it outputs a normalized metric estimate $\ambfhat$. 

%% file: multiclass/elicitation.tex
\subsection{LPM Elicitation}
\label{sssec:lpmelicit}

We now discuss LPM elicitation, where the metrics are assumed to be monotonically decreasing in the off-diagonal confusions. 
Unfortunately, $\partial\Ccal$ may have flat regions due to lack of strict convexity, so the algorithm for the diagonal case does not apply. Instead, we consider a query space given by the sphere $\Scal_\lambda \subset \Ccal$ and propose a coordinate-wise binary-search style algorithm, which is an outcome of our novel geometric characterization and the approach in Derivative-Free Optimization (DFO)~\cite{jamieson2012query}. 

\balgorithm[t]
\caption{LPM Elicitation}
\label{mult-alg:lpm}
\balgorithmic[1]
\STATE \textbf{Input:} $\epsilon > 0$, oracle $\Omega$, $\lambda$, and $\bm{\theta} = \bm{\theta}^{(1)}$
\FOR{$t=1, 2, \cdots, T$}
\STATE Set $\bm{\theta}^a = \bm{\theta}^c=\bm{\theta}^d=\bm{\theta}^e=\bm{\theta}^b = \bm{\theta}^{(t)}$.
\IF{$(t\%(q-1))$} 
\STATE Set $j = t\%(q-1)$ 
\ELSE 
\STATE Set $j = q-1$.
\ENDIF
\IF{$j == q - 1$} \STATE \textbf{Initialize:} $\theta^a_{j} = \pi$, $\theta^b_j = 3\pi/2$. 
\ELSE 
\STATE \textbf{Initialize:} $\theta^a_j = \pi/2$, $\theta^b_j = \pi$.
\ENDIF
\WHILE{$\abs{\theta^b_j - \theta^a_j} > \epsilon$}
\STATE Set $\theta^c_j = \frac{3 \theta^a_j + \theta^b_j}{4}$, $\theta^d_j = \frac{\theta^a_j + \theta^b_j}{2}$, and $\theta^e_j = \frac{\theta^a_j + 3 \theta^b_j}{4}$.
\STATE Set $\cmbfbar^a = \mu(\bm{\theta}^a)$ (i.e. parametrization of $\partial \Scal^-_\lambda$ in  Section~\ref{ssec:lpmproperties}) Similarly, set $\cmbfbar^c, \cmbfbar^d, \cmbfbar^e, \cmbfbar^b$.
\STATE Query $\Omega(\cmbfbar^c, \cmbfbar^a), \Omega(\cmbfbar^d,  \cmbfbar^c)$, $\Omega(\cmbfbar^e, \cmbfbar^d),\Omega(\cmbfbar^b, \cmbfbar^e)$
\STATE $[\theta^a_j, \theta^b_j] \leftarrow$ \emph{ShrinkInterval-2} (responses). 
\ENDWHILE
\STATE Set $\theta^d_j = \frac{1}{2}(\theta^a_j+\theta^b_j)$ and then set $\bm{\theta}^{(t)} = \bm{\theta}^d$.
\ENDFOR
\STATE \textbf{Output:} $\hat a_i =$ $\Pi_{j=1}^{i-1} \sin\theta_j^{(T)} \cos{\theta_i}^{(T)}$ $ \, \forall i \in [q-1]$ and $\hat a_q =$ $\Pi_{j=1}^{q-1} \sin\theta_j^{(T)}$.
\ealgorithmic
\ealgorithm

Suppose that the oracle's metric is $  \phi^{*} \in \varphi_{LPM}$ parametrized by $\ambf^*$ where $\norm{\ambf^*}_2=1$, $\{a^*_i\}_{i=1}^q \leq 0$ (Section~\ref{ssec:metrics}). Using the parametrization $\mu(\thetambf)$ of $\partial \Scal^-_\lambda$ (Section~\ref{ssec:lpmproperties}), Algorithm~4.2 returns an estimate $\ambfhat$ of $\ambf^\ast$. In each iteration, the algorithm updates one angle $\theta_j$ keeping other angles fixed 
by a binary-search procedure, where again the \emph{ShrinkInterval-2} subroutine shrinks the interval $[\theta^a_j, \theta^b_j]$ by half based on the oracle responses (Figure~\ref{append:fig:shrink2}, Appendix~\ref{append:sec:shrink}). Then the algorithm cyclically updates each angle until it converges to a metric sufficiently close to the true metric. The convergence is assured because, intuitively, the algorithm via a dual interpretation minimizes a smooth, strongly convex function $f^\ast(\cmbf)$ measuring the distance of the boundary points from a hyperplane $\bell^\ast$, whose slope is given by $\ambf^\ast$ and is tangent 
at the BO confusion $\cmbfbar^\ast$ (see Figure~\ref{fig:lin-fr}(c)). 

%% file: multiclass/extensions.tex
\section{Extensions}
\label{mult-sec:extensions}

We emphasize that the goal of ME is not simply to choose between default or popularly used metrics but to elicit novel metrics which best match the oracle preferences. As the family of human evaluation metrics is believed to be large and since we already have created strategies for linear metrics, we can now certainly aim at efficient elicitation for flexible metric families. Therefore, in this section, we discuss a variety of extensions to other family of metrics.

 For the purpose of clarity in this section, let us replace the notation of the parametrization $\nu(m;k_1, k_2)$ of the upper boundary $\partial\Dcal^+_{k_1, k_2}$ by $\nu^+(m;k_1, k_2)$. This is useful to disambiguate with the parametrization $\nu^-(m;k_1, k_2)$ of the lower boundary $\partial\Dcal^-_{k_1, k_2}$, which is useful in linear-fractional elicitation. 
 
 In addition to the entities defined in Table~\ref{tab:bayes}, we define some more entities  such as the Inverse Bayes Optimal (IBO) and Restricted Inverse Bayes Optimal (RIBO) classifiers, diagonal confusions, utility in Table~\ref{tab:bayesinv}. The six definitions on the left can be analogously described diagonal metrics and diagonal confusions. The six definitions on the right are of interest for the diagonal case. These are useful in the elicitation of linear-fractional metrics.

\begin{table}[t]
	\caption{Bayes Optimal (BO), Inverse Bayes Optimal (IBO), Restricted Bayes Optimal (RBO), and Restricted Inverse Bayes Optimal (RIBO) entities. 
	}
	\vspace{-0.2cm}
	\label{tab:bayesinv}
	\begin{center}
		\begin{small}
				\begin{tabular}{|l|l|l|l|}
					\hline
					  Name & Definition & Name & Definition  \\ \hline 
					\makecell{BO classifier $\hbar$} & \makecell{$\argmax_{h \in \Hcal}\phi(\cmbf(h))$} & \makecell{RBO classifier $\hbar_{k_1, k_2}$} & \makecell{$\argmax_{h \in \Hcal_{k_1, k_2}}\psi(\dmbf(h))$} \\
					\makecell{BO utility $\btau$ \\ over a subset $\Scal\subseteq\Ccal$}&  \makecell{$\max_{\cmbf \in \Scal \subseteq \Ccal}\phi(\cmbf)$} & \makecell{RBO utility $\btau_{k_1, k_2}$} &  \makecell{$\max_{\dmbf \in \Dcal_{k_1, k_2}}\psi(\dmbf)$}  \\
					\makecell{BO confusion $\cmbfbar$ \\ over a subset $\Scal\subseteq\Ccal$} & \makecell{$\argmax\limits_{\cmbf \in \Scal \subseteq \Ccal}\phi(\cmbf)$} & \makecell{RBO confusion $\dmbfbar_{k_1, k_2}$} & \makecell{$\argmax\limits_{\dmbf \in \Dcal_{k_1, k_2}}\psi(\dmbf)$} \\
					\makecell{IBO classifier $\barbelow{h}$} & \makecell{$\argmin_{h \in \Hcal}\phi(\cmbf(h))$} & \makecell{RIBO classifier $\barbelow{h}_{k_1, k_2}$} & \makecell{$\argmin_{h \in \Hcal_{k_1, k_2}}\psi(\dmbf(h))$} \\
					\makecell{IBO utility $\ttau$  \\ over a subset $\Scal\subseteq\Ccal$} &  \makecell{$\min_{\cmbf \in \Scal\subseteq\Ccal}\phi(\cmbf)$}  & \makecell{RIBO utility $\ttau_{k_1, k_2}$} &  \makecell{$\min_{\dmbf \in \Dcal_{k_1, k_2}}\psi(\dmbf)$}  \\
					\makecell{IBO confusion $\barbelow{\cmbf}$ \\over a subset $\Scal\subseteq\Ccal$} & \makecell{$\argmin\limits_{\cmbf \in \Scal\subseteq\Ccal}\phi(\cmbf)$} & \makecell{RIBO confusion $\barbelow{\dmbf}_{k_1, k_2}$} & \makecell{$\argmin\limits_{\dmbf \in \Dcal_{k_1, k_2}}\psi(\dmbf)$} \\
					\hline
				\end{tabular}
		\end{small}
	\end{center}
	\vskip -0.45cm
\end{table}

Lastly, for linear-fractional elicitation, we need to parametrize the lower boundary $\partial\Dcal^-_{k_1, k_2}$ and upper boundary of the sphere $\partial \Scal^+_{\lambda}$ as well. These parametrizations are defined below.

\bdefinition 
The RBO diagonal confusions for DLPMs parametrized by $\ambf$ with $a_{k_1},  a_{k_2} < 0$ form the lower boundary of $\Dcal_{k_1, k_2}$, denoted by $\partial\Dcal^-_{k_1, k_2}$.
\label{def:d-k1k2-boundarylower}
\edefinition

\textbf{Parametrization of $\partial \Dcal^-_{k_1, k_2}$.} 
We denote this parametrization by a function $\nu^-(m; k_1, k_2)$. Take a parameter $-1 \leq m \leq 0$. Create a DLPM $\psi$ by setting $a_{k_1} = m$, $a_{k_2} = -1 - m$, and $a_i = 0$ for $i \neq k_1, k_2 \in [k]$. RBO diagonal confusions of such DLPMs lie on the lower boundary $\partial\Dcal^-_{k_1, k_2}$. As we vary $m$, we move on the lower boundary $\partial\Dcal^-_{k_1, k_2}$. 

\bdefinition
The optimal off-diagonal confusions over the sphere $S_\lambda$ for LPMs parametrized by $\ambf$ with $ a_{i} \geq 0 \; \forall \,i \in [k]$ form the upper boundary of $S_\lambda$, denoted by $\partial S^+_{\lambda}$. 
\edefinition

\textbf{Parametrization of $\partial S^+_{\lambda}$.} 
The parametrization of the upper boundary $\partial S^+_{\lambda}$ is same as that of the lower boundary $\partial S^-_{\lambda}$ (Section~\ref{ssec:lpmproperties}) except that now all the angles are in the first quadrant i.e. $\{\theta_i \in [0, \pi/2]\}_{i=1}^{q-1}$, so to satisfy the condition $ a_{i} \geq 0 \; \forall \,i \in [k]$.  
\vspace{-0.5cm}
\subsection{Diagonal Linear Fractional Performance Metric (DLFPM) Elicitation}
\label{append:sec:dlfpme}

We start by first defining the diagonal linear fractional performance metric. 

\bdefinition Diagonal Linear-Fractional Performance Metric (DLFPM): We denote this family by $\varphi_{DLFPM}$. Given $\ambf, \bmbf \in \Rmbb^{k}$ and $b_0 \in \Rmbb$, the metric is defined as: 
\begin{align}
\psi(\dmbf) &= \frac{\inner{\ambf}{\dmbf}}{\inner{\bmbf}{\dmbf} + b_0}.
\label{eq:d-linear-fr}
\end{align}
\label{def:d-linear-fr}
\vspace{-0.5cm}
\edefinition

For any $\psi \in \varphi_{DLFPM}$, we assume that $\{a_i\}_{i=1}^k, \{b_i\}_{i=1}^k$ are not all zero simultaneously and 
wlog, we take $\psi(\dmbf) \in [0, 1]$ and monotonically increasing in all $\{d_i\}_{i=1}^k$. We also make the following regularity assumption. 

\bassumption
Let $\psi \in \varphi_{DLFPM}$  parametrized by $\ambf$ and $\bmbf$ (Definition~\ref{def:d-linear-fr}). We assume that $a_i \geq 0$ and $a_{i} \geq b_{i}$ for all $i \in [k]$. In addition, $b_0 = \sum_i (a_i - b_i)\zeta_i$ and $\sum_i a_i = 1$.
\label{as:d-sufficient}
\eassumption

Equivalent to fixing $\Vert \ambf \Vert_1 = 1$, $a_i \geq 0$ for the diagonal linear case (Section~\ref{ssec:metrics}), the conditions in Assumption~\ref{as:d-sufficient} are sufficient conditions for DLFPMs to be bounded and monotonically increasing in diagonal elements of the confusion matrices. This is detailed in the following proposition. 
\bprop
The conditions in Assumption~\ref{as:d-sufficient} are sufficient for $\psi \in \varphi_{DLFPM}$ to be bounded in $[0,1]$ and simultaneously monotonically increasing in $\{d_{i}\}_{i=1}^k$.
\label{prop:d-sufficient}
\eprop

We consider $b_0 = \sum_i (a_i - b_i)\zeta_i$, instead of the derived condition $b_0 \geq \sum_i (a_i - b_i)\zeta_i$, which is sufficient to guarantee a unique metric bounded in $[0, 1]$ for elicitation purposes (instead of one of the equivalent alternatives).
Note that most existing linear-fractional metrics satisfy these conditions~\cite{hiranandani2018eliciting, koyejo2015consistent, narasimhan2018learning}.

Now, suppose that the oracle's metric is $ \psi^{*} \in \varphi_{DLFPM} $. 
Let $\btau^*$ and $\barbelow{\tau}^*$ be the maximum and minimum value of $\psi^{*}$, respectively. 
Due to strict convexity of $\Dcal$, we have a hyperplane
\begin{align}
\bell_f^* := \sum_{i=1}^k(a_i^* - \btau^*b_i^*)d^*_{i} = \btau^* b_0 
\end{align}
tangent at the BO diagonal confusions $\dmbfbar^\ast$ on the upper boundary of $\Dcal$, denoted by $\partial \Dcal^+$. 

\noindent Similarly, we have a hyperplane

\bequation
\tell_f^* := \sum_{i=1}^k(a_i^* - \ttau^*b_i^*)\barbelow{d}^*_{i} = \ttau^* b_0 
\eequation
which touches the set $\Dcal$ only at $\barbelow{\dmbf}^{\ast}$ (IBO diagonal confusions) on the lower boundary, denoted by $\partial \Dcal^-$. See Figure~\ref{fig:lin-fr}(c) for the visual intuition, where assume that the underlying space is $\Dcal$ instead of the sphere $\Scal_\lambda$.

Since DLFPM is quasiconcave, Algorithm~4.1 returns a slope of the hyperplane, say $\smbfbar$. Using that slope, we can compute the Bayes Optimal diagonal confusions $\dmbfbar^*$ using Proposition~\ref{prop:d-bayes} (a more general version of Proposition~\ref{prop:d-k1k2-bayes}), 
which gives us the hyperplane ${\bell^\ast := \inner{\smbfbar}{\dmbf} = \inner{\smbfbar}{\dmbfbar^\ast}}$. This is equivalent to $\bell_f^*$ up to a constant multiple; therefore, the true metric is the solution to the following non-linear system of equations (SoE):
\begin{align}
a^*_i - \btau^* b^*_i = \alpha \sbar_i \;\; \forall \; i \in [k], \quad  \btau^* b^*_0 = \alpha \inner{\smbfbar}{\dmbfbar^*}
\label{eq:d-lin-fr-equi}
\end{align}
where $\alpha \geq 0$, because LHS and ${\sbar_i}$'s are non-negative. 
If we somehow know the true $\ambf^\ast$, then by using the following proposition, we can elicit the DLFPM upto a constant multiple, i.e. we can get $\hat{\psi} \approx \alpha \psi^{*}$, which is sufficient for the elicitation task.

\bprop
Knowing $\ambf^\ast$ i.e. using $\ambfhat = \ambf^\ast$ solves the SoEs~\eqref{eq:d-lin-fr-equi} as:
\begin{align}
\hat b_i = (\hat a_i - \sbar_i)\frac{\Lambda_1}{\Lambda_2}, 
\end{align}
where $\Lambda_1 = \sum_i \hat a_i \zeta_i$, $\Lambda_2 = \inner{\smbfbar}{\dmbfbar^\ast} + \sum_i (\hat a_i - \sbar_i)\zeta_i$, and $\hat b_0$ is as defined in Assumption~\ref{as:d-sufficient}. 
\label{prop:d-uppersystem}
\eprop

Now the question is how do we get the true $\ambf^\ast$. To our rescue, 
we also know that a DLFPM is quasiconvex. Thus, by minimizing the metric (again by using restricted classifiers) using Algorithm~4.3 (described next), we can get a similar hyperplane on the lower boundary $\partial\Dcal^-$. Algorithm~4.3 is described below. 

\balgorithm[t]
\caption{Diagonal (Quasiconcave) Metric Minimization}
\label{mult-alg:diaglinearmin}
\balgorithmic[1]
\STATE Follow Algorithm~\ref{mult-alg:dlpm} except:
\STATE \textbf{Initialize:} $m^a = -1$, $m^b = 0$ in step 3 of Algorithm~\ref{mult-alg:dlpm}.
\STATE \textbf{Invert Responses:} Replace oracle responses $\dmbf\prec \dmbf'$ with $\dmbf\succ \dmbf'$ and vice versa.
\ealgorithmic
\ealgorithm

\textbf{Algorithm~4.3.} \emph{Minimizing diagonal quasiconvex metrics:} This algorithm is same as Algorithm~4.1 with only two changes. First, we start with $m \in [-1, 0]$, because the optimum will lie on the lower boundary $\partial\mathcal D^-$. Second, we check for $\dmbf \prec \dmbf'$ whenever Algorithm~4.1 checks for $\dmbf\succ \dmbf'$, and vice-versa. 
Here, we output the counterpart, i.e., slope $\tsmbf$.

Once we get the slope $\tsmbf$, we can obtain the inverse Bayes diagonal confusion $\barbelow{\dmbf}^\ast$ using Proposition~\ref{prop:d-bayes} (a more general version of Proposition~\ref{prop:d-k1k2-bayes}). This will result in a supporting hyperplane ${\tell^\ast := \inner{\tsmbf}{\dmbf} =  \inner{\tsmbf}{\barbelow{\dmbf}^\ast}}$. This hyperplane  is tangent to the lower boundary $\partial \Dcal^-$, and equivalent to ${\tell}^*_f$ up to a constant multiple; thus, the true metric is also the solution of the following SoE:
\begin{align}
a^*_i - \ttau^* b^*_i = \gamma \barbelow{s}_i \;\; \forall \; i \in [k], \quad
\ttau^* b^*_0 = \gamma \inner{\tsmbf}{\barbelow{\dmbf}^\ast}
\end{align}
where $\gamma \leq 0$ since LHS is positive, but $\barbelow{s}_i$'s are negative. Again, we may assume $\gamma < 0$. By dividing the above equations by $-\gamma$ on both sides, all the coefficients are factored by $-\gamma$. This does not change $\psi^{*}$; thus, the system of equations becomes the following:
\begin{align}
a''_i - \ttau^*b''_i =  \barbelow{s}_i,  \;\; \forall \; i \in [k], \quad
 \ttau^*b''_0 =  \inner{\tsmbf}{\barbelow{\dmbf}^\ast}.
\label{append:eq:d-lin-fr-equi-lower-2}
\end{align}

Now, if we know $\ambf'$ in~\eqref{append:eq:d-lin-fr-equi-2}, then by using Proposition~\ref{prop:d-uppersystem}, we may solve the system~\eqref{append:eq:d-lin-fr-equi-2} and obtain a metric, say $\psi'$. System~\eqref{append:eq:d-lin-fr-equi-lower-2} can be solved analogously, provided we know $\ambf''$ in~\eqref{append:eq:d-lin-fr-equi-lower-2}, to get a metric, say $\psi''$.  
Notice that when when we have the true ratio i.e $a^*_i/a^*_j=a'_i/a'_j=a''_i/a''_j$ for $i, j \in [k]$, then  $\psi^{*} = \psi'/\alpha = -\psi''/\gamma$. This means that when the  true ratios are known, then $\psi'$, $\psi''$ are constant multiples of each other. So, we look for the ratios where the solution to the two systems are just pointwise constant multiple of one another. This is the same idea used in the binary case (see Section~\ref{ssec:elicit_linearfrac}).  However, we have to search for the entire grid $[0,1]^k$ instead of $[0,1]$ as is in the binary case. This is a computationally challenging task. 

Notice that we can randomly sample diagonal confusions on the boundary $\partial \Dcal$. This is done by first randomly generating DLPMs and then computing their BO or IBO diagonal confusions using Proposition~\ref{prop:d-bayes}. After obtaining $\bell^\ast$ and $\tell^\ast$, we run the grid seacrh based Algorithm~\ref{append:alg:d-grid-search} to find the estimates of the true $a_i$'s. Although the grid-search based algorithm is independent of oracle queries, it is computationally efficient. It runs for $(k-1)$ rounds, where in each round it matches the solution of the two SoE's as closely as possible on a number of samples from the boundary $\partial \Dcal_{1, k}$ and figures out the ratio of $a_j/a_1$ for $j\neq1 \in [k]$. Thanks to the property $\sum_i a_i=1$ and access to the restricted diagonal confusions, we are saved from searching the entire grid $[0, 1]^k$ to merely $(k-1)$ times grid-search on $[0, 1]$. 

\balgorithm[t]
\caption{DLFPM: Grid Search for Best Pairwise Ratios}
\label{append:alg:d-grid-search}
\balgorithmic[1]
\STATE \textbf{Input:} $n', \delta$.
\FOR{$j=2,\cdots,k$}
\STATE \textbf{Initialize:} $\sigma_{opt} = \infty, a'_j = 0$.
\STATE Sample $\dmbf^{1},...,\dmbf^{n'}$ on $\partial \Dcal_{1,j}$ (BO or IBO diagonal confusions for random $n'$ DLPMs).
\FOR {($a'_j = 0$; $a'_j \leq 1$; $a'_j = a'_j + \delta$)}
\STATE Compute $\psi'$, $\psi''$ using Proposition \ref{prop:d-uppersystem}. 
\STATE Compute array $r = [\frac{\psi'(\dmbf^1)}{\psi''(\dmbf^1)},...,\frac{\psi'(\dmbf^{n'})}{\psi''(\dmbf^{n'})}]$. Set $\sigma = \text{std}(r).$
\STATE {\bf if }($\sigma < \sigma_{opt}$) Set $\sigma_{opt} = \sigma$ and $a'_{j, opt} = a'_j$.
\ENDFOR
\STATE Set $a'_j = \frac{a'_{j, opt}}{1-a'_{j, opt}}$.
\ENDFOR
\STATE $a'_1 =1$.
\STATE\textbf{Output:} $\ambf' = \left(\frac{a'_1}{\norm{\ambf'}_1}, \cdots , \frac{a'_k}{\norm{\ambf'}_1}\right)$.
\ealgorithmic
\ealgorithm

\vspace{-0.1cm}
\subsection{LFPM Elicitation}
\label{append:sec:lfpme}

We start by defining the linear-fractional performance metric in off-diagonal confusions.  
\bdefinition Linear-Fractional Performance Metric (LFPM): We denote this family by $\varphi_{LFPM}$. Given constants $\ambf, \bmbf \in \Rmbb^{q}$ and $b_0 \in \Rmbb$, the metric is defined as  
\begin{align}
\phi(\cmbf) &= \frac{\inner{\ambf}{\cmbf}}{\inner{\bmbf}{\cmbf} + b_0}.
\label{eq:linear-fr}
\end{align} 
\vskip -0.1cm
\label{def:linear-fr}
\edefinition

For any $\phi \in \varphi_{LFPM}$ (Definition~\ref{def:linear-fr}), we assume that $\{a_i\}_{i=1}^q, \{b_i\}_{i=1}^q$ are not all zero simultaneously. Moroever, w.l.o.g., $\phi(\cmbf) \in [-1, 0]\,\, \forall\,\, \cmbf \in \Ccal$ and is monotonically decreasing in all $\{c_{i}\}_{i=1}^q$. Similar to the diagonal case, we make the following regularity assumption. 

\bassumption
Let $\phi \in \varphi_{LFPM}$ (Definition~\ref{def:linear-fr}). We assume that $a_i \leq 0$ and $a_{i} \leq -b_{i}$ for all $i \in [q]$. In addition, $b_0 = \sum_i -(a_i + b_i)\zeta_i$, and $\sum_i a_i=-1$.
\label{as:sufficient}
\eassumption
Equivalent to fixing $\Vert \ambf \Vert_1 = 1$, $a_i \geq 0$ for the diagonal linear case (Section~\ref{ssec:metrics}), the conditions in Assumption~\ref{as:sufficient} are sufficient conditions for LFPMs to be bounded and monotonically decreasing in off-diagonal elements of the confusion matrices. This is detailed in the following proposition. 
\bprop
Assumption~\ref{as:sufficient} is sufficient for $\phi \in \varphi_{LFPM}$ to be bounded in $[-1,0]$ and simultaneously monotonically decreasing in $\{c_{i}\}_{i=1}^q$.
\label{prop:sufficient}
\eprop

We consider $b_0 = \sum_i -(a_i + b_i)\zeta_i$, instead of the derived condition $b_0 \geq \sum_i -(a_i + b_i)\zeta_i$, which is sufficient to guarantee a unique metric bounded in $[-1, 0]$ for elicitation purposes (instead of one of the equivalent alternatives). Note that most existing linear-fractional metrics satisfy these conditions~\cite{hiranandani2018eliciting, koyejo2015consistent, narasimhan2018learning}.

Now, suppose that the oracle's metric is $ \phi^{*} \in \varphi_{LFPM} $.
Let $\btau^*$ and $\barbelow{\tau}^*$ be the maximum and minimum value of $\phi^{*}$, respectively.
Due to strict convexity of $\Scal_\lambda$, we have a hyperplane
\begin{align}
\bell_f^* := \sum_{i=1}^q(a_i^* - \btau^*b_i^*)\bar c^*_{i} = \btau^* b_0 
\end{align}
touching the set $\Scal_\lambda$ only at BO confusions $\cmbfbar^{*}$ (over the sphere $\Scal_\lambda$) on the lower boundary $\partial \Scal^-_{\lambda}$. Similarly, we have a hyperplane
\begin{ceqn}
\bequation
\tell_f^* := \sum_{i=1}^q(a_i^* - \ttau^*b_i^*)\barbelow{c}^*_{i} = \ttau^* b_0 
\eequation
\end{ceqn}
which touches the set $\Scal_\lambda$ only at inverse Bayes Optimal confusions $\barbelow{\cmbf}^{\ast}$ (over the sphere $\Scal_\lambda$) on the upper boundary $\partial \Scal^+_{\lambda}$. See Figure~\ref{fig:lin-fr}(c) for the visual intuition.

Here, we use strict convexity of $\Scal_\lambda$ and follow the same arguments as in DLFPM to get a hyerplane ${\bell^\ast := \inner{\smbfbar}{\cmbf} = \inner{\smbfbar}{\cmbfbar^{*}}}$ after using Algortihm~4.2. Here, $\cmbfbar^{*}$  is the optimal best (BO) off-diagonal confusion on the sphere.  The only difference is that the BO confusions lie on the lower boundary $\partial \Scal^-_{\lambda}$ (monotonically decreasing). The SoE we get is:  

\begin{align}
a^*_i - \btau^* b^*_i = \alpha \sbar_i \;\; \forall\; i \in [q], \qquad 
\btau^* b^*_0 = \alpha \inner{\smbfbar}{\cmbfbar^{*}}
\label{eq:lin-fr-equi}
\end{align}
where $\alpha \geq 0$. 
Similar to DLFPMs, by knowing $\ambf^\ast$, we can elicit the LFPM upto a constant multiple.

\bprop
Knowing $\ambf^\ast$ i.e. using $\hat \ambf = \ambf^\ast$ solves the SoEs~\eqref{eq:lin-fr-equi} as:
\begin{align}
\hat b_i = (\hat a_i - \sbar_i)\frac{\Lambda'_1}{\Lambda'_2}, 
\end{align}
where $\Lambda'_1 = -\sum_i \hat a_i \zeta_i$, $\Lambda'_2 = \inner{\smbfbar}{\cmbfbar^{*}} + \sum_i (\hat a_i - \sbar_i)\zeta_i$, and $\hat b_0$ is as defined in Assumption~\ref{as:sufficient}. 
\label{prop:lowersystemsphere}
\eprop

Now again the question is how do we get the true $\ambf^\ast$. To our rescue, 
we also know that an LFPM is quasiconvex. Thus, by minimizing the metric using Algorithm~4.5 (described next), we can get a similar hyperplane ${\tell^\ast := \inner{\tsmbf}{\barbelow{\cmbf}} = \inner{\tsmbf}{\barbelow{\cmbf}^*}}$ tangent to the upper boundary $\partial \Scal^+_{\lambda}$. 

\balgorithm[t]
\caption{General (Quasiconcave) Metric Minimization}
\label{mult-alg:linearmin22}
\balgorithmic[1]
\STATE Follow Algorithm~\ref{mult-alg:lpm} except:
\STATE \textbf{Initialize:} $\theta_j^a = 0$, $\theta_j^b = \pi/2$ in steps 9-13 of Algorithm~\ref{mult-alg:lpm}.
\STATE \textbf{Invert Responses:} Replace oracle responses $\cmbf\prec \cmbf'$ with $\cmbf\succ \cmbf'$ and vice versa.
\ealgorithmic
\ealgorithm

\textbf{Algorithm~4.5} \emph{Minimizing quasiconvex metrics of off-diagonal confusions:} This algorithm is same as Algorithm~4.2 with only two changes. First, we start with $\thetambf \in [0, \pi/2]^q$, because the optimum will lie on the upper boundary $\partial \Scal^+_{\lambda}$. Second, we check for $ \cmbf \prec \cmbf'$ whenever Algorithm~4.2 checks for $\cmbf \succ \cmbf'$, and vice versa. 
Here, we output the counterpart, i.e., slope $\tsmbf$.

Thus, a similar SoE~\eqref{eq:lin-fr-equi} whose solution looks like Proposition~\ref{prop:lowersystemsphere} is obtained. After obtaining $\bell^\ast$ and $\tell^\ast$, we run grid-search Algorithm~4.6  to find the estimates of the true $a_i$'s. 
The algebra related to LFPM elicitation is same as the DLFPM case. However, this time we need to search in $[0,1]^{q-1}$ grid. Again, we have easy access to off-diagonal confusions on the sphere $\partial \Scal_\lambda$ corresponding to BO or IBO off-diagonal confusions for different LPMs (Lemma~\ref{mult-lem:spherebayes}); therefore, we can use the following algorithm, which is analogous to Algorithm~\ref{append:alg:d-grid-search}. 

\textbf{Algorithm~4.6} \emph{LFPM: grid-search for best pairwise ratios:} This is same as Algorithm~\ref{append:alg:d-grid-search} except the following two changes. First, the second line of Algorithm~\ref{append:alg:d-grid-search} will have a for loop running from 2 to $q-1$. Second, in line 4, samples will be generated from the surface of the sphere $\partial \Scal_\lambda$ as discussed above, instead of $\partial \Dcal_{1, k}$.

\balgorithm[t]
\caption{LFPM: Grid-Search for Best Pairwise Ratios}
\label{mult-alg:lfpmgridsearch}
\balgorithmic[1]
\STATE Follow Algorithm~\ref{append:alg:d-grid-search} except:
\STATE Run the for loop in step 2 of Algorithm~\ref{append:alg:d-grid-search} for 2 to $q-1$.
\STATE Generate samples from $\partial\Scal_\lambda$.
\ealgorithmic
\ealgorithm

\subsection{Monotonic Metrics of diagonal confusions }
\label{append:sec:monotonic}

Recall that the space $\Dcal$ is strictly convex. Suppose that the oracle's metric is $\psi^\ast$, which is just monotonic increasing in $\{d_i\}_{i=1}^k$. Let $\ambf^\ast$ be the slope of the supporting hyperplane at the optimal diagonal confusions $\dmbf^\ast$. Then we may use Algorithm~4.1 which will return a linear metric $\ambfhat$ by using pairwise comparisons. Notice that, we may then compute an estimate of the BO diagonal confusions $\dmbfhat$ using Proposition~\ref{prop:d-bayes} corresponding to the output $\ambfhat$ of the algorithm.  Since the space $\Dcal$ is strictly convex, $\inner{\ambfhat}{\dmbf} = \inner{\ambfhat}{\dmbfhat}$ becomes the estimate of the unique supporting hyperplane at $\dmbfhat$.

The first order approximation of $\psi^\ast$ at $\dmbfhat$ can be given by:
\bequation
\psi^\ast(\dmbf) = \psi^\ast(\dmbfhat) + \inner{\ambfhat}{\dmbf - \dmbfhat}. 
\eequation
Since performance metrics are not affected by scale and additive biases, then the first order approximation given by $\inner{\ambfhat}{\dmbf}$ suffices for the elicitation task. Notice that this is of high practical importance to practitioners, since this is an estimate of the weighted accuracy at the estimate of the optimal diagonal confusions. 

%% file: multiclass/guarantees.tex
\section{Guarantees}
\label{mult-sec:guarantees}
\vskip -0.2cm
We discuss robustness under the following feedback model, which is useful in practical scenarios, and is borrowed from Definition~\ref{me-def:noise}.  

\bdefinition[Oracle Feedback Noise: $\epsilon_\Omega \geq 0$] The oracle responds correctly as long as $|\phi(\cmbf) - \phi(\cmbf')| > \epsilon_\Omega$ (analogously $|\psi(\dmbf) - \psi(\dmbf')| > \epsilon_\Omega$). Otherwise, it may provide incorrect answers.
\label{def:noise}
\edefinition

In other words, the oracle may respond incorrectly if the confusions are too close as measured by the metric $\phi$ (analogously $\psi$). Next, we discuss elicitation guarantees for DLPM and LPM elicitation. 

\btheorem Given $\epsilon,\epsilon_\Omega\geq 0$, and a 1-Lipschitz DLPM $\psi^{\ast}$ parametrized by $\ambf^\ast$. Then the output $\ambfhat$ of Algorithm~4.1 after $O((k-1)\log \tfrac 1\epsilon)$ queries to the oracle satisfies $\Vert \ambf^*-\ambfhat \Vert_{\infty}\leq O(\epsilon+\sqrt{\epsilon_\Omega})$, which is equivalent to $\Vert \ambf^*-\ambfhat \Vert_{2}\leq O(\sqrt k (\epsilon+\sqrt{\epsilon_\Omega}))$ using standard norm bounds.  
\label{thm:dlpm-elit-error}
\etheorem
\vspace{-0.5cm}

Next, we guarantee LPM elicitation when the sphere radius dominates the oracle noise. 

\btheorem 
 Given $\epsilon,\epsilon_\Omega\geq 0$, and a 1-Lipschitz LPM $\phi^\ast$ parametrized by $\ambf^*$. Suppose $\lambda \gg \epsilon_\Omega$, then the output $\ambfhat$ of Algorithm~4.2 after $O\left(z_1\log(z_2/(q\epsilon^2))(q-1)\log \tfrac \pi {2\epsilon}\right)$ queries satisfies $\Vert \ambf^*-\ambfhat \Vert_{2}\leq O(\sqrt{q}(\epsilon+\sqrt{\epsilon_\Omega/\lambda}))$, where $z_1, z_2$ are constants independent of $\epsilon$ and $q$.  
 \label{thm:lpm-elit-error}
\etheorem 

We see that the algorithms are robust to noise, and their query complexity depends linearly in the unknown entities. The term $z_1\log(z_2/(q\epsilon^2))$ may attribute to the number of cycles in Algorithm~4.2, but due to the  curvature of the sphere, we observe that it is not a dominating factor in the query complexity. 
For instance, we find that when $\epsilon = 10^{-2}$, two cycles (i.e. $T = 2(q-1)$ in Algorithm~4.2) are sufficient for achieving elicitation up to the error tolerance $\sqrt{q}\epsilon$. 
Moreover, the query complexity in Theorem~\ref{thm:lpm-elit-error} is optimal. We show this in Chapter~\ref{chp:quadratic} for the quadratic elicitation case, which in turn applies to the above linear elicitation case as well. One remaining question for LPM elicitation is to select a sufficiently large value of $\lambda$. Algorithm~\ref{mult-alg:lambda} (Appendix~\ref{append:ssec:slambda}) provides an offline procedure to compute a $\lambda\geq \tilde{r}/k$, where $\tilde{r}$ is the radius of the largest ball contained in the set $\Ccal$. 

{\bf ME with Finite Samples:}
As a final step, we consider the following questions when working with finite samples: (a) do we get the correct feedback from querying  $\Omega(\cmbfhat, \cmbfhat')$ instead of querying $\Omega(\cmbf, \cmbf')$? (b) what is the effect of $\hat \eta_i$'s when used in place of true $\eta_i$'s? The answers are straightforward. Since the sample estimates of confusion matrices are consistent estimators and the metrics discussed are $1$-Lipschitz with respect to the confusion matrices, with high probability, we gather correct oracle feedback as long as we have sufficient samples. Furthermore, 
 subject to regularity assumptions, Lemma~\ref{lem:sample-Cs-optimize-well} shows that the errors due to using $\hat\eta$ affect the (binary) confusion matrices on the boundary in a controlled manner. Since  Algorithm~4.1 uses pairwise RBO (binary) classifiers, it inherits the error guarantees in the multiclass case. On the other hand, since Algorithm~4.2 does not use the boundary, its results are agnostic to finite sample error as long as the sphere is contained within $\Ccal$. 

%% file: multiclass/experiments.tex
\vspace{-0.2cm}
\section{Experiments}
\label{sec:extensions}
In this section, we empirically validate the results of theorems~\ref{thm:dlpm-elit-error} and~\ref{thm:lpm-elit-error} and investigate sensitivity due to finite sample estimates.\footnote{A subset of results is shown here. Refer Appendix~\ref{append:sec:experiments} for more results.} 
For the ease of judgments, we show results for $k=3$ and $k=4$ classes.

\subsection{Synthetic Data Experiments}
\label{ssec:theoryexp}
\begin{table}
	\caption{DLPM elicitation at $\epsilon = 0.01$ for synthetic data. The number of queries used for $k=3$ and $k=4$ is 56 and 84, respectively.}  
\vskip -0.3cm
	\label{tab:DLPMtheory}
	\begin{center}
		\begin{small}
				\begin{tabular}{|c|c|c|c|}
					\hline
					\multicolumn{2}{|c|}{Classes $k=3$} & \multicolumn{2}{c|}{Classes $k=4$} \\ \hline
					  $\psi^{\ast} = \ambf^\ast$ & $\hat \psi = \ambfhat$ & $\psi^{\ast} = \ambf^\ast$ & $\hat \psi = \ambfhat$ \\ \hline 
					  (0.21, 0.59, 0.20) & (0.21, 0.60, 0.20) & (0.22, 0.13, 0.14, 0.52) & (0.22, 0.13, 0.14, 0.52)  \\
					  (0.23, 0.15, 0.62) & (0.23, 0.15, 0.62) &  (0.58, 0.17, 0.08, 0.18) & (0.58, 0.17, 0.08, 0.18) \\
					\hline
				\end{tabular}
		\end{small}
	\end{center}
	\vskip -0.4cm
\end{table}
\begin{table}
	\caption{LPM elicitation at $\epsilon = 0.01$ for synthetic data. The number of queries used for $k=3$ and $k=4$ is 320 and 704, respectively.}
	\vskip -0.3cm
	\label{tab:LPMtheory}
	\begin{center}
		\begin{small}
				\begin{tabular}{|c|c|c|}
					\hline
					  Classes & $\sphi = \ambf^\ast$ & $\hphi = \ambfhat$ \\ \hline 
					  3 & (-0.37, -0.89, -0.09, -0.23, -0.04, -0.03) & (-0.37, -0.89, -0.09, -0.23, -0.04, -0.03)  \\ 
					  3 & (-0.80, -0.55, -0.18, -0.08, -0.14, -0.05) & (-0.80, -0.55, -0.18, -0.08, -0.14, -0.05)  \\ \hline
					  \multirow{2}{*}{4}  & (-0.90, -0.28	-0.10, -0.31, -0.04,	-0.05, & (-0.90,	-0.28,	-0.10,	-0.31,	-0.04,	-0.05, \\
					  & -0.03,	-0.04,	-0.02,	-0.01,	-0.01,	-0.01) & -0.03,	-0.04,	-0.02,	-0.01,	-0.01,	-0.01) \\ 
					  \multirow{2}{*}{4} & (-0.54,	-0.10,	-0.62,	-0.52,	-0.03,	-0.07, & (-0.55,	-0.11, -0.62,	-0.51,	-0.03,	-0.07,  \\ 
					  & -0.11,	-0.07,	-0.14,	-0.03,	-0.03,	-0.04) & -0.11,	-0.07,	-0.14, -0.03,	-0.03,	-0.04)\\
					\hline
				\end{tabular}
		\end{small}
	\end{center}
	\vskip -0.40cm
\end{table}

\begin{figure*}[t]
	\centering 
	\subfigure{
		{\includegraphics[width=5cm]{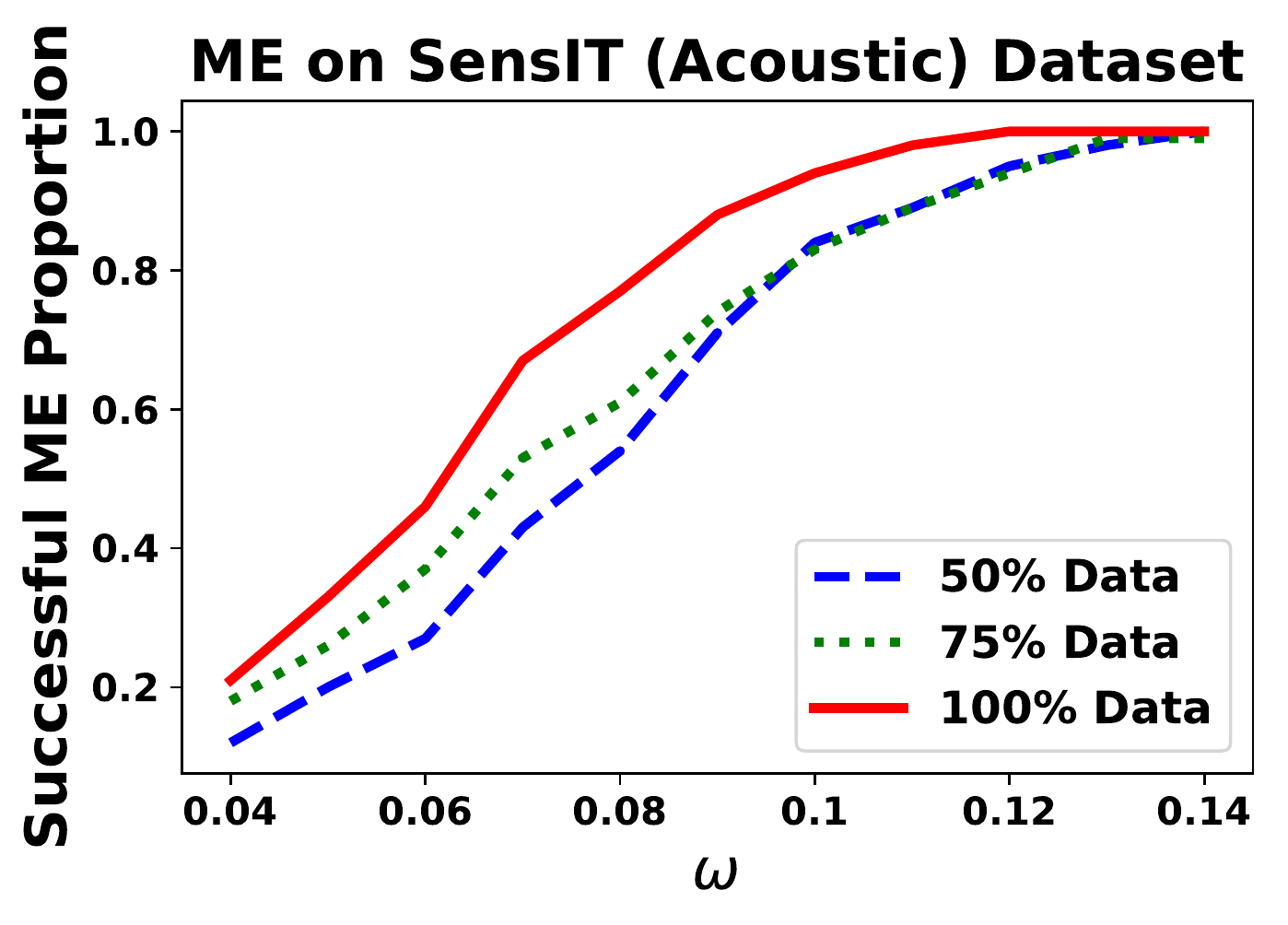}}
		\label{fig:app:lfpm_1}
	}
	\qquad
	\subfigure{
		{\includegraphics[width=5cm]{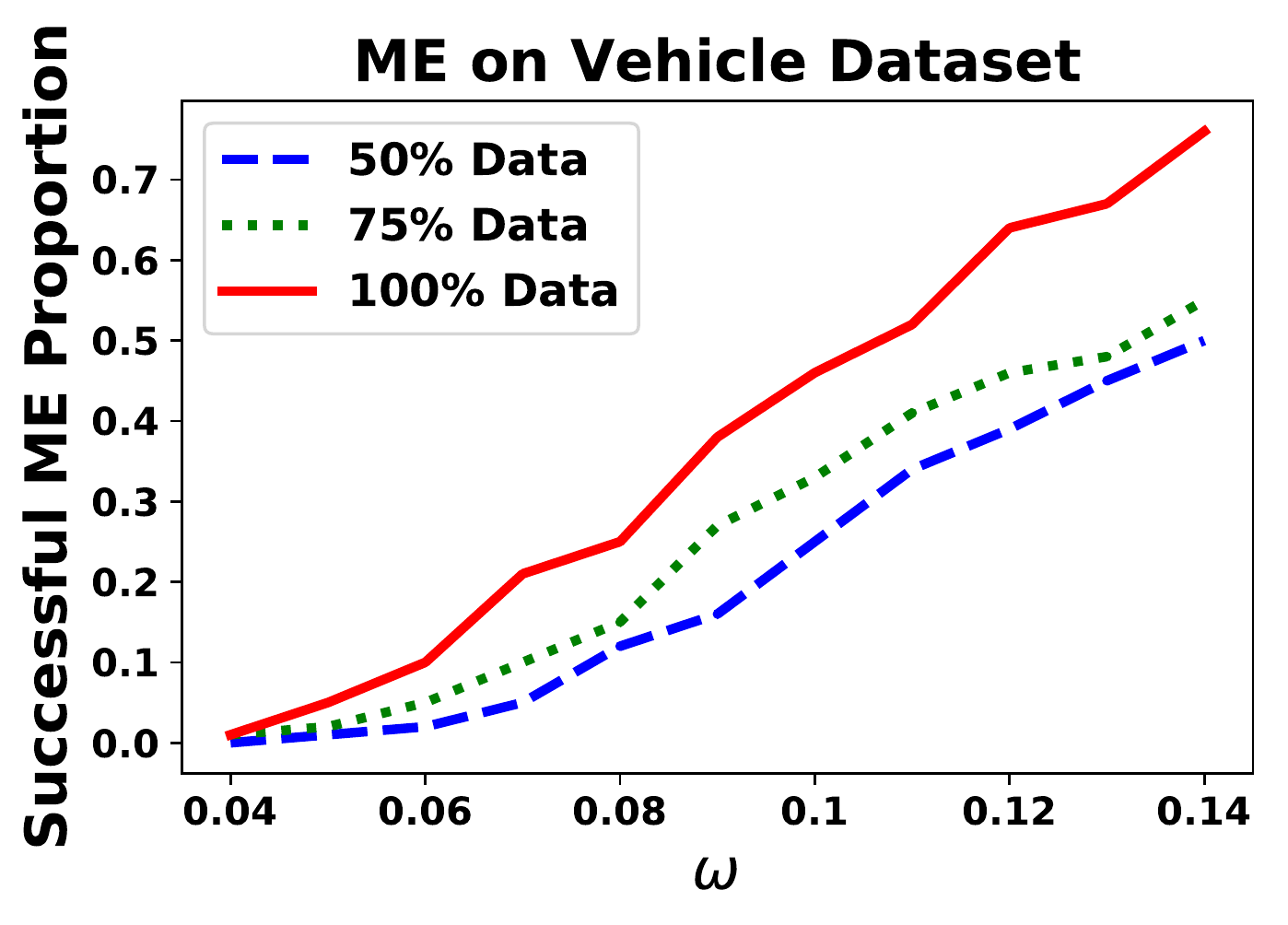}}
		\label{fig:app:lfpm_2}
	}
	\vskip -0.25cm
	\caption{DLPM elicitation on real data for $\epsilon = 0.01$. For randomly chosen hundred $\ambf^\ast$, we show the proportion of times our estimates $\ambfhat$ obtained with $4(k-1)\ceil*{\log(1/\epsilon)}$ queries satisfy $\Vert \ambf^\ast - \ambfhat \Vert_\infty \leq \omega$.}
	\label{fig:dlpmreal}
	\vskip -0.4cm
\end{figure*}

We assume a joint distribution for $\Xcal = [-1,1]$ and $\Ycal = [k]$. This is given by the marginal distribution $f_X = \Umbb[-1,1]$ and $\eta_i(x) = \frac{1}{1 + e^{p_i x}}$ for $i \in [k]$, where $\Umbb[-1,1]$ is the uniform distribution on $[-1, 1]$ and $\{p_i\}_{i=1}^k$ are the parameters controlling the degree of noise in the labels. We fix $(p_1, p_2, p_3) = (1,3,5)$ and $(p_1, p_2, p_3, p_4) = (1,3,6, 10)$ for experiments with three and four classes, respectively. 
To verify elicitation, we first define a true metric $\psi^{\ast}$ or $\sphi$. This specifies the query outputs of Algorithm~4.1 or Algorithm~4.2. Then we run the algorithms to check whether or not we recover the same metric. Some results are shown in Table~\ref{tab:DLPMtheory} and Table~\ref{tab:LPMtheory}. Results verify that we elicit the true metrics even for small $\epsilon = 0.01$, and as predicted, this requires only $ 4(k-1)\ceil*{\log(1/\epsilon)}$ and $ 4T\ceil*{\log(\pi/2\epsilon)}$ queries for DLPM and LPM elicitation respectively, where $\ceil*{\cdot}$ is the ceil function and $T=2(q-1)$.

\subsection{Real-World Data Experiments}
\label{ssec:realexp}
\vskip -0.1cm
Finite samples may affect the size of the sphere $S_\lambda$ in LPM elicitation, but we observe that as long as $\lambda$ is greater than $\epsilon_\Omega$ LPMs can be elicited (Appendix~\ref{append:ssec:sphereexp}). Thus, here  we emprically validate only DLPM elicitation with finite samples. 
We consider two real-world datasets: (a) SensIT (Acoustic) dataset~\cite{duarte2004vehicle} (78823 instances, 3 classes), and (b) Vehicle dataset~\cite{siebert1987vehicle} (846 instances, 4 classes). From each dataset, we create two other datasets containing randomly chosen $50\%$ and $75\%$ of the datapoints. So, we have six datasets in total. For all the datasets, we standardize the features and split the dataset into two parts $\Scal_1$ and $\Scal_2$. On $\Scal_1$, we learn $\{ \hat\eta_i(x) \}_{i=1}^k$ using a regularized softmax regression model. 
We use $\Scal_2$ for making predictions and computing sample confusions. 

We randomly selected 100 DLPMs i.e. $\ambf^\ast$'s. 
We then used Algorithm~4.1 with $\epsilon=0.01$ to recover the estimates $\ambfhat$'s. 
In Figure~\ref{fig:dlpmreal}, we show the proportion of times $\Vert \ambf^\ast - \ambfhat \Vert_\infty \leq \omega$ for different values of $\omega$. 
We see improved elicitation as we increase the number of datapoints in both the datasets, suggesting that ME improves with larger datasets. In particular, for the full SensIT (Acoustic) dataset, we elicit all the metrics within $\omega = 0.12$. We also observe that $\omega \in [0.04, 0.08]$ is an overly tight evaluation criterion that can result in failures. 
This is because the elicitation routine gets stuck at the closest achievable sample confusions, which need not be optimal within the (small) search tolerance $\epsilon$. 

%% file: multiclass/discussion.tex
\vspace{-0.5cm}
\section{Discussion and Future Work}
\label{mult-sec:discussion}
\bitemize[itemsep=0.1em, leftmargin=*]

\item\textbf{Practical Convenience.} Our procedures can also be applied by posing pairwise classifier comparisons directly. One way is to use A/B testing~\cite{tamburrelli2014towards} where the user population acts an oracle. Another way is to use comparisons from a single expert, perhaps combined with interpretable machine learning techniques~\cite{ribeiro2016should,doshi2017towards}. 
We suggest the approach proposed by Narasimhan~\cite{narasimhan2018learning} for estimating the classifier associated with a given confusion matrix. 
\item \textbf{Advantage of Algorithm~4.1.} 
If there is a reason to restrict the metric search to DLPM e.g. due to prior knowledge, then Algorithm~\ref{mult-alg:dlpm} is preferred for its lower query complexity. 
\item \textbf{Future Work.} 
We plan to extend our procedures for the oracles that are only probably correct. This can be done easily by applying majority voting over repeated queries~\cite{kaariainen2006active}.
\eitemize
\vskip -0.3cm

%% file: multiclass/relatedwork.tex
\vspace{-0.5cm}
\section{Related Work}
\label{sec:relatedwork}
\vskip -0.2cm
The closest line of work to this chapter  is the simpler setting of binary classification from Chapter~\ref{chp:binary}. As we move to multiclass performance ME, we find that the form of metrics and the complexity of the query space increases. This results in stark differences in the elicitation algorithms. Algorithm~4.1, which is closest to the binary approach, only works for Restricted Bayes Optimal classifiers, and Algorithm~4.2 requires a coordinate-wise binary-search approach. As a result, novel methods are also required to provide query complexity guarantees. 
The LPM elicitation problem can be posed as a Derivative-Free Optimization~\cite{jamieson2012query} to a certain extent, but only after exploiting the geometry as we have. 
In addition, passively learning linear functions using pairwise comparisons has been studied before~\cite{herbrich2000large, joachims2002optimizing, peyrard2017learning}, but these approaches fail to control sample (i.e. query) complexity and end up utilizing more queries than the active approaches~\cite{settles2009active, jamieson2011active, kane2017active}. Papers which actively control the query samples for linear elicitation, e.g.~\cite{qian2015learning}, exploit the query space like us in order to achieve lower query complexity. However, unlike us,~\cite{qian2015learning} does not provide theoretical bounds and is also applied to a different query space.

%% file: multiclass/conclusion.tex
\section{Concluding Remarks}
\label{mult-sec:conclusion}
We study the space of multiclass confusions and propose efficient algorithms to elicit diagonal-linear and linear performance metrics. We theoretically show that the procedures are robust under feedback and finite sample noise and validate the latter empirically via simulated oracles. 
We extend elicitation to other families e.g. linear-fractional metrics, thus covering a wide range of metrics encountered in practice. 

%% file: fair/header.tex
\chapter{Fair Performance Metric Elicitation}
\label{chp:fair}

\input{fair/introduction}
\input{fair/background}
\input{fair/confusion}
\input{fair/elicitation}
\input{fair/guarantees}
\input{fair/experiments}
\input{fair/relatedwork}
\input{fair/discussion}

%% file: fair/introduction.tex
Machine learning models are increasingly employed for critical decision-making tasks such as
hiring and sentencing~\cite{singla2015learning, angwin2016machine, corbett2017algorithmic, friedler2019comparative, lahoti2019ifair}. Yet, it is increasingly evident that automated decision-making is susceptible to bias, whereby decisions made by the algorithm are unfair to certain subgroups~\cite{barocas2016big,angwin2016machine, chouldechova2017fair, berk2018fairness, lahoti2019ifair}. To this end, a wide variety of group fairness metrics 
have been proposed – all to reduce discrimination and bias from automated decision-making~\cite{kamishima2012fairness, dwork2012fairness, hardt2016equality, kleinberg2017inherent, woodworth2017learning, menon2018cost}. However, 
a dearth of formal principles for selecting the  most appropriate metric has highlighted the confusion of experts, practitioners, and end users in deciding which group fairness metric to employ~\cite{zhang2020joint}. This is further exacerbated by the observation that common metrics often lead to contradictory outcomes~\cite{kleinberg2017inherent}. 

While the problem of selecting an appropriate fairness metric has gained prominence in recent years~\cite{hardt2016equality, menon2018cost,zhang2020joint}, it perhaps best understood as a special case of the task of choosing evaluation metrics in machine learning. For instance, when a cost-sensitive predictive model classifies patients into cancer categories~\cite{yang2014multiclass} even without considering fairness, it is often unclear how the cost-tradeoffs be chosen so that they reflect the expert’s decision-making, i.e., replacing expert intuition by quantifiable metrics. The proposed {\rm Metric Elicitation} (ME) framework provides a solution. 

Existing research suggests a fundamental trade-off between algorithmic fairness and performance~\cite{kamishima2012fairness, zafar2017fairness, corbett2017algorithmic, bechavod2017learning, menon2018cost, zhang2020joint}, 
where in addition to appropriate 
metrics, the practitioner or policymaker must choose a trade-off operating point between the competing objectives~\cite{zhang2020joint}. To this end, in this chapter, we extend the ME framework from eliciting multiclass classification metrics to the task of eliciting \emph{fair} performance metrics from pairwise preference feedback 
in the presence of multiple sensitive groups. In particular, we elicit metrics that reflect, jointly, the (i) predictive performance evaluated as a weighting of classifier's overall predictive rates, (ii) fairness violation assessed as the discrepancy in predictive rates among groups, and (iii) a trade-off between the predictive performance and fairness violation. Importantly, the elicited metrics are sufficiently flexible to encapsulate and generalize many existing predictive performance and fairness violation measures.

In eliciting group-fair performance metrics, we tackle three new challenges. 
First, from preference query perspective, the predictive performance and fairness violations are correlated, thus increasing the complexity of joint elicitation.
Second, we find that in order to measure both positive and negative violations, the fair metrics are necessarily non-linear functions of the predictive rates, thus existing results on linear ME from previous chapters cannot be applied directly. Finally, as we show, the number of groups directly impacts query complexity. 
We overcome these  challenges by proposing a novel query efficient procedure that exploits the geometric properties of the set of predictive rates. 

{\bf Contributions.} We consider metrics for algorithmically group-fair classification and propose a novel approach for eliciting predictive performance, fairness violations, and their trade-off point, from expert pairwise feedback.
Our procedure uses binary-search based subroutines and recovers the metric with linear query complexity.
Moreover, the procedure is robust to both  finite sample and oracle feedback noise thus is useful in practice. Lastly, our method can be applied either by querying preferences over classifiers or predictive  rates, which is our choice of measurements (classifier statistics) for this chapter. 
All the proofs in this chapter are provided in Appendix~\ref{apx:fair}.

\textbf{Notations.} Matrices and vectors are denoted by bold upper case and bold lower case letters, respectively. 
The group membership is denoted by superscripts and coordinates of vectors, matrices, and tuples are denoted by subscripts.

%% file: fair/background.tex
\section{Background}
\label{sec:preliminaries}

The standard multiclass, multigroup classification setting comprises $k$ classes and $m$ groups with $X \in \Xcal$, $G \in [m]$ and $Y \in [k]$ representing the input, group membership, and output random variables, respectively. The groups are assumed to be disjoint and known apriori~\cite{hardt2016equality, kleinberg2017inherent}. We have access to a dataset $\{(\xmbf, g, y)_i\}_{i=1}^n$ of size $n$, generated \emph{iid} from a distribution $ \Pmbb(X, G, Y)$. The measurements (classifier statistics) that we choose to work with in this chapter are the group-specific rates and the overall rates, which are described below.  

\emph{Group-specific rates:} 
We consider separate (randomized) classifiers $h^g : \Xcal \rightarrow \Delta_k$ for each group $g$, and use 
 \bequation
 \Hcal^g = \{h^g : \Xcal \rightarrow \Delta_k\}
 \eequation
 to denote the set of all classifiers for group $g$. 
The group-specific rate matrix $\Rmbf^g(h^g, \Pmbb) \in \Rmbb^{k \times k}$ for a classifier $h^g$ 
is given by:
\begin{align}
	R^g_{ij}(h^g, \Pmbb) \coloneqq \Pmbb(h^g = j | Y = i, G= g )  \quad \text{for} \; i, j \in [k].
	\label{eq:components}
\end{align}
Notice that the predictive rates satisfy the following useful decomposition: 
\begin{equation}
    R_{ii}^g(h^g, \Pmbb) = 1 - \sum\nolimits_{j=1,j\neq i}^k R_{ij}^g(h^g, \Pmbb),
    \label{eq:decomp}
\end{equation}
any rate matrix is uniquely represented by its $q \coloneqq (k^2 - k)$ off-diagonal elements  as a vector $\rmbf^g(h^g, \Pmbb) = \offdiag(\Rmbf^g(h^g, \Pmbb))$. So we  will  interchangeably refer to the rate matrix as a \emph{`vector of rates'}. 
The feasible set of rates associated with a group $g$ is denoted by
\bequation
\Rcal^g = \{\rmbf^g(h^g, \Pmbb) \,:\, h^g \in \Hcal^g \}.
\eequation
For clarity, we will suppress the dependence on $\Pmbb$ and $h^g$ if it is clear from the context.

\emph{Overall rates:} We define the overall classifier $h : (\Xcal, [m]) \rightarrow \Delta_k$ by 
\bequation
h(\xmbf, g) \coloneqq h^g(\xmbf)
\eequation
and denote its tuple of group-specific rates by:
\bequation
\rmbf^{1:m} \coloneqq  (\rmbf^1, \dots, \rmbf^m) \in \Rcal^1 \times \dots \times \Rcal^m =: \prodRcal.
\eequation
This tuple allows us to measure the fairness violation across groups. The fairness violation is believed to be in trade-off with the predictive performance~\cite{kamishima2012fairness, bechavod2017learning, menon2018cost}.  The latter is measured using the overall rate matrix of the classifier $h$:
\begin{align*}
R_{ij} \coloneqq \Pmbb(h=j|Y=i) 
= \sum\nolimits_{g=1}^m t_{i}^gR_{ij}^g,
    \numberthis \label{eq:overallrate}
\end{align*}
where $ t^g_{i} \coloneqq \Pmbb(G=g|Y=i)$ is the prevalence of group $g$ within class $i$. For an overall classifier $h$, the \emph{`vector of rates'} $\rmbf = \offdiag(\Rmbf)$ can be conveniently written in terms of its group-specific tuple of rates as \bequation
\rmbf =  \sum_{g=1}^m \bm{\tau}^g \odot \rmbf^g,
\eequation
where $\bm{\tau}^g \coloneqq \offdiag([\tmbf^g \; \tmbf^g \dots \tmbf^g]).$ 

{\em Fairness violation measure:} The (approximate) fairness of a classifier is often determined by the `discrepancy' in rates across different groups e.g. \emph{equalized odds}~\cite{hardt2016equality, barocas2017fairness}. So given two groups $u, v\in[m]$, we define the discrepancy in their rates as:
\vspace{-0.1cm}
\begin{equation}
    \dmbf^{uv} \coloneqq \vert \rmbf^{u} - \rmbf^{v} \vert.
    \label{eq:discrepancy}
\end{equation}
\vskip -0.3cm
\noindent Since there are $m$ groups, the number of \emph{discrepancy vectors} are $\tiny{{m\choose 2}}$ . 

\vspace{-0.15cm}
\subsection{Fair Performance Metric}
\label{ssec:metric}
\vskip -0.1cm
We aim to elicit a general class of metrics, which recovers and generalizes existing fairness measures, based on trade-off between predictive performance and fairness violation~\cite{kamishima2012fairness, hardt2016equality, chouldechova2017fair, bechavod2017learning, menon2018cost}.
Let $\phi : [0, 1]^{q}  \rightarrow \Rmbb$ be the cost of overall misclassification (aka.\ predictive performance) and $\varphi : [0, 1]^{m \times q} \rightarrow \Rmbb$ be the fairness violation cost for a classifier $h$ 
determined by the overall rates $\rmbf(h)$ and group discrepancies 
$\{\dmbf^{uv}(h)\}_{u,v = 1, v >u}^m$, respectively. 
Without loss of generality (w.l.o.g.), we assume the metrics $\phi$ and $\varphi$ are costs. 
Moreover, the metrics 
are scale invariant as global scale does not affect the learning problem~\cite{narasimhan2015consistent}; hence let $\phi : [0, 1]^{q}  \rightarrow [0,1]$ and $\varphi : [0, 1]^{m\times q}  \rightarrow [0,1]$.

\bdefinition[Fair Performance Metric] Let $\phi$
and $\varphi$ be monotonically increasing linear functions of overall rates and group discrepancies, respectively. The fair metric $\Psi$ is a trade-off between $\phi$ and $\varphi$. 
In particular, given $\ambf \in \Rmbb^q, \ambf \geq 0$ (misclassification weights), a set of vectors $\Bmbf \coloneqq \{\bmbf^{uv} \in \Rmbb^q, \bmbf^{uv}\geq0\}_{u, v=1, v>u}^m$ (fairness violation weights), and a scalar $\lambda$ (trade-off) with
\begin{align*}
    \Vert \ambf \Vert_2 = 1, \quad \quad \sum\nolimits_{u,v=1, v>u}^{m} \Vert \bmbf^{uv} \Vert_2 = 1, \quad \quad 0 \leq \lambda \leq 1, \numberthis
    \label{eq:scaleinvariance}
\end{align*}
(w.l.o.g., due to scale invariance), we define the metric $\Psi$ as:
\begin{align*}
    \Psi(\tupr \,;\, \ambf, \Bmbf, \lambda) \,\coloneqq\,  \underbrace{(1-\lambda)}_{\text{trade-off}}\underbrace{\inner{\ambf}{\rmbf}}_{\phi(\rmbf)} + \lambda \underbrace{\left(\sum\nolimits_{u,v=1,v>u}^{m} \inner{\mathbbm{\bmbf}^{uv}}{\dmbf^{uv}}\right)}_{\varphi(\tupr)}
    \numberthis \label{eq:linmetric}.
\end{align*}
\label{def:linear}
\edefinition

Examples of the misclassification cost $\phi(\rmbf)$ include cost-sensitive linear metrics~\cite{abe2004iterative}. Many existing fairness metrics for two classes and two groups such as \emph{equal opportunity}~\cite{hardt2016equality}, \emph{balance for the negative class}~\cite{kleinberg2017inherent} \emph{error-rate balance} (i.e., $0.5|r_1^1 - r^2_1| + 0.5|r_2^1 - r^2_2|)$~\cite{chouldechova2017fair}, \emph{weighted equalized odds} (i.e., $b_1|r_1^1 - r^2_1| + b_2|r_2^1 - r^2_2|)$~\cite{hardt2016equality, bechavod2017learning}, etc. correspond to fairness violations of the form $\varphi(\tupr)$ considered above. The combination of $\phi(\rmbf)$ and $\varphi(\tupr)$ as defined in $\Psi(\tupr)$ appears regularly in prior work~\cite{kamishima2012fairness, bechavod2017learning, menon2018cost}. 
Notice that the metric is flexible to allow different fairness violation costs for different pairs of groups 
thus capable of enabling reverse discrimination~\cite{opotow1996affirmative}. 
Lastly, while the metric is linear with respect to (w.r.t.) the discrepancies, it is non-linear w.r.t.\ the group-wise rates. Hence, standard linear ME algorithm from Chapters~\ref{chp:binary} and~\ref{chp:multiclass} cannot be trivially applied for eliciting the metric in Definition~\ref{def:linear}.

\subsection{Fair Performance Metric Elicitation; Problem Statement}
\label{ssec:me}
We now state the problem of \emph{Fair Performance Metric Elicitation (FPME)} and define the associated \emph{oracle query}. The broad definitions follow from Chapter~\ref{chp:me}, extended so the predictive rates (classifier statistics) and the performance metrics correspond to the multiclass multigroup-fair classification setting.

\bdefinition
[Oracle Query] Given two classifiers $h_1,  h_2$ (equivalent to a tuple of rates $\tupr_1, \tupr_2$ respectively), a query to the Oracle (with metric $\Psi$) is represented by:
\begin{align}
\Gamma(h_1, h_2\,;\,\Psi) = \Omega\left( \tupr_1, \tupr_2\,;\,\Psi\right) &= \1[\Psi(\tupr_1) > \Psi(\tupr_2)],
\end{align}
where $\Gamma: \Hcal \times \Hcal \rightarrow \{0,1\}$ and $\Omega: \prodRcal \times \prodRcal \rightarrow \{0, 1\}$. In simple words, the query asks whether $h_1$ is preferred to $ h_2$ (equivalent to whether $\tupr_1$ is preferred to $\tupr_2$), as measured by $\Psi$.
\label{fair-def:query}
\edefinition

In practice, the oracle can be an expert, a group of experts, or an entire user population. The ME framework can be applied by posing classifier comparisons directly to them via interpretable learning techniques~\cite{ribeiro2016should, doshi2017towards} or via A/B testing~\cite{tamburrelli2014towards}. For example, one may perform A/B testing for an internet-based application by deploying two classifiers A and B 
and use the population's level of engagement to decide the preference between the two classifiers. For other applications, intuitive visualizations of the predictive rates for two different classifiers (see e.g.,  \cite{zhang2020joint,beauxis2014visualization}) can be used to ask preference feedback from a group of domain experts. 

We emphasize that the metric $\Psi$ used by the  oracle is unknown to us and can be accessed only through queries to the oracle. 
Since the metrics we consider are functions of rates, comparing two classifiers on a metric is equivalent to comparing their corresponding rates. Henceforth, we will 
denote  any query to the oracle by a pair of rates $(\rmbf_1^{1:m}, \rmbf_2^{1:m})$.
Also, whenever we refer to an oracles's dimension, we are referring to the dimension of its rate arguments. For instance,  we will consider the oracle in Definition~\ref{fair-def:query} to be of dimension $m\times q$. 
Next, we formally state the FPME problem.

\bdefinition [Fair Performance Metric Elicitation with Pairwise Comparison Queries (given $\{(\xmbf,g,y)_i\}_{i=1}^n$)] Suppose that the oracle's (unknown) performance metric is $\Psi$. Using oracle queries of the form $\Omega(\hat \rmbf_1^{1:m}, \hat \rmbf_2^{1:m})$, where $\hat \rmbf_1^{1:m}, \hat \rmbf_2^{1:m}$ are the estimated rates from samples, recover a metric $\hPsi$ such that $\Vert \Psi - \hPsi \Vert < \omega$ under a  suitable norm $\Vert \cdot \Vert$ for sufficiently small error tolerance $\omega > 0$.
\label{def:me}
\edefinition
\vskip -0.05cm
Similar to the standard metric elicitation problems (Chapters~\ref{chp:binary} and~\ref{chp:multiclass}), the performance of FPME is evaluated both by the fidelity of the recovered metric and the query complexity. As done in decision theory literature~ \cite{koyejo2015consistent, hiranandani2018eliciting}, we present our FPME solution by first assuming access to population quantities such as the population rates $\tupr(h, \Pmbb)$, and then discuss how elicitation can be performed from finite samples, e.g., with empirical rates $\hat \rmbf^{1:m}(h, \{(\xmbf,g,y)_i\}_{i=1}^n))$.
\subsection{Linear Performance Metric Elicitation -- Warmup}
\label{fair-ssec:mpme}
We revisit the Linear Performance Metric Elicitation (LPME) procedure from Chapter~\ref{chp:multiclass}, which we will use as as a subroutine to elicit fair performance metrics. 
The LPME procedure assumes an enclosed sphere $\Scal \subset \Zcal$, where $\Zcal$ is the $q$-dimensional space of classifier statistics that are feasible, i.e., can be achieved by some classifier. 
It also assumes access to a $q$-dimensional oracle $\Omega'$ whose scale invariant linear metric is of the form $\xi(\zmbf) \coloneqq \inner{\ambf}{\zmbf}$ with $\Vert \ambf \Vert_2=1$, analogous to the misclassification cost in Definition~\ref{def:linear}. Analogously, the oracle queries are of the type $\Omega'( \zmbf_1, \zmbf_2 ) \coloneqq \1[\xi(\zmbf_1) > \xi(\zmbf_2)]$.

When the number of classes $k = 2$, LPME elicits the coefficients $\ambf$ using a simple one-dimensional binary search. When $k > 2$, LPME performs binary search in each coordinate while keeping the others fixed, and performs this in a coordinate-wise fashion until convergence. By restricting this coordinate-wise binary search procedure to posing queries from within a sphere $\Scal$, LPME can be equivalently seen as minimizing a strongly-convex function and shown to converge to a solution $\ambfhat$ close to $\ambf$. Specifically, 
the algorithm
takes the query space $\Scal\subset\Zcal$, binary-search tolerance $\epsilon$, and the oracle $\Omega'$ as input, and by querying $O(q\log(1/\epsilon))$ queries recovers $\ambfhat$ with $\Vert \ambfhat \Vert_2=1$  such that $\Vert \ambf - \ambfhat \Vert_2 \leq O(\sqrt{q}\epsilon)$ (Theorem~\ref{thm:lpm-elit-error} in Chapter~\ref{chp:multiclass}). Please see the details of the LPME procedure in Algorithm~4.2 (Chapter~\ref{chp:multiclass}) for completeness. We summarize the discussion with the following remark.

\bremark
Given a $q$-dimensional space $\Zcal$ enclosing a sphere $\Scal\subset \Zcal$ and an oracle $\Omega'$ with linear  metric $\xi(\zmbf)\coloneqq\inner{\ambf}{\zmbf}$, the LPME algorithm (Algorithm~4.2, Chapter~\ref{chp:multiclass}) provides an estimate $\ambfhat$ with $\Vert \ambfhat \Vert_2=1$ such that the estimated slope is close to the true slope, i.e.,  $\sfrac{{a}_i}{{a}_j} \approx \sfrac{\hat a_i}{\hat a_j} \; \forall \; i, j\in [q]$. 
\label{fair-rm:ratio}
\eremark
\vskip -0.1cm
 Note that the algorithm estimates the direction of the 
 coefficient vector, 
 not its magnitude. 
\vspace{-0.1cm}

%% file: fair/confusion.tex
\vspace{-0.6cm}
\section{Geometry of the Product Set $\prodRcal$}
\label{sec:confusion}

The LPME  procedure described above works with rate queries of dimension $q$. We would like to use this procedure to elicit the fair metrics in Definition~\ref{def:linear} defined on tuples of dimension $m\times q$. So to make use of LPME, we restrict our queries to a $q$-dimensional sphere $\Scal$ which is common to the feasible rate region $\Rcal^g$ for each group $g$, i.e., to a sphere in the intersection $\Rcal^1\cap\ldots\cap\Rcal^m$. We show now that such a sphere does indeed exist under a mild assumption.

\bassumption
For all groups, the conditional-class distributions are not identical, i.e., $\forall\;g\in[m], \forall \;i\neq j, \, \Pmbb(Y=i|X, G=g) \neq \Pmbb(Y=j|X, G=g).$ In other words, there is some non-trivial signal for classification for each group.
\label{as:clsconditional}
\eassumption

\begin{figure}[t]
    \centering
	\begin{tikzpicture}[scale = 1.6]
    \begin{scope}[shift={(4.3,0)},scale = 0.5]\scriptsize
    
    \def\r{0.12};
    
    
    \coordinate (a) at (-0.2,1);
    \coordinate (b) at (0.8, 2.75);
    \coordinate (c) at (7, 4);
    \coordinate (d) at (6.5, 2);
    \coordinate (e) at (2.5, -0.75);
    \coordinate (f) at (-0.1, -0.5);
    
    \coordinate (Cent) at (3,1.75);
    \coordinate (Centcent) at (2.85,1.90);
    \coordinate (CentR) at (3.6,2.35);
    \coordinate (CentL) at (2.4,1.15);
    
    \coordinate (Space1) at (0.2,0.2);
    \coordinate (Space2) at (-0.3,1.3);
    \coordinate (Spacem) at (-0.1,2.5);
    
    \coordinate (Sphere) at (5,2);
    \coordinate (Sphereplus) at (2.6,2.55);
    \coordinate (Sphereminus) at (3.4,0.95);
    
    \coordinate (r) at (3.75,2);
    
    \coordinate (u11) at (-0.25, -0.25);
    \coordinate (uextra121) at (1.25, 3.75);
    \coordinate (u21) at (4, 4.5);
    \coordinate (uextra2k1) at (6, 1);
    \coordinate (uk1) at (3.5, -0.5);
    
    \coordinate (u12) at (0.3, -0.70);
    \coordinate (uextra122) at (0.65, 2.65);
    \coordinate (u22) at (3.4, 4.2);
    \coordinate (uextra2k2) at (5.15, 1.6);
    \coordinate (uk2) at (3.75, 0.1);
    
    \coordinate (u13) at (-0.30, 0.10);
    \coordinate (uextra123) at (1.25, 3.25);
    \coordinate (u23) at (4, 4.5);
    \coordinate (uextra2k3) at (5.75, 1);
    \coordinate (uk3) at (3.75, -0.5);
    
    \coordinate (u14) at (-0.30, 0.10);
    \coordinate (uextra124) at (1.25, 3.25);
    \coordinate (u24) at (4, 4.5);
    \coordinate (uextra2k4) at (5.75, 1);
    \coordinate (uk4) at (3.75, -0.5);
    
    

    \fill[color=black] 
            (Cent) circle (0.08)
            
            
            (u11) circle (0.08)
            (u21) circle (0.08)
            (uk1) circle (0.08);
            
            
            
    
    
    
    \draw[dashed, blue, thick] (u11) .. controls (a) and (b) .. (uextra121) 
    -- (u21) .. controls (c) and (d) .. (uextra2k1) -- (uk1) .. controls  (e) and (f) .. (u11);
    
    \draw[dashed, brown, thick] (u11) .. controls (-2.25, 1.25) and (0.9, 3) .. (u21) .. controls (8.5,4.75) and (7,2) .. (6.5, 1.25) .. controls (6.5, 1) and (5.25, 0.3) .. (uk1) -- (u11);
    
    \draw[dashed, red, thick] (u11) .. controls (-1.5, 1.5) and (-0.5, 3.5) .. (1.5, 4.5) 
    -- (u21) .. controls (6.5, 3.5) and (6, 2) .. (6, 1.75) -- (uk1) .. controls  (3, -1) and (-0.25, -0.75) .. (u11);
    
    
    
    
    
    
    	
    \draw[thick] (Cent) circle (1.5cm);
    
    \draw[thick, dotted] (CentR) circle (0.58cm);
    
    \node at (Space1) {{$\Rcal^1$}};
    \node at (Space2) {{$\Rcal^2$}};
    \node at (Spacem) {{$\Rcal^m$}};
    
    \node at (Sphere) {\large{$\Scal_\rho$}};
    \node at (Sphereplus) {\tiny{$\Scal^+_{\varrho}$}};
    
   	
    \node[below right] at (Centcent) {$\ombf$};
    
    \node[below left] at (u11) {{$\embf_1$}};
    \node[above] at (u21) {{$\embf_2$}};
    \node[below right] at (uk1) {{$\embf_k$}};
    
    
    
    
     \end{scope}
    \end{tikzpicture}
      \captionof{figure}{$\Rcal^1 \times \dots \times \Rcal^m$ (best seen in colors); $\Rcal^u \, \forall \, u \in [m]$ are convex sets with common vertices $\embf_i \, \forall \, i \in [k]$ and enclose the sphere $\Scal_\rho$.}
        \vskip -0.3cm
      \label{fig:geometry}
\end{figure}

Let $\embf_i \in \{0,1\}^q$ be the rate profile for a trivial classifier that predicts class $i$ on all inputs. 
Note that these trivial classifiers evaluate to the same rates $\embf_i$ 
irrespective of which group we apply them  to.
\bprop
[Geometry of $\prodRcal$; Figure~\ref{fig:geometry}] For any group $g\in [m]$, the set of confusion rates $\Rcal^g$ is convex, bounded in $[0, 1]^{q}$, and has vertices $\{\embf_i\}_{i=1}^k$. The intersection of group rate sets $\Rcal^1 \cap \dots \cap \Rcal^m$ is convex and always contains the rate $\ombf = \tfrac{1}{k} \tiny{\sum_{i=1}^k \embf_i}$ in the interior, which is associated with the uniform random classifier that predicts each class with equal probability. 
\label{fair-prop:C}
\eprop
Since $\Rcal^1 \cap \dots \cap \Rcal^m$ is convex and always contains a point $\ombf$ in the interior, we can make the following remark (see Figure~\ref{fig:geometry} for an illustration).

\bremark[Existence of common sphere $\Scal_\rho$] There exists a $q$-dimensional sphere $\Scal_\rho \subset \Rcal^1 \cap \dots \cap \Rcal^m$ of non-zero radius $\rho$ centered at $\ombf$. Thus, any rate $\smbf \in\Scal_\rho$ is feasible for all groups, i.e., $\smbf$ is achievable by some classifier $h^g$ for all groups $g \in [m]$.
\label{remark:sphere}
\eremark
\vskip  -0.1cm
A method to obtain $\Scal_\rho$ with suitable radius $\rho$ from Chapter~\ref{chp:multiclass} is discussed in Appendix~\ref{fair-append:ssec:sphere}. From Remark~\ref{remark:sphere}, we observe that any tuple of group rates $\rmbf^{1:m} = (\smbf^1, \ldots, \smbf^m)$ chosen from $\Scal_\rho \times \ldots \times \Scal_\rho$ is achievable for some choice of group-specific classifiers $h^1, \ldots, h^m$. Moreover, when two groups $u, v$ are assigned the same rate profile $\smbf \in \Scal_\rho$, the fairness discrepancy $\dmbf^{uv} = \bm{0}$. We will exploit these observations in the elicitation strategy we discuss next. 
\vskip -0.1cm

%% file: fair/elicitation.tex
\section{Metric Elicitation}
\label{sec:me}

\begin{figure}[t]
\hspace*{-1cm}
  \centering
    \includegraphics[scale=0.8]{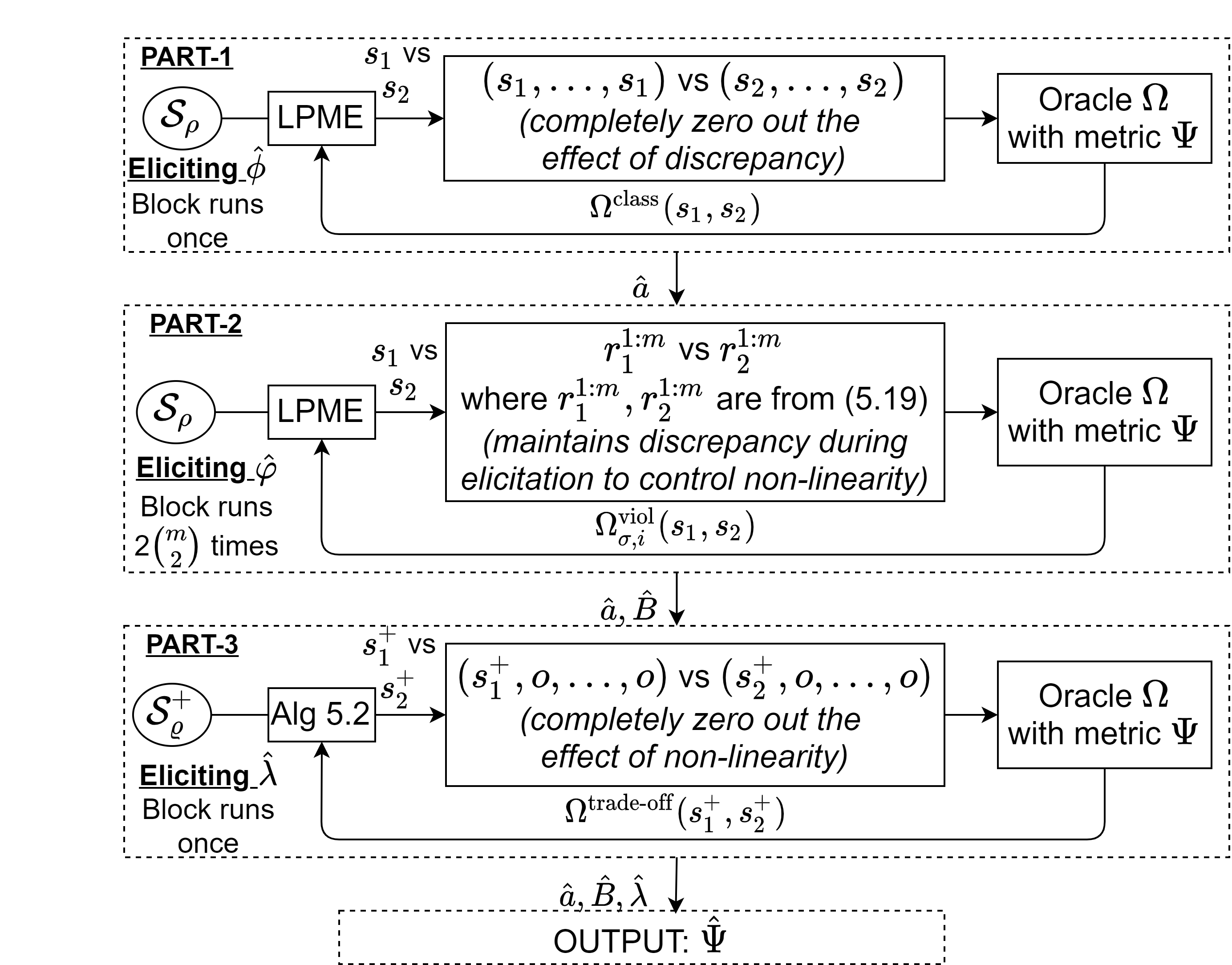}
    \caption{Workflow of the FPME procedure.}
    \label{fig:workflow}
\end{figure}

\balgorithm[t]
\caption{FPM Elicitation}
\label{alg:f-me}
\balgorithmic[1]
\STATE \textbf{Input:} Query spaces $\Scal_{\rho}$, $\Scal_{\varrho}^+$, 
search tolerance $\epsilon > 0$, and oracle $\Omega$
\STATE $\ambfhat \leftarrow$ LPME$(\Scal_\rho, \epsilon, \Omega^{\text{class}})$
\IF{$m==2$}
\STATE $\breve\fmbf\leftarrow$LPME$(\Scal_\rho, \epsilon, \Omega_1^{\text{viol}})$
\STATE $\tilde\fmbf\leftarrow$LPME$(\Scal_\rho, \epsilon, \Omega_2^{\text{viol}})$
\STATE $\bmbfhat^{12} \leftarrow $ normalized solution from~\eqref{eq:bsolm2}
\ELSE 
\STATE Let $\Lcal \leftarrow \varnothing$
\FOR{$\sigma \in \Mcal$}
\STATE $\breve\fmbf^{\sigma}\leftarrow$LPME$(\Scal_\rho, \epsilon, \Omega_{\sigma, 1}^{\text{viol}})$
\STATE $\tilde\fmbf^{\sigma}\leftarrow$LPME$(\Scal_\rho, \epsilon, \Omega_{\sigma, k}^{\text{viol}})$
\STATE Let $\ell^\sigma$ be Eq.~\eqref{eq:midsolb}, extend $\Lcal \leftarrow \Lcal \cup \{\ell^\sigma\}$
\ENDFOR
\STATE $\Bmbfhat \leftarrow $ normalized solution from~\eqref{eq:bsol} using $\Lcal$
\ENDIF
\STATE $\hat \lambda \leftarrow$ Algorithm~\ref{fair-alg:lambda} $(\Scal_{\varrho}^+, \epsilon, \Omega^{\text{trade-off}})$ \\
\STATE \textbf{Output:} $\ambfhat, \Bmbfhat, \hat \lambda$ 
\ealgorithmic
\ealgorithm

We have access to an oracle whose  (unknown) metric $\bPsi$ given 
in Definition~\ref{def:linear} is parameterized by $(\ambfbar, \Bmbfbar, \lambdabar)$. 
The proposed FPME framework for eliciting the oracle's metric is presented in Figure~\ref{fig:workflow} and is summarized in Algorithm~\ref{alg:f-me}. 
 
The  procedure has three parts executed in sequence: (a) eliciting the misclassification cost $\bphi(\rmbf)$ (i.e., $\ambfbar$), (b) eliciting the fairness violation $\bvarphi(\tupr)$ (i.e., $\Bmbfbar$), and (c) eliciting the trade-off between the misclassification cost and fairness violation (i.e., $\lambdabar$). For simplicity, we will suppress the coefficients $(\ambfbar, \Bmbfbar, \lambdabar)$ from the notation $\Psi$ whenever it is clear from context.

Notice that the metric $\Psi$ is \textit{piece-wise linear} in its coefficients.
So our high level idea  
is to restrict the queries we pose to the oracle to lie within regions where the metric $\Psi$ is  linear, so that we can then employ the LPME subroutine to elicit the corresponding linear coefficients. We will show for each of the three components (a)--(c), how we can identify regions in the query space where the metric is linear and apply the LPME procedure (or a variant of it). By restricting the query inputs to those regions, we will essentially be converting the $(m\times q)$-dimensional oracle $\Omega$ in Definition~\ref{fair-def:query} into an equivalent  $q$-dimensional oracle that compares  rates $\smbf_1, \smbf_2$ from the common sphere $ \Scal_\rho \subset \Rcal^1\cap\cdots\cap\Rcal^m$. 
We first discuss our approach assuming the oracle has no \emph{feedback} noise,  and later in Section~\ref{fair-sec:guarantees} show that our approach is robust to noisy feedback and provide query complexity guarantees.

\subsection{Eliciting the Misclassification Cost $\bphi(\rmbf)$: {Part 1 in Figure~\ref{fig:workflow} and Line 1 in Algorithm~\ref{alg:f-me}}}
\label{ssec:elicitphi}

To elicit the misclassification cost coefficients $\ambfbar$, we will query from a region of the query space where the fairness violation term in the metric is zero. Specifically, we will query group rate profile of the form $\rmbf^{1:m} = (\smbf, \dots, \smbf)$, where $\smbf$ is a $q$-dimensional rate from the common sphere $\Scal_\rho$. For these group rate profiles, the metric $\Psi$ simply evaluates to the linear misclassification term, i.e.:
\bequation
\bPsi(\smbf, \dots, \smbf) =  (1-\lambdabar)\inner{\ambfbar}{\smbf}.
\eequation
So given a pair of group rate profiles $\rmbf_1^{1:m} = (\smbf_1, \dots, \smbf_1)$ and $\rmbf_2^{1:m} = (\smbf_2, \dots, \smbf_2)$, where $\smbf_1, \smbf_2 \in \Scal_\rho$, the oracle's response will essentially compare $\smbf_1$ and $\smbf_2$ on the linear metric $(1-\lambdabar)\inner{\ambfbar}{\smbf}$. Hence, we estimate the coefficients $\ambfbar$ by applying LPME over the $q$-dimensional sphere $\Scal_\rho$ with a modified oracle $\Omega^{\text{class}}$ which takes a pair of rate profiles $\smbf_1$ and $\smbf_2$ from $\Scal_\rho$ as input, and responds with:
\bequation
\Omega^{\text{class}}(\smbf_1, \smbf_2) \,=\, \Omega((\smbf_1, \dots, \smbf_1),\, (\smbf_2, \dots, \smbf_2)).
\eequation
This is decribed in line~1 of Algorithm~\ref{alg:f-me}, which applies the LPME subroutine with query space $\Scal_\rho$, binary search tolerance $\epsilon$, and the oracle $\Omega^{\text{class}}$.  
From Remark~\ref{fair-rm:ratio}, this subroutine returns a coefficient vector $\fmbf$ with $\Vert \fmbf \Vert_2=1$ such that:
\begin{equation}
    \frac{(1-\lambdabar)a_i}{(1-\lambdabar)a_j} = \frac{f_i}{f_j} \implies \frac{a_i}{a_j} = \frac{f_i}{f_j}.
\end{equation}
By setting  $\ambfhat = \fmbf$, we recover the classification coefficients independent of the fairness violation coefficients and trade-off parameter. 
See part 1 in Figure~\ref{fig:workflow} for further illustration. 

\subsection{Eliciting the Fairness Violation $\bvarphi(\tupr)$: Part 2 in Figure~\ref{fig:workflow} and lines 3-15 in Algorithm~\ref{alg:f-me}}
\label{ssec:elicitvarphi}
We now discuss eliciting the fairness term $\bvarphi(\tupr)$. We will first discuss the special case of  $m = 2$ groups and later discuss how the proposed procedure can be extended to handle multiple groups. 

\paragraph{Special Case of $m=2$: Lines 4-6 in Algorithm~\ref{alg:f-me}:} 

Recall from Definition \ref{def:linear} that in the violation term, we measure the group discrepancies using the \textit{absolute} difference between the group rates, i.e., $\dmbf^{12} = \vert \rmbf^{1} - \rmbf^{2} \vert$. If we restrict our queries to only those rate profiles $\rmbf^{1:2}$ for which the difference in each coordinate of $\rmbf^{1} - \rmbf^{2}$ is either always positive or always negative, then we can treat the violation term as a linear metric within this region and apply LPME to estimate the associated coefficients. 

To this end, we pose to the oracle queries of the form $\rmbf^{1:2} = (\smbf, \embf_i),$ 
where we assign to group 1 a rate profile $\smbf$ from the common sphere $\Scal_\rho$, and to 
group 2 the  rate profile $\embf_i \in \{0,1\}^q$ for some $i$. Remember that $\embf_i$ is a rate vector associated with a trivial classifier which predicts class $i$ on all inputs, and is therefore a binary vector.
Since we know whether an entry of $\embf_i$ is either a 0 or a 1,  we can decipher the signs of  each entry of the difference vector $\smbf - \embf_i$. 
Hence for group rate profiles of the above form, the metric $\Psi$ can be written as a linear function in $\smbf$:
\begin{align*}
&\bPsi(\smbf, \embf_i) =  
\inner{(1-\lambdabar)\ambfbar \odot (\bm{1}-\bm{\tau}^{2}) + \lambdabar \wmbf_i \odot
\bmbfbar^{12}}{\smbf} + c_i,
\numberthis \label{eq:metricbrichm2}
\end{align*}
where $\wmbf_i \coloneqq 1 - 2\embf_i$ tells us the sign of each entry of $\smbf - \embf_i$,  $c_i$ is a constant, and we have used the fact that $\taumbf^1 = \bm{1} - \taumbf^2$. Fixing a class $i$, we then apply LPME over the $q$-dimensional sphere $\Scal_\rho$ with a modified oracle $\Omega^{\text{viol}}_i$ which takes a pair of rate profiles $\smbf_1,\smbf_2 \in \Scal_\rho$ as input and responds with:
\begin{equation}
\Omega^{\text{viol}}_i(\smbf_1, \smbf_2) \,=\, \Omega((\smbf_1, \embf_i), (\smbf_2, \embf_i)).
\label{eq:parvarphim2}
\end{equation}
One run of LPME with oracle $\Omega^{\text{viol}}_1$ results in $q-1$ independent equations. In order to elicit a $q$-dimensional vector \scl{$\bmbf^{12}$}, we must run LPME again with oracle $\Omega^{\text{viol}}_2$.
This is described in lines 4 and 5 of Algorithm~\ref{alg:f-me}. 
The LPME calls provide us with two slopes $\breve \fmbf, \tilde \fmbf$ such that $\Vert \breve \fmbf \Vert_2= \Vert \tilde \fmbf \Vert_2=1$ from which it is easy to obtain the fairness violation weights:
\begin{align*}
    \bmbfhat^{12} = \frac{\tilde \bmbf^{12}}{\Vert \tilde \bmbf^{12} \Vert_2}, \quad \text{with} \quad
    \tilde \bmbf^{12} = \wmbf_1 \odot \left[ \delta\breve \fmbf - \ambfhat\odot(\bm{1} - \taumbf^{2})  \right],
    \numberthis \label{eq:bsolm2}
\end{align*}
where $\delta$ is a scalar depending on the known entities $\taumbf^{12}, \ambfhat, \breve \fmbf^{12}, \tilde \fmbf^{12}$. The derivation is provided in Appendix~\ref{append:sssec:elicitvarphim2} for completeness. Because $\bvarphi$ is scale invariant (see Definition~\ref{def:linear}), the normalized solution \scl{$\bmbfhat^{12}$} is independent of the true trade-off $\lambdabar$ and depends only on the previously elicited vector $\ambfhat$.

\paragraph{General Case of $m>2$: Lines 8-14 in Algorithm~\ref{alg:f-me}:}
\label{ssec:elicitvarphim}

We briefly outline the elicitation procedure for $m>2$ groups, with details in Appendix~\ref{append:sssec:elicitvarphi}. 
Let $\Mcal$ be a set of subsets of the $m$ groups such that each element $\sigma \in \Mcal$ and $[m] \setminus \sigma$  partition the set of $m$ groups. 
We will later discuss how to choose $\Mcal$ for efficient elicitation. Similar to the two-group case, we pose  queries $\rmbf^{1:m}$ where to a subset of groups $\sigma \in \Mcal$, we assign the trivial rate vector $\embf_i$ and to the rest $[m] \setminus \sigma$ groups, we assign a point $\smbf$ from the common sphere $\Scal_\rho$. Observe that within this query region, the metric $\Psi$ is linear in its inputs. So for a fixed partitioning of groups defined by $\sigma$, we apply LPME with a query space $\Scal_\rho$ using the modified $q$-dimensional oracle:
\begin{equation}
\Omega^{\text{viol}}_{\sigma,i}(\smbf_1, \smbf_2) = \Omega(\rmbf_1^{1:m}, \rmbf_2^{1:m})
~~\text{where}~~ \rmbf_1^g = 
\begin{cases}
    \embf_i & \text{if } g \in \sigma\\
    \smbf_1 & \text{o.w. }
\end{cases}
~~\text{and}~~ \rmbf_2^g = 
\begin{cases}
    \embf_i & \text{if } g \in \sigma\\
    \smbf_2 & \text{o.w. }
\end{cases}.
\label{eq:parvarphi}
\end{equation}
As described in lines 10 and 11 of the algorithm, we repeat this twice fixing class $i$ to 1 and $k$. The guarantees for LPME then give us the following relationship between coefficients $\bmbfbar^{uv}$ we wish to elicit and the already elicited coefficient $\hat{\ambf}$:
\begin{align*}
    \sum\nolimits_{u, v} \1\left[|\{u,v\}\cap \sigma|=1\right] \tilde \bmbf^{uv} = \wmbf_1 \odot \left[ \delta^{\sigma} \breve \fmbf^{\sigma} - \ambfhat\odot (\bm{1} - \taumbf^{\sigma}) \right],  \numberthis 
    \label{eq:midsolb}
\end{align*}
where \scl{$\taumbf^{\sigma} = \sum_{g\in \sigma}\bm{\tau}^{g}$} and $\tilde \bmbf^{uv} \coloneqq \lambdabar\bmbfbar^{uv}/(1 - \lambdabar)$ is a scaled version of the true (unknown) $\bmbfbar^{uv}$.
Since we need to estimate $\tiny{{m\choose 2}}$ coefficients, we repeat the above procedure for $\tiny{{m\choose 2}}$ partitions of the groups defined by $\sigma$  and get a system of $\tiny{{m\choose 2}}$ linear equations. We may choose any $\Mcal$ of size $\tiny{{m\choose 2}}$ so that the equations are independent. From the solution to these equations, we recover $\tilde \bmbf^{uv}$'s, which we further normalize to get estimates of the final fairness violation weights:
\begin{equation}
\bmbfhat^{uv} = \frac{\tilde \bmbf^{uv}}{\sum_{u,v=1, v > u}^m \Vert \tilde \bmbf^{uv} \Vert_2} \quad \text{for} \quad u,v \in [m], v>u.
 \label{eq:bsol}
\end{equation}
Because of normalization, the elicited fairness weights are independent of the trade-off $\lambdabar$.

\subsection{Eliciting Trade-off $\lambdabar$: Part 3 in Figure~\ref{fig:workflow} and Line 16 in Algorithm~\ref{alg:f-me}}
\label{ssec:elicitlambda}
Equipped with estimates of the misclassification and fairness violation coefficients $(\hat{\ambfbar}, \hat{\Bmbfbar})$, the 
 final step is to elicit the trade-off $\lambdabar$ between them. 
 We now show how this can be posed as one-dimensional binary search problem.
Suppose we restrict our queries to be of the form $\rmbf^{1:m} = (\smbf^+, \ombf, \ldots, \ombf),$ where for all but the first group, we assign the rate $\ombf$ associated with a uniform random classifier, and for the first group, we assign some rate $\smbf^+$ such that $\smbf^+ \geq \ombf$. For these rate profiles, the group rate difference terms $\rmbf^1 - \rmbf^v = \smbf^+ -\ombf \geq \mathbf{0}$ for all $v \in \{2, \ldots, m\}$, and all the other difference terms are $\mathbf{0}$. As a result, the metric $\Psi$ is linear in the input rate profiles:
\begin{align*}
        \bPsi(\smbf^+, \ombf, \ldots, \ombf) = \inner{(1-\lambdabar)\taumbf^1\odot\ambfbar + \lambdabar \sum\nolimits_{v=2}^m \bmbfbar^{1v}}{\smbf^+} + c,
         \numberthis \label{eq:metriclambda}
\end{align*}
where $c$ is a constant. 
Despite the metric being linear in the identified input region, we cannot directly apply the LPME procedure described in Section \ref{fair-ssec:mpme} to elicit $\lambda$, because we have one parameter to elicit but the input to the metric is $q$-dimensional. Here we propose a slight variant of LPME.

Similar to the original ME procedure for the binary classification setup in Chapter~\ref{chp:binary}, we first construct a one-dimensional function $\vartheta$, which takes a guess of the trade-off parameter as input, and outputs the quality of the guess. We show that this function is unimodal and its mode coincides with the oracle's true trade-off parameter $\lambda$.
\blemma
Let $\Scal_\varrho^+ \subset \Scal_\rho$ be a $q$-dimensional sphere with radius $\varrho < \rho$ such that $\smbf^+ \geq \ombf, \, \forall\, \smbf^+ \in \Scal^+_\varrho$ (see Figure~\ref{fig:geometry}). Assume the estimates $\hat{\ambf}$ and $\hat{\bmbf}^{uv}$'s satisfy a mild regularity condition $\inner{\hat{\ambf}}{\sum_{v=2}^m \hat{\bmbf}^{1v}}\neq 1$. Define a one-dimensional function $\vartheta$ as:
\begin{equation}
\vartheta(\bar{\lambda}) \coloneqq \Psi(\smbf_{\bar{\lambda}}^*, \ombf, \ldots, \ombf),
\label{eq:vartheta}
\end{equation}
where
\begin{equation}
\smbf^*_{\bar{\lambda}} \,=\, \argmax_{s^+ \in \Scal_\varrho^+}\,\inner{(1-\bar{\lambda})\taumbf^1\odot\hat{\ambfbar} + \bar{\lambda} \sum\nolimits_{v=2}^m \hat{\bmbfbar}^{1v}}{\smbf^+}.
\label{eq:vartheta-max}
\end{equation}
Then 
the function 
$\vartheta$ is strictly quasiconcave (and therefore unimodal) in $\bar{\lambda}$. Moreover, the mode of this function is achieved at the oracle's true trade-off parameter ${\lambda}$.
\label{fair-lm:lambda}
\elemma

For a candidate trade-off $\bar{\lambda}$, 
the function $\vartheta$ first constructs a candidate linear metric based on \eqref{eq:metriclambda}, maximizes this candidate metric over inputs $\smbf^+$, and evaluates the oracle's true metric $\Psi$ at the maximizing rate profile. Note that we cannot directly compute the function $\vartheta$ as it needs the oracle's metric $\Psi$. However, given two candidates for the trade-off parameter $\bar{\lambda}_1$ and $\bar{\lambda}_2$, one can compare the values of $\vartheta(\bar{\lambda}_1)$ and $\vartheta(\bar{\lambda}_2)$ by finding the corresponding maximizers over $\smbf^+$ and querying the oracle to compare them. Because  $\vartheta$ is unimodal, one can use a simple binary search using such pairwise comparisons to find the mode of the function, which we know coincides with the true $\lambda$.

\balgorithm[t]
\caption{Eliciting the trade-off $\lambdabar$}
\label{fair-alg:lambda}
\small
\balgorithmic[1]
\STATE \textbf{Input:} Query space $\Scal_\varrho^+$, binary-search tolerance $\epsilon > 0$, oracle $\Omega^{\text{trade-off}}$
\STATE \textbf{Initialize:} $\lambda^{(a)} = 0$, $\lambda^{(b)} = 1$.
\WHILE{$\abs{\lambda^{(b)} - \lambda^{(a)}} > \epsilon$} 
\STATE Set $\lambda^{(c)} = \frac{3 \lambda^{(a)} + \lambda^{(b)}}{4}$, $\lambda^{(d)} = \frac{\lambda^{(a)} + \lambda^{(b)}}{2}$, $\lambda^{(e)} = \frac{\lambda^{(a)} + 3 \lambda^{(b)}}{4}$
\STATE Set $\smbf^{(a)} = \displaystyle\argmax_{\smbf^+\in\Scal_\varrho^+} \inner{(1-\lambda_a)\taumbf^1\odot\ambfhat + \lambda_a \sum_{v=2}^m \bmbfhat^{1v}}{\smbf^+}$ using Lemma~\ref{mult-lem:spherebayes} (Chapter~\ref{chp:multiclass})
\STATE Similarly, set $\smbf^{(c)}$, $\smbf^{(d)}$, $\smbf^{(e)}$, $\smbf^{(b)}$.
\STATE Query  $\Omega^{\text{trade-off}}(\smbf^{(c)}, \smbf^{(a)})$,  $\Omega^{\text{trade-off}}(\smbf^{(d)}, \smbf^{(c)})$,  $\Omega^{\text{trade-off}}(\smbf^{(e)}, \smbf^{(d)})$, and  $\Omega^{\text{trade-off}}(\smbf^{(b)}, \smbf^{(e)})$.\\
\STATE $[\lambda^{(a)}, \lambda^{(b)}] \leftarrow$ \emph{ShrinkInterval} (responses) using a subroutine analogous to the routine shown in  Figure~\ref{mult-append:fig:shrink1}.
\ENDWHILE
\STATE \textbf{Output:} $\hat\lambda = \frac{\lambda^{(a)}+\lambda^{(b)}}{2}$. 
\ealgorithmic
\ealgorithm

We provide an outline of this procedure in Algorithm~\ref{fair-alg:lambda}, which uses the modified oracle 
\bequation
\Omega^{\text{trade-off}}(\smbf_1^+, \smbf_2^+) = \Omega((\smbf^+_1, \ombf, \ldots, \ombf),\, (\smbf^+_2, \ombf, \ldots, \ombf))
\eequation
to compare the maximizers in \eqref{eq:vartheta-max}. 

\paragraph{Description of Algorithm~\ref{fair-alg:lambda}:} Given the unimodality of $\vartheta(\lambda)$ from Lemma~\ref{fair-lm:lambda}, we devise the binary-search procedure Algorithm~\ref{fair-alg:lambda} for eliciting the true trade-off $\lambdabar$. The algorithm takes in input the query space $\Scal_\varrho^+$, binary-search tolerance $\epsilon$, an equivalent oracle  $\Omega^{\text{trade-off}}$, the elicited $\ambfhat$ from Section~\ref{ssec:elicitphi}, and the elicited $\Bmbfhat$ from Section~\ref{ssec:elicitvarphi}. The algorithm finds the maximizer of the function $\hat{\vartheta}(\lambda)$ defined analogously to~\eqref{eq:vartheta}, where $\ambfbar, \Bmbfbar$ are replaced by $\ambfhat, \Bmbfhat$, using Lemma~\ref{mult-lem:spherebayes} (Chapter~\ref{chp:multiclass}). The algorithm poses four queries to the oracle and shrink the interval $[\lambda^{(a)}, \lambda^{(b)}]$ into half based on the responses using a subroutine analogous to \emph{ShrinkInterval} shown in Figure~\ref{mult-append:fig:shrink1}. The algorithm stops when the length of the search interval $[\lambda^{(a)}, \lambda^{(b)}]$ is less than the tolerance $\epsilon$. 
Combining parts 1, 2 and 3 in Figure~\ref{fig:workflow} completes the FPME procedure. 

%% file: fair/guarantees.tex
\section{Guarantees}
\label{fair-sec:guarantees}
We discuss elicitation guarantees under the following feedback model. 

\bdefinition[Oracle Feedback Noise: $\epsilon_\Omega \geq 0$] For two rates $\tupr_1, \tupr_2 \in \prodRcal$, the oracle responds correctly as long as $|\bPsi(\tupr_1) - \bPsi(\tupr_2)| > \epsilon_\Omega$. Otherwise, it may be incorrect.
\label{fair-def:noise}
\edefinition
In words, the oracle may respond incorrectly if the rates are very close as measured by the metric $\bPsi$. Since deriving the final metric involves offline computations including certain ratios, we discuss guarantees under a regularity assumption that ensures all components are well defined. 
\bassumption
We assume that $1 > c_1 > \lambdabar > c_2 > 0$, $\min_{i}\vert a_i\vert > c_3$, $\min_{i}\vert (1-\lambdabar)a_i {\tau}^{\sigma}_i - \lambdabar w_{ji}
b^{\sigma}_i \vert > c_4 \, \forall\, j\in [q], \sigma \in \Mcal$,  for some $c_1, c_2, c_3, c_4 > 0$, $\rho > \varrho \gg \epsilon_\Omega$, and $\inner{\ambfbar}{\sum_{v=2}^m \bmbfbar^{1v}}\neq 1$. 
\label{as:regularity}
\eassumption
\btheorem 
Given $\epsilon,\epsilon_\Omega\geq 0$, and a 1-Lipschitz fair performance metric $\;\bPsi$ parametrized by $\ambfbar, \Bmbfbar, \lambdabar$, under Assumptions~\ref{as:clsconditional} and~\ref{as:regularity}, Algorithm~\ref{alg:f-me} returns a metric $\hPsi$ with parameters:
\begin{itemize}[itemsep=0pt, leftmargin=1em]
    \item $\ambfhat:$ after $O\left(q\log \tfrac 1 {\epsilon}\right)$ queries such that $\Vert \ambfbar-\ambfhat \Vert_{2}\leq O\left(\sqrt{q}(\epsilon+\sqrt{\epsilon_\Omega/\rho})\right)$.
    \item $\Bmbfhat:$ after $O\left({m\choose 2} q\log \tfrac 1 {\epsilon}\right)$ queries such that $\Vert \text{vec}(\Bmbfbar)-\text{vec}(\Bmbfhat) \Vert_{2}\leq O\left(mq(\epsilon+\sqrt{\epsilon_\Omega/\rho})\right)$, where $\text{vec}(\cdot)$ vectorizes the matrix.
    \item $\lambdahat:$ after $O(\log(\tfrac{1}{\epsilon}))$ queries, with error $\vert \lambdabar - \lambdahat \vert\leq O\left(\epsilon  + \sqrt{\epsilon_\Omega/\varrho} +  \sqrt{m q (\epsilon + \sqrt{\epsilon_\Omega/\rho})/\varrho}\right)$.
\end{itemize}
 \label{thm:error}
\etheorem 
We see that the proposed FPME procedure is robust to noise, and its query complexity depends linearly in the number of unknown entities. For instance, line 2 in Algorithm~\ref{alg:f-me} elicits $\ambfhat \in \Rmbf^q$ by posing $\tilde O(q)$ queries, the `for' loop in line~9 of Algorithm~\ref{alg:f-me} runs for $\tiny{m\choose 2}$ iterations, where each iteration
requires $\tilde O(2q)$ queries, and finally line~16 in Algorithm~\ref{alg:f-me} is a simple binary search requiring $\tilde O(1)$ queries. The work in Chapter~\ref{chp:multiclass} work suggests that linear multiclass elicitation (LPME) elicits misclassification costs ($\phi$) with linear query complexity. Surprisingly, our proposed FPME procedure elicits a more complex (nonlinear) metric without increasing the query complexity order. 
Furthermore, since sample estimates of rates are consistent estimators, and the metrics discussed are $1$-Lipschitz wrt.\ rates, with high probability, we gather correct oracle feedback from querying with finite sample estimates $\Omega(\hat{\rmbf}^{1:m}_1, \hat{\rmbf}^{1:m}_2)$ instead of querying with population statistics $\Omega({\rmbf}^{1:m}_1, {\rmbf}^{1:m}_2)$, as long as we have sufficient samples. Apart from this, Algorithm~1 is agnostic to finite sample errors as long as the sphere $\Scal_\rho$ is contained within the feasible region $\Rcal^1 \cap \dots \cap \Rcal^m$.

%% file: fair/experiments.tex
\section{Experiments}
\label{sec:experiments}

\begin{figure*}[t]
	\centering 
	\subfigure[]{
		{\includegraphics[width=4.5cm]{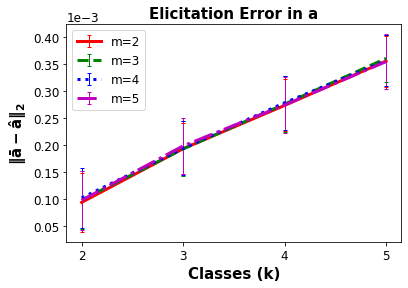}}
		\label{fair-fig:rec_a}
	}\quad\quad
	\subfigure[]{
		{\includegraphics[width=4.5cm]{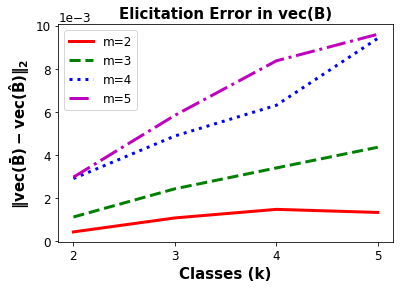}}
		\label{fair-fig:rec_B}
	}\quad\quad
	\subfigure[]{
		{\includegraphics[width=4.5cm]{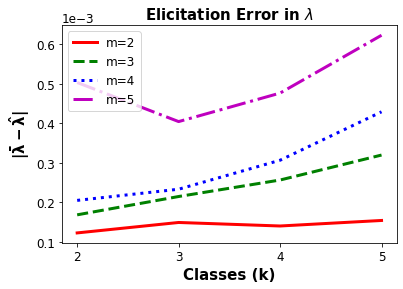}}
		\label{fair-fig:rec_l}
	}
	\caption{Elicitation error in recovering the oracle's metric.}
	\label{fair-fig:recovery}
\end{figure*}

\subsection{Theory Validation}
\label{ssec:exp:theory}

We first empirically validate the FPME procedure and recovery guarantees of Section~\ref{fair-sec:guarantees}. Recall that there exists a sphere $\Scal_\rho \subset \Rcal^1 \cap \dots \cap \Rcal^m$ as long as there is a non-trivial classification signal within each group (Remark~\ref{remark:sphere}). 
Thus for experiments, we assume access to a feasible sphere $\Scal_\rho$ with $\rho = 0.2$. 
We randomly generate 100 oracle metrics each for $k, m \in \{2,3,4,5\}$ parametrized by $\{\ambfbar, \Bmbfbar, \lambdabar\}$. This specifies the query outputs by the oracle for each metric in Algorithm~\ref{alg:f-me}. We then use Algorithm~\ref{alg:f-me} with tolerance $\epsilon = 10^{-3}$ to elicit corresponding metrics parametrized by $\{\ambfhat, \Bmbfhat, \lambdahat\}$. 
Algorithm~\ref{alg:f-me} makes $1 + 2M$ subroutine calls to LPME procedure and $1$ call to Algorithm~\ref{fair-alg:lambda}. 
LPME subroutine requires exactly $16(q-1)\log(\pi/2\epsilon)$ queries, where we use 4 queries to shrink the interval in the binary search loop and fix 4 cycles for the coordinate-wise search. Also, Algorithm~\ref{fair-alg:lambda} requires $4\log(1/\epsilon)$ queries.
In Figure~\ref{fair-fig:recovery}, we report the mean of the $\ell_2$-norm between the oracle's metric and the elicited metric. 
Clearly, we elicit metrics that are close to the true metrics.
Moreover, this holds true across a range of $m$ and $k$ values demonstrating the robustness of the proposed approach. Figure~\ref{fair-fig:rec_a} shows that the error $\Vert \ambfbar - \ambfhat\Vert_2$ increases only with the number of classes $k$ and not groups $m$. This is expected since $\ambfhat$ is elicited by querying rates that zero out the fairness violation (Section~\ref{ssec:elicitphi}). Figure~\ref{fair-fig:rec_B} verifies Theorem~\ref{thm:error} by showing that $\Vert \text{vec}(\Bmbfbar) - \text{vec}(\Bmbfhat) \Vert_2$ increases with both number of  classes $k$ and groups $m$. In accord with Theorem~\ref{thm:error}, Figure~\ref{fair-fig:rec_l} shows that the elicited trade-off $\hat \lambda$ is also close to the true $\lambdabar$. However, the elicitation error increases consistently with groups $m$ but not with classes $k$. A possible reason may be the cancellation of errors from eliciting $\ambfhat$ and $\Bmbfhat$ separately.

\subsection{Ranking of Classifiers}
\label{ssec:exp:ranking}

Next, we highlight the utility of FPME in ranking real-world classifiers. One of the most important applications of performance metrics is evaluating classifiers, i.e., providing a quantitative score for their quality which then allows us to choose the best (or best set of) classifier(s). In this section, we discuss how the ranking of plausible classifiers is affected when a practitioner employs default metrics to rank (fair) classifiers instead of the oracle's metric or our elicited approximation. 

\begin{table}[t]
\centering
\caption{Dataset statistics; the real-valued regressor in \emph{wine} and \emph{crime} datasets is recast to 3 classes based on quantiles.}
\begin{tabular}{|c|ccccc|}
\hline
\textbf{Dataset} & $k$ & $m$ & \textbf{\#samples} & \textbf{\#features} & \textbf{group.feat} \\ 
\hline
default        & 2  & 2  & 30000           &    33    &     gender       \\
adult        &  2 &  3 &    43156       &    74    &    race        \\
wine        &  3 &  2 & 6497          &     13   &  color          \\
crime        &  3 & 3  &    1907       &     99   &      race   \\ 
\hline
\end{tabular}
\label{fair-tab:stats}
\end{table}

We take four real-world classification datasets with $k, m \in \{2,3\}$ (see Table~\ref{fair-tab:stats}). 60\% of each dataset is used for training and the rest for testing. We create a pool of 100 classifiers for each dataset by tweaking hyperparameters under logistic regression models~\cite{kleinbaum2002logistic}, multi-layer perceptron models~\cite{pal1992multilayer}, support vector machines~\cite{joachims1999svmlight}, LightGBM models~\cite{ke2017lightgbm}, and fairness constrained optimization based models~\cite{narasimhan2019optimizing}. We compute the group wise confusion rates on the test data for each model for each dataset. We will compare the ranking of these classifiers achieved by competing baseline metrics with respect to the ground truth ranking. 

\begin{table}[t]
    \centering
    \caption{Common (baseline) metrics usually deployed to rank classifiers.}
    \begin{tabular}{|c|cccccccc|}
    \hline 
    \textbf{Name} $\rightarrow$  & {${\hphi\hvarphi\lambdahat}$\_a} & {$\hphi\hvarphi\lambdahat$\_w} & {$\hphi\hvarphi$\_a} & {$\hphi\hvarphi$\_w} & {$\hphi$\_a} & {$\hphi$\_w} & 
    o\_p & 
    o\_f \\ \hline
    $\ambfhat$  & acc. & w-acc. & acc. & w-acc.  & acc. & w-acc. & $\ambfbar$ & - \\ 
    $\Bmbfhat$  & acc. & w-acc. & acc. & w-acc.  & elicit & elicit & - & $\Bmbfbar$ \\ 
    $\lambdahat$ & $0.5$  & w-acc. & elicit & elicit & elicit  & elicit & 0 & 1 \\ 
    \hline
\end{tabular}
    \label{fair-tab:baselines}
\end{table}

We generate 100 random oracle metrics $\bPsi$. $\bPsi$'s gives us the ground truth ranking of the above classifiers. We then use our proposed procedure FPME (Algorithm~\ref{alg:f-me}) to recover the oracle's metric. For comparison in ranking of real-world classifiers, we choose a few metrics that are routinely employed by practitioners as baselines (see Table~\ref{fair-tab:baselines}). 
The prefixes (i.e., $\hphi, \hvarphi$, or $\lambdahat$) in name of the baseline metrics denote the components that are set to default metrics, and the suffixes (i.e. `a' or `wa') denote whether the assignment is done with \emph{accuracy} (i.e., equal weights) or with \emph{weighted accuracy} (weights are assigned randomly however maintaining the true order of weights as in $\bPsi$). 
For example, $\hphi\hvarphi\lambdahat$\_a  corresponds to the metric where {$\hphi, \hvarphi, \lambdahat$} are set to standard classification accuracy. Similarly, {$\hphi$\_w} denote a metric where the misclassification cost {$\hphi$} is set to weighted accuracy but both $\hvarphi$ and $\lambdahat$ are elicited using Part 2 and Part 3 of the FPME procedure (Algorithm~\ref{alg:f-me}), respectively. 
Assigning weighted accuracy versions is a commonplace since sometimes the order of the costs associated with the types of mistakes in misclassification cost $\bphi$ or fairness violation $\bvarphi$ or preference for fairness violation over misclassification $\lambdabar$ is known but not the actual cost. 
Another example is {$\hphi\hvarphi$\_a} which corresponds to the metric where {$\hphi, \hvarphi$} are set to accuracy and only the trade-off {$\lambdahat$} is elicited using Part 3 of the FPME procedure (Algorithm~\ref{alg:f-me}). 
This is similar to prior work by Zhang et al.~\cite{zhang2020joint} who assumed the classification error and fairness violation known, so only the trade-off has to be elicited -- however they also assume direct ratio queries, which can be challenging in practice. Our approach applies much simnpler pairwise preference queries. 
Lastly, o\_p and o\_f represent \emph{only predictive performance} with $\lambda=0$ and \emph{only fairness} with $\lambda=1$, respectively. 

\begin{figure*}[t]
	\centering 
	\subfigure{
		{\includegraphics[width=5cm]{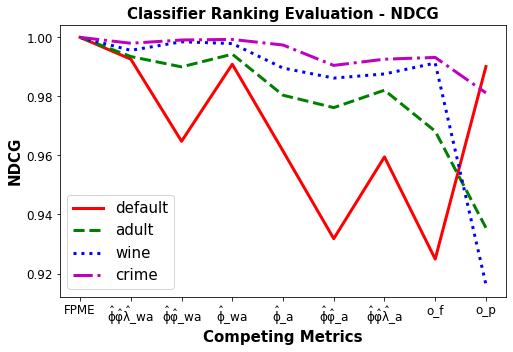}}
		\label{fair-fig:ndcg}
	}\quad\quad
	\subfigure{
		{\includegraphics[width=5cm]{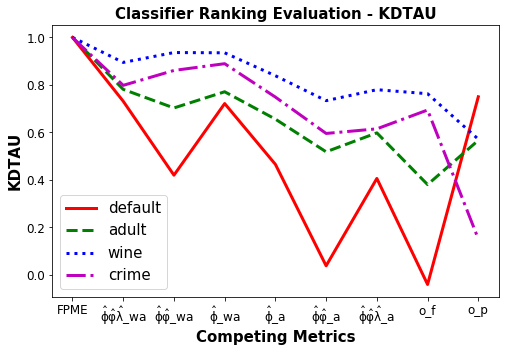}}
		\label{fair-fig:kdtau}
	}
	\caption{Ranking performance of real-world classifiers by competing metrics.}
	\label{fair-fig:ranking}
\end{figure*}

Figure~\ref{fair-fig:ranking} shows average NDCG (with exponential gain)~\cite{valizadegan2009learning} and Kendall-tau coefficient~\cite{shieh1998weighted} over 100 metrics $\bPsi$ and their respective estimates by the competing baseline metrics. We see that FPME, wherein we elicit $\hphi, \hvarphi$, and $\lambdahat$ in sequence, achieves the highest possible NDCG and Kendall-tau coefficient. Even though we make some elicitation error in recovery (Section~\ref{fair-sec:guarantees}), we achieve almost perfect results while ranking the classifiers. 

To connect to practice, this implies that when given a set of classifiers, ranking based on elicited metrics will align most closely to ranking based on the true metric, as compared to ranking classifiers based on default metrics. This is a crucial advantage of metric elicitation for practical purposes. In this experiment, baseline metrics achieve inferior ranking of classifiers in comparison to the rankings achieved by metrics that are elicited using the proposed FPME procedure. Figure~\ref{fair-fig:ranking} also suggests that it is beneficial to elicit all three components $(\ambfbar, \Bmbfbar, \lambdabar)$ of the metric in Definition~\ref{def:linear}, rather than pre-define a component and elicit the rest. For the \emph{crime} dataset, some methods also achieve high NDCG values, so ranking at the top is good; however Kendall-tau coefficient is weak which suggests that overall ranking is poor. With the exception  of the \emph{default} dataset, the weighted versions are better than equally weighted versions in ranking. This is expected because in weighted versions, at least order of the preference for the type of costs matches with the oracle's  preferences.

%% file: fair/relatedwork.tex
\section{Related Work}
\label{sec:relatedwork}
Some early attempts to eliciting 
individual fairness metrics~\cite{ilvento2019metric, mukherjee2020two} are distinct from ours – as we are focused on the more prevalent setting of group fairness, yet for which there
are no existing approaches to our knowledge. 
Zhang et al.~\cite{zhang2020joint} propose an approach that elicits only the trade-off between accuracy and fairness using complicated ratio queries. We, on the other hand, elicit classification cost, fairness violation, and the trade-off together as a non-linear function, all using much simpler pairwise comparison queries. Prior work for constrained classification focus on learning classifiers under constraints for fairness~\cite{goh2016satisfying, hardt2016equality, zafar2017constraints,narasimhan2018learning}. We take the regularization view of algorithmic fairness, where a fairness violation is embedded in the metric definition instead of as constraints~\cite{kamishima2012fairness, bechavod2017learning, corbett2017algorithmic, agarwal2018reductions, menon2018cost}. 
From the elicitation  perspective, the closest line of work to ours is in Chapters~\ref{chp:binary} and \ref{chp:multiclass}, where we proposed the problem of ME but solved it only for a simpler setting of classification without fairness. As we move to multiclass, multigroup fair performance ME, we find that the complexity of both the form of the metrics and the query space increases. This results in starkly different elicitation strategy with novel methods required to provide query complexity guarantees. Learning (linear) functions passively using pairwise comparisons is a mature field~\cite{joachims2002optimizing, herbrich2000large, peyrard2017learning}, but these approaches fail to control sample (i.e.\ query) complexity. 
Active learning in fairness~\cite{noriega2019active} is a related direction; however the aim there is to learn a fair classifier based on fixed metric instead of eliciting the metric itself.  

%% file: fair/discussion.tex
\section{Concluding Remarks and Future Work}
\label{sec:extensions}
\bitemize
\item \textbf{Transportability:} Our elicitation procedure is independent of the population $\Pmbb$ as long as there exists a sphere of rates which is feasible for all groups. Thus, any metric that is learned using one dataset or model class (i.e., by estimated $\hat \Pmbb$) can be applied to other applications and datasets, 
as long as the expert believes the context and tradeoffs are the same.
\item \textbf{Extensions.} 
Our propsal can be modified to leverage the structure in the metric or the groups to further reduce the query complexity. For example, when the fairness violation weights are the same for all pairs of groups,  the procedure in Section~\ref{ssec:elicitvarphim} requires only one partitioning of groups to elicit the metric $\hvarphi$. Such modifications are easy to incorporate. 
In the future, we  plan to extend our approach to more complex metrics such as linear-fractional functions of rates and discrepancies.
\item \textbf{Limitations of group-fair metrics.} Since the metrics we consider depend on a classifier only through its rates, comparing two classifiers on these metrics is equivalent to comparing
their rates. 
Unfortunately, with this setup, all the limitations associated with group-fairness definition of metrics apply to our setup as well. For example, we may discard notions of \emph{individual fairness} when only group-rates are considered for comparing classifiers~\cite{binns2020apparent}. Similarly, issues associated with \emph{overlapping groups}~\cite{kearns2018preventing}, \emph{detailed group specification}~\cite{kearns2018preventing}, \emph{unknown or changing groups}~\cite{hashimoto2018fairness, gillen2018online}, \emph{noisy or biased} group information~\cite{wang2020robust}, among others, pose limitations to our proposed setup. We hope 
that as the first work on the topic, our work will inspire the research community to address many of these open problems for the task of metric elicitation. 
\item \textbf{Optimal bounds.} We conjecture that our query complexity bounds are tight; however, we leave this detail for the future. 
In conclusion, we elicit a more complex (non-linear) group fair-metric with the same query complexity order as standard classification linear elicitation procedures (Chapter~\ref{chp:multiclass}).

\item \textbf{Limitation.} Our work seeks to truly democratize and personalize fair machine learning. Besides, the significance of fair performance metric elicitation lies in how it empowers the practitioner to tune the design of machine learning models to the needs of the target fairness task. However, at the same time, this work may have drawbacks 
because it leaves open the key question of who should be the stakeholders to be queried. This work also assumes a parametric form for the oracle metric, which may not be an exact match to practice. Furthermore, we should be cautious of the result of the failure of the system which could cause disparate impact among sensitive groups when the elicited metric is incorrect, e.g., when applied to settings where the stated assumptions are not met.  
\eitemize

%% file: quadratic/header.tex
\chapter{Quadratic Metric Elicitation for Fairness and Beyond}
\label{chp:quadratic}

\input{quadratic/introduction}
\input{quadratic/background}
\input{quadratic/deg2elicitation}
\input{quadratic/fairelicitationsetup}
\input{quadratic/guarantees}
\input{quadratic/experiments}
\input{quadratic/extensions}
\input{quadratic/relatedwork}
\input{quadratic/discussion}

%% file: quadratic/introduction.tex
The Metric Elicitation (ME) strategies for the binary and multiclass classification setups that are discussed in Chapters~\ref{chp:binary} and~\ref{chp:multiclass}, respectively, only handle linear or quasi-linear function of predictive rates, which can be restrictive for many applications where the metrics are complex and non-linear. For example, in \emph{fair machine learning}, classifiers are often judged by measuring discrepancies between predictive rates for different protected groups~\cite{hardt2016equality}. Similar discrepancy-based measures are also used in \emph{distribution matching} applications~\cite{narasimhan2018learning, Fab1}. A common measure of discrepancy in such
applications is the squared difference, which is appealing for its smoothness properties and a quadratic metric that cannot be handled by existing approaches. 
Similar quadratic metrics also find use in class-imbalanced
learning~\cite{goh2016satisfying, narasimhan2018learning} (see Section~\ref{ssec:metric} for examples).
Motivated by these examples, in this paper, we propose strategies for eliciting metrics defined by \emph{quadratic} functions of rates, that encompass linear metrics as special cases. We further extend our approach to elicit polynomial metrics, a universal family of functions~\cite{stone1948generalized}. 
This allows one to better capture real-world human preferences.  

Our high-level idea is to
 approximate the quadratic metric using multiple linear functions, employ linear ME to estimate the 
 local slopes, 
 and combine the slope estimates to reconstruct the original metric. 
 While natural and elegant, this approach comes with non-trivial challenges. 
 Firstly, we must choose 
 center 
 points for the local-linear approximations,
 and the chosen points must represent
 feasible queries. Secondly, because of 
 pairwise queries, we only receive \emph{slopes (directions)} and not magnitudes for the local-linear functions, requiring intricate analysis to reconstruct the original 
 metric and to deal with multiplicative errors that result.
 Despite the challenges,
 our method requires a query complexity that is only \emph{linear} in the number of unknown entities, which we show is {\em near-optimal}.

Our interest in quadratic metric elicitation is majorly motivated by applications to \emph{fair machine learning}~\cite{dwork2012fairness, hardt2016equality, kleinberg2017inherent}. While several group-based fairness metrics have been proposed to capture bias in automated decision-making, selecting the right metric 
remains a crucial challenge~\cite{zhang2020joint}. In Chapter~\ref{chp:fair}, we proposed an approach for eliciting group-fair metrics that 
measure discrepancies using the absolute differences 
in rates across multiple sensitive groups. 
Unfortunately, that approach specifically handles metrics that are linear in the group discrepancies and does not  generalize easily to other families of metrics. 
We extend this setup 
to allow for more general fairness metrics defined by quadratic functions of group discrepancies and 
show how our proposed quadratic ME approach can be easily adapted to elicit such metrics. 
Like we did in Chapter~\ref{chp:fair}, here 
we jointly elicit three terms: (i) predictive performance 
defined by a weighted error metric,
(ii) a quadratic fairness violation metric, 
and (iii) a trade-off between the predictive performance and fairness violation. 

\textbf{Contributions and chapter organization.} 
We propose a novel quadratic metric elicitation algorithm for classification problems, which requires only pairwise preference feedback 
either over classifiers or rates (Section \ref{sec:quadme}). 
Specific to group-based fairness tasks, we show how to jointly elicit the predictive and fairness metrics, and the trade-off between them (Section \ref{sec:fairme}).
The proposed approach is robust under feedback and 
finite sample noise and requires a  near-optimal number of queries for elicitation (Section \ref{sec:guarantees}). We empirically validate the proposal 
for multiple classes and groups on simulated oracles (Section \ref{quad-sec:experiments}).
Lastly, we discuss how our strategy can be generalized to elicit higher-order polynomials by recursively applying the procedure to elicit lower-order approximations
(Section \ref{sec:poly}). 
All the proofs in this chapter are provided in Appendix~\ref{apx:quadratic}.

\textbf{Notation.} 
$\|\cdot\|_F$ represents the Frobenius norm, and $\alphambf_i \in \Rmbb^q$ denotes the $i$-th standard basis vector, where the $i$-th coordinate is 1 and  others are 0. 

%% file: quadratic/background.tex
\section{Background}
\label{sec:background}

We consider a $k$-class classification setting with $X \in \Xcal$ and $Y \in [k]$ denoting the input and output random variables, respectively. We assume access to an $n$-sized sample $\{(\xmbf, y)_i\}_{i=1}^n$ generated \emph{iid} from a distribution $ \Pmbb(X, Y)$. 
We work with randomized classifiers 
\bequation
h : \Xcal \rightarrow \Delta_k
\eequation
that for any $\xmbf$ gives a distribution $h(\xmbf)$ over the $k$ classes and use 
 \bequation
 \Hcal = \{h : \Xcal \rightarrow \Delta_k\}
 \eequation
 to denote the set of all classifiers. Unlike Chapter~\ref{chp:multiclass}, our choice of measurement space is the space of predictive rates (described next). This is just to suit the application of fairness, where predictive rates for two sensitive groups can be compared; however, it is not suitable for group-fair application purposes to compare confusion matrix entries for two sensitive groups. Nevertheless, the proposed algorithm for quadratic (or, polynomial) metric elicitation will also work if the choice of measurement space is the space of confusion matrices. 

\emph{Predictive rates:} 
We define the predictive rate matrix for a classifier $h$ by $\Rmbf(h, \Pmbb) \in \Rmbb^{k \times k}$, where the $ij$-th entry is the fraction of label-$i$ examples for which the randomized classifier $h$ predicts $j$:
\begin{align}
	R_{ij}(h, \Pmbb) \coloneqq P(h(X) = j | Y = i)  \quad \text{for} \; i, j \in [k],
	\label{eq:components}
\end{align}
where the probability is over draw of $(X, Y) \sim \P$ and the randomness in $h$. 
Notice that each diagonal entry
of $\Rmbf$
can be written in terms of its off-diagonal elements: 
\begin{equation}
    R_{ii}(h, \Pmbb) = 1 - \sum\nolimits_{j=1,j\neq i}^k R_{ij}(h, \Pmbb).
    \label{quad-eq:decomp}
\end{equation}
Thus, we can represent a rate matrix with its $q \coloneqq (k^2 - k)$ off-diagonal elements, write it as a vector $\rmbf(h, \Pmbb) = \offdiag(\Rmbf(h, \Pmbb))$, and interchangeably refer to it as the \emph{`vector of rates'}.

\emph{Metrics:} We consider metrics  that are defined by a general function $\phi : [0, 1]^{q}  \rightarrow \Rmbb$ of rates: 
  \bequation
  \phi(\rmbf(h, \Pmbb)).
  \eequation
This includes the (weighted) error rate
$\phi^{\text{err}}(\rmbf(h, \Pmbb))$ $\,=\, \sum_{i} a_i r_i(h, \Pmbb)$, for weights $a_i \in \mathbb{R}_{+}$, the F-measure, and many more metrics~\cite{sokolova2009systematic}. 
Without loss of generality (w.l.o.g.), we treat metrics as costs. 
Since the metric's scale does not affect the learning problem~\cite{narasimhan2015consistent}, we allow $\phi : [0, 1]^{q}  \rightarrow [-1,1]$.

\emph{Feasible rates:} We will restrict our attention to only those rates that are feasible, i.e., can be achieved by some classifier. The set of all feasible rates is given by: 
\bequation
\Rcal = \{\rmbf(h, \Pmbb) \,:\, h \in \Hcal \}.
\eequation
For simplicity, we will suppress the dependence on $\Pmbb$ and $h$ if it is clear from the context.

\subsection{Metric Elicitation: Problem Setup}
\label{ssec:me}

We now describe the problem of \emph{Metric Elicitation}, which follows from Chapter~\ref{chp:me}. There's an \textit{unknown} metric $\phi$, and we seek to elicit its form by posing queries to an \emph{oracle} asking  which of two classifiers is more preferred by it. The oracle has access to the  metric $\phi$ and provides answers by comparing its value on the two classifiers.
\bdefinition
[Oracle Query] Given two classifiers $h_1, h_2$ (equiv. to rates $\rmbf_1, \rmbf_2$ respectively), a query to the Oracle (with metric $\phi$) is represented by:
\begin{align}
\Gamma(h_1, h_2\,;\, \phi) = \Omega(\rmbf_1, \rmbf_2\,;\,\phi) &= \1[\phi(\rmbf_1) > \phi(\rmbf_2)], 
\end{align}
where $\Gamma: \Hcal \times \Hcal \rightarrow \{0,1\}$ and $\Omega: \Rcal \times \Rcal \rightarrow \{0, 1\}$. The query asks whether $h_1$ is preferred to $h_2$ (equiv. if $\rmbf_1$ is preferred to $\rmbf_2$), as measured by $\phi$. 
\label{def:query}
\edefinition

In practice, the oracle can be an expert, a group of experts, or an entire user population. The ME framework can be applied by posing classifier comparisons directly via interpretable learning techniques~\cite{ribeiro2016should, doshi2017towards} or via A/B testing~\cite{tamburrelli2014towards}. For example, in an internet-based application 
one may perform the A/B test by deploying two classifiers A and B with two different sub-populations of users and use their level of engagement to decide the preference over the two classifiers. 
For other applications, one may present 
visualizations of rates of the two classifiers (e.g.,  \cite{zhang2020joint,beauxis2014visualization}), and have the user provide the  preference. 
Moreover, since the metrics we consider are functions of only the predictive rates, queries comparing classifiers are the same as queries on the associated rates. So for convenience, we will have our algorithms pose queries comparing two (feasible) rates.
Indeed  given a feasible rate, one can efficiently find the associated classifier (see Appendix \ref{append:ssec:sphere} for details). 
We next formally state the ME problem.
\bdefinition [Metric Elicitation with Pairwise Queries (given $\{(\xmbf,y)_i\}_{i=1}^n$)] Suppose that the oracle's (unknown) performance metric is $\phi$.  Using oracle queries of the form $\Omega(\rmbfhat_1, \rmbfhat_2\,;\,\phi)$, where $\rmbfhat_1, \rmbfhat_2$ are the estimated rates from samples, recover a metric $\hphi$ such that $\Vert\phi - \hphi\Vert < \kappa$ under a  suitable norm $\Vert \cdot \Vert$ for sufficiently small error tolerance $\kappa > 0$.
\label{def:me}
\edefinition

As discussed in previous chapters, the performance of ME is evaluated both by the query complexity and the quality of the elicited metric. As is standard in the decision theory literature~\cite{koyejo2015consistent, hiranandani2018eliciting, hiranandani2019multiclass, hiranandani2020fair}, we present our ME approach by first assuming access to population quantities such as the population rates $\rmbf(h, \Pmbb)$, then examine estimation error from finite samples, i.e., with empirical rates $\rmbfhat(h, \{(\xmbf,y)_i\}_{i=1}^n)$. 

\subsection{Linear Metric Elicitation}
\label{ssec:mpme}

\begin{figure}[t]
    \centering
	\begin{tikzpicture}[scale = 1.25]
	\hspace{-0.25cm}
    \begin{scope}[shift={(0,0)}, scale = 0.6]\scriptsize
    
    \def\r{0.12};
    
    \coordinate (a) at (-0.4,1);
    \coordinate (b) at (0.6, 3.25);
    \coordinate (c) at (7, 4);
    \coordinate (d) at (6.5, 2);
    \coordinate (e) at (2.5, -0.75);
    \coordinate (f) at (-0.1, -0.5);
    
    \coordinate (labelleft) at (3, -1.5);
    
    \coordinate (Cent) at (3,1.75);
    \coordinate (Centcent) at (2.85,1.90);
    \coordinate (Cent1) at (4.45,1.75);
    \coordinate (Cent2) at (3,3.2);
    \coordinate (CentL) at (1.55,1.75);
    
    \coordinate (Space1) at (0.2,0.2);
    \coordinate (SpaceR) at (0.25,3.25);
    \coordinate (Spacem) at (-0.1,2.5);
    
    \coordinate (Sphere) at (5,3.25);
    \coordinate (Sphere0) at (3,1);
    \coordinate (Sphere1) at (4.45,1);
    \coordinate (Sphere2) at (3,2.45);
    \coordinate (SphereL) at (1.65,1);
    \coordinate (Sphereminus) at (3.4,0.95);
    
    \coordinate (r) at (3.75,2);
    
    \coordinate (u11) at (-0.25, -0.25);
    \coordinate (uextra121) at (1, 3.8);
    \coordinate (u21) at (4, 4.5);
    \coordinate (uextra2k1) at (6, 1);
    \coordinate (uk1) at (3.5, -0.5);
    
    \coordinate (u12) at (0.3, -0.70);
    \coordinate (uextra122) at (0.65, 2.65);
    \coordinate (u22) at (3.4, 4.2);
    \coordinate (uextra2k2) at (5.15, 1.6);
    \coordinate (uk2) at (3.75, 0.1);
    
    \coordinate (u13) at (-0.30, 0.10);
    \coordinate (uextra123) at (1.25, 3.25);
    \coordinate (u23) at (4, 4.5);
    \coordinate (uextra2k3) at (5.75, 1);
    \coordinate (uk3) at (3.75, -0.5);
    
    \coordinate (u14) at (-0.30, 0.10);
    \coordinate (uextra124) at (1.25, 3.25);
    \coordinate (u24) at (4, 4.5);
    \coordinate (uextra2k4) at (5.75, 1);
    \coordinate (uk4) at (3.75, -0.5);

    \fill[color=black] 
            (Cent) circle (0.08)
            (u11) circle (0.08)
            (u21) circle (0.08)
            (uk1) circle (0.08);
    
    \draw[thick] (u11) .. controls (a) and (b) .. (uextra121) 
    -- (u21) .. controls (c) and (d) .. (uextra2k1) -- (uk1) .. controls  (e) and (f) .. (u11);
    
    
    
    \draw[thick] (Cent) circle (2cm);

    \draw[thick, dotted] (Cent) circle (0.5cm);
    \draw[thick, dotted] (Cent1) circle (0.5cm);
    \draw[thick, dotted] (Cent2) circle (0.5cm);
    \draw[thick, dotted] (CentL) circle (0.5cm);
    \node at (SpaceR) {{$\Rcal$}};
    
    \node at (Sphere) {\large{${\Scal}$}};
    \node at (Sphere0) {\tiny{$\Scal_{\ombf}$}};
    \node at (Sphere1) {\tiny{$\Scal_{\zmbf_1}$}};
    \node at (Sphere2) {\tiny{$\Scal_{\zmbf_2}$}};
    \node at (SphereL) {\tiny{$\Scal_{-\zmbf_1}$}};
    
    \node[below right] at (Centcent) {$\ombf$};
    
    \node[below] at (u11) {{$\embf_1$}};
    \node[above] at (u21) {{$\embf_2$}};
    \node[below right] at (uk1) {{$\embf_k$}};
    
    \node at (labelleft) {{\normalsize{(a)}}};
    
     \end{scope}
     
     \hspace{2cm}
     \begin{scope}[shift={(4.3,0)},scale = 0.5]\scriptsize
    
    \def\r{0.12};
    
    \coordinate (a) at (-0.2,1);
    \coordinate (b) at (0.8, 2.75);
    \coordinate (c) at (7, 4);
    \coordinate (d) at (6.5, 2);
    \coordinate (e) at (2.5, -0.75);
    \coordinate (f) at (-0.1, -0.5);
    
    \coordinate (labelright) at (3, -1.75);
    
    \coordinate (Cent) at (3,1.75);
    \coordinate (Centcent) at (2.85,1.90);
    \coordinate (CentR) at (3.6,2.35);
    \coordinate (CentL) at (2.4,1.15);
    
    \coordinate (Space1) at (0.2,0.2);
    \coordinate (Space2) at (-0.3,1.3);
    \coordinate (Spacem) at (-0.1,2.5);
    
    \coordinate (Sphere) at (4.5,3);
    \coordinate (Sphereplus) at (2.6,2.55);
    \coordinate (Sphereminus) at (3.4,0.95);
    
    \coordinate (r) at (3.75,2);
    
    \coordinate (u11) at (-0.25, -0.25);
    \coordinate (uextra121) at (1.25, 3.75);
    \coordinate (u21) at (4, 4.5);
    \coordinate (uextra2k1) at (6, 1);
    \coordinate (uk1) at (3.5, -0.5);
    
    \coordinate (u12) at (0.3, -0.70);
    \coordinate (uextra122) at (0.65, 2.65);
    \coordinate (u22) at (3.4, 4.2);
    \coordinate (uextra2k2) at (5.15, 1.6);
    \coordinate (uk2) at (3.75, 0.1);
    
    \coordinate (u13) at (-0.30, 0.10);
    \coordinate (uextra123) at (1.25, 3.25);
    \coordinate (u23) at (4, 4.5);
    \coordinate (uextra2k3) at (5.75, 1);
    \coordinate (uk3) at (3.75, -0.5);
    
    \coordinate (u14) at (-0.30, 0.10);
    \coordinate (uextra124) at (1.25, 3.25);
    \coordinate (u24) at (4, 4.5);
    \coordinate (uextra2k4) at (5.75, 1);
    \coordinate (uk4) at (3.75, -0.5);

    \fill[color=black] 
            (Cent) circle (0.08)
            (u11) circle (0.08)
            (u21) circle (0.08)
            (uk1) circle (0.08);
    
    \draw[dashed, blue, thick] (u11) .. controls (a) and (b) .. (uextra121) 
    -- (u21) .. controls (c) and (d) .. (uextra2k1) -- (uk1) .. controls  (e) and (f) .. (u11);
    
    \draw[dashed, brown, thick] (u11) .. controls (-2.25, 1.25) and (0.9, 3) .. (u21) .. controls (8.5,4.75) and (7,2) .. (6.5, 1.25) .. controls (6.5, 1) and (5.25, 0.3) .. (uk1) -- (u11);
    
    \draw[dashed, red, thick] (u11) .. controls (-1.5, 1.5) and (-0.5, 3.5) .. (1.5, 4.5) 
    -- (u21) .. controls (6.5, 3.5) and (6, 2) .. (6, 1.75) -- (uk1) .. controls  (3, -1) and (-0.25, -0.75) .. (u11);
    
    \draw[thick] (Cent) circle (1.5cm);
    
    \node at (Space1) {{$\Rcal^1$}};
    \node at (Space2) {{$\Rcal^2$}};
    \node at (Spacem) {{$\Rcal^m$}};
    
    \node at (Sphere) {\large{$\overline{\Scal}$}};
    
    \node[below right] at (Centcent) {$\ombf$};
    
    \node[below left] at (u11) {{$\embf_1$}};
    \node[above] at (u21) {{$\embf_2$}};
    \node[below right] at (uk1) {{$\embf_k$}};
    
    \node at (labelright) {{\normalsize{(b)}}};
    
     \end{scope}
     
    \end{tikzpicture}
  \caption{(a) Geometry of set of predictive rates $\Rcal$: A convex set enclosing a sphere ${\Scal}$ with trivial rates $\embf_i \, \forall \, i \in [k]$ as vertices; 
(b) Geometry of the product set of group rates $\Rcal^1 \times \dots \times \Rcal^m$ (best seen in color) enclosing a common sphere $\overline{\Scal} \subset \Rcal^1 \cap \dots \cap \Rcal^m$.
  }
  \label{fig:MEgeom}
\end{figure}

As a warm up, we overview the Linear Performance Metric Elicitation (LPME) procedure of Chapter~\ref{chp:multiclass}, which 
we will use as a subroutine.  
Here we assume that the oracle's metric is a linear function of rates $\phi^{\text{lin}}(\rmbf) \coloneqq \inner{\ambf}{\rmbf}$, for some unknown costs $\ambf \in \Rmbf^q$. 
In other words, given two rates $\rmbf_1$ and $\rmbf_2$, the oracle returns $\1[\inner{\ambf}{\rmbf_1} > \inner{\ambf}{\rmbf_2}]$. Since the metrics are scale invariant~\cite{narasimhan2015consistent, hiranandani2019multiclass}, w.l.o.g., one may assume $\Vert \ambf \Vert_2=1$. The goal is to elicit (the slope of) $\ambf$ using pairwise comparisons over rates.

When the number of classes $k = 2$, the coefficients $\ambf$ can be elicited using a simple one-dimensional binary search. When $k > 2$, one can apply a  coordinate-wise procedure, performing a binary search in one coordinate, while keeping the others fixed. The efficacy of this procedure, however, hinges on the geometry of the underlying 
 set of feasible rates $\Rcal$, which we discuss below. 
We first make a mild assumption 
ensuring that there is some signal for non-trivial classification.
\bassumption
\label{assump:distribution}
The conditional-class distributions are distinct, i.e., $P(Y=i|X) \ne P(Y=j|X)\;\;\forall\; i, j \in [k]$.
\label{as:sphere}
\eassumption

Let $\embf_i \in \{0,1\}^q$ denote the rates achieved by a trivial classifier that predicts class $i$ for all inputs. 
\bprop
[Geometry of $\Rcal$; Figure~\ref{fig:MEgeom}(a)] The set of  rates $\Rcal \subseteq [0, 1]^{q}$ is convex, has vertices $\{\embf_i\}_{i=1}^k$, and  
contains the rate profile $\ombf = \tfrac{1}{k} \tiny{\sum_{i=1}^k \embf_i}$ in the interior. Moreover, $\ombf$ is achieved by a classifier which for any input predicts each class with equal probability.
\label{prop:C}
\eprop
\bremark[Existence of sphere ${\Scal}$]
Since $\Rcal$ is convex and contains
the point $\ombf$ in the interior, there exists a 
sphere ${\Scal} \subset \Rcal$ of non-zero radius $\rho$ centered at $\ombf$.  
\label{rem:sphere}
\eremark
\vskip -0.2cm

By restricting the coordinate-wise binary search procedure to posing queries from within a sphere, LPME can be equivalently seen as minimizing a strongly-convex function and shown to converge to a solution $\ambfhat$ close to $\ambf$. 
Specifically, the LPME procedure 
takes any 
sphere $\Scal \subset \Rcal$, binary-search tolerance $\epsilon$, and the oracle $\Omega$ (with metric $\phi^{\text{lin}}$) 
 as input, and by posing 
$O(q\log(1/\epsilon))$
queries recovers coefficients $\ambfhat$ with 
$\Vert \ambf - \ambfhat \Vert_2 \leq O(\sqrt{q}\epsilon)$. 
Please see Chapter~\ref{chp:multiclass} for details. 

\bremark[LPME Guarantee]
Given any $q$-dimensional 
sphere $\Scal \subset \Rcal$ and an oracle $\Omega$ with metric $\phi^{\textrm{\textup{lin}}}(\rmbf)\coloneqq\inner{\ambf}{\rmbf}$, 
the LPME algorithm (Algorithm~4.2, Chapter~\ref{chp:multiclass}) provides an estimate $\ambfhat$ with $\Vert \ambfhat \Vert_2=1$ such that the estimated slope is close to the true slope, i.e.,  $\sfrac{{a}_i}{{a}_j} \approx \sfrac{\hat a_i}{\hat a_j} \; \forall \; i, j\in [q]$.
\label{rm:ratio}
\eremark

Note that the algorithm is closely tied with the scale invariance condition and thus only estimates the direction (slope) of the coefficient vector $\ambf$,
and not its magnitude. Also note the algorithm takes as input an \emph{arbitrary} sphere $\Scal \subset \Rcal$, and restricts its queries to rate vectors within the sphere. 
 In Appendix~\ref{append:ssec:sphere}, we discuss an efficient procedure for identifying a sphere 
 of suitable radius.

\vspace{-0.2cm}
\subsection{Quadratic Performance Metrics}
\label{ssec:metric}
Equipped with the LPME subroutine, our aim is to elicit metrics that are quadratic functions of rates.

\bdefinition[Quadratic Metric] For a vector $ \ambf \in \Rmbb^q$  
and a 
symmetric 
matrix $\Bmbf \in \Rmbb^{q \times q}$ with $\Vert \ambf \Vert_2^2  + \Vert \Bmbf \Vert_F^2 = 1$ (wlog.\ due to scale invariance):
\begin{equation}
    \phi^\quadr(\rmbf \,;\, \ambf, \Bmbf) = \inner{\ambf}{\rmbf} + \frac{1}{2} \rmbf^T \Bmbf \rmbf.
    \label{eq:quadmet}
\end{equation}
\vspace{-0.6cm}
\label{def:quadmet}
\edefinition

This family trivially includes the linear metrics as well as many modern metrics outlined below: 

\bexample[Class-imbalanced learning]
\emph{In problems with imbalanced class proportions, it is common to use metrics that emphasize equal performance across all classes. One example is Q-mean \cite{Lawrence+98,LiuCh11,menon2013statistical}, which is the quadratic mean of rates:}
\bequation
\phi^{\qmean}(\rmbf) \,=\, 1/k\sum_{i=1}^k \left(\sum_{j=1}^{k-1}r_{(i-1)(k-1) + j}\right)^2.
\eequation
\eexample

\bexample[Distribution matching]
\emph{
In certain applications, one needs the proportion of predictions 
for each class (i.e., the coverage) to match a target distribution $\boldsymbol{\pi} \in \Delta_k$ \cite{goh2016satisfying,narasimhan2018learning, narasimhan2019optimizing,Cotter:2019}. A 
measure often used for this task is the squared difference between the per-class coverage and the target distribution:} 
\bequation
\phi^{\cov}(\rmbf) \,=\, \sum_{i=1}^k \left(\cov_i(\rmbf) - \pi_i\right)^2,
\eequation
\emph{where $\cov_i(\rmbf) = 1 - \sum_{j=1}^{k-1}r_{(i-1)(k-1) + j} + \sum_{j> i}r_{(j-1)(k-1)+i}+ \sum_{j<i}r_{(j-1)(k-1)+i-1}$. 
Similar metrics can be found in the quantification literature where the target is set to the class prior $\Pmbb(Y=i)$ \cite{Fab1, 
Kar16}. 
We capture more general quadratic distance measures for distributions, e.g.,} \bequation
(\bf{\cov}(\rmbf) - \boldsymbol{\pi})^{T}\Qmbf (\bf{\cov}(\rmbf)-\boldsymbol{\pi})
\eequation
\emph{for a positive semi-definite matrix $\Qmbf \in PSD_k$ \cite{Lindsay08}.
}
\eexample

\bexample[Fairness violation]
\emph{A popular criterion for group-based fairness is {equalized odds}, which requires equal rates across different protected groups \cite{hardt2016equality,bechavod2017learning}. This can be measured by the squared differences between the group rates. 
With $m$ groups and $\rmbf^g$ denoting the rate vector evaluated on examples from group $g$, this is given by:} 
\bequation
\phi^{\eo}((\rmbf^{1},\dots,\rmbf^{m})) \,=\, \sum_{ v>u}\sum_{i=1}^q \left(r^u_i - r^v_i\right)^2.
\eequation
\emph{Other quadratic fair-criteria for two classes include {equal opportunity} $\phi^{\text{EOpp}}((\rmbf^{1},\dots,\rmbf^{m}))
 = \sum_{v>u}(r_1^u - r_1^v)^2$~\cite{hardt2016equality}, {balance for the negative class} $\phi^{\text{BN}}((\rmbf^{1},\dots,\rmbf^{m}))
  = (r_2^u - r_2^v)^2$~\cite{kleinberg2017inherent}, {error-rate balance} $\phi^{\text{EB}}((\rmbf^{1},\dots,\rmbf^{m})) =  0.5\sum_{v>u}(r_1^u - r_1^v)^2 + (r_2^u - r_2^v)^2$~\cite{chouldechova2017fair}, etc.\ and their weighted variants. 
In Section \ref{sec:fairme}, we consider metrics that trade-off between an 
error term and a quadratic fairness term.}
\eexample
\vskip -0.1cm

Note that, due to the scale invariance condition in Definition~\ref{def:quadmet}, the largest singular value of $\Bmbf$ is bounded by 1. This is because $\Vert \Bmbf \Vert_2 \leq \Vert \Bmbf \Vert_F \leq 1$. Thus the metric $\phi^{\quadr}$ is $1$-smooth and implies that it is locally linear around a given rate. Lastly, we need the following assumption on the metric.

\bassumption
\label{assump:smoothness}
The gradient of  $\phi$ at the trivial rate $\ombf$ is non-zero, i.e., $\nabla \phi^{\quadr}(\rmbf)|_{\rmbf=\ombf} = \ambf + \Bmbf\ombf \neq 0.$
\label{as:smooth}
\eassumption

The non-zero gradient assumption is reasonable for a convex $\phi^{\text{quad}}$, where it merely implies that the optimal classifier for the metric is not the uniform random classifier. 

%% file: quadratic/deg2elicitation.tex
\section{Quadratic Metric Elicitation}
\label{sec:quadme}
We now 
present our procedure for Quadratic Performance Metric Elicitation (QPME). We assume that the oracle's unknown metric is quadratic  (Definition~\ref{def:quadmet}) and seek to estimate its parameters $(\ambf, \Bmbf)$ 
by posing  queries to the oracle. 
Unlike LPME, a simple binary search based procedure cannot be directly applied to elicit these parameters. Our approach instead approximates the quadratic metric by a linear function at a few select rate vectors and invokes LPME to estimate the local-linear approximations' slopes. 
The challenge, of course, is to pick a small number of \emph{feasible} rates for performing the local approximations and to reconstruct the original metric \emph{just} from the estimated local slopes. 

\subsection{Local Linear Approximation}
We will find it convenient to work with a shifted version of the quadratic metric, centered at the  point $\ombf$, the uniform random rate vector (see Proposition \ref{prop:C}): 
\begin{align*}
\phi^\quadr(\rmbf;\, \ambf, \Bmbf) &=  
\inner{\dmbf}{\rmbf - \ombf} + \frac{1}{2}(\rmbf - \ombf)^T \Bmbf (\rmbf - \ombf) + c\\ &=\bphi(\rmbf;\, \dmbf, \Bmbf) + c\numberthis \label{eq:quadmetshift},
\end{align*}
where $\dmbf= \ambf+\Bmbf\ombf$ and $c$ is a constant independent of $\rmbf$, and so the oracle can be equivalently seen as responding with the shifted metric $\bphi(\rmbf;\, \dmbf, \Bmbf)$.

Let $z$ be a fixed point in $\Rcal$. Since the metric in Definition~\ref{def:quadmet} is smooth, the metric  can be closely approximated by its first-order Taylor expansion in a small neighborhood around $\zmbf$, i.e.,
\begin{equation}
\bphi(\rmbf;\, \dmbf, \Bmbf) \approx \inner{\dmbf + \Bmbf (\zmbf - \ombf)}{\rmbf} + c',
\label{eq:loclinapx}
\end{equation}
for a constant $c'$. 
So if we apply LPME to the metric $\bphi$ with the queries $(\rmbf_1, \rmbf_2)$ to the oracle restricted to a small ball around $\zmbf$, the procedure effectively estimates the  slope of the vector $\dmbf + \Bmbf (\zmbf - \ombf)$ in the above linear function (up to a small approximation error). 

We will exploit this idea by applying LPME to small neighborhoods around selected  points to  elicit the coefficients $\ambf$ and $\Bmbf$ for the original metric in~\eqref{eq:quadmet}. For simplicity, we will assume that the oracle is noise-free and later show robustness to noise and the query complexity guarantees in Section~\ref{sec:guarantees}.

\subsection{Eliciting Metric Coefficients}

\balgorithm[t]
\caption{QPM Elicitation}
\label{alg:q-me}
\balgorithmic[1]
\STATE \textbf{Input:} ${\Scal}$, 
Search tolerance $\epsilon > 0$, Oracle $\Omega$ with 
metric $\bphi$
\STATE $\fmbf_0 \leftarrow$ LPME$\left(\Scal_\ombf, \epsilon, \Omega\right)$ with $\Scal_\ombf \subset {\Scal}$ and obtain~\eqref{eq:0col}
\FOR{$j \in \{1,2,\dots,q\}$}
\STATE $\fmbf_j\leftarrow$LPME$\left(\Scal_{\zmbf_j}, \epsilon, \Omega\right)$ with $\Scal_{\zmbf_j} \subset {\Scal}$ and obtain~\eqref{eq:jcol}
\ENDFOR
\STATE $\fmbf^-_{1} \leftarrow$ LPME$\left(\Scal_{-\zmbf_1}, \epsilon, \Omega\right)$ with $\Scal_{-\zmbf_1}\hspace{-2pt}\subset\hspace{-1pt} {\Scal}$ and obtain~\eqref{eq:negativegrad}
\STATE $\ambfhat, \Bmbfhat \leftarrow $ normalized solution dervied from
~\eqref{eq:poly2elicitamatfinal}
\STATE \textbf{Output:} $\ambfhat, \Bmbfhat$ 
\ealgorithmic
\ealgorithm

We outline the main steps of Algorithm~\ref{alg:q-me} below. Please see Appendix~\ref{append:sec:qpme} for the full derivation.  

\textbf{Estimate coefficients $\dmbf$ (Line 2).}\
We first wish to estimate the linear portion $\dmbf$ of the metric $\bphi$ in~\eqref{eq:quadmetshift}. For this, we
apply the LPME subroutine to a small ball $\Scal_\ombf \subset \Scal$ of radius $\varrho < \rho$ around the point $\ombf$. See Figure~\ref{fig:MEgeom}(a) for an illustration. 
Within this ball, the metric $\bphi$ approximately equals the linear function
$\inner{\dmbf}{\rmbf} + c'$ using~\eqref{eq:loclinapx}, and so the LPME gives us an estimate of the slope of $\dmbf$.
 From Remark~\ref{rm:ratio},  
 the estimates $\fmbf_0 =
 (f_{10}, \dots, f_{q0})$  approximately satisfy the following $(q-1)$ equations: 
\begin{equation}
    \frac{d_i}{d_1} = \frac{f_{i0}}{f_{10}} \qquad \forall \; i \in \{2, \dots, q\}.
    \label{eq:0col}
\end{equation}

\textbf{Estimate coefficients $\Bmbf$ (Lines 3--5).}
Next, we wish to estimate each column of the matrix $\Bmbf$ of the metric $\bphi$ in~\eqref{eq:quadmetshift}. For this, we apply LPME to small neighborhoods around points in the direction of standard basis vectors $\alphambf_{j} \in \Rmbb^{q}$, $j = 1, \ldots, q$.
Note that within a small ball around $\ombf + \alphambf_j$, the metric $\ophi$ is approximately  the linear function
$\inner{\dmbf + \Bmbf_{:,j}}{\rmbf} + c'$, and so the LPME procedure when applied to this region will give us an estimate of the slope of $\dmbf + \Bmbf_{:,j}$. However, to ensure that the center point we choose is a feasible rate, we will have to re-scale the standard basis, and apply the subroutine to balls $\Scal_{\zmbf_j}$ of radius $\varrho < \rho$ centered at $\zmbf_j = \ombf + (\rho - \varrho)\alphambf_j$. See Figure~\ref{fig:MEgeom}(a) for the visual intuition. The returned estimates $\fmbf_j = (f_{1j}, \dots, f_{qj})$ approximately satisfy:
\begin{equation}
\frac{d_i + (\rho-\varrho)B_{ij}}{d_1 + (\rho-\varrho)B_{1j}} = \frac{f_{ij}}{f_{1j}} \quad \forall \; i \in \{2, \ldots, q\},\; j \leq i.
\label{eq:jcol}
\end{equation}
Since the matrix $\Bmbf$ is symmetric, so far we have $q(q+1)/2$ equations. 
Now note that since we are only eliciting slopes using LPME, we always lose out on one degree of freedom. Hence, there are $q$ more unknown entities, 
and to estimate them we need $q-1$ more equations beside the one normalization condition. For this, we apply LPME to a sphere $\Scal_{-\zmbf_1}$ of radius $\varrho$ around rate $-\zmbf_1$ as shown in Figure~\ref{fig:MEgeom}(a). The returned slopes $\fmbf_1^- = (f_{11}^-, \dots, f_{q1}^-)$ approximately satisfy:
\begin{equation}
    \frac{d_2-(\rho - \varrho)B_{21}}{d_1-(\rho - \varrho)B_{11}} = \frac{f_{21}^-}{f_{11}^-}.
    \label{eq:negativegrad}
\end{equation}
\textbf{Put together (Line 6).}\ By combining~\eqref{eq:0col},~\eqref{eq:jcol} and~\eqref{eq:negativegrad}, we  express each entry of $\Bmbf$ in terms of $d_1$:
\begin{align*}
    B_{ij} &= \Big(F_{i,1,j} (1 + F_{j,1,1}) - F_{i,1,j} F_{j,1,0} d_{1} - F_{i,1,0}+
    F_{i,1,j}\textstyle\frac{F^-_{2,1,1} + F_{2,1,1} - 2F_{2,1,0}}{F^-_{2,1,1} - F_{2,1,1}}\Big)d_1,
    \numberthis \label{eq:poly2elicitamatfinal}
\end{align*}
where
$F_{i,j,l} = f_{il} / f_{jl}$ and $F^-_{i,j,l} = f^-_{il}/f^-_{jl}$. 
Using $\dmbf= \ambf+\Bmbf\ombf$ and the fact that the coefficients are normalized, i.e., $\Vert \ambf \Vert_2^2  + \Vert \Bmbf \Vert_F^2 = 1$, we can obtain estimates for $\Bmbf$ and $\ambf$ independent of $d_1$. 
Moreover, the derivation so far assumes $d_1 \ne 0$. This is based on Assumption \ref{assump:smoothness} which states that at least one coordinate of $\dmbf$ is non-zero, and we've assumed w.l.o.g.\ that this is $d_1$. 
In practice, we can identify a non-zero coordinate using $q$ trivial queries of the form $(\varrho\alphambf_i + \ombf, \ombf), \forall i \in [q]$.

Here, we emphasize on a key difference with Chapters~\ref{chp:binary} and~\ref{chp:multiclass} which is that, there we relied on a boundary point characterization that does not hold for general nonlinear metrics. 
Instead, we use structural properties of the metric to estimate local-linear approximations. As we discussed in the beginning of this chapter, while this may seem a natural idea, the QPME procedure tackles three key challenges: (a) works with only \emph{slopes} for the local-linear functions, (b) ensures that the center points for approximations are feasible, and (c) handles the multiplicative errors in the slopes
(see Section~\ref{sec:guarantees}). 

%% file: quadratic/fairelicitationsetup.tex
\section{Eliciting Quadratic Fairness Metrics}
\label{sec:fairme}
We now discuss quadratic metric elicitation for \emph{algorithmic fairness}. We consider the setup of Chapter~\ref{chp:fair}, where the goal is to elicit a metric that trades-off between predictive performance and fairness violation~\cite{kamishima2012fairness, hardt2016equality, chouldechova2017fair, bechavod2017learning, menon2018cost}. However, unlike Chapter~\ref{chp:fair}, we handle general quadratic fairness violations and show how QPME can be easily employed to elicit group-fair metrics. 

\subsection{Fairness Preliminaries}
\label{ssec:fpmebackground}
We consider a $k$-class problem comprising $m$ groups and use $g \in [m]$ to denote the group membership. The groups are assumed to be disjoint, fixed, and known apriori~\cite{hardt2016equality, agarwal2018reductions, barocas2016big}. We have access to a dataset of size $n$ denoted by $\{(\xmbf, g, y)_i\}_{i=1}^n$, generated \emph{iid} from a distribution $ \Pmbb(X, G, Y)$.  In this case, we will work with a separate (randomized) classifiers $h^g : \Xcal \rightarrow \Delta_k$ for each group $g$, and use 
 $\Hcal^g = \{h^g : \Xcal \rightarrow \Delta_k\}$
 to denote the set of all classifiers for a  group $g$. 
 
\emph{Group predictive rates:} Similar to~\eqref{eq:components}, we denote the group-conditional rate matrix for a classifier $h^g$ by $\Rmbf^g(h^g, \Pmbb) \in \Rmbb^{k \times k}$, where the $ij$-th entry is additionally conditioned on a group and is given by:
\begin{align}
	R^g_{ij}(h^g, \Pmbb) \coloneqq \Pmbb(h^g = j | Y = i, G= g )\;\;\forall \,i, j \in [k].
	\label{eq:f-components}
\end{align}
\vskip -0.2cm
Analogous to the general setup (Section~\ref{sec:background}), 
 we denote the group rates by vectors $\rmbf^g(h^g, \Pmbb) = \offdiag(\Rmbf^g(h^g, \Pmbb))$,
and the set of feasible rates for group $g$  by 
\bequation
\Rcal^g = \{\rmbf^g(h^g, \Pmbb) \,:\, h^g \in \Hcal^g \}.
\eequation

\emph{Rates for overall classifier:} We construct the overall classifier $h : (\Xcal, [m]) \rightarrow \Delta_k$ by predicting with classifier $h^g$ for group $g$, i.e.\ $h(\xmbf, g) \coloneqq h^g(\xmbf)$. We will be interested in both the predictive performance of the overall classifier and its fairness violation.
For the former, we will measure the overall rate matrix for
$h$ as denoted in~\eqref{eq:components}, which can also be represented as:
\begin{align*}
R_{ij} \coloneqq \Pmbb(h=j|Y=i) 
= \sum\nolimits_{g=1}^m t_{i}^gR_{ij}^g,
    \numberthis \label{eq:overallrate}
\end{align*}
where 
$ t^g_{i} \coloneqq \Pmbb(G=g|Y=i)$
is the prevalence of group $g$ within class $i$. 
For the latter, we will need the $m$ group-specific rates, represented together as a tuple: 
\bequation
\rmbf^{1:m} \coloneqq  (\rmbf^1, \dots, \rmbf^m) \in \Rcal^1 \times \dots \times \Rcal^m =: \prodRcal.
\eequation
Lastly, the overall rates in~\eqref{eq:overallrate} can be 
written as a flattened  vector $\rmbf \in [0,1]^q$, and can be expressed in terms of the group-specific rates as $\rmbf = \sum_{g=1}^m \bm{\tau}^g \odot \rmbf^g$, where
$\bm{\tau}^g \coloneqq \offdiag([\mathbf{t}^g \; \mathbf{t}^g\; \ldots \; \mathbf{t}^g])$. 

\subsection{Fair (Quadratic) Metric Elicitation}
\label{ssec:f-metric}

We seek to elicit a metric 
that trades-off between predictive performance defined by a linear function of the overall rates $\rmbf$ and fairness violation defined by a quadratic function of the group rates $\rmbf^{1:m}$.

\bdefinition \emph{(Fair (Quadratic) Performance Metric)} 
For misclassification costs  $\ambf \in \Rmbb^q$, $\ambf \geq 0$, 
fairness violation costs $\mathbb{B} \,=\, \{\Bmbf^{uv} \in PSD_q\}_{u, v=1, v>u}^m$, 
and a trade-off parameter $\lambda \in [0,1]$, we define:
\begin{align*}
&\phi^\fair(\tupr; \ambf, \mathbb{B}, \lambda) \,\coloneqq\, (1-\lambda)\inner{\ambf}{\rmbf} ~+~
\lambda \frac{1}{2} \left(\sum\nolimits_{v>u} (\rmbf^u - \rmbf^v)^T\mathbbm{\Bmbf}^{uv}(\rmbf^{u} - \rmbf^v)\right)
    \numberthis \label{eq:f-linmetric},
\end{align*}
where w.l.o.g.\
the parameters $\ambf$ and $\Bmbf^{uv}$'s are normalized:
    $\Vert \ambf \Vert_2 = 1, \, \frac{1}{2}\sum_{v>u}^{m} \Vert \Bmbf^{uv} \Vert_F = 1.$
\label{def:f-linmetric}
\edefinition

The coefficients $\ambf, \Bmbf^{uv}$'s are separately normalized so that the predictive performance and fairness violation are in the same scale, and we can additionally elicit the trade-off parameter $\lambda$. Analogous to Definitions \ref{def:query}--\ref{def:me}, we  present the problem of fair quadratic metric elicitation. 
\bdefinition[Fair Quadratic Metric Elicitation with Pairwise Comparison Queries (given $\{(\xmbf,g,y)_i\}_{i=1}^n$)]
Let $\Omega$ be an oracle for the (unknown) metric $\phi^\fair$, which for any given $\tupr_1, \tupr_2$, outputs $\Omega(\tupr_1, \tupr_2) = \1[\phi^\fair(\tupr_1) > \phi^\fair(\tupr_2)]$. 
Using oracle queries of the form $\Omega(\tuprhat_1, \tuprhat_2)$, where $\tuprhat_1, \tuprhat_2$ are the estimated rates from samples, recover a metric $\hphi^\fair = (\ambfhat, \hat{\mathbb{B}}, \lambdahat)$ such that $\Vert\phi^\fair - \hphi^\fair\Vert < \kappa$ under a  suitable norm $\Vert \cdot \Vert$ for sufficiently small error tolerance $\kappa > 0$.
\label{def:fpme}
\edefinition
\vskip -0.1cm

Similar to Section \ref{ssec:mpme}, we study the space of feasible rates $\Rcal^{1:m}$ under the following mild assumption. 
\bassumption
For each group $g\in[m]$, the conditional-class distributions $P(Y=j|X, G=g), \, j \in [q],$ are distinct, i.e.\ there is some signal for non-trivial classification for each group.
\label{as:f-sphere}
\eassumption
\bprop
[Geometry of $\prodRcal$; Figure~\ref{fig:MEgeom}(b)] For each group $g$, a classifier that predicts class $i$ on all inputs results in the same rate vector $\embf_i$. The rate space $\Rcal^g$ for each group $g$ is convex and so is the intersection 
$\Rcal^1 \cap \dots \cap \Rcal^m$, which also contains the rate profile $\ombf = \tfrac{1}{k} \tiny{\sum_{i=1}^k \embf_i}$ (achieved by the uniform random classifier) in the interior. 
\label{prop:f-C}
\eprop

\bremark[Existence of sphere $\overline{\Scal}$ in $\Rcal^1 \cap \dots \cap \Rcal^m$]
There exists a 
sphere $\overline{\Scal} \subset \Rcal^1 \cap \dots \cap \Rcal^m$ of radius $\rho$ centered at $\ombf$. Thus, a rate $\smbf \in\overline{\Scal}$ is feasible for each of the $m$ groups, i.e.\ $\smbf$ is achievable by some classifier $h^g$ for each group $g \in [m]$.
\label{as:f-sphere}
\eremark
\vskip -0.1cm

Because we allow separate classifier for each group, the above remark implies that any rate  $\rmbf^{1:m} = (\smbf^1, \ldots, \smbf^m)$ for arbitrary points $\smbf^1, \ldots, \smbf^m \in \overline{\Scal}$ is achievable for some choice of group-specific classifiers $h^1, \ldots, h^m$. This observation will be useful in the elicitation algorithm we describe next.

\subsection{Eliciting Metric Parameters $({\ambf}, \mathbb{B}, \lambda)$}

\begin{figure}[t]
\hspace*{1.6cm}
    \centering
    \includegraphics[scale=1]{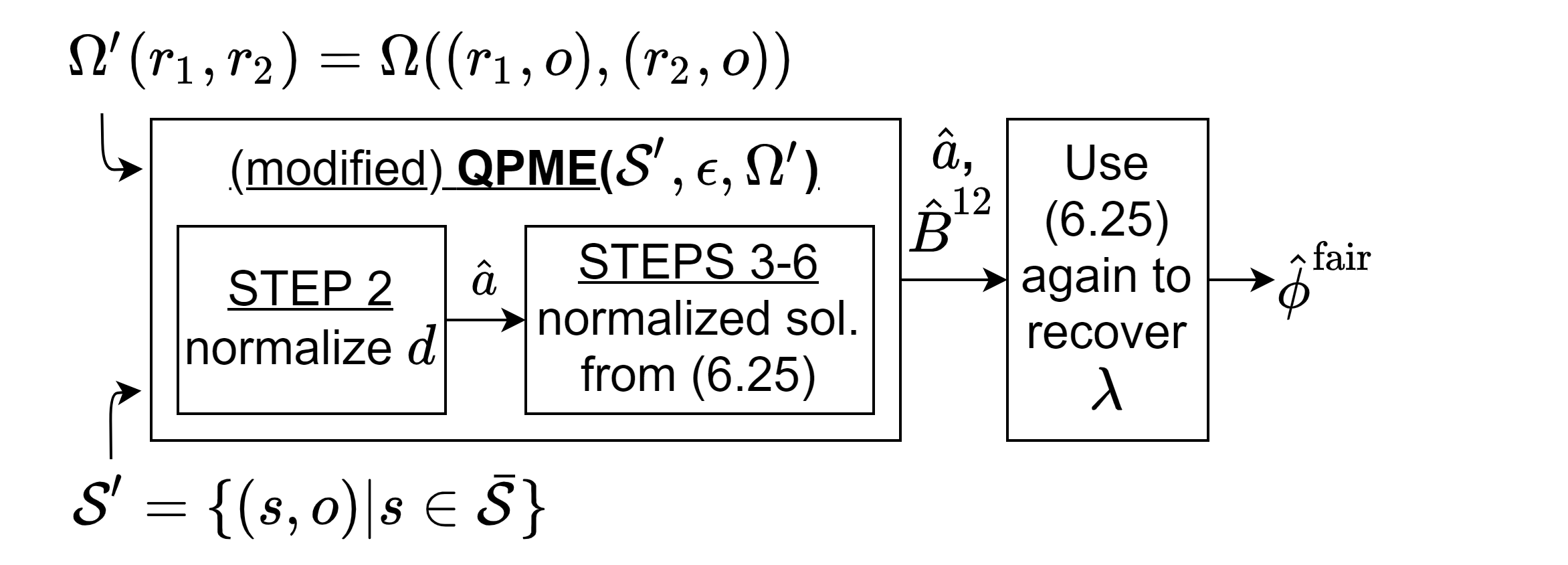}
    \caption{Eliciting Fair Quadratic Metrics (Definition \ref{def:fpme}) for two groups using a minor modification of QPME (Algorithm~\ref{alg:q-me}).}
    \label{fig:fairness-workflow}
\end{figure}

We present a strategy for eliciting fair metrics (Definition \ref{def:f-linmetric}) by adapting the QPME algorithm. For simplicity, we  focus on the $m=2$ case and extend our approach to multiple groups in Appendix \ref{append:sec:fpme}.

Observe that for a rate profile $\rmbf^{1:2} = (\smbf, \ombf)$, where the first group is assigned an arbitrary point in $\overline{\Scal}$ and the second group  is assigned the uniform random classifier's rate $\ombf$, the fair metric~\eqref{eq:f-linmetric} becomes:
\begin{align*}
\phi^\fair((\smbf, \ombf); \ambf,\, \Bmbf^{12}, \lambda) &\coloneqq (1-\lambda)\inner{\ambf}{\bm{\tau}^1 \odot \smbf + \bm{\tau}^2 \odot \ombf} +
 \frac{\lambda}{2} (\smbf - \ombf)^T\Bmbf^{12}(\smbf - \ombf)\vspace{-0.2cm}\\
&\coloneqq \inner{\dmbf}{\smbf - \ombf} +
\frac{1}{2} (\smbf - \ombf)^T\Bmbf(\smbf - \ombf) \\
&\coloneqq \overline{\phi}(\smbf; \dmbf, \Bmbf),
    \numberthis \label{eq:f-linmetricshift}
\end{align*}
where $\dmbf = (1-\lambda)\taumbf^1\odot\ambf$ and $\Bmbf = \lambda \Bmbf^{12}$, and we use $\taumbf^1 + \taumbf^2 =\bm{1}$ (the vector of ones) for the second step. The metric $\overline{\phi}$ above is a particular instance of the quadratic metric in~\eqref{eq:quadmetshift}.  
We can thus apply a slight variant of the QPME procedure in Algorithm~\ref{alg:q-me} to solve the quadratic metric elicitation problem over the sphere $\Scal' = \{(\smbf, \ombf) \,|\, \smbf \in \overline{\Scal}\}$ with the modified oracle $\Omega'(\rmbf_1, \rmbf_2) = \Omega((\rmbf_1, \ombf), (\rmbf_2, \ombf))$. 

The only change needed for the algorithm is in line 7, where 
we need to account for the changed relationship between $\dmbf$ and $\ambf$ and need to separately (not jointly) normalize the linear and quadratic coefficients. With this change, the output of the algorithm directly gives us the required estimates. 
Specifically, from step 2 of Algorithm~\ref{alg:q-me} and \eqref{eq:0col}, we have $\hat{d}_i = (1-\lambda)\tau^1_i \hat{a}_i$. By normalizing $\dmbf$, we 
get $\ambfhat = \frac{\dmbf}{\|\dmbf\|}$ for the linear coefficients.
Similarly, steps 3-6 of Algorithm~\ref{alg:q-me} and \eqref{eq:poly2elicitamatfinal} gives us: 
\begin{align*}
\textstyle
    \hat{B}_{ij} = \lambda\hat{B}^{12}_{ij} = \Big(F_{i,1,j} (1 + F_{j,1,1}) - F_{i,1,j} F_{j,1,0} d_{1}
    + F_{i,1,j}\textstyle\frac{F^-_{2,1,1} + F_{2,1,1} - 2F_{2,1,0}}{F^-_{2,1,1} - F_{2,1,1}}\Big)(1-\lambda)\tau^1_1\hat{a}_1.\numberthis \label{eq:fairBij}
\end{align*}
Again by normalizing we directly get estimates 
$\Bmbfhat^{12} = {\Bmbfhat}/{\|\Bmbfhat\|_F}$ for the quadratic coefficients.

Finally, because the linear and quadratic coefficients are separately  normalized, the estimates $\ambfhat,\, \Bmbfhat^{12}$ are independent of the trade-off parameter $\lambda$. 
Given estimates {\small$\hat{B}^{12}_{ij}$} and $\ahat_1$,  we can now additionally estimate the trade-off parameter {\small$\hat{\lambda}$} 
from~\eqref{eq:fairBij}. See Figure~\ref{fig:fairness-workflow} for an illustration of the entire procedure.


The proposed approach for the fair (quadratic) metric elicitation 
easily extends to multiple groups by applying the QPME procedure described above multiple times after fixing one cluster of groups to the rate $\ombf$ and the remaining to the same rate $\smbf$ in the intersection sphere $\overline{\Scal}$. See Appendix \ref{append:sec:fpme} for details. 
In Appendix~\ref{append:ssec:lambda}, we  also provide an alternate binary search based method similar to Chapter~\ref{chp:fair} for eliciting the trade-off parameter $\lambda$ when the linear predictive and quadratic fairness coefficients are already known. This is along similar lines to the application considered by Zhang et al.~\cite{zhang2020joint}, but unlike them, instead of complicated ratio queries, we require simpler pairwise queries. 

%% file: quadratic/guarantees.tex
\section{Guarantees}
\label{sec:guarantees}

We discuss guarantees for the QPME procedure (Algorithm~\ref{alg:q-me}) under the following feedback model, which is useful in practice. The fair metric elicitation guarantees follow directly as a consequence.

\bdefinition[Oracle Feedback Noise: $\epsilon_\Omega \geq 0$] Given rates $\rmbf_1, \rmbf_2$, 
the oracle responds correctly iff $|\phi^{\quadr}(\rmbf_1) - \phi^{\quadr}(\rmbf_2)| > \epsilon_\Omega$ and may be incorrect otherwise.
\label{def:noise}
\edefinition

In words, the oracle may respond incorrectly if the rates are very close as measured by the metric $\phi^{\quadr}$. 
Since eliciting the metric involves offline computations including certain ratios, we discuss guarantees under the following regularity assumption that ensures all components are well defined. 
\bassumption
For the shifted quadratic metric $\bphi$ in~\eqref{eq:quadmetshift},  
the gradients at the rate profiles $\ombf$, $-\zmbf_1$, and $\{\zmbf_1, \dots, \zmbf_q\}$, are non-zero vectors. 
Additionally, $\rho > \varrho \gg \epsilon_\Omega$.
\label{as:regularity-q}
\eassumption

\btheorem
Given $\epsilon,\epsilon_\Omega\geq 0$, and a 1-Lipschitz metric $\phi^{\quadr}$ (Definition~\ref{def:quadmet}) parametrized by $\ambf, \Bmbf$, under Assumptions~\ref{assump:distribution},  \ref{assump:smoothness},  and \ref{as:regularity-q}, after $O\left(q^2\log \tfrac 1 {\epsilon}\right)$ queries Algorithm~\ref{alg:q-me} returns a metric $\hphi^{\quadr} = (\ambfhat, \Bmbfhat)$ such that 
    $\Vert \ambf-\ambfhat \Vert_{2}\leq O\left({q}(\epsilon+\sqrt{\varrho+\epsilon_\Omega/\varrho})\right)$ 
    and 
    $\Vert \Bmbf -\Bmbfhat \Vert_{F}\leq O\left(q\sqrt{q}(\epsilon+\sqrt{\varrho + \epsilon_\Omega/\varrho})\right)$.
\label{thm:q-me}
\etheorem

\btheorem
While eliciting the metric $\phi^{\quadr}$ (Definition~\ref{def:quadmet}), at least $\Omega(q^2\log(1/q\sqrt q\epsilon))$ pairwise queries are needed to achieve an error of $q\sqrt q\epsilon$ for some (slack) $\epsilon$. 
\label{thm:lb}
\etheorem
Theorem~\ref{thm:q-me} shows that the QPME procedure is robust to noise and its query complexity depends only \emph{linearly} in the number of unknowns. Theorem~\ref{thm:lb} shows that the inherent complexity of the problem is driven by the number of unknowns, which in the most general case (Definition~\ref{def:quadmet}) is $O(q^2)$. Thus, QPME procedure's query complexity is optimal barring the log 
term. We stress that despite eliciting a more complex (nonlinear) metric, the query complexity order is same as prior methods for linear elicitation with respect to the number of unknowns~\cite{hiranandani2018eliciting, hiranandani2019multiclass}. 
With added structural assumptions on the metric, our proposal can be modified to further reduce the query complexity. For example, suppose one knows that the matrix $\Bmbf$ is diagonal, then each  LPME subroutine call  
needs to estimate only one parameter, which can be done in constant number of queries. The resulting query complexity will be  {\small$\tilde O(q)$} which is again \emph{linear} in the number of unknowns.
Moreover, since sample estimates of rates are consistent estimators, and the metrics are $1$-Lipschitz w.r.t.\ rates, with high probability, we gather correct oracle feedback from querying with finite sample estimates $\Omega(\hat{\rmbf}_1, \hat{\rmbf}_2)$ instead of querying with population statistics $\Omega({\rmbf}_1, {\rmbf}_2)$, as long as we have sufficient samples (see Appendix~\ref{append:sec:confusion}). Other than this, Algorithm~\ref{alg:q-me} is agnostic to finite sample errors as long as the sphere $\Scal$ is in the space $\Rcal$.

%% file: quadratic/experiments.tex
\vspace{-0.25cm}
\section{Experiments}
\label{quad-sec:experiments}
We evaluate our approach on simulated oracles. We first present results on a synthetically generated query space and then discuss results on real-world datasets. 

\vspace{-0.25cm}
\subsection{Eliciting Metrics}
\label{ssec:elicitmetrics}

\begin{figure*}[t]
	\centering 
	\vskip -0.1cm
	\subfigure[]{
        \hspace{-0.3cm}
		{\includegraphics[width=4.5cm]{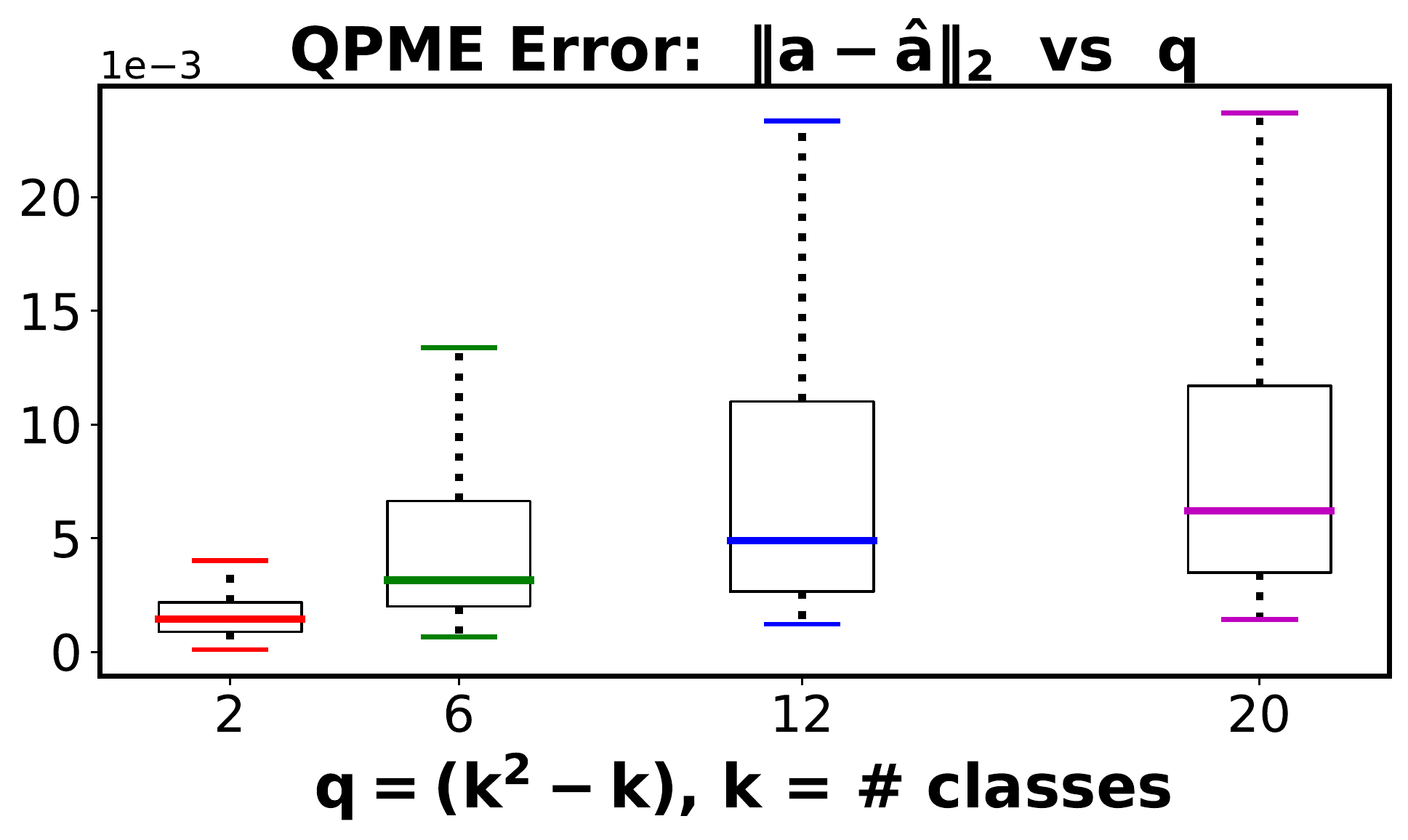}}
		\label{fig:q_rec_a}
	}
	\subfigure[]{
        \hspace{-0.375cm}
		{\includegraphics[width=4.5cm]{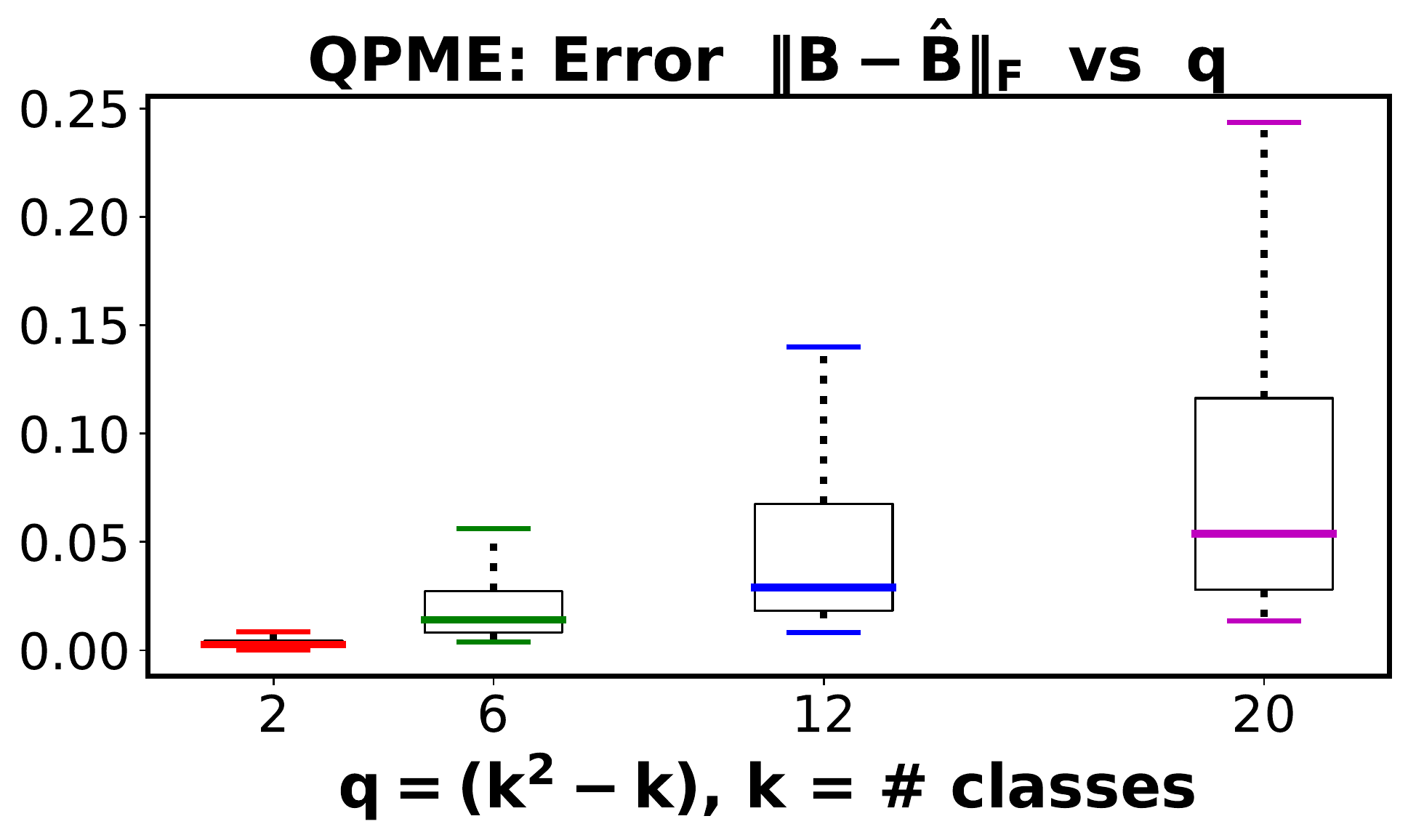}}
		\label{fig:q_rec_B}
	}\\
	\subfigure[]{
        \hspace{-0.375cm}
		{\includegraphics[width=4.5cm]{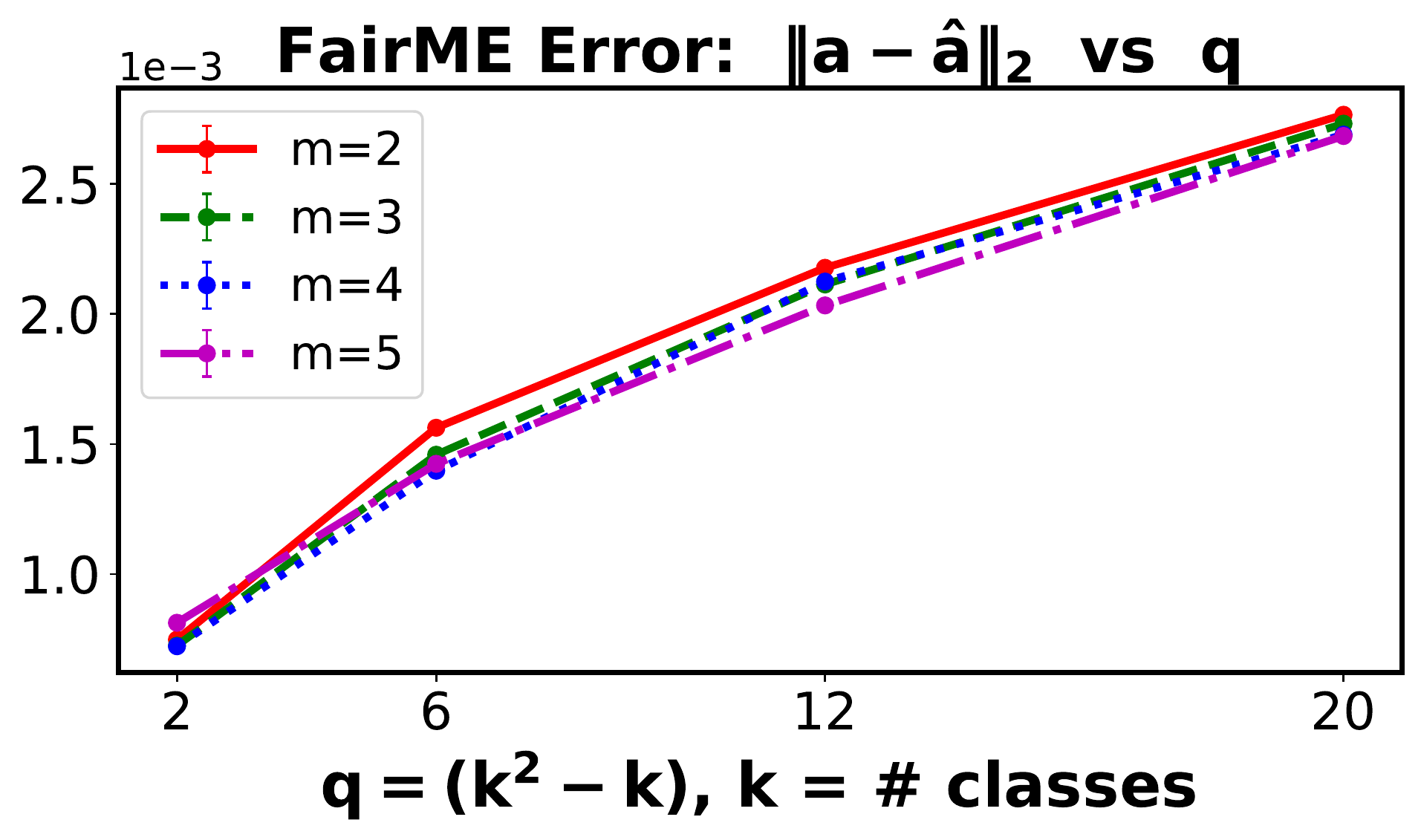}}
		\label{fig:f_rec_a}
	}
	\subfigure[]{
        \hspace{-0.375cm}
		{\includegraphics[width=4.5cm]{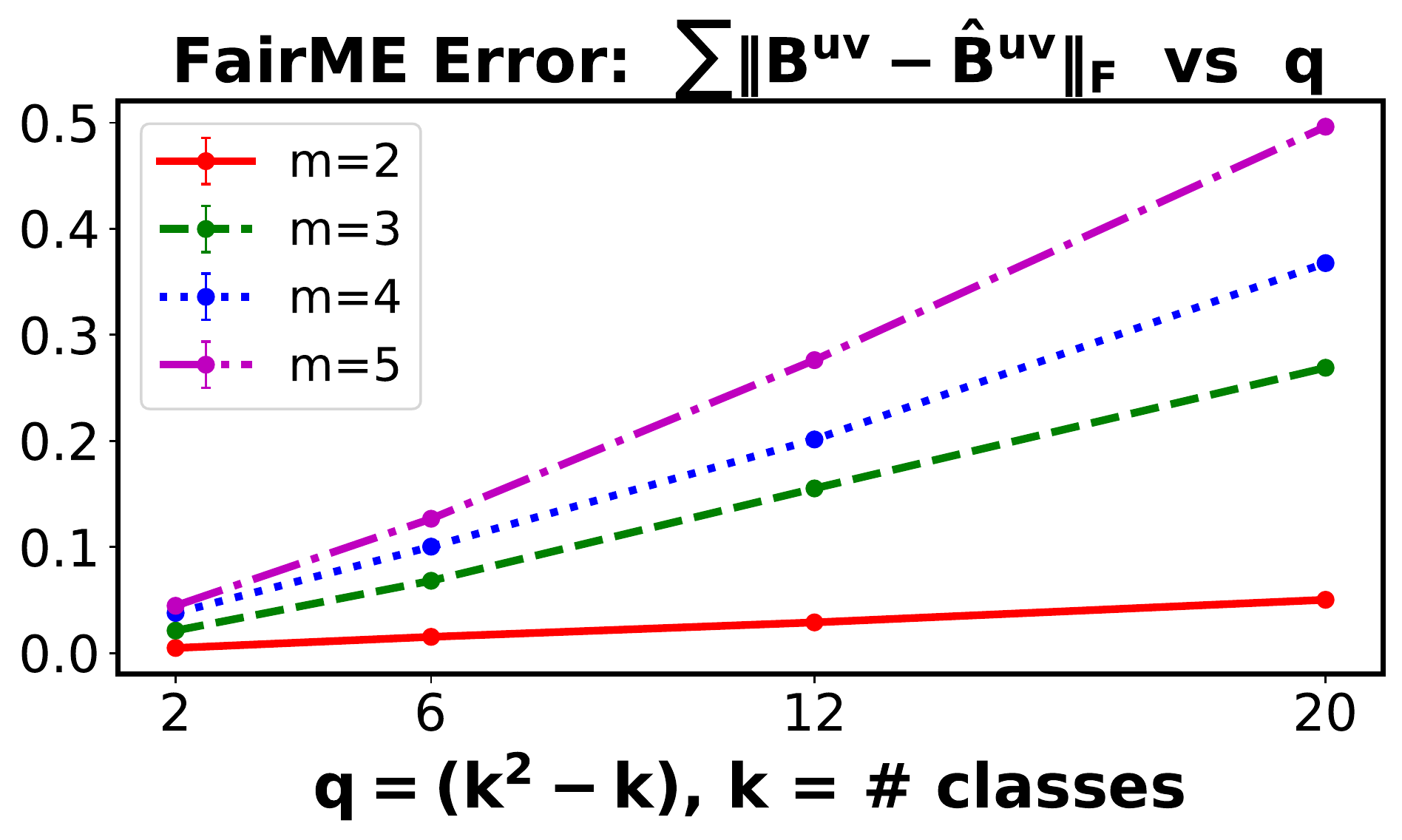}}
		\label{fig:f_rec_B}
	}
	\subfigure[]{
        \hspace{-0.375cm}
		{\includegraphics[width=4.5cm]{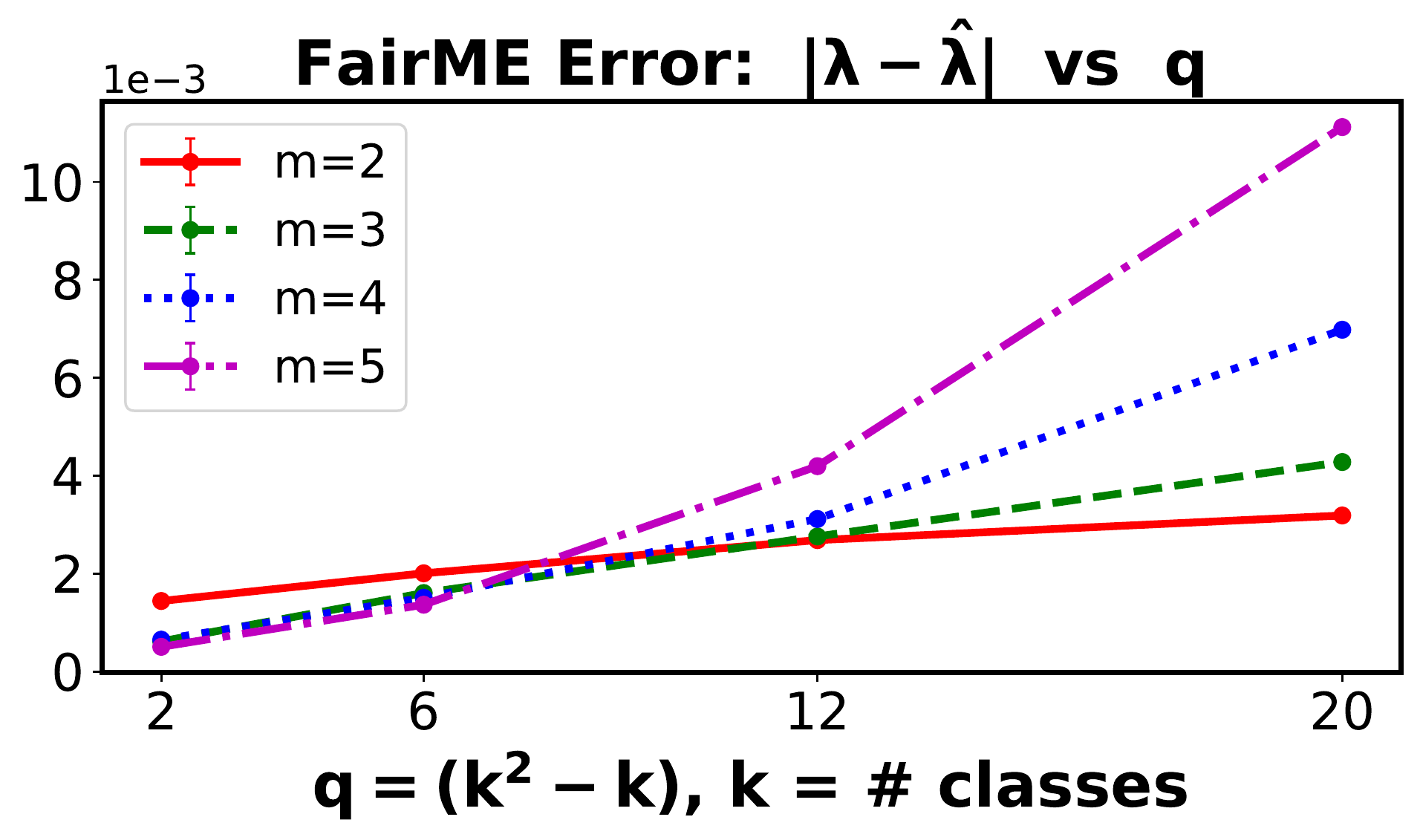}}
		\label{fig:f_rec_l}
	}
	\caption{Average elicitation error over 100 metrics as a function of number of coefficients $q$ and groups $m$ for quadratic metrics in Definition \ref{def:quadmet} (a--b) and fairness metrics in Definition \ref{def:f-linmetric} (c--e).
	}
	\label{fig:recovery}
\end{figure*}

\textbf{Eliciting quadratic metrics.}\ 
We first apply QPME (Algorithm~\ref{alg:q-me}) to elicit quadratic metrics in Definition \ref{def:quadmet}.  We assume access to a $q$-dimensional sphere $\Scal$ centered at rate $\ombf$ with radius $\rho = 0.2$, from which we query rate vectors $\rmbf$. Recall that in practice,
Remark \ref{rem:sphere} guarantees the existence of such a sphere within the feasible region $\Rcal$. We  randomly generate quadratic metrics $\phi^\quadr$ parametrized by $(\ambf, \Bmbf)$ and repeat the experiment over 100 trials for varying  numbers of classes $k \in \{2,3,4,5\}$ (equiv.\ $q \in \{2,6,12,20\}$). 
We run the QPME procedure with tolerance $\epsilon = 10^{-2}$. In Figures~\ref{fig:q_rec_a}--\ref{fig:q_rec_B}, we show box plots 
of the $\ell_2$ (Frobenius) norm between the true and elicited linear (quadratic) coefficients. 
We  generally find that QPME is able to elicit metrics  close to the true ones.
This holds for varying $k$ (and $q$), showing the effectiveness of our approach in handling multiple classes.
The larger standard deviation for $q=20$ is due to Assumption~\ref{as:regularity-q} failing to hold in a few trials and the
resulting estimates not being as accurate. We discuss this in Section \ref{ssec:details}.

\textbf{Eliciting fairness metrics.}\
We next apply the elicitation procedure in Figure \ref{fig:fairness-workflow} with tolerance $\epsilon= 10^{-2}$ to elicit the fairness metrics in Definition \ref{def:f-linmetric}. We randomly generate oracle metrics $\phi^\fair$ parametrized by $(\ambf, \Bmbb, \lambda)$ and repeat the experiment over 100 trials and with varied  number of classes and groups $k, m \in  \{2,3,4,5\}$. Figures~\ref{fig:f_rec_a}--\ref{fig:f_rec_l} show the mean elicitation errors for the the three parameters. For the linear predictive performance, the error {\small$\Vert \ambf - \ambfhat\Vert_2$} increases only with the number of coefficients $q$ and not groups $m$, as it is independent of the number of groups. For the quadratic violation term, the error {\small$\sum_{u,v}\Vert \Bmbf^{uv} - \Bmbfhat^{uv} \Vert_F$} increases with both $q$ and $m$. This is because the QPME procedure is run {\small$m\choose 2$} times for eliciting {\small $m \choose 2$} matrices {\small $\{\Bmbf^{uv}\}_{v > u}$}, and so the elicitation error accumulates with increasing $q$. Lastly,  the elicited trade-off {\small $\hat \lambda$} is seen to be close to the true $\lambda$ as well. 

\subsection{More Details on Simulated Experiments on Quadratic Metric Elicitation} 
\label{ssec:details}

In Figures~\ref{fig:q_rec_a}--\ref{fig:q_rec_B}, we show box plots~\cite{mcgill1978variations}  of the $\ell_2$ (Frobenius) norm between the true and elicited linear (quadratic) coefficients. 
We  generally find that QPME is able to elicit metrics  close to the true ones.

\begin{figure*}[h]
	\centering 
	\subfigure{
		{\includegraphics[width=5cm]{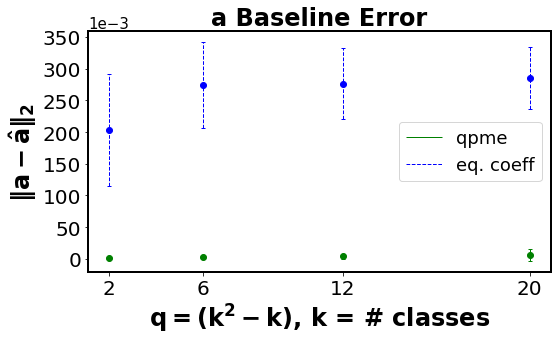}}
		\label{fig:rec_a}
	}\quad\quad
	\subfigure{
		{\includegraphics[width=5cm]{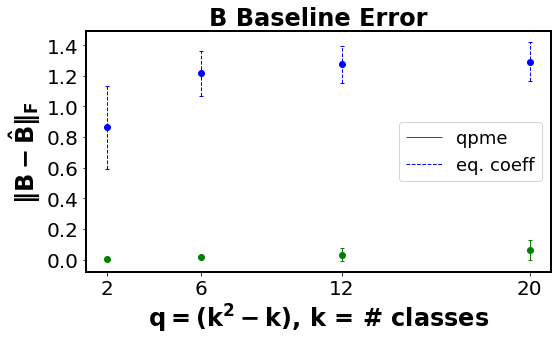}}
		\label{fig:rec_l}
	}
	\caption{Elicitation error in comparison to a baseline which assigns equal coefficients.}
	\label{append:fig:baseline}
\end{figure*}

To reinforce this point, we also compare the elicitation error of the QPME procedure and the elicitation error of a baseline which assigns equal coefficients to $\ambf$ and $\Bmbf$ in Figure~\ref{append:fig:baseline}. We see that the elicitation error of the baseline is order of magnitude higher than the elicitation error of the QPME procedure. This holds for varying $k$ showing that the QPME procedure is able to elicit oracle's multiclass quadratic metrics very well. 

\textbf{Effect of Assumption~\ref{as:regularity-q}}. 
We mentioned in Section \ref{ssec:elicitmetrics} that in a small number of trials,  Assumption~\ref{as:regularity-q} failed to hold with sufficiently large constants $c_{0},c_{-1}, c_1 \ldots, c_q$.
We now analyze in greater detail the effect of this regularity assumption in eliciting quadratic metrics and understand how the lower bounding constants 
impact the elicitation error. Assumption~\ref{as:regularity-q} effectively ensures that the ratios computed in~\eqref{eq:poly2elicitamatfinal} are well-defined. To this end, we generate two sets of 100 quadratic metrics. One set is generated following Assumption~\ref{as:regularity-q} with one coordinate in the gradient being greater than $10^{-2}$, and the other is generated randomly without any regularity condition. For both sets, we run QPME and elicit the corresponding metrics. 

\begin{figure*}[t]
	\centering 
	\subfigure{
		{\includegraphics[width=5cm]{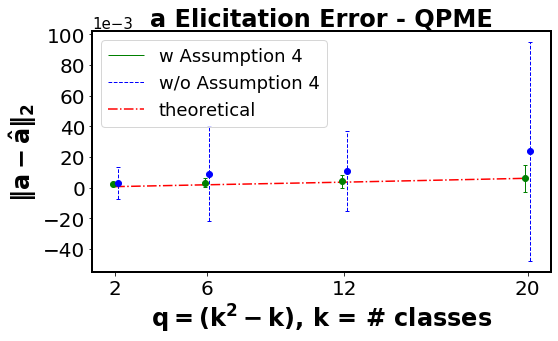}}
		\label{fig:rec_a}
	}\quad\quad
	\subfigure{
		{\includegraphics[width=5cm]{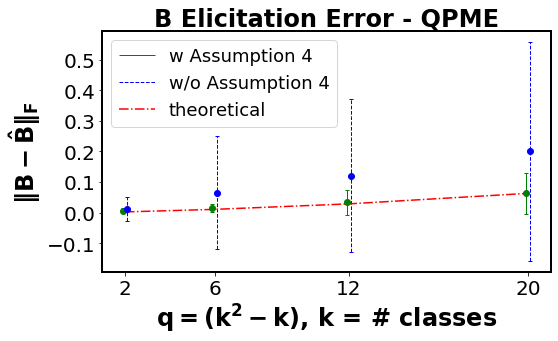}}
		\label{fig:rec_l}
	}
	\caption{Elicitation error for metrics following Assumption~\ref{as:regularity-q} vs elicitation error for completely random metrics.}
	\label{append:fig:regassump}
\end{figure*}

In Figure~\ref{append:fig:regassump}, we see that the elicitation error is much higher when the regularity Assumption~\ref{as:regularity-q} is not followed, owing to the fact that the ratio computation in~\eqref{eq:poly2elicitamatfinal} is more susceptible to errors when gradient coordinates approach zero in some cases of randomly generated metrics. The dash-dotted curve (in red color) shows the trajectory of the theoretical bounds with increasing $q$ (within a constant factor). In Figure~\ref{append:fig:regassump}, we see that  the mean of $\ell_2$ (analogously, Frobenius) norm better follow the theoretical bound trajectory in the case when regularity Assumption~\ref{as:regularity-q}
 is followed by the metrics.

We next analyze the ratio of estimated fractions to the true fractions used in~\eqref{eq:poly2elicitamatfinal} over 1000 simulated runs. Ideally, this ratio should be 1, but as we see in Figure~\ref{append:fig:ratio}, these estimated ratios can be off by a significant amount for a few trials when the metrics are generated randomly. The estimated ratios, however, are more stable under Assumption~\ref{as:regularity-q}. Since we multiply fractions in~\eqref{eq:poly2elicitamatfinal}, even then we may observe the compounding effect of fraction estimation errors in the final estimates. Hence, we see for $k=5$ in Figure~\ref{fig:q_rec_a}-\ref{fig:q_rec_B}, the standard deviation is high due to few trials where the lower bound of $10^{-2}$ on the constants in Assumption~\ref{as:regularity-q}  may not be enough. However, majority of the trials as shown in Figure~\ref{fig:q_rec_a}-\ref{fig:q_rec_B} and Figure~\ref{append:fig:baseline} incur low elicitation error. 

\begin{figure*}[t]
	\centering 
	\subfigure{
		{\includegraphics[width=5cm]{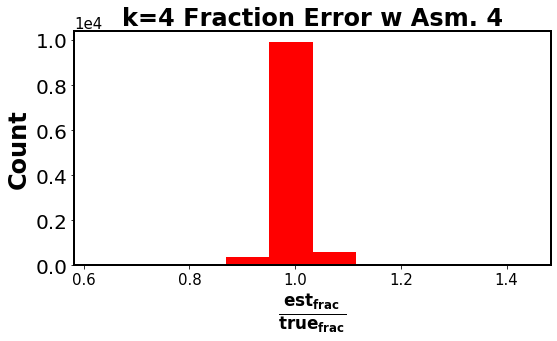}}
		\label{fig:rec_l}
	} 
	\subfigure{
		{\includegraphics[width=5cm]{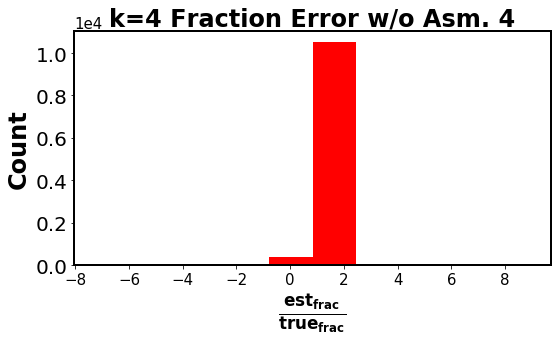}}
		\label{fig:rec_a}
	} \\
	\subfigure{
		{\includegraphics[width=5cm]{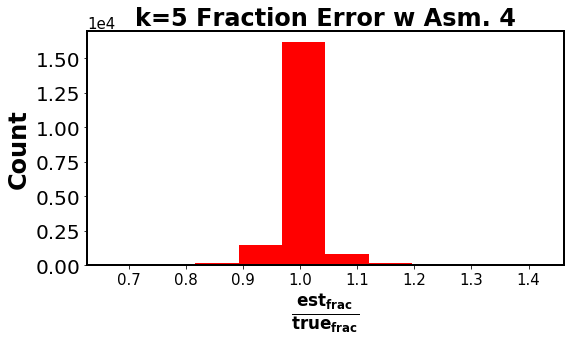}}
		\label{fig:rec_l}
	}
	\subfigure{
		{\includegraphics[width=5cm]{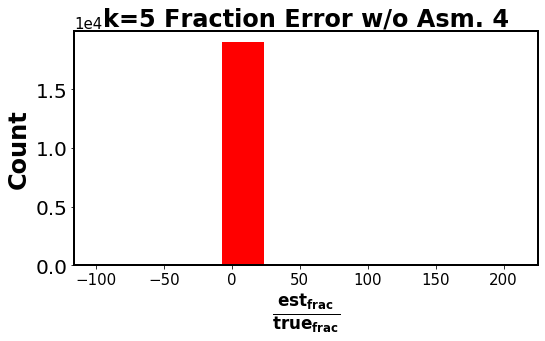}}
		\label{fig:rec_a}
	}
	\caption{Ratio of estimated to true fractions over 1000 simulated runs with and without Assumption~\ref{as:regularity-q}.}
	\label{append:fig:ratio}
\end{figure*}

\subsection{Ranking of Real-World Classifiers}
\label{quad-ssec:ranking}

Performance metrics provide quantifiable scores to classifiers. This score is then often used to rank classifiers and select the best set of classifiers in practice. In this section, we discuss the benefits of elicited metrics in comparison to some default metrics while ranking real-world classifiers. 

\begin{table}[t]
\centering
\caption{Dataset statistics}
\begin{tabular}{|c|ccc|}
\hline
\textbf{Dataset} & $k$ & \textbf{\#samples} & \textbf{\#features} \\ 
\hline
default        & 2   & 30000           &    33    \\
adult        &  2  &    43156       &    74   \\
sensIT Vehicle        &  3  & 98528          &     50   \\
covtype        &  7 &    581012       &     54 \\ 
\hline
\end{tabular}
\label{append:tab:stats}
\end{table}

For this experiment, we work with four real world datasets with varying number of classes $k\in \{2,3, 7\}$. See Table~\ref{append:tab:stats} for details of the datasets. We use 60\% of each dataset to train classifiers. The rest of the data is used to compute (testing) predictive rates. For each dataset, we create a pool of 80 classifiers by tweaking hyper-parameters in some famous machine learning models that are routinely used in practice. Specifically, we create 20 classifiers each from logistic regression models~\cite{kleinbaum2002logistic}, multi-layer perceptron models~\cite{pal1992multilayer},  LightGBM models~\cite{ke2017lightgbm}, and support vector machines~\cite{joachims1999svmlight}. 
We compare ranking of these 80 classifiers provided by competing baseline metrics with respect to the ground truth ranking, which is provided by the oracle's true metric. 

We generate a random quadratic metric $\phi^{\text{quad}}$ following Definition~\ref{def:quadmet}. We treat the true $\phi^{\text{quad}}$ as oracle's metric. It provides us the ground truth ranking of the classifiers in the pool. We then use our proposed procedure QPME (Algorithm~\ref{alg:q-me}) to recover the oracle's metric. For comparison in ranking of real-world classifiers, we choose two linear metrics that are routinely employed by practitioners as baselines. The first is accuracy $\phi^{acc} = 1/\sqrt{q}\inner{\bm{1}}{\rmbf}$, and the second is weighted accuracy, where we just use the linear part  $\inner{\ambf}{\rmbf}$ of the oracle's true quadratic metric $\inner{\ambf}{\rmbf} + \frac{1}{2}\rmbf^T\Bmbf\rmbf$. We repeat this experiment over 100 trials. 

\begin{figure*}[t]
	\centering 
	\subfigure{
		{\includegraphics[width=5cm]{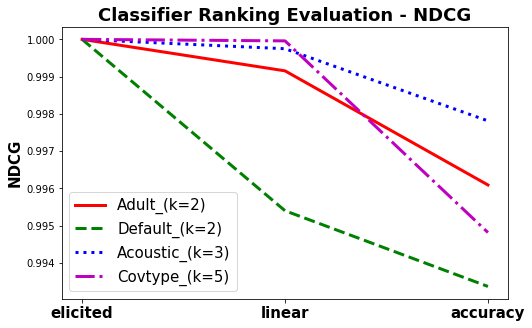}}
		\label{fig:rec_B}
	}\quad\quad
	\subfigure{
		{\includegraphics[width=5cm]{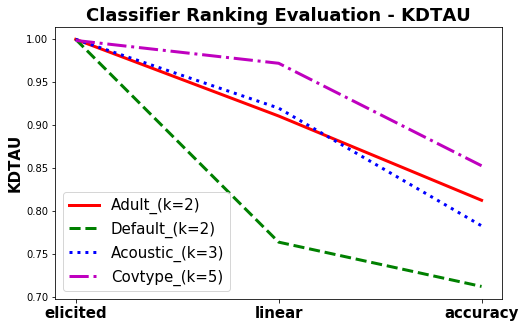}}
		\label{fig:rec_l}
	}
	\caption{Performance of competing metrics while ranking real-world classifiers. `elicited' is the metric elicited by QPME, `linear' is the metric that comprises only the linear part of the oracle's true quadratic metric, and `accuracy' is the linear metric which weigh all classification errors equally (often used in practice).}
	\label{append:fig:ranking}
\end{figure*}

We report NDCG (with exponential gain)~\cite{valizadegan2009learning} and Kendall-tau coefficient~\cite{shieh1998weighted} averaged over the 100 trials in Figure~\ref{append:fig:ranking}. We observe consistently for all the datasets that the elicited metrics using the QPME procedure achieve the highest possible NDCG and Kendall-tau coefficient of 1. As we saw in Section~\ref{sec:guarantees}, QPME may incur elicitation error, and thus the elicited metrics may not be very accurate; however, Figure~\ref{append:fig:ranking} shows that the elicited metrics may still achieve near-optimal ranking results. This implies that when given a set of classifiers, ranking based on elicited metric scores align most closely to true ranking in comparison to ranking based on default metric scores. Consequentially, the elicited metrics may allow us to select or discard classifiers for a given task. This is advantageous in practice. 
For the \emph{covtype} dataset, we see that the \emph{linear} metric also achieves high NDCG values, so perhaps ranking at the top is quite accurate; however Kendall-tau coefficient is low suggesting that the overall ranking of classifiers is poor. We also observe that, in general, the weighted version (\emph{linear} metric) is better than \emph{accuracy} while ranking classifiers.

\begin{figure*}[t]
	\centering 
	\subfigure{
		{\includegraphics[width=5cm]{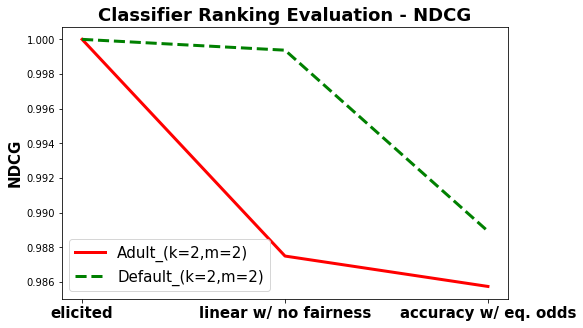}}
		\label{fig:rec_B}
	}\quad\quad
	\subfigure{
		{\includegraphics[width=5cm]{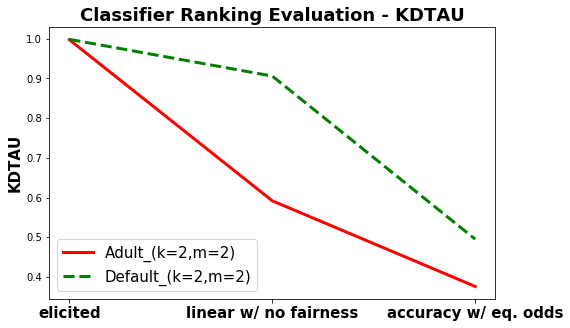}}
		\label{fig:rec_l}
	}
	\caption{Performance of competing metrics while ranking real-world classifiers for fairness. `elicited' is the metric elicited by the (quadratic) fairness metric elicitation procedure from Section~\ref{sec:fairme} (also depicted in Figure~\ref{fig:fairness-workflow}), `linear w/ no fairness' is the metric that comprises only the linear part of the oracle's true quadratic fair metric from Definition~\ref{def:f-linmetric} without the fairness violation, and `accuracy w/ eq. odds' is the metric which weigh all classification errors and fairness violations equally (often used in practice).}
	\label{append:fig:rankingfpme}
\end{figure*}

With regards to fairness, we performed a similar experiment as above for comparing fair-classifiers' ranking on Adult and Default datasets with gender as the protected group. There are two genders provided in the datasets, i.e., $m=2$. We simulate fairness metrics as given in Definition~\ref{def:f-linmetric} that gives ground-truth ranking of classifiers and evaluate the ranking by the elicited (fair-quadratic) metric using the procedure described in Section~\ref{sec:fairme} (also depicted in Figure~\ref{fig:fairness-workflow}). In Figure~\ref{append:fig:rankingfpme}, we show the NDCG and KD-Tau values for our method and for two baselines: (a) `linear w/ no fairness', which is the metric that comprises only the linear part of the oracle's true quadratic fair metric from Definition~\ref{def:f-linmetric} without the fairness violation, and (b) `accuracy w/ eq. odds' is the metric which weigh all classification errors and fairness violations equally. We again see that the elicited (fairness) metric’s ranking is closest to the ground-truth. 

%% file: quadratic/extensions.tex
\section{Extension to Higher Order Polynomials}
\label{sec:poly}

Our approach can be generalized to \textit{higher-order polynomials} of rates. 
Consider e.g.\ a cubic polynomial:
\begin{align}
    \phi^{\text{cubic}}(\rmbf)\coloneqq \sum_{i}a_ir_i + \frac{1}{2}\sum_{i,j}B_{ij}r_ir_j + \frac{1}{6}\sum_{i,j,l}C_{ijl}r_ir_jr_l,
\end{align}
where $\Bmbf$ and $\Cmbf$ are symmetric, and $\sum_i a_i^2 +\sum_{ij} B_{ij}^2 + \sum_{ijl} C_{ijl}^2 = 1$ (w.l.o.g., due to scale invariance).  A quadratic approximation to this metric around a point $\zmbf$ is given by: 
\bequation
\sum_{i}a_ir_i + \frac{1}{2}\left(\sum_{i,j}B_{ij}r_ir_j + \sum_{i,j,l}C_{ijl}(r_i - z_i)(r_j - z_j)z_l\right) + c,
\eequation
where $c$ is a constant not affecting the oracle responses. We can estimate the parameters of this  approximation by applying the QPME procedure from Algorithm~\ref{alg:q-me} with the metric centered at an appropriate point, and its queries restricted to a small neighborhood around $\zmbf$. Running QPME once using a sphere around the point $\zmbf_l = \ombf + (\varrho - \varrho')\alphambf_l$, where $\varrho' < \varrho$ will elicit one face of the tensor $\Cmbf_{[:, :, l]}$ upto a scaling factor. Thus, it will require us to run the QPME procedure $q$ times around the basis points $\zmbf_l = \ombf + (\varrho - \varrho')\alphambf_l \; \; \forall l \in [q]$. Since we elicit scale-invariant quadratic approximation, we would need additional run of QPME procedure around the point $\Scal_{-\zmbf_1}$ to elicit all the coefficients. Thus, we can recover the metric $\hat{\phi}^{\text{cubic}} = (\hat{\ambf}, \hat{\Bmbf},\hat{\Cmbf})$ with as many queries as the number of unknowns, i.e, $\tilde O(q^3)$ in the cubic case. 

For a $d$-th order polynomial, one can recursively apply this procedure to estimate $(d-1)$-th order approximations at multiple points, and similarly derive the polynomial coefficients from the estimated local approximations.

%% file: quadratic/relatedwork.tex
\section{Related Work}
\label{sec:relatedwork}

Chapter~\ref{chp:me} formalized the problem of ME, Chapter~\ref{chp:binary} put forward an ME procedure for binary classification and then later Chapter~\ref{chp:multiclass} extends ME to the multiclass setting~\cite{hiranandani2019multiclass}. The focus in the previous chapters, however, was on eliciting linear and fractional-linear metrics; whereas, in this chapter, we elicit more complex quadratic metrics. 
Learning linear functions passively using pairwise comparisons is a mature field~\cite{joachims2002optimizing, herbrich2000large, peyrard2017learning}, but unlike their active learning counter-parts ~\cite{settles2009active, jamieson2011active, kane2017active}, these methods are not query efficient. 
Other related work include active classification~\cite{settles2009active, kane2017active, noriega2019active}, which 
learn classifiers for a fixed (known) metric. In contrast, we seek to elicit an unknown metric by posing queries to an oracle.
There is also some work on active linear elicitation, e.g. Qian et al.~\cite{qian2015learning}, 
but they do not provide theoretical bounds and work with a different query space. We are unaware of prior work on eliciting a quadratic function, either passively 
or
actively using pairwise comparisons. 

The use of metric elicitation for fairness is relatively new, with 
some work on eliciting \textit{individual} fairness metrics~\cite{ilvento2019metric, mukherjee2020two}. To the best of our knowledge, the work in Chapter~\ref{chp:fair} is the only work that elicits \textit{group-fair} metrics, which we extend in this chapter to handle more general 
metrics. 
Zhang et al.~\cite{zhang2020joint} 
elicit  the trade-off between accuracy and fairness using complex ratio queries. In contrast, we jointly  elicit the predictive performance, fairness violation, and trade-off 
using simpler pairwise queries.  
Lastly, prior work 
has also focused on  learning fair classifiers under constraints 
\cite{hardt2016equality, zafar2017constraints,narasimhan2018learning}.
We take the regularization view of fairness, where the fairness violation is included in the objective itself~\cite{kamishima2012fairness, bechavod2017learning, corbett2017algorithmic, agarwal2018reductions}.

Our work is also related to decision-theoretic \emph{preference elicitation}, however, with the following key  differences. We focus on estimating the utility function (metric) explicitly, whereas prior work such as~\cite{boutilier2006constraint, benabbou2017incremental} seek to find the optimal decision via minimizing the max-regret over a set of utilities. Studies that directly learn the utility~\cite{white1984model, perny2016incremental} do not provide query complexity guarantees for pairwise comparisons. Formulations that consider a finite set of alternatives~\cite{white1984model, chajewska2000making, boutilier2006constraint}, are starkly different than ours, because the set of alternatives in our case (i.e. classifiers or rates) is infinite. 
Most papers focus on linear~\cite{white1984model} or bilinear~\cite{perny2016incremental} utilities except for~\cite{braziunas2012decision} (GAI utilities) and~\cite{benabbou2017incremental} (Choquet integral); whereas, we focus on quadratic metrics which are useful for classification tasks, especially, fairness. 

%% file: quadratic/discussion.tex
\section{Discussion, Limitations, and Future Work}
\label{sec:discussion}
We have provided an efficient quadratic metric elicitation strategy and shown its application to the pressing issue in algorithmic fairness. Interestingly, the query complexity for these non-linear metrics has the same dependence on the number of unknowns as that for linear metrics. We have also shown how this idea can be extended to elicit higher order polynomial metrics. This significantly increases the use-cases for ME and opens the door for non-linear metric elicitation. A notable advantage of our proposal is that it is independent of the population $\Pmbb$. Thus any metric that is learned using one dataset or model class can be applied to other applications, as long as the expert believes the tradeoffs are the same.
A key challenge that we tackle throughout elicitation is maintaining the feasibility of rates, i.e., rates that are achievable by classifiers. This has a practical advantage, because now one has the flexibility to deploy systems that either compare classifiers or compare rates.  

At the same time, our work has limitations, too. We assume a parametric form for the quadratic oracle metric, which may not be a good match to practice. Extension to polynomial elicitation helps but may lead to overburdening the oracle with the huge number of queries if the degree of the polynomial is high. Another limitation is that it leaves open the question of who the oracles should be. Furthermore, one should be cautious 
of the failure of the metric elicitation system especially while eliciting fairness metrics, because that can cause varying impacts among protected groups. 
We look forward to future work answering these practical questions. 

%% file: blackbox/header.tex
\chapter{Optimizing Black-box Metrics through Metric Elicitation}
\label{chp:blackbox}

In this chapter, we discuss an interesting application of Metric Elicitation (ME), where the tools and procedures provided in the previous chapters play a key role. We aim to optimize a black-box performance metric, where instead of  a \emph{human} oracle, we have a \emph{machine} oracle that responds with absolute quality value of a classifier. As we discuss later, such settings are prevalent in literature. The motivation for using ME for black-box optimization comes from the fact that many existing optimization algorithms are iterative in nature, where in each iteration, they tend to optimize a local-linear approximation. This local-linear approximation of an unknown (black-box) metric can be elicited using the existing ME tools and results~\cite{hiranandani2020quadratic}. 
We discuss briefly how these procedures can be extended in the presence of \emph{human} oracles that provide pairwise preference feedback (including the A/B tests based scenarios). 
We next discuss the formal black-box optimization problem setup and how our tools from ME can be used to optimize metrics in this setup. 

\input{blackbox/preamble}
\input{blackbox/introduction}
\input{blackbox/setup}

\input{blackbox/methods}
\input{blackbox/theory}
\input{blackbox/relatedwork}
\input{blackbox/experiments}
\input{blackbox/discussion}

%% file: blackbox/preamble.tex
\newcommand{\Fig}[1]{Figure~\ref{#1}}
\newcommand{\Sec}[1]{Section~\ref{#1}}
\newcommand{\Tab}[1]{Table~\ref{#1}}
\newcommand{\Tabs}[2]{Tables~\ref{#1}--\ref{#2}}
\newcommand{\Eqn}[1]{Eq.~(\ref{#1})}
\newcommand{\Eqs}[2]{Eqs.~(\ref{#1}-\ref{#2})}
\newcommand{\Lem}[1]{Lemma~\ref{#1}}
\newcommand{\Thm}[1]{Theorem~\ref{#1}}
\newcommand{\Prop}[1]{Proposition~\ref{#1}}
\newcommand{\Cor}[1]{Corollary~\ref{#1}}
\newcommand{\App}[1]{Appendix~\ref{#1}}
\newcommand{\Def}[1]{Definition~\ref{#1}}
\newcommand{\Algo}[1]{Algorithm~\ref{#1}}
\newcommand{\Exmp}[1]{Example~\ref{#1}}
\definecolor{Gray}{gray}{0.8}
\renewcommand{\hat}{\widehat}
\renewcommand{\tilde}{\widetilde}
\renewcommand{\>}{{\rightarrow}}
\renewcommand{\=}{\stackrel{\triangle}{=}}
\newcommand{\half}{\textstyle{\frac{1}{2}}}
\newcommand{\grad}{\nabla}
\newcommand{\poly}{\operatorname{poly}}
\newcommand{\rank}{\operatorname{rank}}
\newcommand{\sign}{\operatorname{sign}}
\newcommand{\R}{{\mathbb R}}
\newcommand{\Z}{{\mathbb Z}}
\newcommand{\N}{{\mathbb N}}
\renewcommand{\P}{{\mathbf P}}
\newcommand{\E}{{\mathbf E}}
\newcommand{\Var}{{\mathbf{Var}}}
\newcommand{\I}{{\mathbf I}}
\newcommand{\0}{{\mathbf 0}}
\newcommand{\A}{{\mathbf A}}
\newcommand{\cA}{{\mathcal A}}
\newcommand{\B}{{\mathcal B}}
\newcommand{\cC}{{\mathcal C}}
\newcommand{\C}{{\mathbf C}}
\newcommand{\cE}{{\mathcal E}}
\newcommand{\cF}{{\mathcal F}}
\newcommand{\F}{{\mathbf F}}
\newcommand{\G}{{\mathcal G}}
\renewcommand{\H}{{\mathcal H}}
\newcommand{\cI}{{\mathcal I}}
\newcommand{\K}{{\mathcal K}}
\renewcommand{\L}{{\mathbf L}}
\newcommand{\M}{{\mathcal M}}
\newcommand{\cN}{{\mathcal N}}
\newcommand{\cP}{{\mathcal P}}
\newcommand{\bR}{{\mathbf R}}
\newcommand{\Q}{{\mathbf Q}}
\renewcommand{\S}{{\mathcal S}}
\renewcommand{\B}{{\mathbf B}}
\newcommand{\cT}{{\mathcal T}}
\newcommand{\X}{{\mathcal X}}
\newcommand{\Y}{{\mathcal Y}}
\newcommand{\cZ}{{\mathcal Z}}
\renewcommand{\b}{{\mathbf b}}
\renewcommand{\c}{{\mathbf c}}
\newcommand{\f}{{\mathbf f}}
\newcommand{\h}{{h}}
\newcommand{\m}{{\mathbf m}}
\newcommand{\p}{{\mathbf p}}
\newcommand{\q}{{\mathbf q}}
\renewcommand{\r}{{\mathbf r}}
\newcommand{\s}{{\mathbf s}}
\renewcommand{\u}{{\mathbf u}}
\renewcommand{\v}{{\mathbf v}}
\newcommand{\w}{{\mathbf w}}
\newcommand{\x}{{\mathbf x}}
\newcommand{\zo}{\textup{\textrm{0-1}}}
\newcommand{\ord}{\textup{\textrm{ord}}}
\newcommand{\TPR}{\textup{\textrm{TPR}}}
\newcommand{\TNR}{\textup{\textrm{TNR}}}
\newcommand{\AM}{\textup{\textrm{AM}}}
\newcommand{\GM}{\textup{\textrm{GM}}}
\newcommand{\lin}{\textup{\textrm{lin}}}
\newcommand{\bloss}{{\boldsymbol \ell}}
\newcommand{\balpha}{{\boldsymbol \alpha}}
\newcommand{\bxi}{{\boldsymbol \xi}}
\newcommand{\bpsi}{{\boldsymbol \psi}}
\newcommand{\bpi}{{\boldsymbol \pi}}
\newcommand{\bophi}{{\boldsymbol \phi}}
\newcommand{\bsigma}{{\boldsymbol \sigma}}
\newcommand{\btheta}{{\boldsymbol \theta}}
\newcommand{\seta}{{\boldsymbol \eta}}
\newcommand{\bGamma}{{\boldsymbol \Gamma}}
\newcommand{\bSigma}{{\boldsymbol \Sigma}}

\newcommand{\argsort}{\textup{\textrm{argsort}}}
\newcommand{\er}{\textup{\textrm{er}}}

\newcommand{\boldeta}{{\boldsymbol \eta}}
\newcommand{\rl}{{}}
\newcommand{\mn}{{}}
\newcommand{\LMO}{\text{\tt{LMO}}}
\newcommand{\wLMO}{\widehat{\text{\tt{LMO}}}}
\newcommand{\optsolver}{\text{\tt{OptSolver}}}
\newcommand{\CPE}{\textup{\textrm{CPE}}}
\newcommand{\lwr}{\textup{\textrm{lower}}}
\newcommand{\upr}{\textup{\textrm{upper}}}
\newcommand{\FW}{\textup{\textrm{FW}}}
\newcommand{\BS}{\textup{\textrm{BS}}}
\newcommand{\bmu}{{\boldsymbol \mu}}
\newcommand{\T}{{\mathbf T}}
\newcommand{\svmp}{$\text{SVM}^\text{perf~}$}
\newcommand{\oargmin}{{\textup{\textrm{argmin}}^*}}

\newcommand{\vol}{\textup{\textrm{vol}}}
\newcommand{\V}{{\mathcal V}}
\newcommand{\g}{{\mathbf g}}
\newcommand{\e}{{\mathbf e}}
\newcommand{\cB}{{\mathcal B}}
\newcommand{\hbC}{\widehat{\C}}
\newcommand{\hC}{\widehat{C}}
\newcommand{\bG}{\mathbf{G}}
\newcommand{\blambda}{\boldsymbol{\lambda}}
\newcommand{\perf}{\cE}
\newcommand{\bo}{\mathbf{o}}

\newcommand{\Dtrue}{D}
\newcommand{\Dshift}{\mu}
\newcommand{\tr}{\textup{\textrm{tr}}}
\newcommand{\val}{\textup{\textrm{val}}}
\newcommand{\onehot}{\textup{\textrm{onehot}}}

\newcommand{\W}{\mathbf{W}}
\newcommand{\D}{\mathbf{D}}
\newcommand{\boPhi}{\boldsymbol{\Phi}}
\newcommand{\bM}{\mathbf{M}}
\newcommand{\bO}{\mathbf{O}}
\newcommand{\bD}{\mathbf{D}}
\newcommand{\bcE}{\boldsymbol{\cE}}

\newcommand{\bbeta}{\boldsymbol{\beta}}
\newcommand{\bz}{\mathbf{z}}
\newcommand{\bupsilon}{\boldsymbol{\upsilon}}
\newcommand{\bE}{\mathbf{E}}

\newcommand{\todohari}[1]{{\color{red}TODO(Hari): #1}}
\newcommand{\todogh}[1]{{\color{red}TODO(GH): #1}}


%% file: blackbox/introduction.tex
\section{Introduction}
\label{sec:introduction}
In many real-world machine learning tasks, the evaluation metric one seeks to optimize is not explicitly available in closed-form. This is true for metrics that are evaluated through live experiments or by querying human users \cite{tamburrelli2014towards, hiranandani2018eliciting}, or that require access to private or legally protected data \cite{awasthi+21},
and hence cannot be written as an explicit training objective. This is also the case when the learner only has access to data with skewed training distribution or labels with heteroscedastic noise~\cite{huang2019addressing,jiang2020optimizing}, and hence cannot directly optimize the metric on the training set despite knowing its mathematical form. 

These problems can be framed as black-box learning tasks, where the goal is to optimize an unknown classification metric on a large (possibly noisy) training data, given  access to evaluations of the metric on a small, clean validation sample \citep{jiang2020optimizing}. 
Our high-level approach to these learning tasks is to adaptively assign weights to the training examples, 
so that the resulting weighted training objective closely approximates the black-box metric on the validation sample. We then construct a classifier by using the example weights to  post-shift a class-probability estimator pre-trained on the training set. This results in an efficient, iterative approach that does not require any re-training. 

Indeed, example weighting strategies have been  widely used  to both optimize metrics and to correct for distribution shift, but prior works either handle specialized forms of metric
or data noise~\cite{sugiyama2008direct, natarajan2013learning, patrini2017making}, formulate the example-weight learning task as a difficult non-convex problem that is hard to analyze \cite{ren2018learning,zhao2019metric}, or employ an expensive surrogate re-weighting strategy that comes with limited statistical guarantees~\cite{jiang2020optimizing}. 
In contrast, we propose a simple and effective approach to optimize a general black-box metric (that is  a function of the confusion matrix)  and provide a rigorous statistical analysis. 

A key element of our approach is eliciting the weight coefficients by probing the  black-box metric at few select classifiers and solving a system of linear equations matching the weighted training errors to the validation metric. 
We choose the ``probing'' classifiers so that the linear system is well-conditioned, for which we provide both theoretically-grounded options and practically efficient variants. This weight elicitation procedure is then used as a subroutine to iteratively construct the final plug-in classifier.

The contributions in this chapter are as follows:

\bitemize
\item We provide a method for eliciting example weights for linear black-box metrics 
(Section \ref{sec:example-weights}). 
\item We  use  this procedure to iteratively learn a plug-in classifier for general black-box metrics (Section \ref{sec:algorithms}). 
\item We provide theoretical guarantees for metrics that are concave functions of the confusion matrix under distributional assumptions
(Section \ref{sec:theory}). 
\item We experimentally show that our approach is competitive with (or better than) the state-of-the-art methods for tackling label noise in CIFAR-10~\cite{krizhevsky2009learning} and domain shift in Adience~\cite{eidinger2014age}, and optimizing with proxy labels and a black-box fairness metric on Adult~\cite{Dua:2019} (Section \ref{sec:experiments}). 
\eitemize
All the proofs in this chapter are provided in Appendix~\ref{apx:blackbox}.

\textbf{Notations:} 
$\onehot(j) \in \{0,1\}^k$  returns the one-hot encoding of  $j \in [k]$. In this chapter, the $\ell_2$ norm of a vector is denoted by $\|\cdot\|$. 

%% file: blackbox/setup.tex
\section{Problem Setup}
\label{sec:setup}
We consider a standard multiclass setup with an instance space  $\X \subseteq \R^d$ and a label space $\Y = [k]$.
We wish to learn a randomized multiclass classifier $h: \X \> \Delta_k$ that for any input $x \in \X$ predicts a distribution $h(x) \in \Delta_k$ over the $k$ classes. We will also consider deterministic classifiers $h: \X \>[k]$ which map an instance $x$ to one of $k$ classes.

\textbf{Evaluation Metrics.} Let $\Dtrue$ denote the underlying data distribution over $\X \times \Y$. 
We will evaluate the performance of a classifier $h$ on $\Dtrue$ using an evaluation metric $\perf^\Dtrue[h]$, with higher values indicating better performance. Our goal is to  learn a classifier $h$ 
that maximizes this evaluation measure:
\begin{equation}
\textstyle
\max_{h}\,\perf^{\Dtrue}[h].
\label{eq:unconsrained}
\end{equation}

We will  focus on metrics $\perf^\Dtrue$ that can be written in terms of  classifier's  confusion matrix $\C[h] \in [0,1]^{k\times k}$, where the $i,j$-th entry is the probability that the true label is $i$ and the randomized classifier $h$ predicts $j$:
\bequation
C^{\Dtrue}_{ij}[h] = \E_{(x, y) \sim \Dtrue}\left[\1(y = i)h_j(x)\right].
\eequation

The performance of the classifier can then be evaluated using a (possibly unknown) function $\psi: [0,1]^{k\times k}\>\R_+$ of the confusion matrix: 
\begin{equation}
\perf^D[h] = \psi(\C^D[h]).
\label{eq:perf-conf}
\end{equation} 
Several common classification metrics take this form, including typical linear metrics $\psi(\C) \,=\, \sum_{ij}L_{ij}\,C_{ij}$ for some reward matrix $\L \in \R_+^{k\times k}$,  the  F-measure {$\psi(\C) \,=\,\sum_i \frac{2C_{ii}}{\sum_j C_{ij} + \sum_j C_{ji}}$} \cite{Lewis95}, and the G-mean {$\psi(\C) = \big(\prod_i \big({C_{ii}}/\sum_j C_{ij}\big)\big)^{1/k}$} \cite{Daskalaki+06}.

We consider settings where the learner has query-access to the evaluation metric $\perf^D$, i.e., can evaluate the metric for any given classifier $h$ but cannot directly write out the metric as an explicit mathematical objective. This happens when the metric is truly a black-box function, i.e., $\psi$ is unknown, or when $\psi$ is known, but we have access to only a noisy version of the distribution $D$ needed to compute the metric. 

\textbf{Noisy Training Distribution.} 
For learning a classifier, we assume access to a large  sample $S^{\tr}$ of $n^{\tr}$ examples drawn from a 
distribution $\Dshift$, which we will refer to as the ``training'' distribution. The training distribution $\Dshift$ may be the same as the true  distribution $\Dtrue$, or 
may differ from the true  distribution $\Dtrue$ 
in  the  feature distribution $\P(x)$, the conditional label distribution $\P(y|x)$, or both. 
 We also assume access to a  smaller sample $S^{\val}$ of
$n^{\val}$ examples
drawn from the true distribution
$\Dtrue$. 
We will refer to the sample $S^{\tr}$ 
as the ``training'' sample, and the smaller sample $S^{\val}$ 
as the ``validation'' sample. We seek to solve \eqref{eq:unconsrained} using both these  samples. 

The following are some examples of noisy training distributions in the literature:
\bexample[Independent label noise (ILN) \cite{natarajan2013learning, patrini2017making}]
\label{ex:iln}
\emph{
The distribution  $\Dshift$ draws an example $(x,y)$ from $\Dtrue$, and randomly flips $y$ to $\tilde{y}$ with probability 
$\P(\tilde{y}|y)$, independent of the instance $x$.
}
\eexample
\bexample[Cluster-dependent label noise (CDLN) \cite{wang2020fair}]
\emph{
Suppose each  $x$ belongs to one of $m$ disjoint clusters $g(x) \in [m]$. The distribution  $\Dshift$ draws  $(x,y)$ from $\Dtrue$ and randomly flips $y$ to $\tilde{y}$ with probability $\P(\tilde{y}|y,g(x))$. 
}
\eexample
\bexample[Instance-dependent label noise (IDLN) \cite{menon2018learning}] 
\emph{
$\Dshift$ draws  $(x,y)$ from $\Dtrue$ and randomly flips $y$ to $\tilde{y}$ with probability $\P(\tilde{y}|y,x)$, which may depend on  $x$.
}
\label{ex:instnoise}
\eexample
\bexample[Domain shift (DS) \cite{sugiyama2008direct}] 
\label{ex:ds}
\emph{
$\Dshift$ draws $\tilde{x}$ according to a distribution $\P^{\Dshift}({x})$ different from $\P^{\Dtrue}({x})$, but draws $y$ from the true conditional $\P^{\Dtrue}(y|\tilde{x})$.
}
\eexample

\begin{table}[t]
    \centering
    \caption{Example weights $\W: \X \> \R_+^{k\times k}$ for linear metric {\small $\perf^\Dtrue[h]=\langle\L, \C^\Dtrue[h]\rangle$} under the  noise models in Exmp.\ \ref{ex:iln}--\ref{ex:ds}, where $W_{ij}(x)$ is the weight on entry $C_{ij}$. 
    In Sec.\ \ref{sec:example-weights}--\ref{sec:algorithms}, we consider metrics that are functions of the diagonal confusion  entries alone (i.e.\ $\L$ and $\T$ are diagonal), and handle  general  metrics in  Appendix \ref{app:linear-gen}.}
    \begin{tabular}{ccc}
    \hline
         Model &
         Noise Transition Matrix & Correction Weights\\
        \hline
        ILN 
        & $T_{ij} = \P(\tilde{y}=j|y=i)$& 
        $\W(x) = \L \odot \T^{-1}$
        \\
        CDLN 
        & $T^{[m]}_{ij} = \P(\tilde{y}=j|y=i, g(x)=m)$& 
        $\W(x) = 
        \L \odot (\T^{[g(x)]})^{-1}$
        \\
        IDLN 
        & $T_{ij}(x) = \P(\tilde{y}=j|y=i, x)$& 
        $\W(x) = \L \odot (\T(x))^{-1}$
        \\
        DS & 
        - &  $W_{ij}(x) = {\P^\Dtrue(x)}/{\P^\Dshift(x)}, \forall i,j$\\
        \hline
    \end{tabular}
    \label{tab:correction-weights}
\end{table}

Our approach is to learn example weights on the training sample $S^\tr$, so that the resulting weighted empirical objective (locally, if not globally) approximates an estimate of the metric $\perf^\Dtrue$ on the validation sample $S^\val$. 
For ease of presentation, we will assume that the metrics 
only depend on the diagonal entries of the confusion matrix, i.e., $C_{ii}$'s. In Appendix \ref{app:linear-gen}, we elaborate how our ideas can be extended to handle metrics that depend on the entire confusion matrix. 

While our approach uses randomized classifiers, in practice one can replace them with similarly performing deterministic classifiers using, e.g., the techniques of \cite{cotter19stochastic}.
In what follows, we will need the empirical confusion matrix on the validation set $\hat{\C}^\val[h]$, where  \bequation
\hat{C}_{ij}^{\val}[h] = \frac{1}{n^\val}\sum_{(x, y) \in S^\val}\1(y=i)h_j(x).
\eequation

%% file: blackbox/methods.tex
\section{Example Weighting for  Linear Metrics}
\label{sec:example-weights}

We first describe our example weighting strategy for 
linear functions of the diagonal entries of the confusion matrix, which is given by:
\begin{equation}
  \textstyle \perf^\Dtrue[h] \,=\, 
\sum_{i}\beta_i\,C^D_{ii}[h]
\label{eq:linearmetric}
\end{equation}
for 
some (unknown)
weights $\beta_1,\ldots,\beta_k$. In the next section, we will discuss how to use this procedure as a subroutine to handle more complex metrics.

\subsection{Modeling Example Weights} 
We define an example weighting function $\W: \X  \> \R^{k}_+$ which associates 
 $k$ \emph{correction weights} $[W_{i}(x)]_{i=1}^k$ with each example $x$ 
 so that:
\begin{equation}
\textstyle
\E_{(x, y) \sim \Dshift}\Big[\sum_{i} W_{i}(x)\,\1(y = i)h_i(x)\Big] \,\approx\, 
\perf^\Dtrue[h],  \; \forall \; h.\hspace{-2pt}
\label{eq:example-weights}
\end{equation}
Indeed for the noise models in Examples \ref{ex:iln}--\ref{ex:ds}, there exist weighting functions $\W$ for which the above holds with equality.  Table \ref{tab:correction-weights} shows the form of the weighting function for general linear metrics.

Ideally, the weighting function $\W$ assigns $k$ independent weights for each example $x \in \X$. However, in practice, we estimate $\perf^\Dtrue$ using a small validation sample $S^\val \sim \Dtrue$.  
So to avoid having the example weights over-fit to the validation sample, 
we restrict the flexibility of $\W$ and set it to a  weighted sum of $L$ basis functions $\phi^\ell: \X \> [0,1]$:
\begin{equation}
\textstyle
W_{i}(x) \,=\, \sum_{\ell=1}^L \alpha^{\ell}_{i}\phi^\ell(x),
\label{eq:weighting}
\end{equation}
where $\alpha^{\ell}_{i} \in \R$ is the coefficient associated with basis function $\phi^\ell$ and diagonal confusion  entry $(i,i)$. 

In practice, the basis functions can be as simple as a  partitioning of the instance space into $L$ clusters, i.e.,:
\begin{equation}
\phi^\ell(x) =  \1(g(x) = \ell),
\label{eq:hardlcuster}
\end{equation}
for a clustering function $g: \X\>[L]$,
or may define a more complicated soft clustering using, e.g., radial basis functions \cite{sugiyama2008direct} with centers
$x^\ell$ and width $\sigma$:
\begin{equation}
\phi^\ell(x) =  \text{exp}\left( -\Vert x - x^\ell \Vert/2\sigma^2 \right).
\label{eq:softlcuster}
\end{equation}

\subsection{$\phi$-transformed Confusions}
Expanding the weighting function in \eqref{eq:example-weights} gives us:
\begin{equation}
\sum_{\ell=1}^L\sum_{i=1}^k
\alpha^{\ell}_{i}\,
\underbrace{
\E_{(x, y) \sim \Dshift}\big[ \phi^\ell(x)\,\1(y=i)h_i(x) \big]}_{\Phi_i^{\Dshift, \ell}[h]} \,\approx\, \perf^\Dtrue[h],  \; \forall \; h,
\end{equation}
where $\boPhi^{\Dshift, \ell}[h] \in [0,1]^{k}$ can be seen as a $\phi$-transformed  confusion matrix  for the training distribution $\Dshift$. 
For example, if one had only one basis function 
$\phi^1(x) = 1,\forall x$, then $\Phi_{i}^{\Dshift, 1}[h] = \E_{(x, y) \sim \Dshift}\big[ \1(y=i)h_i(x)\big]$ gives the standard confusion entries for the training distribution. If the basis functions divides the data into $L$  clusters, as in \eqref{eq:hardlcuster}, then  $\Phi_{i}^{\Dshift, \ell}[h] = \E_{(x, y) \sim \Dshift}\big[\1(g(x)=\ell, y=i)h_i(x)\big]$  gives the training confusion entries evaluated on examples from cluster $\ell$. 
We can thus re-write equation~\eqref{eq:example-weights} as a weighted combination of the $\Phi$-confusion entries:
\begin{equation}
\sum_{\ell=1}^L\sum_{i=1}^k
\alpha^{\ell}_i
\Phi^{\Dshift, \ell}_i[h] \,\approx\, 
\perf^D[h], \forall h.
\label{eq:example-weights-reduced}
\end{equation}
\vspace{-1cm}
\subsection{Eliciting Weight Coefficients $\balpha$ -- The Metric Elicitation Step}
\label{sec:estimating-weights}
We  next discuss how to estimate the weighting function coefficients $\alpha^\ell_{i}$'s from the training sample $S^\tr$ and validation sample $S^\val$. 
Notice that \eqref{eq:example-weights-reduced} gives a relationship between statistics $\boPhi^{\mu,\ell}$'s computed on the training distribution $\Dshift$, and the evaluation metric of interest computed on the true distribution $\Dtrue$. Moreover, for a fixed classifier $h$, the left-hand side is \textit{linear} in the unknown coefficients
$\balpha = [\alpha_1^1, \ldots, \alpha_1^L, \ldots, \alpha_k^1, \ldots, \alpha_k^L]\in \R^{Lk}$. Thus, this step is similar to eliciting linear metrics (Chapters~\ref{chp:binary},\ref{chp:multiclass})  in the presence of an oracle which provides absolute quality feedback.

We therefore probe the metric $\hat{\perf}^\val$ at $Lm$ different classifiers $h^{1,1}, \ldots, h^{1,k},\ldots,h^{L,1},\ldots, h^{L,k}$, which results in a set of $Lk$ linear equations of the form in \eqref{eq:example-weights-reduced}: 
\begin{align}
\textstyle\sum_{\ell,i}\alpha^{\ell}_i\,
\hat{\Phi}^{\tr, \ell}_i[h^{1,1}]&=
\hat{\perf}^\val[h^{1,1}],
\nonumber
\\
&\vdots\label{eq:system-of-equations-emp}\\
\textstyle\sum_{\ell,i}\alpha^{\ell}_i\,
\hat{\Phi}^{\tr, \ell}_i[h^{L,k}]&=
\hat{\perf}^\val[h^{L,m}],\nonumber
\end{align}
where $\hat{\Phi}_{i}^{\tr, \ell}[h] = \frac{1}{n^\tr}\sum_{(x, y) \in S^\tr}\phi^\ell(x)\,\1(y=i)h_i(x)$ is evaluated on the training sample and the metric $\hat{\perf}^\val[h] = \sum_{i}\beta_i\,\hat{C}^\val_{ii}[h]$ 
is evaluated on the validation sample.

More formally, let $\hat{\bSigma} \in \R^{Lk\times Lk}$ and $\hat{\bcE} \in \R^{Lk}$ denote the left-hand and right-hand side observations in \eqref{eq:system-of-equations-emp}, i.e.,:
\begin{align*}
\textstyle
\hat{\Sigma}_{(\ell,i), (\ell',i')} &\,=\,
\frac{1}{n^\tr}\sum_{(x,y)\in S^\tr}\phi^{\ell'}(x)\1(y=i')h^{\ell,i}_{i'}(x), \\ \nonumber
\hat{\perf}_{(\ell,i)} &\,=\, \hat{\perf}^\val[h^{\ell,i}]. \numberthis 
\label{eq:perf-emp}
\end{align*}
Then the weight coefficients are given by  $\hat{\balpha} = \hat{\bSigma}^{-1}\hat{\bcE}$.

\subsection{Choosing the Probing Classifiers $h^{1,1},\ldots,h^{L,k}$}
\label{subsec:probing-classifier}
We will have to choose the $Lk$ probing classifiers so that $\hat{\bSigma}$ 
is well-conditioned. One way to do this is to choose  the  classifiers so that  $\hat{\bSigma}$ has a high value on the diagonal entries and a low value on the off-diagonals, i.e.\ choose each 
classifier $h^{\ell,i}$ to evaluate to a high value on $\hat{\Phi}^{\tr,\ell}_i[h]$ and a low value on $\hat{\Phi}^{\tr,\ell'}_{i'}[h], \, \;\forall \; (\ell',i') \ne (\ell,i)$. 
This can be framed as the following constraint satisfaction problem on $S^\tr$: 
\begin{center}
For $h^{\ell,i}$ pick $h \in \H$ such that:
\begin{align}
\hat{\Phi}^{\tr,\ell}_i[h] \geq \gamma,~\text{and}~
\hat{\Phi}^{\tr,\ell'}_{i'}[h] \leq \omega, \forall (\ell',i') \ne (\ell,i),
    \label{eq:con-opt}
\end{align}
\end{center}
for 
some $\gamma > \omega > 0$ and a sufficiently flexible hypothesis class $\H$ for which the constraints are feasible. These problems can generally be solved by formulating a constrained classification problem \cite{cotter2019optimization,narasimhan2018learning}. We show in Appendix \ref{app:con-opt} that this problem is feasible and can be efficiently solved for a range of settings.

\begin{figure}
\begin{algorithm}[H]
\caption{\hspace{-0.075cm}\textbf{: ElicitWeights} for Diagonal Linear Metrics}\label{algo:weight-coeff}
\begin{algorithmic}[1]
\STATE \textbf{Input:} $\hat{\perf}^\val$, Basis functions $\phi^1, \ldots, \phi^L: \X \> [0,1]$, 
Training set $S^\tr \sim \Dshift$, Val. set $S^\val \sim \Dtrue$, $\bar{h}$, ${\epsilon}$, $\H$, $\gamma, \omega$
\STATE \textbf{If} 
\textit{fixed classifier}:
\STATE ~~~Choose $h^{\ell,i}(x) = 
\epsilon\phi^\ell(x)\,e^i(x) + (1 - \epsilon\phi^\ell(x))\,\bar{h}(x)$ 
\STATE \textbf{Else:} 
\STATE ~~~$\bar{\H} = \{\tau h + (1-\tau)\bar{h}\,|\, h \in \H, \tau \in [0, \epsilon]\}$
\STATE ~~~Pick $h^{\ell,i} \in \bar{\H}$ to satisfy \eqref{eq:con-opt} 
with slack $\gamma,\omega, \forall (\ell,i)$
\STATE Compute $\hat{\bSigma}$ and $\hat{\bcE}$ using \eqref{eq:perf-emp} with metric $\hat{\perf}^\val$
\STATE \textbf{Output:} $\hat{\balpha} = \hat{\bSigma}^{-1}\hat{\bcE}$
\end{algorithmic}
\end{algorithm}
\vspace{-20pt}
\begin{algorithm}[H]
\caption{\hspace{-0.075cm}\textbf{:} {\textbf{P}lug-\textbf{i}n with \textbf{E}licited \textbf{W}eights} (\textbf{PI-EW}) for Diagonal Linear Metrics
}
\label{algo:linear-metrics}
\begin{algorithmic}[1]
\STATE \textbf{Input:}  $\hat{\perf}^\val$, Basis functions $\phi^1, \ldots, \phi^L: \X \> [0, 1]$, Class probability model $\hat{\eta}^\tr: \X \> \Delta_k$ for  $\Dshift$, 
Training set $S^\tr \sim \Dshift$, Validation set $S^\val \sim \Dtrue$, $\bar{h}$, $\epsilon$
\STATE $\widehat{\balpha} = \textbf{ElicitWeights}(\hat{\perf}^\val, 
\phi^1, \ldots, \phi^L, S^\tr, S^\val, \bar{h}, \epsilon)$ 
\STATE Example-weights:
$
\widehat{W}_{i}(x) \,=\, \sum_{\ell=1}^L \widehat{\alpha}^{\ell}_{i}\phi^\ell_{i}(x)
$
\STATE Plug-in:
$
\widehat{h}(x) \,\in\, \argmax_{i \in [k]} \widehat{W}_{i}(x)\hat{\eta}^\tr_i(x)
$
\STATE \textbf{Output:} $\widehat{h}$
\end{algorithmic}
\end{algorithm}
\vspace{-18pt}
\end{figure}

In practice, we do not explicitly solve \eqref{eq:con-opt} over a hypothesis class $\H$.  Instead, a simpler and surprisingly effective strategy is to 
set the probing classifiers to 
\textit{trivial} classifiers that predict the same class on all (or a subset of) examples. 
To build intuition for why this is a good idea, 
consider a simple setting with only one basis function $\phi^1(x) = 1, \forall x$, where the $\phi$-confusions $\hat{\Phi}^{\tr,1}_i[h] = \frac{1}{n^\tr}\sum_{(x,y)\in S^\tr}\1(y=i)h_i(x)$ are the standard confusion  entries on the training set. In this case, a trivial classifier $e^i(x) = \onehot(i), \forall x$, which predicts class $i$ on all examples, yields the highest value for $\hat{\Phi}^{\tr,1}_i$ and 0 for all other $\hat{\Phi}^{\tr,1}_j, \forall j \ne i$. In fact, in our experiments,  we set the probing classifier $h^{1,i}$ to a randomized combination of $e^i$ and some fixed base classifier $\bar{h}$:
\bequation
h^{1,i}(x) = \epsilon e^i(x) + (1-\epsilon)\bar{h}(x),
\eequation
for large enough $\epsilon$ so that $\hat{\bSigma}$ is well-conditioned.

Similarly, if the basis functions divide the data into $L$  clusters (as in \eqref{eq:hardlcuster}), then we can 
 randomize between $\bar{h}$ and a {trivial} classifier that predicts a particular class $i$ on all examples assigned to the cluster $\ell \in [L]$. The confusion matrix for the resulting classifiers will have higher values than $\bar{h}$ on the $(\ell,i)$-th diagonal entry and a lower value on  other entries. These classifiers can be succinctly written as:
\begin{equation}
h^{\ell,i}(x) = \epsilon\phi^\ell(x)e^i(x) + (1-\epsilon\phi^\ell(x))\bar{h}
\label{eq:trivial-classifiers}
\end{equation}
where we again tune $\epsilon$ to make sure that the resulting $\hat{\bSigma}$ is well-conditioned. This choice of the probing classifiers also works well in practice for general basis functions $\phi^\ell$'s.

Algorithm \ref{algo:weight-coeff} summarizes the weight elicitation procedure, where  the probing classifiers are either constructed by solving the constrained satisfaction problem \eqref{eq:con-opt} or set to the ``fixed'' classifiers in \eqref{eq:trivial-classifiers}. In both cases, the algorithm takes a base classifier $\bar{h}$ and the parameter $\epsilon$ as input, where $\epsilon$ controls the extent to which $\bar{h}$ is perturbed to construct the probing classifiers. This radius parameter $\epsilon$ restricts the probing classifiers to a neighborhood around $\bar{h}$ and will prove handy in the algorithm we develop 
in Section~\ref{ssec:iterativeFW}.

\section{Plug-In Based Algorithms}
\label{sec:algorithms}

\begin{figure}[t]
    \centering
    \includegraphics[scale=1.1]{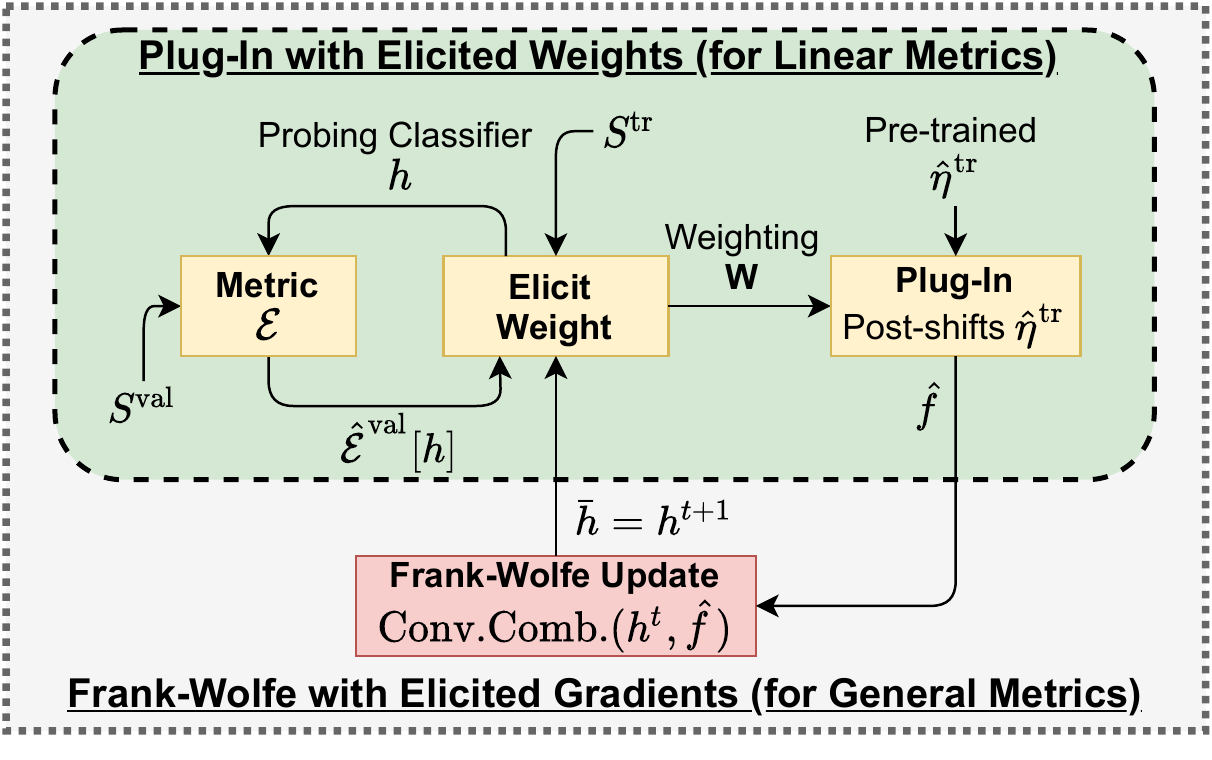}
    \caption{Overview of our apporach.}
    \label{fig:overview}
\end{figure}

Having elicited the weight coefficients $\balpha$, we now seek to learn a classifier that optimizes the left hand side of~\eqref{eq:example-weights-reduced}. 
We do this via the \textit{plug-in} approach: first \textit{pre-train} a model $\hat{\eta}^\tr: \X \> \Delta_k$  on the noisy training distribution $\Dshift$ to estimate the conditional class probabilities $\hat{\eta}^\tr_i(x) \approx \P^\mu(y=i|x)$, and then apply the correction weights to \emph{post-shift} 
$\hat{\eta}^\tr$. 

\begin{figure}
\begin{algorithm}[H]
\caption{\hspace{-0.075cm}\textbf{:} \textbf{F}rank-\textbf{W}olfe with \textbf{E}licited \textbf{G}radients (\textbf{FW-EG}) for General Diagonal Metrics (also depicted in Fig.~\ref{fig:overview})}\label{algo:FW}
\begin{algorithmic}[1]
\STATE \textbf{Input:} $\hat{\perf}^\val$, Basis functions $\phi^1, \ldots, \phi^L: \X \> [0,1]$, Pre-trained $\hat{\eta}^\tr: \X \> \Delta_k$,
$S^\tr \sim \Dshift$, 
$S^\val \sim \Dtrue$, $T$, $\epsilon$
\STATE Initialize classifier $h^0$ and $\c^0 = \diag(\widehat{\C}^\val[h^0])$
\STATE \textbf{For} $t =  0$ \textbf{to} $T-1$ \textbf{do}
\STATE ~~~\textbf{if} $\perf^\Dtrue[h] = \psi(C_{11}^D[h],\ldots,C_{kk}^D[h])$ for known $\psi$:
\STATE ~~~~~~~$\bbeta^{t}\,=\, \nabla\psi(\c^{t})$
\STATE ~~~~~~~$\hat{\perf}^\lin[h]\,=\, \sum_i \beta^t_i \hat{C}^\val_{ii}[h]$
\STATE ~~~\textbf{else}
\STATE ~~~~~~~$\hat{\perf}^\lin[h]= \hat{\perf}^\val[h]$ \hspace{7cm}\COMMENT{small $\epsilon$ recommendeded}
\STATE ~~~$\widehat{f} = \text{\textbf{PI-EW}}(\hat{\perf}^\lin, \phi^1,..., \phi^L, \hat{\eta}^\tr, S^\tr, S^\val, h^{t}, \epsilon)$
\STATE ~~~$\tilde{\c} = \diag(\widehat{\C}^\val[\widehat{f}])$
\STATE ~~~${h}^{t+1} = \big(1-\frac{2}{t+1}\big) {h}^{t} + \frac{2}{t+1} \onehot(\widehat{f})$
\STATE ~~~${\c}^{t+1} = \big(1-\frac{2}{t+1}\big) {\c}^{t} + \frac{2}{t+1}\tilde{\c}$
\STATE \textbf{End For}
\STATE \textbf{Output:} $\hat{h} = h^T$
\end{algorithmic}
\end{algorithm}
\vspace{-18pt}
\end{figure}

\subsection{Plug-in  Algorithm for Linear Metrics}
We first describe our approach for (diagonal) linear metrics $\perf^\Dtrue[h] \,=\,  \sum_{i}\beta_i\,C^D_{ii}[h]$ 
in 
Algorithm \ref{algo:linear-metrics}.
Given the correction weights $\hat{\W}: \X \> \R_+^k$, 
we seek to maximize the following weighted objective on the training distribution: 
\bequation
\textstyle
\max_{h}\,\E_{(x, y) \sim \Dshift}\left[\sum_{i} \hat{W}_{i}(x)\,\1(y = i)h_i(x)\right].
\eequation
This is a standard example-weighted learning problem, for which the following \emph{plug-in} (also known as \emph{post-shift}) classifier is a consistent estimator \cite{narasimhan2015consistent,yang2020fairness}:
\bequation
\widehat{h}(x) \,\in\, \argmax_{i \in [k]} \hat{W}_{i}(x)\,\hat{\eta}^\tr_i(x).
\eequation

\subsection{Iterative Algorithm for General Metrics}
\label{ssec:iterativeFW}

To optimize generic non-linear metrics of the form $\perf^\Dtrue[h] = \psi(C^D_{11}[h], \ldots,C^D_{kk}[h])$ for $\psi: [0,1]^k\>\R_+$, we  apply Algorithm \ref{algo:linear-metrics} iteratively. We consider both cases where $\psi$ is unknown, and where $\psi$ is known, but needs to be optimized using the noisy distribution $\Dshift$. 
The idea is to first elicit local linear approximations to $\psi$ and to then learn plug-in classifiers for the resulting linear metrics in each iteration. 

Specifically, following Narasimhan et al.\cite{narasimhan2015consistent},
we derive our algorithm from the classical Frank-Wolfe  method \cite{Jaggi13} for maximizing a smooth concave function $\psi(\c)$ over a convex set $\cC \subseteq \R^m$. In our case, $\cC$ is the set of confusion matrices $\C^\Dtrue[h]$ achieved by any classifier $h$, and is convex when we allow randomized classifiers (see Lemma~\ref{lem:C-convexity}, Appendix \ref{app:proof-FW}). The algorithm maintains iterates $\c^t$, and at each step, maximizes a linear approximation to $\psi$ at $\c^t$: $\tilde{\c} \,\in\, \argmax_{\c \in \cC}\langle \nabla\psi(\c^{t}), \c\rangle$. The next iterate $\c^{t+1}$ is then a convex combination of $\c^{t}$ and the current solution $\tilde{\c}$.

In Algorithm \ref{algo:FW}, we outline an adaptation of this Frank-Wolfe algorithm to our setting, where
  we maintain a classifier $h^t$ and an estimate of the diagonal confusion entries $\c^t$ from the validation sample $S^\val$. At each step, we linearize $\psi$ using $\hat{\perf}^\lin[h] = \sum_{i} \beta^t_i \hat{C}_{ii}^\val[h]$, where $\bbeta^t = \nabla \psi(\c^t)$, and invoke the plug-in method in Algorithm \ref{algo:linear-metrics} to optimize the linear approximation $\hat{\perf}^\lin$. 
When the mathematical form of $\psi$ is known, one can directly compute the gradient $\bbeta^t$. When it is not known, we can simply set $\hat{\perf}^\lin[h] = \hat{\perf}^\val[h]$, but restrict the  weight elicitation routine (Algorithm \ref{algo:weight-coeff}) 
to choose its probing classifiers $h^{\ell,i}$'s from a small neighborhood  around the current classifier $h^t$ (in which $\psi$ is effectively linear). This can be done by passing $\bar{h}=h^t$ to the weight elicitation routine, and setting the radius ${\epsilon}$  to a small value. 
    
    Each call to Algorithm \ref{algo:linear-metrics} 
    uses the training and validation set
    to elicit example weights for a local linear approximation to $\psi$, and uses the  weights to 
    construct a plug-in classifier. 
    The final  output is a randomized combination of the plug-in classifiers  from each step. 
     Note that Algorithm~\ref{algo:FW} runs efficiently for reasonable values of $L$ and $k$. 
    Indeed the runtime is almost always dominated by the pre-training of the base model $\hat\eta^\tr$, with the time taken to elicit the weights (e.g.\ using~\eqref{eq:trivial-classifiers}) being relatively inexpensive (see Appendix \ref{app:run-time}).

%% file: blackbox/theory.tex
\section{Theoretical Guarantees}
\label{sec:theory}
We provide theoretical guarantees for the weight elicitation procedure and the plug-in methods in Algorithms \ref{algo:weight-coeff}--\ref{algo:FW}. 
\bassumption
\label{asp:alpha-star}
The distributions $\Dtrue$ and $\Dshift$ are such that
for any linear metric $\perf^\Dtrue[h] = \sum_{i}\beta_i C_{ii}[h]$, with $\|\bbeta\| \leq 1$, 
$\exists \bar{\balpha} \in \R^{Lk}$ s.t. $\left|\sum_{\ell,i}
\bar{\alpha}^{\ell}_i
\Phi^{\Dshift, \ell}_i[h] -
\perf^\Dtrue[h]\right| \,\leq\, \nu, \forall h$ and $\|\bar{\balpha}\|_1 \leq B$, for some $\nu \in [0,1)$ and $B>0$.
\eassumption
The assumption states that our choice of basis functions $\phi^1,\ldots,\phi^L$ are such that, any linear metric on  $\Dtrue$ can be approximated (up to a slack $\nu$) by a weighting $W_i(x) = \sum_{\ell}\bar{\alpha}^\ell_i\phi^\ell(x)$ of the  training examples  from $\Dshift$.
The existence of such a weighting function depends on how well the basis functions capture the underlying distribution shift. 
Indeed, the assumption holds for some common settings in Table \ref{tab:correction-weights}, e.g., when the noise  transition $\T$ is diagonal (Appendix \ref{app:linear-gen} handles a general $\T$), and the basis functions are set to $\phi^1(x)=1,\forall x,$ for the IDLN setting, and $\phi^\ell(x) = \1(g(x)=\ell), \forall x,$ for the CDLN setting.

We  analyze the coefficients $\hat{\balpha}$ elicited by Algorithm \ref{algo:weight-coeff} when the probing classifiers $h^{\ell,i}$ are chosen to satisfy \eqref{eq:con-opt}. 
In Appendix  \ref{app:norm-sigma-bound}, we provide an analysis  when the probing classifiers $h^{\ell,i}$ are set to the fixed choices in \eqref{eq:trivial-classifiers}. 

\btheorem[\hspace{-1pt}\textbf{Error bound on elicited weights}]
\label{thm:alpha-diagonal-linear-conopt}
Let $\gamma, \omega > 0$ be such that the constraints in \eqref{eq:con-opt}  are feasible for hypothesis class $\bar{\H}$, for all $\ell, i$. Suppose Algorithm \ref{algo:weight-coeff} chooses each
 classifier $h^{\ell,i}$ to  satisfy \eqref{eq:con-opt}, with $\perf^D[h^{\ell,i}] \in [c, 1], \forall \ell, i$, for some $c>0$. 
Let $\bar{\alpha}$ be defined as in Assumption \ref{asp:alpha-star}. 
Suppose $\gamma > 2\sqrt{2}Lm\omega$ and $n^\tr \geq \frac{L^2m\log(Lm|\H|/\delta)}{(\frac{\gamma}{2} - \sqrt{2}Lm\omega)^2}.$
Fix $\delta\in (0,1)$. Then w.p.\  $\geq 1 - \delta$ over draws of $S^\tr$ and $S^\val$ from $\Dshift$ and $\Dtrue$ resp., 
the coefficients $\hat{\balpha}$ output by Algorithm \ref{algo:weight-coeff} satisfies:
\vspace{-5pt}
\begin{equation}
\|\hat{\balpha} - \bar{\balpha}\| \,\leq\,
\mathcal{O}\bigg(
\frac{Lk}{\gamma^2}\bigg(\sqrt{\frac{L\log(\textstyle \frac{Lk|\H|}{\delta})}{n^\tr}} + 
 \sqrt{\frac{L\log(\textstyle \frac{Lk}{\delta})}{c^2 n^\val}}\bigg) +  \frac{\nu\sqrt{Lk}}{\gamma}\bigg),
\end{equation}
where the term $|\H|$ can be replaced by a measure of capacity of the hypothesis class $\H$.
\etheorem

Because the probing classifiers are chosen using the training set alone, it is only the sampling errors from the training set that depend on the complexity of $\H$, and not those from the validation set. 
This suggests robustness of our approach to  a small validation set as long as the training set is sufficiently large and the number of basis functions is reasonably small. 

For the iterative plug-in method in Algorithm \ref{algo:FW}, we bound the gap between the metric value $\perf^D[\hat{h}]$ for the output classifier $\hat{h}$ on the true distribution $\Dtrue$, and the optimal value. We handle the case where the function $\psi$ is \textit{known} and its gradient $\nabla\psi$ can be computed in closed-form. The more general case of an unknown $\psi$ is handled in Appendix \ref{app:complex-unknown}.
The above bound depends on the gap between the estimated  class probabilities $\hat{\eta}_i^\tr(x)$ for the training distribution and true class probabilities $\eta_i^\tr(x) = \P(y=i|x)$, as well as the quality of the coefficients $\hat{\balpha}$ provided by the weight estimation subroutine, as measured by $\kappa(\cdot)$. 
One can substitute $\kappa(\cdot)$ with, e.g., the error bound provided in Theorem \ref{thm:alpha-diagonal-linear-conopt}. 

\btheorem[\textbf{Error Bound for FW-EG}]
\label{thm:iterative-plugin}
Let $\perf^\Dtrue[h] = \psi(C^\Dtrue_{11}[h],\ldots, C^\Dtrue_{kk}[h])$ for a \emph{known} concave function $\psi: [0,1]^k \>\R_+$, which is $Q$-Lipschitz and $\lambda$-smooth.  Fix $\delta \in (0, 1)$.
Suppose Assumption \ref{asp:alpha-star} holds, and for any linear metric $\sum_i\beta_i C^D_{ii}[h]$, whose associated weight coefficients is $\bar{\balpha}$ with $\|\bar{\balpha}\| \leq B$, 
 w.p. $\geq 1-\delta$ over draw of $S^\tr$ and $S^\val$, the weight estimation routine in Alg.\ \ref{algo:weight-coeff} outputs coefficients $\hat{\balpha}$ with
 $\|\hat{\balpha} -\bar{\balpha}\| \leq 
 \kappa(\delta, n^\tr, n^\val)
 $, for some function $\kappa(\cdot) > 0$. Let $B' = B + \sqrt{Lk}\,\kappa(\delta/T, n^\tr, n^\val).$
 Then w.p.\  $\geq 1 - \delta$ over draws of $S^\tr$ and $S^\val$ from $\Dtrue$ and $\Dshift$ resp., the classifier $\hat{h}$ output by Algorithm \ref{algo:FW} after $T$ iterations satisfies:
\begin{align*}
\max_{h}\perf^D[h] - \perf^D[\hat{h}]\,\leq
\hspace{0.2cm} & 2QB'\E_x\left[\|\eta^\tr(x)-
\hat{\eta}^\tr(x)\|_1\right] +
 4Q\sqrt{Lk}\,\kappa(\textstyle\frac{\delta}{T}, n^\tr, n^\val) +\\
 &
\mathcal{O}\left(\lambda k\sqrt{\frac{k\log(k)\log(n^\val) + \log(k/\delta)}{n^\val}} + \frac{\lambda}{T} + Q\nu\right). \numberthis
\end{align*}
\etheorem

The proof in turn derives an error bound for the plug-in classifier in Algorithm \ref{algo:linear-metrics} for linear metrics (see Appendix \ref{app:pi-ew}). 

%% file: blackbox/relatedwork.tex
\section{Related Work}
\label{sec:relatedwork}
\textbf{Methods for closed-form metrics.} There has been a variety of work on optimizing complex evaluation metrics, including both plug-in type algorithms~\cite{ye2012optimizing, narasimhan2014statistical, koyejo2014consistent, narasimhan2015consistent, yan2018binary}, and those that use convex surrogates for the metric~\cite{joachims2005support, kar2014online, kar2016online, narasimhan2015optimizing,eban2017scalable, narasimhan2019optimizing,hiranandani2020optimization}. 
These methods rely on the test metric having a specific closed-form structure and do not handle black-box metrics. 

\textbf{Methods for black-box metrics.} 
Among recent black-box metric learning works, the closest to ours is by Jiang et al.\cite{jiang2020optimizing}, who  learn a weighted combination of surrogate losses to approximate the metric on a validation set.  
Like us, they probe the metric at multiple classifiers, 
but their approach has several drawbacks on both practical and theoretical fronts. Firstly, Jiang et al.~\cite{jiang2020optimizing} require retraining the model in each iteration, which can be time-intensive, whereas we only post-shift a pre-trained model. Secondly, the procedure they
prescribe for eliciting gradients requires perturbing the model parameters multiple times, which can be very expensive for large deep networks, whereas we only require perturbing the predictions from the model. Moreover, the number of perturbations they need grows \textit{polynomially} with the precision with which they need to estimate the loss coefficients, whereas we only require a \textit{constant} number of them.
Lastly, their approach does not come with strong statistical guarantees, whereas  ours does. Besides these benefits over~\cite{jiang2020optimizing}, we will also see in Section~\ref{sec:experiments} that our method yields better accuracies. 
Other related black-box learning methods
include~\cite{zhao2019metric, ren2018learning, huang2019addressing}, who learn 
a (weighted) loss  to approximate the metric,
but do so using computationally expensive procedures (e.g.\ meta-gradient descent or RL) that often require retraining the model from scratch, and come with limited theoretical analysis. 

\textbf{Methods for distribution shift.} The literature on distribution shift is vast, and so we cover a few representative papers; see~\cite{frenay2013classification,csurka2017comprehensive} 
for a comprehensive discussion. 
For the \textit{independent label noise} setting~\cite{natarajan2013learning}, Patrini et al.~\cite{patrini2017making} propose a loss correction approach that first trains a model with noisy label, use its predictions to estimate the noise transition matrix, and then re-trains model with the corrected loss. This approach is however tailored to optimize linear metrics; whereas, we can handle more complex metrics as well without re-training the underlying model. A plethora of  approaches exist for tackling \textit{domain shift}, including classical importance weighting (IW) strategies~\cite{sugiyama2008direct,shimodaira2000improving, kanamori2009least, lipton2018detecting}
that work in two steps:  estimate the density ratios and train a model with the resulting weighted loss. One such approach is
Kernel Mean Matching~\cite{huang2006correcting}, which matches covariate distributions between training and test sets in a high dimensional RKHS feature space. These IW approaches are however prone to over-fitting when used with deep networks~\cite{byrd2019effect}. More recent iterative variants seek to remedy this~\cite{fang2020rethinking}.

%% file: blackbox/experiments.tex
\section{Experiments}
\label{sec:experiments}
We run experiments on four classification tasks, with both known and black-box metrics, and under different label noise and domain shift settings. 
All our experiments use a large training sample, which is either noisy or contains missing attributes, and a smaller clean (and complete) validation sample. We always optimize the cross-entropy loss for learning  $\hat\eta^{\tr}(x) \approx \P^\Dshift(Y|x)$ using the training set (or $\hat\eta^{\val}(x) \approx \P^\Dtrue(Y|x)$ for some baselines), where the  models are varied across experiments. For monitoring the quality of $\hat\eta^{\tr}$ and $\hat\eta^{\val}$, we sample small subsets \emph{hyper-train} and \emph{hyper-val} data from the original training and validation data, respectively. We repeat our experiments over 5 random train-vali-test splits, and report the mean and standard deviation for each metric. We will use $^*$, $^{**}$, and $^{***}$ to denote that the differences between our method and the closest baseline are statistically significant (using Welch's t-test)  at a confidence level of 90\%, 95\%, and 99\%, respectively. We provide the data statistics in Table~\ref{tab:stats}. Observe that we always use small validation data in comparison to the size of the training data. The source code (along with random seeds) is provided on the link below.\footnote{\url{https://github.com/koyejolab/fweg/}}

\begin{table*}[t]
    \centering
    \caption{Data Statistics for different problem setups in Section~\ref{sec:experiments}.}
    {\scriptsize
    \begin{tabular}{lllll}
    \hline
        \textbf{Problem Setup} &\textbf{Dataset} & \textbf{\#Classes} &\textbf{\#Features} & \textbf{train / val / test split} \\
        \hline
        Indepen. Label Noise (Section~\ref{ssec:cifar10}) & CIFAR-10 & 10 & 32 $\times$ 32 $\times$ 3 & 49K / 1K / 10K  \\ \hline
        Proxy-Label (Section~\ref{ssec:adult}) & Adult & 2 & 101 & 32K / 350 / 16K  \\ \hline
        Domain-Shift (Section~\ref{ssec:adience}) & Adience & 2 & 256 $\times$ 256 $\times$ 3 & 12K / 800 / 3K \\ \hline
        Black-Box Fairness Metric (Section~\ref{ssec:adultbb}) & Adult & 2 (2 prot. groups) & 106 & 32K / 1.5K / 14K  \\
        \hline
    \end{tabular}
    }
    \label{tab:stats}
\end{table*}

\textbf{Common baselines}:
We use representative baselines from the black-box learning \cite{jiang2020optimizing},  
iterative re-weighting \cite{ren2018learning}, label noise correction \cite{patrini2017making}, and importance weighting \cite{huang2006correcting} literatures.
First, we list the ones common to all experiments.
\begin{enumerate}
    \item \textbf{Cross-entropy [train]:} 
    Maximizes accuracy on the training set and predicts:
    \bequation
    \widehat{h}(x) \,\in\, \argmax_{i \in [k]} \hat{\eta}^\tr_i(x).
    \eequation
    \item \textbf{Cross-entropy [val]:} 
    Maximizes accuracy on the validation set and predicts:
    \bequation
    \widehat{h}(x) \,\in\, \argmax_{i \in [k]} \hat{\eta}^\val_i(x).
    \eequation
    \item \textbf{Fine-tuning:} Fine-tunes the pre-trained $\hat\eta^{\text{tr}}$ using the validation data, monitoring the cross-entropy loss on the hyper-val data for early stopping. 
    \item \textbf{Opt-metric [val]:} 
    For  metrics $\psi(\C^D[h])$, for which $\psi$ is  \textit{known}, trains a model to directly maximize the metric on the small \textit{validation} set using the Frank-Wolfe based algorithm of \cite{narasimhan2015consistent}.
    \item \textbf{Learn-to-reweight} \cite{ren2018learning}:  Jointly  learns example weights, with the model,  to maximize accuracy on the validation set; does not handle specialized metrics.
    \item \textbf{Plug-in [train-val]:} 
    Constructs a classifier $\widehat{h}(x) \,\in\, \argmax_{i} w_i\hat{\eta}^\val_i(x)$, where the weights $w_i \in \R$ are tuned to maximize the given metric on the validation set, using a coordinate-wise line search (details in  Appendix~\ref{ssec:multiclass-plugin}). 
    \item \textbf{Adaptive Surrogates}~\cite{jiang2020optimizing}: Learns a weighted combination of surrogate losses (evaluated on clusters of examples) 
    to approximate the metric on the validation set. Since this method is not directly amenable for use with large neural networks (see Section~\ref{sec:relatedwork}), we compare with it only when using linear models, and present additional comparisons in App.~\ref{app:exp} (Table~\ref{tab:addedexp}).
    \vspace{-5pt}
\end{enumerate}
\textbf{Hyper-parameters:} The learning rate for Fine-tuning is chosen from  $1\text{e}^{\{-6,\dots,-4\}}$. For PI-EW and FW-EG, we tune the parameter $\epsilon$ from $\{1, 0.4, 1\text{e}^{-\{4,3,2,1\}}\}$. 
The line search for Plug-in is performed with a  spacing of $1\text{e}^{-4}$. The only hyper-parameters the other baselines have are those for training $\hat{\eta}^\tr$ and $\hat{\eta}^\val$, which we state in the individual tasks.

\subsection{Maximizing Accuracy under Label Noise}
\label{ssec:cifar10}
    
In our first task, we train a 10-class image classifier for the CIFAR-10 dataset~\cite{krizhevsky2009learning}, replicating the independent (asymmetric) label noise setup
from~\cite{patrini2017making}. The evaluation metric we use is accuracy. We take 2\% of original training data as validation data and flip labels in the remaining training set based on the following transition matrix: {\small TRUCK $\rightarrow$ AUTOMOBILE, BIRD $\rightarrow$ PLANE, DEER $\rightarrow$ HORSE, CAT $\leftrightarrow$ DOG}, with a flip probability of 0.6. For $\hat\eta^{\text{tr}}$ and $\hat\eta^{\text{val}}$, we use the same ResNet-14 architecture as~\cite{patrini2017making}, trained using SGD with momentum 0.9, weight decay $1\text{e}^{-4}$, 
and  learning rate 0.01, which we divide by 10 after 40 and 80 epochs (120 in total). 

We additionally compare with the \emph{Forward Correction} method of~\cite{patrini2017making}, a specialized method for correcting independent label noise, which estimates the noise transition matrix $\T$ using predictions from $\hat\eta^{\text{tr}}$ on the training set, and retrains it with the corrected loss, thus  training the ResNet twice. 
We saw a notable drop with this method when we used the (small) validation set to estimate $\T$. 

We apply the proposed PI-EW method for linear metrics, using a weighting function $\W$ defined with one of two choices for the basis functions (chosen via
cross-validation): (i) a default basis function that clusters all the points together $\phi^{\text{def}}(x) = 1 \, \forall x$, and (ii) 
ten basis functions $\phi^1, \dots, \phi^{10}$, each one being the average of the RBF kernels (see \eqref{eq:softlcuster}) centered at validation points belonging to a true class. 
The RBF kernels are computed with width 2 on UMAP-reduced 50-dimensional image embeddings~\cite{mcinnes2018umap}.

\begin{table}[t]
    \small
    \centering
    \caption{Test accuracy for noisy label experiment on CIFAR-10.}
\begin{tabular}{ll}
    \hline
        Cross-entropy [train] &  0.582 $\pm$ 0.007\\
        Cross-entropy [val] & 0.386 $\pm$ 0.031 \\
        Learn-to-reweight & 0.651	$\pm$ 0.017\\
        Plug-in [train-val] & 0.733	$\pm$ 0.044\\
        Forward Correction & 0.757 $\pm$	0.005\\
        Fine-tuning & 0.769 $\pm$	0.005\\
        \hline
        PI-EW & $\textbf{0.781} \pm \textbf{0.019}$\\
        \hline
    \end{tabular}
    \label{tab:cifar10assym}
\end{table}

As shown in Table~\ref{tab:cifar10assym}, PI-EW achieves significantly better test accuracies than all the baselines. The results for Forward Correction matches those in \cite{patrini2017making}; unlike this method, we train the ResNet only once, but achieve  2.4\% higher accuracy. 
 
Cross-entropy [val] over-fits badly, and yields the least test accuracy. Surprisingly, the simple fine-tuning yields the second-best accuracy. A possible reason is that the pre-trained model learns a good feature representation, and the fine-tuning step adapts well to the domain change. We also observed that PI-EW achieves better accuracy during cross-validation with ten basis functions,  highlighting the benefit of the underlying modeling in PI-EW. Lastly, in Figure~\ref{fig:weightscifar}, we show the elicited (class) weights with the default basis function  ($\phi^{\text{def}}(x) = 1 \, \forall x$), where e.g. because BIRD $\rightarrow$ PLANE, the weight on BIRD is upweighted and that on PLANE is down-weighted. 

\subsection{Maximizing G-mean with Proxy Labels}
\label{ssec:adult}

Our next experiment borrows the  ``proxy label'' setup from~\cite{jiang2020optimizing} on the Adult dataset~\cite{Dua:2019}. The task is to predict whether a candidate's gender is male, but the training set contains only a proxy for the true label. 
We  sample 1\% validation data from the original training data, and replace the labels in the remaining sample with the feature `relationship-husband'. 
The label noise here is instance-dependent 
(see Example~\ref{ex:instnoise}), and
we seek to maximize the G-mean metric: \bequation
\psi(\C) \,=\,\big(\prod_i \big({C_{ii}}/\sum_j C_{ij}\big)\big)^{1/m}.
\eequation

We train $\hat\eta^{\text{tr}}$ and $\hat\eta^{\text{val}}$ using linear logistic regression using SGD with a learning rate of 0.01. 
As additional baselines, we include the Adaptive Surrogates method of \cite{jiang2020optimizing} and \emph{Forward Correction}~\cite{patrini2017making}. The inner and outer 
learning rates for Adaptive Surrogates are each cross-validated in $\{0.1, 1.0\}$. We also compare with a simple Importance Weighting strategy, 
where we first train a  logistic regression model $f$ to predict if an example $(x,y)$ belongs to the validation data, and train a gender classifier with the training examples weighted by {\small $f(x,y)/(1 - f(x,y))$}. 

We choose between three sets of basis functions (using cross-validation): (i) a default basis function 
$\phi^{\text{def}}(x) = 1 \, \forall x$, (ii) $\phi^{\text{def}}, \phi^{\text{pw}}, \phi^{\text{npw}}$, where  $\phi^{\text{pw}}(x) = \1(x_{\text{pw}} = 1)$ and $\phi^{\text{npw}}(x) = \1(x_{\text{npw}} = 1)$ use features  `private-workforce' and `non-private-workforce' to form hard clusters, (iii) $\phi^{\text{def}}$, $\phi^{\text{pw}}, \phi^{\text{npw}}, \phi^{\text{inc}}$, where $\phi^{\text{inc}}(x) = \1(x_{\text{inc}} = 1)$ uses the binary feature `income'.
These choices are motivated from those used by~\cite{jiang2020optimizing}, who compute surrogate losses on the individual clusters. 
We provide their Adaptive Surrogates method with the same clustering choices.

\begin{table}[t]
    \small
    \centering
    \caption{Test G-mean for proxy label experiment on Adult.}
    \begin{tabular}{ll}
    \hline
        Cross-entropy [train] &  0.654	$\pm$ 0.002\\
        Cross-entropy [val] & 0.394 $\pm$	0.064 \\
        Opt-metric [val] & 0.652 $\pm$	0.027\\
        Learn-to-reweight & 0.668	$\pm$ 0.003\\
        Plug-in [train-val] & 0.672 $\pm$ 	0.013\\
        Forward Correction & 0.214 $\pm$	0.004\\
        Fine-tuning & 0.631	$\pm$ 0.017\\
        Importance Weights & 0.662 $\pm$	0.024\\
        Adaptive Surrogates & 0.682	 $\pm$ 0.002\\
        \hline
        FW-EG [unknown $\psi$] & $\textbf{0.685}	\pm \textbf{0.002}^{**}$\\
        FW-EG [known $\psi$] & $\textbf{0.685}	\pm \textbf{0.001}^{*}$\\
        \hline
    \end{tabular}
    \label{tab:adultproxy}
\end{table}

Table~\ref{tab:adultproxy} summarizes our results. We apply both variants of
our FW-EG method for a non-linear metric $\psi$, one where $\psi$ is \textit{known} and its gradient is available in closed-form, and the other where $\psi$ is assumed to be \textit{unknown}, and is treated as a general black-box metric. Both variants perform similarly and are better than the baselines. Adaptive Surrogates comes a close second, but underperforms by 0.3\% (with results being statistically significant). While the improvement of FW-EG over Adaptive Surrogates is small, the latter is time intensive as, in each iteration, it re-trains a logistic regression model. 
We verify this empirically in Figure~\ref{fig:time} by reporting run-times for Adaptive Surrogates and our method FW-EG (including the pre-training time) against the choices of basis functions (clustering features). We see that our approach is 5$\times$ faster for this experiment. 
Lastly, Forward Correction performs poorly, likely because its loss correction is not aligned with this label noise model.

\subsection{Maximizing F-measure under Domain Shift}
\label{ssec:adience}

We now move on to a domain shift application  (see Example~\ref{ex:ds}). The task is to learn a gender recognizer for the Adience face image dataset~\cite{eidinger2014age}, but with the training and test datasets containing images from different age groups (domain shift based on age). We use images belonging to age buckets 1--5 for training (12.2K images), and evaluate on images from age buckets 6--8 (4K images). For the validation set, we sample 20\% of the 6--8 age bucket images. Here we aim to maximize the F-measure.

For $\hat\eta^{\text{tr}}$ and $\hat\eta^{\text{val}}$, we
use the same ResNet-14 model from the CIFAR-10 experiment, except that the learning rate is divided by 2 after 10 epochs (20 in total). As an additional baseline, we compute importance weights using \emph{Kernel Mean Matching (KMM)}~\cite{huang2006correcting}, and train the same ResNet model with a weighted loss. Since the image size is large for directly applying KMM, we first compute the 2048-dimensional ImageNet embedding~\cite{krizhevsky2012imagenet} for the images and further reduce them to 10-dimensions via UMAP. 
The KMM weights are learned on the 10-dimensional embedding. For the basis functions, besides the default basis $\phi^{\text{def}}(x) = 1 \, \forall x$, we choose from subsets of six RBF basis functions $\phi^1,\ldots,\phi^6$, centered at points from the validation set, each representing one of six age-gender combinations. We use the same UMAP embedding as KMM to compute the RBF kernels.

\begin{table}[t]
    \small
    \centering
    \caption{Test F-measure for domain shift experiment on Adience.}
    \begin{tabular}{ll}
    \hline
        Cross-entropy [train] & 0.760 $\pm$	0.014\\
        Cross-entropy [val] & 0.708 $\pm$	0.022 \\
        Opt-metric [val] & 0.760 $\pm$	0.014\\
        Plug-in [train-val] & 0.759 $\pm$	0.014\\
        Importance Weights [KMM] & 0.760 $\pm$	0.013\\
        Learn-to-reweight & 0.773	$\pm$ 0.009\\
        Fine-tuning & 0.781	$\pm$ 0.014\\
        \hline
        FW-EG [unknown $\psi$] & $\textbf{0.815}	\pm \textbf{0.013} ^{***}$\\
        FW-EG [known $\psi$] & $\textbf{0.804} \pm	\textbf{0.015}^{***}$\\
        \hline
    \end{tabular}
    \label{tab:adiencecovs}
\end{table}

Table~\ref{tab:adiencecovs} presents the test F-measure values. Both variants of FW-EG algorithm provide statistically significant improvements over the baselines. Both Fine-tuning and Learning-to-reweight improve over plain cross-entropy optimization (train), however only moderately, likely because of the small size of the validation  set, and because these methods are not tailored to optimize the F-measure. 

\subsection{Maximizing Black-box Fairness Metric}
\label{ssec:adultbb}

We next handle a black-box metric given only query access to its value. We consider a fairness application where the goal is to balance classification performance across multiple protected groups.  The groups that one cares about are known,
but due to privacy or legal restrictions, the protected attribute for an individual cannot be  revealed~\cite{awasthi+21}. Instead, we have access to an oracle that reveals the value of the fairness metric for predictions on a validation sample, with the protected attributes absent from the training sample. 
This setup is different from recent work on learning fair classifiers from incomplete group information \cite{lahoti2019ifair, wang2020robust}, in that the focus here is on optimizing \textit{any} given black-box fairness metric.

We use the Adult dataset, and seek to predict whether the candidate's income is greater than \$50K, with \textit{gender} as the protected group. The black-box  metric we consider (whose form is unknown to the learner) is the geometric mean of the true-positive (TP) and true-negative (TN) rates, evaluated separately on the male and female examples, which promotes equal performance for both groups and classes:
\begin{equation}
\perf^\Dtrue[h] = \left( \text{TP}^{\text{male}}[h]\,\text{TN}^{\text{male}}[h])\,\text{TP}^{\text{female}}[h]\,\text{TN}^{\text{female}}[h] \right)^{1/4}.
\label{eq:fairmetric}
\end{equation}

We train the same logistic regression models as in previous Adult experiment in Section \ref{ssec:adult}. Along with the  basis functions $\phi^{\text{def}}$, $\phi^{\text{pw}}$ and $\phi^{\text{npw}}$ we used there, we additionally include two basis $\phi^{\text{hs}}$ and $\phi^{\text{wf}}$  based on features `relationship-husband' and `relationship-wife', which we expect to have correlations with gender.\footnote{The only domain knowledge we use is that the protected group is ``gender"; beyond this, the
form of the metric is unknown, and importantly, an individual's gender is not available.} 
We include two baselines that can handle black-box metrics: Plug-in [train-val], which tunes a threshold on $\hat{\eta}^\tr$ by querying the metric on the validation set, and Adaptive Surrogates. The latter is cross-validated on the same set of clustering features (i.e., basis functions in our method) for computing the surrogate losses. 

\begin{table}[t]
    \small
    \centering
    \caption{Black-box fairness metric on the test set for Adult.}
    \begin{tabular}{ll}
    \hline
        Cross-entropy [train] & 0.736 $\pm$	0.005\\
        Cross-entropy [val] & 0.610 $\pm$	0.020 \\
        Learn-to-reweight &  0.729	$\pm$ 0.007 \\
        Fine-tuning & 0.738 $\pm$	0.005\\
        Adaptive Surrogates & 0.812	$\pm$ 0.004\\
        Plug-in [train-val] & 0.812 $\pm$	0.005\\
        \hline
        FW-EG & $\textbf{0.822}	\pm \textbf{0.002}^{***}$\\
        \hline
    \end{tabular}
    \label{tab:adultbb}
\end{table}

As seen in Table~\ref{tab:adultbb}, FW-EG yields the highest black-box metric on the test set, Adaptive Surrogates comes in second, and surprisingly the simple plug-in approach fairs better than the other baselines. During cross-validation, we also observed that the performance of FW-EG improves with more basis functions, particularly with the ones that are better correlated with gender. Specifically, 
FW-EG with basis functions  $\{\phi^{\text{def}}, \phi^{\text{pw}}, \phi^\text{npw},  \phi^\text{wf}, \phi^\text{hs}\}$ achieves
approximately
1\% better performance than both
FW-EG with $\phi^{\text{def}}$ basis function and FW-EG with basis functions $\{\phi^{\text{def}},\phi^{\text{pw}}, \phi^\text{npw}\}$.

\subsection{Ablation Studies}
\label{ssec:ablation}

\begin{figure*}[t]
	\centering 
	\subfigure[]{
		{\includegraphics[width=5cm]{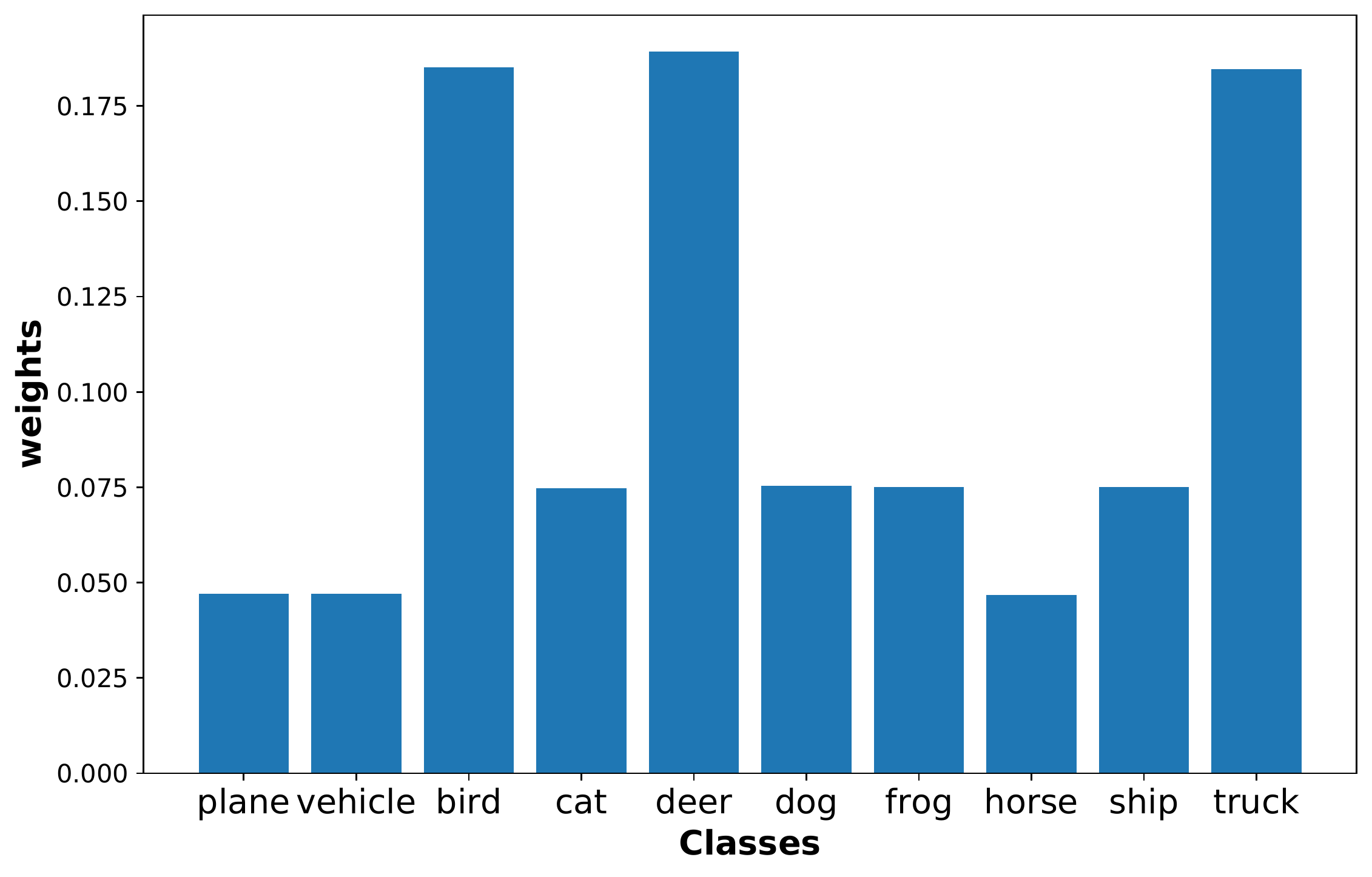}}
		\label{fig:weightscifar}
	}
	\subfigure[]{
		{\includegraphics[width=5cm]{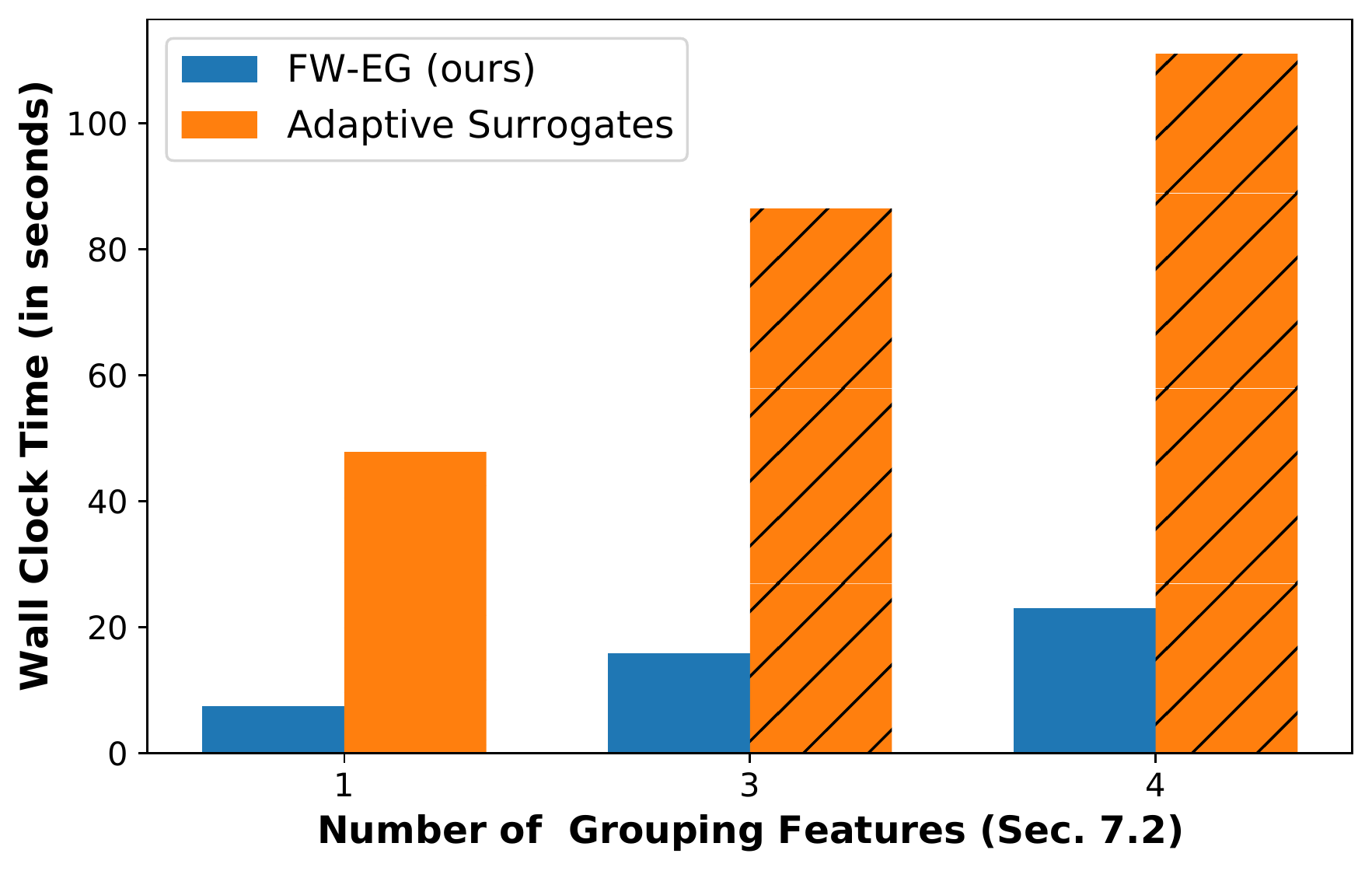}}
		\label{fig:time}
	}
	\\
	\subfigure[]{
		{\includegraphics[width=5cm]{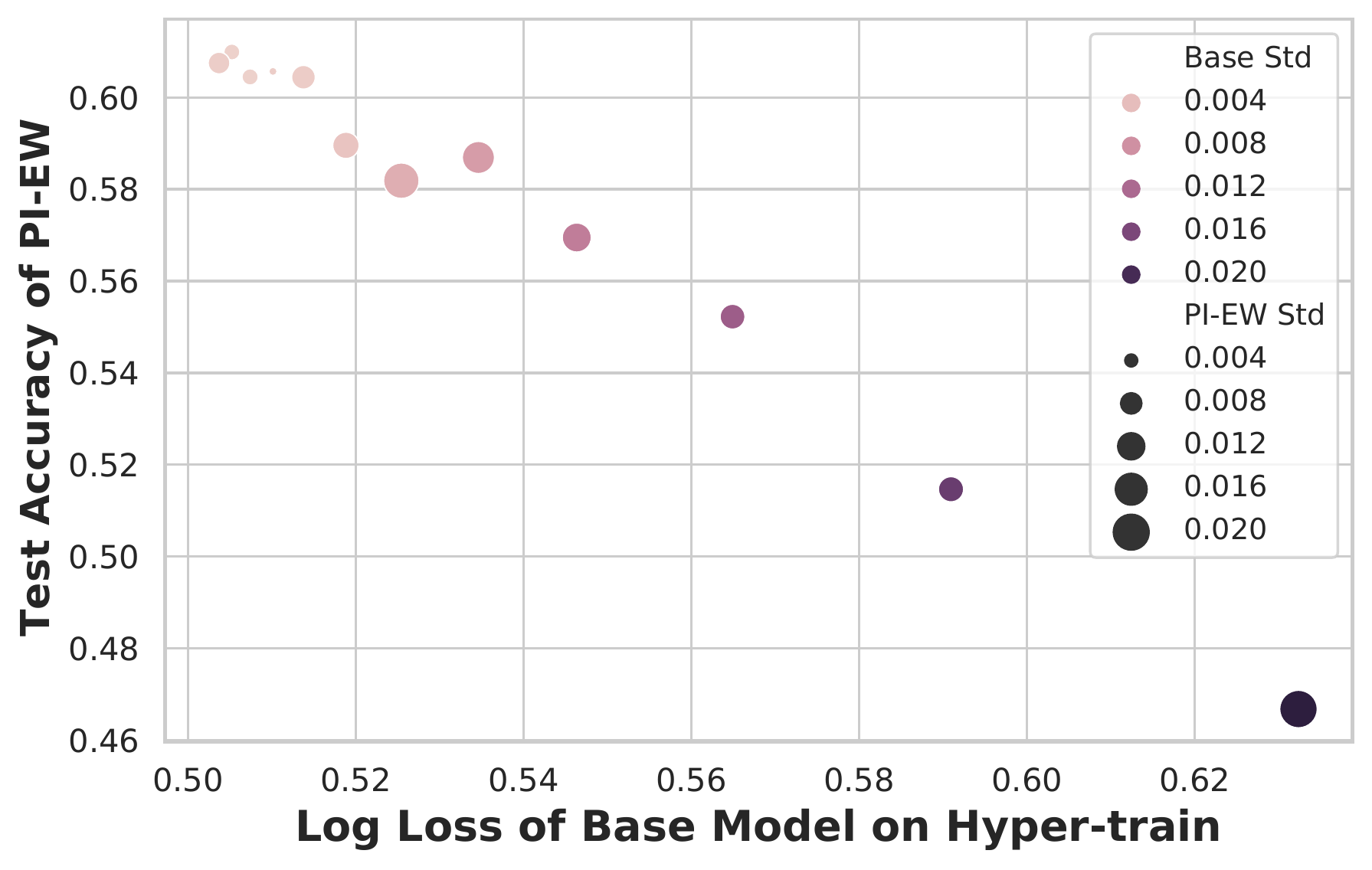}}
		\label{fig:periodicmain}
	}
	\subfigure[]{
		{\includegraphics[width=5cm]{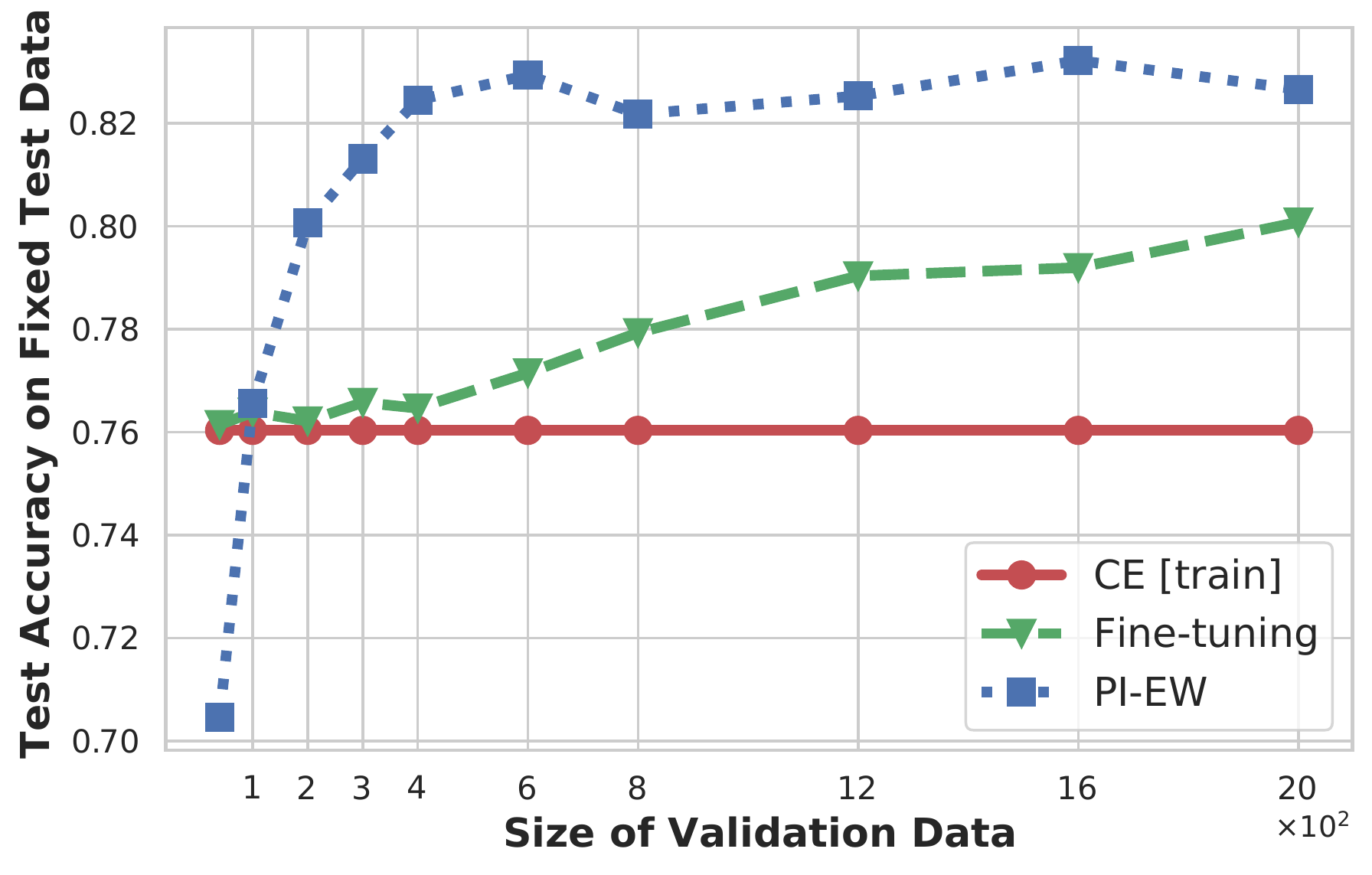}}
		\label{fig:valsizemain}
	}
	\caption{(a) Elicited (class) weights for CIFAR-10 by PI-EW for the default basis (Sec.~\ref{ssec:cifar10}); (b) Run-time for FW-EG and Adaptive Surrogates~\cite{jiang2020optimizing}
    vs no.\ of grouping features on proxy label task (Sec.~\ref{ssec:adult}); (c)  Effect of quality of the base model $\hat\eta^{\tr}$ on Adult (Sec.~\ref{ssec:adult}): as the base model's quality improves, the test accuracies of PI-EW also improves; (d) Effect of the validation set size on Adience (Sec.~\ref{ssec:adience}): PI-EW performs better than fine-tuning even for small validation sets, while both improve with larger ones.
	}
	\label{fig:mainabalation}
\end{figure*}

We close with two sets of experiments. First, we analyze how the performance of PI-EW, while optimizing accuracy for the Adult experiment (Section~\ref{ssec:adult}), varies with the quality of the base model $\hat{\eta}^\tr$. We save an estimate of $\hat\eta^{\tr}$ after every 50 batches (batch size 32) while training the logistic regression model, and use these estimates as inputs to PI-EW. As shown in Figure~\ref{fig:periodicmain}, the test accuracies for PI-EW improves with the quality of $\hat{\eta}^\tr$ (as measured by the log loss on the hyper-train set).  This is in accordance with Theorem~\ref{thm:iterative-plugin}. 
One can further improve the quality of the estimate $\eta^{\text{tr}}$ by using calibration techniques~\cite{guo2017calibration}, which will likely enhance the performance of PI-EW as well.

Next, we show that PI-EW is robust to changes in the validation set size when trained on the Adience experiment in Section~\ref{ssec:adience} to optimize accuracy. We set aside 50\% of 6--8 age bucket data for testing, and sample varying sizes of validation data from the rest. As shown in Figure~\ref{fig:valsizemain}, 
PI-EW generally performs better than fine-tuning even for small validation sets, while both improve with larger ones. The only exception is 100-sized validation set (0.8\% of training data), where we see overfitting due to small validation size.

%% file: blackbox/discussion.tex
\subsection{Black-box optimization with pairwise comparison oracle}
\label{bbopt:pair}

The proposed algorithm in this chapter works with \emph{machine} oracles that when queried for a classifier $h$ respond with the metric value $\Ecal^D[h].$ We saw various cases, e.g., validation set in distribution shift settings or a regulator in fairness setups, where we have access to such an oracle. We exploit the fact that the example weights act as a gradient or a local linear objective in a small neighborhood for the unknown metric, and elicit such linear metrics through the use of value queries to the machine oracle. 

The same idea can be extended in the presence of a \emph{human} oracle that provides pairwise preferences. This also includes A/B testing scenarios commonly used in the web based applications~\cite{satyal2018ab, bhat2020near,  hiranandani2020cascading}. In order to elicit a local-linear objective around a classifier's confusion matrix $C^D[h]$, one can first construct a small sphere around $C^D[h]$ and the corresponding classifiers by the process discussed  in Section~\ref{append:ssec:slambda}, and then run Algorithm~\ref{mult-alg:lpm} to elicit the local-linear performance metric using the pairwise comparisons. Once the local-linear objective is estimated, then one can post-shift a pre-trained class-conditional estimator similar to the proposed FW-EG algorithm (Algorithm~\ref{algo:FW}). 

However, this approach comes with its own challenges. Firstly, in order to apply the iterative Frank-Wolfe approach, one will need to create the spheres and the corresponding classifiers multiple times. This would make the algorithm time-intensive as it would require to solve an optimization problem in each iteration. Secondly, it is not clear how to elicit the local-linear objective, when one chooses overlapping or softly clustered basis functions. We hope to overcome these challenges in the future.

\section{Concluding Remarks}
\label{sec:discussion}
In this chapter, we proposed the Frank Wolfe with Elicited Gradient (FW-EG) method 
for optimizing black-box metrics given query access to the evaluation metric on a small validation set. Our framework includes common distribution shift settings as special cases, and unlike prior distribution correction strategies, is able to 
handle general non-linear metrics. 
A key benefit of our method is that it is agnostic to the choice of  $\hat\eta^{\tr}$, and can thus be used to post-shift 
pre-trained deep networks, without having to retrain them. We showed that the post-shift example weights can be flexibly modeled with various choices of basis functions (e.g., 
hard clusters, RBF kernels, etc.) and  empirically demonstrated their efficacies. We exploit the fact that the example weights act as a gradient for the unknown metric and estimated through metric elicitation procedure, where a \emph{machine} oracle responds with absolute quality value of a classifier on a clean validation dataset. Moreover, the novel geometrical characterizations discussed in Chapters~\ref{chp:binary},  \ref{chp:multiclass}, and \ref{chp:fair} led us to devise an efficient and a smart method for creating the probing classifiers (see~\eqref{eq:trivial-classifiers}). 
We look forward to further improving the results with more nuanced basis functions. 

%% file: practical/header.tex
\chapter{Practical Metric Elicitation}
\label{chp:practical}

\input{practical/introduction}
\input{practical/method}

\input{practical/results}

\input{practical/conclusion}

%% file: practical/introduction.tex
Till now, our contributions towards the Metric Elicitation (ME) framework with pairwise comparisons have been algorithmic. So, to bring theory closer to practice, in this chapter, we conduct a preliminary real-user study that shows the efficacy of the metric elicitation framework in recovering the users' preferred performance metrics in a binary classification setup. 

 We choose \emph{cancer diagnosis}~\cite{yang2014multiclass} as the application for this task, where the ground-truth label is a binary feature denoting whether or not the patient has cancer. This choice is motivated by Application~1 discussed in Chapter~\ref{chp:introduction}, since there are asymmetric costs associated with False Positives and False Negatives -- based on known consequences of misdiagnosis, i.e, side-effects of treating a healthy patient vs. mortality rate for not treating a sick patient. Our work (a) builds upon existing visualizations for confusion matrices to ask for pairwise preferences, and (b) then try to elicit a linear performance metric using our proposed procedure in Algorithm~\ref{bin-alg:linear} in the binary classification setup. We work with ten subjects in this preliminary study, who have some experience either with machine learning or biomedical research in the university setup. 
 
 We create a web User Interface (UI),\footnote{The user-interface is shown  later and is also available at http://safinahali.com/elicitation-graphs-static/} which broadly has three parts to it. First, it shows subjects a couple of confusion matrices and asks questions related to \emph{comprehension, comparison, and simulation}~\cite{shen2020designing}. These questions familiarize the subjects with the visualizations and the components associated with the correct and incorrect predictions. Second, it shows a bunch of pairwise preference queries over confusion matrices. The UI involves running the binary-search procedure from Algorithm~\ref{bin-alg:linear} at the back end, which chooses the next set of queries based on the subject's current real-time responses. Third, the UI comprises of fifteen pairwise comparison queries, where the confusion matrices are randomly chosen from the feasible set. The responses to these queries are used to evaluate the fidelity of the recovered metric through metric elicitation framework. At the end of the web-based task, the subjects are asked some subjective questions which essentially lead to our guidelines that we recommend for implementing the metric elicitation framework in real-life scenarios.

The goal of this preliminary study is to check workflow of the practical implementation of the metric elicitation framework with real data, and to a certain extent, support or reject the hypothesis that the implicit user preferences can be quantified using the pairwise comparison queries over confusion matrices. In addition, the goal includes testing certain assumptions regarding the noise in the subject's (oracle's) responses, work around with finite samples, and provide future guidance on visualizing confusion matrices for pairwise comparisons, eliciting actual performance metrics in real-life scenarios, and evaluating the quality of the recovered metric.

The contributions from this chapter are summarized as follows:

\begin{itemize}
    \item We create a web UI that uses existing visualizations of confusion matrices that are refined to capture preferences over pairwise comparisons. 
    \item The UI implements the binary-search procedure from Algorithm~\ref{bin-alg:linear} at the back end that make use of the real-time responses over confusion matrices to elicit a linear performance metric in the cancer diagnosis setup.  
    \item We perform a user study with ten subjects and elicit their linear performance metrics using the proposed web UI. We compare the quality of the recovered metric  by comparing their responses to the elicited metric's responses over a set of randomly chosen pairwise comparison queries. The study also includes a post-task, \emph{think-aloud}-style interview regarding the utility of the framework.
    \item Lastly, using the task results and the post-task interviews, we present guidelines regarding practical implementation of the ME framework that can be used for future research in this direction. 
\end{itemize}

%% file: practical/method.tex
\section{Dataset and Visualization Choice}
\label{pme-sec:data-vis}

In this section, we first discuss the details of the dataset used and how the feasible set of confusion matrices is constructed. Then, we discuss the choice of visualizations for confusion matrices, which are borrowed from prior work, but are refined to allow for better pairwise comparisons. 

\subsection{Choice of Task and Dataset Used}
\label{pme-ssec:dataset}

Our choice of task domain and the dataset is motivated by Application~1 discussed in Chapter~\ref{chp:introduction}. The task is \emph{cancer diagnosis}~\cite{yang2014multiclass} for which we use the Breast Cancer Wisconsin (Original) dataset from the UCI repository.\footnote{The dataset can be downloaded from https://tinyurl.com/dn2esyvw.} The dataset has been extensively used in the literature for binary classification, where the label $1$ denotes \emph{malignant} cancer and label $0$ denotes \emph{benign} cancer. There are 699 samples in total, wherein each sample has 9 features. Around 35\% of the data is labelled as $1$ and the rest as $0$. The task for any classifier is to take the 9 features of a patient as input and predict whether or not the patient has cancer. 

We divide this data into two equally sized parts -- the training and the test data. Using the training data, we learn a logistic regression model to obtain an estimate of the class-conditional probability, i.e., $\hat\eta(x) = \hat\Pmbb(y=1 | X)$. We then create a pool of thresholded classifiers of the type:
\bequation
h_\tau(x) = \1[\hat\eta(x)\geq \tau],
\label{pme-eq:classifiers}
\eequation
where we vary the threshold $\tau$ from $0$ to $1$ in steps of $1e^{-4}$. Subsequently, we compute confusion matrices for the above threhsolded classifiers on the test data (resulting in 10001 confusion matrices). As discussed in Chapter~\ref{chp:binary} (see Figure~\ref{bin-fig:lin-fr}), the space of confusion matrices is a two-dimensional space and the confusions (tuple of true positives and true negatives) associated with the thresholed classifiers above form the upper boundary. This upper boundary for the estimated confusions on the test data is shown in Figure~\ref{pme-fig:confusions} (see solid, red line). 

\begin{figure}[t]
    \centering
    \includegraphics[scale=0.5]{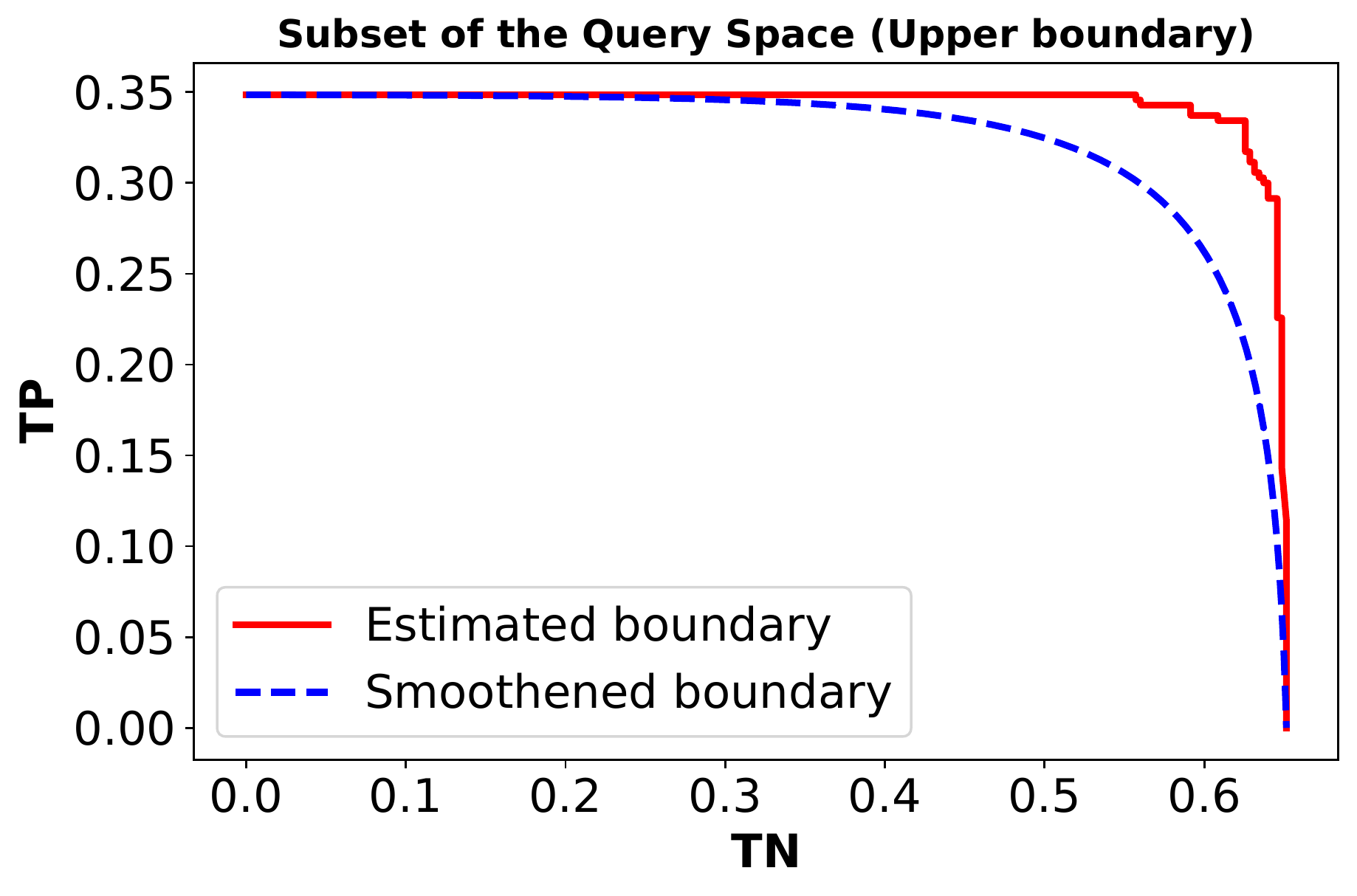}
    \caption{Estimated confusions on test data forming the upper boundary of the space of confusion matrices and the associated smoothened version of the upper boundary.}
    \label{pme-fig:confusions}
\end{figure}

As discussed in Chapter~\ref{chp:binary}, one can use these estimated confusion matrices in practice to elicit linear performance metrics. However, in the binary classification setup, we can easily smoothen the upper boundary, and that too using feasible confusion matrices. This allows to reduce the \emph{staircase} type bumps due to estimation from finite data, and consequentially, lead to better convergence from the binary-search based Algorithm~\ref{bin-alg:linear}. To generate confusions on the smoothened version of the upper boundary, we take the same simulated distribution setting from Section~\ref{bin-ssec:theoryexp}. 

Specifically, we take a joint probability for $\Xcal = [-1,1]$ and $\Ycal = \{0, 1\}$ given by $f_X = \Umbb[-1,1]$ and $\eta(x) = \frac{1}{1 + e^{ax + b}}$, where $\Umbb[-1,1]$ is the uniform distribution on $[-1, 1]$. Then we estimate the parameters $a$ and $b$ such that they minimize the squared error between the (10K) confusions obtained on the test data and the ones simulated by using the above distribution. The smoothened upper boundary is shown as dashed, blue curve in Figure~\ref{pme-fig:confusions}. Clearly, all these confusions are feasible as they would lie inside the region enclosed by the upper and lower boundary, and thus we can use the confusions on the smoothened upper boundary for elicitation purposes.

\subsection{Choice of Visualization}
\label{pme-ssec:vis}

In modern times, ensuring effective public understanding of algorithmic decisions, especially, machine learning models has become an imperative task. With this view in mind, we borrow the visualizations of confusion matrices for the binary classifications setup from Shen et al.~\cite{shen2020designing}. The authors provide a concrete step towards the above goal by redesigning confusion matrices to support non-experts in understanding the performance of machine learning models. The final visualizations that we use from Shen et al.~\cite{shen2020designing} are created over multiple iterative user-studies. 

In the first study, the authors conduct interviews with $7$ subjects and a survey with $102$ subjects and map out two major sets of challenges lay people have in understanding standard confusion matrices. These are (a) general terminologies and (b) the matrix design. These challenges are further elaborated with three sub-challenges that include confusion about the direction of reading the data, layered relations, and the quantities involved. In order to tackle these challenges, the authors came up with four alternative visualizations of the confusion matrix. In the second study, the authors evaluate the efficacy of the proposed visualizations over $483$ subjects on a recidivism prediction task~\cite{shen2020designing}. The authors conclude that the \emph{flow-chart} is the most preferred visualization of a confusion matrix followed by a \emph{bar-chart}. Both these visualizations are shown in Figure~\ref{pme-fig:vis-prior} in the context of a recidivism prediction task.

\begin{figure}
    \centering
    \includegraphics[scale=0.5]{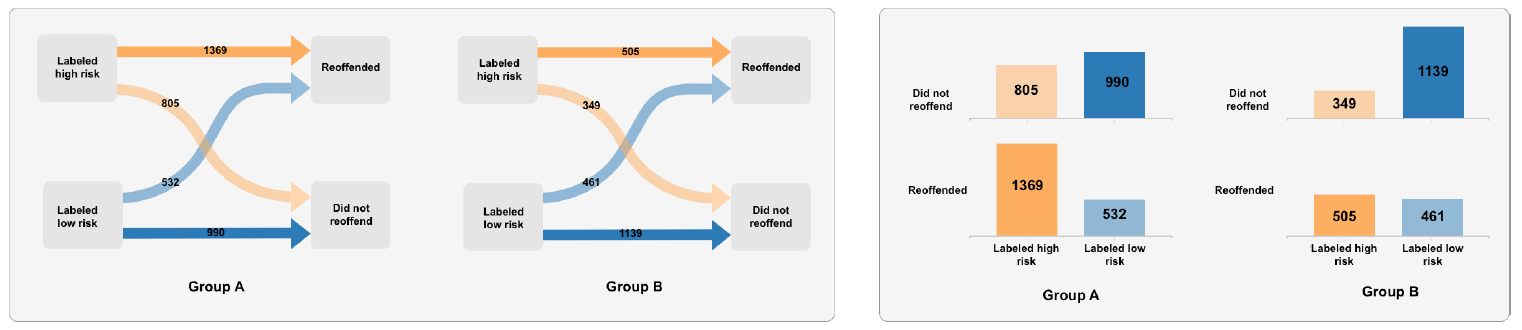}
    \caption{Flow-chart and bar-chart based visualizations for (binary classification) confusion matrices in the recidivism prediction task from Shen et al.~\cite{shen2020designing}.}
    \label{pme-fig:vis-prior}
\end{figure}

However, in light of our preliminary discussions with Human-Computer Interaction (HCI) and machine learning researchers, we make/recommend the following changes in the visualization for pairwise comparison purposes in the metric elicitation framework.

\begin{figure}[t]
    \centering
    \includegraphics[scale=0.75]{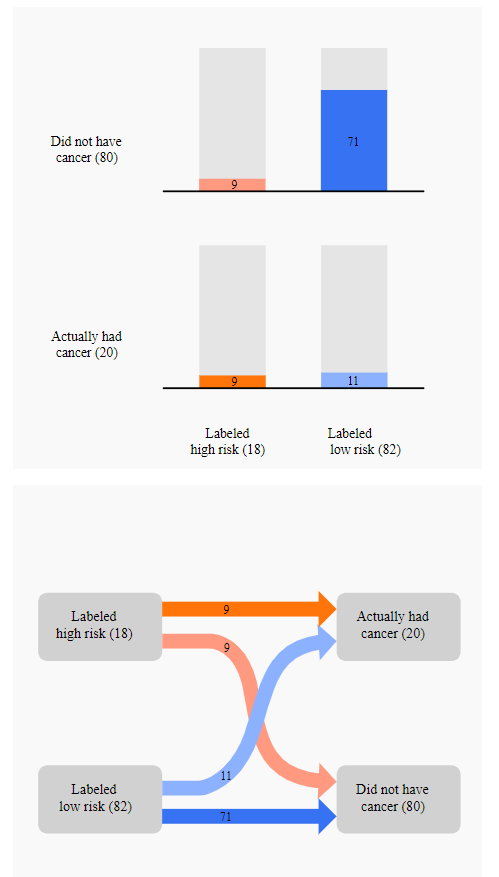}
    \caption{Our modified visualization of a confusion matrix for a cancer diagnosis task. Modification is  from the perspective of obtaining better pairwise preferences.}
    \label{pme-fig:vis-our}
\end{figure}

\begin{enumerate}
    \item Based on the observation that multiple visualizations of the information help in better user understanding~\cite{mazza2009introduction}, we choose to use the top two performing visualizations, i.e., the \emph{flow-chart} and the \emph{bar-chart}, together to depict a confusion matrix. 
    \item We transform the data statistics so that the numbers denote out-of-100 samples. 
    \item We found that the total number of positive and negative labels along with total number of positive and negative predictions are very helpful in comparing two confusion matrices. Therefore, we add the total numbers in the flow-chart boxes and on axes in the bar-charts. 
    \item We also add a zoom-in feature for both the graphs for better understanding.
    \item Although, in this preliminary user study, we have not changed the direction in the flow-chart, in our discussions with HCI and machine learning researchers, we also noted that the current direction is perhaps more important for the recidivism task (that is because there is time component involved with it) but can be changed for the cancer diagnosis task. This allows one to have constants (i.e., total positive and negative labels) in the left column and the varying component (i.e., total positive and negative predictions) on the right column making the comparison easier. Moreover, this change ensures that the bar-chart and the flow-chart represent similar information. We plan to implement this change and record its impact in our future user studies. 
\end{enumerate}
Our modified visualization incorporating the first four points above for a confusion matrix in the context of cancer diagnosis is shown in Figure~\ref{pme-fig:vis-our}. We next discuss the web user interface. 

\section{User Interface}
\label{sec:ui}

We discuss our proposed web User Interface (UI) in detail and discuss our rationale behind its several components. We also provide images of the UI at the end of this chapter. 

The UI starts with a questionnaire asking about demographic information like age, gender, race, highest level of school, and the subjects' expertise in machine learning and healthcare as shown in Figure~\ref{pme-fig:index}. Then the UI has three parts to it as explained in the following sub-sections.

\begin{figure}[t]
    \centering
    \includegraphics[scale=0.6]{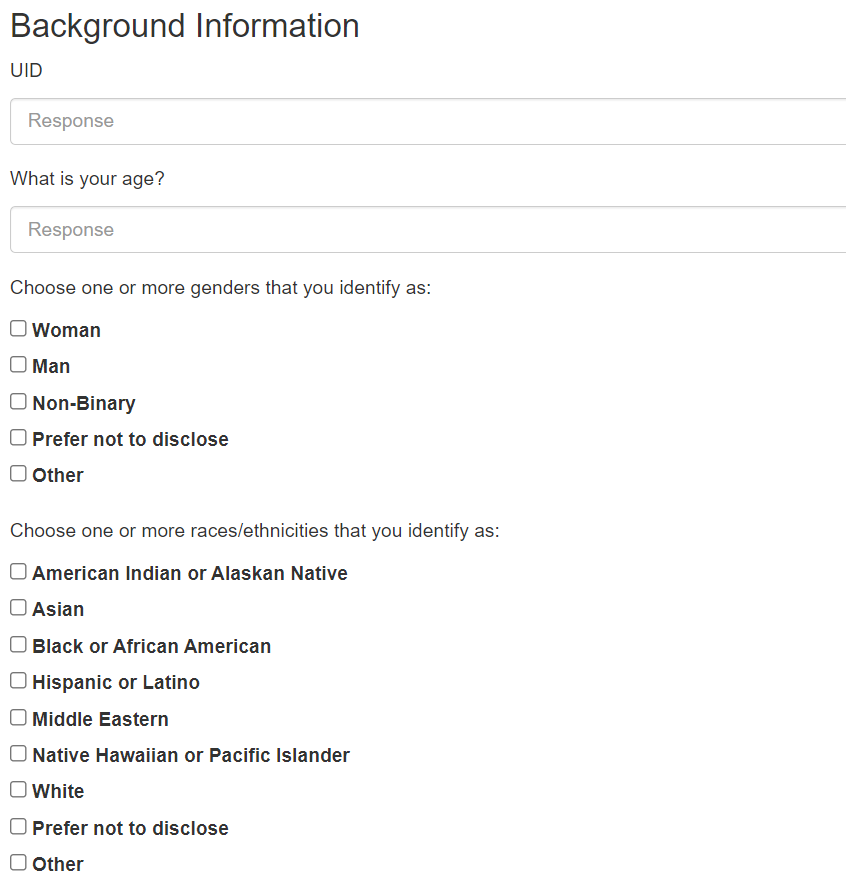}
    \includegraphics[scale=0.6]{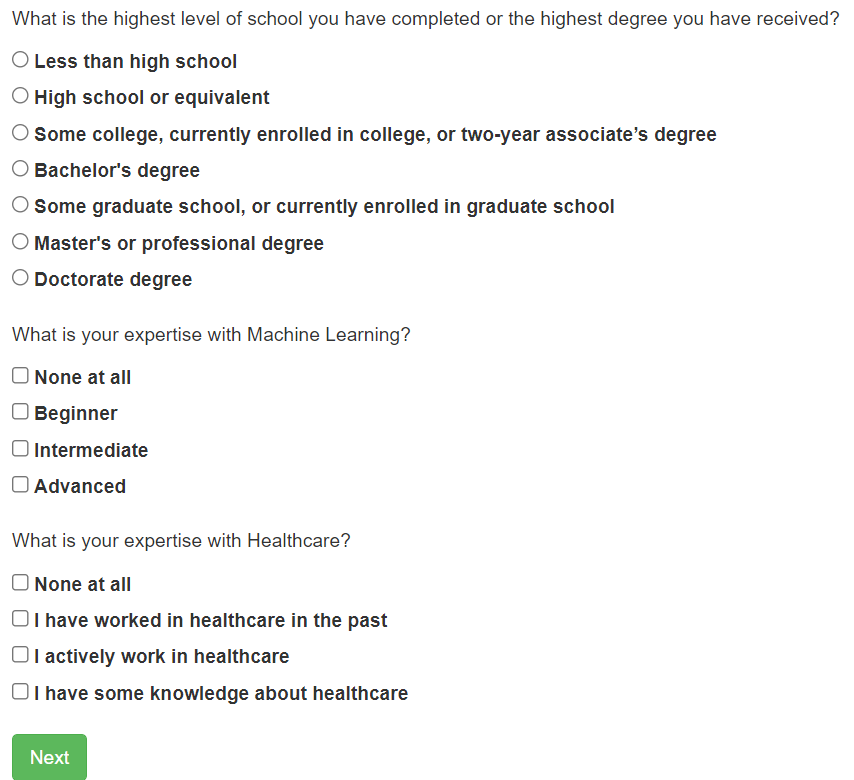}
    \caption{Questionnaire on the first page of the UI.}
    \label{pme-fig:index}
\end{figure}

\subsection{Understanding and Familiarizing with the Visualizations}
\label{pme-ssec:understanding}

After the questionnaire, we describe the task of cancer diagnosis and provide details on how classifiers can be inaccurate in their predictions in layman terms. We also show the proposed visualization of a confusion matrix along with the description as exhibited in Figure~\ref{pme-fig:prelim}.

\begin{figure}[t]
    \centering
    \includegraphics[scale=0.5]{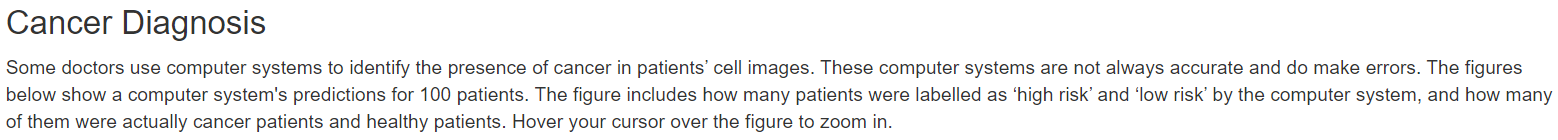}
    \includegraphics[scale=0.75]{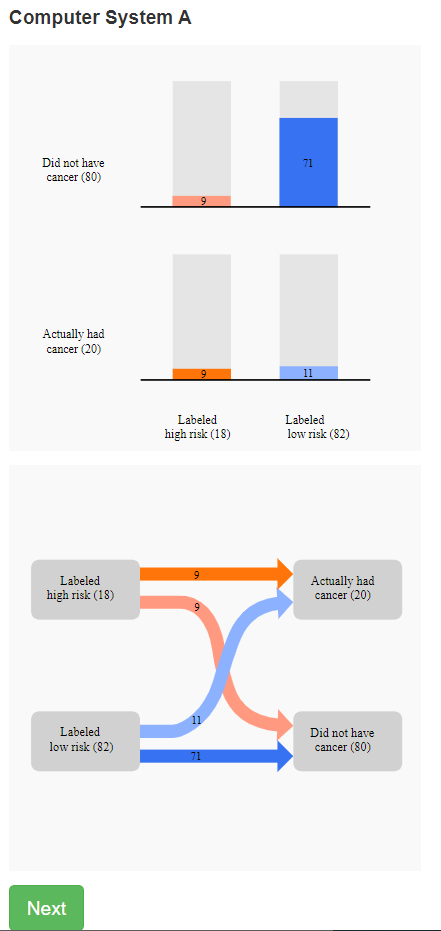}
    \caption{Description of cancer diagnosis along with visualization of a confusion matrix.}
    \label{pme-fig:prelim}
\end{figure}

On the next four pages, we show visualizations of two confusion matrices side by side and ask a series of questions regarding the data depicted in them. The first three adapt the questions from Shen et al.~\cite{shen2020designing} for the cancer diagnosis task. See Figures~\ref{pme-fig:q1}-\ref{pme-fig:q3} for the UI snapshots. Shen et al.~\cite{shen2020designing} framed these questions to evaluate the \emph{comprehension}, \emph{comparison}\footnote{The comparison questions in Shen et al.~\cite{shen2020designing} are different than pairwise comparisons like ours. They focus on comparing just one component, e.g., true positives, at a time.}, and \emph{simulation}-based understanding of the subjects. We use these questions to make them familiarize with the visualizations. The fourth page asks the subjects to actually compare two hypothetically created confusion matrices (see Figure~\ref{pme-fig:q4}). Here, one of the matrices has both higher false positives and false negatives. This question has a definitive answer and was added to make the subjects familiarize with the type of pairwise comparison questions that would follow. In addition, this question indicates how good the subject has grasped the context around cancer diagnosis and the task of pairwise comparisons.  

\begin{figure}
    \centering
    \includegraphics[scale=0.6]{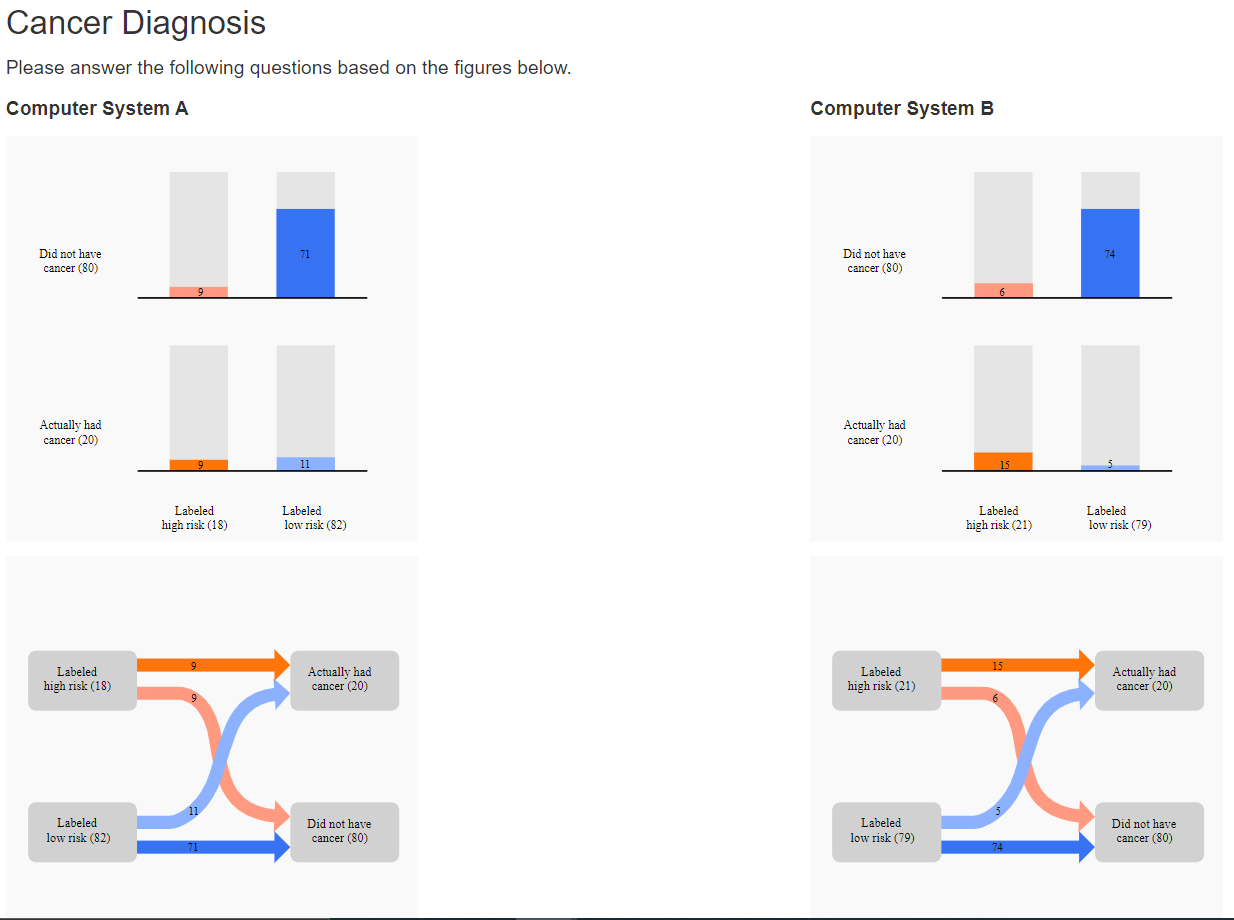}
    \includegraphics[scale=0.6]{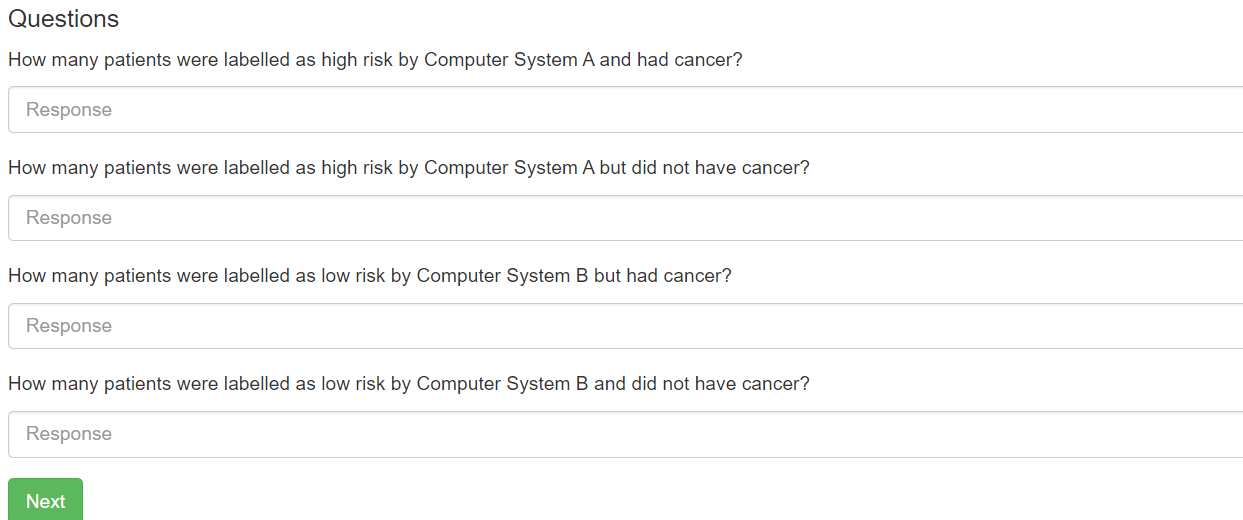}
    \caption{Two confusion matrices side by side. This page asks questions about \emph{comprehending} confusion matrices.}
    \label{pme-fig:q1}
\end{figure}

\begin{figure}
    \centering
    \includegraphics[scale=0.6]{practical/plots/q1_1.PNG}
    \includegraphics[scale=0.6]{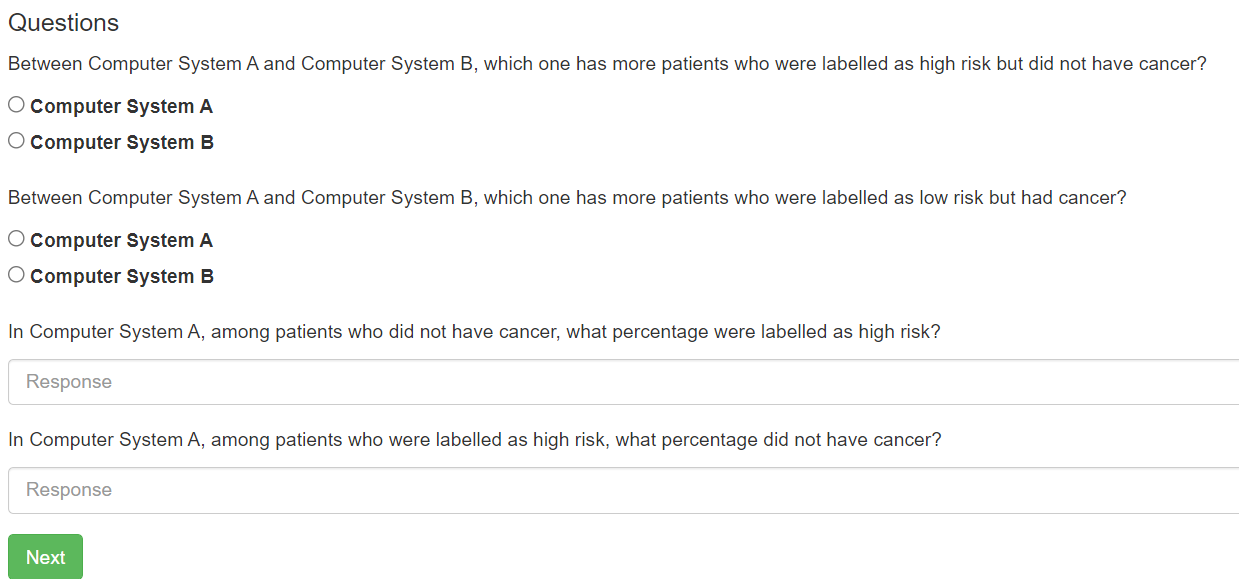}
    \caption{Two confusion matrices side by side. This page asks questions about \emph{comparing} the values in the two confusion matrices.}
    \label{pme-fig:q2}
\end{figure}

\begin{figure}
    \centering
    \includegraphics[scale=0.6]{practical/plots/q1_1.PNG}
    \includegraphics[scale=0.6]{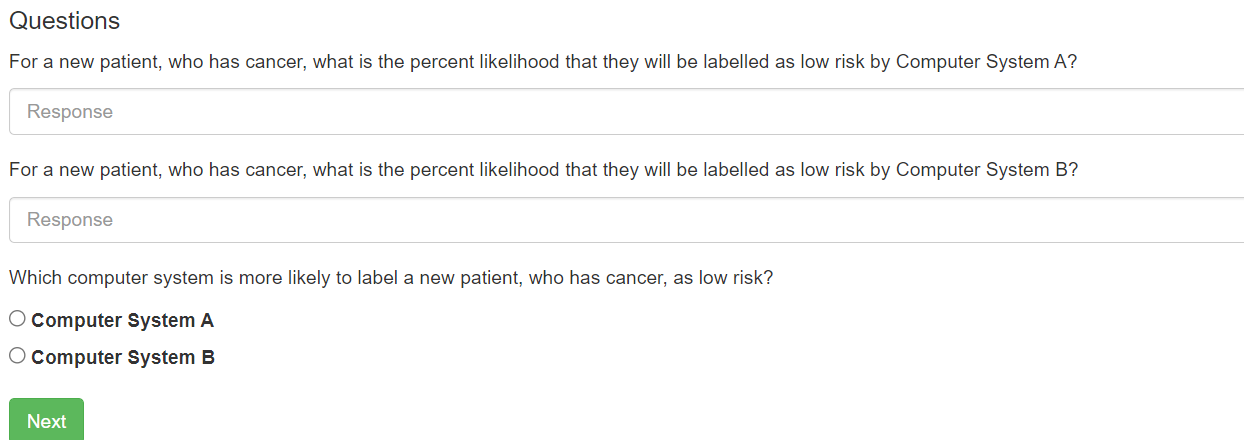}
    \caption{Two confusion matrices side by side. This page asks questions about \emph{simulating} a scenario based on the values in the two confusion matrices.}
    \label{pme-fig:q3}
\end{figure}

\begin{figure}
    \centering
    \includegraphics[scale=0.6]{practical/plots/q1_1.PNG}
    \includegraphics[scale=0.6]{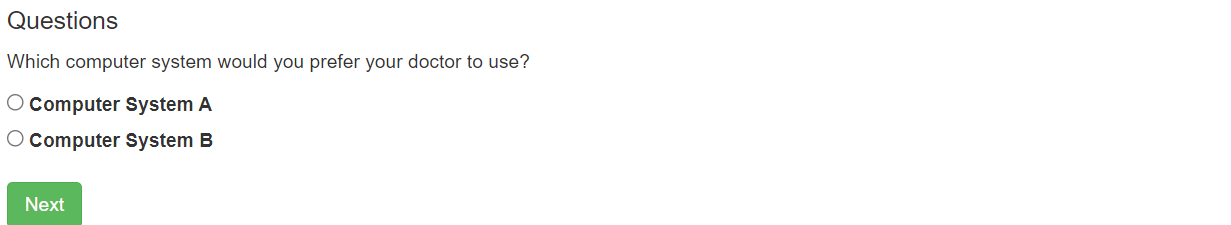}
    \caption{Two hypothetically created confusion matrices side by side. This page asks questions about \emph{pairwise comparison} of the confusion matices. Since one of the confusion matrices is worse in both false positives and false negatives, this question has a definitive answer.}
    \label{pme-fig:q4}
\end{figure}

\subsection{Practically Eliciting Linear Performance Metrics}
\label{pme-ssec:me}

We next explain the second phase of the UI, where we actually ask subjects for pairwise preferences over confusion matrices, and implement our binary-search procedure from Algorithm~\ref{bin-alg:linear}. The confusion matrices used for this procedure are from the smoothened upper boundary shown in Figure~\ref{pme-fig:confusions}; thus, the subjects have to make a choice reflecting on the trade-off between false positives and false negatives. Algorithm~\ref{bin-alg:linear} takes in real-time preferences of the subjects, generates next set of queries based on the current responses, and converge to a linear performance metric at the back end. We save this (linear) performance metric for each subject. We stop the binary-search when the search interval becomes less than or equal to 0.05 ($\epsilon$ in line 3 of Algorithm~\ref{bin-alg:linear}). Moreover, in practice, we do not need to ask four queries per round of binary search; instead, we can reduce the search interval into half by just using at most three pairwise queries in each round (i.e., by querying $\Omega(\Cbar_{\theta_c}, \Cbar_{\theta_a}), \Omega(\Cbar_{\theta_d}, \Cbar_{\theta_c}), \Omega(\Cbar_{\theta_e}, \Cbar_{\theta_d}),$ in line 6 of Algorithm~\ref{bin-alg:linear}). A sample of a pairwise comparison query from a run of the binary search algorithm in the UI is shown in Figure~\ref{pme-fig:me}.

\begin{figure}
    \centering
    \includegraphics[scale=0.5]{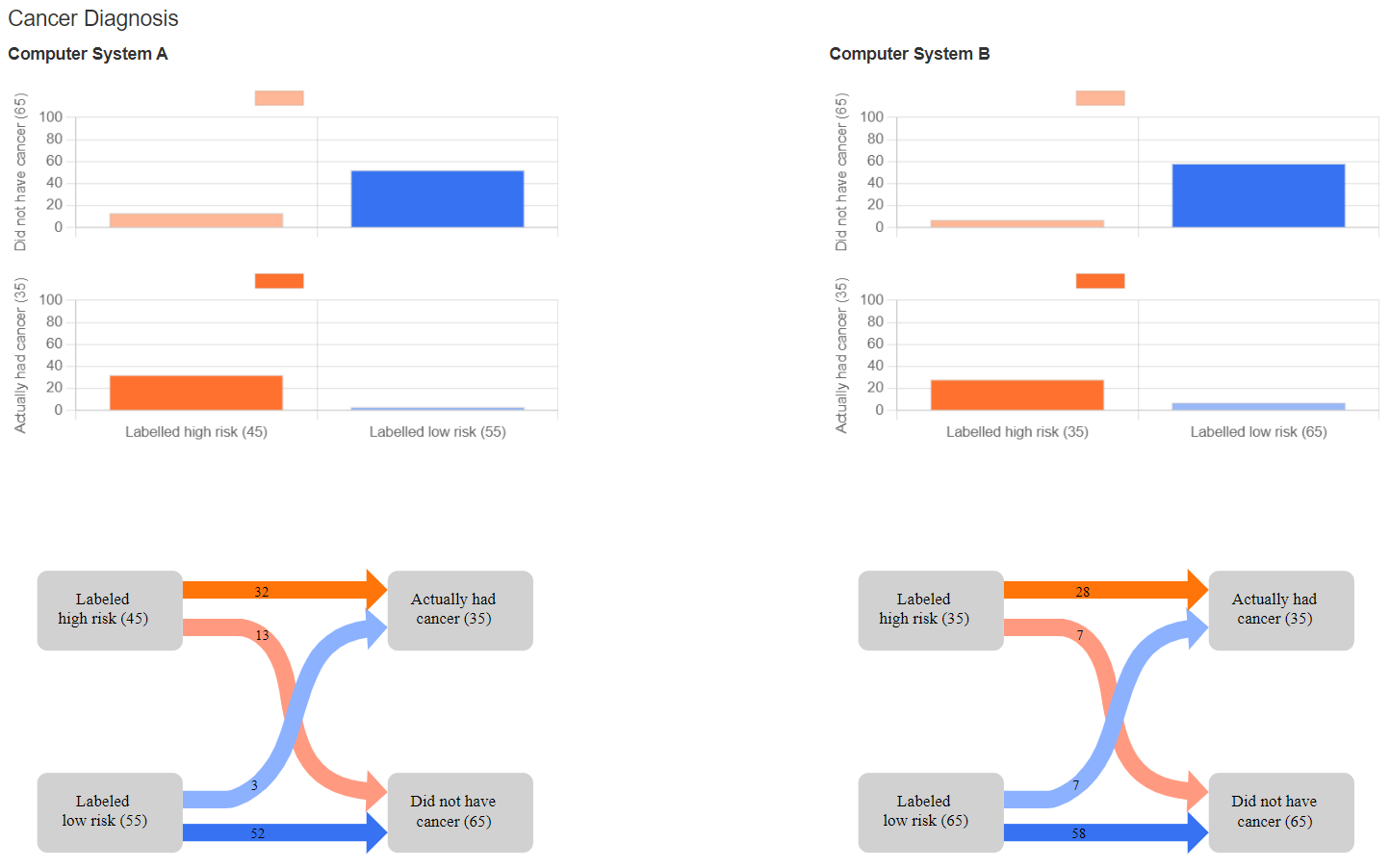}
    \includegraphics[scale=0.5]{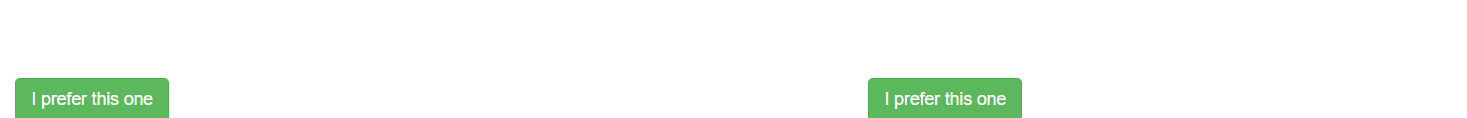}
    \caption{A sample of a pairwise comparison query from a run of the binary-search based procedure Algorithm~\ref{bin-alg:linear}.}
    \label{pme-fig:me}
\end{figure}

\subsection{Pairwise Preferences on a Random Set of Queries}
\label{pme-ssec:eval}

In order to evaluate the quality of the recovered metric, we ask the subjects fifteen pairwise comparison queries, each on a separate web page, right after the binary search algorithm  has converged, and we have elicited the metric. The subjects do not know this information and are shown evaluation queries in continuation to the previous phase (i.e., the binary search). The query comprises of two randomly selected confusion matrices that lie inside the feasible region. The confusion matrices are generated from a sphere of radius 0.1 around the center (0.35/2, 0.65/2). This set of  queries are used to evaluate the effectiveness of the elicited metric. We compute the fraction of times our elicited metric's preferences matches with the subject's preferences on these fifteen queries. A sample of a pairwise comparison query from this phase of the UI is shown in  Figure~\ref{pme-fig:eval}. We ask fifteen such queries. 

\begin{figure}
    \centering
    \includegraphics[scale=0.5]{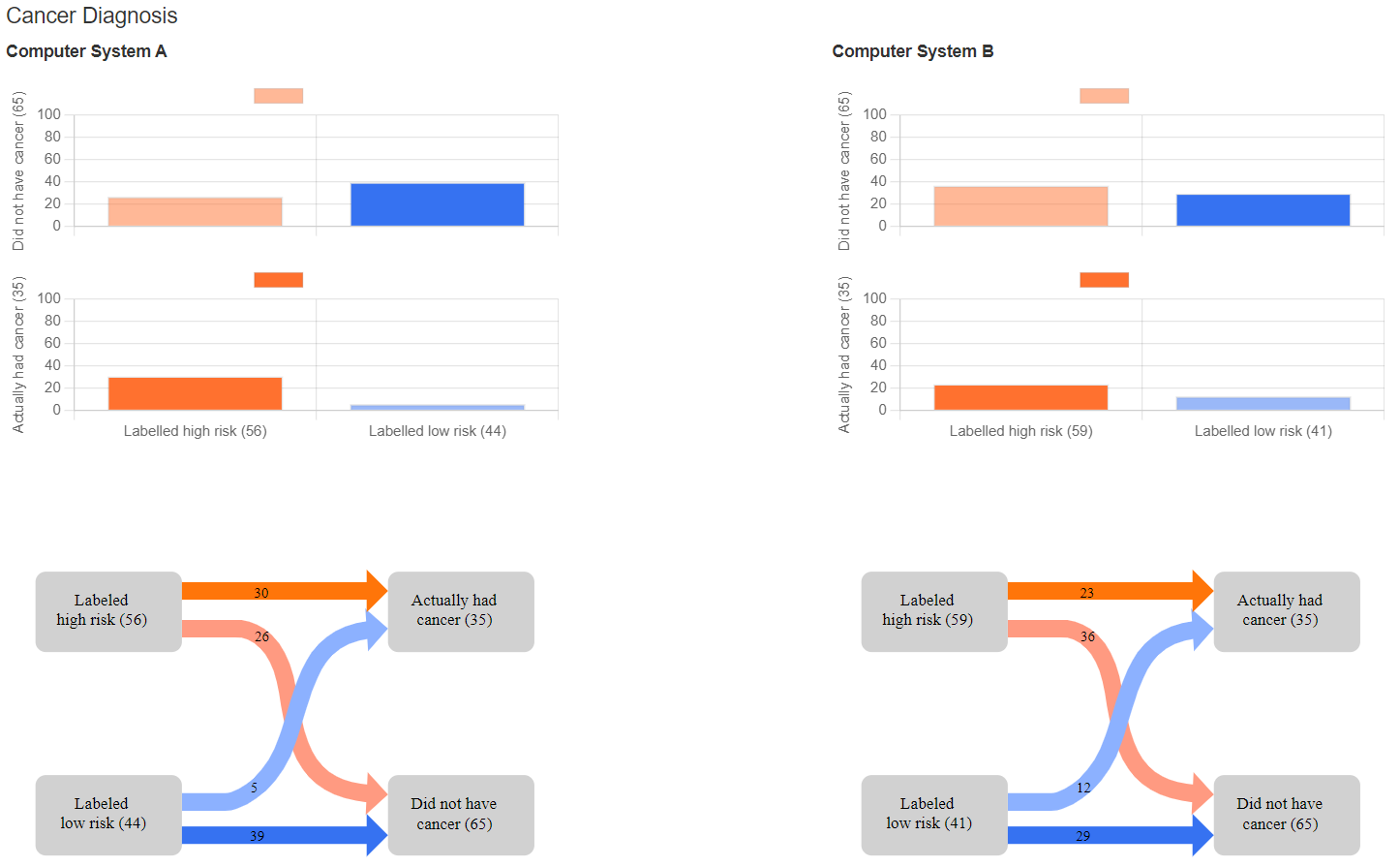}
    \includegraphics[scale=0.5]{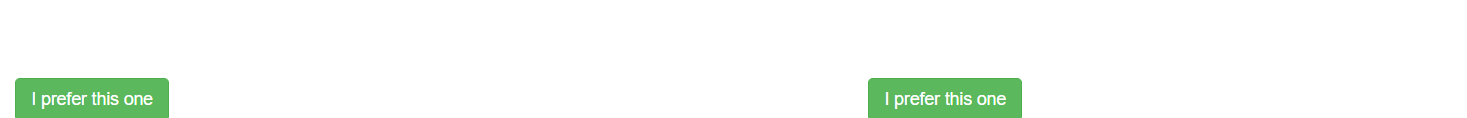}
    \caption{A sample of a pairwise comparison query comprising of randomly selected confusion matrices in the feasible region. These queries are used in evaluating the quality of the recovered metric.}
    \label{pme-fig:eval}
\end{figure}

\section{User Study}
\label{sec:userstudy}

We hired ten subjects in total for this preliminary study. The study was conducted over a video call, where the participants were asked to share the screen after they had filled the questionnaire on the first page. The distributions of the responses from the questionnaire are provided in Table~\ref{pme-tab:questionnaire}. The rest of the responses regarding the confusion matrices were over screen share and were logged in the UI. After the task was done, the web UI showed a `thank you' page and asked the subjects to close the web browser and screen share. The subjects were then asked post-task, \emph{think-aloud} interview questions, which are shown in Table~\ref{pme-tab:posttask}, to reflect on how they performed the given task. The responses from the interviews help us formulate guidelines and recommendations for future research in this direction. 

\begin{table}[t]
    \centering
    \caption{Subjects' demographics: Distribution of responses from the questionnaire. The values in parenthesis show the number of subjects. }
    \begin{tabular}{|c|c|c|c|}
    \hline
         \textbf{Age} & 25 (2) & 26 (3) & 28 (5) \\
         \textbf{Education Level} & in Graduate College (4) & Master's (3) & Doctorate (3) \\
         \textbf{ML Expertise} & None (5) & Beginner (3) & Intermediate (2) \\
         \textbf{Healthcare Knowledge} & None (5) & Some (2) & No response (3)\\
    \hline
    \end{tabular}
    \label{pme-tab:questionnaire}
\end{table}

\begin{table}[t]
    \centering
    \caption{Post-task interview questions.}
    \begin{tabular}{|p{0.035\linewidth}|p{0.91\linewidth}|}
    \hline
         \textbf{Q1} &  What do you think is worse: (a) Large number of patients that actually have cancer but are labelled as low risk by a computer system, or 
(b) Large number of patients that do not have cancer but are labelled as high risk by a computer system.
\\
         \textbf{Q2} & Could you quantify how much worse the chosen option is in comparison to the other? Why or why not? Could you quantify this personally? i.e, 10x worse for me
 \\
         \textbf{Q3} & For the questions presented in this task, how did you decide which system you would prefer your doctor to use?
 \\
         \textbf{Q4} & What was difficult about making these choices?
 \\
         \textbf{Q5} & What additional information would have helped you to make these choices?
 \\
         \textbf{Q6} & Do you have any feedback for us on your experience today? 
\\
    \hline
    \end{tabular}
    \label{pme-tab:posttask}
\end{table}

%% file: practical/results.tex
\section{Results}
\label{pme-sec:results}

In this section, we discuss results from the preliminary user-study both quantitatively and qualitatively. We will try to answer some of the practical questions that surround the metric elicitation framework as discussed in the beginning of this chapter. Specifically, we focus on checking workflow of the practical implementation, support or reject the hypothesis that the implicit user preferences can be quantified using the pairwise comparison queries, testing assumptions regarding the noise model, work around with finite samples, visualizing confusion matrices for pairwise comparisons, eliciting actual performance metrics in real-life scenarios, and evaluating the quality of the recovered metric. We emphasize that the aim behind discussing results from the user study is to formulate guidelines and recommendations for future research on practical metric elicitation. We provide these recommendations as we discuss quantitative and qualitative results and summarize them in Table~\ref{pme-tab:guidance}. 

\begin{table}[t]
    \centering
    \caption{Summary of guidelines and recommendations from the user-study.}
    \begin{tabular}{|p{0.04\linewidth}|p{0.9\linewidth}|}
    \hline
         \textbf{G1} &  Whenever possible, smoothen the query space so to run the binary-search based algorithms with reduced finite sample errors.\\
         \textbf{G2} &  Depending on the search tolerance of the binary-search, show probabilities in the confusion matrix as out-of-$n$ samples, where bigger the $n$, the better it is to differentiate between confusion matrices in a query.\\
         \textbf{G3} & The direction in the flow-chart based visualization of the confusion matrix can be swapped with total number of labels shown in the left column and total predictions on the right. \\
         \textbf{G4} & Perhaps, showing only flow-chart for pairwise comparisons is better than showing flow-chart and bar-chart together. One may also just show, the false positives and false negatives to further reduce the information load.\\
         \textbf{G5} & Measure time to respond for each query. Spending more time on queries that comprise close confusion matrices lead credence to the noise model in Definition~\ref{me-def:noise}.\\
         \textbf{G6} & The terminology ``labelled as high risk/low risk" can be replaced with ``predicted as high risk/low risk" to avoid confusions regarding ground-truth label.\\
         \textbf{G7} & In view of the post-interview question number 2, one needs to devise a UI so to ask for the intuitive guess for the false negative cost. This would also act as a baseline metric for evaluation purposes (see Section~\ref{pme-ssec:quant}).\\
         \textbf{G8} & One can also have a toggle button that shows percentages conditioned on the true classes (i.e., in addition to false positive and false negative, one can have false positive rate and false negative rate). This would aid in making comparisons.\\
         \textbf{G9} & Extend the description on cancer diagnosis and mention the associated (subjective) cost or excerpts that cover different aspects of the cost. For example, how much financial burden a false positive prediction would put on a patient, how much emotional burden would it put, what are the possible side-effects of drugs, etc. \\ 
    \hline
    \end{tabular}
    \label{pme-tab:guidance}
\end{table}

\subsection{Quantitative Results and Findings}
\label{pme-ssec:quant}

\paragraph{Impact of Smoothened Query Space and Out-of-100 Samples:} We first discuss the impact of smoothening of the upper boundary from Section~\ref{pme-ssec:dataset}. Since we choose to ask pairwise preferences over confusion matrices directly, and not over classifiers, we provided a way to generate feasible confusion matrices in Section~\ref{pme-ssec:dataset} that lie on the smoothened version of the upper boundary. As we discussed in Section~\ref{bin-ssec:realexp}, working with finite samples has a drawback that the elicitation routine can get stuck at the closest achievable confusion matrix from finite samples, which need not be optimal within the given (small) tolerance. We find that working with the smoothened version almost always avoids asking pairs that comprise same confusion matrices, and thus guaranteeing better convergence within the chosen binary-search tolerance. We also note that showing probabilities in the form of out-of-10000 or bigger samples instead of out-of-samples 100 allows us to further reduce the cases where the confusion matrices are same in a pair or the comparisons becomes trivial (e.g., same false negatives but different false positives) for the subjects.

\paragraph{Elicited Metrics and Quality Evaluation:} We next discuss the metrics that were elicited for the ten subjects using our web UI, which runs the binary-search based procedure Algorithm~\ref{bin-alg:linear} at the back end. Once the search interval is less than or equal to 0.05, the subjects were asked fifteen queries that we use for evaluation. The measure of effectiveness that we choose is the fraction of times (in \%)  our elicited metric's preferences matches with the subject's preferences over the fifteen queries, i.e., 

\begin{equation}
\Mcal := \frac{\sum_{i=1}^{15} \1[\text{subject's prefer. for query } i == \text{metric's prefer. for query } i]}{15} \times 100.
    \label{pme-eq:fraction}
\end{equation}

We show the elicited metric for the fifteen subjects and the measure $\Mcal$ values in Table~\ref{pme-tab:metrics}. We see for nine out of ten subjects that more than 85\% of the time our elicited metric's preferences matches with the subject's preferences on the fifteen evaluation queries. For three subjects, our metric's preference matches exactly for all the evaluation queries. 

\begin{table}[t]
    \centering
    \caption{The elicited linear performance metrics for the ten subjects along with the fraction of times (in \%) the elicited metric's preferences matches with the subject's preferences over the fifteen evaluation queries.}
    \begin{tabular}{|c|c|c|}
    \hline
    \textbf{Subjects} & \textbf{Linear Performance Metric} & $\Mcal$ \\
    \hline
         S1 & 0.125 \text{TN} + 0.875 \text{TP}  & 87\\
         S2 & 0.141 \text{TN} + 0.859 \text{TP}  & 100\\
         S3 & 0.125 \text{TN} + 0.875 \text{TP}  & 93\\
         S4 & 0.141 \text{TN} + 0.859 \text{TP}  & 100\\
         S5 & 0.328 \text{TN} + 0.672 \text{TP}  & 73\\
         S6 & 0.031 \text{TN} + 0.969 \text{TP}  & 87\\
         S7 & 0.031 \text{TN} + 0.969 \text{TP}  & 100\\
         S8 & 0.359 \text{TN} + 0.641 \text{TP}  & 87\\
         S9 & 0.125 \text{TN} + 0.875 \text{TP}  & 93\\
         S10 & 0.141 \text{TN} + 0.859 \text{TP}  & 87\\
    \hline
    \end{tabular}
    \label{pme-tab:metrics}
\end{table}

The absolute numbers for the $\Mcal$ measure look good; however, how good they are is still a missing piece in this study because of the lack of a baseline. In future, we plan to devise ways to develop a baseline for the metric elicitation task and compare to that baseline on the measure $\Mcal$.

\subsection{Qualitative Feedback}
\label{pme-ssec:qual}

We first describe the general feedback that was observed and discussed with the subjects during the user study over the video sessions. We formulate some guidelines from this feedback. We then mention a few excerpts from the post-task interviews again formulating some recommendations for practical metric elicitation. 

\paragraph{Observations during Study Sessions:} Similar to the observation by Shen et al.~\cite{shen2020designing}, in our user study as well, we also noted that subjects were not very comfortable with answering the \emph{simulation}-based questions (see Figure~\ref{pme-fig:q3}). A possible reason is that the direction of the flow-chart is opposite to the conditioning of probability that is asked in those questions. Bar-chart allows them to answer this question easily; however, we find that by this point in the UI, the subject becomes more comfortable with using the flow-chart. Some users when asked in the post-interview session also mentioned that this could help them better in the pairwise comparison, too.

While comparing confusion matrices in the UI, we observed that after a few rounds, the subjects tend to look at only the flow-charts for comparison. This may mean showing the bar-charts and flow-charts together is overwhelming, and perhaps only the flow-charts are enough. After a few more rounds, some subjects started comparing only flow of  false positives and false negatives in the flow-chart. This suggests that one may further reduce the information load by showing only false positives and false negatives in the flow chart.

Although, we do not quantitatively measure \emph{time to respond} in this version of the UI, but we did observe that the subjects tend to take more time while comparing two confusion matrices that are close (i.e., the queries in the later part of the binary search when the search interval is narrow). This means that the subjects are more prone to make errors for such queries, leading credence to the noise model in Definition~\ref{me-def:noise} that is used in this manuscript throughout. 

Lastly, during the study, we found that some subjects, who were familiar with machine learning, confused the terminology ``labelled as high risk/low risk" for predictions to the ground-truth labels. One suggestion is to replace the word ``labelled" with ``predicted".

\paragraph{Post-task Interview Sessions:} We now discuss post-task interviews and formulate some guidelines. We also mention some excerpts (anonymously) from the interviews. Please see Table~\ref{pme-tab:posttask} for the interview questions.

\textbf{Q1.} Every subject clearly figured out the direction of the costs and mentioned that (in the words of S1), \emph{``a patient who has cancer but was predicted as low risk is a costlier mistake than a patient who does not have cancer but was predicted as high risk."}

\textbf{Q2.} None of the subjects could answer this question with full confidence. This acts as a testimony to the importance of the metric elicitation framework. Often, practitioners make a guess to quantify the asymmetric costs in class-imbalanced learning; however, the guess may be far from innate costs of the practitioner. The subjects agreed that it is easier to compare two confusion matrices using the proposed visualizations than to answer this question. 

\textbf{Q3.} Most of the subjects mention that they preferred the one where false negatives were less. Although some subjects looked at the trade-off, for example, (in the words of S2) \emph{``I was trying to minimize the false negatives but not when very large number of false positives were there."} This reflects that some subjects had to think hard about the trade-offs.

\textbf{Q4.} The subjects mention that deciding on the trade-offs between false positives and false negatives was difficult. (In words of S6) \emph{``It was difficult to pick a preference where both false positives and false negatives needed to be compared"}. Some subjects also mentioned that, (in words of S4), \emph{``In some cases, numbers are really close; thus, it becomes difficult to select one of them"}. This feedback certainly agrees with the choice of the noise model in this  manuscript (see Definition~\ref{me-def:noise}). 

\textbf{Q5.} The responses to this question were important for constructing the guidelines, and this question had varied responses. One subject mentioned that having false positive rate and false negative rate, in addition to false positives and false negatives, would be helpful in making comparisons. (In words of S1), \emph{``One can have percentages on the arrow conditioned on the samples in the box from which they are flowing."} Similarly, some subjects mentioned that it would have been easier to compare if the stages of cancer were mentioned in the predictions; the different stages would have lead to difference preferences. Some subjects quote that some description of the associated costs or excerpts that cover different aspects of the cost, at least subjectively should be described in the beginning of the study. For example, (in words of S2), \emph{``how much financial burden a false positive prediction would put on a patient, how much emotional burden would it put, what are the possible side-effects of drugs, etc. should be highlighted in the beginning."}  

\textbf{Q6.} Most subjects enjoyed the exercise and liked the web UI. Some subjects mentioned that the task allowed them to reflect closely on some important questions regarding performance metrics in machine learning.

%% file: practical/conclusion.tex
\section{Concluding Remarks}
\label{pme-sec:conclusion}

We created a web user-interface (UI) to practically elicit (linear) performance metrics with real users in a binary classification setup. We chose cancer diagnosis as the task domain, because it involves asymmetric costs for false positives and false negatives. We build upon existing visualizations of confusion matrices that are refined to capture preferences over pairwise comparisons. Via this user-study, we demonstrated an implementation of the binary performance metric elicitation procedure from Chapter~\ref{chp:binary} that make use of the real-time user responses over pairwise comparisons of confusion matrices. We also proposed and implemented an evaluation scheme to judge the quality of the recovered metric. 

Using the proposed web UI, we then conducted a preliminary user study with ten subjects and elicited their linear performance metrics. We also compared the quality of the recovered metric  by comparing their responses to the elicited metric's responses over a set of randomly chosen pairwise comparison queries. The study also included a post-task, \emph{think-aloud}-style interviews regarding the utility of the framework. Using the task results and the feedback during the post-task interviews, we presented guidelines and recommendations for practical implementation of the ME framework. In the future, we plan to build upon this pilot study and conduct a comprehensive user study that includes the guidelines presented in this chapter with more subjects. We also plan to extend the current web UI to elicit metrics in the multiclass classification setup.

%% file: conclusions.tex
\chapter{Conclusion and Future Work}
\label{chp:conclusion}

Typical default metrics in machine learning, such as accuracy applied to classification tasks, may not capture tradeoffs relevant to the problem at hand. Thus, optimizing such default metrics can have an undesirable impact on short and long-term utility, including the fairness of the resulting predictions across sensitive subgroups since the same issues plague default fairness
measures. In this thesis, we formalized the problem of \emph{Metric Elicitation (ME)} and proposed it as a principled framework for determining supervised classification metrics from user feedback. Through theoretical and empirical avenues, we showed that under certain conditions metric elicitation is equivalent to learning preferences between pairs of classifier statistics. 

When the underlying metric is linear in the binary classification setup, we proposed an elicitation strategy to recover the oracle's metric, whose query complexity decays logarithmically with the desired resolution. We also showed that our query-complexity rates match the lower bound.  We further extended our strategies to eliciting linear-fractional binary classification performance metrics. 

We then broadened the scope of metric elicitation by proposing ME strategies for the more complicated multiclass classification setting. We proposed two algorithms for multiclass classification metric elicitation that use multiple binary-search subroutines that recover the oracle's linear metric. One of the proposed algorithms assumes that the oracle's metric is dependent on only the diagonal entries of the confusion matrices (a unique sparsity condition on the metric), and thus is useful when the number of classes is large. Similar to the binary case, we further provided algorithms for eliciting  linear-fractional multiclass classification performance metrics. 

With respect to applications to fairness, we  devised a novel strategy to elicit group-fair performance metrics for multiclass classification problems with multiple sensitive groups that also includes selecting the trade-off between predictive performance and fairness violation. The procedure exploited the \emph{piecewise} linearity of the metric in group-specific predictive rates, used binary-search based subroutines, and recovered the metric with linear query complexity. It was interesting to note that we were able to elicit a non-linear metric while maintaining the same query complexity order (linear in the number of unknowns) as the linear elicitation case. 

We then used the tools and geometric characterizations build so far to solve three important problems that benefit the practical aspects of the proposed ME framework. The first involved increasing the complexity of the elicited metrics. The second was to exploit the current linear elicitation framework so to train deep neural networks for optimizing black-box metrics. The third was to conduct real-user study in order to elicit real-user metrics and reflect on the practical nuances of the ME framework. We draw out conclusions from each of these applications below.

The ME strategies for linear or quasi-linear functions of classifier statistics, can be restrictive in domains  where the metrics are more complex and nuanced. 
Thus, we proposed novel strategies for eliciting metrics defined by \emph{quadratic} functions of classifier statistics, which can easily be applied to fair metric elicitation setups as well. We were thus able to handle a more general family of metrics that can better capture a practitioner's innate preferences. We further generalized quadratic elicitation strategy to higher-order polynomial functions. All our metric elicitation procedures were shown to be robust to both  finite sample and oracle feedback noise. 

We then considered learning to optimize a classification metric defined by a black-box function of the confusion matrix. We proposed the Frank Wolfe with Elicited Gradient (FW-EG) method for optimizing black-box metrics given query access to the evaluation metric on a small validation set. Our framework included common distribution shift settings as special cases, and unlike prior distribution correction strategies, was able to handle general non-linear metrics. We showed how to model and estimate the example weights, but more importantly, we exploited the fact that the example weights can be seen as a gradient for the metric and estimated through metric elicitation procedure  in the presence of a \emph{machine} oracle. Experiments on various label noise, domain shift, and fair classification setups confirmed that our proposal compares favorably to the state-of-the-art baselines for each application. We briefly discussed how this procedure can be extended to optimize black-box metrics in the presence of a  \emph{human} oracle providing pairwise comparison feedback.

 Lastly, we created a web UI for eliciting binary classification performance metrics that incorporates enhanced visualizations of confusion matrices for obtaining pairwise feedback. We then conducted a preliminary user-study in the binary classification setup in order to elicit real-users' performance metrics. In the process, we touched upon several practical aspects related to ME. In particular, we focused on checking workflow of the practical implementation, found support for the hypothesis that the implicit user preferences can be quantified using pairwise comparison queries, tested assumptions regarding the noise model, worked around with finite samples, elicited actual performance metrics in real-life scenarios, and evaluated the quality of the recovered metric. Using the quantitative and qualitative results from the pilot study, we formulated several guidelines and recommendations for practically implementing the metric elcitiation framework.
 
 \begin{figure}[t]
	\centering 
		\includegraphics[scale=0.6]{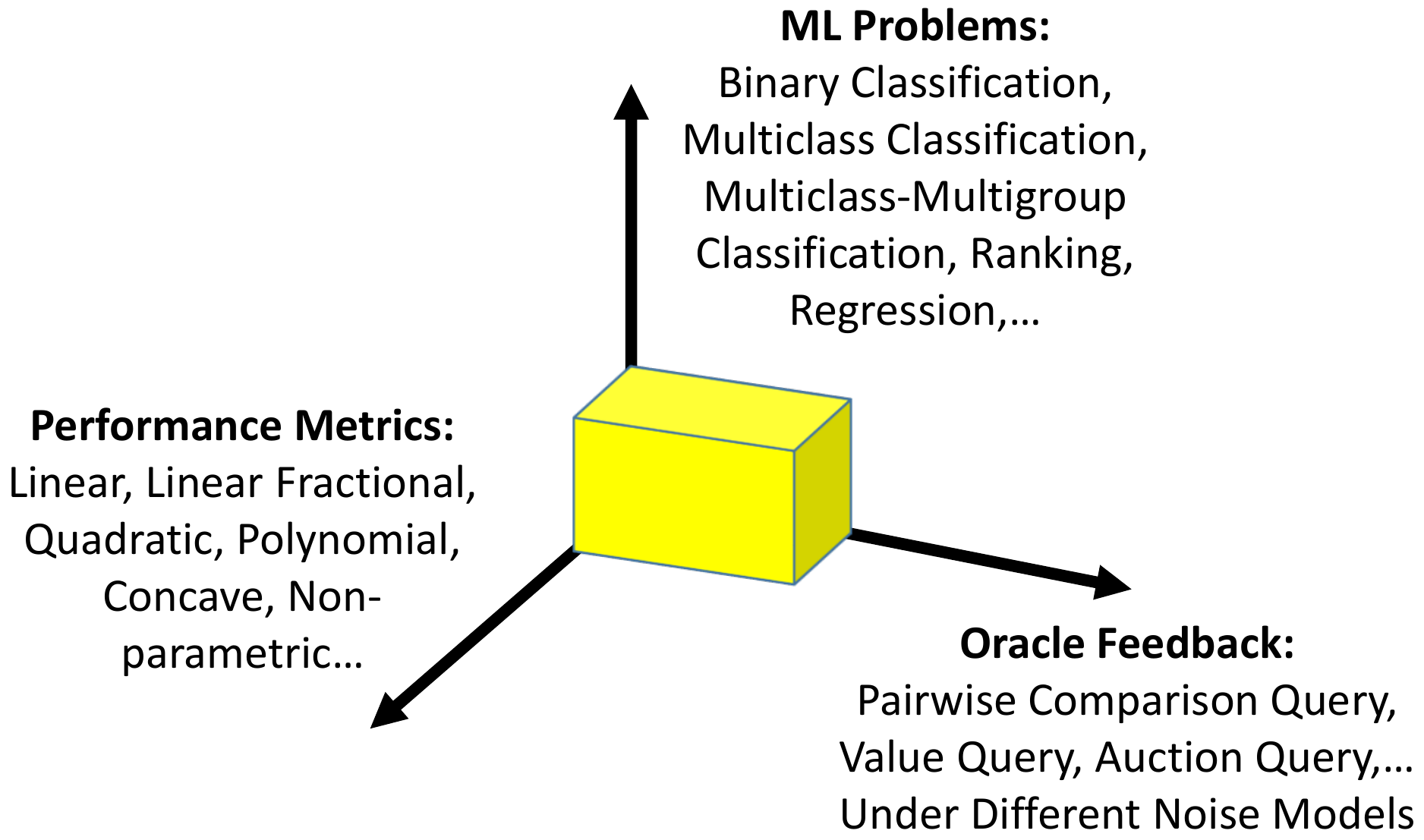}
		\vspace{-0.1cm}
	\caption{\emph{Metric Elicitation for Predictive Machine Learning - Vision:} The three axes show three different nuances of metric elicitation. The first axis contain different predictive machine learning problems. On the second axis, there are various forms of performance metrics that can be elicited. Several oracle feedback and noise models lie on the third axis. This thesis provides solution to the box (shown in yellow color) covering a few parts of the larger problem of metric elicitation.}
	\label{fig:thesis}
	\vskip -0.1cm
\end{figure}
 
We envision the problem of \emph{metric elicitation} to be an important, interesting, and challenging topic for the future with many practical applications in the broad field of artificial intelligence. The underlying space of open problems can be broken into three separate axes. The axes are shown in Figure~\ref{fig:thesis}. On the first axis, there are different predictive machine learning problems such as classification, regression, ranking, etc. Each type of predictive problem involves new frontiers to be explored and exploited like we have done in this manuscript. For example, to elicit ranking metrics, one may require a thorough understanding of the space of statistics that summarize ranking effects. On the second axis, one may deal with various functional forms of performance metrics that can be elicited. Currently, we have focused on eliciting quasi-linear and polynomial functions of classifier statistics. Metric elicitation becomes much more challenging yet more practical when the functional forms are not assumed. The third axis stretches to different forms of oracle queries including various noise models. This direction guarantees the applicability of metric elicitation for real-world scenarios. The expected contribution in the future would be to solve the entire space of problems comprising the three axes, which may then result in a separate sub-field of artificial intelligence under the name -- \emph{Metric Elicitation for Predictive Machine Learning.} Once the metrics are elicited, sophisticated methods may be created to optimize those metrics similar to Chapter~\ref{chp:blackbox}. Thus this entire line of work will answer important open questions in machine learning, impact several multi-disciplinary applications, and transform the way machine learning systems are deployed in practice. 

%% file: supp_main.tex
\chapter{Binary Classification Performance Metric Elicitation}
\label{apx:binary}
\input{binary/supplement}  

\chapter{Multiclass Classification Performance Metric Elicitation}
\label{apx:multiclass}
\input{multiclass/supplement}

\chapter{Fair Performance Metric Elicitation}
\label{apx:fair}
\input{fair/supplement}  

\chapter{Quadratic Performance Metric Elicitation}
\label{apx:quadratic}
\input{quadratic/supplement}  

\chapter{Optimizing Black-box Metrics through Metric Elicitation}
\label{apx:blackbox}
\input{blackbox/supplement}

%% file: binary/supplement.tex
\section{Visualizing the Set of Confusion Matrices}
\label{appendix:visualization}
To clarify the geometry of the feasible set, we visualize one instance of the set of confusion matrices $\Ccal$ using the dual representation of the supporting hyperplanes. The steps are:
\benumerate[wide, labelwidth=!, labelindent=0pt]
\item \emph{Population Model:} We assume a joint probability for $\Xcal = [-1,1]$ and $\Ycal = \{0, 1\}$ given by
\begin{ceqn}
\begin{equation}
f_X = \Umbb[-1,1] \quad \text{and} \quad \eta(x) = \frac{1}{1 + e^{ax}},
\label{prob-dist}
\end{equation}
\end{ceqn}
where $\Umbb[-1,1]$ is the uniform distribution on $[-1, 1]$ and $a>0$ is a parameter controlling the degree of noise in the labels. If $a$ is large, then with high probability, the true label is $1$ on [-1, 0] and $0$ on [0, 1]. On the contrary, if $a$ is small, then there are no separable regions and the classes are mixed in $[-1,1]$.  

Furthermore, the integral $\int_{-1}^{1} \frac{1}{1 + e^{ax}}dx = 1$ for $a \in \Rmbb$ implying $ \Pmbb(Y = 1) = \zeta = \frac{1}{2} \; \forall \; a \in \Rmbb$.
\item \emph{Generate Hyperplanes:} Take $\theta \in [0, 2\pi]$ and set $\mmbf = (m_{11}, m_{00}) = (\cos\theta, \sin\theta)$. Let us denote $x'$ as the point where the probability of positive class $\eta(x)$ is equal to the optimal threshold of Proposition \ref{pr:bayeslinear}. Solving for $x$ in the equation $1/(1 + e^{ax}) = m_{00}/(m_{00} + m_{11})$ gives us
\begin{ceqn}
\begin{align}
x' &= \Pi_{[-1, 1]} \big\{\tfrac{1}{a}\ln\big(\tfrac{m_{11}}{m_{00}}\big)\big\},
\end{align}
\end{ceqn}
where $\Pi_{[-1,1]} \{z\}$ is the projection of $z$ on the interval $[-1,1]$. If $m_{11} + m_{00} \geq 0$, then the Bayes classifier $\hbar$ predicts class $1$ on the region $[-1, x']$ and $0$ on the remaining region. If $m_{11} + m_{00} < 0$, $\hbar$ does the opposite. Using the fact that $Y | X$ and $\hbar | X$ are independent, we have that
\begin{enumerate}[wide, labelwidth=!, labelindent=0pt]
	\item if $m_{11} + m_{00} \geq 0$, then 
	\bequation
	\oline{TP}_\mmbf = \frac{1}{2} \textstyle \int\limits_{-1}^{{x'}} \frac{1}{1 + e^{ax}}dx, \qquad \oline{TN}_\mmbf = \frac{1}{2} \int\limits_{{x'}}^{1} \frac{e^{ax}}{1 + e^{ax}}dx.
	\eequation
	\item if $m_{11} + m_{00} < 0$, then 
	\bequation
	\oline{TP}_\mmbf = \frac{1}{2} \textstyle\int\limits_{{x'}}^{1} \frac{1}{1 + e^{ax}}dx, \qquad \oline{TN}_\mmbf = \frac{1}{2} \int\limits_{-1}^{{x'}} \frac{e^{ax}}{1 + e^{ax}}dx.
	\eequation

Now, we can obtain the hyperplane as defined in \eqref{eq:support} for each $\theta$. 
We sample around thousand $\theta 's \in [0, 2\pi]$ 
randomly. We then obtain the hyperplanes following the above process and plot them.
\end{enumerate}

\begin{figure*}[t]
	\centering 
	\subfigure[a = 0.5]{
		{\includegraphics[width=5cm]{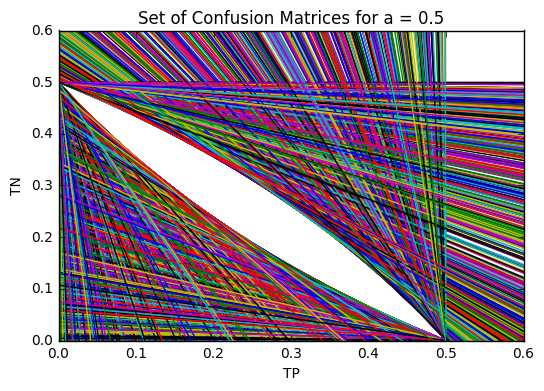}}
		\label{fig:cf_a_0_5}
	}
	\subfigure[a = 1]{
		{\includegraphics[width=5cm]{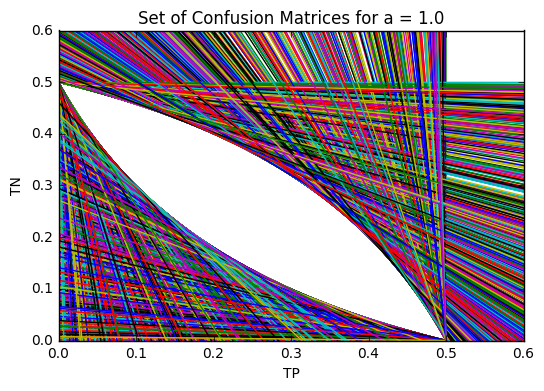}}
		\label{fig:cf_a_1}
	}
	\subfigure[a = 2]{
		{\includegraphics[width=5cm]{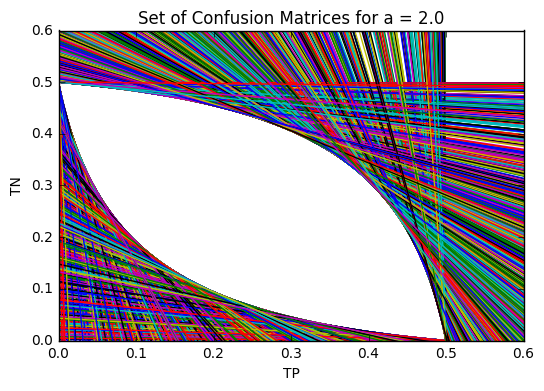}}
		\label{fig:cf_a_3}
	}
	\subfigure[a = 5]{
		{\includegraphics[width=5cm]{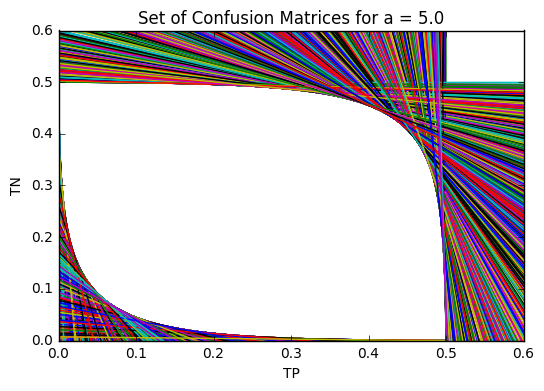}}
		\label{fig:cf_a_5}
	}
	\subfigure[a = 10]{
		{\includegraphics[width=5cm]{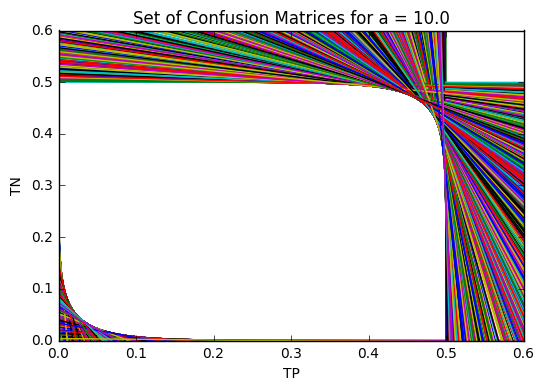}}
		\label{fig:cf_a_10}
	}
	\subfigure[a = 50]{
		{\includegraphics[width=5cm]{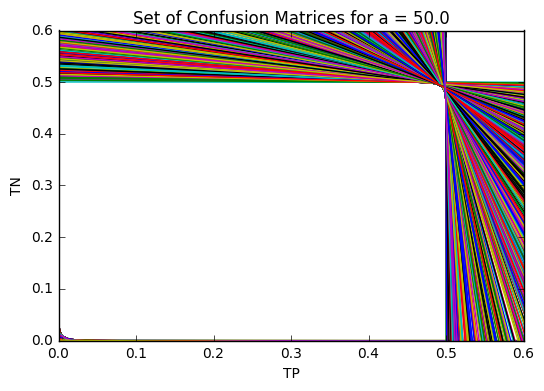}}
		\label{fig:cf_a_50}
	}
	\caption{Supporting hyperplanes and associated set of feasible confusion matrices for exponential model described in equation~\eqref{prob-dist} with $a = 0.5, 1, 2, 5, 10$ and $50$. The middle white region is $\Ccal$, which is the intersection of half-spaces associated with its supporting hyperplanes.}
	\label{fig:fs_cm}
	\end{figure*}
	
The sets of feasible confusion matrices $\Ccal$'s for $ a = 0.5, 1, 2, 5, 10$, and $50$ are shown in Figure \ref{fig:fs_cm}. The middle white region is $\Ccal$: the intersection of the half-spaces associated with its supporting hyperplanes. 
The curve on the right corresponds to the confusion matrices on the upper boundary $\partial\Ccal_+$. Similarly, the curve on the left corresponds to the confusion matrices on the lower boundary $\partial\Ccal_-$. Points $(\zeta, 0) = (\frac{1}{2}, 0)$ and $(0, 1 - \zeta) = (0, \frac{1}{2})$ are the two vertices. The geometry is 180-degree rotationally symmetric around the center point $(\frac{1}{4}, \frac{1}{4})$, which corresponds to the confusion matrix of the uniform random classifier, i.e., the classifier which predicts both classes with equal probability for any input.

Notice that as we increase the separability of the two classes via $a$, all the points in $[0, \zeta] \times [0, 1-\zeta]$ becomes feasible. In other words, if the data is completely separable, then the corners on the top-right and the bottom left are achievable. If the data is `inseparable', then the feasible set contains only the diagonal line joining $(0,\frac{1}{2})$ and $(\frac{1}{2},0)$, which passes through $(\frac{1}{4},\frac{1}{4})$.   
\eenumerate

\section{Proofs}\label{appendix:proofs}
\newcommand{\tp}{TP}
\newcommand{\tn}{TN}
\newcommand{\df}{\,\mathrm df_X}
\newcommand{\dfz}{\,\mathrm df_Z}
\begin{lemma}\label{lem:properties-C}
The feasible set of confusion matrices $\mathcal C$ has the following properties:
	\renewcommand{\theenumi}{(\roman{enumi})}
	\begin{enumerate}[leftmargin=1cm]
	\item For all $(\tp,\tn)\in \mathcal C$, $0\leq \tp\leq \zeta$, and $0\leq \tn\leq 1-\zeta$.
	\item $(\zeta,0)\in\mathcal C$ and $(0,1-\zeta)\in \mathcal C$.
	\item For all $(\tp,\tn) \in \mathcal C$, $(\zeta-\tp, 1-\zeta-\tn)\in \mathcal C$.
	\item $\mathcal C$ is convex.
	\item $\mathcal C$ has a supporting hyperplane associated to every normal vector. 
	\item Any supporting hyperplane with positive slope is tangent to $\mathcal C$
	at $(0,1-\zeta)$ or $(\zeta,0)$.
	\end{enumerate}
\end{lemma}
\begin{proof} We prove the statements as follows: 

\begin{enumerate}[label=(\roman*)., leftmargin=1cm]
\item $0\leq \Pmbb[h=Y=1]\leq \Pmbb[Y=1]=\zeta$, and similarly, $0\leq \Pmbb[h=Y=0]\leq \Pmbb[Y=0]=1-\zeta$.

\item If $h$ is the trivial classifier which always predicts $1$, then  $\tp(h)=\Pr[h = Y=1] = \Pr[Y=1]=\zeta$, and $\tn(h)=0$. This means that $(\zeta, 0) \in \Ccal$. Similarly, if $h$ is the classifier which always predicts 0, then $\tp(h)=\Pr[h =Y=1] = 0$, and $\tn(h)=\Pr[h = Y=0] = \Pr[Y=0]= 1 - \zeta$. Therefore, $(0, 1 - \zeta) \in \Ccal$. 

\item Let $h$ be a classifier such that $\tp(h)=\tp$, $\tn(h)=\tn$. Now, consider the classifier $1 - h$ (which predicts exactly the opposite of $h$). We have that

\begin{align}
	\tp(1-h)&=\Pmbb[(1-h)=Y=1] \nonumber \\ &=\Pmbb[Y=1]-\Pmbb[h=Y=1] \nonumber \\
	&=\zeta-\tp(h).
\end{align}
\vspace{-0.2cm}
A similar argument gives
\bequation
\tn(1-h)=1-\zeta-\tn(h).
\eequation

\item Consider any two confusion matrices $(\tp_1,\tn_1),\,(\tp_2,\tn_2)\in \mathcal C$, attained by the classifiers $h_1, h_2 \in \Hcal$, respectively. Let $0\leq \lambda\leq 1$. Define a classifier $h'$ which predicts the output from the classifier $h_1$ with probability $\lambda$ and predicts the output of the classifier $h_2$ with probability $1 - \lambda$. Then,

	\begin{align}
		\tp(h')&=\Pmbb[h'=Y=1] \nonumber \\
		&=\Pmbb[h_1=Y=1|h=h_1]\Pmbb[h=h_1] +  \Pmbb[h_2=Y=1|h=h_2]\Pmbb[h=h_2]\\
		&=\lambda\tp(h_1)+(1-\lambda)\tp(h_2).
	\end{align}
	A similar argument gives the convex combination for $\tn$. Thus, $\lambda(\tp(h_1),\tn(h_1)) +(1-\lambda)(\tp(h_2), \tn(h_2)) \in \Ccal$ and hence, $\Ccal$ is convex.
\item This follows from convexity (iv) and boundedness (i). 

\item For any bounded, convex region in $[0,\zeta]\times [0,1-\zeta]$ which contains the points $(0,\zeta)$ and $(0,1-\zeta)$, it is true that any positively sloped supporting hyperplane will be tangent to $(0,\zeta)$ or  $(0,1-\zeta)$.
\end{enumerate}
\end{proof}

\begin{lemma}\label{lem:max-at-bdy}	
	The boundary of $\mathcal C$ is exactly the confusion matrices of 
	estimators of the form $\lambda \1[\eta(x)\geq t] + (1-\lambda)\1[\eta(x)>t]$ and 
    $\lambda \1[\eta(x)< t] + (1-\lambda)\1[\eta(x)\leq t]$ for some $\lambda, t \in [0, 1]$.
\end{lemma}
\begin{proof}
To prove that the boundary is attained by estimators of these forms, consider solving the problem under the constraint $\Pmbb[h=1]=c$. 
We have $\Pmbb[h=1]=TP+FP$, and $\zeta=\Pmbb[Y=1]=TP+FN$, so we get
\bequation
	TP-TN\  =\  c + \zeta - TP - TN - FP - FN\  =\  c+\zeta - 1, 
\eequation
which is a constant. Note that no confusion matrix has two values of $TP- TN$. This effectively partitions $\Ccal$, 
since all confusion matrices are attained by varying $c$ from 0 to 1.
Furthermore, since $A:= TN = TP - c - \zeta + 1$ is an affine space (a line in tp-tn coordinate system), $\mathcal C \cap A$ has at least one endpoint, because $A$ would pass through the box $[\zeta, 0] \times [0, 1- \zeta]$ and has at most two endpoints due to convexity and boundedness of $\Ccal$. 
Since $A$ is a line with positive slope, $\mathcal C\cap A$ is a single point only when $A$ is tangent to $\mathcal C$ at $(0,1-\zeta)$ or $(\zeta,0)$, from Lemma~\ref{lem:properties-C}, part (vi).

Since the affine space $A$ has positive slope, we claim that the two endpoints are attained by maximizing or minimizing $TP(h)$ subject to $\Pr[h=1]=c$.
It remains to show that this happens for estimators of the
form $h_{t+}^\lambda := {\lambda \1[\eta(x)\geq t]} + {(1-\lambda)\1[\eta(x)>t]}$ and 
$h_{t-}^\lambda:=\lambda \1[\eta(x)< t] + (1 - \lambda)\1[\eta(x)\leq t]$, respectively.

Let $h$ be any estimator, and recall
\bequation
	TP(h):=\int_{\mathcal X} \eta(x)\Pmbb[h=1|X=x]\df. 
\eequation
It should be clear that under a constraint $\Pmbb[h=1]=c$, the optimal choice of $h$ puts all the weight onto the larger values of $\eta$. One can begin by classifying those $X$ into the positive class where $n(X)$ is maximum, until one exhausts the budget of $c$. Let $t$ be such that $\Pmbb[h_{t+}^0=1]\leq c\leq \Pmbb[h_{t+}^1=1]$, and let $\lambda \in [0, 1]$ be chosen such that $\Pmbb[h_{t+}^\lambda=1]=c$,
then $h_{t+}^\lambda$ must maximize $TP(h)$ subject to $\Pmbb[h=1]=c$.

A similar argument shows that all TP-minimizing boundary points are attained by the $h_{t-}$'s.
\end{proof}

\bremark
Under Assumption \ref{bin-as:eta}, $\1[\eta(x)>t] = \1[\eta(x)\geq t]$ and $\1[\eta(x)<t] = \1[\eta(x)\leq t]$. Thus, the boundary of $\mathcal C$ is the confusion matrices of estimators of the form $\1[\eta(x)\geq t]$  and $\1[\eta(x)\leq t]$ for some $t \in [0, 1]$.
\eremark

\begin{proof}[Proof of Proposition~\ref{pr:bayeslinear}]
Note, we are maximizing a linear function on a convex set. There are 6 cases to consider:
    \begin{enumerate}[leftmargin=0.5cm]
    \item If the signs of $m_{11}$ and $m_{00}$ differ, the maximum is attained either at $(0,1-\zeta)$ 
    or $(\zeta,0)$, as per Lemma~\ref{lem:properties-C}, part (vi). 
    Which of the two is optimum depends on whether $|m_{11}|\geq |m_{00}|$, i.e. on the sign of
    $m_{11}+m_{00}$. It should be easy to check that in all four possible cases, the statement holds,
    noting that in all four cases, $0 \leq m_{00}/(m_{11}+m_{00}) \leq 1.$
    \item If $m_{11},m_{00}\geq 0$, then the maximum is attained on $\partial\mathcal{C}_+$, and the
    proof below gives the desired result.
    
	We know, from Lemma~\ref{lem:max-at-bdy}, that $\hbar$ must be of the form
	$\1[\eta(x)\geq t]$ for some $t$. It suffices to find $t$. 
	Thus, we wish to maximize $m_{11}TP(h_t)+m_{00}TN(h_t)$.
	Now, let $Z:=\eta(X)$ be the random variable obtained by evaluating $\eta$ at random $X$. Under Assumption~\ref{bin-as:eta}, $df_X = df_Z$ and we have that 
	\bequation
		TP(h_t)\ = \int_{x:\eta(x)\geq t} \eta(x) \df\ =
		\int_{t}^1 z \dfz.
	\eequation
	Similarly, $\TN(h_t) = \int_0^t (1-z)\dfz$. Therefore, 
	
	\begin{align}
		\tfrac{\partial}{\partial t} \big(m_{11}&TP(h_t)+m_{00}
		TN(h_t)\big)
		= -m_{11}tf_Z(t) + \cdot m_{00}(1-t)f_Z(t). 
	\end{align}
    
	So, the critical point is attained at $t=m_{00}/(m_{11}+m_{00})$, as desired.
	A similar argument gives the converse result for $m_{11} + m_{00}< 0$.
	\item if $m_{11},m_{00}<0$, then the maximum is attained on $\partial\mathcal{C}_-$, and an
    argument identical to the proof above gives the desired result. 
	\eenumerate
\end{proof}

\begin{proof}[Proof of Proposition~\ref{pr:strict-convex}]
    
    That $\Ccal$ is convex and bounded is already proven in Lemma~\ref{lem:properties-C}.
    To see that $\mathcal C$ is closed, note that, from Lemma~\ref{lem:max-at-bdy},
    every boundary point is attained. From Lemma~\ref{lem:properties-C}, part (iii), it follows that $\Ccal$ is $180$-degree rotationally symmetric around the point $(\frac{\zeta}{2}, \frac{1-\zeta}{2})$.  
    
    Further, recall every boundary point of $\mathcal C$ can be attained by a thresholding estimator. By the discussion in Section~\ref{bin-sec:confusion}, every boundary point is the optimal classifier for some linear performance metric, and the vector defining this linear metric is exactly the normal vector of the supporting hyperplane at the boundary point.
    
    A vertex exists if (and only if) some point is supported by more than one tangent hyperplane in two dimensional space. This means it is optimal for more than one linear metric. Clearly, all the hyperplanes corresponding to the slope of the metrics where $m_{11}$ and $m_{00}$ are of opposite sign (i.e. hyperplanes with positive slope) support either $(\zeta, 0)$ or $(0, 1-\zeta)$. So, there are at least two supporting hyperplanes at these points, which make them the vertices. Now, it remains to show that there are no other vertices for the set $\Ccal$. 
    
    Now consider the case when the slopes of the hyperplanes are negative, i.e. $m_{11}$ and $m_{00}$ have the same sign for the corresponding linear metrics. We know from Proposition~\ref{pr:bayeslinear} that optimal classifiers for linear metrics are threshold classifiers. Therefore there exist more than one threshold classifier of the form $h_t = \1[\eta(x)\geq t]$ with the same confusion matrix. Let's call them $h_{t_1}$ and $h_{t_2}$ for the two thresholds $t_1, t_2 \in [0, 1]$. This means that 
    \bequation
    \int_{x: \eta(x) \geq t_1} \eta(x)df_X = \int_{x: \eta(x) \geq t_2} \eta(x)df_X.
    \eequation
    Hence, there are multiple values of $\eta$ which are never attained! This contradicts that $g$ is strictly decreasing. Therefore,  there are no vertices other than $(\zeta, 0)$ or $(0, 1-\zeta)$ in $\Ccal$. 
    
    Now, we show that no supporting hyperplane is tangent at multiple points (i.e., there no flat regions on the boundary). If suppose there is a hyperplane which supports two points on the boundary. Then there exist two threshold classifiers with arbitrarily close threshold values, but confusion matrices that are well-separated. Therefore, there must exist some value of $\eta$ which exists with non-zero probability, contradicting the continuity of $g$. 
    By the discussion above, we conclude that under Assumption~\ref{bin-as:eta}, every supporting hyperplane to the convext set $\Ccal$ is tangent to only one point. This makes the set $\Ccal$ strictly convex.
\end{proof}

\begin{proof}[Proof of Lemma~\ref{lem:quasiconcave}] 
We will prove the result for $\phi\circ \rho^+$ on $\partial\mathcal C^+$, and the argument for $\psi\circ \rho^-$ on $\partial\mathcal C^+$ is essentially the same. For simplicity, we drop
the $+$ symbols in the notation. Recall that a function is quasiconcave if and only if its superlevel sets are convex. 

It is given that $\phi$ is quasiconcave. Let $S$ be some superlevel set of $\phi$. We first want to show that for any $r<s<t$, if $\rho(r)\in S$ and $\rho(t)\in S$,
then $\rho(s)\in S$. Since $\rho$ is a continuous bijection, due to the geometry of $\Ccal$ (Lemma~\ref{lem:properties-C} and Proposition~\ref{pr:strict-convex}), we must have --- without loss of generality ---
$TP(\rho(r))< TP(\rho(s)) < TP(\rho(t))$, and $TN(\rho(r))>TN(\rho(s))>TN(\rho(t))$.
(otherwise swap $r$ and $t$). Since the set $\Ccal$ is strictly convex and the image of 
$\rho$ is $\partial \mathcal C$, then $\rho (s)$ must dominate (component-wise) a point in the convex combination
of $\rho(r)$ and $\rho(t)$. Say that point is $z$. Since $\phi$ is monotone increasing, then $x\in S\implies y \in S$ for all $y\geq x$ componentwise. Thereofore, $\phi(\rho(s)) \geq \phi(z)$. Since, $S$ is convex, $z \in S$ and, due to the argument above, $\rho(s) \in S$.

This implies that $\rho^{-1}(\partial \mathcal C\cap S)$ is an interval, and is therefore convex. Thus, the superlevel sets of $\phi\circ \rho$ are convex, so it is quasiconcave,
as desired. This implies unimodaltiy as a function over the real line which has more than one local maximum can not be quasiconcave (consider the super-level set for some value slightly less than the lowest of the two peaks).
\end{proof}

\begin{proof}[Proof of Proposition~\ref{bin-prop:sufficient}] 

For this proof, we denote $TP$ and $TN$ as $C_{11}$ and $C_{00}$, respectively. Let us take a linear-fractional metric
\begin{ceqn}
\begin{align}
\phi(C) = \frac{p_{11}C_{11}+p_{00}C_{00}+p_0}{q_{11}C_{11}+q_{00}C_{00}+q_0}
\label{eq:linear-f}
\end{align}
\end{ceqn}
where $p_{11}, q_{11},p_{00},q_{00}$ are not zero simultaneously.
We want $\phi(C)$ to be monotonic in TP, TN and bounded. If for any $C \in \Ccal$, $\phi(C) < 0$, we can add a large positive constant such that $\phi(C) \geq 0$, and still the metric would remain linear fractional. So, it is sufficient to assume $\phi(C) \geq 0$. Furthermore, boundedness of $\phi$ implies $\phi(C) \in [0,D]$, for some $ \Rmbb \ni D \geq 0$. Therefore, we may divide $\phi(C)$ by $D$ so that $\phi(C) \in [0,1]$ for all $C \in \Ccal$. Still, the metric is linear fractional and $\phi(C) \in [0,1]$.

Taking derivative of $\phi(C)$ w.r.t. $C_{11}$.
\begin{align}
\frac{\partial \phi(C)}{\partial C_{11}} &= \frac{p_{11}}{q_{11}C_{11}+q_{00}C_{00}+q_0} - \frac{q_{11}(p_{11}C_{11}+p_{00}C_{00}+p_0)}{(q_{11}C_{11}+q_{00}C_{00}+q_0)^2} \geq 0 
\end{align}

\begin{align}
\Rightarrow p_{11}(q_{11}C_{11}+q_{00}C_{00}+q_0) \geq q_{11}(p_{11}C_{11}+p_{00}C_{00}+p_0)
\end{align}

If denominator is positive then the numerator is positive as well.
\begin{itemize}
\item Case 1: The denominator $q_{11}C_{11}+q_{00}C_{00}+q_0 \geq 0$.
\begin{itemize}
\item Case (a) $q_{11} > 0$.

\begin{align*}
\Rightarrow p_{11} &\geq q_{11} \phi(C) \\
\Rightarrow p_{11} &\geq q_{11}\sup_{C\in \Ccal} \phi(C)\\
\Rightarrow p_{11} &\geq q_{11}\btau \qquad \text{ (Necessary Condition)} \numberthis 
\end{align*}
We are considering sufficient condition, which means $\btau$ can vary from $[0, 1]$. Hence, a sufficient condition for monotonicity in $C_{11}$ is $p_{11} \geq q_{11}$. Furthermore,
$p_{11} \geq 0$ as well.
\item Case (b) $q_{11} < 0$.
\begin{align}
\Rightarrow p_{11} &\geq {q_{11}} \btau
\end{align}
Since $q_{11} <0$ and $\btau \in [0,1]$, sufficient condition is $p_{11} \geq 0$. So, in this case as well we have that
\begin{align}
p_{11} \geq q_{11}, ~p_{11} \geq 0.
\end{align}
\item Case(c) $q_{11} = 0$.
\begin{align} 
\Rightarrow p_{11} &\geq 0
\end{align}
We again have $p_{11}\geq q_{11}$ and $p_{11} \geq 0$ as sufficient conditions. 

A similar case holds for $C_{00}$, implying $p_{00} \geq q_{00}$ and $p_{00} \geq 0$.
\end{itemize}
\item Case 2: The denominator $q_{11}C_{11}+q_{00}C_{00} + q_0$ is negative. 
\begin{align*}
p_{11} &\leq q_{11} \Big(\frac{p_{11}C_{11}+p_{00}C_{00}+p_0}{q_{11}C_{11}+q_{00}C_{00}+q_0}\Big)\\
\Rightarrow p_{11} &\leq q_{11} \btau \numberthis
\end{align*}
\begin{itemize}
\item Case(a) If $q_{11} > 0$. So, we have $p_{11} \leq q_{11}$ and $p_{11} \leq 0$  as sufficient condition.
\item Case(b) If $q_{11} < 0$, $\Rightarrow p_{11} \leq q_{11}$. So, we have $q_{11} < 0$, $\Rightarrow p_{11} <0$ as sufficient condition.
\item Case(c) If $q_{11} = 0$, $\Rightarrow p_{11} \leq 0$ and $p_{11} \leq q_{11}$ as sufficient condition.

So in all the cases we have that 
\begin{align}
p_{11} \leq q_{11} &\text{ and } p_{11} \leq 0
\end{align}
\vskip -1cm
as the sufficient conditions. A similar case holds for $C_{00}$ resulting in $p_{00} \leq q_{00}$ and $p_{00} \leq 0$. 
\end{itemize}
\end{itemize}

Suppose the points where denominator is positive is $\Ccal^{+}\subseteq \Ccal$. Suppose the points where denominator is negative is $\Ccal^{-} \subseteq \Ccal$. For gradient to be non-negative at points belonging to $\Ccal^{+}$, the sufficient condition is
\vspace{-0.2cm}
\begin{align*}
p_{11} \geq q_{11} &\text{ and } p_{11} \geq 0\\
p_{00} \geq q_{00} &\text{ and } p_{00} \geq 0 \numberthis  
\end{align*}
\vskip -0.25cm
For gradient to be non-negative at points belonging to $\Ccal^{-}$, the sufficient condition is
\begin{align*}
p_{11} \leq q_{11} &\text{ and } p_{11} \leq 0\\
p_{00} \leq q_{00} &\text{ and } p_{00} \leq 0 \numberthis 
\end{align*}
\vskip -0.25cm
If $\Ccal_{+}$ and $\Ccal_{-}$ are not empty sets, then the gradient is non-negative only when $p_{11}, p_{00} = 0$ and $q_{11}, q_{00} = 0$. This is not possible by the definition described in \eqref{eq:linear-f}. Hence, one of $\Ccal_{+}$ or $\Ccal_{-}$ should be empty. WLOG, we assume $\Ccal_{-}$ is empty and conclude that $\Ccal_{+} = \Ccal$. \\
An immediate consequence of this is, WLOG, we can take both the numerator and the denominator to be positive, and the sufficient conditions for monotonicity are as follows:
\begin{align*}
p_{11} \geq q_{11} \text{ and } p_{11} \geq 0\nonumber\\
p_{00} \geq q_{00} \text{ and } p_{00} \geq 0 \numberthis 
\end{align*}

Now, let us take a point in the feasible space $(\zeta,0)$. We know that
\begin{ceqn}
\begin{align}
\phi((\zeta,0)) &= \frac{p_{11}\zeta+p_0}{q_{11}\zeta + q_0} \leq \btau \nonumber \\ 
&\Rightarrow p_{11}\zeta + p_0 \leq \btau (q_{11}\zeta + q_0) \nonumber\\
&\Rightarrow (p_{11} - \btau q_{11})\zeta + (p_0 - \btau q_0) \leq 0\nonumber \\
&\Rightarrow (p_0 - \btau q_0) \leq -\underbrace{(p_{11}-\btau q_{11})}_{\text{positive}}\underbrace{\zeta}_{\text{positive}} \nonumber \\
&\Rightarrow (p_0 - \btau q_0) \leq 0.
\label{eq:p0zero}
\end{align}
\end{ceqn}
\vspace{-0.2cm}
Metric being bounded in $[0,1]$ gives us 
\begin{align*}
\frac{p_{11}C_{11}+p_{00}C_{00}+p_0}{q_{11}C_{11}+q_{00}C_{00}+q_0} & \leq 1 \\
\Rightarrow p_{11}C_{11} + p_{00}C_{00} + p_0 &\leq q_{11}C_{11} + q_{00}C_{00} + q_0 \numberthis 
\end{align*}

\bequation
\Rightarrow q_0 \geq (p_{11}-q_{11})c_{11} + (p_{00}-q_{00})c_{00} + p_0 \qquad \forall C \in \Ccal.
\eequation
Hence, a sufficient condition is 
\bequation
q_0 = (p_{11}-q_{11})\zeta + (p_{00}-q_{00})(1-\zeta) + p_0.
\eequation
Equation \eqref{eq:p0zero}, which we derived from monotonicity, implies that
\begin{itemize}[leftmargin=0.5cm]
\item Case (a) $q_0 \geq 0$, $\Rightarrow p_0 \leq 0$ as a sufficient condition. 
\item Case (b) $q_0 \leq 0$, $\Rightarrow p_0 \leq q_0 \leq 0$ as a sufficient condition. 
\end{itemize}
Since the numerator is positive for all $C \in \Ccal$ and $p_{11}, p_{00} \geq 0$, a sufficient condition for $p_0$ is $p_0 = 0$.

Finally, a monotonic, bounded in $[0,1]$, linear fractional metric is defined by
\begin{align}
\phi(C) &= \frac{p_{11}c_{11}+p_{00}c_{00}+p_0}{q_{11}c_{11}+q_{00}c_{00}+q_0},
\end{align}
where $p_{11} \geq q_{11}, p_{11} \geq 0,
p_{00} \geq q_{00}, p_{00} \geq 0,
q_0 = (p_{11}-q_{11})\zeta + (p_{00}-q_{00})(1-\zeta) + p_0,
p_0 = 0$, and $p_{11}, q_{11}, p_{00}$, and $q_{00}$ are not simulataneously zero. Further, we can divide the numerator and denominator with $p_{11} + p_{00}$ without changing the metric $\phi$ and the above sufficient conditions. Therefore, for elicitation purposes, we can take $p_{11} + p_{00} = 1$.
\end{proof}

\begin{proof}[Proof of Proposition~\ref{pr:solvesystem}] 

For this proof as well, we use $TP = C_{11}$ and $TN = C_{00}$. Since the linear fractional matrix is monotonically increasing in $C_{11}$ and $C_{00}$, it is maximized at the upper boundary $\partial \Ccal_+$. Hence $m_{11} \geq 0$ and $m_{00} \geq 0$. So, after running Algorithm \ref{bin-alg:linear}, we get a hyperplane such that 
\vspace{-0.25cm}
\begin{align}
p_{11} - \tau q_{11} &= \alpha m_{11}, \quad
p_{00} - \tau q_{00} = \alpha m_{00}, \nonumber \\ 
p_0 -\tau q_0 &= -\alpha\underbrace{(m_{11}C_{11}^* + m_{00}C_{00}^*)}_{=: C_0}.
\label{eq:system}
\end{align}
Since $p_{11} - \btau q_{11} \geq 0$ and $m_{11} \geq 0$, $\Rightarrow \alpha \geq 0$. As discussed in the main paper, we avoid the case when $\alpha = 0$. Therefore, we have that $\alpha > 0$.

Equation \eqref{eq:system} implies that 
\begin{align*}
\frac{p_{11}}{\alpha} - \frac{\tau q_{11}}{\alpha} &= m_{11}, \quad
\frac{p_{00}}{\alpha} - \frac{\tau q_{00}}{\alpha} = m_{00}, \nonumber \\
\frac{p_0}{\alpha} - \frac{\tau q_0}{\alpha} &= -C_0. \numberthis
\end{align*}
Assume $p_{11}' = \frac{p_{11}}{\alpha}, p_{00}' = \frac{p_{00}}{\alpha}$, $q_{11}' = \frac{q_{11}}{\alpha}$, $q_{00}' = \frac{q_{00}}{\alpha}$, $p_0' = \frac{p_0}{\alpha}$, $q_0' = \frac{q_0}{\alpha}$. Then, the above system of equations turns into
\begin{align*}
p_{11}' - \btau q_{11}' &= m_{11}, \quad
p_{00}' - \btau q_{00}' = m_{00}, \nonumber \\
p_0' - \btau q_0' &= -C_0. \numberthis
\end{align*}
A $\phi'$ metric defined by the $\ppone, \ppzero, \pqone, \pqzero, \pqnot$ is monotonic, bounded in $[0,1]$, and satisfies all the sufficient conditions of Assumptions~\ref{assump:sufficient}, i.e.,
\begin{align*}
p_{11}' \geq q_{11}' ~,~ p_{00}' \geq q_{11}',~
p_{11}' \geq 0 ~,~ p_{00}' \geq 0, \nonumber \\
q_0' = (p_{11}' - q_{11})\pi + (p_{00}' - q_{00}')\pi + p_0', ~
p_0' = 0. \numberthis
\end{align*}
As discussed in Chapter~\ref{chp:binary}, solving the above system does not harm the elicitation task. For simplicity, replacing the `` $'$ " notation with the normal one, we have that
\begin{align*}
p_{11} - \btau q_{11} &= m_{11}, \quad
p_{00} - \btau q_{00} = m_{00}, \nonumber \\
p_0 - \btau q_0 &= -C_0 \numberthis
\end{align*}
\vskip -0.25cm
From last equation, we have that $\btau  = \frac{C_0 + p_0}{q_0}$. Putting it in the rest gives us

\begin{align}
q_0 p_{11} - (C_0 + p_0)q_{11} = m_{11}q_0 \quad \text{and} \quad 
q_0 p_{00} - (C_0 + p_0) q_{00} = m_{00} q_0.
\end{align}

We already have
\vspace{-0.2cm}
\begin{align*}
q_0 &= (p_{11}-q_{11})\zeta + (p_{00}-q_{00})(1-\zeta)+p_0\\
\Rightarrow q_{11} &= \frac{p_{00}(1-\zeta)-q_{00}(1-\zeta) + p_{11}\zeta - q_0+p_0}{\zeta}, \numberthis
\end{align*}
which further gives us
\begin{align*}
q_0 &= \frac{(C_0 + p_0)[p_{00}(1-\zeta)+p_{11}\zeta +  p_0]}{p_{11}\zeta + p_{00}(1-\zeta) + p_0 + C_0 - m_{11}\zeta - m_{00}(1-\zeta)},\\
q_{00} &= \frac{(p_{00}-m_{00})[p_{00}(1-\zeta)+p_{11}\zeta + p_0]}{p_{11}\zeta + p_{00}(1-\zeta) + p_0 + C_0 -m_{11}\zeta -m_{00}(1-\zeta)},\\
q_{11} &= \frac{(p_{11}-m_{11})[p_{00}(1-\zeta) + p_{11}\zeta + p_0]}{p_{11}\zeta + p_{00}(1-\zeta) + p_0 + C_0 - m_{11}\zeta -m_{00}(1-\zeta)}. \numberthis
\end{align*}
Define
\begin{align}
P := p_{00}(1-\zeta) + p_{11}\zeta + p_0 \quad \text{and} \quad 
Q := P + C_0 - m_{11}\zeta -m_{00}(1-\zeta).
\end{align}

Hence, 
\begin{align}
q_0 = (C_0 + p_0)\frac{P}{Q}, \quad
q_{11} = (p_{11}-m_{11})\frac{P}{Q}, \quad 
q_{00} = (p_{00}-m_{00})\frac{P}{Q}.
\end{align}
Now using sufficient conditions, we have $p_0 = 0$. The final solution is the following:

\begin{align}
q_0 = C_0 \frac{P}{Q}, \quad
q_{11} = (p_{11} - m_{11})\frac{P}{Q}, \quad 
q_{00} = (p_{00} - m_{00})\frac{P}{Q},
\label{eq:systemsolve}
\end{align}
where $P:= p_{11}\zeta + p_{00}(1-\zeta) $ and $Q:= P + C_0 - m_{11}\zeta - m_{00}(1-\zeta)$. 
We have taken ${p}_{11} + {p}_{00} = 1$, but the original $p'_{11} + p'_{00} = \frac{1}{\alpha}$. Therefore, we learn $\hat{\phi}(C)$ such that such that $\hat{\phi}(C) = \alpha \phi(C)$.
\end{proof}

\bcorollary \label{cor:f-beta} 
For $F_\beta$-measure, where $\beta$ is unknown, Algorithm \ref{bin-alg:linear} elicits the true performance metric up to a constant in $O(\log(\frac{1}{\epsilon}))$ queries to the oracle.
\ecorollary
\begin{proof}
Algorithm \ref{bin-alg:linear} gives us the supporting hyperplane, the trade-off, and the Bayes confusion matrix. If we know $p_{11}$, then we can use Proposition \ref{pr:solvesystem} to compute the other coefficients. In $F_\beta$-measure, $p_{11}=1$, and we do not require Algorithms 3.2 and \ref{alg:grid-search}.
\end{proof}

\begin{proof}[Proof of Theorem~\ref{thm:quasi}] We prove the points one by one.

\begin{enumerate}[leftmargin=0.5cm, label=(\roman*)]
\item As a direct consequence of our representation of the points on the boundary via their supporting hyperplanes (Section~\ref{ssec:parametrization}), when we search for the maximizer (mimimizer), we also get the associated supporting hyperplane as well.

\item  By the nature of binary search, we are effectively narrowing our search interval around some target angle $\theta_0$. Furthermore, since the oracle queries are correct unless the $\phi$ values are within $\epsilon_\Omega$, we must have $|\phi(C_{\oline \theta})-\phi(C_{\theta_0})|<\epsilon_\Omega$, and we output $\theta'$ such that $|\theta_0-\theta'|<\epsilon$. Now, we want to check the bound $|\phi(C_{\theta'}) - \phi(C_{\oline{\theta}})|$. In order to do that, we will also consider the threshold corresponding to the supporting hyperplanes at $C_\theta$'s, i.e. $\delta_\theta = \sfrac {\sin\theta}{\sin\theta + \cos\theta}$. 

Notice that,
\begin{ceqn}
\begin{align}
|\phi(C_{\oline{\theta}}) - \phi(C_{\theta'})| &= |\phi(C_{\oline{\theta}}) -\phi(C_{\theta_0}) \nonumber + \phi(C_{\theta_0}) - \phi(C_{\theta'})| \nonumber\\
&\leq |\phi(C_{\oline{\theta}}) -\phi(C_{\theta_0})| + |\phi(C_{\theta_0}) - \phi(C_{\theta'})| 
\end{align}
\end{ceqn}
The first term is bounded by $\epsilon_{\Omega}$ due to the oracle assumption.  For the bounds the second term, consider the following.
$$
|TP(C_{\theta_0}) - TP(C_{\theta'})|
$$

\begin{ceqn}
\begin{align}
&= \left|\int\limits_{x:\frac{sin\theta_0}{sin\theta_0 + cos\theta_0}\geq\eta(x)\geq\frac{sin\theta'}{sin\theta' + cos\theta'}}\!\!\!\!\!\!\!\!\!\!\!\! \eta(x)\df\right| \nonumber \\
&\leq \left|\int\limits_{x:\frac{sin\theta_0}{sin\theta_0 + cos\theta_0} - \oline{\delta}\geq\eta(x) - \oline{\delta}\geq\frac{sin\theta'}{sin\theta' + cos\theta'}- \oline{\delta}}\!\!\!\!\!\!\!\!\!\!\!\! \df\right| \nonumber \\
&= \left|\int\limits_{x:\frac{sin\theta_0}{sin\theta_0 + cos\theta_0} - \frac{sin\oline{\theta}}{sin\oline{\theta} + cos\oline{\theta}} \geq\eta(x) - \oline{\delta}\geq\frac{sin\theta'}{sin\theta' + cos\theta'}- \frac{sin\oline{\theta}}{sin\oline{\theta} + cos\oline{\theta}}}\!\!\!\!\!\!\!\!\!\!\!\! \df\right| \nonumber \\
&= \left|\int\limits_{x:\frac{sin(\theta_0 - \oline{\theta})}{sin(\theta_0 + \oline{\theta})  + cos(\theta_0 - \oline{\theta})} \geq\eta(x) - \oline{\delta}\geq\frac{sin\theta'}{sin\theta' + cos\theta'}- \frac{sin\oline{\theta}}{sin\oline{\theta} + cos\oline{\theta}}}\!\!\!\!\!\!\!\!\!\!\!\! \df\right|,  
\label{eq:integrals}
\end{align}
\end{ceqn}
where the inequality in the second step follows from the fact that $\eta(x) \leq 1$. 

Recall that the left term in the integral limits is actually, $\delta_{\theta_0} - \delta_{\oline\theta}$. When $|\phi(C_{\delta_{\theta_0}})-\phi(C_{\delta_{\oline\theta}})|<\epsilon_\Omega$, then we have $|\oline\delta-\delta_0|<\frac 2{k_0}\sqrt{
k_1\epsilon_\Omega}$. The proof of this statement is given in the proof of Theorem~\ref{thm:linear} (proved later).
Since sin is 1-Lipschitz, adding and subtracting $\sin\theta_0/(\sin\theta_0 + \cos\theta_0)$ in the right term of the integration limit gives us the minimum value of the right term to be $-\epsilon-\frac{2\sqrt{k_1\epsilon_\Omega}}{k_0}$.  
This implies that the quantity in ~\eqref{eq:integrals} is less than 
\vspace{-0.4cm}
\begin{ceqn}
\begin{align}
&\Pmbb[\{(\eta(X) - \oline{\delta}) \leq \frac{2}{k_0}\sqrt{k_1\epsilon_\Omega}\} \cap \{(\oline{\delta} - \eta(X)) \leq \epsilon + \frac{2}{k_0}\sqrt{k_1\epsilon_\Omega}\}]\nonumber \\
&\leq\Pmbb[(\oline{\delta} - \eta(X)) \leq \epsilon + \frac{2}{k_0}\sqrt{k_1\epsilon_\Omega}] \nonumber \\
&\leq \frac{2k_1}{k_0}\sqrt{k_1\epsilon_\Omega} + k_1\epsilon. \quad \text{(by Assumption~\ref{as:low-weight-around-opt})}
\end{align}
\end{ceqn}

As $\Pmbb(A\cap B) \leq min\{\Pmbb(A), \Pmbb(B)\}$, the inequality used in the second step is rather loose, but it shows the dependency on sufficiently small $\epsilon$. It could be independent of the tolerance $\epsilon$ depending on the $\Pmbb(\eta(X) - \oline\delta)$ or the sheer big value of $\epsilon$. Nevertheless, a similar result applies to the true negative rate. 
Since $\phi$ is 1-Lipschitz, we have that $|\phi(C)-\phi(C')|\leq 1\cdot \Vert C-C'\Vert$,
but 

\bequation
\Vert C(\theta_0)-C(\theta')\Vert_\infty 
\leq \frac{2k_1}{k_0}\sqrt{k_1\epsilon_\Omega} + k_1\epsilon. 
\eequation

Hence,
\bequation
|\phi(C_{\theta'}) - \phi(C_{\oline{\theta}})| \leq \sqrt{2}(\frac{2k_1}{k_0}\sqrt{k_1\epsilon_\Omega} + k_1\epsilon) + \epsilon_\Omega.
\eequation
Since the metrics are in $[0, 1]$, $\epsilon_\Omega \in [0, 1]$. Therefore, $\sqrt{\epsilon_\Omega} 
\geq \epsilon_\Omega$. This gives us the desired result.

\item We needed only, for part (ii), that the interval of possible values of $\theta'$ be at most $\epsilon$ to the target angle $\theta_0$. Ideally, this is obtained by making $\log_2(1/\epsilon)$ queries, but due to the region where oracle misreport its preferences, we can be off to the target angle $\theta_0$ by more than $\epsilon$. 

However, binary search will again put us back in the correct direction, once we leave the misreporting region. And this time, even if we are off to the target angle $\theta_0$, we will be closer than before. Therefore, for the interval of possible values of $\theta'$ to be at most $\epsilon$, we require at least $\log(\frac{1}{\epsilon})$ rounds of the algorithm, each of which is a constant number of pairwise queries.

\end{enumerate}\vspace*{-2em}
\end{proof}

\begin{proof}[Proof of Lemma~\ref{lem:lower-bound}] 

For any fixed $\epsilon$, divide the search space $\theta$ into bins of length $\epsilon$, resulting in $\ceil[\big]{\frac{1}{\epsilon}}$ classifiers. When the function evaluated on these classifiers is unimodal, and when the only operation allowed is pairwise comparison, the optimal worst case complexity for finding the argument maximum (of function evaluations) is $O(\log\frac1\epsilon)$ \cite{cormen2009introduction}, which is achieved by binary search. 
\end{proof}

\bprop\label{pr:sample-concentration-confusion} 
    Let $(y_1,x_1,h(x_1)),\,\dotsc,\,(y_n,x_n,h(x_n))$ be $n$ i.i.d.~samples from the joint distribution on $Y$, $X$, and $h(X)$. Then by H\"offding's inequality, 
	\bequation
	\Pmbb\left[\left|\tfrac1n\textstyle\sum_{i=1}^n\1[h_i=y_i=1] - TP(h)\right|\geq \epsilon \right]\leq 2e^{-2n\epsilon^2}.
	\eequation
	The same holds for the analogous estimator on TN.
\eprop
\bproof
Direct application of H\"offding's inequality.
\eproof

\begin{proof}[Proof of Theorem~\ref{thm:linear}]

We will show this for threshold classifiers, as in the statement of the Assumption~\ref{as:low-weight-around-opt}, but it is not difficult to extend the argument to the case of querying angles. (Involves a good bit of trigonometric identities...)

Recall, the threshold estimator $h_\delta$ returns positive if $\eta(x)\geq \delta$, and zero otherwise. Let $\oline\delta$ be the threshold which maximizes performance with respect to $\phi$, and $C_{\oline\delta}$ be its confusion matrix. 
For simplicity, suppose that $\delta'<\oline\delta$. Recall, from Assumption~\ref{as:low-weight-around-opt} that $\Pr[\eta(X)\in [\oline\delta-\frac{k_0}{2k_1}\epsilon,\,\oline\delta]]\leq k_0\epsilon/2$,
but $\Pr[\eta(X)\in[\oline\delta-\epsilon,\oline\delta]]\geq k_0\epsilon$, and therefore 
\bequation
	\Pmbb\Big[\eta(X)\in[\oline\delta-\epsilon,\oline\delta-\tfrac{k_0}{2k_1}\epsilon]\Big]\geq k_0\epsilon/2
\eequation
Denoting $\phi(C)=\langle\mmbf,C\rangle$ and since $\oline\delta = m_{00}/(m_{11}+m_{00})$, by expanding the integral, we get
\begin{align*}
	&\phi(C_{\oline\delta})-\phi(C_{\delta'})= \int_{x:\delta'\leq \eta(x) \leq \oline\delta}\!\!\!\!\!\!\!\!\!\!\!\! [m_{00}(1-\eta(x))-m_{11}\eta(x)]\df \\
    &=\int_{x:\oline\delta - (\oline\delta - \delta')\leq \eta(x)\leq \oline\delta}\!\!\!\!\!\!\!\!\!\!\!\! [m_{00}(1-\eta(x))-m_{11}\eta(x)]\df \\
    &\geq\int_{x:\oline\delta - (\oline\delta - \delta')\leq \eta(x)\leq \oline\delta - \frac{k_0}{2k_1}(\oline\delta - \delta')}\!\!\!\!\!\!\!\!\!\!\!\! [m_{00}(1-\eta(x))-m_{11}\eta(x)]\df \\
    &\geq[(m_{11} + m_{00})\big(\frac{-m_{00}}{m_{00} + m_{11}} + \frac{k_0}{2k_1}(\oline\delta - \delta')\big) + m_{00}] \times  \int_{x:\oline\delta - (\oline\delta - \delta')\leq \eta(x)\leq \oline\delta - \frac{k_0}{2k_1}(\oline\delta - \delta')}\!\!\!\!\!\!\!\!\!\!\!\! \df \\
    &=[(m_{11} + m_{00}) \frac{k_0}{2k_1}(\oline\delta - \delta')] \times \Pmbb[\oline\delta - (\oline\delta - \delta')\leq \eta(x)\leq \oline\delta - \frac{k_0}{2k_1}(\oline\delta - \delta')]\\
    &\geq \tfrac{k_0}2(\oline\delta-\delta') \cdot \tfrac{k_0}{2k_1}(\oline\delta-\delta')=\frac{k_0^2}{4k_1}(\oline\delta-\delta')^2. \numberthis
\end{align*}
Similar results hold when $\delta'>\oline\delta$. 
Therefore, if we have $|\phi(\oline C)-\phi(C(\delta'))|<\epsilon_\Omega$, then we must have $|\oline\delta-\delta'|<\frac 2{k_0}\sqrt{
k_1\epsilon_\Omega}$. Thus, if we are in a regime where the oracle is misreporting the preference ordering, it must be the case that the thresholds are sufficiently close to the optimal threshold.

Again, as in the proof of Theorem~\ref{thm:quasi}, when the tolerance $\epsilon$ is small, our binary search closes in on a parameter $\theta'$ which has $\phi(C_{\delta_{\theta'}})$ within $\epsilon_\Omega$ of the optimum, but from the above discussion, this also implies that the search interval itself is close to the true value, and thus, the total error in the threshold is at most $\epsilon + \frac 2{k_0}\sqrt{k_1\epsilon_\Omega}$. Since $\oline\delta = m_{00}/(m_{11}+m_{00})$, this bound extends to the cost vector with a factor of $\sqrt2$, thus giving the desired result.

We observe that the above theorem actually provide bounds on the slope of the hyperplanes. Thus, the guarantees for LFPM elicitation follow naturally. It only requires that we recover the slope at the upper boundary and lower boundary correctly (within some bounds). This theorem provides those guarantees.   Algorithm~\ref{alg:grid-search} is independent of oracle queries and thus can be run with high precision, making the solutions of the two systems match.  
\end{proof}

\begin{proof}[Proof of Lemma~\ref{lem:sample-Cs-optimize-well}] 

Suppose the performance metric of the oracle is characterized by the parameter $\oline\theta$. Recall the Bayes optimal classifier would be $h_{\oline{\theta}} = \1 [\eta\geq \oline{\delta}]$. Let us assume we are given a classifier $\hhat_{\oline{\theta}} = \1 [\hat\eta\geq \oline{\delta}]$. Notice that the optimal threshold $\oline\delta$ is the property of the metric and not the classifier or $\eta$. We want to bound the difference in the confusion matrices for these two classifiers. Notice that, by  Assumption~\ref{as:sup-norm-convergence}, we can take $n$ sufficiently large so that $\Vert \eta-\hat\eta_n\Vert_\infty$ is arbitrarily small. Consider the quantity

\begin{align}
	TP(h_{\oline{\theta}}) - TP(\hat h_{\oline{\theta}}) = \int_{\eta\geq \oline{\delta}} \!\!\!\!\!\!\!\eta \df  -
    \int_{\hat\eta\geq \oline{\delta}} \!\!\!\!\!\!\!\eta \df
    \label{eq:loss}
\end{align}

Now the maximum loss in the above quantity can occur when, in the region where the classifiers' predictions differ, there $\hat\eta$ is less than $\eta$ with the maximum possible difference. This is equal to

\begin{align*}
    &\int\limits_{x:\oline\delta \leq \eta(x) \leq \oline\delta + \Vert\eta - \hat\eta\Vert_\infty} \!\!\!\!\!\!\!\eta \df \\ \nonumber
    &\leq \Pmbb[\oline\delta \leq \eta(X) \leq \oline\delta + \Vert\eta - \hat\eta\Vert_\infty] \\ \nonumber
    &\leq k_1\Vert\eta - \hat\eta\Vert_\infty. \qquad \text{(by Assumpition~\ref{as:low-weight-around-opt})} \numberthis 
\end{align*}

Similarly, we can look at the maximum gain in the following quantity.

\begin{align}
	 TP(\hat h_{\oline{\theta}}) - TP(h_{\oline{\theta}}) &= 
    \int_{\hat\eta\geq \oline{\delta}} \!\!\!\!\!\!\!\eta \df - 
    \int_{\eta\geq \oline{\delta}} \!\!\!\!\!\!\!\eta \df  
    \label{eq:gain}
\end{align}
Now the maximum gain in the above quantity can occur when, in the region where the classifiers' predictions differ, there $\hat\eta$ is greater than $\eta$ with the maximum possible difference. This is equal to
\begin{align*}
    \int\limits_{x: \oline\delta - \Vert\eta - \hat\eta\Vert_\infty\leq\eta(x)\leq \oline\delta} \!\!\!\!\!\!\!\eta \df 
    &\leq \Pmbb[\oline\delta - \Vert\eta - \hat\eta\Vert_\infty \\ 
    &\leq \eta(X) \leq \oline\delta] \\
    &\leq k_1\Vert\eta - \hat\eta\Vert_\infty. \qquad \text{(by Assumpition~\ref{as:low-weight-around-opt})} \numberthis 
\end{align*}
Hence, 
\bequation
|TP(\hat h_{\oline{\theta}}) - TP(h_{\oline{\theta}})| \leq k_1\Vert\eta - \hat\eta\Vert_\infty.
\eequation
Similar arguments apply for $TN$, which gives us the desired result.
\end{proof}

\section{Extended Experiments}
\label{appendix:experiments}

\begin{table}[t]
	\caption{Empirical Validation for LPM elicitation at tolerance $\epsilon = 0.02$ radians. $\sphi$ and $\hphi$ denote the true and the elicited metric, respectively.}
	\label{tab:app:LPMtheory}
	\begin{center}
		\begin{small}
				\begin{tabular}{|c|c|c|c|}
					\hline
					  $\sphi = \smmbf$ & $\hphi = \hmmbf$ & $\sphi = \smmbf$ & $\hphi = \hmmbf$ \\ \hline 
					(0.98,0.17) & (0.99,0.17) & (-0.94,-0.34) & (-0.94,-0.34) \\
					(0.87,0.50) & (0.87,0.50) &(-0.77,-0.64)& (-0.77,-0.64)  \\
					(0.64,0.77) & (0.64,0.77) & (-0.50,-0.87) & (-0.50,-0.87)  \\
					(0.34,0.94) &(0.34,0.94) &(-0.17,-0.98) & (-0.17,-0.99 ) \\
					\hline
				\end{tabular}
		\end{small}
	\end{center}
\end{table}

In this section, we empirically validate the theory and robustness to finite samples.


\subsection{Synthetic Data Experiments}
\label{ssec:app:theoryexp}

We take the same distribution as in \eqref{prob-dist} with the noise parameter $a = 5$. In the LPM elicitation case, we define a true metric $\sphi$ by $\smmbf = (\smone, \smzero)$. This defines the query outputs in line 6 of Algorithm \ref{bin-alg:linear}. Then we run Algorithm \ref{bin-alg:linear} to check whether or not we get the same metric. The results for both monotonically increasing and monotonically decreasing LPM are shown in Table~\ref{tab:app:LPMtheory}. We achieve the true metric even for very tight tolerance $\epsilon=0.02$ radians.

Next, we elicit LFPM. We define a true metric $\sphi$ by $\{(\spone, \spzero), (\sqone, \sqzero, \sqnot)\}$.  Then, we run Algorithm \ref{bin-alg:linear} with $\epsilon=0.05$ to find the hyperplane $\bell$ and maximizer on $\partial C_+$, Algorithm~3.2 with $\epsilon=0.05$ to find the hyperplane $\tell$ and minimizer on $\partial C_-$, and Algorithm \ref{alg:grid-search} with $n = 2000$ (1000 confusion matrices on both $\partial\Ccal_+$ and $\partial\Ccal_-$ obtained by varying parameter $\theta$ uniformly in $[0, \pi/2]$ and $[\pi, 3\pi/2]$) and $\Delta = 0.01$. This gives us the elicited metric $\hphi$, which we represent by $\{(\hpone, \hpzero), (\hqone, \hqzero, \hqnot)\}$.
In Table \ref{tab:app:LFPMtheoryreal}, we present the elicitation results for LFPMs (column 2). We also present the mean ($\alpha$) and the standard deviation ($\sigma$) of the ratio of the elicited metric $\hphi$ to the true metric $\phi$ over the set of confusion matrices (column 3 and 4 of Table \ref{tab:app:LFPMtheoryreal}). As suggested in Corollary~\ref{cor:f-beta}, if we know the true ratio of $\sfrac{\spone}{\spzero}$, then we can elicit the LFPM up to a constant by only using Algorithm $\ref{bin-alg:linear}$ resulting in better estimate of the true metric, because we avoid errors due to Algorithms 3.2 and~\ref{alg:grid-search}. Line 1 and line 2 of Table \ref{tab:app:LFPMtheoryreal} represent $F_1$ measure and $F_\frac{1}{2}$ measure, respectively. In both the cases, we assume the knowledge of $p_{11}^*=1$. Line 3 to line 6 correspond to some arbitrarily  chosen linear fractional metrics to show the efficacy of the proposed method. For a better judgment, we show function evaluations of the true metric and the elicited metric on selected pairs of $(TP, TN) \in \partial\Ccal_+$ (used for Algorithm \ref{alg:grid-search}) in Figure \ref{fig:app:lfpm-theory}. The true and the elicited metric are plotted together after sorting values based on slope parameter $\theta$. We see that the elicited metric is a constant multiple of the true metric. The vertical solid and dashed line corresponds to the \emph{argmax} of the true and the elicited metric, respectively. In Figure \ref{fig:app:lfpm-theory}, we see that the \emph{argmax} of the true and elicited metrics coincides, thus validating Theorem~\ref{thm:quasi}.

\begin{table}[t]
	\caption{LFPM Elicitation for synthetic distribution (Section \ref{ssec:app:theoryexp}) and Magic (\textsc{M}) dataset  (Section \ref{ssec:app:realexp}) with $\epsilon = 0.05$ radians. $(\spone, \spzero), (\sqone, \sqzero, \sqnot)$ denote the true LFPM. $(\hpone, \hpzero), (\hqone, \hqzero, \hqnot)$ denote the elicited LFPM. $\alpha$ and $\sigma$ denote the mean and the standard deviation in the ratio of the elicited to the true metric (evaluated on the confusion matrices in $\partial\Ccal_+$ used in Algorithm~\ref{alg:grid-search}), respectively. We empirically verify that the elicited metric is constant multiple ($\alpha$) of the true metric.}
	\label{tab:app:LFPMtheoryreal}
	\begin{center}
		\begin{small}
            \resizebox{\textwidth}{!}{%
			\begin{tabular}{|c|c|c|c|c|c|c|}
				\hline
                
				True Metric & \multicolumn{3}{|c|}{Results on Synthetic Distribution (Section \ref{ssec:app:theoryexp})} & \multicolumn{3}{|c|}{Results on Real World Dataset \textsc{M}  (Section \ref{ssec:app:realexp})} \\ \hline
				$(\spone, \spzero), (\sqone, \sqzero, \sqnot)$ & $(\hpone, \hpzero), (\hqone, \hqzero, \hqnot)$ & $\alpha$ & $\sigma$ & $(\hpone, \hpzero), (\hqone, \hqzero, \hqnot)$ & $\alpha$ & $\sigma$ \\
				\hline
				(1.00,0.00),(0.50,-0.50,0.50) & (1.00,0.00),(0.25,-0.75,0.75) & 0.92 & 0.03  & (1.00,0.00),(0.25,-0.75,0.75) & 0.90 & 0.06 \\
				(1.0,0.0),(0.8,-0.8,0.5) & (1.0,0.0),(0.73,-1.09,0.68) & 0.94 & 0.02& (1.0,0.0),(0.72,-1.13, 0.57) & 1.06 & 0.05 \\
				(0.8,0.2),(0.3,0.1,0.3)  & (0.86,0.14),(-0.13,-0.07, 0.60) & 0.90 & 0.06 & (0.23,0.77),(-0.87,0.66,0.76) & 0.84  & 0.09\\
				(0.60,0.40),(0.40,0.20,0.20) & (0.67,0.33),(-0.07,-0.44,76) & 0.82 & 0.05 & (0.16,0.84),(-0.89,0.25,0.89) & 0.65 & 0.05\\
				(0.40,0.60),(-0.10,-0.20,0.65) & (0.36,0.64),(-0.21,-0.25,0.73) & 0.97 & 0.01 & (0.08,0.92),(-0.75,0.12,0.82) & 0.79 & 0.08\\
				(0.20,0.80),(-0.40,-0.20,0.80) & (0.12, 0.88),(-0.43, 0.002, 0.71) & 1.02 & 0.006 & (0.19,0.81),(-0.38,-0.13,0.70) & 1.02 & 0.004  \\
				\hline
			\end{tabular}}
		\end{small}
	\end{center}
\end{table}

\begin{figure}[t]
	\centering 
	\subfigure[Table \ref{tab:app:LFPMtheoryreal}, Line 1, Column 2]{
		{\includegraphics[width=5cm]{binary/plots/lfpm_1_AISTATS.png}}
		\label{fig:app:lfpm_1}
	}
	\subfigure[Table \ref{tab:app:LFPMtheoryreal}, Line 2, Column 2]{
		{\includegraphics[width=5cm]{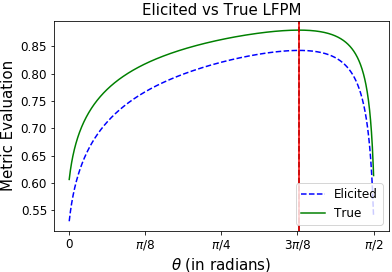}}		\label{fig:app:lfpm_2}
	}
	\subfigure[Table \ref{tab:app:LFPMtheoryreal}, Line 3, Column 2]{
		{\includegraphics[width=5cm]{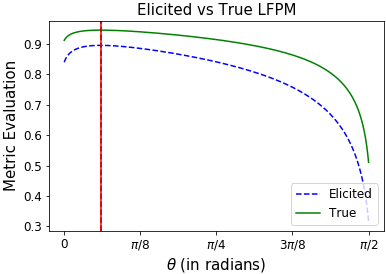}}
		\label{fig:app:lfpm_3}
	}
	\subfigure[Table \ref{tab:app:LFPMtheoryreal}, Line 4, Column 2]{
		{\includegraphics[width=5cm]{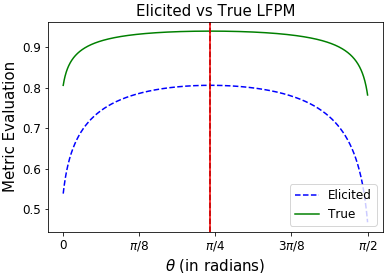}}
		\label{fig:app:lfpm_4}
	}
	\subfigure[Table \ref{tab:app:LFPMtheoryreal}, Line 5, Column 2]{
		{\includegraphics[width=5cm]{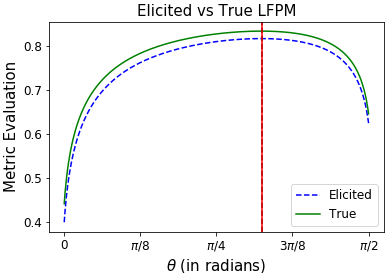}}
		\label{fig:app:lfpm_5}
	}
	\subfigure[Table \ref{tab:app:LFPMtheoryreal}, Line 6, Column 2]{
		{\includegraphics[width=5cm]{binary/plots/lfpm_6_AISTATS.png}}
		\label{fig:app:lfpm_6}
	}
	\caption{True and elicited LFPMs for synthetic distribution from Table \ref{tab:app:LFPMtheoryreal}. The solid green curve and the dashed blue curve are the true and the elicited metric, respectively. The solid red and the dashed black vertical lines represent the maximizer of the true metric and the elicited metric, respectively. The elicited LFPMs are constant multiple of the true metrics with the same maximizer (solid red and dashed black vertical lines overlap).}
	\label{fig:app:lfpm-theory}
	\vskip -0.25cm
\end{figure}

\vspace{-0.25cm}
\subsection{Real-World Data Experiments}
\label{ssec:app:realexp}

In real-world datasets, we do not know $\eta(x)$ and only have finite samples. 
Thus, the feasible space $\Ccal$ is not as well behaved as shown in Figure \ref{fig:fs_cm}, and poses a challenge for the elicitation task. 
Now, we validate the elicitation procedure with two real-world datasets. The datasets are: (a) \textsc{Breast Cancer (BC)} Wisconsin Diagnostic dataset \cite{street1993nuclear} containing 569 instances, and  (b) \textsc{Magic (M)} dataset \cite{dvovrak2007softening} containing 19020 instances. For both the datasets, we standardize the attributes and split the data into two parts $\Scal_1$ and $\Scal_2$. On $\Scal_1$, we learn an estimator $\hat{\eta}$ using regularized logistic regression model with regularizing constant $\lambda=10$ and $\lambda=1$. We use $\Scal_2$ for making predictions and computing sample confusions. 

We generated twenty eight different LPMs $\sphi$ by generating $\theta^*$ (or say, $\smmbf = (\cos{\theta}^*, \sin{\theta}^*))$. Fourteen from the first quadrant starting from $\pi/18$ radians to $5\pi/12$ radians in step of $\pi/36$ radians. Similarly, fourteen from the third quadrant starting from $19\pi/18$ to $17\pi/12$ in step of $\pi/36$ radians. We then use Algorithm~\ref{bin-alg:linear} (Algorithm~3.2) for different tolerance $\epsilon$, for different datasets, and for different regularizing constant $\lambda$ in order to recover the estimate $\hat{\mmbf}$. We compute the error in terms of the proportion of the number of times when Algorithm \ref{bin-alg:linear} (Algorithm~3.2) failed to recover the true ${\smmbf}$ within $\epsilon$ threshold.

We report our results in Table \ref{tab:app:LPMreal}. We see improved elicitation for dataset $M$, suggesting that ME improves with larger datasets. In particular, for dataset $M$, we elicit all the metrics within threshold $\epsilon = 0.11$ radians. We also observe that $\epsilon = 0.02$ is an overly tight tolerance for both the datasets leading to many failures. This is because the elicitation routine gets stuck at the closest achievable confusion matrix from finite samples, which need not be optimal within the given (small) tolerance. Furthermore, both of these observations are consistent for both the regularized  logisitic regression models with regularizer $\lambda$. 

Next, we discuss the case of LFPM elicitation. We use the same true metrics $\sphi$ as described in Section \ref{ssec:app:theoryexp} and follow the same process for eliciting LFPM, but this time we work with \textsc{MAGIC} dataset. 
In Table \ref{tab:app:LFPMtheoryreal} (columns 5, 6, and 7), we present the elicitation results on \textsc{MAGIC} dataset along with the mean $\alpha$ and the standard deviation $\sigma$ of the ratio of the elicited metric and the true metric. Again, for a better judgment, we show the function evaluation of the true metric and the elicited metric on the selected pairs of $(TP, TN) \in \partial\Ccal_+$ (used for Algorithm \ref{alg:grid-search}) in Figure \ref{fig:app:lfpm-real}, ordered by the parameter $\theta$. Although we do observe that the \emph{argmax} is different in two out of six cases (see Sub-figure \subref{fig:app:lfpm_2_magic} and Sub-figure \subref{fig:app:lfpm_3_magic}) due to finite samples, elicited LFPMs are almost equivalent to the true metric up to a constant. 

\begin{table}[t]
    \vskip -0.3cm
	\caption{LPM elicitation results on real datasets ($\epsilon$ in radians). \textsc{M} and \textsc{BC} represent Magic and Breast Cancer dataset, respectively. $\lambda$ is the regularization parameter in the regularized logistic regression models. The table shows error in terms of the proportion of the number of times when Algorithm \ref{bin-alg:linear} (Algorithm~3.2) failed to recover the true $\smmbf (\theta^*)$ within $\epsilon$ threshold. 
	The observations made in Chapter~\ref{chp:binary} are consistent for both models.}
	\label{tab:app:LPMreal}
	\begin{center}
		\begin{small}
			\begin{tabular}{|c|c|c|c|c|}
				\hline
				&
				\multicolumn{2}{|c|}{$\lambda = 10$} &  \multicolumn{2}{|c|}{$\lambda = 1$} \\ \hline
				$\epsilon$ & \textsc{M}  & \textsc{BC} & \textsc{M}  & \textsc{BC} \\
				\hline
				0.02 & 0.57 & 0.79  & 0.54 & 0.79 \\
				0.05 & 0.14  & 0.43  & 0.36 & 0.64 \\
				0.08 & 0.07  & 0.21  & 0.14 & 0.57 \\
				0.11 & 0.00  & 0.07  & 0.07 & 0.43 \\
				\hline
			\end{tabular}
		\end{small}
	\end{center}
	\vskip -0.3cm
\end{table}

\vspace{-0.2cm}
\section{Monotonically Decreasing Case}\label{appendix:decreasing}

If the oracle's metric is monotonically decreasing in \emph{TP, TN}, we can find the supporting hyperplanes at the maximizer and the minimizer. It would require to pose one query $\Omega(C^*_{\pi/4}, C^*_{5\pi/4})$.  The response determines whether we want to search over $\partial\Ccal_+$ or $\partial\Ccal_-$ and apply Algorithms \ref{bin-alg:linear} and 3.2  accordingly. If $C^*_{\pi/4} \prec C^*_{5\pi/4}$, then the metric is monotonically decreasing, and we search for the maximizer on the lower boundary $\partial\Ccal_-$ (and vice-versa).

\begin{figure}[H]
	\centering 
	\subfigure[Table \ref{tab:app:LFPMtheoryreal}, Line 1, Column 5]{
		{\includegraphics[width=5cm]{binary/plots/lfpm_1_magic_AISTATS}}
    \label{fig:app:lfpm_1_magic}
    }
	\subfigure[Table \ref{tab:app:LFPMtheoryreal}, Line 2, Column 5]{
		{\includegraphics[width=5cm]{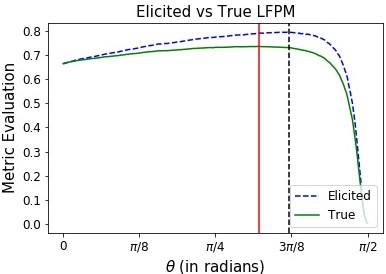}}
		\label{fig:app:lfpm_2_magic}
	}
	\subfigure[Table \ref{tab:app:LFPMtheoryreal}, Line 3, Column 5]{
		{\includegraphics[width=5cm]{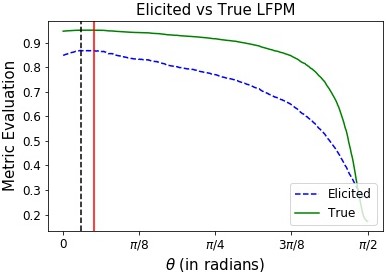}}
		\label{fig:app:lfpm_3_magic}
	}
	\subfigure[Table \ref{tab:app:LFPMtheoryreal}, Line 4, Column 5]{
		{\includegraphics[width=5cm]{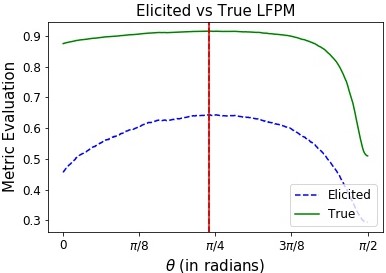}}
		\label{fig:app:lfpm_4_magic}
	}
	\subfigure[Table \ref{tab:app:LFPMtheoryreal}, Line 5, Column 5]{
		{\includegraphics[width=5cm]{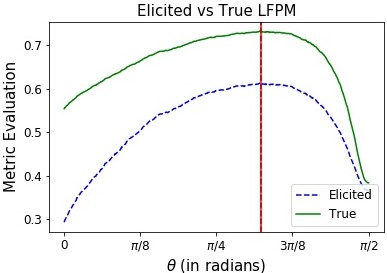}}
		\label{fig:app:lfpm_5_magic}
	}
	\subfigure[Table \ref{tab:app:LFPMtheoryreal}, Line 6, Column 5]{
		{\includegraphics[width=5cm]{binary/plots/lfpm_6_magic_AISTATS}}
		\label{fig:app:lfpm_6_magic}
	}
	\caption{True and elicited LFPMs for dataset \textsc{M} from Table \ref{tab:app:LFPMtheoryreal}. The solid green curve and the dashed blue curve are the true and the elicited metric, respectively. The solid red and the dashed black vertical lines represent the maximizer of the true metric and the elicited metric, respectively. We see that the elicited LFPMs are constant multiple of the true metrics with almost the same maximizer (solid red and dashed black vertical lines overlap except for two cases).}
	\label{fig:app:lfpm-real}
\end{figure}  

%% file: multiclass/supplement.tex
Let $f_X$ be the marginal distribution for $\Xcal$.  

\section{ShrinkInterval-1 and ShrinkInterval-2 Subroutines}
\label{append:sec:shrink}

\begin{figure}[t]
\begin{minipage}[h]{\textwidth}
  \centering \hspace{-0.5em}
  \begin{minipage}[h]{.48\textwidth}
     \centering
\fbox{\parbox[t]{1.0\textwidth}{\footnotesize{\underline{\bf Subroutine \emph{ShrinkInterval-1}}\normalsize}    \\
\small
\textbf{Input:} Oracle responses for $\Omega(\dmbfbar^c_{1, i}, \dmbfbar^a_{1, i}),$ \\
\text{ \ \ \ \ }\text{ \ \ \ \ } $\Omega(\dmbfbar^d_{1, i}, \dmbfbar^c_{1, i}), \Omega(\dmbfbar^e_{1, i}, \dmbfbar^d_{1, i}), \Omega(\dmbfbar^b_{1, i}, \dmbfbar^e_{1, i})$.\\
\textbf{If} \, ($\dmbfbar^a_{1, i} \succ \dmbfbar^c_{1, i}$) $m^b = m^d$.\\
\textbf{elseif} \, ($\dmbfbar^a_{1, i} \prec \dmbfbar^c_{1, i} \succ \dmbfbar^d_{1, i}$) $m^b = m^d$.\\
\textbf{elseif} \, ($\dmbfbar^c_{1, i} \prec \dmbfbar^d_{1, i} \succ \dmbfbar^e_{1, i}$) $m^a = m^c$,  $m^b = m^e$.\\
\textbf{elseif} \, ($\dmbfbar^d_{1, i} \prec \dmbfbar^e_{1, i} \succ \dmbfbar^b_{1, i}$) $m^a = m^d$.\\
\textbf{else} 
 $m^a = m^d$.\\
\textbf{Output:} $[m^a, m^b]$. 
\normalsize \vspace{-0.25em}
}}
  \end{minipage} \hspace{0.1em}
  \begin{minipage}[h]{.495\textwidth}
     \centering
\fbox{\parbox[t]{1.0\textwidth}

\begin{tikzpicture}[scale = 3]
    

    	\begin{scope}[shift={(-5.0,0)},scale = 0.565]\scriptsize

\def\r{0.06};
	
    
    \draw[thick] (0,0) .. controls (0.2,0) and (0.4,1.6) .. (0.8,1.6) 
    ..controls (1.4,1.6) and (2.2,0) .. (4,0);
    
    \draw[-latex] (0,-.1)--(0,2.505); 
    \draw[-latex] (-0.1,0)--(4.4,0);
    \node[left] at (0,2.35) {$\psi$};
    \node[below right] at (4.1,0) {$m$};
   
   	\coordinate (C1) at (0,0.00);
    \coordinate (C2) at (1,1.56);
    \coordinate (C3) at (2,0.76);
    \coordinate (C4) at (3,0.18);
    \coordinate (C5) at (4,0.00);
    
    \node[below] at (0,0) {$m^a$};
    \node[below] at (1,0) {$m^c$};
    \node[below] at (2,0) {$m^d$};
    \node[below] at (3,0) {$m^e$};
    \node[below] at (4,0) {$m^b$};
    
    \foreach \x in {1,2,3,4} {
    	\draw (\x,-.1) -- (\x,.1);
        \draw[dotted] (\x,0) -- (\x,2);
    }
    \fill[color=black] 
    		(C1) circle (\r)
    		(C2) circle (\r)
            (C3) circle (\r)
            (C4) circle (\r)
            (C5) circle (\r);   
    
    \draw[very thick,-latex] (0.3,0.4) -- (0.7,0.8);
    \draw[very thick,-latex] (1.7,0.4) -- (1.3,0.8);
    \draw[very thick,-latex] (2.7,0.4) -- (2.3,0.8);
    \draw[very thick,-latex] (3.7,0.4) -- (3.3,0.8);
    
    \draw (0,1.8)--(0,2.2) (2,1.8)--(2,2.2);
    \draw[<->] (0,2)--(2,2);
    
    \end{scope}
\end{tikzpicture}
}
  \end{minipage}
\end{minipage}
\caption{(Left): Description of Subroutine \emph{ShrinkInterval-1}. (Right): Visual intuition of the subroutine \emph{ShrinkInterval-1}; in search of the maximizer of a quasiconcave metric $\psi$, the subroutine shrinks the current interval to half based on oracle responses to the four queries.}
\label{mult-append:fig:shrink1}
\end{figure}
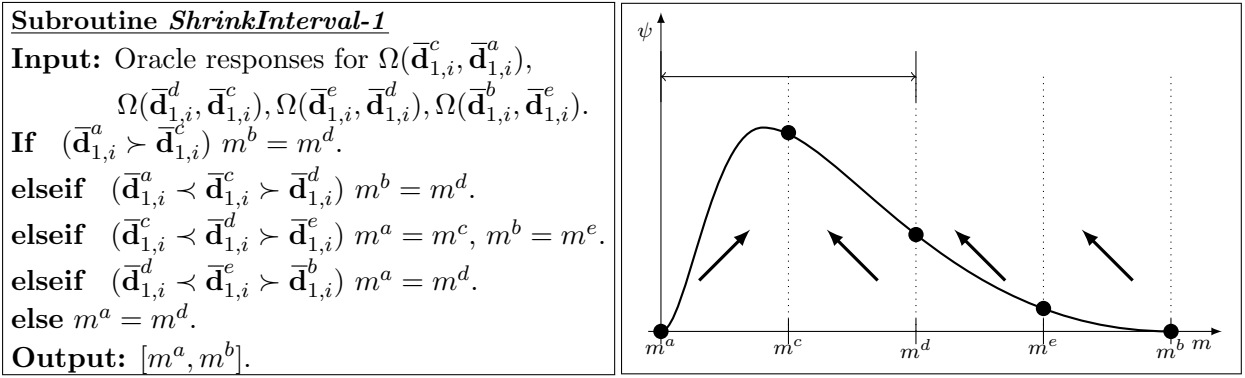

\begin{figure}[t]
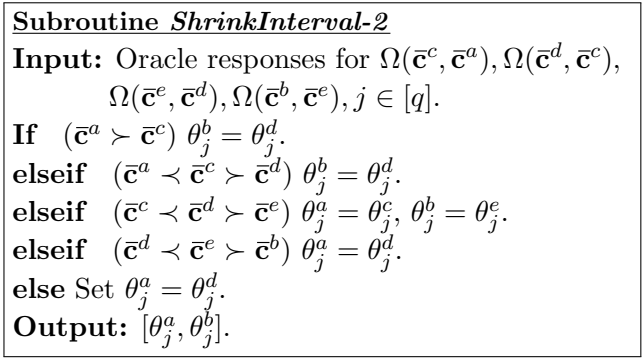

\begin{minipage}[h]{\textwidth}
  \centering \hspace{-0.5em}
  \begin{minipage}[t]{.5\textwidth}
     \centering
\fbox{\parbox[t]{1.0\textwidth}{\footnotesize{\underline{\bf Subroutine \emph{ShrinkInterval-2}}\normalsize}    \\
\small
\textbf{Input:} Oracle responses for $\Omega(\cmbfbar^c, \cmbfbar^a),\Omega(\cmbfbar^d, \cmbfbar^c),$ \\
\text{ \ \ \ \ \ \ \ \ } $\Omega(\cmbfbar^e, \cmbfbar^d), \Omega(\cmbfbar^b,  \cmbfbar^e), j \in [q].$\\
\textbf{If} \, ($\cmbfbar^a \succ \cmbfbar^c$) $\theta_j^b = \theta_j^d$.\\
\textbf{elseif} \, ($\cmbfbar^a \prec \cmbfbar^c \succ \cmbfbar^d$) $\theta_j^b = \theta_j^d$.\\
\textbf{elseif} \, ($\cmbfbar^c \prec \cmbfbar^d \succ \cmbfbar^e$) $\theta_j^a = \theta_j^c$,  $\theta_j^b = \theta_j^e$.\\
\textbf{elseif} \, ($\cmbfbar^d \prec \cmbfbar^e \succ \cmbfbar^b$) $\theta_j^a = \theta_j^d$.\\
\textbf{else} 
 Set $\theta_j^a = \theta_j^d$.\\
\textbf{Output:} $[\theta_j^a, \theta_j^b]$. 
\normalsize
}}
  \end{minipage} 
\end{minipage}
\caption{Formal description of the subroutine \emph{ShrinkInterval-2}. \emph{ShrinkInterval-2} is same as \emph{ShrinkInterval-1} except that it applies to the parameter $\theta_j$ and works with responses to off-diagonal confusions based queries.}
\vskip -0.2cm
\label{append:fig:shrink2}
\end{figure}

Notice that both \emph{ShrinkInterval} sub-routines work with responses to four queries, and based on the responses divides the interval into two. Since the metric  dealt in Algorithm~\ref{mult-alg:dlpm} is concave and unimodal (see Lemma~\ref{lemma:slice} and Remark~\ref{rem:concave}), four queries are required to shrink the interval into by half in every iteration. Since we use the enclosed sphere for LPM elicitation, we can shrink the interval into half based on just two queries in Algorithm~\ref{mult-alg:lpm}, i.e. by querying $\Omega(\cmbfbar^d, \cmbfbar^c)$ and $\Omega(\cmbfbar^e,  \cmbfbar^d)$, due to strong convexity of the sphere (see proof of Theorem~\ref{thm:lpm-elit-error}). However, we show use of four queries in Algorithm~\ref{mult-alg:lpm} just to make the algorithms consistent for the readers to understand.  

\section{Proofs of Section~\ref{mult-sec:confusion}}
\label{append:sec:confusion}

\bproof[Proof of Proposition~\ref{prop:Cd}] The following are the properties of $\Dcal$.

\bitemize[leftmargin=1em]
\item \emph{Convex}: Let us take two classifiers $h_1, h_2 \in \Hcal$ which achieve the diagonal confusions $\dmbf{(h_1)}, \dmbf{(h_2)} \in \Dcal$. We need to check  whether there exists a classifier, which achieves the off-diagonal confusion $\lambda \dmbf{(h_1)} + (1-\lambda)\dmbf{(h_2)}$. Consider a classifier $h$, which with probability $\lambda$ predicts what classifier $h_1$  predicts and with probability $1-\lambda$ predicts what classifier $h_2$ predicts. Then the first component 
\begin{align*}
d_{1}(h) &= \Pmbb(Y=1, h=1)  \\ \nonumber
&= \Pmbb(Y=1, h=h_1|h=h_1)\Pmbb(h=h_1) + \Pmbb(Y=1, h=h_2|h=h_2)\Pmbb(h=h_2)   \\ \nonumber
&= \lambda d_{1}{(h_1)} + (1-\lambda)d_{1}{(h_2)}. \numberthis 
\end{align*}
Similarly, this hold true for $d_{i}{(h)}$ for $i \in [k]$. Hence, $C$ is convex.
\item \emph{Bounded:} Since $D_{i} = P[Y=i, h=i]\leq\zeta_i$ for all $i \in [K]$, $\Dcal \subseteq [0, \zeta_1] \times \cdots \times [0, \zeta_k]$.
\item \emph{Strictly convex and closed:} Since $\Ccal$ is convex, its boundary is intersection of half spaces. Furthermore, any linear functional is maximized at the boundary of a convex set~\cite{boyd2004convex}. Suppose we are given a diagonal linear functional (DLPM) $\ambf$. The BO classifier $h^\ambf$ for that function is given by Proposition~\ref{prop:d-bayes} (whose proof is discussed later). Let the value achieved by the corresponding $BO$ diagonal confusion $\dbar$ is $\alpha$. That is,

\begin{align}
\alpha &= \sum_{i=1}^k a_id_{i} = \sum_{i=1}^k \int_\Xcal a_i \eta_i(\xmbf) \1[h^\ambf(\xmbf) = i|X=\xmbf]df_X. 
\end{align}

Now, if we want to construct another classifier which achieves the same value $\alpha$, there has to be some weight shift from one class to another class without changing the maximum value $\alpha$, but note that $\Pmbb[a_i\eta_i(X) = a_j\eta_j(X)]=0$ for all $i, j \in [k]$ due to Assumption~\ref{mult-as:eta}. 

\noindent Hence, there is a unique maximizer of this linear functional on the boundary. Therefore, the space is strictly convex. One characterization of the boundary of the space $\partial \Dcal$ can be given by BO diagonal-confusions corresponding to any linear functional $\ambf$. These diagonal confusions are achieved by the corresponding BO classifiers. Therefore, these diagonal confusions are always achievable, and the space is closed as well. 

\item \emph{$\vmbf_i$ are always achieved:} It is easy to see that any trivial classifier which predicts only class $i\in [k]$, will achieve the diagonal confusion defined by $\vmbf_i$. 
\item \emph{$\vmbf_i$ are the only vertices:} Certainly, a vertex exists if (and only if) some point is supported by more than $k$ tangent hyperplanes in $k$ dimensional space. This means that the vertex is optimal for more than $k$ linear metric (linear functional). Clearly, all the metrics with slope $\ambf$ such that $a_i > a_j > 0$ and $a_l=0 \; \forall \; l \in [k], l\neq i, j$ support $\vmbf_i$. So, there are at least $k$ supporting hyperplanes at these points, which make them the vertices. Now, we show that these are the only vertices.

Suppose there is a point other than $\vmbf_i$'s which is supported by two hyperplanes given by the slopes $\ambf^1$ and $\ambf^2$. From Proposition~\ref{prop:d-bayes} (discussed later), we can get Bayes optimal classifiers $h^{\ambf^1}$ and $h^{\ambf^2}$, which achieve the same diagonal confusions. This means that
\begin{align}
    &\int_{\xmbf:\frac{\eta_1(\xmbf)}{\eta_j(\xmbf)} \geq t_j, j \in \{2, \cdots, K\}} \eta_1(\xmbf)df_X = \int_{\xmbf:\frac{\eta_1(\xmbf)}{\eta_j(\xmbf)} \geq t'_j, j \in \{2, \cdots, K\}} \eta_1(\xmbf)df_X,  
\end{align}
i.e., the first component $d_{1}$ should be equal for the two classifiers, where $t_j, t'_j$'s are dependent on $\ambf^1$ and $\ambf^2$. Since, these classifiers are different at least for one $j$, $t_j \neq t'_j$. This will mean that there are multiple values of $\frac{\eta_1(\xmbf)}{\eta_j(\xmbf)}$ which are not attained. This contradict with our Assumption~\ref{mult-as:eta} that $g_{1j}$ is strictly decreasing. By strict convexity, there are no supporting hyperplane tangent at multiple points. Hence, $\vmbf_i$ are the only vertices of the set $\Dcal$.   
\eitemize

Since we take classifiers which predict only classes $k_1$ and $k_2$, the values of any diagonal confusion $\dmbf \in \Dcal_{k_1, k_2}$ evaluate to zero at indices except $k_1, k_2$. Therefore, the properties of the space $\Dcal_{k_1, k_2}$ can be proved on similar lines to Chapter~\ref{chp:binary}. 
\eproof

\bproof[Proof of Proposition~\ref{mult-prop:C}] The following are the properties of the space $\Ccal$.
\bitemize[leftmargin=1em, noitemsep]
\item \emph{Convex} The space is convex follows from first point of Proposition~\ref{prop:Cd}.
\item \emph{Bounded:} $C_{ij} = \Pmbb[Y=i, h=j] \leq \Pmbb[Y=i]=\zeta_i$ for $i, j \in [k]$. When confusion matrices written in row major form excluding the diagonal terms, then it is easy to see that $\Ccal \subseteq [0, \zeta_1]^{(k-1)}\times [0, \zeta_2]^{(k-1)}\times \cdots \times [0, \zeta_k]^{(k-1)}$. 
\item \emph{$\umbf_i$'s and $\ombf$ are always achieved:} The classifier which always predicts class $i$, will achieve the confusion matrix $\umbf_i$. Thus, $\umbf_i \in \Ccal \, \forall\, i \in [q]$. Furthermore, a classifier which predicts similar to one of the trivial classifiers with probability $1/k$ will achieve the confusions $\ombf$ (the centroid). 
\item \emph{$\umbf_i$'s are vertices:} Any supporting hyperplane with slope $a_{1i} < a_{1j} < 0$ and $a_{1l}=0$ for $l \in [k], l \neq i, j$ will be supported by $\umbf_1$ (corresponding to BO classifier which predict class 1). Thus, $\umbf_1$ is supported by at least $q$ hyperplanes. Thus, it becomes a vertex of the convex set. Similar is the case with other $\umbf_i$'s.
\eitemize
\eproof

Proposition~\ref{prop:d-k1k2-bayes} can be considered as a corollary of the following more general Proposition.

\bprop
Let $\psi \in \varphi_{DLPM}$, parametrized by $\ambf$, then 
\begin{align}
\hbar(\xmbf) &= \argmax_{i \in [k]} a_i\eta_i(\xmbf), \, \; \; \text{and} \, \;\; 
\barbelow{h}(\xmbf) = \argmin_{i \in [k]} a_i\eta_i(\xmbf)
\label{eq:bayes-dlpm}
\end{align}
are the BO and IBO classifiers \emph{w.r.t} $\psi$, respectively.
\label{prop:d-bayes}
\eprop

\bproof
Let
\begin{align}
\psi &= \sum_i a_i d_{i} = \sum_i \int_\Xcal a_i \eta_i(\xmbf)\1[h(\xmbf)=i]. 
\end{align}
From this mathematical form, it is easy to see that the metric achieves its maximum when a class that maximizes the
expected utility conditioned on the instance is predicted. That is, the metric achieves its maximum when a classifier deterministically predicts class $i$ when $i=\argmax_{j\in [k]} a_j\eta_j(x)$. This is the form of the classifier written in the proposition. Similarly, this metric is minimized when when a classifier minimizes the
expected utility conditioned on the instance, by predicting class $i = \argmin_{j\in [k]} a_j\eta_j(x).$
\eproof

\bproof[Proof of Proposition~\ref{prop:d-k1k2-bayes}]
Recall that classifiers which predict only class $k_1$ and $k_2$ will achieve diagonal confusions, which have zeros at every other index except $k_1, k_2$. Therefore, 
\begin{align*}
\psi &= \sum_i a_i d_{i} = a_{k_1}d_{k_1} + a_{k_2}d_{k_2}\\
&= \int_\Xcal a_{k_1} \eta_{k_1}(x)\1[h(x)=k_1] + \int_\Xcal a_{k_2} \eta_{k_2}(x)\1[h(x)=k_2]. \numberthis 
\end{align*}
Again, using the idea used in the previous proof, the metric achieves its maximum when a class that maximizes the
expected utility conditioned on the instance is predicted. Therefore, 
\bequation
\hbar_{k_1, k_2}(x) = \left\{\begin{array}{lr}
			 k_1, & \; \text{if} \; a_{k_1} \eta_{k_1}(\xmbf) \geq a_{k_2} \eta_{k_2} (\xmbf) \\
			 k_2,& \;  o.w. 
	 	 \end{array}\right\}
\eequation
is the RBO classifier (restricted to classes $k_1, k_2$) with respect to
$\psi$. Furthermore, the RIBO classifier is given by $\barbelow{h}_{k_1, k_2}(\xmbf)= k_2\1[\hbar_{k1, k_2}(\xmbf)=k_1] + k_1\1[\hbar_{k1, k_2}(\xmbf)=k_2]$.
RIBO classifier does exactly the opposite of RBO, i.e., it predicts class $k_1$, wherever RBO predicts class $k_2$ on the instance space $\Xcal$ and vice-versa.
\eproof

\bproof[Proof of Lemma~\ref{mult-lem:spherebayes}]
Suppose the origin is at $\ombf$ and the constrained set is the sphere $
\Scal_\lambda$ with radius $\lambda$ centered at $\ombf$. We want to maximize $\inner{\ambf}{\cmbf}$ such that $\cmbf \in \Scal_\lambda$. Since a linear metric over a convex set is maximized at the boundary~\cite{boyd2004convex}, it is easy to see that $c_i = \lambda a_i$ will maximize this metric. Moving the reference point to the original origin i.e. $\bm{0}^q$ gives us the required answer.  
\eproof

\section{Proofs of Section~\ref{mult-sec:me}}
\label{append:sec:me}

We write Lemma~\ref{lemma:slice} in the following more general form.

\blemma
Let $\psi:\Dcal \to \mathbb R \, (\xi:\Dcal \to \mathbb R)$  be a quasiconcave (quasiconvex) function, which is 
monotone increasing in all $\{d_{i}\}_{i=1}^k$. For $k_1, k_2 \in [k]$, let 
$\rho^+:[0,1]\to \partial\Dcal^+_{k_1, k_2}$ ($\rho^-:[0,1]\to \partial\Dcal^-_{k_1, k_2}$) be a continuous, bijective, parametrization of the upper (lower) boundary. 
Then the composition $\psi \circ \rho^+: [0,1]\to\mathbb R$ ($\xi \circ \rho^-: [0,1]\to\mathbb R$) is
quasiconcave (quasiconvex) and thus unimodal on the interval $[0, 1]$. 
\label{lemma:sliceext}
\elemma

\bproof
A function is quasiconcave iff super-level sets are convex.  We already know from Proposition~\ref{prop:Cd} $\Dcal_{k_1, k_2}$ is convex. Moreover, any vector of diagonal confusions has zeros at every index except at indices $k_1, k_2$. Let $\psi:\Dcal\rightarrow \Rmbb$ be a quasiconcave metric, which implies  that its super-level sets $\Lcal_r^{\Dcal}(\psi) = \{\dmbf \in \Dcal \;:\; \psi(\dmbf) \geq r\}$ are convex. Now, consider the super-level sets of $\psi$ restricted to the diagonal confusions in $\Dcal_{k_1, k_2}$ i.e. $\Lcal_r^{\Dcal_{k_1, k_2}}(\psi) = \{\dmbf \in \Dcal_{k_1, k_2} \;:\; \psi(\dmbf)\geq r\}$. Take any $\dmbf^1, \dmbf^2 \in \Lcal_r^{\Dcal_{k_1, k_2}}(\psi)$. Since $\dmbf^1, \dmbf^2 \in \Dcal$ as well, they belong to the set 
$\Lcal_r^{\Dcal}(\psi)$, which is convex. Hence, for $t \in [0, 1]$, 
$t\dmbf^1 + (1-t)\dmbf^2 \in \Lcal_r^{\Dcal}(\psi)$, which implies that $\psi(t\dmbf^1 + (1-t)\dmbf^2)\geq r$. 
Furthermore, $t\dmbf^1 + (1-t)\dmbf^2 \in \Dcal_{k_1, k_2}$, because $\Dcal_{k_1, k_2}$ is convex. By the above two arguments, we have that $t\dmbf^1 + (1-t)\dmbf^2 \in \Lcal_r^{\Dcal_{k_1, k_2}}(\psi)$. This implies that $\Lcal_r^{\Dcal_{k_1, k_2}}(\psi)$ is convex, and hence $\psi$ restricted to $\Dcal_{k_1, k_2}$ is quasiconcave. The proof analogously follows for quasiconvex metric $\xi$. 

Now, it remains to show that $\psi \circ \rho^+: [0,1]\to\mathbb R$ ($\psi \circ \rho^-: [0,1]\to\mathbb R$) is
quasiconcave (quasiconvex). This can be proved by readily extending the proof of Lemma~\ref{lem:quasiconcave} (Chapter~\ref{chp:binary}) to the diagonal multiclass case. For the sake of completeness, we also provide the proof here. 

We will prove the result for $\psi \circ \rho^+$ on $\partial\mathcal D^+_{k_1, k_2}$, and the argument for $\xi\circ \rho^-$ on $\partial\mathcal D^-_{k_1, k_2}$ is essentially the same. For simplicity, we drop
the $+$ symbols in the notation. It is given that $\psi$ is quasiconcave. Let $S$ be some superlevel set of $\psi$. We first want to show that for any $r<s<t$, if $\rho(r)\in S$ and $\rho(t)\in S$, then $\rho(s)\in S$. Since $\rho$ is a continuous bijection, due to the geometry of $\Dcal_{k_1, k_2}$, we must have --- wlog ---
$d_{k_1}(\rho(r))< d_{k_1}(\rho(s)) < d_{k_1}(\rho(t))$, and $d_{k_2}(\rho(r))>d_{k_2}(\rho(s))>d_{k_2}(\rho(t))$ (otherwise swap $r$ and $t$). Since the set $\Dcal_{k_1, k_2}$ is strictly convex and the image of $\rho$ is $\partial \Dcal_{k_1, k_2}$, then $\rho (s)$ must dominate (component-wise) a point in the convex combination of $\rho(r)$ and $\rho(t)$. Say that point is $z$. Since $\psi$ is monotone increasing, then $x\in S\implies y \in S$ for all $y\geq x$ component-wise. Therefore, $\psi(\rho(s)) \geq \psi(z)$. Since, $S$ is convex, $z \in S$ and, due to the argument above, $\rho(s) \in S$.

This implies that $\rho^{-1}(\partial \mathcal D_{k_1, k_2} \cap S)$ is an interval, and is therefore convex. Thus, the superlevel sets of $\psi\circ \rho$ are convex, so it is quasiconcave,
as desired. This implies unimodaltiy as a function over the real line since a function which has more than one local maximum can not be quasiconcave (consider the super-level set for some value slightly less than the lowest of the two peaks).
\eproof

\section{Proofs of Section~\ref{mult-sec:guarantees}}\label{append:sec:guarantees}
\vskip -0.6cm

\begin{proof}[Proof of Theorem~\ref{thm:dlpm-elit-error}]
     In Chapter~\ref{chp:binary}, it is shown that for binary classification, the inner loop of Algorithm~\ref{mult-alg:dlpm} will estimate the value of $\hat m$ for the Bayes-optimal binary classifier corresponding to a linear metric $\ambf^\ast=(m^\ast,1-m^\ast) \in \Rmbb^2$, such that $|\hat m-m^\ast|<\epsilon+\sqrt{\epsilon_{\Omega}}$ after $O(\log \tfrac1\epsilon)$ iterations. Now, in the multiclass case, this allows us to argue that, for any $1\leq i<j\leq k$, we can estimate a value $m_{ij}$ such that $a^\ast_i/a^\ast_j = (1-m_{ij})/m_{ij}$.
    
    For the required guarantees, wlog, we assumed throughout the algorithm that $a_1 \geq a_k/2$ for all $k$. This is because, if $a_1$ does not satisy this condition, then we can always choose an index $z \in [k]$ which does satisfy this from the following procedure:
\balgorithmic
\STATE Set $z\gets 1$
\FOR{$t=2, 3, \cdots, k$}
\STATE Compute an estimate $\hat m_{tz}$ of $m_{tz}$.
\STATE \textbf{if} {$\hat m_{tz}<\tfrac 12$} \textbf{then} {$z\gets t$} \textbf{else} {do nothing} 
\ENDFOR
\STATE\textbf{Output:} $z$.
\ealgorithmic
\vskip -0.2cm
Let $\varepsilon = \epsilon+\sqrt{\epsilon_\Omega}$.
Now, if $\hat m_{tz}<\tfrac 12$, then $a^\ast_t\geq a^\ast_z\cdot (\tfrac 12-\varepsilon)/(\tfrac 12+\varepsilon)=\frac{1-2\varepsilon}{1+2\varepsilon}$.
It can be shown that this ratio is at least $1-4\varepsilon$. Therefore, if $z$ is the final coordinate output, we must have that $a_z\geq (1-4\varepsilon)^k a_t$ for all $t$. But $(1-4\varepsilon)^k\approx e^{-4k\varepsilon}$, and so for $\varepsilon$ sufficiently small, we have $a_z\geq a_t/2$ for all $t$ as desired.
    Now that we have our assumption, we may proceed to show that the algorithm is correct. We wish to show that $\Vert \ambfhat/|\hat a_z| - \ambf/|a_z|\Vert_{\infty}<O(\varepsilon)$.
We have 
\begin{align*}
    \left|\frac {\hat a_t}{\hat a_z} - \frac{ a_t}{ a_z}\right| &= 
    \left|\frac{1-\hat m_t}{\hat m_t}-\frac{1- m_t}{ m_t}\right|
    = \left|\frac 1{\hat m_t}-\frac 1{ m_t}\right| \\
    &\leq \frac 1{m_t-\varepsilon}-\frac 1{m_t}\leq \frac 1{m_t}\left(\frac 1{1-2\varepsilon}-1\right)
    \leq 2\cdot 2\varepsilon/(1-2\varepsilon)\leq 5\varepsilon \numberthis
\end{align*}
for $\varepsilon<0.1$. This gives us the deisred bound.
\end{proof}

\begin{proof}[Proof of Theorem~\ref{thm:lpm-elit-error}]

Consider the geometry shown in the Figure~\ref{append:fig:theorem2} (left). This shows a function $f [-1,1]^q \rightarrow \Rmbb$ which follow the trajectory of a unit semicircle (semisphere). Let $\xmbf$ be a q-dimensional vector, then this function is given by:
\begin{equation}
    f(\xmbf) = 1 - \sqrt{1 - \sum_i^q x_i^2}.
    \label{eq:distance}
\end{equation}

\begin{figure}[t]
\begin{minipage}[h]{\textwidth}
  \centering \hspace{-0.5em}
  \begin{minipage}[h]{.48\textwidth}
     \centering
\fbox{\parbox[t]{1.0\textwidth}

\begin{tikzpicture}[scale = 2.5]


     \begin{scope}[shift={(4.3,0)},scale = 0.5]\scriptsize
     
     \def\r{0.08};
	
    
    \draw[-latex] (2,-.1)--(2,2.2); 
    \draw[-latex] (-0.1,0)--(4,0);
    \node[left] at (2,2) {$f(\xmbf)$};
    \node[below right] at (3.8,0) {$\xmbf$};
    
    \coordinate (Cent) at (2, 1);
    \coordinate (xl) at (1.3, 0.3);
    \coordinate (xr) at (2.9, 0.6);
    \coordinate (xa) at (1.3, 0);
    \coordinate (xb) at (2.9, 0);
    \coordinate (xstar) at (2.1, 0);
    
    \fill[color=black] 
            (Cent) circle (\r)
            (xl) circle (\r)
            (xr) circle (\r);
    \foreach \x in {1,3} {
    	\draw (\x,-.1) -- (\x,.1);
        \draw[dotted] (\x,0) -- (\x,2);
    }
    \draw[dotted] (1.3,0) -- (1.3,0.3);
    \draw[dotted] (2.9,0) -- (2.9,0.6);
    
     \node[below] at (xa) {$\xmbf^a$};
     \node[below] at (xb) {$\xmbf^b$};
     \node[below] at (xstar) {$\xmbf^\ast$};
    
    
    \draw[thick] (3,1) arc (360:180:1);

     \end{scope}

\end{tikzpicture}
}
  \end{minipage} \hspace{0.3em}
  \begin{minipage}[h]{.48\textwidth}
     \centering
\fbox{\parbox[t]{1.0\textwidth}

\begin{tikzpicture}[scale = 1]
    

    	\begin{scope}[shift={(4.3,0)},scale = 0.5]\scriptsize
     
     \def\r{0.12};
	
    
    \coordinate (C*) at (3.8,3.05);
    \coordinate (l*) at (0.0,4.0);
    \coordinate (Ct) at (2.7,0.3);
    \coordinate (lt) at (-0.6,0.45);
    \coordinate (Cent) at (3,1.75);
    \coordinate (Space) at (0.20,0.1);
    \coordinate (Sphere) at (5,2.65);
    \coordinate (Lambda) at (3.75,2);
    \coordinate (C) at (1.525,1.35);
    \coordinate (fC) at (0.80,1.25);
    
    
    \draw (C) -- (1.25,0.70);
    \draw (Cent) -- (4.5, 1.75);
    
    \draw[-latex] (-2,1.75)--(7,1.75); 
    \draw[-latex] (3,-1.75)--(3,4.75);
    
    \fill[color=black] 
            (Cent) circle (0.08)
            (Ct) circle (\r)
            (C) circle (0.04);
            

    \draw (Cent) circle (1.5cm);
    
    \node[above left] at (lt) {$\bar{\ell}^\ast$};
    \node at (Sphere) {\large{$\Scal_\lambda$}};
    \node[right] at (C) {\tiny$\mu(\thetambf)$\normalsize};
    \node at (fC) {\tiny$f^\ast_{\mu(\thetambf)}$\normalsize};
    
    \draw[dotted] (Cent) +(-3,0.9) -- +(2,-0.6);
   	
    \draw (Ct) +(-4,1) -- +(3,-0.75);
    
    \draw[thick] (4.45,1.4) arc (345:163:1.5);
    \node[below left] at ($(Ct)+(0,0.15)$) {$\mu(\thetambf^*)$};
    \node[below] at (Cent) {$\ombf$};
     \end{scope}
\end{tikzpicture}
}
  \end{minipage}
\end{minipage}
\caption{(Left): A function for the semicircle with unit radius. (Right): Visual intuition for the distance between the boundary points and tangent place at the optimal off-diagonal confusions.}
\label{append:fig:theorem2}
\end{figure}
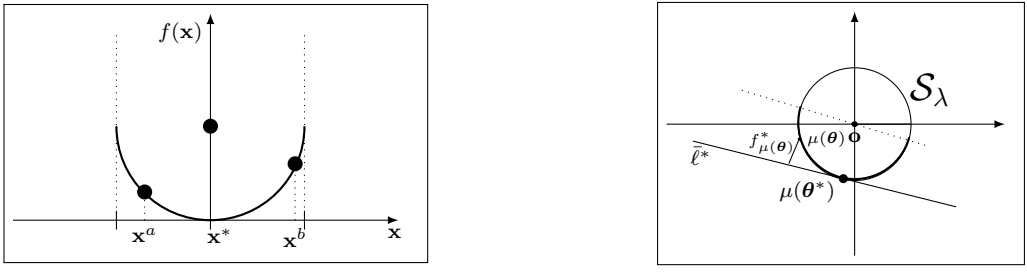

Intuitively, this function evaluates the distance of the points lying on the surface of the semisphere. The point $\xmbf^\ast$ (the origin) is the unique minimizer of this function. Let us restrict the domain of this function to the points $Q = [\xmbf^a, \xmbf^b]$, where $\xmbf^a > -1$ (component-wise) and $\xmbf^b < 1$ (component-wise). Then it is easy to see that the derivative of this function:
\begin{align}
    \nabla f = \left( \frac{x_1}{\sqrt{1 - \sum_i^q x_i^2}}, \dots, \frac{x_q}{\sqrt{1 - \sum_i^q x_i^2}}\right)
\end{align}
is continuously differentiable on a compact domain $Q$. Thus, $\nabla f$ is Lipschitz with some Lipschitz parameter $L$ i.e.:
\begin{equation}
    \Vert \nabla f(\ymbf) - \nabla f(\xmbf) \Vert_2 \leq L \Vert y - x \Vert_2
\end{equation}
which makes the function $f$ to be $L$-smooth. In addition, we observe that:
\begin{align*}
    f(\xmbf) = 1 - \sqrt{1 - \sum_i^q x_i^2}  &\geq \frac{1}{2}\sum_i^q x_i^2. \numberthis
\end{align*}
This implies that there exists a paraboloid always below the function $f$, which by definition, makes the function $f$ a strongly convex function (say with strong convexity parameter $\tau$). Thus, this function satisfies all the requirements i.e smoothness, strong convexity, and has unique minimizer, to inherit the guarantees from Derivative Free Optimization~\cite{jamieson2012query}. Notice that if we apply the coordinate-wise binary search Algorithm~\ref{mult-alg:lpm}, where the inner loop is run for $\log(1/\epsilon)$ queries, to minimize this function using pairwise comparison queries (i.e. the oracle responds with the point that evaluate to lesser value of $f$ out of the two), then by Theorem~5 of~\cite{jamieson2012query} one can guarantee that after $\frac{4L}{\tau} \log(\frac{f(\xmbf^0) - f(\xmbf^\ast)}{\epsilon^2 2 qL^2/\tau})q \log(1/\epsilon)$ queries to the oracle, we can get an estimate of the minimizer $\xmbf^T$ such that $f(\xmbf^T) - f(\xmbf^\ast) < 4qL^2\epsilon^2/\tau$. Notice that for this function $f(x^0) - f(x^\ast) = f(x^0) - 0 = f(x^0) \leq 1$.

Now, for simplicity assume $\lambda=1$. As we discussed, LPM elicitation problem, where queries are asked on a sphere $S_\lambda$ has a dual form, where we use a $(q-1)$ dimensional bijective parametrization based on $\thetambf$ to denote the points on the surface of the sphere. Notice that this parametrization is a function of $\sin$ and $\cos$ and hence it is Lipschitz as well. Due to monotonicity condition, we assume that the points lie on one orthant of the sphere. Now, suppose the true oracle's metric is denoted by $\ambf^\ast$, where $a^\ast_i = \Pi_{j=1}^{i-1} \sin\theta_j \cos{\theta_i}$ for $i \in [q-1]$ and $a^\ast_q = \Pi_{j=1}^{q-1} \sin\theta_j$. Let us denote this parametrization of LPMs by $\Upsilon$, i.e. $\ambf^\ast = \Upsilon(\thetambf^\ast)$. This hyperplane is tangent to the unit sphere on a particular point whose coordinates are $\Upsilon(\thetambf^\ast)$ itself. Since the metric is linear, by posing pairwise comparisons to the oracle, we ask which off-diagonal confusion is closer to the hyperplane.  So, to reach the tangent point on the boundary of the sphere by pairwise comparisons, we are actually decreasing a distance-like function $f^\ast{(\cmbf)}$ shown in Figure~\ref{append:fig:theorem2} (right). This function can be represented as $
f^\ast(\thetambf) = 1 - \inner{\Upsilon(\thetambf^\ast)}{\Upsilon(\thetambf)}$
where $\Upsilon(\thetambf^\ast)$ are fixed coefficients and $\thetambf$ changes in our algorithm. This is equivalent to the $f$ function discussed above. Thus using the above guarantees, after $z_1\log(z_2/(q\epsilon^2)) (q-1)\log(1/\epsilon)$ queries to the oracle, where $z_1, z_2$ are constants independent on $\epsilon$ and $q$, we have:

\begin{align*}
    f^\ast(\thetambf) - f^\ast(\thetambf^\ast) &= f^\ast(\thetambf) - 0 \\
    &= 1 - \inner{\Upsilon(\thetambf^\ast)}{\Upsilon(\thetambf))} \\
    &\leq z_3 q\epsilon^2, \numberthis
\end{align*}
where $z_3$ is a constant depending on curvature of the above function $f$. 
This implies that:
\begin{align*}
    \Vert \ambf^\ast - \ambfhat \Vert_2^2 &= \Vert \ambf^\ast \Vert_2^2 + \Vert \ambfhat \Vert_2^2 - 2 \inner{\ambf^\ast}{\ambfhat} \\
    &= 2(1 - \inner{\ambf^\ast}{\ambfhat}) \\
    &\leq 2 z_3 q\epsilon^2. \numberthis
\end{align*}
Using the inequality proved before we have that $\Vert \ambf^\ast - \ambfhat \Vert_2 \leq O(\sqrt{q}\epsilon)$.
Therefore, in $O\left(T\log \tfrac{1}{\epsilon}\right)$, we can achieve a point $O(\sqrt{q}\epsilon)$ close to the minimizer, where the number of iterations $T \geq z_1\log(z_2/(q\epsilon^2)) (q-1)$. The term $z_1\log(z_2/(q\epsilon^2))$ can be considered as the number of cycles, but due to the curvature of the sphere, we find that it is not a dominating factor in the query complexity. For example, when working with a sphere and $\epsilon = 10^{-2}$, two cycles (i.e. $T=2(q-1)$ in Algorithm~\ref{mult-alg:lpm}) suffices in practice. Thus, updating each $\theta_j$ twice in cycles is sufficient for obtaining the required metric.

It remains to show that, whenever the queried angle is at least $\sqrt{3\epsilon_\Omega/\lambda}$ from the optimal angle, then the oracle gives a correct response. To see this, restrict attention to the hyperplane in which the current angle is moving, say $j$, for the binary-search phase of the loop. Let $\theta^\ast_j$ be the optimal angle. Observe that for any $\theta_j$ such that $\lambda \cos(\theta_j-\theta^*_j)\geq \lambda-\epsilon_\Omega$, the oracle may return a false value. This is because the performance metric is a 1-Lipschitz linear map, and the optimal value on the sphere of radius $\lambda$ is $\lambda$. However, $\cos(x)\leq 1-x^2/3$, and so for $|\theta_j-\theta^*_j| \geq \sqrt{3\epsilon_\Omega/\lambda}$, we have $\lambda \cos(\theta_j-\theta^*_j)\leq \lambda - \lambda (3\epsilon_\Omega/\lambda)/3 = \lambda-\epsilon_\Omega$.
Therefore, so long as $|\theta_j-\theta^*_j|\geq \sqrt{3\epsilon_\Omega/\lambda}$, the oracle provides a correct answer, and the binary search proceeds in the correct direction. 
\end{proof}

\subsection{Finding the Sphere $\Scal_\lambda$}
\label{append:ssec:slambda}

Now, we discuss how a sufficiently large sphere $\Scal_\lambda$ with radius $\lambda$ may be found. Consider the following optimization problem, which is a special case of OP2 in~\cite{narasimhan2018learning}. This problem corresponds to feasiblity check problem for a given off-diagonal confusion $\cmbf^0$ for small $\delta \in \Rmbb$. 

\begin{align}
    \min_{\cmbf \in \Ccal} \; 0 \qquad s.t. \;\; \Vert \cmbf - \cmbf^0 \Vert_2 \leq \delta
    \label{mult-eq:op1}
\end{align}

If a solution to the above problem exists, then Algorithm~1 of~\cite{narasimhan2018learning} returns it. Basically, the approach in~\cite{narasimhan2018learning} will try to construct a classifier whose off-diagonal confusions are $\delta$-close to the given off-diagonal confusion $\cmbf^0$. Hence, checking the feasibility. 

\balgorithm[t]
\caption{Approximating the $\lambda$ Radius}
\label{mult-alg:lambda}
\small
\balgorithmic[1]
\STATE \textbf{Input:} The center $o$ of the feasible region of classifiers.
\FOR{$j=1, 2, \cdots, q$}
\STATE Let $\mathbf e_j$ be the standard basis vector for the $j$-th dimension. 
\STATE Compute the maximum $\ell_j$ such that $o + \ell_j \mathbf e_j$ is feasible by solving~\eqref{mult-eq:op1}.
\ENDFOR
\STATE Let $CONV$ be the convex hull of $\{o\pm \ell_j\mathbf e_j\}_{j=1}^{q}$.
\STATE Compute the radius $r$ of the largest ball which can fit inside of $CONV$, centered at $o$.
\STATE\textbf{Output:} $\lambda=r$.
\ealgorithmic
\ealgorithm

Algorithm~\ref{mult-alg:lambda} computes a value of $\lambda\geq \tilde{r}/k$, where $\tilde{r}$ is the radius of the largest ball contained in the set $\Ccal$. Notice that this algorithm is run offline and does not impact query complexity. Notice that the approach in~\cite{narasimhan2018learning} is consistent, thus we should get a good estimate of the sphere, provided we have sufficient samples. 

\blemma
    Let $\tilde{r}$ be the radius of the largest ball centered at $o$ which fits in the feasible space of classifiers. Then Algorithm~\ref{mult-alg:lambda} returns a radius $\lambda\geq \tilde{r}/k$.
\elemma
\begin{proof}
    Let $\ell_j$ be as computed in the algorithm, and let $\ell:= \min_j \ell_j$. 
    We must have $\ell\geq \tilde{r}$. Furthermore, the region $CONV$ contains the convex hull of $\{o\pm \ell\mathbf e_j\}_{j=1}^{q}$. But this region contains a ball of radius $\ell/\sqrt{q} = \ell/\sqrt{k^2-k}\geq \ell/k\geq \tilde{r}/k$, and so $\lambda\geq \tilde{r}/k$.
\end{proof}

\section{Proofs of Section~\ref{mult-sec:extensions}}\label{append:sec:extensions}

\begin{proof}[Proof of Proposition~\ref{prop:d-sufficient}]
We can add a large positive constant if for any $\dmbf \in \Dcal$, $\psi(\dmbf) < 0$. The metric would remain linear fractional. So, it is sufficient to assume $\psi(\dmbf) \geq 0$. Furthermore, boundedness and scale invariance of $\psi$ implies $\psi(\dmbf) \in [0,1]$, without compromising the linear-fractional form. Now, we look at the sufficient conditions for monotonicity in $\{d_{i}\}_{i=1}^k$ and the numerator and denominator to be positive. Consider the derivative:

\begin{align*}
    \frac{\partial\psi}{\partial d_{1}} &= \frac{a_1}{\sum_i b_i d_{i} + b_0} - \frac{b_1(\sum_i a_i d_{i})}{(\sum_i b_i d_{i} + b_0)^2} \geq 0 \numberthis
\end{align*}

\noindent Assuming denominator is positive, we have the numerator to be positive and 
\begin{align*}
    a_1 \geq b_1\frac{\sum_i a_i d_{i}}{\sum_i b_i d_{i} + b_0} \implies a_1 \geq b_1\sup_{\dmbf \in\Dcal}\frac{\sum_i a_i d_{i}}{\sum_i b_i d_{i} + b_0} \implies a_i \geq b_i\btau \numberthis
\end{align*}
The above condition is necessary. Since $\btau \in [0, 1]$, by considering all the three cases $b_i=0, b_i > 0, b_i < 0$, the  following are the sufficient conditions for monotonicity: $a_1\geq b_1$ and $a_1\geq0$. Similarly, this is true for all $a_i$'s and $b_i$'s i.e. $a_i\geq b_i, a_i \geq 0 \; \forall \; i \in [k]$ for monotonically increasing DLFPMs. Furthermore, as we assumed that  $\psi \in [0, 1]$ i.e. 
\begin{align*}
    \frac{\sum_i a_i d_{i}}{\sum_i b_i d_{i}+b_0} \leq 1 \implies \sum_i (a_i - b_i)d_{i} &\leq b_0 \numberthis
\end{align*}
So, it is sufficient to take $b_0= \sum_i (a_i - b_i)\zeta_i$ to make the metric bounded in $[0,1]$ and denominator positive. In addition, we can divide the numerator and denominator by $\sum_i a_i$ without changing the metric $\psi$. Therefore,  we take $\sum_i a_i=1$ during the elicitation task.
\eproof

\begin{proof}[Proof of Proposition~\ref{prop:d-uppersystem}]
We continue from Equation~\eqref{eq:d-lin-fr-equi}, where we saw that $\alpha\geq0$. Additionally, we ignore the case when $\alpha = 0$, since this would imply a constant $\psi^{*}$. 
Next, we may divide the above equations by $\alpha > 0$ on both sides so that all the coefficients $\ambf^\ast$ and $\ambf^\ast$ are factored by $\alpha$. This does not change the metric $\psi^{*}$; thus, the SoE becomes:
\begin{align}
a'_i - \btau^*b'_i =  \sbar_i  \;\; \forall \; i \in [k], \quad \btau^*b'_0 =  \inner{\smbfbar}{\dmbfbar^\ast}.
\label{append:eq:d-lin-fr-equi-2}
\end{align}
Notice that none of the conditions in Assumption \ref{as:d-sufficient} are changed except $\sum_i a_i = 1$. However, we may still use this condition to learn a constant $\alpha$ times the true metric, which does not harm the elicitation problem. From the last equation, we have that $\btau = \inner{\smbfbar}{\dmbfbar^\ast}/b'_0$. Putting this into rest of the equations gives us:
\bequation
\frac{a'_i - \sbar_i}{\inner{\smbfbar}{\dmbfbar^\ast}} = \frac{b'_i}{b'_0}.
\eequation
By replacing $b'_i$ in the rest of equations further gives us the solution mentioned in the proposition. 
\eproof

\begin{proof}[Proof of Proposition~\ref{prop:sufficient}]
Recall that our metric $\phi$ is monotonically decreasing in $c_i$'s. As LFPMs are transitional and scale invariant, w.l.o.g., we can assume that $\phi \in [-1,0]$. Taking the derivative in $c_1$ gives us:

\begin{align*}
    \frac{\partial\phi}{\partial c_{1}} &= \frac{a_1}{\sum_i b_i a_{i} + b_0} - \frac{b_1(\sum_i a_i c_{i})}{(\sum_i b_i c_{i} + b_0)^2} \leq 0. \numberthis
\end{align*}
Assuming denominator is positive, we have the numerator to be negative and
\vspace{-0.125cm}
\begin{align*}
    a_1 \leq b_1\frac{\sum_i a_i c_{i}}{\sum_i b_i c_{i} + b_0} \implies 
    \leq b_1\phi(\cmbf) \implies b_1. \ttau \numberthis
\end{align*}
\vskip -0.125cm
The above condition is necessary. Since $\ttau \in [-1, 0]$, by considering all the cases i.e. $b_i = 0, b_i > 0, b_i < 0$ the following are the sufficient condition for monotonicity decreasing LFPMs: $a_1\leq -b_1$ and $a_1\leq0$. Similarly, this is true for $a_i\leq -b_i, a_i \leq 0 \; \forall \; i \in [q]$ for monotonically decreasing LFPMs. Furthermore, as we assumed that  $\phi \in [-1, 0]$, i.e.,
\begin{align*}
    \frac{\sum_i a_i c_{i}}{\sum_i b_i c_{i}+b_0} \geq -1 \implies 
    \sum_i -(a_i + b_i)c_{i} \leq b_0 \numberthis
\end{align*}
Again, so it is sufficient to take $b_0= \sum_i -(a_i + b_i)\zeta_i$ to make the metric bounded in $[-1,0]$ and denominator positive. In addition, we can divide the numerator and denominator by $\sum_i |a_i|$ without changing the metric $\phi$. This gives us the condition $\sum_i a_i = -1$.
\eproof

\begin{proof}[Proof of Proposition~\ref{prop:lowersystemsphere}]
We start from~\eqref{eq:lin-fr-equi}, where we saw $\alpha \geq 0$. Additionally, we ignore the case when $\alpha = 0$, since this would imply a constant $\phi^{*}$. 
Next, we may divide the above equations by $\alpha > 0$ on both sides so that all the coefficients $a_i^*$'s and $b_i^*$'s are factored by $\alpha$. This does not change $\phi^{*}$; thus, the SoE becomes:
\begin{align}
a'_i - \btau^*b'_i =  \sbar_i,  \;\; \forall \;\; i \in [q], \qquad 
 \btau^*b'_0 =  \inner{\smbfbar}{\cmbfbar^{*}}.
\label{eq:lin-fr-equi-2}
\end{align}
Notice that none of the conditions in Assumption \ref{as:sufficient} are changed except $\sum_i a_i = -1$. However, we may still use this condition to learn a constant $\alpha$ times the true metric, which does not harm the elicitation problem. 
Similar to DLFPMs, if we somehow know the true $a'_i$'s, we can elicit the LFPM upto a constant multiple. 
From the last equation, we have that $\btau = \inner{\smbfbar}{\cmbfbar^{*}}/b'_0$. Putting this into rest of the equations gives us:
\bequation
\frac{a_i' - \sbar_i}{\inner{\smbfbar}{\cmbfbar^{*}}} = \frac{b_i'}{b_0}.
\eequation
By replacing $b_i$ in the rest of equations gives us the solution mentioned in the proposition.
\eproof

\section{Extended Experiments}
\label{append:sec:experiments}
\vskip -0.2cm

In this section, we empirically validate the theory and investigate the sensitivity and robustness due to finite sample estimates. For the ease of judgments, we show results corresponding to classes $k=3$ and $k=4$. The results and discussion extends to larger number of classes as well. To show the efficacy of the proposed methods, we run experiments on standard machine learning datasets.\footnote{The datasets can be downloaded from: https://www.csie.ntu.edu.tw/~cjlin/libsvmtools/datasets/multiclass.html, www.csie.ntu.edu.tw/~cjlin/libsvmtools/datasets/multiclass.html}

\subsection{DLPM and LPM Elicitation on Simulated Data (Extended)}
\label{append:ssec:lintheoryexp}

We show an extended set of results for the experimental setting discussed in Section~\ref{ssec:theoryexp}. Table~\ref{append:tab:DLPMtheory} and Table~\ref{append:tab:LPMtheory} show elicitation results on the simulated data for DLPMs and LPMs, respectively. We verify that our algorithms elicit the true metrics even for $\epsilon = 0.01$, and as expected, require $ 4(k-1)\ceil*{\log(1/\epsilon)}$ and $ 4T\ceil*{\log(\pi/2\epsilon)}$ queries for DLPM and LPM elicitation, respectively, where $\ceil*{\cdot}$ is the ceil function and $T=2(q-1)$.

\begin{table}[t]
	\caption{DLPM elicitation at $\epsilon = 0.01$ for synthetic data. The  number of queries used for $k=3$ and $k=4$ is 56 and 84, respectively. Since the digits are rounded to two decimal places, $\Vert \ambf^\ast \Vert_1$ or $\Vert \ambfhat \Vert_1$ might not be exactly equal to one.}  
	\label{append:tab:DLPMtheory}
		\vskip 0.15in
	\begin{center}
		\begin{small}
				\begin{tabular}{|c|c|c|c|}
					\hline
					\multicolumn{2}{|c|}{Classes $k=3$} & \multicolumn{2}{|c|}{Classes $k=4$} \\ \hline
					  $\psi^\ast = \ambf^\ast$ & $\hat \psi = \ambfhat $ & $\psi^\ast = \ambf^\ast$ & $\hat \psi = \ambfhat$ \\ \hline 
					  (0.21, 0.59, 0.20) & (0.21, 0.60, 0.20) & (0.13, 0.37, 0.12, 0.38) & (0.13, 0.37, 0.12, 0.38) \\
					  (0.44, 0.26, 0.31) & (0.44, 0.26, 0.31)  & (0.21, 0.26, 0.31, 0.22) & (0.21, 0.26, 0.31, 0.22) \\
					  (0.46, 0.33, 0.22) & (0.46, 0.33, 0.22) & (0.23, 0.17, 0.11, 0.48) & (0.23, 0.17, 0.11, 0.48)  \\
					  (0.23, 0.15, 0.62) & (0.23, 0.15, 0.62) & (0.25, 0.13, 0.45, 0.18) & (0.25, 0.12, 0.45, 0.18) \\
					  (0.31, 0.15, 0.54) & (0.3, 0.15, 0.54) & (0.22, 0.17, 0.31, 0.29) & (0.22, 0.17, 0.31, 0.29)  \\
					  (0.29, 0.40, 0.31) & (0.29, 0.40, 0.31) & (0.38, 0.21, 0.22, 0.20) & (0.38, 0.21, 0.21, 0.20) \\
					  (0.35, 0.32, 0.33) & (0.35, 0.33, 0.33) & (0.22, 0.13, 0.14, 0.52) & (0.22, 0.13, 0.14, 0.52) \\
					  (0.33, 0.35, 0.32) & (0.33, 0.35, 0.31) & (0.58, 0.17, 0.08, 0.18) & (0.58, 0.17, 0.08, 0.18) \\
					\hline
				\end{tabular}
		\end{small}
	\end{center}
	\vskip -0.3cm
\end{table}
\begin{table}[t]
	\caption{LPM elicitation at $\epsilon = 0.01$ for synthetic data. The number of queries used for $k=3$ and $k=4$ is 320 and 704, respectively. Since the digits are rounded to two decimal places, $\Vert \ambf^\ast \Vert_2$ or $\Vert \ambfhat \Vert_2$ might not be exactly equal to one.}
	\label{append:tab:LPMtheory}
		\vskip 0.15in
	\begin{center}
		\begin{small}
				\begin{tabular}{|c|c|c|}
					\hline
					  Classes & $\sphi = \ambf^\ast$ & $\hphi = \ambfhat$ \\ \hline 
					  3 & (-0.37, -0.89, -0.09, -0.23, -0.04, -0.03) & (-0.37, -0.89, -0.09, -0.23, -0.04, -0.03) \\ 
					  3 & (-0.80, -0.55, -0.18, -0.08, -0.14, -0.05) & (-0.80, -0.55, -0.18, -0.08, -0.14, -0.05) \\
					  3 & (-0.19, -0.88, -0.28, -0.10, -0.08, -0.30) & (-0.19, -0.88, -0.28, -0.10, -0.08, -0.30) \\
					  3 & (-0.44, -0.55, -0.33, -0.51, -0.23, -0.28) &  (-0.44, -0.55, -0.33, -0.51, -0.23, -0.28) \\
					  3 & (-0.79, -0.27, -0.25, -0.21, -0.38, -0.23) & (-0.79, -0.27, -0.25, -0.21, -0.38, -0.23) \\ \hline
					  \multirow{2}{*}{4}  & (-0.90, -0.28	-0.10, -0.31, -0.04,	-0.05, & (-0.90,	-0.28,	-0.10,	-0.31,	-0.04,	-0.05, \\
					  & -0.03,	-0.04,	-0.02,	-0.01,	-0.01,	-0.01) & -0.03,	-0.04,	-0.02,	-0.01,	-0.01,	-0.01) \\ 
					  \multirow{2}{*}{4} & (-0.54,	-0.10,	-0.62,	-0.52,	-0.03,	-0.07, & (-0.55,	-0.11, -0.62,	-0.51,	-0.03,	-0.07,  \\ 
					  & -0.11,	-0.07,	-0.14,	-0.03,	-0.03,	-0.04) & -0.11,	-0.07,	-0.14, -0.03,	-0.03,	-0.04) \\
					  \multirow{2}{*}{4}  & (-0.56, -0.07, -0.79, -0.05, -0.16, -0.16, & (-0.56, -0.07, -0.79, -0.05, -0.16, -0.17,  \\
					  & -0.04, -0.02, -0.03, -0.00, -0.01, -0.01) & -0.04, -0.02, -0.03, -0.00, -0.01, -0.01) \\
					  \multirow{2}{*}{4}  & (-0.60, -0.79, -0.09, -0.01, -0.01, -0.02, & (-0.60, -0.79, -0.09, -0.01, -0.01, -0.02, \\
					  & -0.02, -0.01, -0.01, -0.01, -0.00, -0.00) & -0.02, -0.01, -0.01, -0.01,	-0.00, -0.00) \\
					  \multirow{2}{*}{4}  & (-0.45, -0.38, -0.42, -0.19, -0.21, -0.63, & (-0.46, -0.38, -0.41, -0.19, -0.20, -0.62, \\
					  & -0.09, -0.00, -0.00, -0.00,	-0.01, -0.01) & -0.09,	-0.00,	-0.00, -0.00, -0.01, -0.01) \\
					\hline
				\end{tabular}
		\end{small}
	\end{center}
	\vskip -0.3cm
\end{table}
\begin{table}[t]
	\caption{LPM elicitation on sphere with varying radius and $\epsilon = 0.01$. For randomly chosen hundred $\ambf^\ast$, we show the fraction of times our estimates $\ambfhat$ obtained with $4\times2(q-1)\ceil*{\log (1/\epsilon)}$ queries satisfy $\Vert \ambf^\ast - \ambfhat \Vert_\infty \leq \omega$. Notice that we incur error only when the radius is of the order of practical computation error, which can be attributed to $\epsilon_\Omega$ in the simulated setting.}
	\label{append:tab:LPMsphere}
		\vskip 0.15in
	\begin{center}
		\begin{small}
				\begin{tabular}{|c||c|c|c|c|c|}
					\hline
					\diagbox[width=2cm]{\raisebox{0ex}{\quad $\lambda$ }}{ \raisebox{-1.5\height}{$\omega$}} & 0.02 & 0.04 & 0.06 & 0.08 & 0.10 \\ \hline 
					 $1.250 \times 10^{-12}$  & 0.03 & 0.38 &	0.74 &	0.92 &	0.94   \\
  $1.875 \times 10^{-12}$ 	& 0.09 &  0.49 & 0.77 &	0.94  &	0.98 \\
 $2.500 \times 10^{-12}$  &	0.12 & 0.73 & 0.93 & 0.97 &	0.99   \\ \hline
				\end{tabular}
		\end{small}
	\end{center}
\end{table}
\begin{table}[t]
	\caption{DLFPM Elicitation for synthetic distribution for $k=3$ classes with $\epsilon = 0.01$. $(\ambf^\ast, \bmbf^\ast, b^\ast_0)$ denote the true DLFPM. $(\ambfhat, \bmbfhat, \hat b_0)$ denote the elicited LFPM. We empirically verify that the elicited metric is constant multiple ($\alpha$) of the true metric.}
	\label{append:tab:DLFPMtheory3}

	\begin{center}
		\begin{footnotesize}
			\begin{tabular}{|c|c|c|c|}
				\hline
                
				True Metric & \multicolumn{3}{|c|}{Results on Synthetic Distribution (Appendix \ref{append:ssec:linfractheoryexp})} \\ \hline
				\makecell{$(a^\ast_1, a^\ast_2, a^\ast_3),$ \\ $(b^\ast_1, b^\ast_2, b^\ast_3), b^\ast_0$} & \makecell{$(\hat a_1,\hat a_2, \hat a_3),$ \\ $(\hat b_1, \hat b_2, \hat b_3), \hat b_0$} & $\alpha$ & $\sigma$ \\
				\hline
				\makecell{(0.21, 0.59, 0.20), \\
				(0.11, -0.22, -0.27), 0.41} &	\makecell{(0.25, 0.58, 0.18), \\ (0.20, -0.03, -0.17), 0.29} & 1.23 & 0.03  \\
				\makecell{(0.45, 0.27, 0.29),	\\
				(0.39, 0.22, -0.76), 0.43} & \makecell{(0.46, 0.34, 0.20), \\
				(0.42, 0.30, -0.73), 0.38} & 1.03 &	0.04 \\
				\makecell{(0.08, 0.42, 0.50), \\ 
				(0.07, -0.63, 0.20), 0.37} &	\makecell{(0.16, 0.38, 0.47), \\ (0.17, -0.41, 0.23), 0.27} & 1.22 &	0.05 \\
				\hline
			\end{tabular}
		\end{footnotesize}
	\end{center}
\end{table}

\subsection{Effect of Sphere Size on LPM Elicitation}
\label{append:ssec:sphereexp}

For real-world datasets, Algorithm~\ref{mult-alg:lpm} is agnostic to the error from $\hat \eta_i$'s as long as we get a sphere inside the feasible region of sufficient size. With the following experiment, we show that we incur errors in elicitation when the radius $\lambda$ is of the order of $\epsilon_\Omega$. Recall that, when we are working in a simulated setting, a good proxy for $\epsilon_\Omega$ is the practical computation error.

Here, we work with $k=4$ classes. We took $\lambda = 2.500 \times 10^{-12}$ and performed elicitation by considering three spheres of size $\sfrac{1}{2} \lambda$, $\sfrac{3}{4} \lambda$, and $\lambda$. We randomly selected hundered DLPMs i.e. $\ambf^\ast$'s.  
We then used Algorithm~\ref{mult-alg:lpm} with $\epsilon=0.01$ to recover the estimates $\ambfhat$'s. 
In Table~\ref{append:tab:LPMsphere}, we report the proportion of the number of times $\Vert \ambf^\ast - \ambfhat \Vert_\infty \leq \omega$ for different values of $\omega$. We see improved elicitation when we work with $\lambda$ and incur more errors when the sphere's radius is less than that. In particular, if we take the radius of the order (a little) higher than $10^{-12}$ then we perform perfect elicitation. Needless to say, when working with real oracle (users), the magnitude of the oracle's feedback noise $\epsilon_\Omega$ and the size of the sphere will play a role in elicitation performance as suggested in Theorem~\ref{thm:lpm-elit-error}.

\subsection{DLFPM and LFPM Elicitation}
\label{append:ssec:linfractheoryexp}

Now, we validate elicitation for DLFPMs for classes $k=3$ and  $k=4$ using the routine discussed in Section~\ref{append:sec:dlfpme}. We use the same distribution setting of Section~\ref{ssec:theoryexp} for both the classes. We define a true metric $\psi^{\ast}$ by $\{\ambf^\ast, \bmbf^\ast, b_0^\ast\}$.  Then, we run Algorithm~\ref{mult-alg:dlpm} with $\epsilon=0.01$ to find the hyperplane $\bell$ and maximizer on $\partial \Dcal^+$, Algorithm~4.3 with $\epsilon=0.01$ to find the hyperplane $\tell$ and minimizer on $\partial \Dcal^-$, and Algorithm~\ref{append:alg:d-grid-search} with $n' = 1000$ (1000 diagonal confusions on $\partial\Dcal^+$ obtained by varying parameter $m$) and $\delta = 0.01$. This gives us the elicited metric $\hat\psi$, which we represent by $\{\ambfhat, \bmbfhat, \hat b_0\}$. 
In Table~\ref{append:tab:DLFPMtheory3} and Table~\ref{append:tab:DLFPMtheory4}, we present the elicitation results for DLFPMs for classes $k=3$ and $k=4$, respectively. We also present the mean ($\alpha$) and the standard deviation ($\sigma$) of the ratio of the elicited metric $\hat \psi$ to the true metric $\psi^{\ast}$ over the set of diagonal confusions used in Algorithm~\ref{append:alg:d-grid-search} (column 3 and 4 of Table \ref{append:tab:DLFPMtheory3} and Table \ref{append:tab:DLFPMtheory4}). For a better judgment, we show function evaluations of the true metric and the elicited metric in Figure \ref{append:fig:dlfpm-theory}. 
The true and the elicited metric are plotted together after vectorizing the set of diagonal confusions in a certain order based on their parametrizations. As expected, we see that the elicited metric is a constant multiple of the true metric. 

Now, we validate elicitation for LFPMs for classes $k=3$ and  $k=4$ using the routine discussed in Section~\ref{append:sec:lfpme}. We define a true metric $\phi^{\ast}$ by $\{\ambf^\ast, \bmbf^\ast, b_0^\ast\}$.  Then, we run Algorithm~\ref{mult-alg:lpm} with $\epsilon=0.01$ to find the hyperplane $\bell$ and maximizer on $\partial \Scal^-_{\lambda}$, Algorithm~4.5 with $\epsilon=0.01$ to find the hyperplane $\tell$ and minimizer on $\partial \Scal^+_{\lambda}$, and Algorithm~4.6 with $n' = 1000$ (1000 off-diagonal confusions on $\partial \Scal^-_{\lambda}$ obtained by varying parameter $\thetambf$) and $\delta = 0.01$. This gives us the elicited metric $\hat \phi$, which we represent by $\{\ambfhat, \bmbfhat, \hat b_0\}$. 
In Table~\ref{append:tab:LFPMtheory3}, we present the elicitation results for LFPMs for classes $k=3$. We also present the mean ($\alpha$) and the standard deviation ($\sigma$) of the ratio of the elicited metric $\hat \phi$ to the true metric $\phi^{\ast}$ over the set of off-diagonal confusions used in Algorithm~4.6 (column 3 and 4 of Table~\ref{append:tab:LFPMtheory3}). 

For a better judgment, we show function evaluations of the true metric and the elicited metric evaluated on selected off-diagonal confusions in the top row of Figure \ref{append:fig:lfpm-theory}. Due to many terms in the LFPM for $k=4$, we skip providing true metric and the elicited metric and only mention the $\alpha$ and $\sigma$ of the true and elicited metric similar to Table~\ref{append:tab:LFPMtheory3}. We obtained $\alpha = 0.79, 0.72, 0.72$ and $\sigma = 0.007, 0.007, 0.006$ for the three metrics plotted in the bottom row of Figure~\ref{append:fig:lfpm-theory}.     
The true and the elicited metric are plotted together after vectorizing the set of confusions in a certain order based on their parametrizations. As expected, the elicited metric is a constant multiple of the true metric for both $k=3$ and $k=4$.

\begin{table}[t]
	\caption{DLFPM Elicitation for synthetic distribution for $k=4$ classes  with $\epsilon = 0.01$. $(\ambf^\ast, \bmbf^\ast, b^\ast_0)$ denote the true DLFPM. $(\ambfhat, \bmbfhat, \hat b_0)$ denote the elicited LFPM. We empirically verify that the elicited metric is constant multiple ($\alpha$) of the true metric.}
	\vskip -0.2cm
	\label{append:tab:DLFPMtheory4}
	\begin{center}
		\begin{footnotesize}
			\begin{tabular}{|c|c|c|c|}
				\hline
                
				True Metric & \multicolumn{3}{|c|}{Results on Synthetic Distribution (Appendix \ref{append:ssec:linfractheoryexp})} \\ \hline
				\makecell{$(a^\ast_1, a^\ast_2, a^\ast_3, a^\ast_4),$ \\ $(b^\ast_1, b^\ast_2, b^\ast_3, b^\ast_4), b^\ast_0$} & \makecell{$(\hat a_1,\hat a_2, \hat a_3, \hat a_4),$ \\ $(\hat b_1, \hat b_2, \hat b_3,\hat b_4), \hat b_0$} & $\alpha$ & $\sigma$ \\
				\hline
				\makecell{(0.32, 0.35, 0.06, 0.27), \\ (-1, -0.3, -0.32, 0.25), 0.6} & \makecell{(0.2, 0.29, 0.19, 0.32), \\ (-0.4, -0.01, 0.08, 0.33), 0.26} & \makecell{1.58 \\ } & \makecell{0.12 \\ }  \\
				\makecell{(0.31, 0.22, 0.27, 0.2), \\ (-0.17, -0.01, 0.18, 0.09), 0.25} & \makecell{(0.2, 0.3, 0.26, 0.24), \\ (-0.38, 0.07, 0.16, 0.14), 0.28} & 0.95 & 0.04 \\
				\makecell{(0.22, 0.16, 0.41, 0.21), \\ (-0.22, -0.43, -0.18, 0.14), 0.33} & \makecell{(0.19, 0.2, 0.35, 0.26), \\ (-0.09, -0.12, -0.03, 0.24), 0.19} & 1.38 & 0.06 \\
				\hline
			\end{tabular}
		\end{footnotesize}
	\end{center}
\end{table}
\begin{table}[h]
	\caption{LFPM Elicitation  for $k=3$ classes with $\epsilon = 0.01$. $(\ambf^\ast, \bmbf^\ast, b^\ast_0)$ denote the true LFPM. There are thirteen terms to elicit in LFPM. $(\ambfhat, \bmbfhat, \hat b_0)$ denote the elicited LFPM. We empirically verify that the elicited metric is constant multiple ($\alpha$) of the true metric.}
	\label{append:tab:LFPMtheory3}
	\begin{center}
		\begin{footnotesize}
			\begin{tabular}{|c|c|c|c|}
				\hline
                
				True Metric & \multicolumn{3}{|c|}{Results on Synthetic Distribution (Appendix \ref{append:ssec:linfractheoryexp})} \\ \hline
				\makecell{$(a^\ast_1, a^\ast_2, a^\ast_3, a^\ast_4, a^\ast_5, a^\ast_6),$ \\ $(b^\ast_1, b^\ast_2, b^\ast_3, b^\ast_4, b^\ast_5, b^\ast_6), b^\ast_0$} & \makecell{$(\hat a_1,\hat a_2, \hat a_3, \hat a_4, \hat a_5, \hat a_6),$ \\ $(\hat b_1, \hat b_2, \hat b_3, \hat b_4, \hat b_5, \hat b_6), \hat b_0$} & $\alpha$ & $\sigma$ \\
				\hline
				\makecell{(-0.16, -0.05, -0.29, -0.21, -0.17, -0.12), \\ (-0.76, 0.02, -0.88, 0.09, -0.23, -0.38), 2.36} & \makecell{(-0.11, -0.08, -0.15, -0.17, -0.24, -0.25), \\ (-0.66, 0.07, -0.86, 0.04, -0.04, -0.09), 1.89} & 1.11 & 0.01  \\
				\makecell{(-0.17, -0.19, -0.09, -0.18, -0.16, -0.2), \\ (-0.3, -0.74, -0.54, -0.37, -0.89, -0.14), 2.99} & \makecell{(-0.05, -0.08, -0.11, -0.16, -0.31, -0.31), \\ (-0.46, -0.82, -0.43, -0.34, -0.48, 0.09), 2.58} & 1.08 & 0.01 \\
				\makecell{(-0.3, -0.08, -0.1, -0.12, -0.21, -0.18), \\ (-0.24, -0.52, -0.45, 0, -0.41, -0.94), 2.67} & \makecell{(-0.06, -0.08, -0.11, -0.15, -0.27, -0.33), \\ (-0.59, -0.45, -0.37, 0.07, -0.24, -0.57), 2.36} & 1.07 & 0.01 \\
				\hline
			\end{tabular}
		\end{footnotesize}
	\end{center}
\end{table}

\begin{figure}[t]
	\centering 
	\subfigure[Table \ref{append:tab:DLFPMtheory3}, Line 1]{
		{\includegraphics[width=5cm]{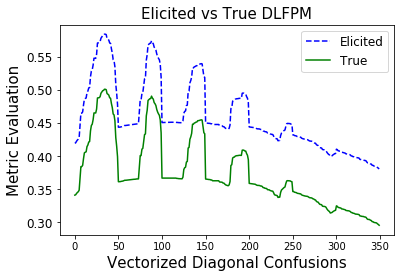}}
		\label{fig:app:lfpm_1}
	}
	\subfigure[Table \ref{append:tab:DLFPMtheory3}, Line 2]{
		{\includegraphics[width=5cm]{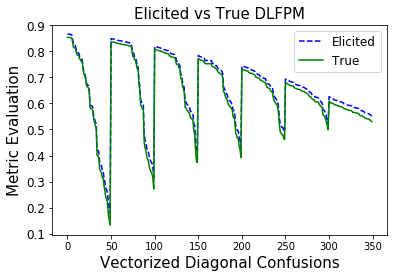}}
		\label{fig:app:lfpm_2}
	}
	\subfigure[Table \ref{append:tab:DLFPMtheory3}, Line 3]{
		{\includegraphics[width=5cm]{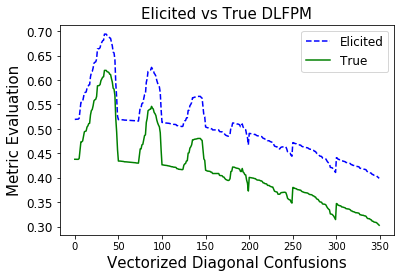}}
		\label{fig:app:lfpm_3}
	}
	\subfigure[Table \ref{append:tab:DLFPMtheory4}, Line 1]{
		{\includegraphics[width=5cm]{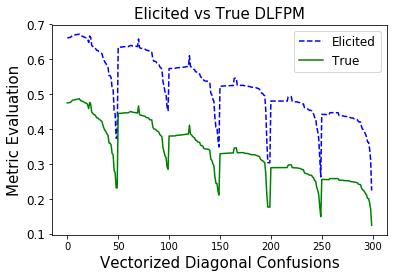}}
		\label{fig:app:lfpm_4}
	}
	\subfigure[Table \ref{append:tab:DLFPMtheory4}, Line 2]{
		{\includegraphics[width=5cm]{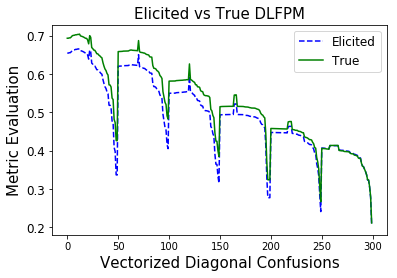}}
		\label{fig:app:lfpm_5}
	}
	\subfigure[Table \ref{append:tab:DLFPMtheory4}, Line 3]{
		{\includegraphics[width=5cm]{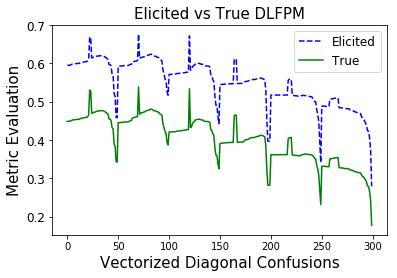}}
		\label{fig:app:lfpm_6}
	}
	\caption{True and elicited DLFPMs for synthetic distribution from Tables~\ref{append:tab:DLFPMtheory3} and \ref{append:tab:DLFPMtheory4}. The solid green curve and the dashed blue curve are the true and the elicited metric, respectively. 
	We see that the elicited DLFPMs are constant multiple of the true metrics.}
	\label{append:fig:dlfpm-theory}
\end{figure}

\begin{figure}[t]
	\centering 
	\subfigure[Table \ref{append:tab:LFPMtheory3}, Line 1]{
		{\includegraphics[width=5cm]{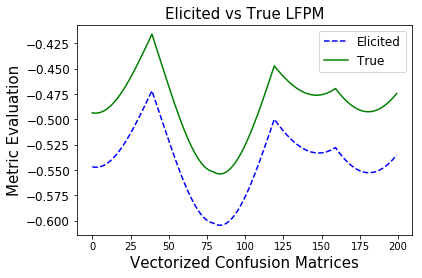}}
    \label{fig:app:lfpm_1_magic}
    }
	\subfigure[Table \ref{append:tab:LFPMtheory3}, Line 2]{
		{\includegraphics[width=5cm]{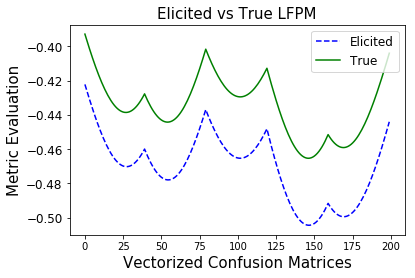}}
		\label{fig:app:lfpm_2_magic}
	}
	\subfigure[Table \ref{append:tab:LFPMtheory3}, Line 3]{
		{\includegraphics[width=5cm]{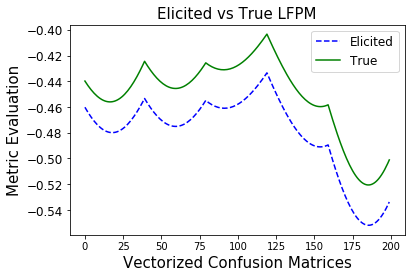}}
		\label{fig:app:lfpm_3_magic}
	}
	\subfigure[LFP Metric 1, $k=4$]{
		{\includegraphics[width=5cm]{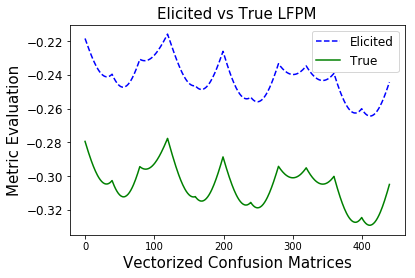}}
		\label{fig:app:lfpm_4_magic}
	}
	\subfigure[LFP Metric 2, $k=4$]{
		{\includegraphics[width=5cm]{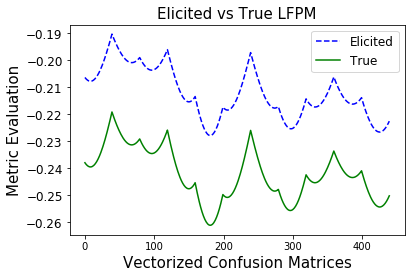}}
		\label{fig:app:lfpm_5_magic}
	}
	\subfigure[LFP Metric 3, $k=4$]{
		{\includegraphics[width=5cm]{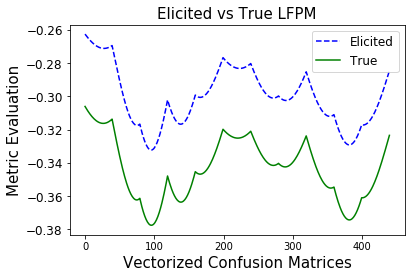}}
		\label{fig:app:lfpm_6_magic}
	}
	\caption{True and elicited LFPMs. The plots in the top row correspond to the metrics in Table~\ref{append:tab:LFPMtheory3} for $k=3$. The bottom row corresponds to metrics for $k=4$. The solid green curve and the dashed blue curve are the true and the elicited metric, respectively. 
	We see that the elicited LFPMs are constant multiple of the true metrics.}
	\label{append:fig:lfpm-theory}
\end{figure}

%% file: fair/supplement.tex
\section{Proofs and Details of Section~\ref{sec:confusion}}
\label{append:sec:confusion}

\bproof[Proof of Proposition~\ref{fair-prop:C}]
The set of rates $\Rcal^g$ for a group $g$ satisfies the following properties:
\bitemize[leftmargin=1em]
\item \emph{Convex}: Let us take two classifiers $h_1^g, h_2^g \in \Hcal^g$ which achieve the rates $\rmbf_1^g, \rmbf_2^g \in \Rcal^g$. We need to check whether or not  the convex combination $\alpha \rmbf_1^g + (1-\alpha)\rmbf_2^g$ is feasible, i.e., there exists some classifier which achieve this rate. Consider a classifier $h^g$, which with probability $\alpha$ predicts what classifier $h_1^g$  predicts and with probability $1-\alpha$ predicts what classifier $h_2^g$ predicts. Then the elements of the rate matrix $R_{ij}^g(h)$ is given by: 

\begin{align*}
R_{ij}^g(h) &= \Pmbb(h^g=j | Y=i)  \\ \nonumber
&= \Pmbb(h^g_1=j|h^g=h^g_1, Y=i)\Pmbb(h^g=h_1^g) + \Pmbb( h_2^g=j|h^g=h_2^g, Y=i)\Pmbb(h^g=h^g_2)   \\ \nonumber
&= \alpha \rmbf_{1}^g + (1-\alpha)\rmbf_{2}^g. \numberthis
\end{align*}
Therefore, $\Rcal^g \; \forall \; g \in [m]$ is convex.
\item \emph{Bounded:} Since $R^g_{ij}(h) = P[h=j|Y=i] = P[h=j, Y=i]/P[Y=i] \leq 1$ for all $i,j \in [k]$, $\Rcal^g \subseteq [0, 1]^q$.
\item \emph{$\embf_i$'s and $\ombf$ are always achieved:} The classifier which always predicts class $i$, will achieve the rate $\embf_i$. Thus, $\embf_i \in \Rcal^g \, \forall\, i \in [k], g \in [m]$ are feasible. Just like the convexity proof, a classifier which predicts similar to one of the trivial classifiers with probability $1/k$ will achieve the rates $\ombf$. 
\item \emph{$\embf_i$'s are vertices:} Any supporting hyperplane with slope $\ell_{1i} < \ell_{1j} < 0$ and $\ell_{1p}=0$ for $p \in [k], p \neq i, j$ will be supported by $\embf_1$ (corresponding to the trivial classifier which predict class 1). Thus, $\embf_i$'s  are vertices of the convex set. As long as the class-conditional distributions are not identical, i.e., there is some signal for non-trivial classification conditioned on each group (Assumption~\ref{as:clsconditional}), one can construct a ball around the trivial rate $\ombf$ and thus $\ombf$ lies in the interior.
\eitemize
\vspace{-0.5cm}
\eproof

\vspace{-0.5cm}
\subsection{Finding the Sphere $\Scal_\rho$}
\label{fair-append:ssec:sphere}
\balgorithm[t]
\caption{Obtaining the sphere $\Scal_\rho$ with radius $\rho$}
\label{alg:sphere}
\small
\balgorithmic[1]
\STATE \textbf{Input:} The center $\ombf$ of the feasible region of rates across groups.
\FOR{$j=1, 2, \cdots, q$}
\STATE Let $\mathbf r_j$ be the standard basis vector for the $j$-th dimension. 
\STATE Compute the maximum $\ell_j$ such that $\ombf + \ell_j \mathbf r_j$ is feasible for all groups by solving~\eqref{fair-eq:op1}.
\ENDFOR
\STATE Let $CONV$ be the convex hull of $\{\ombf\pm \ell_j\mathbf r_j\}_{j=1}^{q}$.
\STATE Compute the radius $s$ of the largest ball which can fit inside of $CONV$, centered at $\ombf$.
\STATE\textbf{Output:} Sphere $\Scal_\rho$ with radius $\rho=s$ centered at $\ombf$.
\ealgorithmic
\ealgorithm

In this section, we discuss how a sufficiently large sphere $\Scal_\rho$ with radius $\rho$ may be found. The following discussion is extended from Chapter~\ref{chp:multiclass} (Section~\ref{append:ssec:slambda}) to multiple groups setting and provided here for completeness. 

The following optimization problem is a special case of OP2 in~\cite{narasimhan2018learning,Tavker+2020}. The problem corresponds to feasiblity check problem for a given rate $\rmbf_0$ achieved by all groups within small error $\epsilon >0$. 

\begin{align}
    \min_{\rmbf^g \in \Rcal^g \, \forall g \in [m]} \; 0 \qquad s.t. \;\; \Vert \rmbf^g - \rmbf_0 \Vert_2 \leq \epsilon \quad \forall \; g \in [m].
    \label{fair-eq:op1}
\end{align}

The above problem checks the feasibility and if a solution to the above problem exists, then Algorithm~1 of~\cite{narasimhan2018learning} returns it. The approach in~\cite{narasimhan2018learning} constructs a classifier whose group-wise rates are $\epsilon$-close to the given rate $\rmbf_0$. 

Furthermore, Algorithm~\ref{fair-append:ssec:sphere} computes a value of $\rho\geq \tilde{s}/k$, where $\tilde{s}$ is the radius of the largest ball contained in the set $\Rcal^1 \cap \cdots \cap \Rcal^m$. Notice that the approach in~\cite{narasimhan2018learning} is consistent, thus we should get a good estimate of the sphere, provided we have sufficient samples. The algorithm runs offline and does not impact query complexity.

\blemma
    Let $\tilde{s}$ be the radius of the largest ball centered at $\ombf$ in $\Rcal^1 \cap \cdots \cap \Rcal^m$. Then Algorithm~\ref{fair-append:ssec:sphere} returns a radius $\rho\geq \tilde{s}/k$.
\elemma
\begin{proof}
    Let $\ell_j$ be as computed in the algorithm and $\ell:= \min_j \ell_j$, then we have $\ell\geq \tilde{s}$. Moreover, the region $CONV$ contains the convex hull of $\{o\pm \ell\mathbf r_j\}_{j=1}^{q}$; however, this region contains a ball of radius $\ell/\sqrt{q} = \ell/\sqrt{k^2-k}\geq \ell/k\geq \tilde{s}/k$, and thus $\rho\geq \tilde{s}/k$.
\end{proof}

\section{Derivations of Section~\ref{sec:me}}
\label{append:sec:me}

Notice that $\sum_{g=1}^m \taumbf^g = \mathbf{1}$, i.e., the vector of ones. 

\vspace{-0.1cm}
\subsection{Eliciting the Misclassification Cost $\bphi(\rmbf)$; Part 1 in Figure~\ref{fig:workflow} and line 2 in Algorithm~\ref{alg:f-me}}
\label{append:ssec:phi}

The key to eliciting $\bphi$ is to remove the effect of fairness violation $\bvarphi$ in the oracle responses. As explained in Section~\ref{ssec:elicitphi}, we run the LPME procedure (Algorithm~\ref{mult-alg:lpm}) with the $q$-dimensional query space $\Scal_\rho$, binary search tolerance $\epsilon$, the equivalent oracle $\Omega^{\text{class}}$. From Remark~\ref{fair-rm:ratio}, this subroutine returns a slope $\fmbf$ with $\Vert \fmbf \Vert_2=1$ such that:
\vspace{-0.2cm}
\begin{equation}
    \frac{(1-\lambdabar)a_i}{(1-\lambdabar)a_j} = \frac{f_i}{f_j} \implies \frac{a_i}{a_j} = \frac{f_i}{f_j}.
    \label{append:eq:phisolutionf}
\end{equation}
\vskip -0.2cm
Thus, we set $\ambfhat \coloneqq \fmbf$ (line~2, Algorithm~\ref{alg:f-me}). 

\subsection{Eliciting the Fairness Violation $\bvarphi(\tupr)$; Part 2 in Figure~\ref{fig:workflow} and lines 3-15 in Algorithm~1}
\label{append:ssec:varphi}

\paragraph{Eliciting the Fairness Violation $\bvarphi(\tupr)$ for $m=2$; lines 3-6 in Algorithm~1:}
\label{append:sssec:elicitvarphim2}

For $m=2$, we have only one vector of unfairness weights $\bmbf^{12}$, which we now aim to elicit given $\ambfhat$. As discussed in Section~\ref{ssec:elicitvarphi}, we fix trivial rates (through trivial classifiers) to one group and allow non-trivial rates from $\Scal_\rho$ on another group. This essentially makes the metric in Definition~\ref{def:linear} linear. The elicitation procedure is as follows. 

Fix trivial classifier predicting class $1$ for group 2, i.e., fix $h^{2}(x) = 1 \, \forall \, x \in \Xcal$, and thus $\rmbf^{2}  = \embf_1$. For group 1, we constrain the confusion rates to lie in the sphere $\Scal_\rho$, i.e., $\rmbf^{1}  = \smbf$ for $\smbf \in \Scal_\rho$. 
Then the metric in Definition~\ref{def:linear} amounts to:
    \begin{align*}
    \bPsi((\smbf, \embf_1); \ambfbar, \bmbfbar^{12}, \lambdabar) &=
    (1-\lambdabar)\inner{\ambfbar \odot (1-\bm{\tau}^{2})}{\smbf} + \lambdabar  \inner{\bmbfbar^{12}}{\vert \embf_1 - \smbf \vert} + c_1.
    \numberthis \label{append:eq:metricbm2}
    \end{align*}
The above is a function of $\smbf \in \Scal_\rho$. Since $\embf_i$'s are binary vectors and since $0 \leq \smbf \leq 1$, the sign of the absolute function with respect to $\smbf$ can be recovered. Recall that the rates are defined in row major form of the rate matrices, thus $\embf_1$ is $1$ at every $(k + j*(k-1))$-th coordinate, where $j \in \{0,\dots,k-2\}$, and 0 otherwise. The coordinates where the confusion rates are $1$ in $\embf_1$, the absolute function  opens with a negative sign (wrt. $\smbf$) and with a positive sign otherwise. In particular, define a $q$-dimensional vector $\wmbf_1$ with entries $-1$ at every $(k + j*(k-1))$-th coordinate, where $j \in \{0,\dots,k-2\}$, and $1$ otherwise. One may then write the metric $\bPsi$ as:
\begin{align*}
&\bPsi((\smbf, \embf_1)\,;\, \ambfbar, \bmbfbar^{12}, \lambdabar) =  
\inner{(1-\lambdabar)\ambfbar \odot (\bm{1}-\bm{\tau}^{2}) + \lambdabar \wmbf_1 \odot 
\bmbfbar^{12}}{\smbf} + c_1.
\numberthis \label{append:eq:metricbrich1m2}
\end{align*} 
This is again a linear metric elicitation problem where $\smbf \in \Scal$. We may again use the LPME procedure (Algorithm~\ref{mult-alg:lpm}), which outputs a (normalized) slope $\breve \fmbf$ with $\Vert \breve \fmbf \Vert_2 = 1$ in line 4 of Algorithm~\ref{alg:f-me}. Using Remark~\ref{fair-rm:ratio}, we get $q-1$ independent equations and may represent every element of $\bmbfbar^{12}$ based on one element, say $\bbar^{12}_{k-1}$, i.e.:

\begin{align*}
    \frac{\breve f_{k-1}}{\breve f_{i}} &= \frac{(1-\lambdabar){(1-\tau^2_{k-1})\abar_{k-1} + \lambdabar \bbar^{12}_{k-1}}}{(1-\lambdabar){(1-\tau^2_{i})\abar_{i} + \lambdabar w_{1i}\bbar^{12}_{i}}} \qquad \forall \; i \in [q].\\
    \implies 
    \lambdabar \bmbfbar^{12} &= \wmbf_1 \odot \left[ \left( \frac{(1-\lambdabar)(1-\tau^{2}_{k-1})\abar_{k-1} + \lambdabar \bbar^{12}_{k-1}}{\breve f_{k-1}}\right) \breve \fmbf - (1-\lambdabar)((1-\taumbf^{2})\odot\ambfbar)  \right].
    \numberthis \label{append:eq:metricbfirstm2}
\end{align*}

In order to elicit entire $\bmbfbar^{12}$, we need one more linear relation such as~\eqref{append:eq:metricbfirstm2}. So, we now fix the trivial classifier predicting class $k$ for group 2, i.e., fix $h^{2}(x) = k \, \forall \, \xmbf \in \Xcal$, and thus $\rmbf^{2}  = \embf_k$. For group 1, we constrain the rates to again lie in the sphere $\Scal_\rho$ i.e. $\rmbf^{1}  = \smbf$ for $\smbf \in \Scal_\rho$. 
Since the rate vectors are in row major form of the rate matrices, notice that $\embf_k$ is $1$ at every $(k-1 + j*(k-1))$-th coordinate, where $j \in \{0,\dots,k-2\}$, and 0 otherwise. 
In particular, define a $q$-dimensional vector $\wmbf_k$ with entries $-1$ at every $(k-1 + j*(k-1))$-th coordinate, where $j \in \{0,\dots,k-2\}$, and $1$ otherwise. One may then write the metric $\bPsi$ as:
\begin{align*}
&\bPsi((\smbf, \embf_k); \ambfbar, \bmbfbar^{12}, \lambdabar) =
    (1-\lambdabar)\inner{\ambfbar \odot (1-\bm{\tau}^{2})}{\smbf} + \lambdabar  \inner{\bmbfbar^{12}}{\vert \embf_k - \smbf \vert} + c_k.
\numberthis \label{append:eq:metricbrich2m2}
\end{align*}
This is a linear metric elicitation problem where $\smbf \in \Scal$. Thus, line 5 of Algorithm~\ref{alg:f-me} applies LPME subroutine (Algorithm~\ref{mult-alg:lpm}), which outputs a (normalized) slope $\tilde \fmbf$ with $\Vert \tilde \fmbf \Vert_2 = 1$. Using Remark~\ref{fair-rm:ratio}, we extract the following relation between two of its coordinates, say the $(k-1)$-th and $((k-1)^2+1)$-th coordinates:
\begin{align*}
\frac{\tilde f_{k-1}}{\tilde f_{(k-1)^2+1}} = \frac{(1-\lambdabar)(1 - \tau^{2}_{k-1})\abar_{k-1} - \lambdabar \bbar^{12}_{k-1}}{(1-\lambdabar)(1- \tau^{2}_{(k-1)^2+1}) \abar_{(k-1)^2+1} + \lambdabar \bbar^{12}_{(k-1)^2+1}}.
\numberthis \label{append:eq:metricbsecondm2}
\end{align*}
Combining equations~\eqref{append:eq:metricbfirstm2} and~\eqref{append:eq:metricbsecondm2} and replacing the true $\ambfbar$ with the estimated $\ambfhat$ from Section~\ref{ssec:elicitphi}, we have an estimate of the scaled substitute as:
\begin{align*}
    \tilde \bmbf^{12} &= \wmbf_1 \odot \left[ \delta \breve \fmbf^{12} - \ambfhat \odot (\bm{1} - \taumbf^{2})  \right], \numberthis \label{append:eq:ellsolm2}
    \\
    \text{where} \; \delta &= \frac{2(1-\tau^{2}_{k-1})\hat a_{k-1}}{\breve f_{k - 1}} \left[ \frac{ \frac{(1-\tau^{2}_{(k-1)^2+1})\hat a_{(k-1)^2+1}}{(1-\tau^{2}_{k-1})\hat a_{k-1}} -  \frac{\tilde f_{(k-1)^2+1}}{\tilde f_{k-1}} }{\left( \frac{\breve f_{(k-1)^2+1}}{\breve f_{k-1}} - \frac{\tilde f_{(k-1)^2+1}}{\tilde f_{k-1}} \right)} \right],
\end{align*}
and $\tilde \bmbf$ is a scaled substitute defined as $\tilde \bmbf^{12}\coloneqq \frac{\lambdabar}{(1-\lambdabar)} \bmbfbar^{12}$, which nonetheless is computable from~\eqref{append:eq:ellsolm2}. Since we require a solution $\bmbfhat$ such that $\Vert \bmbfhat \Vert_2 = 1$ (Definition~\ref{def:linear}), we normalize $\tilde \bmbf$ and get the final solution:

\begin{equation}
\bmbfhat^{12} = \frac{\tilde \bmbf^{12}}{\Vert \tilde \bmbf^{12} \Vert_2}.
 \label{append:eq:bsolm2}
\end{equation} 
Notice that, due to normalization, the solution is  independent of the true trade-off $\lambdabar$.

\paragraph{Eliciting the Fairness Violation $\bvarphi(\tupr)$ for $m>2$; line 8-14 in Algorithm~\ref{alg:f-me}:}
\label{append:sssec:elicitvarphi}

Consider a non-empty set of sets $\Mcal \subset 2^{[m]} \setminus \{\varnothing, [m]\}$. We will later discuss how to choose $\Mcal$ for efficient elicitation. When $m>2$, we partition the set of groups $[m]$ into two sets of groups. Let $\sigma \in \Mcal$ and $[m] \setminus \sigma$ be one such partition of the $m$ groups defined by the set $\sigma$. We follow exactly similar procedure as in the previous section, i.e., fixing trivial rates (through trivial classifiers) on the groups in $\sigma$ and allowing non-trivial rates from $\Scal_\rho$ on the groups in $[m] \setminus \sigma$. In particular, consider a paramterization $\nu : (\Scal_\rho, \Mcal, [k]) \rightarrow \Rcal^{1:m}$ defined as:
\begin{equation}
    \nu(\smbf, \sigma, i) \coloneqq \tupr \quad \text{such that} \quad \rmbf^g = \begin{cases}
    \embf_i & \text{if } g \in \sigma\\
    \smbf & \text{o.w. }
    \end{cases}
\label{append:eq:parvarphi}
\end{equation}
i.e., $\nu$ assigns trivial confusion rates $\embf_i$ on the groups in $\sigma$ and assigns $\smbf \in \Scal_\rho$ on the rest of the groups. 
Similar to the previous section, we first fix trivial classifier predicting class $1$ for groups in $\sigma$ and constrain the rates for groups in $[m] \setminus \sigma$ to be on the sphere $\Scal_\rho$. Such a setup is governed by the parametrization $\nu(\cdot,\sigma, 1)$ in equation~\eqref{append:eq:parvarphi}. Specifically, fixing $h^g(\xmbf)=1\; \forall\; g \in \sigma$ would entail the metric in Definition~\ref{def:linear} to be:
\vspace{-0.1cm}
    \begin{align*}
    \bPsi(\nu(\smbf, \sigma, 1); \ambfbar, \Bmbfbar, \lambdabar) &= 
    (1-\lambdabar)\inner{\ambfbar\odot(\bm{1} - \taumbf^{\sigma})}{\smbf} +  \lambda \inner{\etambfbar^{\sigma}}{\vert \embf_1 - \smbf \vert} + c_1,
    \numberthis \label{append:eq:metricb}
    \end{align*}
where $\taumbf^{\sigma} = \sum_{g\in \sigma}\bm{\tau}^{g}$ and $\etambfbar^{\sigma} = \sum_{u, v \in [m], v > u} \1\left[|\{u,v\}\cap\sigma|=1\right]\bmbfbar^{uv}$. 
Similar to the previous section, since $\embf_i$'s are binary vectors, the sign of the absolute function w.r.t. $\smbf$ can be recovered. In particular, the metric amounts to:
\begin{align*}
&\bPsi(\nu(\smbf, \sigma, 1); \ambfbar, \Bmbfbar, \lambdabar) = \inner{(1-\lambdabar)\ambfbar \odot (\bm{1}-\bm{\tau}^{2}) + \lambdabar \wmbf_1 \odot 
\etambfbar^{\sigma}}{\smbf} + c_1,
\numberthis \label{append:eq:metricbrich1}
\end{align*}
where $\wmbf_1 \coloneqq 1 - 2\embf_1$ and $c_1$ is a constant not affecting the responses. Notice that~\eqref{append:eq:metricb} and~\eqref{append:eq:metricbrich1} are analogous to~\eqref{append:eq:metricbm2} and~\eqref{append:eq:metricbrich1m2}, respectively, except that $\taumbf^2$ is replaced by $\taumbf^\sigma$ and $\bmbfbar^{12}$ is replaced by $\etambfbar^{\sigma}$. This is a linear metric in $\smbf$. We again the use the LPME procedure in line 10of Algorithm~\ref{alg:f-me}, which outputs a normalized slope $\breve \fmbf^\sigma$ such that $\Vert \breve \fmbf^\sigma\Vert_2=1$, and thus we get an analogous solution to~\eqref{append:eq:metricbfirstm2} as:

\begin{align*}
    \lambdabar \etambfbar^{\sigma} &= \wmbf_1 \odot \left[ \left( \frac{(1-\lambdabar)(1 - \tau^{\sigma}_{k-1})\abar_{k-1} + \lambdabar \etabar^{\sigma}_{k-1}}{\breve f^{\sigma}_{k-1}}\right) \breve \fmbf^{\sigma} - (1-\lambdabar)((\bm{1} - \taumbf^{\sigma})\odot\ambfbar  \right].
    \numberthis \label{append:eq:metricbfirst}
\end{align*}
In order to elicit entire $\etambfbar^{\sigma}$, we need one more linear relation such as~\eqref{append:eq:metricbfirst}. So, we now fix the trivial rates through trivial classifier predicting class $k$ for the groups in $\sigma$, i.e., fix $h^{g}(x) = k \, \forall \, \xmbf \in \Xcal$ if $g \in \sigma$, and thus $\rmbf^{g}  = \embf_k$ for all groups $g \in \sigma$. For the rest of the groups, we constrain the confusion rates to again lie in the sphere $\Scal_\rho$ i.e. $\rmbf^{g}  = \smbf$ for $\smbf \in \Scal_\rho$ for all groups $g \in [m] \setminus \sigma$. Such a setup is governed by the parametrization $\nu(\cdot,\sigma, k)$~\eqref{append:eq:parvarphi}. The metric $\bPsi$ in Definition~\ref{def:linear} amounts to:
\begin{align*}
&\bPsi(\nu(\smbf, \sigma, k); \ambfbar, \Bmbfbar, \lambdabar) = (1-\lambdabar)\inner{\ambfbar \odot (1-\bm{\tau}^{\sigma})}{\smbf} + \lambdabar  \inner{\etambfbar^{\sigma}}{\vert \embf_k - \smbf \vert} + c_k.
\numberthis \label{append:eq:metricbrich2}
\end{align*}
Thus by running LPME procedure again in line 11 of Algorithm~\ref{alg:f-me} results in $\tilde \fmbf^{12}$ with $\Vert \tilde \fmbf^{12} \Vert_2 = 1$. Using Remark~\ref{fair-rm:ratio}, we extract the following relation between the $(k-1)$-th and $((k-1)^2+1)$-th coordinates:
\begin{align*}
\frac{\tilde f^{\sigma}_{k-1}}{\tilde f^{\sigma}_{(k-1)^2+1}} = \frac{(1-\lambdabar)(1- \tau^{\sigma}_{k-1})\abar_ {k-1} - \lambdabar \etabar^{\sigma}_{k-1}}{(1-\lambdabar)(1 - \tau^{\sigma}_{(k-1)^2+1})\abar_{(k-1)^2+1} + \lambdabar \etabar^{\sigma}_{(k-1)^2+1}}.
\numberthis \label{append:eq:metricbsecond}
\end{align*}
Combining equations~\eqref{append:eq:metricbfirst} and~\eqref{append:eq:metricbsecond}, we have:

\begin{align*}
    \sum\nolimits_{u, v} \1\left[|\{u,v\}\cap \sigma|=1\right] \tilde \bmbf^{uv} &= \gammambf^{\sigma}, \numberthis \label{append:eq:ellsol}
\end{align*}
where
\begin{align*}
    \gammambf^{\sigma} &= \wmbf_1 \odot \left[ \delta^{\sigma} \fmbf^{\sigma} - \ambfhat \odot (\bm{1} - \taumbf^{\sigma})  \right], \\ \delta^{\sigma} &= \frac{2(1-\tau^{\sigma}_{k-1})\hat a_{k-1}}{f^\sigma_{k - 1}} \left[ \frac{ \frac{(1 - \tau^{\sigma}_{(k-1)^2+1})\hat a_{(k-1)^2+1}}{(1-\tau^{\sigma}_{k-1})\hat a_{k-1}} -  \frac{\tilde f^{\sigma}_{(k-1)^2+1}}{\tilde f^{\sigma}_{k-1}} }{\left( \frac{f^{\sigma}_{(k-1)^2+1}}{f^{\sigma}_{k-1}} - \frac{\tilde f^{\sigma}_{(k-1)^2+1}}{\tilde f^{\sigma}_{k-1}} \right)} \right], \numberthis
\end{align*}
and $\tilde \bmbf^{uv} \coloneqq \lambdabar\bmbfbar^{uv}/(1 - \lambdabar)$ is a scaled version of the true (unknown) $\bmbfbar$, which nonetheless can be computed from~\eqref{append:eq:ellsol}.

By two runs of LPME algorithm, we can get $\gammambf^{\sigma}$ and solve~\eqref{append:eq:ellsol}. However, the left hand side of~\eqref{append:eq:ellsol} does not allow us to recover the $\tilde \bmbf$'s separately and provides only one equation. Let us denote the Equation~\eqref{append:eq:ellsol} by $\ell^\sigma$ corresponding to the set $\sigma$. In order to elicit all $\tilde \bmbf$'s we need a system of $M \coloneqq {m\choose 2}$ independent equations in order to elicit the   $M$ weight vectors. 

This is easily achievable by choosing $M$ $\sigma$'s so that we get $M$ set of unique equations like~\eqref{append:eq:ellsol}. Let $\Mcal$ be those set of sets. 

\noindent In most cases,  pairing two groups to have trivial rates (through trivial classifiers) and rest of the groups to have rates from the sphere $\Scal$ will work. For example, when $m=3$, fixing $\Mcal = \{ \{1,2\}, \{1,3\}, \{2,3\}\}$ suffices. Thus, running over all the choices of sets of groups $\sigma \in \Mcal$ provides the system of equations $\Lcal \coloneqq \cup_{\sigma \in \Mcal} \ell^\sigma$ (line 12 in Algorithm~\ref{alg:f-me}), which is formally described as follows:

\begin{equation}
    \left[ \begin{array}{cccc} \Xi & 0 & \dots & 0\\
    0 & \Xi & \dots & 0 \\
    \dots & \dots & \dots & \dots \\
    0 & 0 & \dots & \Xi 
    \end{array}\right] \left[ \begin{array}{c} \tilde \bmbf_{(1)} \\
    \tilde \bmbf_{(2)} \\
    \dots \\
    \tilde \bmbf_{(q)}
    \end{array}\right] = \left[ \begin{array}{c} \bm\gamma_{(1)} \\
    \bm\gamma_{(2)} \\
    \dots \\
    \bm\gamma_{(q)}
    \end{array}\right],
    \label{fair-append:btilde}
\end{equation}
where $\tilde \bmbf_{(i)} = (\tilde b_{i}^1,\tilde b_{i}^2, \cdots, \tilde b_{i}^M)$ and $\gammambf_{(i)} = (\gamma_{i}^1, \gamma_{i}^2, \cdots, \gamma_{i}^M)$ are vectorized versions of the $i$-th entry across groups for $i \in [q]$, and $\Xi \in \{0,1\}^{M\times M}$ is a binary full-rank matrix denoting membership of groups in the set $\sigma \in \Mcal$. For instance, for the choice of $\Mcal = \{ \{1,2\}, \{1,3\}, \{2,3\}\}$ when $m=3$ gives:
\bequation
\Xi = \left[ \begin{array}{ccc} 0 & 1 & 1\\
    1 & 0 & 1\\
    1 & 1 & 0\\
\end{array}\right].
\eequation
From technical point of view, one may choose any $\Mcal$ such that the resulting group membership matrix $\Xi$ is non-singular. Hence the solution of the system of equations $\Lcal$ is:
\begin{equation}
    \left[ \begin{array}{c} \tilde \bmbf_{(1)} \\
    \tilde \bmbf_{(2)} \\
    \dots \\
    \tilde \bmbf_{(q)}
    \end{array}\right] = \left[ \begin{array}{cccc} \Xi & 0 & \dots & 0\\
    0 & \Xi & \dots & 0 \\
    \dots & \dots & \dots & \dots \\
    0 & 0 & \dots & \Xi
    \end{array}\right]^{(-1)} \left[ \begin{array}{c} \bm\gamma_{(1)} \\
    \bm\gamma_{(2)} \\
    \dots \\
    \bm\gamma_{(q)}
    \end{array}\right].
    \label{fair-append:eq:sol-b}
\end{equation}
When we normalize $\tilde \bmbf$, we get the final fairness violation weight estimates as:
\begin{equation}
\bmbfhat^{uv} = \frac{\tilde \bmbf^{uv}}{\sum_{u,v=1, v > u}^m \Vert \tilde \bmbf^{uv} \Vert_2} \quad \text{for} \quad u,v \in [m], v>u.
 \label{append:eq:bsol}
\end{equation} 
Notice that, due to the above normalization, the solution is again independent of the true trade-off $\lambdabar$.

\subsection{Eliciting Trade-off $\lambdabar$; Part 3 in Figure~\ref{fig:workflow} and line 16 in Algorithm~\ref{alg:f-me}}
\label{append:ssec:lambda}

For ease of notation, let us construct a parametrization $\nu' : \Scal^+_\varrho  \rightarrow \Rcal^{1:m}$:
\vspace{-0.1cm}
\begin{equation}
\nu'(\smbf^+) \coloneqq (\smbf^+, \ombf, \dots, \ombf).
    \label{append:eq:parlambda}
\end{equation}

Using the parametrization $\nu'$ from~\eqref{append:eq:parlambda}, the metric in Definition~\ref{def:linear} reduces to a linear metric in $\smbf^+$ as discussed in~\eqref{eq:metriclambda}, i.e:
\begin{align*}
        \bPsi(\nu'(\smbf^+) \,;\,\ambfbar, \Bmbfbar, \lambdabar) = \inner{(1-\lambdabar)\taumbf^1\odot\ambfbar + \lambdabar \sum\nolimits_{v=2}^m \bmbfbar^{1v}}{\smbf^+} + c.
         \numberthis \label{append:eq:metriclambda}
\end{align*}

We first show the proof of Lemma~\ref{fair-lm:lambda} and then discuss the trade-off elicitation algorithm (Algorithm~\ref{fair-alg:lambda}).

\bproof[Proof of Lemma~\ref{fair-lm:lambda}]

For simplicity, let us abuse notation for this proof and denote $\taumbf^1\odot\ambfbar$ simply by $\ambf$,  $\sum\nolimits_{v=2}^m \bmbfbar^{1v}$ simply by $\bmbf$, and $\Scal_\varrho^+$ simply by $\Scal$. 

$\Scal$ is a convex set. Let $\Zcal = \{\zmbf = (z_1, z_2) \,|\, z_1 = <\ambf, \smbf>, z_2 = <\bmbf, \smbf>, \smbf \in \Scal\}$.

\emph{Claim:} $\Zcal$ is convex.

Let $z, z' \in \Zcal$.

$\alpha z_1+ (1-\alpha) z'_1 ~=~
\alpha <\ambf, \smbf> + (1-\alpha)  <\ambf, \smbf'> 
~=~
<\ambf, \alpha \smbf + (1-\alpha)  \smbf'> 
$

$\alpha z_2+ (1-\alpha) z'_2 ~=~
\alpha <\bmbf, \smbf> + (1-\alpha)  <\bmbf, \smbf'> 
~=~
<\bmbf, \alpha \smbf + (1-\alpha)  \smbf'> 
$

Since $\alpha \smbf + (1-\alpha)  \smbf' \in \Scal$,
$\alpha z + (1-\alpha) z' \in \Zcal$. Hence $\Zcal$ is convex. 

\emph{Claim:} The boundary of the set $\Zcal$ is a strictly convex curve with no vertices for $\ambf\neq \bmbf$.  

Recall that, the required function is given by:
\begin{align}
\vartheta(\lambda)  = \max\nolimits_{\zmbf \in \Zcal} (1-\lambda)z_1 + \lambda z_2 + c \label{eq:metricwomod}
\end{align}

(i) Since the set $\Zcal$ is convex, every boundary point is supported by a hyperplane. 

(ii)  Since $\ambf \neq \bmbf$, notice that the slope is uniquely defined by $\lambda$. Since the sphere $\Scal$ is strictly convex, the above linear functional defined by $\lambda$ is maximized by a  unique point in $\Zcal$ (similar to Lemma~\ref{mult-lem:spherebayes}). Thus, the the hyperplane is tangent at a unique point on the boundary of $\Zcal$. 

(iii) It only remains to show that there are no vertices on the boundary of $\Zcal$. Recall that a vertex exists if (and only if) some point is supported by more than one tangent hyperplane in two dimensional space. This means there are two values of $\lambda$ that achieve the same maximizer. This is contradictory since there are no two linear functionals that achieve the same maximizer on $\Scal$. 

This implies that the boundary of $\Zcal$ is a strictly convex curve. Since we are interested in the maximization of $\vartheta$, let this boundary be the upper boundary denoted by $\partial\Zcal_+$. 

\emph{Claim:} Let  $\upsilon:[0,1]\to \partial\mathcal \Zcal_+$ be continuous, bijective, parametrizations of the upper boundary. Let $\vartheta:\mathcal \Zcal \to \mathbb R$ be a quasiconcave function which is 
monotone increasing in both $z_1$ and $z_2$.
Then the composition $\vartheta\circ \upsilon: [0,1]\to\mathbb \Rmbb$ is strictly
quasiconcave (and therefore unimodal with no flat regions) on the interval $[0, 1]$.

Let $S$ be some superlevel set of the quasiconcave function $\vartheta$. 
Since $\upsilon$ is a continuous bijection and since the boundary $\partial \Zcal_+$ is a strictly convex curve with no vertices, w.l.o.g., for any $r<s<t$, 
$z_1(\upsilon(r))< z_1(\upsilon(s)) < z_1(\upsilon(t))$, and $z_2(\upsilon(r))>z_2(\upsilon(s))>z_2(\upsilon(t))$.
(otherwise, swap $r$ and $t$). Since the boundary $\partial \Zcal_+$ is a strictly convex curve, then $\upsilon (s)$ must be greater (component-wise) a point in the convex combination of $\upsilon(r)$ and $\upsilon(t)$. Let us denote that point by $u$. Since $\vartheta$ is monotone increasing, then $x\in S$ implies that $y \in S$, too,  for all $y\geq x$ componentwise. Therefore, $\vartheta(\upsilon(s)) \leq \vartheta(u)$. Since $S$ is convex, $u \in S$ and thus $\upsilon(s) \in S$.

This implies that $\upsilon^{-1}(\partial \Zcal_+\cap S)$ is an interval; hence it is convex, which in turn tells us that the superlevel sets of $\vartheta\circ \upsilon$ are convex. So, $\vartheta\circ \upsilon$ is quasiconcave, as desired. This implies unimodaltiy, because a function  defined on real line which has more than one local maximum can not be quasiconcave. Moreover, since there are no vertices on the boundary $\partial \Zcal_+$, the $\vartheta\circ \upsilon: [0,1]\to\mathbb \Rmbb$ is strictly quasiconcave (and thus unimodal with no flat regions) on the interval $[0, 1]$. This completes the proof of Lemma~\ref{fair-lm:lambda}.
\eproof

\vspace{-0.5cm}
\section{Proof of Section~\ref{fair-sec:guarantees}}
\label{append:sec:guarantees}
\bproof[Proof of Theorem~\ref{thm:error}] We break this proof into three parts. 
\vspace{-0.2cm}
\begin{enumerate}[leftmargin=*]
\item \emph{Elicitation guarantees for the misclassification cost $\hphi$ (i.e., $\ambfhat$)}

Since Algorithm~\ref{alg:f-me} elicits a linear metric using the $q$-dimensional sphere $\Scal$, the guarantees on $\ambfhat$ follows from Theorem~\ref{thm:lpm-elit-error}. Thus, under Assumption~\ref{as:regularity}, the output $\ambfhat$ from line 2 of Algorithm~\ref{alg:f-me} satisfies $\Vert \ambf^*-\ambfhat \Vert_{2}\leq O(\sqrt{q}(\epsilon+\sqrt{\epsilon_\Omega/\rho}))$ after $O\left(q\log \tfrac \pi {2\epsilon}\right)$ queries. 

\item \emph{Elicitation guarantees for the fairness violation cost $\hvarphi$ (i.e., $\Bmbfhat$)}

We start with the definition of true $\bm{\gamma}$ (i.e. when all the elicited entities are true) from~\eqref{append:eq:ellsol} and let us drop the superscript $\sigma$ for simplicity. Furthermore, let $\epsilon+\sqrt{\epsilon_\Omega/\rho}$ be denoted by $\epsilon$.

\begin{align*}
\gammambf = \wmbf_1 \odot \left[ \delta  \breve \fmbf - \ambfbar\odot(\bm{1}-\taumbf)  \right], \quad \text{where} \numberthis
\end{align*}
\vspace{-0.5cm}
\begin{align*}
\delta = \frac{2(1-\tau_{k-1})\abar_{k-1}}{ \breve f_{k - 1}} \left[ \frac{ \frac{(1-\tau_{(k-1)^2+1}) \abar_{(k-1)^2+1}}{(1-\tau_{k-1}) \abar_{k-1}} -  \frac{ \tilde f_{(k-1)^2+1}}{\tilde f_{k-1}} }{\left( \frac{\breve f_{(k-1)^2+1}}{ \breve f_{k-1}} - \frac{ \tilde f_{(k-1)^2+1}}{ \tilde f_{k-1}} \right)} \right]. \numberthis
\end{align*}
Let us look at the derivative of the $i$-th coordinate of $\bm{\gamma}$.  
\bequation
\frac{\partial \gamma_i}{\partial a_j} = \begin{cases}
0 & \text{if } j\neq i,j\neq k-1,j\neq(k-1)^2+1\\
-\tau_i & \text{if } j = i\\
c_{i, 1} & \text{if } j = k-1\\
c_{i, 2} & \text{if } j = (k-1)^2+1,
\end{cases}
\eequation
where $c_{i, 1}$ and $c_{i, 2}$ are some bounded constants due to Assumption~\ref{as:regularity}. Similarly, $\partial\gamma_i/\partial f_j$ is bounded as well due to the regularity Assumption~\ref{as:regularity}. This means that $\gamma_i$ is Lipschitz in $\ell_2$-norm w.r.t. $\ambf$ and $\fmbf$. Thus,
\bequation
\Vert \bm{\gamma} - \bm{\hat \gamma} \Vert_\infty \leq c_3 \Vert \ambfbar - \ambfhat  \Vert_2  + c_4 \Vert \breve \fmbf - \hat{\breve{\fmbf}} \Vert_2,
\eequation
for some Lipschits constants $c_3$ and $c_4$. From the bounds of Part~1 of this proof, we have:
\bequation
\Vert \bm{\gamma} - \bm{\hat \gamma} \Vert_\infty \leq O(\sqrt q\epsilon).
\eequation

Recall the construction of $\tilde \bmbf_{(i)}$ from~\eqref{fair-append:btilde}. We then have from the solution of system of equations~\eqref{fair-append:eq:sol-b} that:

\vspace{-0.35cm}
\bequation
 \tilde \bmbf_{(i)} = \Xi^{-1}\bm{\gamma}_{(i)} \quad \forall \; i \in [q],
\eequation
where $\tilde \bmbf_{(i)} =  (\tilde b_{i}^1, \tilde b_{i}^2, \cdots, \tilde b_{i}^M)$ and $\tilde \gammambf_{(i)} = (\gamma_{i}^1, \gamma_{i}^2, \cdots, \gamma_{i}^M)$ are vectorized versions of the $i$-th entry across groups for $i \in [q]$. $\Xi \in \{0,1\}^{M \times M}$ is a full-rank symmetric matrix with bounded infinity norm $\Vert \Xi^{-1} \Vert_\infty \leq c$ (here, infinity norm of a matrix is defined as the maximum absolute row sum of the matrix). Thus we have: $\Vert \tilde \bmbf_{(i)} - \hat {\tilde\bmbf}_{(i)}\Vert_\infty = $
\vspace{-0.4cm}
\bequation
\Vert \Xi^{-1} \bm{\gamma}_{(i)} - \Xi^{-1} \hat{\bm{\gamma}}_{(i)}
\Vert_\infty = \Vert \Xi^{-1} (\bm{\gamma}_{(i)} - \hat{\bm{\gamma}}_{(i)}) \Vert_\infty \leq \Vert \Xi^{-1} \Vert_\infty \Vert \bm{\gamma}_{(i)} - \hat{\bm{\gamma}}_{(i)} \Vert_\infty,
\eequation
which gives
\bequation
\Vert \tilde \bmbf_{(i)} - \hat {\tilde\bmbf}_{(i)}\Vert_\infty \leq O(\sqrt q \epsilon).
\eequation

Now, our final estimate is the normalized form of $\hat{\tilde{\bmbf}}$ from~\eqref{append:eq:bsol}, so the final error in the stacked version $vec(\Bmbfbar)$ and $vec(\Bmbfhat)$ is:

\begin{equation}
\Vert vec(\Bmbfbar) - vec(\Bmbfhat) \Vert_\infty \leq O(\sqrt q\epsilon).
\label{eq:errinB}
\end{equation}

Since there are $q\times M$ entities in $vec(\Bmbf)$, we have:

\begin{align*}
\Vert vec(\Bmbfbar) - vec(\Bmbfhat) \Vert_2 \leq O(\sqrt{qM}\sqrt q\epsilon) = O(mq\epsilon).
\numberthis \label{append:bguarantee}
\end{align*}

Due to elicitation on sphere and the oracle noise $\epsilon_\Omega$ as defined in Definition~\ref{fair-def:noise}, we can replace $\epsilon$ with $\epsilon + \sqrt{\epsilon_\Omega/\rho}$ back to get the final bound on fairness violation weights as in Theorem~\ref{thm:error}.

\item \emph{Elicitation guarantees for the trade-off parameter (i.e., $\lambdahat$)}

The metric for our purpose is a linear metric in $\smbf^+ \in \Scal_\rho^+$ with the following slope:
\vspace{-0.2cm}
\begin{align*}
    \bPsi(\nu'''(\smbf^+) \,;\,\ambfbar, \Bmbfbar, \lambdabar) = \inner{(1-\lambdabar)\taumbf^{1}\odot\ambfbar + \lambdabar \sum_{v=2}^m \bmbfbar^{1v}}{\smbf^+}. \numberthis
    \label{append:eq:lambdametric}
\end{align*}
\vskip -0.2cm
Since we elicit $\lambda$ through queries over a surface of the sphere, we pose this problem as finding the right angle (slope) defined by the true $\lambdabar$. Note that
$\lambdabar$ is what we want to elicit; however, due to oracle noise $\epsilon_\Omega$, we can only aim to achieve a target angle $\lambda_t$. Moreover, we do not have true $\ambfbar$ and $\Bmbfbar$ but have only estimates $\ambfhat$ and $\Bmbfhat$. Thus we query proxy solutions always and can only aim to achieve an estimated version $\lambda_e$ of the target angle. Lastly, Algorithm~\ref{fair-alg:lambda} is stopped within an $\epsilon$ threhsold, thus the final solution $\lambdahat$  is within $\epsilon$ distance from $\lambda_e$. In total, we want to find:

\bequation
\vert \lambdabar - \lambdahat \vert \leq \underbrace{\vert \lambdabar - \lambda_t \vert}_{\text{oracle error}} +  \underbrace{\vert \lambda_t - \lambda_e \vert}_{\text{estimation error}} + \underbrace{\vert \lambda_e - \lambdahat \vert}_{\text{optimization error}}.
\eequation

\begin{itemize}[itemsep = 0pt, leftmargin = 0.5cm]
    \item optimization error: $\vert \lambda_e - \lambdahat \vert\leq \epsilon$. 
    \item oracle error: Notice that the oracle correctly answers as long as $\varrho(1 - \cos(\lambdabar - \lambda_t)) > \epsilon_\Omega$. This is because the metric is a 1-Lipschitz linear function, and the optimal value on the sphere of radius $\varrho$ is $\varrho$. However, as $1 - \cos(x) \geq x^2/3$, so oracle is correct as long as $\vert \lambdabar - \lambda_e\vert \geq \sqrt{3\epsilon_\Omega/\varrho}$. Given this, the binary search proceeds in the correct direction. 
    \item estimation error: We make this error because we only have access to the estimated $\ambfhat$ and $\Bmbfhat$ not the true $\ambfbar$ and $\Bmbfbar$. However, since the metric in~\eqref{append:eq:lambdametric} is Lipschitz in $\ambfbar$ and $\sum_{v=2}^m \bmbfbar^{1v}$, this error can be treated as oracle feedback noise where the oracle responses with the estimated $\ambfhat$ and $\Bmbfhat$. Thus, if we replace $\epsilon_\Omega$ from the previous point to the error in $\ambfhat$ and $\sum_{v=2}^m\bmbfhat^{1v}$, the binary search moves in the right direction as long as 
    \bequation
    \vert \lambda_t -\lambda_e\vert \geq O\left(\sqrt{\frac{\Vert \ambfbar - \ambfhat \Vert_2 + \sum_{v=2}^m \Vert \bmbfbar^{1v} - \bmbfhat^{1v} \Vert_2}{\varrho}}\right) = O\left(\sqrt{m q (\epsilon + \sqrt{\epsilon_\Omega/\rho})/\varrho}\right),
    \eequation
    where we have used~\eqref{append:bguarantee} to bound the error in $\{\bmbfhat^{1v}\}_{v=2}^m$.
\end{itemize}
Combining the three error bounds above gives us the desired result for trade-off parameter in Theorem~\ref{thm:error}.
\end{enumerate}
\vspace{-0.4cm}
\eproof

%% file: quadratic/supplement.tex
\section{Geometry of the Feasible Space (Proofs of Section~\ref{ssec:mpme}, \ref{ssec:f-metric})}
\label{append:sec:confusion}

\bproof[Proof of Proposition~\ref{prop:C} and Proposition~\ref{prop:f-C}]

The proof of Proposition~\ref{prop:f-C} is same as Proposition~\ref{fair-prop:C}. The proof of Proposition~\ref{prop:C} is analogous where the probability measures (corresponding to classifiers and their rates) are not conditioned on any group. 
\eproof

\subsection{Finding the Sphere $\Scal\subset \Rcal$}
\label{append:ssec:sphere}

In this section, we provide details regarding how a sphere $\Scal$ with sufficiently large radius $\rho$ inside the feasible region $\Rcal$ may be found (see Figure~\ref{fig:MEgeom}(b)). The following discussion is borrowed from Appendix~\ref{apx:fair} and provided here for completeness. 

\balgorithm[t]
\caption{Obtaining the sphere $\Scal \subset \Rcal$ (Figure~\ref{fig:MEgeom}(b)) of radius $\rho$ centered at $\ombf$}
\label{alg:sphere}
\small
\balgorithmic[1]
\FOR{$j=1, 2, \cdots, q$}
\STATE Let $\mathbf \alphambf_j$ be the standard basis vector. 
\STATE Compute the maximum constant $c_j$ such that $\ombf + c_j \mathbf \alphambf_j$ is feasible by solving~\eqref{quad-eq:op1}.
\ENDFOR
\STATE Let $CONV$ denote the convex hull of $\{\ombf\pm c_j\mathbf \alphambf_j\}_{j=1}^{q}$. It will be centered at $\ombf$.
\STATE Compute the radius $\rho$ of the largest ball that fits in $CONV$.
\STATE\textbf{Output:} Sphere $\Scal$ with radius $\rho$ centered at $\ombf$.
\ealgorithmic
\ealgorithm

The following optimization problem is a special case of OP2 in~\cite{narasimhan2018learning}. The problem is associated with a feasibility check problem.  Given a rate profile $\rmbf_0$, the optimization routine tries to construct a classifier that achieves the rate $\rmbf_0$ within small error $\epsilon >0$. 

\begin{align}
    \min_{\rmbf \in \Rcal} \; 0 \qquad s.t. \;\; \Vert \rmbf - \rmbf_0 \Vert_2 \leq \epsilon.
    \label{quad-eq:op1}
\end{align}

The above optimization problem checks the feasibility, and if there exists a solution to the above problem, then Algorithm~1 of~\cite{narasimhan2018learning} returns it. 
Furthermore, Algorithm~\ref{alg:sphere} computes a value of $\rho\geq \tilde{p}/k$, where $\tilde{p}$ is the radius of the largest ball contained in the set $\Rcal$. Also, the approach in~\cite{narasimhan2018learning} is consistent, thus we should get a good estimate of the sphere, provided we have sufficiently large number of samples. The algorithm is completely offline and does not impact oracle query complexity.

\blemma
    Let $\tilde{p}$ denote the radius of the largest ball in $\Rcal$ centered at $\ombf$. Then Algorithm~\ref{alg:sphere} returns a sphere with radius $\rho\geq \tilde{p}/k$, where $k$ is the number of classes. 
\elemma

The idea in Algorithm~\ref{alg:sphere} can be trivially extended to finding a sphere $\Sbar \subset \Rcal^1\cap\dots\cap\Rcal^m$ corresponding to Remark~\ref{as:f-sphere}.

\section{{Quadratic Performance Metric Elicitation Procedure}}
\label{append:sec:qpme}

In this section, we describe how the subroutine calls to LPME in Algorithm~\ref{alg:q-me} elicit a quadratic metric in Definition~\ref{def:quadmet}. We start with the shifted metric  of Equation~\eqref{eq:loclinapx}. 

As explained in Chapter~\ref{chp:quadratic}, we may assume $d_1 \neq 0$ due to Assumption~\ref{assump:smoothness}. We can derive the following solution using any non-zero coordinate of $\dmbf$, instead of $d_1$. We can identify a non-zero coordinate using $q$ trivial queries of the form $(\varrho\alphambf_i + \ombf, \ombf), \forall i \in [q]$. 

\begin{enumerate}
    \item From line 2 of Algorithm~\ref{alg:q-me}, we get local linear approximation at $\ombf$. Using Remark~\ref{rm:ratio}, we have~\eqref{eq:0col} which is
    \begin{equation}
    d_i = \frac{f_{i0}}{f_{10}}d_1 \qquad \forall \; i \in \{2, \dots, q\}.
    \label{append:eq:0col}
\end{equation}
\item Similarly, if we apply LPME on small balls around rate profiles $\zmbf_j$, Remark~\ref{rm:ratio} gives us:
\begin{equation}
\frac{d_i + (\rho-\varrho)B_{ij}}{d_1 + (\rho-\varrho)B_{1j}} = \frac{f_{ij}}{f_{1j}} \quad \forall \; i \in \{2, \ldots, q\},\; j \leq i.
\label{append:eq:jcol}
\end{equation}

\begin{align*}
    &\implies d_i + (\rho-\varrho)B_{ij} = \frac{f_{ij}}{f_{1j}}(d_1 + (\rho-\varrho)B_{1j})\\
    &\implies (\rho-\varrho)B_{ij} = \frac{f_{ij}}{f_{1j}}(d_1 + (\rho-\varrho)B_{j1}) - d_i \\
    &\implies (\rho-\varrho)B_{ij} = \frac{f_{ij}}{f_{1j}}(d_1 +  \frac{f_{j1}}{f_{11}} (d_1 + (\rho - \varrho)B_{11}) - d_j ) - \frac{f_{i0}}{f_{10}}d_1\\
    &\implies (\rho-\varrho)B_{ij} = \left(\frac{f_{ij}}{f_{1j}} - \frac{f_{i0}}{f_{10}} + \frac{f_{ij}}{f_{1j}} \left(\frac{f_{j1}}{f_{11}} - \frac{f_{j0}}{f_{10}}\right)\right)  d_1 + (\rho-\varrho)  \frac{f_{j1}}{f_{11}}B_{11}, \numberthis \label{append:eq:solvemidssystem}
\end{align*}
where we have used that the matrix $\Bmbf$ is symmetric in the second step, and~\eqref{append:eq:0col} in the last two steps. We can represent each element in terms of $B_{11}$ and $d_1$. So, a relation between $B_{11}$ and $d_1$ may allow us to represent each element of $\ambf$ and $\Bmbf$ in terms of $d_1$.

\item Therefore, by applying LPME on small balls around rate profiles $-\zmbf_1$, Remark~\ref{rm:ratio} gives us~\eqref{eq:negativegrad}:

\begin{equation}
    \frac{d_2-(\rho - \varrho)B_{21}}{d_1-(\rho - \varrho)B_{11}} = \frac{f_{21}^-}{f_{11}^-}.
    \label{append:eq:negativegrad}
\end{equation}

\item Using~\eqref{append:eq:jcol} and~\eqref{append:eq:negativegrad}, we have:

\begin{align*}
    (\rho - \varrho)B_{11} = \frac{ \frac{f_{21}^-}{f_{11}^-} + \frac{f_{21}}{f_{11}} - 2\frac{f_{20}}{f_{10}}  }{ \frac{f_{21}^-}{f_{11}^-} - \frac{f_{21}}{f_{11}} }d_{1}.
    \numberthis \label{append:eq:firsttermB}
\end{align*}
Putting~\eqref{append:eq:firsttermB} in~\eqref{append:eq:solvemidssystem}, we get:
\begin{align*}
    B_{ij} &=  \left[\frac{f_{ij}}{f_{1j}}\left(1 + \frac{f_{j1}}{f_{11}} \right) - \frac{f_{ij}}{f_{1j}}\frac{f_{j0}}{f_{10}} - \frac{f_{i0}}{f_{10}} +  \frac{f_{ij}}{f_{1j}}\frac{f_{j1}}{f_{11}} \frac{ \frac{f_{21}^-}{f_{11}^-} + \frac{f_{21}}{f_{11}} - 2\frac{f_{20}}{f_{10}}  }{ \frac{f_{21}^-}{f_{11}^-} - \frac{f_{21}}{f_{11}}  }\right]d_1 \\
    &= \left(F_{i,1,j} (1 + F_{j,1,1}) - F_{i,1,j} F_{j,1,0}  - F_{i,1,0} + F_{i,1,j}\frac{F^-_{2,1,1} + F_{2,1,1} - 2F_{2,1,0}}{F^-_{2,1,1} - F_{2,1,1}}\right)d_1,
    \numberthis \label{append:eq:poly2elicitamatfinal}
\end{align*}
where
$F_{i,j,l} = \frac{f_{il}}{f_{jl}}$ and $F^-_{i,j,l} = \frac{f^-_{il}}{f^-_{jl}}$. As $\ambf = \dmbf + \Bmbf \ombf$, we can represent each element of $\ambf$ and $\Bmbf$ using~using~\eqref{append:eq:0col}  and \eqref{append:eq:poly2elicitamatfinal} in terms of $d_1$. We can then use the normalization condition $\Vert \ambf\Vert_2^2 + \Vert \Bmbf \Vert_F^2 = 1$ to get estimates of $\ambf, \Bmbf$ which are independent of $d_1$. 
\end{enumerate}

This completes the derivation of solution from QPME (section~\ref{sec:quadme}).

\section{{Fair (Quadratic) Performance Metric Elicitation Procedure}}
\label{append:sec:fpme}

\balgorithm[t]
\caption{Fair (Quadratic) Performance Metric Elicitation}
\label{alg:fqme}
\balgorithmic[1]
\STATE \textbf{Input:} Query set $\Scal'$, search tolerance $\epsilon > 0$, oracle $\Omega'$ 
\STATE Let $\Lcal \leftarrow \varnothing$ 
\FOR{$\sigma \in \Mcal$}
\STATE $\bm{\beta}^{\sigma}\leftarrow$ QPME$(\Scal', \epsilon, \Omega')$
\STATE Let $\ell^\sigma$ be Eq.~\eqref{append:eq:fairBij}, extend $\Lcal \leftarrow \Lcal \cup \{\ell^\sigma\}$
\ENDFOR
\STATE $\hat{\Bmbb} \leftarrow $ normalized solution from~\eqref{append:eq:fairbsol} using $\Lcal$
\STATE $\hat \lambda \leftarrow$ trace back normalized solution from~\eqref{append:eq:fairbsol} for any $\sigma$
\STATE \textbf{Output:} $\ambfhat, \hat{\Bmbb}, \hat \lambda$ 
\ealgorithmic
\ealgorithm

We first discuss eliciting the fair (quadratic) metric in Definition~\ref{def:f-linmetric}, where all the parameters are unknown. We then provide an alternate procedure for eliciting just the trade-off parameter $\lambda$ when the predictive performance and fairness violation coefficients are known. The latter is a separate application as discussed in~\cite{zhang2020joint}. However, unlike Zhang et al.~\cite{zhang2020joint}, instead of ratio queries, we use simpler pairwise comparison queries.

In this section, we work with any number of groups $m\geq 2$. The idea, however, remains the same as described in Chapter~\ref{chp:quadratic} for number of groups $m=2$. We specifically select queries from the sphere $\overline{\Scal} \subset \Rcal^1 \cap \dots \cap\Rcal^m$, which is common to all the group-specific feasible region of rates, so to reduce the problem into multiple instances of the proposed QPME procedure of Section~\ref{sec:quadme}. 

Suppose that the oracle's fair performance metric is $\phi^{\text{fair}}$ parametrized by $(\ambf, \Bmbb, \lambda)$  as in Definition~\ref{def:f-linmetric}. The overall fair metric elicitation procedure framework is summarized in Algorithm~\ref{alg:fqme}. The framework exploits the sphere $\overline{\Scal} \subset \Rcal^1 \cap \dots\cap\Rcal^m$ and uses the QPME procedure (Algorithm~\ref{alg:q-me}) as a subroutine multiple times. 

Let us consider a non-empty set of sets $\Mcal \subset 2^{[m]} \setminus \{\varnothing, [m]\}$. We will later discuss how to choose such a set $\Mcal$. 
We partition the set of groups $[m]$ into two sets of groups. Let $\sigma \in \Mcal$ and $[m] \setminus \sigma$ be one such partition of the $m$ groups defined by the set of groups $\sigma$. For example, when $m=3$, one may choose the set of groups $\sigma = \{1, 2\}$. 

Now, consider a sphere $\Scal'$ whose elements $\rmbf^{1:m} \in \Scal'$ are given by:
\begin{equation}
    \rmbf^g = \begin{cases}
    \smbf & \text{if } g \in \sigma\\
    \ombf & \text{o.w. }
    \end{cases}
\label{eq:parvarphi}
\end{equation}
This is an extension of the sphere $\Scal'$ defined in Chapter~\ref{chp:quadratic} for the $m>2$ case. Elements in $\Scal'$ have rate profiles $\smbf \in \overline{\Scal}$ to the groups in $\sigma$ and trivial rate profile $\ombf$ to the remaining groups in $[m] \setminus \sigma$. 
Analogously, the modified oracle is $\Omega'(\rmbf_1, \rmbf_2) = \Omega((\rmbf^{1:m}_1), (\rmbf^{1:m}_2))$, where $\rmbf^{1:m}_1, \rmbf^{1:m}_2$ are the elements of the spheres $\Scal'$ above. 
Thus, for elements in $\Scal'$, the metric in Definition~\ref{def:f-linmetric} reduces to:

\begin{align*}
\phi^{\text{fair}}(\rmbf^{1:m} \in \Scal' \,;\, \ambf, \Bmbb, \lambda) =  
(1-\lambda)\inner{\ambf \odot \taumbf^\sigma}{\smbf - \ombf} + \lambda \frac{1}{2} (\smbf - \ombf)^T\Wmbf^\sigma(\smbf - \ombf) + c^\sigma 
\numberthis \label{eq:metricbrich}
\end{align*}
where $\taumbf^\sigma = \sum_{g\in \sigma}\taumbf^g$, $\Wmbf^\sigma = \sum_{u \in \sigma, v \in [m]\setminus\sigma} B^{uv}$, and $c^\sigma$ is a constant not affecting the oracle responses.

The above metric is a particular instance of $\bphi(\smbf; \dmbf, \Bmbf)$ in~\eqref{eq:quadmetshift} with $\dmbf \coloneqq (1-\lambda)\ambf\odot\taumbf^\sigma$ and $\Bmbf \coloneqq \lambda \Wmbf^\sigma$; thus, we apply QPME procedure as a subroutine in  Algorithm~\ref{alg:fqme} to elicit the metric in~\eqref{eq:metricbrich}. 

The only change needed to be made to the algorithm is in line 7, where 
we need to take into account the changed relationship between $\dmbf$ and $\ambf$, and need to separately (not jointly) normalize the linear and quadratic coefficients. With this change, the output of the algorithm directly gives us the required estimates. 
Specifically, we have from line 2 of Algorithm~\ref{alg:q-me} and \eqref{eq:0col} 
an estimate 

\begin{equation}
 \frac{{d}_{i}}{{d}_{1}} = \frac{\tau^\sigma_{i} {a}_i}{\tau^\sigma_{1} {a}_1} = \frac{f_{i0}}{f_{10}} \implies    {a}_i = \frac{f_{i0}}{f_{10}} \frac{\tau^\sigma_{1}}{\tau^\sigma_{i}} {a}_1.
 \label{append:eq:fair0col}
\end{equation}

Using the normalization condition (i.e., $\Vert \ambf \Vert_2 = 1$), we directly get an estimate $\ambfhat$ for the linear coefficients. Similarly, steps 3-5 of Algorithm~\ref{alg:q-me} and \eqref{eq:poly2elicitamatfinal} gives us:$\hat{B}_{ij} = $
\begin{align*}
    \sum_{u \in \sigma, v \in [m]\setminus\sigma} \tilde B^{uv}_{ij} &= \Big(F_{i,1,j}^\sigma (1 + F_{j,1,1}^\sigma) - F_{i,1,j}^\sigma F_{j,1,0}^\sigma d_{1}
    - F_{i,1,0}^\sigma + F_{i,1,j}^\sigma\textstyle\frac{F^{-, \sigma}_{2,1,1} + F_{2,1,1}^\sigma - 2F_{2,1,0}^\sigma}{F^{-, \sigma}_{2,1,1} - F_{2,1,1}^\sigma}\Big)\tau^1_1\hat{a}_1 \\
    &= \beta^\sigma,  \numberthis \label{append:eq:fairBij}
\end{align*}
where the above solution is similar to the two group case in~\eqref{eq:fairBij}, but here it is corresponding to a partition of groups defined by $\sigma$, and $\tilde \Bmbf^{uv} \coloneqq \lambda\Bmbf^{uv}/(1 - \lambda)$ is a scaled version of the true (unknown) $\Bmbf^{uv}$. Let equation~\eqref{append:eq:fairBij} be denoted by $\ell^\sigma$. Also, let the right hand side term of~\eqref{append:eq:fairBij} be denoted by $\beta^\sigma$. 

Since we want to elicit $m\choose 2$ fairness violation weight matrices in $\Bmbb$, we require $m\choose 2$ ways of partitioning the groups into 
two sets so that we construct $m\choose 2$ independent matrix equations similar to~\eqref{append:eq:fairBij}. 
Let $\Mcal$ be those set of sets. 
Thus, running over all the choices of sets of groups $\sigma \in \Mcal$ provides the system of equations $\Lcal \coloneqq \cup_{\sigma \in \Mcal} \ell^\sigma$ (line 5 in Algorithm~\ref{alg:fqme}), which is:

\begin{equation}
    \left[ \begin{array}{cccc} \Xi & 0 & \dots & 0\\
    0 & \Xi & \dots & 0 \\
    \dots & \dots & \dots & \dots \\
    0 & 0 & \dots & \Xi 
    \end{array}\right] \left[ \begin{array}{c} \tilde \bmbf_{(11)} \\
    \tilde \bmbf_{(12)} \\
    \dots \\
    \tilde \bmbf_{(qq)}
    \end{array}\right] = \left[ \begin{array}{c} \bm\beta_{(11)} \\
    \bm\beta_{(12)} \\
    \dots \\
    \bm\beta_{(qq)}
    \end{array}\right],
    \label{append:btilde}
\end{equation}

where $\tilde \bmbf_{(ij)} = (\tilde b_{ij}^1,\tilde b_{ij}^2, \cdots, \tilde b_{ij}^{m\choose 2})$ and $\gammambf_{(ij)} = (\beta_{ij}^1, \beta_{ij}^2, \cdots, \beta_{ij}^{m\choose 2})$ are vectorized versions of the $ij$-th entry across groups for $i, j \in [q]$, and $\Xi \in \{0,1\}^{{m\choose 2}\times {m\choose 2}}$ is a binary full-rank matrix denoting membership of groups in the set $\sigma$. For example, when one chooses $\Mcal = \{ \{1,2\}, \{1,3\}, \{2,3\}\}$ for $m=3$, $\Xi$ is given by:
\bequation
\Xi = \left[ \begin{array}{ccc} 0 & 1 & 1\\
    1 & 0 & 1\\
    1 & 1 & 0\\
\end{array}\right].
\eequation
One may choose any set of sets $\Mcal$ that allows the resulting group membership matrix $\Xi$ to be  non-singular. The solution of the system of equations $\Lcal$ is:

\begin{equation}
    \left[ \begin{array}{c} \tilde \bmbf_{(11)} \\
    \tilde \bmbf_{(12)} \\
    \dots \\
    \tilde \bmbf_{(qq)}
    \end{array}\right] = \left[ \begin{array}{cccc} \Xi & 0 & \dots & 0\\
    0 & \Xi & \dots & 0 \\
    \dots & \dots & \dots & \dots \\
    0 & 0 & \dots & \Xi
    \end{array}\right]^{(-1)} \left[ \begin{array}{c} \bm\beta_{(11)} \\
    \bm\beta_{(12)} \\
    \dots \\
    \bm\beta_{(qq)}
    \end{array}\right].
    \label{append:eq:sol-b}
\end{equation}
When all $\tilde \Bmbf^{uv}$'s are normalized, we have the estimated fairness violation weight matrices as:
\begin{equation}
\Bmbfhat^{uv} = \frac{\tilde \Bmbf^{uv}}{\frac{1}{2}\sum_{u,v=1, v > u}^m \Vert \tilde \Bmbf^{uv} \Vert_F} \quad \text{for} \quad u,v \in [m], v>u.
 \label{append:eq:fairbsol}
\end{equation} 
Due to the above normalization, the solution is again independent of the true trade-off $\lambda$.

Given estimates $\hat{B}^{uv}_{ij}$ and $\ahat_1$,  we can now additionally estimate the trade-off parameter $\hat{\lambda}$ from  $\ell^\sigma$~\eqref{append:eq:fairBij} for any $\sigma \in \Mcal$. This completes the fair (quadratic) metric elicitation procedure. 

\subsection{Eliciting Trade-off $\lambda$ when (linear) predictive performance and (quadratic) fairness violation coefficients are known}
\label{append:ssec:lambda}

We  now provide an alternate binary search based method similar to Chapter~\ref{chp:fair} for eliciting the trade-off parameter $\lambda$ when the linear predictive and quadratic fairness coefficients are already known. This is along similar lines to the application considered by Zhang et al.~\cite{zhang2020joint}, but unlike them, instead of ratio queries, we require simpler pairwise queries. 

Here, the key insight is to approximate the non-linearity posed by the fairness violation in Definition~\ref{def:f-linmetric}, which then reduces the problem to a  one-dimensional binary search. We have:
\begin{align*}
&\phi^{\text{fair}}(\tupr \,;\, \ambf, \Bmbb, \lambda) \,\coloneqq\,  (1-\lambda)\inner{\ambf}{\rmbf} + \lambda \frac{1}{2} \left(\sum\nolimits_{u,v=1,v>u}^{m} (\rmbf^u - \rmbf^v)^T\mathbbm{\Bmbf}^{uv}(\rmbf^{u} - \rmbf^v)\right). \numberthis \label{append:eq:fairmetshifted}
\end{align*}
To this end, we define a new sphere $\Scal' = \{ (\smbf,\ombf, \dots, \ombf)  | \smbf \in \overline{\Scal}\}$. The elements in $\Scal'$ is the set of rate profiles whose first group achieves rates $\smbf \in \overline{\Scal}$ and rest of the groups achieve trivial rate $\ombf$ (corresponding to uniform random classifier). For any element in $\Scal'$, the associated discrepancy terms $(\rmbf^u - \rmbf^v) = 0$ for $u,v \neq 1$. 
Thus for elements in $\Scal'$, the metric in Definition~\ref{def:f-linmetric} reduces to:
\vspace{-0.2cm}
\begin{align*}
        \phi^{\text{fair}}((\smbf, \ombf, \dots, \ombf) \,;\,\ambf, \Bmbb, \lambda) =& (1-\lambda)\inner{\taumbf^1\odot\ambf}{\smbf - \ombf} + 
        \lambda\frac{1}{2} (\smbf - \ombf)^T\sum_{v=2}^m \Bmbf^{1v} (\smbf - \ombf) + c.
         \numberthis \label{append:eq:metriclambda}
\end{align*}
\vskip -0.3cm
Additionally, we consider a small sphere $\overline{\Scal}'_{\zmbf_1}$, where $\zmbf_1 \coloneqq (\rho - \varrho)\bm{\alpha}_1 + \ombf$, similar to what is shown in Figure~\ref{fig:MEgeom}(a). We may approximate the quadratic term on the right hand side above by its first order Taylor approximation as follows:

\begin{align*}
        \phi^{\text{fair}}( (\smbf, \ombf, \dots, \ombf) ;\ambf, \Bmbb, \lambda) &\approx  \phi^{\text{fair, apx}}( (\smbf, \ombf, \dots, \ombf) ;\ambf, \Bmbb, \lambda) \\ &= \inner{(1-\lambda)\taumbf^1\odot\ambf + \lambda \sum_{v=2}^m \Bmbf^{1v}(\zmbf_1 - \ombf)}{\smbf}
         \numberthis \label{eq:metriclambdalinear}
\end{align*}

\noindent for $\smbf$ in a small neighbourhood around the rate profile $\zmbf_1$. Since the metric is essentially linear in $\smbf$, the following lemma from Chapter~\ref{chp:fair} shows that the metric in~\eqref{eq:metriclambdalinear} is quasiconcave in $\lambda$. 

\blemma
Under the regularity assumption that \bequation
\inner{\taumbf^1\odot\ambf}{\sum_{v=2}^m \Bmbf^{1v}(\zmbf_1 - \ombf)}\neq 1,
\eequation
the function
\begin{equation}
\vartheta(\lambda) \coloneqq \max_{\smbf \in \overline{\Scal}'_{\zmbf_1}} \phi^{\text{fair, apx}}( (\smbf, \ombf, \dots, \ombf) ;\ambf, \Bmbb, \lambda)
\label{append:eq:vartheta}
\end{equation}
is strictly quasiconcave (and therefore unimodal) in $\lambda$.
\label{lm:lambda}
\elemma
The unimodality of $\vartheta(\lambda)$ allows us to perform the one-dimensional binary search in Algorithm~\ref{alg:lambda} using the query space $\overline{\Scal}'_{\zmbf_1}$, tolerance $\epsilon$, and the oracle $\Omega$. The binary search algorithm is same as Algorithm~\ref{fair-alg:lambda} and provided here for completeness. 

\balgorithm[t]
\caption{Eliciting the trade-off $\lambda$ when predictive performance and fairness violation are known}
\label{alg:lambda}
\small
\balgorithmic[1]
\STATE \textbf{Input:} Query space $\overline{\Scal}'_{\zmbf_1}$, binary-search tolerance $\epsilon > 0$, oracle $\Omega$
\STATE \textbf{Initialize:} $\lambda^{(a)} = 0$, $\lambda^{(b)} = 1$.
\WHILE{$\abs{\lambda^{(b)} - \lambda^{(a)}} > \epsilon$} 
\STATE Set $\lambda^{(c)} = \frac{3 \lambda^{(a)} + \lambda^{(b)}}{4}$, $\lambda^{(d)} = \frac{\lambda^{(a)} + \lambda^{(b)}}{2}$, $\lambda^{(e)} = \frac{\lambda^{(a)} + 3 \lambda^{(b)}}{4}$
\STATE Set $\smbf^{(a)} = \displaystyle\argmax_{\smbf \in\overline{\Scal}'_{\zmbf_1}} \inner{(1-\lambda^{(a)})\taumbf^1\odot\ambfhat + \lambda^{(a)} \sum_{v=2}^m \Bmbfhat^{1v}(\zmbf_1 - \ombf)}{\smbf}$ using Lemma~\ref{mult-lem:spherebayes}
\STATE Similarly, set $\smbf^{(c)}$, $\smbf^{(d)}$, $\smbf^{(e)}$, $\smbf^{(b)}$.
\STATE Query  $\Omega(\smbf^{(c)}, \smbf^{(a)})$,  $\Omega(\smbf^{(d)}, \smbf^{(c)})$,  $\Omega(\smbf^{(e)}, \smbf^{(d)})$, and  $\Omega(\smbf^{(b)}, \smbf^{(e)})$.\\
\STATE $[\lambda^{(a)}, \lambda^{(b)}] \leftarrow$ \emph{ShrinkInterval} (responses) -- subroutine analogous to the routine in  Fig.~\ref{mult-append:fig:shrink1}.
\ENDWHILE
\STATE \textbf{Output:} $\hat\lambda = \frac{\lambda^{(a)}+\lambda^{(b)}}{2}$. 
\ealgorithmic
\ealgorithm
\vspace{-0.3cm}

\section{Elicitation Guarantee for the QPME Procedure}
\label{append:sec:guarantees}
\vskip -0.2cm
\subsection{Sample complexity bounds} Recall from Definition~\ref{def:noise} that the oracle responds correctly as long as $|\phi(\rmbf_1) - \phi(\rmbf_2)| > \epsilon_\Omega$. For simplicity, we assume that our algorithm 
has  access to the population rates $\rmbf$ defined in Eq.~(1). 
In practice, we expect  to estimate the rates using a sample $D\coloneqq \{\xmbf, y\}_{i=1}^n$ drawn from the distribution $\Pmbb$, and to query classifiers from a hypothesis class $\mathcal{H}$ with finite capacity. Standard generalization bounds (e.g.\ Daniely et al.~\cite{Daniely:2015}) give us that with high probability over draw of $D$, the estimates  $\hat{\rmbf}$ are close to the population rates $\rmbf$, up to the desired  tolerance $\epsilon_\Omega$, 
as long as we have sufficient samples. Further, since the metrics $\phi$ are Lipschitz w.r.t.\ rates, with high probability, 
we thus gather correct oracle feedback from querying with finite sample estimates $\Omega(\hat{\rmbf}_1, \hat{\rmbf}_2)$.

More formally, for $\delta \in (0,1)$, as long as the  sample size $n$ is greater than ${O\big(\sfrac{\log(|\mathcal{H}|/\delta)}{\epsilon_\Omega^2}\big)}$, the guarantee in Theorem 1 holds with probability at least $1 - \delta$ (over draw of $D$), where $|\mathcal{H}|$ can in turn be replaced by a measure of capacity of the hypothesis class $\mathcal{H}$. For example, one can show the following corollary to Theorem \ref{thm:q-me} for a
hypothesis class $\mathcal{H}$ in which each classifier is a randomized combination of a finite number of deterministic classifiers chosen from $\bar{\mathcal{H}}$, and whose capacity is measured in terms of the Natarajan dimension~\cite{Natarajan:1989} of $\bar{\mathcal{H}}$.
\begin{corollary}
Suppose the hypothesis class $\mathcal{H}$ of randomized classifiers used to choose queries to the oracle is  of the form:
\bequation
\mathcal{H} =\bigg\{x \mapsto \sum_{t=1}^T\alpha_t h_t(x) \,\bigg|\, T \in \mathbb{Z}_+, \alpha \in \Delta_T, h_1, \ldots, h_T \in \bar{\mathcal{H}}\bigg\},
\eequation
for some class $\bar{\mathcal{H}}$ of deterministic multiclass classifiers $h: \mathcal{X} \rightarrow \{0,1\}^k$. Suppose the deterministic hypothesis class $\bar{\mathcal{H}}$  has   Natarajan dimension $d > 0$, and $\phi$ is $1$-Lipschitz. Then for any $\delta \in (0,1)$,
as long as the  sample size $n 
\geq O\Big(\frac{d\log(k) + \log(1/\delta)}{\epsilon_\Omega^2}\Big)$, the guarantee in Theorem 1 hold with probability at least $1 - \delta$ (over draw of $D = \{\xmbf_i, y_i\}_{i=1}^n$ from $\Pmbb$). 
\end{corollary}
The proof adapts generalization bounds from Daniely et al.~\cite{Daniely:2015}, and uses the fact that the predictive rate for any randomized classifier in $\mathcal{H}$ is a convex combination of rates for deterministic classifiers in $\bar{\mathcal{H}}$ (due to linearity of expectation). 

\vspace{-0.5cm}
\subsection{Proofs}
Before presenting the proof of Theorem \ref{thm:q-me}, we re-write the LPME guarantees from~\cite{hiranandani2019multiclass} for linear metrics in the presence of an oracle noise parameter $\epsilon_\Omega$ from Definition~\ref{def:noise}. 

\blemma[LPME guarantees with oracle noise (Chapter~\ref{chp:multiclass})]
\label{lem:LPMEwnoise}
Let the oracle $\Omega$'s metric be $\phi^{\text{lin}} = \inner{\ambf}{\rmbf}$ and its feedback noise parameter from Definition~\ref{def:noise} be $\epsilon_\Omega$. Then, if the LPME procedure (Algorithm~\ref{mult-alg:lpm}) is run using a 
sphere $\Scal \subset \Rcal$ of radius $\varrho$ and the  binary-search tolerance $\epsilon$, then by posing $O(q\log(1/\epsilon))$
queries it recovers coefficients $\ambfhat$ with $\Vert \ambf - \ambfhat \Vert_2 \leq O\left(\sqrt{q}(\epsilon + \sqrt{\epsilon_\Omega/\varrho})\right)$.
\elemma

We will use the above result while proving Theorem~\ref{thm:q-me}. 

\bproof[Proof of Theorem~\ref{thm:q-me}] 

We first find the smoothness coefficient of the metric in Definition~\ref{def:quadmet}.

A function $\phi$ is said to be $L$-smooth if for some bounded constant $L$, we have:

\vspace{-0.1cm}
\bequation
\Vert \nabla \phi(x) - \nabla \phi(y) \Vert_2 \leq L\Vert x - y \Vert_2.
\eequation

For the metric in Definition~\ref{def:quadmet}, we have:
\vspace{-0.2cm}
\begin{align*}
\Vert \nabla \phi^{\text{quad}}(x) - \nabla \phi^{\text{quad}}(y) \Vert_2 &= \Vert \ambf + \Bmbf\xmbf - (\ambf + \Bmbf\ymbf) \Vert_2 \\
&\leq \Vert \Bmbf \Vert_2 \Vert x - y \Vert_2\\
&\leq \Vert \Bmbf \Vert_F \Vert x - y \Vert_2 \leq 1\cdot\Vert x - y \Vert_2, \numberthis 
\end{align*}
where in the last step, we have used the scale invariance condition from Definition~\ref{def:quadmet}, i.e., $\Vert \ambf \Vert_2 + \Vert \Bmbf \Vert_F = 1$, which implies that  $\Vert \Bmbf \Vert_F = 1 - \Vert \ambf \Vert_2 \leq 1$. 
Hence, the metrics in Definition~\ref{def:quadmet} are $1$-smooth. 

Now, we look at the error in Taylor series approximation when we approximate the metric $\phi^{\text{quad}}$ in  Definition~\ref{eq:quadmet} with a linear approximation. Our metric is 

\vspace{-0.1cm}
\bequation
\phi^{\text{quad}}(\rmbf) = \inner{\ambf}{\rmbf} + \frac{1}{2}\rmbf^T\Bmbf\rmbf.
\eequation

We approximate it with the first order Taylor polynomial around a point $\zmbf$:

\bequation
T_1(\rmbf) = \inner{\ambf}{\zmbf} + \frac{1}{2}\zmbf^T\Bmbf\zmbf + \inner{\ambf + \Bmbf\zmbf}{\rmbf}
\eequation
The bound on the error  in this approximation is:
\begin{align*}
\vert E(\rmbf) \vert &= \vert \phi^{\text{quad}}(\rmbf) - T_1(\rmbf) \vert   \\
&= \frac{1}{2} \vert (\rmbf -\zmbf)^T \Delta\phi^{\text{quad}}|_\cmbf  (\rmbf -\zmbf) \vert \quad \text{(First-order Taylor approximation error)} \\
&= \frac{1}{2} \vert (\rmbf -\zmbf)^T \Bmbf  (\rmbf -\zmbf) \vert \qquad \quad \;\;\; \text{(Hessian at any point $\cmbf$ is the matrix $\Bmbf$)}\\
&\leq \frac{1}{2}\Vert \Bmbf \Vert_2 \Vert \rmbf - \zmbf \Vert _2^2 \\
&\leq \frac{1}{2}\Vert \Bmbf \Vert_F \varrho^2 \leq  \frac{1}{2} \varrho^2 \qquad\qquad \;\;\;\; \text{(Due to the scale invariance condition)} \numberthis
\end{align*}

So when the oracle is asked $\Omega(\rmbf_1, \rmbf_2) = \1[\phi^{\text{quad}}(\rmbf_1) > \phi^{\text{quad}}(\rmbf_2)]$, the approximation error can be treated as feedback error from the oracle with feedback noise 
$2\times \frac{1}{2} \varrho^2$. 
Thus, the overall feedback noise by the oracle is 
$\epsilon_\Omega + \varrho^2$ for the purposes of using Lemma~\ref{lem:LPMEwnoise} later. 

We first prove guarantees for the matrix $\Bmbf$ and then for the vector $\ambf$. We write Equation~\eqref{eq:poly2elicitamatfinal} in the following form assuming $d_1 = 1$ (since we normalize the coefficients at the end due to scale invariance): 

\begin{align*}
B_{ij} &= F_{ij} =  \left[\frac{f_{ij}}{f_{1j}}\left(1 + \frac{f_{j1}}{f_{11}} \right) - \frac{f_{ij}}{f_{1j}}\frac{f_{j0}}{f_{10}} - \frac{f_{i0}}{f_{10}} + \frac{f_{ij}}{f_{1j}}\frac{f_{j1}}{f_{11}} \frac{ \frac{f_{21}^-}{f_{11}^-} + \frac{f_{21}}{f_{11}} - 2\frac{f_{20}}{f_{10}}  }{ \frac{f_{21}^-}{f_{11}^-} - \frac{f_{21}}{f_{11}}  }\right]. \\
\implies \Bmbf[:, j] &= \fmbf_j\left( \frac{1}{f_{1j}} + \frac{f_{j1}}{f_{1j}f_{11}} + \frac{f_{j0}}{f_{1j}f_{10}} + \frac{f_{j1}}{f_{1j}f_{11}}\left(  \frac{ \frac{f_{21}^-}{f_{11}^-} + \frac{f_{21}}{f_{11}} - 2\frac{f_{20}}{f_{10}}  }{ \frac{f_{21}^-}{f_{11}^-} - \frac{f_{21}}{f_{11}}  } \right) \right) + \fmbf_0\frac{1}{f_{10}} \\
&= c_j\fmbf_j + c_0\fmbf_0, \numberthis \label{eq:Bj}
\end{align*} 
where $\Bmbf[:, j]$ is the $j$-th column of the matrix $\Bmbf$, and the constants $c_j$ and $c_0$ are well-defined due to the regularity Assumption~\ref{as:regularity-q}. Notice that,
\bequation
\frac{\partial \Bmbf[:, j]}{\partial \fmbf_j} = \diag(\cmbf'_j)\odot\Imbf \quad, \text{and} \quad 
\frac{\partial \Bmbf[:, j]}{\partial \fmbf_0} = \diag(\cmbf'_0)\odot\Imbf,
\eequation
where $\cmbf'_j, \cmbf'_0$ are vector of Lipschitz constants (bounded due to Assumption~\ref{as:regularity-q}). This implies

\begin{align*}
\Vert \Bmbfbar[:, j] - \Bmbfhat[:, j]\Vert_2 &\leq c'_j \Vert \fmbfbar_j - \fmbfhat_j \Vert_2 + c'_0 \Vert \fmbfbar_0 - \fmbfhat_0 \Vert_2\\
&\leq c'_j\sqrt{q}\left(\epsilon + \sqrt{\varrho + \epsilon_\Omega/\varrho}\right) + c'_0\sqrt{q}\left(\epsilon + \sqrt{\varrho + \epsilon_\Omega/\varrho}\right) \\
&= O\left(\sqrt{q}\left(\epsilon + \sqrt{\varrho + \epsilon_\Omega/\varrho}\right)\right), \numberthis 
\end{align*}
where we have used LPME guarantees from Lemma~\ref{lem:LPMEwnoise} under the oracle-feedback noise parameter $\epsilon_\Omega + \varrho^2$. 

The above inequality provides bounds on each column of $\Bmbf$. Since $\Vert \xmbf \Vert_\infty \leq \Vert \xmbf \Vert_2$, we have $\max_{ij}\vert B_{ij} - \hat{B}_{ij} \vert \leq O\left(\sqrt{q}\left(\epsilon + \sqrt{\varrho + \epsilon_\Omega/\varrho}\right)\right)$, and consequentially, $\Vert \Bmbf - \Bmbfhat \Vert_F \leq O\left(q\sqrt{q}\left(\epsilon + \sqrt{\varrho + \epsilon_\Omega/\varrho}\right)\right)$. 

Now let us look at guarantees for $\ambf$. Since $\ambf = \dmbf - \Bmbf\ombf$ from~\eqref{eq:quadmetshift}, we can write 

\bequation
\ambf = c_0\fmbf_0 - \sum_{j=1}^qo_j\Bmbf[:, j],
\eequation
where $c_0 = 1/f_{10}$. Since $\ombf$ is the rate achieved by random classifier, $o_j = 1/k \; \forall j \in [k]$, and thus we have
\bequation
\frac{\partial \ambf}{\partial \fmbf_0} = c_0\Imbf \quad \text{and} \quad \frac{\partial \ambf}{\partial \Bmbf[:, j]} = \frac{1}{k}\Imbf.
\eequation
Thus,
\begin{align*}
\Vert \ambf - \ambfhat \Vert_2 &\leq c'_0 \sqrt{q}\left(\epsilon + \sqrt{\varrho + \epsilon_\Omega/\varrho}\right) + \frac{1}{k}\sum_{j=1}^q \sqrt{q}\left(\epsilon + \sqrt{\varrho + \epsilon_\Omega/\varrho}\right) \\
& =c'_0 \sqrt{q}\left(\epsilon + \sqrt{\varrho + \epsilon_\Omega/\varrho}\right) + \frac{1}{\sqrt{q}}\sum_{j=1}^q c'_j\sqrt{q}\left(\epsilon + \sqrt{\varrho + \epsilon_\Omega/\varrho}\right) \\
&= O\left(q\left(\epsilon + \sqrt{\varrho + \epsilon_\Omega/\varrho}\right)\right), \numberthis 
\end{align*}
where $c'_0, c'_j$'s are some Lipschitz constants (bounded due to Assumption~\ref{as:regularity-q}), and we have used the fact that $q = k^2 - k$ in the second step.
\eproof

Notice the trade-off in the elicitation error that depends on the size of the sphere. As expected, when the radius of the sphere $\varrho$ increases, the error due to approximation increases, but at the same time, error due to feedback reduces because we get better responses from the oracle. In contrast, when the radius of the sphere $\varrho$ decreases, the error due to approximation decreases, but the error due to feedback increases.

The following corollary translates our guarantees on the elicited metric to the guarantees on the optimal rate of the elicited metric. This is useful in practice, because the optimal classifier (rate) obtained by optimizing a certain metric is often the key entity for many applications. 

\bcorollary
Let $\phi^{\quadr}$ be the oracle's quadratic metric  and $\hat\phi^{\quadr}$ be its estimate obtained by the QPME procedure (Algorithm~\ref{alg:q-me}). Moreover, let $\rmbf^*$ and $\hat\rmbf^*$ be the minimizers of $\phi^{\quadr}$ and $\hat\phi^{\quadr}$, respectively. Then,  $\phi^{\quadr}(\hat\rmbf^*) \leq \phi^{\quadr}(\rmbf^*) + O\left(q^2\sqrt{q}\left(\epsilon + \sqrt{\varrho + \epsilon_\Omega/\varrho}\right)\right).$
\ecorollary

\bproof
We first show that if $\vert \phi^{\quadr}(r) - \hat\phi^{\quadr}(r)\vert \leq \epsilon$ for all rates $r$  and some slack $\epsilon$, then it follows that $\phi^{\quadr}(\hat\rmbf^*) \leq \phi^{\quadr}(\rmbf^*) + 2\epsilon.$ This is because:

\begin{align*}
    \phi^{\quadr}(\hat\rmbf^*) &\leq \hat\phi^{\quadr}(\hat\rmbf^*) + \epsilon \qquad\qquad \left(\text{as $\hat\phi^{\quadr}$ approximates $\phi^{\quadr}$}\right)\\
    &\leq \hat\phi^{\quadr}(\rmbf^*) + \epsilon \qquad\qquad \left(\text{as $\hat\rmbf^*$ minimizes $\hat\phi^{\quadr}$}\right) \\
    &\leq \phi^{\quadr}(\rmbf^*) + 2\epsilon \qquad\quad\;\; \left(\text{as $\hat\phi^{\quadr}$ approximates $\phi^{\quadr}$}\right) \numberthis \label{eq:metapx}
\end{align*}

Now, let us derive the trivial bound $\vert \phi^{\quadr}(r) - \hat\phi^{\quadr}(r)\vert$ for any rate $\rmbf$. 

\begin{align*}
    \vert \phi^{\quadr}(r) - \hat\phi^{\quadr}(r)\vert &= \vert \inner{\ambf - \hat \ambf}{\rmbf} + \frac{1}{2}\rmbf^T (\Bmbf - \hat\Bmbf)\rmbf \vert \\
    &\leq \vert \inner{\ambf - \hat \ambf}{\rmbf} \vert + \frac{1}{2}\vert \rmbf^T (\Bmbf - \hat\Bmbf)\rmbf \vert\\
    &\leq \Vert \ambf - \ambf \Vert_2 \Vert \rmbf\Vert_2 + \frac{1}{2}\Vert \Bmbf - \Bmbf \Vert_2 \Vert \rmbf\Vert_2^2\\
    &\leq \Vert \ambf - \ambf \Vert_2 \sqrt q + \frac{1}{2}\Vert \Bmbf - \Bmbf \Vert_F q\\
    &\leq O\left(q^2\sqrt{q}\left(\epsilon + \sqrt{\varrho + \epsilon_\Omega/\varrho}\right)\right), \numberthis \label{eq:metapx2}
\end{align*}
where in the fourth step, we have used the fact that the rates are bounded in $[0, 1]$; hence $\Vert \rmbf \Vert_2 \leq \sqrt{q}$, and in the fifth step, we have used the guarantees from Theorem~\ref{thm:q-me}. Combining\eqref{eq:metapx} and~\eqref{eq:metapx2} gives us the desired result. 
\eproof

\bproof[Proof of Theorem~\ref{thm:lb}]
For the purpose of this proof, let us replace $\left(\epsilon + \sqrt{\varrho + \epsilon_\Omega/\varrho}\right)$ by some slack $\epsilon$. Theorem 1 guarantees that after running the QPME procedure for $O(q^2\log(1/\epsilon)$ queries, we have $\norm {a - \hat a}_2 \leq O(q\epsilon)$ and $\norm {B - \hat B}_F \leq O(q\sqrt q\epsilon).$

If we vectorize the tuple $(\ambf, \Bmbf)$ and denote it by $w$, we have $\norm{w - \hat w}_2 \leq O(q\sqrt q\epsilon)$, where both $\Vert w\Vert_2, \Vert \hat w\Vert_2=1$, due to the scale invariance condition from Definition~\ref{def:quadmet}. Note that $w$ is $\frac{q^2 + 3q}{2}$-dimensional vector and defines the scale-invariant quadratic metric elicitation problem. 
Now, we have to count the minimum number of $\hat w$ that are possible such that $\norm{w - \hat w}_2 \leq O(q\sqrt q\epsilon)$.

This translates to finding the covering number of a ball in $\Vert \cdot \Vert_2$ norm with radius 1, where the covering balls have radius $q\sqrt q\epsilon$. Let us denote the cover by $\{u_i\}_{i=1}^N$ and the ball with radius 1 as $\Bmbb$. We then have:

\begin{align*}
Vol(\Bmbb) &= \leq \sum_{i=1}^N Vol(q\sqrt q\epsilon \Bmbb + u_i) \\
&= NVol(q\sqrt q\epsilon \Bmbb) \\
&= (q\sqrt q\epsilon)^{\frac{q^2 + 3q}{2} - 1}. \numberthis
\end{align*}
\vskip -0.1cm
Thus the number of $\hat w$ that are possible are at least 
\vspace{-0.2cm}
\bequation
c\left(\frac{1}{q\sqrt q\epsilon}\right)^{\frac{q^2 + 3q}{2} - 1} \leq N,
\eequation
where $c$ is a constant. Since each pairwise comparison provides at most one bit, at least $O(q^2)\log(\frac{1}{q\sqrt q\epsilon})$ bits are required to get a possible $\hat w$. We require $O(q^2)\log(\frac{1}{\epsilon})$ queries, which is near-optimal barring log terms. 
\eproof

%% file: blackbox/supplement.tex
\textbf{Notation:} For an index $j \in [k]$, $\onehot(j) \in \{0,1\}^k$ denotes a one-hot encoding of $j$, and for a classifier $h: \X \> [k]$, $\tilde{h} = \onehot(h)$ denotes the same classifier with one-hot outputs, i.e.\ $\tilde{h}(x) = \onehot(h(x))$.

\vspace{-0.5cm}
\section{Extension to General Linear Metrics}
\label{app:linear-gen}
We describe how our proposal extends to black-box metrics $\perf^\Dtrue[h] = \psi(\C[h])$ defined by a function $\psi:[0,1]^{k\times k}\>\R_+$ of \textit{all} confusion matrix entries. 
This handles, for example, the label noise models in Table \ref{tab:correction-weights} with a general (non-diagonal) noise transition matrix $\T$. We begin with metrics that are linear functions of the diagonal and off-diagonal confusion matrix entries $\perf^\Dtrue[h] = \sum_{ij} \beta_{ij}C_{ij}[h]$ for some $\bbeta \in \R^{k\times k}$. 
In this case, we will use an example weighting function $\W: \X \> \R_+^{k\times k}$ that maps an instance $x$ to an $k\times k$ weight matrix $\W(x)$, where $W_{ij}(x) \in \R_+^{k\times k}$ is the weight associated with the $(i,j)$-th confusion matrix entry.

\textbf{\textit{Note that in practice, the metric $\perf^\Dtrue$ may depend on only a subset of $d$ entries of the confusion matrix, in which case, the weighting function only needs to weight those entries. Consequently, the weighting function can be parameterized with  $Ld$ parameters, which can then be estimated by solving a system of $Ld$ linear equations. 
For the sake of completeness, here we describe our approach for  metrics that depend on all $k^2$ confusion entries.}}

\textbf{Modeling weighting function:} Like in \eqref{eq:weighting}, we propose modeling this function as a weighted sum of $L$ basis functions:
\bequation
W_{ij}(x) \,=\, \sum_{\ell=1}^L \alpha^{\ell}_{ij}\phi^\ell(x),
\eequation
where each $\phi^\ell:\X\>[0,1]$ and $\alpha^\ell_{ij} \in \R$. Similar to \eqref{eq:example-weights}, our goal is to then estimate coefficients $\balpha$ so that:
\begin{equation}
\E_{(x, y) \sim \Dshift}\Big[\sum_{ij} W_{ij}(x)\,\1(y = i)h_j(x)\Big] \,\approx\, 
\perf^\Dtrue[h],  \forall h.
\label{eq:example-weights-off-diag}
\end{equation}
Expanding the weighting function in \eqref{eq:example-weights-off-diag}, we get:
\vspace{-0.25cm}
\begin{equation}
\sum_{\ell=1}^L\sum_{i,j}
\alpha^{\ell}_{ij}\,
\underbrace{
\E_{(x, y) \sim \Dshift}\big[ \phi^\ell(x)\,\1(y=i)h_j(x) \big]}_{\Phi_{i,j}^{\Dshift, \ell}[h]} \,\approx\, \perf^\Dtrue[h],  \forall h,
\end{equation}

\noindent which can be re-written as:
\begin{equation}
\sum_{\ell=1}^L\sum_{i,j}
\alpha^{\ell}_{ij}
\Phi^{\Dshift, \ell}_{ij}[h] \,\approx\, 
\perf^D[h], \forall h.
\label{append:eq:example-weights-reduced}
\end{equation}

\textbf{Estimating coefficients $\balpha$:} To estimate $\balpha \in \R^{Lk^2}$, our proposal is to probe the metric $\perf^\Dtrue$ at $Lk^2$ different classifiers $h^{\ell,1,1}, \ldots, h^{\ell,k,k}$, with one classifier for each combination $(\ell,i,j)$ of basis functions and confusion matrix entries, and to solve the following  system of $Lk^2$ linear equations:
\begin{align}
\sum_{\ell,i,j}\alpha^{\ell}_{ij}\,
\hat{\Phi}^{\tr, \ell}_{ij}[h^{1,1,1}]&=
\hat{\perf}^\val[h^{1,1,1}]
\nonumber
\\
&\vdots\label{append:eq:system-of-equations-emp}\\
\sum_{\ell,i,j}\alpha^{\ell}_{ij}\,
\hat{\Phi}^{\tr, \ell}_{ij}[h^{L,m,m}]&=
\hat{\perf}^\val[h^{L,k,k}]\nonumber
\end{align}
Here $\hat{\Phi}^{\tr,\ell}_{ij}[h]$ is an estimate of ${\Phi}^{\Dshift,\ell}_{ij}[h]$ using training sample $S^\tr$ and $\hat{\perf}^\val[h]$ is an estimate of $\perf^\Dtrue[h]$ using the validation sample $S^\val$. Equivalently, defining $\hat{\bSigma} \in \R^{Lk^2 \times Lk^2}$ and $\hat{\bcE} \in \R^{Lk^2}$ with each:
\bequation
\hat{\Sigma}_{(\ell,i,j), (\ell',i',j')} = \hat{\Phi}^{\tr, \ell'}_{i'j'}[h^{\ell,i,j}];~~~
\hat{\perf}_{(\ell,i,j)} = \hat{\perf}^\val[h^{\ell,i,j}],
\eequation
we compute $\hat{\balpha} = \hat{\bSigma}^{-1}\hat{\bcE}$.

\textbf{Choosing probing classifiers:} As described in Section \ref{subsec:probing-classifier}, we propose picking each probing classifier $h^{\ell,i,j}$ so that the $(\ell,i,j)$-th diagonal entry  of $\hat{\bSigma}$ is large and the off-diagonal entries are all small. This can be framed as the following constrained satisfaction problem:
\begin{center}
For $h^{\ell,i,j}$ pick $h \in \H$ such that:
\begin{align*}
\hat{\Phi}^{\tr,\ell}_{i,j}[h] \geq \gamma,~\text{and}~
\hat{\Phi}^{\tr,\ell'}_{i',j'}[h] \leq \omega, \forall (\ell',i',j') \ne (\ell,i,j), \numberthis
\end{align*}
\end{center}
for some $0 < \omega < \gamma < 1$. While the more practical  approach prescribed in Section \ref{subsec:probing-classifier} of constructing the probing classifiers from trivial classifiers that predict the same class on all or a subset of examples does not apply here (because here we need to take into account both the diagonal and off-diagonal confusion entries), the above problem {can be solved} using off-the-shelf tools available for rate-constrained optimization problems \cite{cotter2019optimization}.

\textbf{Plug-in classifier:} Having estimated an example weighting function $\hat{\W}: \X \> \R^{k\times k}$, we seek to maximize a weighted objective on the training distribution:

\bequation
\max_{h}\,\E_{(x, y) \sim \Dshift}\left[\sum_{ij} \hat{W}_{ij}(x)\,\1(y = i)h_j(x)\right],
\eequation
for which we can construct a plug-in classifier that post-shifts a pre-trained class probability model $\hat{\eta}^\tr: \X \> \Delta_k$:
\bequation
\widehat{h}(x) \,\in\, \argmax_{j \in [k]} \sum_{i=1}^k\hat{W}_{ij}(x)\,\hat{\eta}^\tr_i(x).
\eequation
For handling general non-linear metrics $\perf^\Dtrue[h] = \psi(\C[h])$ with a  smooth $\psi:[0,1]^{k\times k}\>\R_+$, we can directly adapt the iterative plug-in procedure in Algorithm \ref{algo:FW}, which would in turn construct a plug-in classifier of the above form in each iteration (line 9). See \cite{narasimhan2015consistent} for more details of the iterative Frank-Wolfe based procedure for optimizing general metrics, where the authors consider non-black-box metrics in the absence of distribution shift.

\section{Proofs}
\subsection{Proof of Theorem \ref{thm:alpha-diagonal-linear-conopt}}
\btheorem[(Restated) \textbf{Error bound on elicited weights}]
Let the input metric be of the form $\hat{\perf}^\lin[h] = \sum_{i}\beta_i \hat{C}^\val_{ii}[h]$ for some (unknown) coefficients $\bbeta \in \R_+^k, \|\bbeta\|\leq 1$.
Let $\perf^\Dtrue[h] = \sum_{i}\beta_i C^\Dtrue_{ii}[h]$. 
Let $\gamma, \omega > 0$ be such that the constraints in \eqref{eq:con-opt}  are feasible for hypothesis class $\bar{\H}$, for all $\ell, i$. Suppose Algorithm \ref{algo:weight-coeff} chooses each
 classifier $h^{\ell,i}$ to  satisfy \eqref{eq:con-opt}, with $\perf^D[h^{\ell,i}] \in [c, 1], \forall \ell, i$, for some $c>0$.
Let $\bar{\alpha}$ be the associated coefficient in Assumption \ref{asp:alpha-star} for metric $\perf^D$. 
Suppose $\gamma > 2\sqrt{2}Lk\omega$ and $n^\tr \geq \frac{L^2k\log(Lk|\H|/\delta)}{(\frac{\gamma}{2} - \sqrt{2}Lk\omega)^2}.$
Fix $\delta\in (0,1)$. Then w.p.\  $\geq 1 - \delta$ over draws of $S^\tr$ and $S^\val$ from $\Dshift$ and $\Dtrue$ resp., 
the coefficients $\hat{\balpha}$ output by Algorithm \ref{algo:weight-coeff} satisfies:
\begin{eqnarray}
{\|\hat{\balpha} - \bar{\balpha}\| \,\leq\,}
\mathcal{O}\Big(
\frac{Lk}{\gamma^2}\sqrt{\frac{L\log(Lk|\H|/\delta)}{n^\tr}} + 
\frac{\sqrt{Lk}}{\gamma} \Big( \sqrt{\frac{L^2k\log(Lk/\delta)}{c^2\gamma^2 n^\val}} +  \nu\Big)\Big), 
\end{eqnarray}
where the term $|\H|$ can be replaced by a measure of capacity of the hypothesis class $\H$.
\etheorem
The solution from Algorithm \ref{algo:weight-coeff} is given by 
 $\hat{\balpha} = \hat{\bSigma}^{-1}\hat{\bcE}$. 
Let $\bar{\balpha}$ be the ``true'' coefficients given in Assumption \ref{asp:alpha-star}. 
Let ${\bSigma} \in \R^{Lk\times Lk}$ denote the population version of  $\hat{\bSigma}$, with
$
\Sigma_{(\ell,i), (\ell',i')} \,=\, \E_{(x,y)\sim \mu}\big[\phi^{\ell'}(x)\1(y=i')h^{\ell,i}_{i'}(x)\big]
$.
Similarly, denote the population version of $\hat{\bcE}$ by:
$
{\perf}_{(\ell, i)} \,=\, \perf^\Dtrue[h^{\ell,i}]
$. 
Let  ${\balpha} = \bSigma^{-1}{\bcE}$ be the solution we obtain had we used the population versions of these quantities.
Further, define 
the vector $\bar{\bcE} \in \R^{Lk}$:

\begin{equation}
\bar{\perf}_{(\ell', i')} = \sum_{\ell,i}
\bar{\alpha}^{\ell}_i
{\Phi}^{\Dshift, \ell}_i[h^{\ell',i'}].
\label{eq:perf-bar}
\end{equation}
It trivially follows that the coefficient $\bar{\balpha}$ given by Assumption \ref{asp:alpha-star} can be written as $\bar{\balpha} = \bSigma^{-1}\bar{\bcE}$. 

We will find the following lemmas useful. 
Our first two lemmas bound the gap between the empirical and population versions of $\bSigma$ (the left-hand side of the linear system) and $\bcE$ (the right-hand side of the linear system).
\blemma[Confidence bound for $\bSigma$]
\label{lem:Sigma-diff}
Fix $\delta \in (0,1)$. With probability at least $1 - \delta$ over draw of $S^\tr$ from $\Dshift$, 
\begin{eqnarray}
|\Sigma_{(\ell,i),(\ell',i')} - \hat{\Sigma}_{(\ell,i),(\ell',i')}| \leq 
\mathcal{O}\left(\sqrt{\frac{p_{\ell,i}\log(Lk|\H|/\delta)}{n^\tr}}\right),
\end{eqnarray}
where $p_{\ell,i} = \E_{(x,y)\sim \mu}[\phi^\ell(x)\1(y=i)]$, 
and consequently,
\begin{equation}
\|\bSigma - \hat{\bSigma}\| \leq \mathcal{O}\left(\sqrt{\frac{L^2k\log(Lk|\H|/\delta)}{n^\tr}}\right).
\end{equation}
\elemma
\begin{proof}
Each row of $\bSigma - \hat{\bSigma}$ contains the difference between the elements $\Phi^{\Dshift,\ell}_{i}[h]$ and $\hat{\Phi}^{\tr,\ell}_{i}[h]$ for a classifier $h$ chosen from $\H$. Using multiplicative Chernoff bounds, we have for a fixed $h$, with probability at least $1 - \delta$ over draw of $S^\tr$ from $\Dshift$
\begin{eqnarray}
|\Phi^{\Dshift,\ell}_i[h] - \hat{\Phi}^{\tr,\ell}_i[h]| \leq 
\mathcal{O}\left(\sqrt{\frac{p_{\ell,i}\log(1/\delta)}{n^\tr}}\right),
\end{eqnarray}
where $p_{\ell,i} = \E_{(x,y)\sim \mu}[\phi^\ell(x)\1(y=i)]$.
Taking a union bound over all $h \in \H$, we have with probability at least $1 - \delta$ over draw of $S^\tr$ from $\Dshift$, for any $h \in \H$:
\begin{eqnarray}
|\Phi^{\Dshift,\ell}_i[h] - \hat{\Phi}^{\tr,\ell}_i[h]| \leq 
\mathcal{O}\left(\sqrt{\frac{p_{\ell,i}\log(|\H|/\delta)}{n^\tr}}\right).
\end{eqnarray}
Taking a union bound over all $Lk \times Lk$ entries,  we have with probability at least $1 - \delta$, for all $(\ell,i),(\ell',i')$:
\begin{eqnarray}
|\Sigma_{(\ell,i),(\ell',i')} - \hat{\Sigma}_{(\ell,i),(\ell',i')}| \leq 
\mathcal{O}\left(\sqrt{\frac{p_{\ell,i}\log(Lk|\H|/\delta)}{n^\tr}}\right)
.
\end{eqnarray}
Upper bounding the operator norm of $\bSigma - \hat{\bSigma}$ with the Frobenius norm, we have
\begin{eqnarray*}
\|\bSigma - \hat{\bSigma}\| &\leq&
\mathcal{O}\left(\sqrt{\frac{\log(Lk|\H|/\delta)}{n^\tr}}\sqrt{\sum_{(\ell,i),(\ell',i')}p_{\ell',i'}}\right)\\&\leq&
\mathcal{O}\left(\sqrt{\frac{\log(Lk|\H|/\delta)}{n^\tr}}\sqrt{\sum_{\ell,i,\ell'}(1)}\right)\,\leq\, \mathcal{O}\left(\sqrt{\frac{L^2k\log(Lk|\H|/\delta)}{n^\tr}}\right), \numberthis 
\end{eqnarray*}
where the second inequality uses the fact that $\sum_{i'}p_{\ell',i'} = \E_{x\sim \P^\mu}\left[\phi^{\ell'}(x)\right]\leq 1$.
\end{proof}

\blemma[Confidence bound for $\bcE$]
Fix $\delta \in (0,1)$. With probability at least $1 - \delta$ over draw of $S^\val$ from $\Dtrue$, 
\bequation
\|\bcE - \hat{\bcE}\| \leq \mathcal{O}\left(\sqrt{\frac{Lk\log(Lk/\delta)}{n^\val}}\right).
\eequation
\label{lem:Perf-diff}
\elemma
\begin{proof}
From an application of Hoeffding's inequality, we have for any fixed $h^{\ell,i}$:
\begin{align*}
|\perf_{(\ell,i)} \,-\, \hat{\perf}_{(\ell,i)}| =
|\perf^\Dtrue[h^{\ell,i}] \,-\, \hat{\perf}^\val[h^{\ell,i}]|
\,&=\, \left|\sum_i\beta_i C^\Dtrue_{ii}[h^{\ell,i}] \,-\, \sum_i\beta_i \hat{C}^\val_{ii}[h^{\ell,i}]\right| \\
&\leq \mathcal{O}\left(\sqrt{\frac{\log(1/\delta)}{n^\val}}\right), \numberthis
\end{align*}
which holds with probability at least $1-\delta$ over draw of $S^\val$ and uses the fact that each $\beta_i$ and $C^D_{ii}[h]$ is bounded. 
Taking a union bound over all $Lk$ probing classifiers, we have:
\bequation
\|\bcE \,-\, \hat{\bcE}\| 
\leq \mathcal{O}\left(\sqrt{Lk}\sqrt{\frac{\log(Lk/\delta)}{n^\val}}\right).
\eequation
Note that we do not need a uniform convergence argument like in Lemma \ref{lem:Sigma-diff} as the probing classifiers are chosen independent of the validation sample.
\end{proof}

Our last two lemmas show that $\bSigma$ is well-conditioned. We first show that because the probing classifiers $h^{\ell,i}$'s are chosen to satisfy \eqref{eq:con-opt}, the  diagonal and off-diagonal entries of $\bSigma$ can be lower and upper bounded respectively as follows.
\blemma[Bounds on diagonal and off-diagonal entries of $\bSigma$]
Fix $\delta \in (0,1)$. 
With probability at least $1 - \delta$ over draw of $S^\tr$ from $\Dshift$, 
\bequation
\Sigma_{(\ell, i), (\ell, i)} \geq \gamma \,-\, \mathcal{O}\left(\sqrt{\frac{p_{\ell,i}\log(Lk|\H|/\delta)}{n^\tr}}\right), \forall (\ell,i)
\eequation
and
\bequation
\Sigma_{(\ell, i), (\ell', i')} \leq \omega \,+\, \mathcal{O}\left(\sqrt{\frac{p_{\ell,i}\log(Lk|\H|/\delta)}{n^\tr}}\right), \forall (\ell,i) \ne (\ell', i'),
\eequation
where $p_{\ell,i} = \E_{(x,y)\sim \mu}[\phi^\ell(x)\1(y=i)]$.
\label{lem:Sigma-inv-concentration}
\elemma
\begin{proof}
Because the probing classifiers $h^{\ell,i}$'s are chosen from $\H$ to satisfy \eqref{eq:con-opt}, we have $\hat{\Sigma}_{(\ell, i), (\ell, i)} \geq \gamma, \forall (\ell,i)$ and 
$\hat{\Sigma}_{(\ell, i), (\ell', i')} \leq \omega, \forall (\ell,i) \ne (\ell', i').$ The proof follows from generalization bounds similar to Lemma \ref{lem:Sigma-diff}.
\end{proof}
The bounds on the diagonal and off-diagonal entries of $\bSigma$ then allow us to bound its smallest and largest singular values.

\blemma[Bounds on singular values of $\bSigma$]
We have $\|\bSigma\| \,\leq\, L\sqrt{k}$.
Fix $\delta \in (0,1)$. Suppose $\gamma > 2\sqrt{2}Lk\omega$ and $n^\tr \geq \frac{L^2k\log(Lk|\H|/\delta)}{(\frac{\gamma}{2} - \sqrt{2}Lk\omega)^2}.$
With probability at least $1 - \delta$ over draw of $S^\tr$ from $\Dshift$, 
$
\|\bSigma^{-1}\|  \,\leq\, \mathcal{O}\left(\frac{1}{\gamma}\right).
$
\label{lem:Sigma-inv}
\elemma
\begin{proof}
We first derive a straight-forward upper bound on the 
the operator norm of $\bSigma$ in terms of its Frobenius norm: $\|\bSigma\| \,\leq\,$
\begin{eqnarray} \sqrt{\sum_{(\ell,i),(\ell',i')}\Sigma^2_{(\ell,i),(\ell',i')}}
\,\leq\, \sqrt{\sum_{(\ell,i),(\ell',i')}p_{\ell',i'}^2}
\,\leq\, \sqrt{\sum_{(\ell,i),(\ell',i')}p_{\ell',i'}}
\,\leq\, \sqrt{\sum_{\ell,i,\ell'}1} \,=\, L\sqrt{k},
\end{eqnarray}
where $p_{\ell,i} = \E_{(x,y)\sim \mu}[\phi^\ell(x)\1(y=i)]$ and the last inequality uses the fact that $\sum_{i'}p_{\ell',i'} = \E_{x\sim \P^\mu}\left[\phi^{\ell'}(x)\right]\leq 1$.

To bound the operator norm of $\|\bSigma^{-1}\|$, denote
 $\upsilon_{\ell, i} = \mathcal{O}\left(\sqrt{\frac{p_{\ell,i}\log(Lk|\H|/\delta)}{n^\tr}}\right)$.  
From Lemma \ref{lem:Sigma-inv-concentration}, we can express $\bSigma$ as a sum of a matrix $\A$ and a diagonal matrix $\D$, i.e.\ $\bSigma = \A + \D$, where each 
$A_{(\ell, i), (\ell, i)} = 0$,
$A_{(\ell, i), (\ell', i')} \leq \omega + \upsilon_{\ell,i}, \forall (\ell, i) \ne (\ell', i')$ and $D_{(\ell, i), (\ell, i)} \geq \gamma - \upsilon_{\ell,i}$. 
Let $\sigma_{\ell,i}(\bSigma)$ denote the $(\ell,i)$-th largest singular value of $\bSigma$.
By Weyl's inequality, we have that the singular values of $\bSigma$ can be bounded
in terms of the singular values $\D$
(see e.g., \cite{stewart1998perturbation}):
\bequation
|\sigma_{\ell,i}(\bSigma) - \sigma_{\ell,i}(\D)| \leq \|\A\|, \quad \text{or} \quad 
\sigma_{\ell,i}(\D) - \sigma_{\ell,i}(\bSigma) \leq \|\A\|.
\eequation
We further have:
\allowdisplaybreaks
\begin{eqnarray*}
\sigma_{\ell,i}(\bD)  - \sigma_{\ell,i}(\bSigma) &\leq& \|\A\| \leq \sqrt{\sum_{(\ell,i) \ne (\ell',i')}(\omega+\upsilon_{\ell,i})^2}
+ \upsilon_{\ell, i}\\
&\leq&
\sqrt{2}\sqrt{\sum_{(\ell,i) \ne (\ell',i')}\omega^2 + \sum_{(\ell,i) \ne (\ell',i')}\upsilon_{\ell,i}^2}
+ \upsilon_{\ell, i}
\\
&\leq&
\sqrt{2}\sqrt{\sum_{(\ell,i) \ne (\ell',i')}\omega^2} + \sqrt{2}\sqrt{\sum_{(\ell,i) \ne (\ell',i')}\upsilon_{\ell,i}^2}\\
&\leq&
\sqrt{2}Lk\omega \,+\, \mathcal{O}\left(\sqrt{\frac{\log(Lk|\H|/\delta)}{n^\tr}}\right)\sqrt{\sum_{(\ell,i) \ne (\ell',i')}p_{\ell,i}}\\
&\leq& 
\sqrt{2}Lk\omega \,+\, \mathcal{O}\left(\sqrt{\frac{L^2k\log(Lk|\H|/\delta)}{n^\tr}}\right). \numberthis
\end{eqnarray*}
Since $\sigma_{\ell,i}(\D) \geq \gamma - \max_{\ell, i}\upsilon_{\ell, i}$, and 
\bequation
\sigma_{\ell,i}(\bSigma) \,\geq\,
\gamma \,-\, \sqrt{2}Lk\omega \,-\, \mathcal{O}\left(\sqrt{\frac{L^2k\log(Lk|\H|/\delta)}{n^\tr}}\right) - \max_{\ell, i}\upsilon_{\ell, i}.
\eequation
Substituting for $\max_{\ell, i}\upsilon_{\ell, i} \leq \mathcal{O}\left(\sqrt{\frac{\log(Lk|\H|/\delta)}{n^\tr}}\right)$, and denoting $\sqrt{2}Lk\omega \,+\, \mathcal{O}\left(\sqrt{\frac{L^2k\log(Lk|\H|/\delta)}{n^\tr}}\right)$ by $\xi$, we have $\sigma_{\ell,i}(\bSigma) \,\geq\, \xi$.
With this, we can bound operator norm of $\|\bSigma^{-1}\|$ as:
\bequation
\|\bSigma^{-1}\| = \frac{1}{\min_{\ell,i} \sigma_{\ell,i}(\bSigma)} \,\leq\,
\frac{1}{\gamma - \xi}\,\leq\,
\mathcal{O}\left(\frac{1}{\gamma}\right),
\eequation
where the last inequality follows from 
the assumption that $n^\tr \geq \frac{L^2k\log(Lk|\H|/\delta)}{(\frac{\gamma}{2} - \sqrt{2}Lk\omega)^2}$ and hence
$\xi \leq \mathcal{O}\left(\gamma/2\right)$. 
\end{proof}

We are now ready to prove Theorem \ref{thm:alpha-diagonal-linear-conopt}.
\begin{proof}[Proof of Theorem \ref{thm:alpha-diagonal-linear-conopt}]
\allowdisplaybreaks
The solution from Algorithm \ref{algo:weight-coeff} is given by 
 $\hat{\balpha} = \hat{\bSigma}^{-1}\hat{\bcE}$. 
Recall we can write the ``true'' coefficients by
$\bar{\balpha} = \bSigma^{-1}\bar{\bcE}$, 
where $\bar{\bcE}$ is defined in \eqref{eq:perf-bar},
and we also defined $\balpha = \bSigma^{-1}{\bcE}$.
The left-hand side of Theorem \ref{thm:alpha-diagonal-linear-conopt} can then be expanded as:
\begin{eqnarray}
    \|\hat{\balpha} - \bar{\balpha}\| &\leq&
    \|\hat{\balpha} - {\balpha}\| + \|{\balpha} - \bar{\balpha}\|\\
    &\leq&
    \|\hat{\balpha} - {\balpha}\| + \|\bSigma^{-1}({\bcE} - \bar{\bcE})\|\\
    &\leq&
    \|\hat{\balpha} - {\balpha}\| + \|\bSigma^{-1}\|\|({\bcE} - \bar{\bcE})\|\\
    &\leq&
    \|\hat{\balpha} - {\balpha}\| + \nu\sqrt{Lk} \|\bSigma^{-1}\|\label{eq:lhs-penultimate} \\
    &\leq& \|\hat{\balpha} - {\balpha}\| \,+\, \frac{2\nu\sqrt{Lk}}{\gamma}.
    \label{eq:lhs}
\end{eqnarray}
The second-last step follows from Assump. \ref{asp:alpha-star}, particularly, from
$\left|\sum_{\ell,i}
\bar{\alpha}^{\ell}_i
\Phi^{\Dshift, \ell}_i[h] -
\perf^\Dtrue[h]\right| \,\leq\, \nu, \forall h$, which gives us that
$\left|\sum_{\ell,i}
\bar{\alpha}^{\ell}_i
\Phi^{\Dshift, \ell}_i[h^{\ell',i'}] -
\perf^\Dtrue[h^{\ell',i'}]\right| \,\leq\, \nu$, for all $\ell', i'$. 
The last step follows from Lemma \ref{lem:Sigma-inv} and holds with probability at least $1-\delta$ over draw of $S^\tr$.

All that remains is to bound the term $\|\hat{\balpha} - {\balpha}\|$. 
Given that 
$\hat{\balpha} = \hat{\bSigma}^{-1}\hat{\bcE}$. 
and $\balpha = \bSigma^{-1}{\bcE}$, we can
use standard error analysis for linear systems (see e.g.,
\cite{demmel1997applied}) to bound:
\begin{eqnarray*}
{\|\hat{\balpha} - {\balpha}\|}
&\leq&
\|{\balpha}\|\|\bSigma\|\|\bSigma^{-1}\|\left(
\frac{\|\bSigma - \hat{\bSigma}\|}{\|\bSigma\|}
\,+\,
\frac{\|\bcE - \hat{\bcE}\|}{\|\bcE\|}\right)\\
&\leq&
\|\bSigma^{-1}\|^2\|\bcE\|\left(
\|\bSigma - \hat{\bSigma}\|
\,+\,
\|\bSigma\|\frac{\|\bcE - \hat{\bcE}\|}{\|\bcE\|}\right)
~~(\text{from  $\balpha = \bSigma^{-1}\bcE$})\\
&\leq&
\|\bSigma^{-1}\|^2\|\bcE\|\left(
\|\bSigma - \hat{\bSigma}\|
\,+\,
L\sqrt{k}\frac{\|\bcE - \hat{\bcE}\|}{\|\bcE\|}\right)
~(\text{from Lemma \ref{lem:Sigma-inv}})\\
&\leq&
\|\bSigma^{-1}\|^2\sqrt{Lk}\left(
\|\bSigma - \hat{\bSigma}\|
\,+\,
\frac{L\sqrt{k}}{\sqrt{Lk}c}{\|\bcE - \hat{\bcE}\|}\right)
~(\text{using  $\perf_{(\ell,i)} \in (c,1]$})
\\
&\leq&
\|\bSigma^{-1}\|^2\sqrt{Lk}\left(
\|\bSigma - \hat{\bSigma}\|
\,+\,
\frac{\sqrt{L}}{c}{\|\bcE - \hat{\bcE}\|}\right)\\
&\leq&
\mathcal{O}\left(\frac{\sqrt{Lk}}{\gamma^2}\left(\sqrt{\frac{L^2k\log(Lk|\H|/\delta)}{n^\tr}}
\,+\,
\frac{\sqrt{L}}{c}\sqrt{\frac{Lk\log(Lk/\delta)}{n^\val}}\right)\right)\\
&=&
\mathcal{O}\left(\frac{Lk}{\gamma^2}\left(\sqrt{\frac{L\log(Lk|\H|/\delta)}{n^\tr}}
\,+\,
\frac{1}{c}\sqrt{\frac{L\log(Lk/\delta)}{n^\val}}\right)\right), \numberthis 
\end{eqnarray*}
where the last two steps follow from Lemmas \ref{lem:Sigma-diff}--\ref{lem:Perf-diff} and Lemma \ref{lem:Sigma-inv},
and  hold with probability at least $1-\delta$ over draws of $S^\tr$ and $S^\val$.
Plugging this back into \eqref{eq:lhs} completes the proof.
\end{proof}

\subsection{Error Bound for PI-EW}
\label{app:pi-ew}
We will first provide error bound for the PI-EW algorithm, which is a special case of the FW-EG algorithm. When the metric  is linear, we have the following bound on the gap between the metric value achieved by classifier $\hat{h}$ output by Algorithm \ref{algo:linear-metrics}, and the optimal  value. This result will then be useful in proving an error bound for the FW-EG procedure (Algorithm \ref{algo:FW}) in the next section, that essentially focuses on the non-linear metric optimization. 
\blemma[\textbf{Error Bound for PI-EW}]
\label{lem:plugin-linear}
Let the input metric be of the form $\hat{\perf}^\lin[h] = \sum_{i}\beta_i \hat{C}^\val_{ii}[h]$ for some (unknown) coefficients $\bbeta \in \R_+^k, \|\bbeta\|\leq 1$, and denote $\perf^\lin[h] = \sum_{i}\beta_i C^\Dtrue_{ii}[h]$.
Let $\bar{\balpha}$ be the associated weighting coefficient for $\perf^\lin$  in Assumption \ref{asp:alpha-star}, with $\|\bar{\balpha}\|_1 \leq B$ and with slack $\nu$. Fix $\delta>0$.
Suppose w.p. $\geq 1-\delta$ over draw of $S^\tr$ and $S^\val$, the weight elicitation routine in line 2 of Algorithm \ref{algo:linear-metrics}
provides coefficients $\hat{\balpha}$ with
 $\|\hat{\balpha} -\bar{\balpha}\| \leq 
 \kappa(\delta, n^\tr, n^\val)
 $, for some function $\kappa(\cdot) > 0$. Let $B' = B + \sqrt{Lk}\,\kappa(\delta, n^\tr, n^\val).$
 Then with the same probability, the classifier $\hat{h}$ output by Algorithm \ref{algo:linear-metrics} satisfies:
\begin{eqnarray}
{
\max_{h}\perf^\lin[h] - \perf^\lin[\hat{h}]} &\leq
B'\E_x\left[\|\eta^\tr(x)-
\hat{\eta}^\tr(x)\|_1\right]
\,+\,
 2\sqrt{Lk}\,\kappa(\delta, n^\tr, n^\val)
  \,+\, 2\nu,
\end{eqnarray}
where $\eta_i^\tr(x) = \P^\mu(y=i|x)$. Furthermore, when the metric coefficients $\|\bbeta\| \leq Q$, for some $Q > 0$, then 
\begin{eqnarray}
{
\max_{h}\perf^\lin[h] - \perf^\lin[\hat{h}]} &\leq
Q\left(B'\E_x\left[\|\eta^\tr(x)-
\hat{\eta}^\tr(x)\|_1\right]
\,+\,
 2\sqrt{Lk}\,\kappa(\delta, n^\tr, n^\val)\,+\, 2\nu\right).
\end{eqnarray}
\elemma

\begin{proof}
For the proof, we will treat $\widehat{h}$ as a classifier that outputs one-hot labels,
i.e.\ as classifier  $\widehat{h}: \X \> \{0,1\}^k$ with
\begin{equation}
    \widehat{h}(x) \,=\, \onehot\left(\argmax^*_{i \in [k]} \widehat{W}_{i}(x)\hat{\eta}^\tr_i(x)\right),
    \label{eq:h-hat-onehot}
\end{equation}
where $\argmax^*$ breaks ties in favor of the largest class.

Let
$\bar{W}_i(x)  = \sum_{\ell=1}^L \bar{\alpha}^{\ell}_{i}\phi^\ell(x)$ and $\hat{W}_i(x) = \sum_{\ell=1}^L \hat{\alpha}^{\ell}_{i}\phi^\ell(x)$. It is easy to see that 
\begin{equation}
\label{eq:w-bar-w-hat}
|\bar{W}_i(x) - \hat{W}_i(x)| \leq 
\|\bar{\balpha} - \hat{\balpha}\|\sqrt{\sum_{\ell=1}^L \phi^\ell(x)^2} \leq
\sqrt{Lk}\|\bar{\balpha} - \hat{\balpha}\| \leq \sqrt{Lk}\kappa,
\end{equation}
where in the second inequality we use $|\phi^\ell(x)| \leq 1$, and in the last inequality, we have shortened the notation $\kappa(\delta, n^\tr, n^\val)$ to $\kappa$ and for simplicity will avoid mentioning that this holds with high probability.

\allowdisplaybreaks
Further, recall from Assumption \ref{asp:alpha-star} that
\begin{equation*}
|\bar{W}_i(x)| \leq \|\bar{\balpha}\|_1\max_{\ell}|\phi^\ell(x)| \leq B(1) = B \numberthis 
\end{equation*}
and so from \eqref{eq:w-bar-w-hat},
\begin{equation}
|\hat{W}_i(x)| \leq B + \sqrt{Lk}\kappa.
\label{eq:W-bound}
\end{equation}
We also have from Assumption \ref{asp:alpha-star} that
\begin{equation*}
\left|\perf^\lin[h] \,-\,  \E_{(x,y)\sim \mu}\left[\sum_{i=1}^k\bar{W}_i(x)\1(y=i)h_i(x)\right]\right| \leq \nu, \forall h. \numberthis 
\end{equation*}
Equivalently, this can be re-written in terms of the conditional class probabilities $\eta^\tr(x) = \P^\mu(y=1|x)$:
\begin{equation}
   \left|\perf^\lin[h] \,-\,  \E_{x\sim \P^\mu}\left[\sum_{i=1}^k\bar{W}_i(x)\eta^\tr_i(x)h_i(x)\right]\right| \leq \nu, \forall h,
\label{eq:perf-ex-over-P-mu} 
\end{equation}
where $\P^\mu$ denotes the marginal distribution of $\Dshift$ over $\X$.
Denoting $h^* \in \argmax_{h}\,\perf^\lin[h]$, we then have from \eqref{eq:perf-ex-over-P-mu},
\begin{eqnarray*}
\allowdisplaybreaks
\lefteqn{\max_{h}\,\perf^\lin[h] \,-\, \perf^\lin[\hat{h}]}\\
&=&
\sum_{i=1}^k \E_x\left[\bar{W}_{i}(x)\eta^\tr_{i}(x)h^*_i(x))\right]
\,-\, \sum_{i=1}^k \E_x\left[ \bar{W}_{i}(x)\eta^\tr_{i}(x)\hat{h}_i(x)\right] \,+\, 2\nu
\\
&\leq&
\sum_{i=1}^k \E_x\left[\hat{W}_{i}(x)\eta^\tr_{i}(x)h^*_i(x))\right]
\,-\, \sum_{i=1}^k \E_x\left[ \hat{W}_{i}(x)\eta^\tr_{i}(x)\hat{h}_i(x)\right] \,+\, 2\nu \,+\,
2\sqrt{Lk}\kappa\\
&&
\hspace{6cm}\text{(from \eqref{eq:w-bar-w-hat}, $\textstyle\sum_{i=1}^k\eta^\tr_i(x) = 1$ and $h_i(x) \leq 1$)}\\
&\leq&
\sum_{i=1}^k \E_x\left[\hat{W}_{i}(x)\eta^\tr_{i}(x)h^*_i(x))\right]
\,-\, \sum_{i=1}^k\E_x\left[\hat{W}_{i}(x)\hat{\eta}^\tr_{i}(x)h^*_i(x))\right]\\
&&
\hspace{0cm}
\,+\, \sum_{i=1}^k\E_x\left[\hat{W}_{i}(x)\hat{\eta}^\tr_{i}(x)h^*_i(x))\right]
\,-\, \sum_{i=1}^k \E_x\left[ \hat{W}_{i}(x)\eta^\tr_{i}(x)\hat{h}_i(x)\right] \,+\, 2\nu \,+\,
2\sqrt{Lk}\kappa \numberthis
\end{eqnarray*}
From definition of $\widehat{h}$ in \eqref{eq:h-hat-onehot}, we 
have that $\sum_{i=1}^k\hat{W}_{i}(x)\hat{\eta}^\tr_{i}(x)\hat{h}_i(x) \geq \sum_{i=1}^k\hat{W}_{i}(x)\hat{\eta}^\tr_{i}(x)h_i(x),$ for all $h: \X \> \Delta_k$. Therefore,
\begin{eqnarray*}
\lefteqn{\max_{h}\,\perf^\lin[h] \,-\, \perf^\lin[\hat{h}]}\\
&\leq&
\sum_{i=1}^k \E_x\left[\hat{W}_{i}(x)\eta^\tr_{i}(x)h^*_i(x))\right]
\,-\, \sum_{i=1}^k\E_x\left[\hat{W}_{i}(x)\hat{\eta}^\tr_{i}(x)h^*_i(x))\right] \,+\, 2\nu \,+\,
2\sqrt{Lk}\kappa \\
&&
\,+\, \sum_{i=1}^k\E_x\left[\hat{W}_{i}(x)\hat{\eta}^\tr_{i}(x)\hat{h}_i(x))\right]
\,-\, \sum_{i=1}^k \E_x\left[ \hat{W}_{i}(x)\eta^\tr_{i}(x)\hat{h}_i(x)\right] \tag{E.43 cont.}\\
&\leq&
\sum_{i=1}^k \E_x\left[\hat{W}_{i}(x)|\eta^\tr_{i}(x) - \hat{\eta}^\tr_i(x)||h^*_i(x) - \hat{h}_i(x)|\right]\,+\, 2\nu \,+\,
2\sqrt{Lk}\kappa
\\
&\leq&
\E_x\left[\max_i \left(\hat{W}_{i}(x)|h^*_i(x) - \hat{h}_i(x)|\right)\|\eta(x)-
\hat{\eta}(x)\|_1\right]\,+\, 2\nu \,+\,2\sqrt{Lk}\kappa
\\
&\leq&
(B + \sqrt{Lk}\kappa)\,\E_x\left[\|\eta(x)-
\hat{\eta}(x)\|_1\right]\,+\, 2\nu \,+\, 2\sqrt{Lk}\kappa, \numberthis 
\end{eqnarray*}
where the last step follows from \eqref{eq:W-bound} and $|h_i(x) - \hat{h}_i(x)|\leq 1$. This completes the proof. The second part, where $\|\bbeta\| \leq Q$, follows by applying Assumption \ref{asp:alpha-star} to normalized coefficients $\bbeta/\|\bbeta\|$, and scaling the associated slack $\nu$ by $Q$.
\end{proof}

\subsection{Proof of Theorem \ref{thm:iterative-plugin}}
\label{app:proof-FW}
We will make a couple of minor changes to the algorithm to simplify the analysis. Firstly, instead of using the same sample $S^\val$ for both estimating the example weights (through call to \textbf{PI-EW} in line 9) and estimating confusion matrices $\hat{\C}^{\val}$ (in line 10), we split $S^\val$ into two halves, use one half for the first step and the other half for the second step. Using independent samples for the two steps, we will be able to derive straight-forward confidence bounds on the estimated confusion matrices in each case. In our experiments however, we find the algorithm to be effective even when a common sample is used for both steps. Secondly, we modify line 8 to include a shifted version of the metric $\hat{\perf}^\val$, so that later in Appendix \ref{app:complex-unknown} when we handle the case of ``unknown $\psi$'', we can avoid having to keep track of an additive constant in the gradient coefficients.
\begin{algorithm}[t]
\caption{\hspace{-0.075cm}\textbf{:} \textbf{F}rank-\textbf{W}olfe with \textbf{E}licited \textbf{G}radients (\textbf{FW-EG}) for General Diagonal Metrics} \label{algo:FW-modified}
\begin{algorithmic}[1]
\STATE \textbf{Input:} $\hat{\perf}^\val$, Basis functions $\phi^1, \ldots, \phi^L: \X \> [0,1]$, Pre-trained $\hat{\eta}^\tr: \X \> \Delta_k$,
$S^\tr \sim \Dshift$, 
$S^\val \sim \Dtrue$ split into two halves
$S^\val_1$ and $S^\val_2$ of sizes $\lceil n^\val/2\rceil$ and $\lfloor n^\val/2\rfloor$ respectively, $T$, $\epsilon$
\STATE Initialize classifier $h^0$ and $\c^0 = \diag(\widehat{\C}^\val[h^0])$
\STATE \textbf{For} $t =  0$ \textbf{to} $T-1$ \textbf{do}
\STATE ~~~\textbf{if} $\perf^\Dtrue[h] = \psi(C_{11}^D[h],\ldots,C_{kk}^D[h])$ for known $\psi$:
\STATE ~~~~~~~$\bbeta^{t}\,=\, \nabla\psi(\c^{t})$
\STATE ~~~~~~~$\hat{\perf}^\lin[h]\,=\, \sum_i \beta^t_i \hat{C}^\val_{ii}[h],$  evaluated using $S^\val_1$
\STATE ~~~\textbf{else}
\STATE ~~~~~~~$\hat{\perf}^\lin[h]= \hat{\perf}^\val[h]- \hat{\perf}^\val[h^t]$, evaluated using $S^\val_1$ \hspace{2cm}\COMMENT{small $\epsilon$ recommended}
\STATE ~~~$\widehat{f} = \text{\textbf{PI-EW}}(\hat{\perf}^\lin, \phi^1,..., \phi^L, \hat{\eta}^\tr, S^\tr, S^\val_1, h^{t}, \epsilon)$
\STATE ~~~$\tilde{\c} = \diag(\widehat{\C}^\val[\widehat{f}])$, evaluated using $S^\val_2$
\STATE ~~~${h}^{t+1} = \big(1-\frac{2}{t+1}\big) {h}^{t} + \frac{2}{t+1} \onehot(\widehat{f})$
\STATE ~~~${\c}^{t+1} = \big(1-\frac{2}{t+1}\big) {\c}^{t} + \frac{2}{t+1}\tilde{\c}$
\STATE \textbf{End For}
\STATE \textbf{Output:} $\hat{h} = h^T$
\end{algorithmic}
\end{algorithm}

\btheorem[(Restated) \textbf{Error Bound for FW-EG with known $\psi$}]
Let $\perf^\Dtrue[h] = \psi(C^\Dtrue_{11}[h],\ldots, C^\Dtrue_{kk}[h])$ for a \emph{known} concave function $\psi: [0,1]^k \>\R_+$, which is $Q$-Lipschitz, and  $\lambda$-smooth w.r.t.\ the $\ell_1$-norm.  
Let $\hat{\perf}^\val[h]=\psi(\hat{C}^\val_{11}[h],\ldots, \hat{C}^\val_{kk}[h])$. 
Fix $\delta \in (0, 1)$.
Suppose Assumption \ref{asp:alpha-star} holds with slack $\nu$, and for any linear metric $\sum_i\beta_i C^\Dtrue_{ii}[h]$ with $\|\bbeta\|\leq 1$, whose associated weight coefficients is $\bar{\balpha}$ with $\|\bar{\balpha}\| \leq B$, 
 w.p. $\geq 1-\delta$ over draw of $S^\tr$ and $S^\val_1$, the weight elicitation routine in Algorithm \ref{algo:weight-coeff} outputs coefficients $\hat{\balpha}$ with
 $\|\hat{\balpha} -\bar{\balpha}\| \leq 
 \kappa(\delta, n^\tr, n^\val)
 $, for some function $\kappa(\cdot) > 0$. Let $B' = B + \sqrt{Lk}\,\kappa(\delta/T, n^\tr, n^\val).$
Assume $k \leq n^\val$.
 Then w.p.\  $\geq 1 - \delta$ over draws of $S^\tr$ and $S^\val$ from $\Dtrue$ and $\Dshift$ resp., the classifier $\hat{h}$ output by Algorithm \ref{algo:FW-modified} after $T$ iterations satisfies:

\vspace{0.5cm}
\begin{eqnarray*}
\lefteqn{
\hspace{-1cm}\max_{h}\perf^D[h] - \perf^D[\hat{h}]\,\leq\,
2QB'\E_x\left[\|\eta^\tr(x)-
\hat{\eta}^\tr(x)\|_1\right] +
4Q\nu +
 4Q\sqrt{Lk}\,\kappa(\delta/T, n^\tr, n^\val)}\\
 &&
 \hspace{3.5cm}
+\, \mathcal{O}\left(\lambda k\sqrt{\frac{k\log(n^\val)\log(k) + \log(k/\delta)}{n^\val}} + \frac{\lambda}{T}\right). \numberthis 
\end{eqnarray*}
\etheorem
The proof adapts techniques from \cite{narasimhan2015consistent}, who
show guarantees for a Frank-Wolfe based learning algorithm with a known $\psi$ in the \textit{absence} of distribution shift.
The main proof steps are listed below:
\begin{itemize}
    \item Prove a generalization bound for the confusion matrices $\hat{\C}^\val$ evaluated in line 10 on the validation sample (Lemma \ref{lem:C-genbound})
    \item Establish an error bound for the call to \textbf{PI-EW} in line 9 (Lemma \ref{lem:plugin-linear} in previous section)
    \item Combine the above two results to show that the classifier $\hat{f}$ returned in line 9 is an approximate linear maximizer needed by the Frank-Wolfe algorithm (Lemma \ref{lem:lmo})
    \item Combine Lemma \ref{lem:lmo} with a convergence guarantee for the outer Frank-Wolfe algorithm  \cite{narasimhan2015consistent, Jaggi13} (using convexity of the space of confusion matrices $\cC$) to complete the proof (Lemmas \ref{lem:C-convexity}--\ref{lem:fw}).
\end{itemize}
\blemma[Generalization bound for $\C^\Dtrue$]
\label{lem:C-genbound}
Fix $\delta \in (0,1)$. 
Let $\hat{\eta}^\tr: \X \> \Delta_m$ be a fixed class probability estimator. Let $\mathcal{G} = \{h:\X\>[m]\,|\,h(x) \in \argmax_{i\in[m]}\beta_i\hat{\eta}^\tr_i(x) \text{ for some }\bbeta \in \R_+^m\}$ be the set of plug-in classifiers defined with $\hat{\eta}^\tr$. 
Let 
\bequation
\bar{\mathcal{G}} = \{h(x) = \textstyle\sum_{t=1}^T u_t h_t(x)\,|\, T \in \N, h_1,\ldots,h_T \in \mathcal{G}, \mathbf{u} \in \Delta_T\}
\eequation
be the set of all randomized classifiers constructed from a finite number of plug-in classifiers in $\mathcal{G}$. 
Assume $m \leq n^\val$. 
Then with probability at least $1 - \delta$ over draw of $S^\val$ from $\Dtrue$, then for $h\in \bar{\mathcal{G}}$:
\bequation
\|\C^\Dtrue[h] - \hat{\C}^\val[h]\|_\infty \leq \mathcal{O}\left(\sqrt{\frac{m\log(m)\log(n^\val) + \log(m/\delta)}{n^\val}}\right).
\eequation
\elemma
\begin{proof}

The proof follows from standard convergence based generalization arguments, where we bound the capacity of the class of plug-in classifiers $\mathcal{G}$ in terms of its Natarajan dimension \cite{natarajan1989learning, daniely2011multiclass}.
Applying Theorem 21 from \cite{daniely2011multiclass}, we have that the Natarajan dimension  of $ \mathcal{G}$ is at most $d = k\log(k)$. 
Applying the generalization bound in Theorem 13 in \cite{daniely2015multiclass}, along with the assumption that $k\leq n^\val$, we have for any $i \in [k]$,
with probability at least $1 - \delta$ over draw of $S^\val$ from $\Dtrue$,  for any $h\in \mathcal{G}$:
\bequation
|C_{ii}^\Dtrue[h] - \hat{C}_{ii}^\val[h]| \leq \mathcal{O}\left(\sqrt{\frac{k\log(k)\log(n^\val) + \log(1/\delta)}{n^\val}}\right).
\eequation
Further note that for any randomized classifier $\bar{h}(x) = \sum_{t=1}^T u_t h_t(x) \in\bar{\mathcal{G}},$ for some $\mathbf{u} \in \Delta_T$,
\bequation
|C_{ii}^\Dtrue[\bar{h}] - \hat{C}_{ii}^\val[\bar{h}]| \leq
\sum_{t=1}^Tu_t|C_{ii}^\Dtrue[h_t] - \hat{C}_{ii}^\val[h_t]|
\leq
\mathcal{O}\left(\sqrt{\frac{k\log(k)\log(n^\val) + \log(1/\delta)}{n^\val}}\right),\eequation
where the first inequality follows from linearity of expectations. 
Taking a union bound over all diagonal entries $i \in[k]$ completes the proof.
\end{proof}

We next show that the call to \textbf{PI-EW} in line 9 of Algorithm \ref{algo:FW} computes an approximate  maximizer $\hat{f}$ for $\hat{\perf}^\lin$. This is an extension of Lemma 26 in \cite{narasimhan2015consistent}.
\blemma[Approximation error in linear maximizer $\hat{f}$]
\label{lem:lmo}
For each iteration $t$ in Algorithm \ref{algo:FW}, denote $\bar{\c}^t = \diag(\C^D[h^t])$, 
and $\bar{\bbeta}^t = \nabla\psi(\bar{\c}^t)$. Suppose the assumptions in Theorem \ref{thm:iterative-plugin} hold. Let $B' = B + \sqrt{Lk}\,\kappa(\delta, n^\tr, n^\val)$. 
Assume $k \leq n^\val$. Then w.p.\ $\geq 1 - \delta$ over draw of $S^\tr$ and $S^\val$ from $\Dshift$ and $\Dtrue$ resp., for any $t = 1, \ldots, T$, the classifier $\hat{f}$ returned by \emph{\textbf{PI-EW}} in line 9 satisfies:
\begin{eqnarray*}
\lefteqn{
\max_{h}\sum_{i}\bar{\beta}^t_i C^\Dtrue_{ii}[h] \,-\, 
\sum_{i}\bar{\beta}^t_iC^\Dtrue_{ii}[\hat{f}] \,\leq\,
QB'\E_x\left[\|\eta^\tr(x)-
\hat{\eta}^\tr(x)\|_1\right] 
\,+\, 2Q\nu
}\\
 &&
 \hspace{1cm}
 +\,
 2Q\sqrt{Lk}\,\kappa\left(\textstyle\frac{\delta}{T}, n^\tr, n^\val\right)
 \,+\, 
 \mathcal{O}\left(\lambda k\sqrt{\frac{k\log(k)\log(n^\val) + \log(k/\delta)}{n^\val}}\right). \numberthis 
\end{eqnarray*}
\elemma
\begin{proof}
The proof uses Theorem \ref{thm:alpha-diagonal-linear-conopt} to bound the approximation errors in the linear maximizer $\hat{f}$ (coupled with a union bound over $T$ iterations), and  
Lemma \ref{lem:C-genbound} to bound the estimation errors in the confusion matrix $\c^t$ used to compute the gradient $\nabla\psi(\c^t)$. 

Recall from Algorithm \ref{algo:FW} that $\c^t = \diag(\hat{\C}^\val[h^t])$ and $\bbeta^t =\nabla \psi(\c^t)$.
Note that these are approximations to the actual quantities we are interested in $\bar{\c}^t = \diag(\C^D[h^t])$
and $\bar{\bbeta}^t = \nabla\psi(\bar{\c}^t)$, both of which are evaluated using the population confusion matrix.
Also, $\|\bbeta\| = \|\nabla\psi(\c^t)\| \leq Q$ from $Q$-Lipschitzness of $\psi$.

Fix iteration $t$, and let $h^* \in \argmax_{h}\sum_{i}\bar{\beta}^t_i C^\Dtrue_{ii}[h]$ for this particular iteration. Then:
\allowdisplaybreaks
\begin{align*}
\lefteqn{\sum_{i}\bar{\beta}^t_i C^\Dtrue_{ii}[h^*] \,-\, 
\sum_{i}\bar{\beta}^t_iC^\Dtrue_{ii}[\hat{f}]}\\
&= 
\sum_{i}\bar{\beta}^t_i C^\Dtrue_{ii}[h^*] \,-\, 
\sum_{i}{\beta}^t_iC^\Dtrue_{ii}[h^*]
\,+\,
\sum_{i}{\beta}^t_iC^\Dtrue_{ii}[h^*] \,-\,
\sum_{i}{\beta}^t_i C^\Dtrue_{ii}[\hat{f}] \\
& \hspace{1cm}\,+\, 
\sum_{i}{\beta}^t_iC^\Dtrue_{ii}[\hat{f}] \,-\,
\sum_{i}\bar{\beta}^t_iC^\Dtrue_{ii}[\hat{f}] \\
&\leq
\|\bbeta^t - \bar{\bbeta}^t\|_\infty\sum_{i}C_{ii}^D[h^*]\,+\,
 \sum_{i}{\beta}^t_iC^\Dtrue_{ii}[h^*] \,-\,
\sum_{i}{\beta}^t_i C^\Dtrue_{ii}[\hat{f}] \,+\, 
\|\bbeta^t - \bar{\bbeta}^t\|_\infty\sum_{i}C_{ii}^D[\hat{f}]\\
&\leq
\|\bbeta^t - \bar{\bbeta}^t\|_\infty(1)\,+\,
\max_h \sum_{i}{\beta}^t_iC^\Dtrue_{ii}[h] \,-\,
\sum_{i}{\beta}^t_i C^\Dtrue_{ii}[\hat{f}] \,+\, 
\|\bbeta^t - \bar{\bbeta}^t\|_\infty(1)
~~\text{($\textstyle\sum_{i,j}C^D_{ij}[h] = 1$)}
\\
&=
2\|\bbeta^t - \bar{\bbeta}^t\|_\infty\,+\,
\max_h \sum_{i}{\beta}^t_iC^\Dtrue_{ii}[h] \,-\,
\sum_{i}{\beta}^t_i C^\Dtrue_{ii}[\hat{f}]\\
&=
2\|\nabla\psi(\c^t)- \nabla\psi(\bar{\c}^t)\|_\infty\,+\,
\max_h \sum_{i}{\beta}^t_iC^\Dtrue_{ii}[h] \,-\,
\sum_{i}{\beta}^t_i C^\Dtrue_{ii}[\hat{f}]\\
&\leq
2\lambda\|\c^t- \bar{\c}^t\|_1\,+\,
\max_h \sum_{i}{\beta}^t_iC^\Dtrue_{ii}[h] \,-\,
\sum_{i}{\beta}^t_i C^\Dtrue_{ii}[\hat{f}]
~~~~\text{($\psi$ is $\lambda$-smooth w.r.t.\ the $\ell_1$ norm)}\\
&\leq
2\lambda k\|\c^t- \bar{\c}^t\|_\infty\,+\,
\max_h \sum_{i}{\beta}^t_iC^\Dtrue_{ii}[h] \,-\,
\sum_{i}{\beta}^t_i C^\Dtrue_{ii}[\hat{f}]\\
&\leq
\mathcal{O}\left(\lambda k\sqrt{\frac{k\log(k)\log(n^\val) + \log(k/\delta)}{n^\val}}\right)+
QB'\E_x\left[\|\eta^\tr(x)-
\hat{\eta}^\tr(x)\|_1\right] \\
 & \hspace{2cm} + 2Q\sqrt{Lk}\,\kappa(\delta, n^\tr, n^\val) + 2Q\nu,
 \numberthis\label{eq:approx-lin-max-last}
\end{align*}
where $B' = B + \sqrt{Lk}\,\kappa(\delta, n^\tr, n^\val)$. The last step 
holds with probability at least $1-\delta$ over draw of $S^\val$ and $S^\tr$, and follows from  Lemma \ref{lem:C-genbound} and Lemma \ref{lem:plugin-linear} (using  $\|\bbeta^t\| \leq Q$). The first bound on $\|\c^t- \bar{\c}^t\|_\infty = \|\hat{\C}^\val[h^t] - {\C}^\Dtrue[h^t]\|_\infty$ holds for any randomized classifier $h^t$ constructed from a finite number of plug-in classifiers. The second bound on the linear maximization errors holds only for a fixed $t$, and so
we need to take a union bound over all iterations $t = 1, \ldots, T$, to complete the proof. 

Note that
because we use two independent samples $S^\val_1$ and $S^\val_2$ for the two bounds, they each hold with high probability over draws of $S^\val_1$ and $S^\val_2$ respectively, and hence with high probability over draw of $S^\val$.
\end{proof}

Our last two lemmas restate results from \cite{narasimhan2015consistent}. The first shows  convexity of the space of confusion matrices (Proposition 10 from their paper), and the second applies  a result from \cite{Jaggi13} to show convergence of the classical Frank-Wolfe algorithm with approximate linear maximization steps (Theorem 16 in \cite{narasimhan2015consistent}).
\blemma[{Convexity of space of confusion matrices}]
\label{lem:C-convexity}
Let $\cC = \{\diag(\C^D[h])\,|\,h: \X \> \Delta_k\}$ denote the set of all confusion matrices achieved by some randomized classifier $h: \X \> \Delta_k$. Then $\cC$ is convex.
\elemma
\begin{proof}
For any two confusion matrices $\C^1, \C^2 \in \cC$, there exist classifiers $h_1, h_2: \X \> \Delta_k$ such that $\c^1 = \diag(\C^D[h_1])$ and $\c^2 = \diag(\C^D[h_2])$. We need to show that for any $u \in [0,1]$, 
\bequation
u\c^1 + (1-u)\c^2 \in \cC.
\eequation
This is true because the randomized classifier $h(x) = u h_1(x) + (1-u)h_2(x)$ yields a confusion matrix $\diag(\C^D[h]) = u\,\diag(\C^D[h_1]) + (1-u)\diag(\C^D[h_2]) = u\c^1 + (1-u)\c^2 \in \cC$.
\end{proof}

\blemma[{Frank-Wolfe with approximate linear maximization} \cite{narasimhan2015consistent}]
\label{lem:fw}
Let the metric $\perf^\Dtrue[h] = \psi(C^\Dtrue_{11}[h],\ldots, C^\Dtrue_{kk}[h])$ for a  concave function $\psi: [0,1]^k \>\R_+$ that is $\lambda$-smooth w.r.t.\ the $\ell_1$-norm. 
For each iteration $t$, define $\bar{\bbeta}^t = \nabla\psi(\diag(\C^D[h^t]))$.
Suppose line 9 of Algorithm \ref{algo:FW} returns a classifier $\hat{f}$ such that $\max_{h}\sum_{i}\bar{\beta}^t_i C^\Dtrue_{ii}[h] \,-\,
\sum_{i}\bar{\beta}^t_iC^\Dtrue_{ii}[\hat{f}] \leq \Delta,\forall t \in [T]$. Then the  classifier $\hat{h}$ output by Algorithm \ref{algo:FW} after $T$ iterations satisfies:
\begin{eqnarray*}
\max_{h}\perf^D[h] - \perf^D[\hat{h}]\,\leq\,
2\Delta
+\, \frac{8\lambda}{T+2}. \numberthis 
\end{eqnarray*}
\elemma

\begin{proof}[Proof of Theorem \ref{thm:iterative-plugin}]
The proof follows by plugging in the result from Lemma \ref{lem:lmo} into Lemma \ref{lem:fw}.
\end{proof}
\section{Error Bound for Weight Elicitation with Fixed Probing Classifiers}
\label{app:norm-sigma-bound}
We first state a general error bound for Algorithm \ref{algo:weight-coeff} in terms of the singular values of $\bSigma$ for any \textit{fixed} choices for the probing classifiers. We then bound the singular values for the fixed choices in \eqref{eq:trivial-classifiers} under some specific assumptions.
\btheorem[\textbf{Error bound on elicited weights with fixed probing classifiers}]
Let $\perf^\Dtrue[h] = \sum_{i}\beta_i C^\Dtrue_{ii}[h]$ for some (unknown) $\bbeta \in \R^k$, and let 
$\hat{\perf}^\val[h] = \sum_{i}\beta_i \hat{C}^\val_{ii}[h]$. 
Let $\bar{\balpha}$ be the associated coefficient in Assumption \ref{asp:alpha-star} for metric $\perf^D$. 
Fix $\delta\in (0,1)$. Then for any fixed choices of the probing classifiers $h^{\ell,i}$, we have with probability  $\geq 1 - \delta$ over draws of $S^\tr$ and $S^\val$ from $\Dshift$ and $\Dtrue$ resp., 
the $\hat{\balpha}$ output by Algorithm \ref{algo:weight-coeff} satisfies: ${\|\hat{\balpha} - \bar{\balpha}\| \,\leq\,}$
\begin{eqnarray*}
\mathcal{O}\left(\frac{1}{\sigma_{\min}(\bSigma)^2}\left(Lk\sqrt{\frac{L\log(Lk/\delta)}{n^\tr}}
\,+\,
\sigma_{\max}(\bSigma)\sqrt{\frac{Lk\log(Lk/\delta)}{n^\val}}\right) \,+\, \frac{\nu\sqrt{Lk}}{\sigma_{\min}(\bSigma)}\right), \numberthis 
\end{eqnarray*}
where $\sigma_{\min}(\bSigma)$ and $\sigma_{\min}(\bSigma)$
are respectively the smallest and largest singular values of $\bSigma$.
\etheorem
\begin{proof}
The proof follows the same steps as Theorem \ref{thm:alpha-diagonal-linear-conopt}, except for the bound on  $\|\hat{\balpha} - {\balpha}\|$. Specifically, we have from \eqref{eq:lhs-penultimate}:
\begin{equation}
{\|\hat{\balpha} - \bar{\balpha}\|}
\,\leq\, {\|\hat{\balpha} - {\balpha}\|} \,+\, \nu\sqrt{Lk}\|\bSigma^{-1}\|.
\label{eq:lhs-fixed}
\end{equation}
We next bound: ${\|\hat{\balpha} - {\balpha}\|}$
\begin{eqnarray*}
&\leq&
\|{\balpha}\|\|\bSigma\|\|\bSigma^{-1}\|\left(
\frac{\|\bSigma - \hat{\bSigma}\|}{\|\bSigma\|}
\,+\,
\frac{\|\bcE - \hat{\bcE}\|}{\|\bcE\|}\right)\\
&\leq&
\|\bSigma^{-1}\|^2\|\bcE\|\left(
\|\bSigma - \hat{\bSigma}\|
\,+\,
\|\bSigma\|\frac{\|\bcE - \hat{\bcE}\|}{\|\bcE\|}\right)
~~(\text{from  $\balpha = \bSigma^{-1}\bcE$})\\
&\leq&
\|\bSigma^{-1}\|^2\left(
\|\bcE\|\|\bSigma - \hat{\bSigma}\|
\,+\,
\|\bSigma\|{\|\bcE - \hat{\bcE}\|}\right)\\
&\leq&
\|\bSigma^{-1}\|^2\left(\sqrt{Lk}
\|\bSigma - \hat{\bSigma}\|
\,+\,
\|\bSigma\|{\|\bcE - \hat{\bcE}\|}\right)
~~(\text{as  $\perf^\Dtrue[h] \in [0,1]$})\\
&\leq&
\mathcal{O}\left(\frac{1}{\sigma_{\min}(\bSigma)^2}\left(
\sqrt{Lk}\sqrt{\frac{L^2k\log(Lk/\delta)}{n^\tr}}
\,+\,
\sigma_{\max}(\bSigma)\sqrt{\frac{Lk\log(Lk/\delta)}{n^\val}}\right)\right), \numberthis 
\end{eqnarray*}
where the last step follows from 
 an adaptation of Lemma \ref{lem:Sigma-diff} (where $\H$ contains the $Lk$ fixed classifiers in \eqref{eq:trivial-classifiers}) and from Lemma \ref{lem:Perf-diff}. The last statement holds with probability at least $1-\delta$ over draws of $S^\tr$ and $S^\val$. Substituting this bound back in \eqref{eq:lhs-fixed} completes the proof.
\end{proof}

We next provide a bound on the singular values of $\bSigma$ for a specialized setting where the the probing classifiers $h^{\ell,i}$ are set to \eqref{eq:trivial-classifiers},  the basis functions $\phi^\ell$'s divide the data into disjoint clusters, and the base classifier $\bar{h}$ is close to having ``uniform accuracies'' across all the clusters and classes.
\blemma
Let $h^{\ell,i}$'s be defined as in \eqref{eq:trivial-classifiers}. Suppose for any $x$, $\phi^\ell(x) \in \{0,1\}$ and $\phi^\ell(x)\phi^{\ell'}(x) = 0, \forall \ell \ne \ell'$. Let $p_{\ell,i} = \E_{(x,y)\sim \mu}[\phi^\ell(x)\1(y=i)]$. 
Let $\bar{h}$ be such that
$\kappa - \tau \leq \Phi^{\Dshift,\ell}_i[\bar{h}] \leq \kappa, \forall \ell, i$ and for some $\kappa < \frac{1}{k}$ and $\tau < \kappa$. Then:
\vspace{-2pt}
\bequation
\sigma_{\max}(\bSigma) \,\leq\, L\max_{\ell,i}\,p_{\ell,i} \,+\, \Delta;~~~~~
\sigma_{\min}(\bSigma) \,\geq\, \epsilon(1-k\kappa)\min_{\ell,i}\,p_{\ell,i} \,-\, \Delta,
\eequation
\vskip -0.2cm
where $\displaystyle \Delta = Lk\tau\max_{\ell,i}p_{\ell,i}$.
\elemma
\vskip -0.2cm
\begin{proof}
We first write the matrix $\bSigma$ as $\bSigma = \bar{\bSigma} + \bE$, where
{\small
\bequation
\bar{\bSigma} = 
\begin{bmatrix}
p_{1,1}\left(\epsilon + (1-\epsilon)\kappa\right) &
p_{1,2}(1-\epsilon)\kappa & \ldots & p_{1,k}(1-\epsilon)\kappa &
p_{2,1}\kappa & \ldots & p_{L,k}\kappa\\
p_{1,1}(1-\epsilon)\kappa &
p_{1,2}\left(\epsilon + (1-\epsilon)\kappa\right) & \ldots & p_{1,k}(1-\epsilon)\kappa &
p_{2,1}\kappa & \ldots & p_{L,k}\kappa\\
&&&\vdots\\
p_{1,1}\kappa & p_{1,2}\kappa & \ldots & p_{1,k}\kappa & p_{2,1}\kappa & \ldots &p_{L,k}\left(\epsilon + (1-\epsilon)\kappa\right)
\end{bmatrix},
\eequation
}
and 
$\bE \in \R^{Lk\times Lk}$ with each
$\displaystyle|E_{(\ell,i), (\ell',i')}| \leq \max_{\ell,i}p_{\ell,i}\left(\kappa - \Phi^{\Dshift,\ell}_i[\bar{h}]\right) \leq \tau\max_{\ell,i}p_{\ell,i}$.

The matrix $\bar{\bSigma}$ can in turn be written as a product of a \textit{symmetric} matrix $\A\in \R^{Lk\times Lk}$ and a \textit{diagonal} matrix $\D\in \R^{Lk\times Lk}$:
\bequation
\bar{\bSigma} = \A\D,
\eequation
where
\vspace{-8pt}
\begin{align*}
\A &= 
\begin{bmatrix}
\epsilon + (1-\epsilon)\kappa &
(1-\epsilon)\kappa & \ldots &  (1-\epsilon)\kappa &
\kappa & \ldots & \kappa\\
(1-\epsilon)\kappa &
\epsilon + (1-\epsilon)\kappa & \ldots & (1-\epsilon)\kappa &
\kappa & \ldots & \kappa\\
&&&\vdots\\
(1-\epsilon)\kappa &
(1-\epsilon)\kappa & \ldots & \epsilon + (1-\epsilon)\kappa &
\kappa & \ldots & \kappa\\
&&&\vdots\\
\kappa & \kappa & \ldots & \kappa & \epsilon + (1-\epsilon)\kappa & \ldots &(1-\epsilon)\kappa\\
&&&\vdots\\
\kappa & \kappa & \ldots & \kappa & (1-\epsilon)\kappa & \ldots &\epsilon + (1-\epsilon)\kappa
\end{bmatrix} \\
\D &= \diag(p_{1,1},\ldots,p_{L,k}). \numberthis 
\end{align*}

We can then bound the largest and smallest singular values of $\bSigma$ in terms of those of $\A$ and $\D$. Using Weyl's inequality (see e.g.,\ \cite{stewart1998perturbation}), we have
\begin{equation}
\sigma_{\max}(\bSigma) \leq \sigma_{\max}(\bar{\bSigma})  + \|\bE\| \leq
\|\A\|\|\D\| \,+\, \|\bE\|
= \sigma_{\max}(\A)\sigma_{\max}(\D)  + \|\bE\|.
\label{eq:sigma-max-}
\end{equation}
and
\begin{equation}
\sigma_{\min}(\bSigma) \geq
\sigma_{\min}(\bar{\bSigma}) \,-\, \|\bE\|
=
\frac{1}{\|\bar{\bSigma}^{-1}\|} \,-\, \|\bE\|\geq \frac{1}{\|\A^{-1}\|\|\D^{-1}\|} \,-\, \|\bE\|= \sigma_{\min}(\A)\sigma_{\min}(\D)\,-\, \|\bE\|.
\label{eq:sigma-min-}
\end{equation}
Further, we have $\|\bE\| \leq \|\bE\|_F \leq  \displaystyle Lk\tau\max_{\ell,i}p_{\ell,i} = \Delta$, giving us:
\begin{equation}
\sigma_{\max}(\bSigma) \leq \sigma_{\max}(\A)\sigma_{\max}(\D)  + \Delta.
\vspace{-10pt}
\label{eq:sigma-max}
\end{equation}
\begin{equation}
\sigma_{\min}(\bSigma) \geq  \sigma_{\min}(\A)\sigma_{\min}(\D) - \Delta.
\label{eq:sigma-min}
\end{equation}
All that remains is to bound the singular values of $\bSigma$ and $\D$. Since $\D$ is a diagonal matrix, it's singular values are given by its diagonal entries:
\bequation
\sigma_{\max}(\D) = \max_{\ell,i}\,p_{\ell,i};~~~~\sigma_{\min}(\D) = \min_{\ell,i}\,p_{\ell,i}.
\eequation
The matrix $\A$ is symmetric and has a certain block structure. It's singular values are the same as the positive magnitudes of its Eigen values. We first write out it's $Lk$ Eigen vectors:

\bequation
\begin{matrix}
\x^{1,1} &= [&
\overbrace{1, -1, 0, \ldots, 0}^{k~\text{entries}},& \overbrace{ 0, \ldots, 0}^{k~\text{entries}}, &\ldots &\overbrace{ 0, \ldots, 0}^{k~\text{entries}}&]\\[5pt]
\x^{1,2} &= [&
{1, 0, -1, \ldots, 0}, &{ 0, \ldots, 0}, &\ldots & 0, \ldots, 0&]\\
&&&\vdots\\[5pt]
\x^{1,k-1} &= [&
{1, 0, 0, \ldots, -1},& { 0, \ldots, 0},& \ldots &{ 0, \ldots, 0}&]\\[5pt]
\x^{1,k} &= [&
{1, \ldots, 1},& { -1, \ldots, -1},& \ldots &{ 0, \ldots, 0}&]\\[5pt]
\x^{2,1} &= [&
{0, \ldots, 0},& { 1,-1,0, \ldots, 0},& \ldots& { 0, \ldots, 0}&]\\
&&&\vdots\\[5pt]
\x^{2,k-1} &= [&
 { 0, \ldots, 0},& {1, 0, 0, \ldots, -1},&\ldots &{ 0, \ldots, 0}&]\\[5pt]
\x^{2,k} &= [&
{-1, \ldots, -1},& { 1, \ldots, 1},& \ldots &{ 0, \ldots, 0}&]\\[5pt]
&&&\vdots\\[5pt]
\x^{L,1} &= [&
{0, \ldots, 0},& { 0, \ldots, 0},& \ldots& { 1,-1,0, \ldots, 0}&]\\[5pt]
&&&\vdots\\[5pt]
\x^{L,k-1} &= [&
{0, \ldots, 0},& { 0, \ldots, 0},& \ldots& { 1,0,0, \ldots, -1}&]\\[5pt]
\x^{L,k} &= [&
{1, \ldots, 1},& { 1, \ldots, 1},& \ldots& { 1, \ldots, 1}&]\\[5pt]
\end{matrix}
\eequation
One can then verify that the $Lk$  Eigen values of $\A$ are $\epsilon$ with a multiplicity of $(L-1)k$, $\epsilon(1-k\kappa)$ with a multiplicity of $k-1$ and $(L-\epsilon)k\kappa + \epsilon$ with a multiplicity of 1. Therefore: 
\bequation
\sigma_{\max}(\A) \leq L;~~~~\sigma_{\min}(\A) = \epsilon(1-k\kappa).
\eequation
Substituting the singular (Eigen) values of $\A, \D$ into \eqref{eq:sigma-max} and \eqref{eq:sigma-min} completes the proof. 
\end{proof}
In the above lemma, the base classifier $\bar{h}$ is assumed to have roughly uniformly low accuracies for all classes and clusters, and the closer it is to having uniform accuracies, i.e.\ the smaller the value of $\tau$, the tighter are the bounds.

We have shown a bound on the singular values of $\bSigma$ for a specific setting where the basis functions $\phi^\ell$'s divide the data into disjoint clusters. When this is not the case (e.g.\ with overlapping clusters \eqref{eq:hardlcuster}, or soft clusters \eqref{eq:softlcuster}), the singular values of $\bSigma$ would depend on how correlated the basis functions are.

\vspace{-0.4cm}
\section{Error Bound for FW-EG with Unknown $\psi$}
\label{app:complex-unknown}
In this section, we provide an error bound for Algorithm \ref{algo:FW-modified} for evaluation metrics of the form $\perf^\Dtrue[h] = \psi(C^\Dtrue_{11}[h],\ldots, C^\Dtrue_{kk}[h]),$ for a smooth, but \textit{unknown} $\psi: \R^k \> \R_+$. In this case, we do not have a closed-form expression for the gradient of $\psi$, but instead apply the example weight elicitation routine in Algorithm \ref{algo:weight-coeff} using probing classifiers chosen from within a small neighborhood around the current iterate $h^t$, where  $\psi$ is effectively linear. Specifically, we invoke Algorithm \ref{algo:weight-coeff} with the current iterate $h^t$ as the base classifier and with the radius parameter $\epsilon$ set to a small value. In the error bound that we state below for this version of the algorithm, we explicitly take into account the ``slack'' in using a local approximation to $\psi$ as a proxy for its gradient.

\btheorem[\textbf{Error Bound for Frank Wolfe with Elicited Gradients with unknown $\psi$}]
\label{thm:fw-eg-error-bound-unknown-psi}

Let $\perf^\Dtrue[h] = \psi(C^\Dtrue_{11}[h],\ldots, C^\Dtrue_{kk}[h])$ for an \emph{unknown} concave  $\psi: [0,1]^k \>\R_+$, which is $Q$-Lipschitz, and also $\lambda$-smooth w.r.t.\ the $\ell_1$-norm.  Let $\hat{\perf}^\val[h]=\psi(\hat{C}^\val_{11}[h],\ldots, \hat{C}^\val_{kk}[h])$. Fix $\delta \in (0, 1)$. 
Suppose Assumption \ref{asp:alpha-star} holds with slack $\nu$. Suppose for any linear metric $\sum_i\beta_i C^\Dtrue_{ii}[h]$, whose associated weight coefficients in the assumption is $\bar{\balpha}$  with $\|\bar{\balpha}\| \leq B$, the following holds.
For any $\delta \in (0,1)$, with probability $\geq 1-\delta$ over draw of $S^\tr$ and $S^\val$, 
 when the weight elicitation routine in Algorithm \ref{algo:weight-coeff} is given an 
 input metric $\hat{\perf}^\val$ with $|\hat{\perf}^\val - \sum_i\beta_i \hat{C}^\val_{ii}[h]| \leq \chi, \forall h$, it outputs coefficients $\hat{\balpha}$ such that
 $\|\hat{\balpha} -\bar{\balpha}\| \leq 
 \kappa(\delta, n^\tr, n^\val, \chi)
 $, for some function $\kappa(\cdot) > 0$. Let $B' = B + \sqrt{Lk}\,\kappa(\delta, n^\tr, n^\val, 2\lambda\epsilon^2)$. 
Assume $k \leq n^\val$.
 
 \noindent Then w.p.\  $\geq 1 - \delta$ over draws of $S^\tr$ and $S^\val$ from $\Dtrue$ and $\Dshift$ respectively, the classifier $\hat{h}$ output by Algorithm \ref{algo:FW-modified} with radius parameter $\epsilon$ after $T$ iterations satisfies:
 
\begin{eqnarray*}
\lefteqn{
\max_{h}\perf^D[h] - \perf^D[\hat{h}]\,\leq\,
2QB'\E_x\left[\|\eta^\tr(x)-
\hat{\eta}^\tr(x)\|_1\right] +
 4Q\sqrt{Lk}\,\kappa(\delta/T, n^\tr, n^\val, 2\lambda\epsilon^2)
 }
 \\
 &&
 \hspace{3.5cm}
\,+\, 4Q\nu \,+\, 
\mathcal{O}\left(\lambda k\sqrt{\frac{k\log(n^\val)\log(k) + \log(k/\delta)}{n^\val}} + \frac{\lambda}{T}\right). \numberthis 
\end{eqnarray*}
\etheorem

One can plug-in $\kappa(\cdot)$ with e.g. the error bound we derived for Algorithm \ref{algo:weight-coeff} in Theorem \ref{thm:alpha-diagonal-linear-conopt}, suitably modified to accommodate input metrics $\hat{\perf}^\val$ that may differ from the desired linear metric by at most $\chi$. Such modifications can be easily made to Theorem \ref{thm:alpha-diagonal-linear-conopt} and would result in an  additional term $\sqrt{Lk}\chi$  in the error bound to take into account the additional approximation errors in computing the right-hand side of the linear system in \eqref{eq:perf-emp}.

Before proceeding to prove Theorem \ref{thm:fw-eg-error-bound-unknown-psi}, we state a few useful lemmas. The following  lemma shows that because $\psi(\C)$ is $\lambda$-smooth, it is effectively linear within a small neighborhood around $\C$.
\blemma
\label{lem:eps-linear-approx}
Suppose $\psi$ is $\lambda$-smooth w.r.t.\ the $\ell_1$-norm. 
For each iteration $t$ of Algorithm \ref{algo:FW-modified}, let ${\bupsilon}^t = \nabla\psi(\c^t)$ denote the true gradient of $\psi$ at $\c^t$.
Then for any classifier $h^\epsilon(x) = (1-\epsilon)h^t(x) + \epsilon h(x),$
\begin{eqnarray*}
    \left|\hat{\perf}^\val[h^\epsilon] - \hat{\perf}^\val[h^t] - \sum_{i}\upsilon^t_i \hat{C}^\val_{ii}[h^\epsilon]\right|
    &\leq& 2\lambda\epsilon^2. \numberthis
\end{eqnarray*}
\elemma
\begin{proof}
For any randomized classifier $h^\epsilon(x) = (1-\epsilon)h^t(x) + \epsilon h(x),$
\allowdisplaybreaks
\begin{eqnarray*}
    \left|\hat{\perf}^\val[h^\epsilon] - \hat{\perf}^\val[h^t] - \sum_{i}\upsilon^t_i \hat{C}^\val_{ii}[h^\epsilon]\right| &=&
    \left|\psi(\diag(\hat{\C}^\val[h^\epsilon])) - \psi(\diag(\hat{\C}^\val[h^t])) -  \sum_{i}\upsilon^t_i \hat{C}^\val_{ii}[h^\epsilon]\right|\\
    &\leq& \frac{\lambda}{2}\|\diag(\hat{\C}^\val[h^\epsilon]) \,-\, \diag(\hat{\C}^\val[h^t])\|^2_1\\
    &=& \frac{\lambda}{2}\|\epsilon(\diag(\hat{\C}^\val[h]) \,-\, \diag(\hat{\C}^\val[h^t]))\|_1^2\\
    &=&  \frac{\lambda}{2}\epsilon^2\|\diag(\hat{\C}^\val[h]) \,-\, \diag(\hat{\C}^\val[h^t])\|_1^2\\
     &\leq&  \frac{\lambda}{2}\epsilon^2\left(\|\diag(\hat{\C}^\val[h])\|_1 \,+\, \|\diag(\hat{\C}^\val[h^t])\|_1\right)^2\\
    &\leq& \frac{\lambda}{2}\epsilon^2(2)^2 ~=~2\lambda\epsilon^2. \numberthis 
\end{eqnarray*}
Here the second line follows from the fact that $\psi$ is $\lambda$-smooth w.r.t.\ the $\ell_1$-norm,
and $\bupsilon^t = \nabla\psi(\diag(\hat{\C}^\val[h^t]))$. The third line follows from linearity of expectations. The last line follows from the fact that the sum of the entries of a confusion matrix  (and hence the sum of its diagonal entries) cannot exceed 1.
\end{proof}

We next restate the error bounds for the call to \textbf{PI-EW} in line 9 and the corresponding bound on the approximation error in the linear maximizer $\hat{f}$ obtained.
\blemma[Error bound for call to \textbf{PI-EW}  in line 9 with unknown $\psi$]
\label{lem:plugin-linear-unknown-psi}
For each iteration $t$ of Algorithm \ref{algo:FW}, let ${\bupsilon}^t = \nabla\psi(\c^t)$ denote the true gradient of $\psi$ at $\c^t$,  when the algorithm is run with an unknown $\psi$ that is $Q$-Lipschitz and $\lambda$-smooth w.r.t.\ the $\ell_1$-norm.
Let
 $\bar{\balpha}$ be the associated weighting coefficient for the linear metric $\sum_i {\upsilon}^t_i C^D_{ii}[h]$ (whose coefficients are unknown) in Assumption \ref{asp:alpha-star}, with $\|\bar{\balpha}\|_1 \leq B$, and with slack $\nu$. Fix $\delta>0$.
Suppose w.p. $\geq 1-\delta$ over draw of $S^\tr$ and $S^\val$, when the weight elicitation routine used in \emph{\textbf{PI-EW}} is called with the input metric $\hat{\perf}^\val[h] - \hat{\perf}^\val[h^t]$  with $|\hat{\perf}^\val[h] - \sum_i\upsilon_i \hat{C}^\val_{ii}[h]| \leq \chi, \forall h$, it
outputs coefficients $\hat{\balpha}$ such that
 $\|\hat{\balpha} -\bar{\balpha}\| \leq 
 \kappa(\delta, n^\tr, n^\val, \chi)
 $, for some function $\kappa(\cdot) > 0$.  Let $B' = B + \sqrt{Lk}\,\kappa(\delta, n^\tr, n^\val, 2\lambda\epsilon^2)$. 
 Then with the same probability, the classifier $\hat{h}$ output by \emph{\textbf{PI-EW}}  when called 
 by Algorithm \ref{algo:FW-modified}
 with metric $\hat{\perf}^\lin[h]= \hat{\perf}^\val[h] - \hat{\perf}^\val[h^t]$
 and radius $\epsilon$ satisfies:
\begin{align*}
{
\max_{h}\sum_i {\upsilon}^t_i C^D_{ii}[h] - \sum_i {\upsilon}^t_i C^D_{ii}[\hat{h}]} &\leq
Q\left(B'\E_x\left[\|\eta^\tr(x)-
\hat{\eta}^\tr(x)\|_1\right] \right.\\
& \left. \hspace{1cm} \,+\, 2\sqrt{Lk}\,\kappa(\delta, n^\tr, n^\val, 2\lambda\epsilon^2)
 \,+\, 2\nu 
 \right),
 \numberthis 
\end{align*}
where $\eta_i^\tr(x) = \P^\mu(y=i|x)$.
\elemma
\begin{proof}
The proof is the same as that of Lemma \ref{lem:plugin-linear} for the ``known $\psi$'' case, except that the  $\kappa(\cdot)$ guarantee for the call to weight elicitation routine in line 2 is different, and takes into account the fact that the input metric $\hat{\perf}^\val[h] - \hat{\perf}^\val[h^t]$ to the weight elicitation routine  is only a local approximation to the (unknown) linear metric $\sum_i\upsilon_i \hat{C}^\val_{ii}[h]$. We use Lemma \ref{lem:eps-linear-approx} to compute the value of slack $\chi$ in $\kappa(\cdot)$.
\end{proof}
\blemma[Approximation error in linear maximizer $\hat{f}$ in line 9 with unknown $\psi$]
\label{lem:lmo-unknown}
For each iteration $t$ in Algorithm \ref{algo:FW-modified}, let $\bar{\c}^t = \diag(\C^D[h^t])$ and let $\bar{\bbeta}^t = \nabla\psi(\bar{\c}^t)$ denote the unknown gradient of $\psi$ evaluated at $\bar{\c}^t$.
Suppose the assumptions in Theorem \ref{thm:fw-eg-error-bound-unknown-psi} hold. 
Let $B' = B + \sqrt{Lk}\,\kappa(\delta, n^\tr, n^\val, 2\lambda\epsilon^2)$. 
Assume $k \leq n^\val$. 

\noindent Then w.p.\ $\geq 1 - \delta$ over draw of $S^\tr$ and $S^\val$ from $\Dshift$ and $\Dtrue$ resp., for any $t = 1, \ldots, T$, the classifier $\hat{f}$ returned by \emph{\textbf{PI-EW}} in line 9 satisfies:
\begin{eqnarray*}
\lefteqn{
\max_{h}\sum_{i}\bar{\beta}^t_{ii}C^\Dtrue_i[h] \,-\, 
\sum_{i}\bar{\beta}^t_iC^\Dtrue_{ii}[\hat{f}] \,\leq\,
QB'\E_x\left[\|\eta^\tr(x)-
\hat{\eta}^\tr(x)\|_1\right] 
\,+\, 2Q\nu 
}\\
 &&
 \hspace{1cm}
 +\,
 2Q\sqrt{Lk}\,\kappa\left(\textstyle\frac{\delta}{T}, n^\tr, n^\val, 2\lambda \epsilon^2\right)
 \,+\, 
 \mathcal{O}\left(\lambda k\sqrt{\frac{k\log(k)\log(n^\val) + \log(k/\delta)}{n^\val}}\right). \numberthis 
\end{eqnarray*}
\elemma
\begin{proof}
The proof is the same as that of Lemma \ref{lem:lmo} for the ``known $\psi$'' case, with the only difference being that we use Lemma \ref{lem:plugin-linear-unknown-psi} (instead of Lemma \ref{lem:plugin-linear}) to bound the linear maximization errors in equation \eqref{eq:approx-lin-max-last}.
\end{proof}

\begin{proof}[Proof of Theorem \ref{thm:fw-eg-error-bound-unknown-psi}]
The proof follows from plugging Lemma \ref{lem:lmo-unknown} into the Frank-Wolfe convergence guarantee in Lemma \ref{lem:fw} stated in Appendix \ref{app:proof-FW}.
\end{proof}

\section{Running Time of Algorithm \ref{algo:FW}}
\label{app:run-time}
We discuss how one iteration of FW-EG (Algorithm~\ref{algo:FW}) compares with  one iteration (epoch) of training a class-conditional probability estimate $\hat\eta^{\tr}(x) \approx \P^\Dshift(y=1|x)$. In each iteration of FW-EG, we create $Lk$ probing classifiers, where each probing classifier via~\eqref{eq:trivial-classifiers} only  requires perturbing the predictions of the base classifier $\bar h = h^t$ and hence requires $n^\tr + n^\val$ computations. After constructing the $Lk$ probing classifiers, FW-EG solves a system of linear equations with $Lk$ unknowns, where a na\"ive matrix inversion approach requires $O((Lk)^3)$ time. Notice that this can be further  improved with efficient methods, e.g., using state-of-the-art linear regression solvers. Then FW-EG creates a plugin classifier and combines the predictions with the Frank-Wolfe style updates, requiring $Lk(n^\tr + n^\val)$ computations. So, the overall time complexity for each iteration of 
FW-EG is $O\left(Lk(n^\tr + n^\val) + (Lk)^3\right)$. On the other hand, one iteration (epoch) of training $\hat\eta^{\tr}(x)$ requires $O(n^\tr Hk)$ time, where $H$ represents the 
total number of parameters in the underlying model architecture up to the penultimate layer.
For deep networks such as ResNets (Sections~\ref{ssec:cifar10} and~\ref{ssec:adience}), clearly, the run-time is dominated by the training of $\hat\eta^{\tr}(x)$, as long as $L$ and $k$ are relatively small compared to the number of parameters in the neural network. Thus our approach is reasonably faster than having to train the model for $\hat\eta^\tr$ in each iteration~\cite{jiang2020optimizing}, training the model (such as ResNets) twice~\cite{patrini2017making}, or making multiple forward/backward passes on the training and validation set requiring three times the time for each epoch compared to training $\hat\eta^\tr$~\cite{ren2018learning}. 


\section{Plug-In with Coordinate-Wise Search Baseline}

\label{ssec:multiclass-plugin}
We describe the Plug-in [train-val] baseline used in Section \ref{sec:experiments}, which constructs a classifier $\widehat{h}(x) \,\in\, \argmax_{i \in [k]} w_i\hat{\eta}^\val_i(x)$, by tuning the weights $w_i \in \R$ to maximize the given metric on the validation set .
Note that there are $k$ parameters to be tuned, and a na\"{i}ve approach would be to use an $k$-dimensional grid search. Instead, we use a trick from \cite{hiranandani2019multiclass} to decompose this search into  an independent coordinate-wise search for each $w_i$. Specifically, one can estimate the relative weighting  $w_i/w_j$ between any pair of classes $i,j$ by constructing a classifier of the form
\bequation
h^{\zeta}(x) = \begin{cases}
i & \text{if}~~~\zeta \hat{\eta}^\tr_i(x) > (1-\zeta) \hat{\eta}^\tr_j(x)\\
j & \text{otherwise}
\end{cases},
\eequation
that  predicts either class $i$ or  $j$
based on which of these receives a higher (weighted) probability estimates, 
and (through a line search) finding the  parameter  $\zeta \in (0,1)$ for which $h^{\zeta}$ yields the highest validation metric:
\bequation
w_i/w_j \approx \argmax_{\zeta \in [0,1]} \hat{\perf}^\val[h^{\zeta}].
\eequation
By fixing $i$ to class $k$, and repeating this for classes $j \in [k-1]$, one can estimate  $w_j/w_k$ for each $j \in [k-1]$, and normalize the estimated related weights  to get estimates for $w_1, \ldots, w_k$. 

\section{Solving Constrained Satisfaction Problem in \eqref{eq:con-opt}}
\label{app:con-opt}

We describe some common special cases where one can easily identify classifiers $h^{\ell,i}$'s which satisfy the constraints in \eqref{eq:con-opt}.
We will make use of a pre-trained class probability model $\hat{\eta}_i^\tr(x) \approx \P^\Dshift(y=i|x)$, also used in Section \ref{sec:algorithms} to construct the plug-in classifier in Algorithm \ref{algo:linear-metrics}. The hypothesis class $\H$ we consider is the set of all plug-in classifiers obtained by post-shifting $\hat{\eta}^\tr$.

We start with a binary classification problem ($k = 2$) with basis functions $\phi^\ell(x) = \1(g(x) =\ell)$, which divide the data points into $L$ \textit{disjoint} groups according to $g(x) \in [L]$. For this setting, one can show under mild assumptions on the data distribution that \eqref{eq:con-opt} does indeed have a feasible solution (using e.g. the geometric techniques used by \cite{hiranandani2018eliciting} and also elaborated in the figure above).
One such feasible $h^{\ell, i}$  predicts class $i \in \{0,1\}$ on all example belonging to group $\ell$, and uses a thresholded of $\hat{\eta}^\tr$ for examples from other groups, with per-cluster thresholds. 
This would have the effect of maximizing the diagonal entry $\hat{\Phi}^{\tr,\ell}_i[h^{\ell, i}]$ of $\hat{\bSigma}$ and the thresholds can be tuned so that the off-diagonal entries $\hat{\Phi}^{\tr,\ell'}_{i'}[h^{\ell, i}], \forall (\ell',i') \ne (\ell,i)$ are small. 

More specifically, for any $\ell \in [L], i \in \{0,1\}$, the classifier $h^{\ell,i}$ can be constructed as:

\begin{equation}
h^{\ell, i}(x) = \begin{cases}
i &\text{ if  $g(x) = \ell$}\\
\1(\hat{\eta}^\tr(x) \leq 
\tau_{g(x)}) &\text{ otherwise},
\end{cases}
\label{eq:binaryconopt}
\end{equation}

\noindent where the thresholds $\tau_{\ell'} \in [0,1], \ell' \ne \ell$ can each be tuned independently using a line search to minimize $\max_{i'}\hat{\Phi}^{\tr,\ell'}_{i'}[h^{\ell, i}]$. As long as $\hat{\eta}^\tr$ is a close approximation of $\P(y|x)$, the above procedure is guaranteed to find an approximately feasible solution for \eqref{eq:con-opt}, provided one exists. Indeed one can tune the values of $\gamma$ and $\omega$ in \eqref{eq:con-opt}, so that the above construction (with tuned thresholds) satisfies the constraints.

\begin{figure}[t]
\centering
\includegraphics[scale=1]{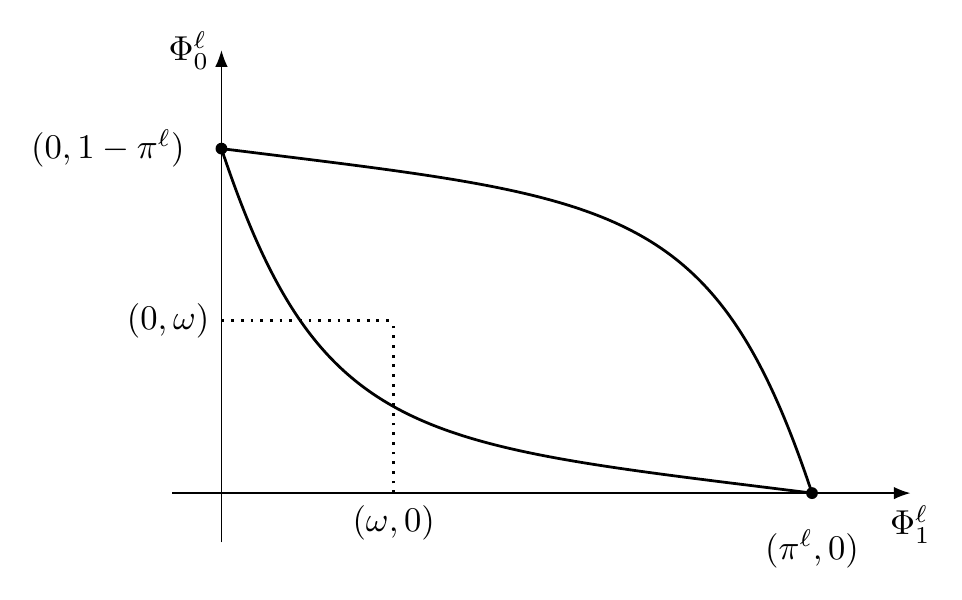}
\caption{Geometry of the space of $\Phi$-confusions~\cite{hiranandani2018eliciting} for $k=2$ classes and with basis functions $\phi^\ell(x) = \1(g(x) =\ell)$ which divide the data into $L$ disjoint clusters.
For a fixed cluster $\ell$, we plot the values of ${\Phi}^{\Dshift,\ell}_{0}[h]$ and ${\Phi}^{\Dshift,\ell}_{1}[h]$ for all randomized classifiers, with  $\pi^\ell = \P^{\Dshift}(y=1,g(x) = \ell)$. The points on the lower boundary correspond to  classifiers of the form $\1({\eta}^\tr(x) \leq 
\tau)$ for varying thresholds $\tau \in [0,1]$. The points on the lower boundary
within the dotted box correspond to the thresholded classifiers $h$ which yield both values $\Phi^{\Dshift,\ell}_0[h] \leq \omega$ and $\Phi^{\Dshift,\ell}_1[h] \leq \omega$. 
One can thus find a feasible probing classifier $h^{\ell,i}$ for the constrained optimization problem in~\eqref{eq:con-opt} using the construction from \eqref{eq:binaryconopt}
as long as 
$\pi^{\ell} \geq \gamma$ and $1 - \pi^{\ell} \geq \gamma$, 
and the lower boundary intersects with the dotted box for clusters $\ell'\neq\ell$. 
If the latter fails, one can increase $\omega$ slowly until the classifier given in~\eqref{eq:binaryconopt} is feasible for  \eqref{eq:con-opt}.
} 
\end{figure}

We next look a multiclass problem ($k > 2$) with basis functions $\phi^\ell(x) = \1(g(x) =\ell)$ which again divide the data points into $L$ \textit{disjoint} groups. 
Here again, one can show under mild assumptions on the data distribution that \eqref{eq:con-opt} does indeed have a feasible solution (using e.g. the geometric tools from \cite{hiranandani2019multiclass}). 
We can once again construct a feasible $h^{\ell, i}$ 
by predicting class $i \in [k]$ on all example belonging to group $\ell$, and using a post-shifted classifier for examples from other groups. In particular, for any $\ell \in [L], i \in [k]$, the classifier $h^{\ell,i}$ can be constructed as:

\begin{equation}
h^{\ell, i}(x) = 
\begin{cases}
i &\text{ if  $g(x) = \ell$}\\
\argmax_{j\in [k]} w^{g(x)}_j\hat{\eta}_j^\tr(x) &\text{ otherwise}
\end{cases},
\label{eq:h-ell-i-con-opt}
\end{equation}
where we use $k$ parameters $w^{\ell'}_1, \ldots, w^{\ell'}_k$ for each cluster $\ell' \ne \ell$. 
We can then tune these $k$ parameters  to minimize the maximum of the off-diagonal entries of $\hat{\bSigma}$, i.e.\ {minimize} $\max_{i'}\hat{\Phi}^{\tr,\ell'}_{i'}[h^{\ell, i}]$. However, this may require an $k$-dimensional grid search. Fortunately, as described in Appendix \ref{ssec:multiclass-plugin}, we can use a trick from \cite{hiranandani2019multiclass} 
to reduce the problem of tuning $k$ parameters into $k$ independent line searches. 
This is based on the idea that the optimal relative weighting $w^{\ell'}_i/w^{\ell'}_j$ between any pair of classes can be determined through a line search. In our case, we will fix $w^{\ell'}_k = 1, \forall \ell' \ne \ell$ and compute  $w^{\ell'}_i, i = 1, \ldots, k-1$ by solving the following one-dimensional optimization problem to determine the relative weighting $w^{\ell'}_i / w^{\ell'}_k = w^{\ell'}_i$.
\bequation
w^{\ell'}_i \in \underset{\zeta \in [0,1]}{\argmin}\left( \max_{i'}\hat{\Phi}^{\tr,\ell'}_{i'}[h^{\zeta}]\right),
~~
\text{where}
~~
h^{\zeta}(x) = \begin{cases}
i & \text{if}~~~\zeta \hat{\eta}^\tr_i(x) < (1-\zeta) \hat{\eta}^\tr_k(x)\\
k & \text{otherwise}
\end{cases}.
\eequation
We can repeat this for each cluster $\ell' \ne \ell$ to construct the $(\ell,i)$-th probing classifier $h^{\ell,i}$ in \eqref{eq:h-ell-i-con-opt}.

For the more general setting, where the basis functions $\phi^\ell$'s cluster the data into overlapping or soft clusters (such as in \eqref{eq:softlcuster}),  one can find feasible classifiers for \eqref{eq:con-opt} by posing this problem as a ``rate'' constrained optimization problem of the form below to pick $h^{\ell,i}$:
\bequation
\max_{h\in \H}\hat{\Phi}^{\tr,\ell}_i[h]
~~\text{s.t.}~~
\hat{\Phi}^{\tr,\ell'}_{i'}[h] \leq \omega, \forall (\ell',i') \ne (\ell,i),
\eequation
which can be solved using off-the-shelf toolboxes such as the open-source library offered by \cite{cotter2019optimization}.\footnote{\url{https://github.com/google-research/tensorflow_constrained_optimization}}
Indeed one can tune the hyper-parameters $\gamma$ and $\omega$ so that the solution to the above problem is feasible for \eqref{eq:con-opt}.
If $\H$ is the set of plug-in classifiers obtained by post-shifting $\hat{\eta}^\tr$, then one can alternatively use the approach of \cite{narasimhan2018learning} to identify the optimal post-shift on $\hat{\eta}^\tr$ that solves the above constrained problem.

\section{Additional Experimental Details}
\label{app:exp}

Below we provide some more details regarding the experiments:  

\begin{table}[t]
    \small
    \centering
    \caption{Test macro F-measure for the maximization task in Section 6.2 of~\cite{jiang2020optimizing}.}
    \begin{tabular}{ccc}
    \hline
     $\downarrow$ Data, Method $\rightarrow$&
         Adaptive Surrogates~\cite{jiang2020optimizing} & FW-EG \\
        \hline
        COMPAS 
        & 0.629 & \textbf{0.652}
        \\
        Adult 
        & 0.665 & \textbf{0.670} 
        \\
        Default
        & 0.533 & {0.536}
        \\
        \hline
    \end{tabular}
    \label{tab:addedexp}
\end{table}

\begin{itemize}
    \item \textit{Maximizing Accuracy under Label Noise on CIFAR-10 (Section~\ref{ssec:cifar10}):} The metric that we aim to optimize  is test accuracy, which is a linear metric in the diagonal entries of the confusion matrix. Notice that we work with the \emph{asymmetric} label noise model from Patrini et al.~\cite{patrini2017making}, which corresponds to the setting where a label is flipped to a particular label with a certain probability. This involves a non-diagonal noise transition matrix $\T$, and consequently the corrected  training objective  is a linear function of the entire confusion matrix. Indeed, the loss correction approach from~\cite{patrini2017making} makes use of the estimate of the entire noise-transition matrix, including the off-diagonal entries. Whereas, our approach in the experiment elicits weights for the diagonal entries alone, but assigns a different set of weights for each basis function, i.e., cluster. We are thus able to achieve better performance than \cite{patrini2017making}  by optimizing 
    correcting for the noise using a linear function of per-cluster diagonal entries. 
    Indeed, we also observed that PI-EW often achieves better accuracy during cross-validation with ten basis functions,  highlighting the benefit of underlying modeling in PI-EW. We expect to get further improvements by incorporating off-diagonal entries in PI-EW optimization on the training side as explained in Appendix~\ref{app:linear-gen}. We also stress that 
    the results from our methods can  be further improved by cross-validating over kernel width, UMAP dimensions, and selection of the cluster centers, which are currently set to fixed values in our experiments. Lastly, we did not compare to the Adaptive Surrogates~\cite{jiang2020optimizing} for this experiment as this baseline requires to re-train the ResNet model in every iteration, and more importantly, this method constructs its probing classifiers by perturbing the parameters of the ResNet model several times in each iteration, which can be prohibitively expensive in practice.
    
    \item \textit{Maximizing G-mean with Proxy Labels on Adult (Section~\ref{ssec:adult}):} In this experiment, we use binary features as basis functions instead of RBF kernels as done in CIFAR-10 experiment.  This reflects the flexibility of the proposed PI-EW and FW-EG methods. Our approach can incorporate any indicator features as basis function as long as it reflects cluster memberships. Moreover, our choice of basis function was motivated from choices made in~\cite{jiang2020optimizing}. We expect to further improve our results by incorporating more binary features as basis functions. 
    
    \item \textit{Maximizing F-measure under Domain Shift on Adience (Section~\ref{ssec:adience}):} As mentioned in Section~\ref{ssec:adience}, for the basis functions, in addition to the default basis $\phi^{\text{def}}(x) = 1 \, \forall x$, we choose from subsets of six basis  functions $\phi^1,\ldots,\phi^6$ that are averages of the RBFs, centered at points  from the validation set corresponding to each one of the six age-gender combinations. 
    We choose these subsets using knowledge of the underlying image classification task.
    Specifically, besides the default basis function, we cross-validate over three subsets of basis functions. The first subset comprises two basis functions, where the basis functions are averages of the RBF kernels with cluster centers belonging to the two true class. The second subset comprises three basis functions, where the basis functions are averages of the RBF kernels with cluster centers belonging to the three age-buckets. The third subset comprises six basis functions, where the basis functions are averages of the RBF kernels with cluster centers belonging to the combination of true class $\times$ age-bucket. We expect to further improve our results by cross-validating over kernel width and selection of the cluster centers. Lastly, we did not compare to Adaptive Surrogates, as this experiment again requires training a deep neural network model, and perturbing or retraining the model in each iteration can be prohibitively expensive in practice.
    
    \item \textit{Maximizing Black-box Fairness Metric on Adult (Section~\ref{ssec:adultbb}):}  In this experiment, since we treat the metric as a black-box, we do not assume access to gradients and thus do not run the [$\psi$ known] variant of FW-EG. We only report the [$\psi$ unknown] variant of FW-EG with varied basis functions as shown in Table~\ref{tab:adultbb}. 
    
    \item In Table~\ref{tab:addedexp}, we replicate the ``Macro F-measure'' experiment (without noise) from Section 6.2 in~\cite{jiang2020optimizing} and report results of maximizing the macro F-measure on Adult, COMPAS and Default datasets. We see that our approach yields notable gains on two out of the three datasets in comparison to Adaptive Surrogates approach~\cite{jiang2020optimizing}. 
\end{itemize}
